\documentclass[twoside,openright,titlepage,numbers=noenddot,headinclude,%
               footinclude=true,cleardoublepage=empty,abstractoff,BCOR=5mm,%
               paper=a4,fontsize=11pt,english]{scrreprt}

\usepackage{amssymb}

\PassOptionsToPackage{eulerchapternumbers,listings,
				 pdfspacing,eulermath,
				 subfig,parts,dottedtoc}{classicthesis}

\usepackage{ifthen}
\newboolean{enable-backrefs} 
\setboolean{enable-backrefs}{true} 

\newcommand{\myTitle}{Importance measures derived from Random Forests\xspace}
\newcommand{\mySubtitle}{Characterisation and Extension\xspace}

\newcommand{\myName}{Antonio Sutera\xspace}

\newcommand{\myFaculty}{Faculty of Applied Sciences\xspace}

\newcommand{\myUni}{University of Liege\xspace}

\newcommand{\myVersion}{version 1.0\xspace}

\newcounter{dummy} 
\providecommand{\mLyX}{L\kern-.1667em\lower.25em\hbox{Y}\kern-.125emX\@}

\PassOptionsToPackage{T1}{fontenc} 
\usepackage{fontenc}

\PassOptionsToPackage{utf8}{inputenc}	
 \usepackage{inputenc}

 \usepackage{babel}

\PassOptionsToPackage{square,authoryear}{natbib}
 \usepackage{natbib}
 \usepackage{bibentry}
 \nobibliography*

\PassOptionsToPackage{fleqn}{amsmath}		
 \usepackage{amsmath}

\usepackage{lipsum}
\usepackage{textcomp} 
\usepackage{scrhack} 
\usepackage{xspace} 
\usepackage{mparhack} 
\PassOptionsToPackage{printonlyused,smaller}{acronym}
	\usepackage{acronym} 

\usepackage{tabularx} 
\usepackage{caption}
\usepackage{subfig}

\usepackage{listings}
\PassOptionsToPackage{hyperfootnotes=true,pdfpagelabels}{hyperref}
	\usepackage{hyperref}  
\PassOptionsToPackage{pdftex}{graphicx}
	\usepackage{graphicx}

\newcommand{\backrefnotcitedstring}{\relax}
\newcommand{\backrefcitedsinglestring}[1]{(Cited on page~#1.)}
\newcommand{\backrefcitedmultistring}[1]{(Cited on pages~#1.)}
\ifthenelse{\boolean{enable-backrefs}}%
{%
		\PassOptionsToPackage{hyperpageref}{backref}
		\usepackage{backref} 
		   \renewcommand*{\backref}[1]{}  
		   \renewcommand*{\backrefalt}[4]{
		      \ifcase #1 %
		         \backrefnotcitedstring%
		      \or%
		         \backrefcitedsinglestring{#2}%
		      \else%
		         \backrefcitedmultistring{#2}%
		      \fi}%
}{\relax}

\hypersetup{%
    colorlinks=true, linktocpage=true, pdfstartpage=1, pdfstartview=FitV,%
    breaklinks=true, pdfpagemode=UseNone, pageanchor=true, pdfpagemode=UseOutlines,%
    plainpages=false, bookmarksnumbered, bookmarksopen=true, bookmarksopenlevel=1,%
    hypertexnames=true, pdfhighlight=/O,
    urlcolor=webbrown, linkcolor=RoyalBlue, citecolor=webgreen, 
    pdftitle={\myTitle},%
    pdfauthor={\textcopyright\ \myName, \myUni, \myFaculty},%
    pdfsubject={},%
    pdfkeywords={},%
    pdfcreator={pdfLaTeX},%
    pdfproducer={LaTeX with hyperref and classicthesis}%
}

\makeatletter
\@ifpackageloaded{babel}%
    {%
       \addto\extrasamerican{%
				}%
       \addto\extrasngerman{%
				}%
    }{\relax}
\makeatother

\usepackage{classicthesis}

\usepackage{helvet}

\usepackage{amsthm}
\newtheorem{theorem}{Theorem}
\newtheorem*{theorem*}{Theorem}
\newtheorem{lemma}[theorem]{Lemma}
\newtheorem{proposition}[theorem]{Proposition}
\newtheorem{corollary}[theorem]{Corollary}
\newtheorem{definition}{Definition}
\newtheorem{property}{Property}

\newtheorem{example}{Example}

\newtheorem{algorithm}{Algorithm}

\numberwithin{theorem}{chapter}
\numberwithin{property}{chapter}
\numberwithin{example}{chapter}
\numberwithin{definition}{chapter}
\numberwithin{algorithm}{chapter}
\numberwithin{figure}{chapter}
\numberwithin{table}{chapter}

\DeclareMathOperator*{\argmin}{arg\,min}
\DeclareMathOperator*{\argmax}{arg\,max}

\numberwithin{equation}{chapter}
\allowdisplaybreaks
\usepackage{algorithm}
\usepackage{algorithmic}
\usepackage{enumerate}
\usepackage{ulem}
\renewcommand{\emph}{\textit}

\definecolor{rulecolor}{rgb}{0.80,0.80,0.80}
\definecolor{bgcolor}{rgb}{1.0,1.0,1.0}

\usepackage{colortbl}
\newcommand{\g}{\cellcolor{gray!20}}
\newcolumntype{L}[1]{>{\raggedright\let\newline\\\arraybackslash\hspace{0pt}}m{#1}}
\newcolumntype{C}[1]{>{\centering\let\newline\\\arraybackslash\hspace{0pt}}m{#1}}
\newcolumntype{R}[1]{>{\raggedleft\let\newline\\\arraybackslash\hspace{0pt}}m{#1}}

\makeatletter
\newcommand*{\indep}{%
  \mathbin{%
    \mathpalette{\@indep}{}%
  }%
}
\newcommand*{\nindep}{%
  \mathbin{
    \mathpalette{\@indep}{\not}
  }%
}
\newcommand*{\@indep}[2]{%
  \sbox0{$#1\perp\m@th$}
  \sbox2{$#1=$}
  \sbox4{$#1\vcenter{}$}
  \rlap{\copy0}
  \dimen@=\dimexpr\ht2-\ht4-.2pt\relax
  \kern\dimen@
  {#2}%
  \kern\dimen@
  \copy0 
} 
\makeatother
\usepackage{multicol}

\usepackage{booktabs}       
\usepackage[super]{nth}

\usepackage{rotating}

\usepackage{pifont}

\usepackage{fontawesome}
\newcommand{\checkedbox}{{\Large \faCheckSquareO \,}}
\newcommand{\tocheckbox}{{\Large \faPencilSquareO \,}}
\newcommand{\qmark}{{\large \faFastForward \;}}

\newcommand{\pmark}{{\huge \ding{46} \,}}

\usepackage{framed}

\usepackage{csquotes}

\usepackage[most]{tcolorbox}

\newtcolorbox{nutshell}[1]{colback=RoyalBlue!15!white, colframe=RoyalBlue!95!black,fonttitle=\bfseries, title=\qmark  #1}

\newtcolorbox{overview}{boxrule=0pt, frame hidden,interior hidden, borderline west={1pt}{0pt}{RoyalBlue!95!black}, colbacktitle=white, enhanced jigsaw, arc=0mm,outer arc=0pt, colframe=RoyalBlue!95!black, colback=white,fonttitle=\bfseries, coltitle=RoyalBlue,title=\tocheckbox Overview}

\newtcolorbox{summary}[0]{boxrule=0pt, frame hidden,interior hidden, borderline west={1pt}{0pt}{RoyalBlue!95!black}, colbacktitle=white, enhanced jigsaw, arc=0mm,outer arc=0pt, colframe=RoyalBlue!95!black, colback=white,fonttitle=\bfseries, coltitle=RoyalBlue,title=\checkedbox Chapter take-away}

\newtcolorbox{todobox}[0]{boxrule=0pt, frame hidden,interior hidden, borderline west={1pt}{0pt}{Red!95!black}, colbacktitle=white, enhanced jigsaw, arc=0mm,outer arc=0pt, colframe=RoyalBlue!95!black, colback=white,fonttitle=\bfseries, coltitle=Red,title=TODO}

\usepackage{framed}
\usepackage[framemethod=TikZ]{mdframed}
\usetikzlibrary{positioning}
\usetikzlibrary{arrows}
\usetikzlibrary{patterns}

\usepackage{pgfplots}

\usepgfmodule{oo}
\usepgfplotslibrary{fillbetween}

\newenvironment{sidenote_}[1]{\begin{mdframed}[linewidth=1.5pt,roundcorner=10pt, leftmargin =+0cm,
		rightmargin=+0cm,frametitlerule=true,
		frametitlebackgroundcolor=RoyalBlue,
		usetwoside=false,linecolor=RoyalBlue!95!black,backgroundcolor=RoyalBlue!15!white,frametitle=\textsc{\textbf{\color{white}\noindent \pmark #1}}]}{\end{mdframed}}

\newenvironment{sidenote}[1]{%
	\begin{figure}[htbp]
	\begin{sidenote_}{#1}
	}{%
	\end{sidenote_}%
\end{figure}
}

\usepackage{siunitx}

\usepackage{ltablex}

\usepackage{epigraph}
\newcommand{\epi}[2]{\epigraph{{\LARGE ``}\textit{#1}{\LARGE ''}}{--- \textbf{#2}}}
\newcommand{\epianonymous}[1]{\epigraph{{\LARGE ``}\textit{#1}{\LARGE ''}}{}}

\usepackage{enumerate}

\usepackage{arydshln}

\usepackage{datetime}

\usepackage{glossaries}

\usepackage[multiple]{footmisc}
\newdateformat{monthyeardate}{%
	\monthname[\THEMONTH] \THEYEAR}

\usepackage{appendix}
\AtBeginEnvironment{subappendices}{%
	\chapter*{Appendix}
	\addcontentsline{toc}{chapter}{\textsc{Appendices}}
	\counterwithin{figure}{section}
	\counterwithin{table}{section}
}

\usepackage{scrextend}

\newenvironment{mldescription}{%
	\vspace{-1em}
	\begin{addmargin}[2em]{0em}
		\setlength{\parindent}{-2em}%
		\newcommand*{\mlitem}[1]{\par\vspace{1em}\textsc{##1}\;}\indent 
	}{%
	\end{addmargin}
	\bigskip
}

\usepackage{multirow}

\usepackage{pgfplots}
\usepackage{pgfplotstable}
\pgfplotsset{compat=newest}

\DeclareCaptionLabelFormat{andtable}{#1~#2  \&  \tablename~\thetable}
\usepackage[export]{adjustbox}

\usepackage{SevenSeg}

\def\uncommentversion{0} 

\if\uncommentversion1
\newcommand{\lwhnote}[1]{}
\newcommand{\asnote}[1]{}
\newcommand{\astobefinalized}[1]{\textit{To be finalized.}}
\newcommand{\pgnote}[1]{}

\else
\newcommand{\lwhnote}[1]{{\color{red}LW: #1}}
\newcommand{\asnote}[1]{{\color{blue}AS: #1}}
\newcommand{\astobefinalized}[1]{#1}
\newcommand{\pgnote}[1]{{\color{green}PG: #1}}

\fi

\begin{document}
\frenchspacing
\raggedbottom
\pagenumbering{arabic}
\pagestyle{scrheadings}


\begin{titlepage}
	\begin{addmargin}[-1cm]{-3cm}
    \begin{center}
        \large
        {\Large \textsc{University of Li{\`e}ge}}\\[1ex]
        Faculty of Applied Sciences\\
        Department of Electrical Engineering \& Computer Science\\

        \vfill \vfill

        PhD dissertation\\ \vskip1cm
        \begingroup
            \huge
            \color{RoyalBlue}\spacedallcaps{\myTitle} \\ \bigskip
        \endgroup
        \spacedlowsmallcaps{\mySubtitle} \\ \vskip1cm \vfill \vfill
        \textbf{by} \textsc{Antonio Sutera}

        \vfill
        \vfill
        \vfill
	
		\hfill
        \begin{tabular}{rl}
        Advisors: & Prof. \textsc{Pierre Geurts}\\
        & Prof. \textsc{Louis Wehenkel}\\
        & \\
        & June 2019
        \end{tabular}
    \end{center}
    \vspace{-3.5cm}\includegraphics[width=0.5\textwidth]{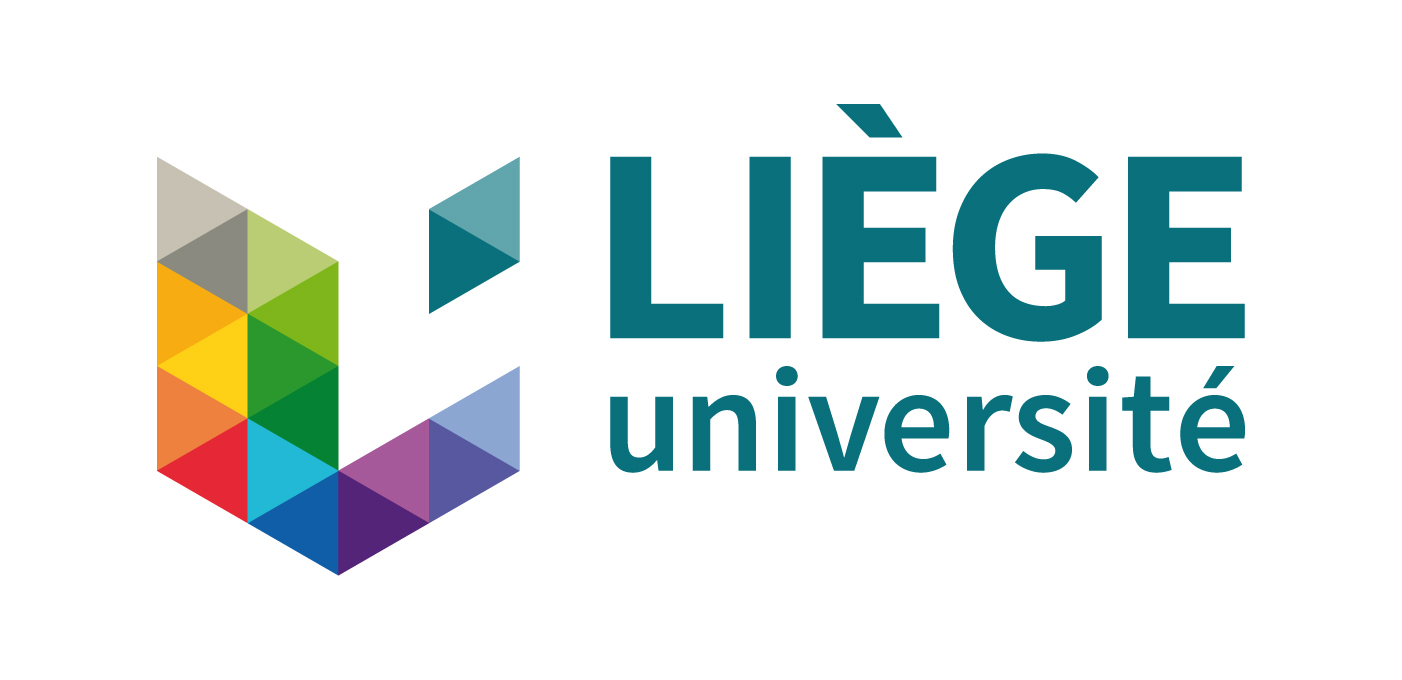}
  \end{addmargin}
\end{titlepage}


\pdfbookmark[1]{Jury members}{Jury members}
\chapter*{Jury members}

\noindent \textsc{Gilles Louppe}, Professor at the Universit{\'e} de Li{\`e}ge (President); \\[2ex]
\noindent \textsc{Pierre Geurts}, Professor at the Universit{\'e} de Li{\`e}ge (Advisor); \\[2ex]
\noindent \textsc{Louis Wehenkel}, Professor at the Universit{\'e} de Li{\`e}ge (Co-advisor); \\[2ex]
\noindent \textsc{Benoît Frénay}, Professor at the Universit{\'e} de Namur;\\[2ex]
\noindent \textsc{Robin Genuer}, Professor at the Universit{\'e} de Bordeaux (France);\\[2ex]
\noindent \textsc{Patrick Meyer}, Professor at the Universit{\'e} de Li{\`e}ge; \\[2ex]
\noindent \textsc{Erwan Scornet}, Professor at Ecole Polytechnique (France).\\[2ex]


\definecolor{ARED}{HTML}{E91E23}
\definecolor{ABLUE}{HTML}{02B2E7}
\definecolor{AYELLOW}{HTML}{FABC16}

\newcommand{\bd}[1]{\textcolor{#1}{\textbullet}\hspace{+0.1mm}}

\pdfbookmark[1]{Acknowledgments}{acknowledgments}
\chapter*{Acknowledgments}

To all the family members, friends and colleagues that helped me through the accomplishment of this thesis. 

\begin{center}
\textsc{T\bd{ARED}h\bd{ABLUE}a\bd{AYELLOW}n\bd{ARED}k \bd{AYELLOW} y\bd{ABLUE}o\bd{ARED}u}
\end{center}

\pdfbookmark[1]{Abstract}{Abstract}
\chapter*{Abstract}

Nowadays new technologies, and especially artificial intelligence,  are more and more established in our society.
Big data analysis and machine learning, two sub-fields of artificial intelligence, are at the core of many recent  breakthroughs  in many application fields (e.g., medicine, communication, finance, ...), including some that are strongly related to our day-to-day life (e.g., social networks, computers, smartphones, ...). In machine learning, significant improvements are usually achieved at the price of an increasing computational complexity and thanks to bigger datasets. Currently, cutting-edge models built by the most advanced machine learning algorithms typically became simultaneously very efficient and profitable but also extremely complex. Their complexity is to such an extent that these models are commonly seen as black-boxes providing a prediction or a decision which can not be interpreted or justified.
Nevertheless, whether these models are used autonomously or as a simple decision-making support tool, they are already being used in machine learning applications where health and human life are at stake.
Therefore, it appears to be an obvious necessity not to blindly believe everything coming out of those models without a detailed understanding of their predictions or decisions.

Accordingly, this thesis aims at improving the interpretability of models built by a specific family of machine learning algorithms, the so-called tree-based methods. Several mechanisms have been proposed to interpret these models and we aim along this thesis to improve their understanding, study their properties, and define their limitations.

The first part of this thesis introduces the techniques used to build these models, i.e. decision tree and ensemble of randomised trees induction algorithms. It also presents the basis of feature selection, a data analysis method aiming at identifying the essential features of a model and allowing to improve the model performances and/or its interpretability. 

The second part of this thesis focuses on the two most popular importance measures, aiming at measuring the relative importance of features in the model, derived from tree-based methods. Our contribution in this part is two-fold. On one hand, we review the main literature on that topic, with a focus on theoretical analyses. On the other hand, we improve the theoretical characterisation of one subclass of these importance measures, known as the Mean Decrease of Impurity (MDI), and study it in greater details, both theoretically and practically. 

The last part of this thesis is a collection of several works addressing some limitations of existing importance measures in some specific applications. We thus propose an extension of the MDI importance measure that can take into account different contexts in which the problem can be put, so as to provide further insight into the feature importances. We also study a new tree-based method that yields an efficient feature selection even in presence of large datasets and/or under memory constraints. Lastly we discuss the strengths and weaknesses of a solution to the network inference problem based on a tree-based importance measure, and propose a non tree-based method that we have designed as part of a network inference challenge that we eventually won. 

\pdfbookmark[1]{Résumé}{Résumé}
\chapter*{Résumé}

De nos jours, les nouvelles technologies, et tout particulièrement l'intelligence artificielle, sont toujours plus ancrées dans notre société. L'analyse de grands volumes de données et l'apprentissage automatique, deux sous-domaines de l'intelligence artificielle, sont  au centre des plus récentes percées  dans de nombreux domaines (e.g., la médecine, la communication, la finance, ...), et en particulier des applications intimement liées à notre vie quotidienne (réseaux sociaux, ordinateurs, smartphones, ...). 
En apprentissage automatique, les améliorations significatives sont souvent obtenues au prix d'une plus grande complexité computationelle et grâce à des quantités de données toujours plus grandes. A l'heure actuelle, les modèles de pointe obtenus par les algorithmes d'apprentissage automatique les plus sophistiqués sont généralement à la fois très efficaces et extrêmement complexes.  Leur complexité est telle qu'ils sont souvent vus comme des \guillemotleft bo\^ites noires\guillemotright\ fournissant une prédiction ou une décision qui ne peut ni être interprétée ni être justifiée. Néanmoins, que ces modèles soient considérés de manière autonome ou comme de simples outils d'aide à la décision, ils sont déjà utilisés dans des applications d'apprentissage automatique desquelles dépendent la santé et des vies humaines. Par conséquent, il apparait comme une évidente nécessité de ne pas croire les prédictions de ces modèles aveuglément, sans les avoir comprises.

Dans ce contexte, cette thèse a pour but d'améliorer l'interprétation qui peut être faite de modèles construits par une famille particulière d'algorithmes d'apprentissage automatique basées sur les arbres de décision. Plusieurs mécanismes ont été mis en œuvre pour interpréter ces modèles et nous visons tout au long de cette thèse à améliorer leur compréhension, à étudier leurs propriétés et à en définir les limites.

La première partie de cette thèse introduit les techniques de construction de ces modèles, à savoir les arbres de décision et les ensembles d'arbres aléatoires. Elle présente également les bases de la sélection de variables, méthode d'analyse de données qui a pour but d'identifier les variables essentielles d'un problème permettant à la fois d'améliorer les performances des modèles et leur interprétabilité. 

La seconde partie de cette thèse se concentre sur les deux mesures d'importance les plus populaires, visant à déterminer l'importance relative des variables dans le modèle, dérivées des méthodes \`a base d'arbres. Notre contribution dans cette partie est double. D'une part, nous examinons la littérature traitant ce sujet, avec une attention toute particulière pour les analyses théoriques. D'autre part, nous améliorons la caractérisation théorique d'une sous-classe de mesures d'importance, à savoir celle basée sur la réduction d'impureté (MDI), et nous l'étudions de manière détaillée théoriquement et pratiquement.

La dernière partie de cette thèse est une collection de plusieurs travaux qui se concentrent sur certaines limitations des mesures d'importance existantes dans des applications spécifiques. 
Ainsi, nous proposons une extension de la mesure d'importance MDI capable de prendre en compte les différents contextes dans lesquels le problème peut être placé, et cela de manière à fournir une connaissance approfondie sur l'importance des variables. Nous étudions également une nouvelle méthode à base d'arbres capable de fournir une sélection de variables performante, et ce, même en présence de grands volumes de données et/ou en cas de contraintes de mémoire. Et enfin, nous discutons les forces et les faiblesses d'une solution à ce problème au d'inférence de réseaux utilisant les mesures d'importances dérivées d'arbres de décision. Nous proposons également une méthode développée lors d'une compétition d'inférence de réseaux, qui ne fait pas intervenir les arbres de décision mais qui nous a permis de remporter cette compétition.


\refstepcounter{dummy}
\pdfbookmark[1]{\contentsname}{tableofcontents}
\setcounter{tocdepth}{3} 
\setcounter{secnumdepth}{3} 
\manualmark
\markboth{\spacedlowsmallcaps{\contentsname}}{\spacedlowsmallcaps{\contentsname}}
\tableofcontents
\automark[section]{chapter}
\renewcommand{\chaptermark}[1]{\markboth{\spacedlowsmallcaps{#1}}{\spacedlowsmallcaps{#1}}}
\renewcommand{\sectionmark}[1]{\markright{\thesection\enspace\spacedlowsmallcaps{#1}}}

\cleardoublepage

\newcommand\mytreeinabox{
	\draw[rounded corners=15pt]
	(0,0) rectangle ++(8,-6.5);
	
	\def\y{-1};
	\def\x{4};
	\def\xx{1.75};
	\def\xxx{0.9};
	\def\xxxx{0.5};
	
	\node[fill, circle, radius=0.1] at (\x,\y) (a) {};  
	\node[fill, circle, radius=0.1] at (\x-\xx,\y-1) (b1) {};
	\node[fill, circle, radius=0.1] at (\x+\xx,\y-1) (b2) {};
	
	\node[fill, circle, radius=0.1] at (\x-\xx-\xxx,\y-2) (c1) {};
	\node[fill, circle, radius=0.1] at (\x-\xx+\xxx,\y-2) (c2) {};
	\node[fill, circle, radius=0.1] at (\x+\xx-\xxx,\y-2) (c3) {};
	\node[fill, circle, radius=0.1] at (\x+\xx+\xxx,\y-2) (c4) {};
	
	\node[fill, circle, radius=0.1] at (\x-\xx+\xxx-\xxxx,\y-3) (d1) {};
	\node[fill, circle, radius=0.1] at (\x-\xx+\xxx+\xxxx,\y-3) (d2) {};
	\node[fill, circle, radius=0.1] at (\x+\xx-\xxx-\xxxx,\y-3) (d3) {};
	\node[fill, circle, radius=0.1] at (\x+\xx-\xxx+\xxxx,\y-3) (d4) {};
	\node[fill, circle, radius=0.1] at (\x+\xx+\xxx-\xxxx,\y-3) (d5) {};
	\node[fill, circle, radius=0.1] at (\x+\xx+\xxx+\xxxx,\y-3) (d6) {};
	
	\node[fill, circle, radius=0.1] at (\x-\xx+\xxx-\xxxx-\xxxx,\y-4) (e1) {};
	\node[fill, circle, radius=0.1] at (\x-\xx+\xxx-\xxxx+\xxxx,\y-4) (e2) {};
	\node[fill, circle, radius=0.1] at (\x+\xx-\xxx+\xxxx+\xxxx,\y-4) (e3) {};
	\node[fill, circle, radius=0.1] at (\x+\xx-\xxx+\xxxx-\xxxx,\y-4) (e4) {};
	
	\draw[] (a) -- (b1);
	\draw[] (a) -- (b2);
	
	\draw[] (b1) -- (c1);
	\draw[] (b1) -- (c2);
	\draw[] (b2) -- (c3);
	\draw[] (b2) -- (c4);
	
	\draw[] (c2) -- (d1);
	\draw[] (c2) -- (d2);
	\draw[] (c3) -- (d3);
	\draw[] (c3) -- (d4);
	\draw[] (c4) -- (d5);
	\draw[] (c4) -- (d6);
	
	\draw[] (d1) -- (e1);
	\draw[] (d1) -- (e2);
	\draw[] (d4) -- (e3);
	\draw[] (d4) -- (e4);
}

\newcommand\mytreeinaboxtwo{
	\draw[rounded corners=15pt]
	(0,0) rectangle ++(8,-6.5);
	
	\def\y{-1};
	\def\x{4};
	\def\xx{1.75};
	\def\xxx{0.9};
	\def\xxxx{0.5};
	
	\node[fill, circle, radius=0.1] at (\x,\y) (a) {};  
	\node[fill, circle, radius=0.1] at (\x-\xx,\y-1) (b1) {};
	\node[fill, circle, radius=0.1] at (\x+\xx,\y-1) (b2) {};
	
	\node[fill, circle, radius=0.1] at (\x-\xx-\xxx,\y-2) (c1) {};
	\node[fill, circle, radius=0.1] at (\x-\xx+\xxx,\y-2) (c2) {};
	\node[fill, circle, radius=0.1] at (\x+\xx-\xxx,\y-2) (c3) {};
	\node[fill, circle, radius=0.1] at (\x+\xx+\xxx,\y-2) (c4) {};
	
	\node[fill, circle, radius=0.1] at (\x-\xx-\xxx-\xxxx,\y-3) (d1) {};
	\node[fill, circle, radius=0.1] at (\x-\xx-\xxx+\xxxx,\y-3) (d2) {};
	\node[fill, circle, radius=0.1] at (\x-\xx+\xxx-\xxxx,\y-3) (d3) {};
	\node[fill, circle, radius=0.1] at (\x-\xx+\xxx+\xxxx,\y-3) (d4) {};
	\node[fill, circle, radius=0.1] at (\x+\xx+\xxx-\xxxx,\y-3) (d5) {};
	\node[fill, circle, radius=0.1] at (\x+\xx+\xxx+\xxxx,\y-3) (d6) {};
	
	\node[fill, circle, radius=0.1] at (\x-\xx-\xxx-\xxxx-\xxxx,\y-4) (e1) {};
	\node[fill, circle, radius=0.1] at (\x-\xx-\xxx-\xxxx+\xxxx,\y-4) (e2) {};

	\draw[] (a) -- (b1);
	\draw[] (a) -- (b2);
	
	\draw[] (b1) -- (c1);
	\draw[] (b1) -- (c2);
	\draw[] (b2) -- (c3);
	\draw[] (b2) -- (c4);
	
    \draw[] (c1) -- (d1);
	\draw[] (c1) -- (d2);
	\draw[] (c2) -- (d3);
	\draw[] (c2) -- (d4);
	\draw[] (c4) -- (d5);
	\draw[] (c4) -- (d6);
	
   	\draw[] (d1) -- (e1);
	\draw[] (d1) -- (e2);
}

\newcommand\mytreeinaboxthree{
	\draw[rounded corners=15pt]
	(0,0) rectangle ++(8,-6.5);
	
	\def\y{-1};
	\def\x{4};
	\def\xx{1.75};
	\def\xxx{0.9};
	\def\xxxx{0.5};
	
	\node[fill, circle, radius=0.1] at (\x,\y) (a) {};  
	\node[fill, circle, radius=0.1] at (\x-\xx,\y-1) (b1) {};
	\node[fill, circle, radius=0.1] at (\x+\xx,\y-1) (b2) {};
	
	\node[fill, circle, radius=0.1] at (\x-\xx-\xxx,\y-2) (c1) {};
	\node[fill, circle, radius=0.1] at (\x-\xx+\xxx,\y-2) (c2) {};
	\node[fill, circle, radius=0.1] at (\x+\xx-\xxx,\y-2) (c3) {};
	\node[fill, circle, radius=0.1] at (\x+\xx+\xxx,\y-2) (c4) {};
	
	\node[fill, circle, radius=0.1] at (\x-\xx-\xxx-\xxxx,\y-3) (d1) {};
	\node[fill, circle, radius=0.1] at (\x-\xx-\xxx+\xxxx,\y-3) (d2) {};
	\node[fill, circle, radius=0.1] at (\x-\xx+\xxx-\xxxx,\y-3) (d3) {};
	\node[fill, circle, radius=0.1] at (\x-\xx+\xxx+\xxxx,\y-3) (d4) {};
	\node[fill, circle, radius=0.1] at (\x+\xx-\xxx-\xxxx,\y-3) (d5) {};
	\node[fill, circle, radius=0.1] at (\x+\xx-\xxx+\xxxx,\y-3) (d6) {};
	\node[fill, circle, radius=0.1] at (\x+\xx+\xxx-\xxxx,\y-3) (d7) {};
	\node[fill, circle, radius=0.1] at (\x+\xx+\xxx+\xxxx,\y-3) (d8) {};

	\draw[] (a) -- (b1);
	\draw[] (a) -- (b2);
	
	\draw[] (b1) -- (c1);
	\draw[] (b1) -- (c2);
	\draw[] (b2) -- (c3);
	\draw[] (b2) -- (c4);
	
	\draw[] (c1) -- (d1);
	\draw[] (c1) -- (d2);
	\draw[] (c2) -- (d3);
	\draw[] (c2) -- (d4);
	\draw[] (c3) -- (d5);
	\draw[] (c3) -- (d6);
	\draw[] (c4) -- (d7);
	\draw[] (c4) -- (d8);

}

\newcommand{\oldmyrandomforest}{
\begin{scope}[scale=0.5, xshift=0,yshift=0,every node/.append style={transform shape}]
	\mytreeinabox
\end{scope}		

\begin{scope}[scale=0.5, xshift=9cm,yshift=0,every node/.append style={transform shape}]
	\mytreeinaboxtwo
\end{scope}		

\begin{scope}[scale=2, xshift=4.65cm,yshift=-22,every node/.append style={transform shape}]
	\node[] at (0,0) {$\dots$};
\end{scope}

\begin{scope}[scale=0.5, xshift=20.2cm,yshift=0,every node/.append style={transform shape}]
	\mytreeinaboxthree
\end{scope}


\node[draw] at (7,-4) (P) {Aggregate};
\draw[->,line width=2pt] (P) -- +(0,-1) node[anchor=north] () {$\hat{y}$};  

\coordinate (P1) at (2,-3);
\draw [->,line width=2pt] (P1.south) to [out=270,in=180] node [above] {$\hat{y}_1$} (P.west);

\coordinate (P2) at (6.39,-3);
\draw [->,line width=2pt] (P2.south) to [out=270,in=115] node [left] {$\hat{y}_2$} (P.north);

\coordinate (P3) at (12.1,-3);
\draw [->,line width=2pt] (P3.south) to [out=270,in=0] node [above] {$\hat{y}_T$} (P.east);

}

\newcommand{\myrandomforest}{

	\begin{scope}[scale=0.5, xshift=0,yshift=0,every node/.append style={transform shape}]
		\mytreeinabox		
		\coordinate (mid1) at (4,-6.5);
	\end{scope}		
	
	\begin{scope}[scale=0.5, xshift=9cm,yshift=0,every node/.append style={transform shape}]
		\mytreeinaboxtwo
		\coordinate (mid2) at (4,-6.5);
	\end{scope}

	\begin{scope}[scale=2, xshift=4.65cm,yshift=-22,every node/.append style={transform shape}]
		\node[] at (0,0) {$\dots$};
	\end{scope}
	
	\begin{scope}[scale=0.5, xshift=20.2cm,yshift=0,every node/.append style={transform shape}]
		\mytreeinaboxthree
		\coordinate (mid3) at (4,-6.5);
	\end{scope}

	 \node[draw,rectangle,rounded corners=5pt,fill=ForestGreen!20] at (7.5,-4.5) (agg) {\textit{Aggregation}};

	\draw [->,line width=2pt] (mid1) to [out=270,in=180] node [above] {$\hat{y}_1$} (agg.west);
	\draw [->,line width=2pt] (mid2) to [out=270,in=90] node [above,pos=0.7] {$\hat{y}_2$} (agg.north);
	\draw [->,line width=2pt] (mid3) to [out=270,in=0] node [above] {$\hat{y}_{N_T}$} (agg.east);
	
	\draw[->,line width=2pt] (agg) -- +(0,-1) node[below] () {$\mathbf{T}(\mathbf{x}) = \hat{y}$};

}

\newcommand{\myLS}[9]{
\def\scale{0.7};

\draw[fill=white] (0,0) rectangle (\scale*3,-2*\scale);

\def\s{-0.4*\scale};

\foreach \i in {#1,#2,#3,#4} {
	
	\ifthenelse{\equal{\i}{0}}
	{
	}
	{\fill[RoyalBlue!20] (0,\i*\s-\s) rectangle (3*\scale,\i*\s);}

}

\def\c{0.6*\scale};

\foreach \j in {#5,#6,#7,#8,#9}{
	\ifthenelse{\equal{\j}{0}}
	{
	}
	{\fill[RoyalBlue!20] (\j*\c -\c,0) rectangle (\j*\c,-2*\scale);}
}


\node[] at (\c*0.5,0.2) () {\footnotesize $X_1$};
\node[] at (\c*1.5,0.2) () {\footnotesize $X_2$};
\node[] at (\c*2.5,0.2) () {\footnotesize $X_3$};
\node[] at (\c*3.5,0.2) () {\footnotesize $X_4$};
\node[] at (\c*4.5,0.2) () {\footnotesize $X_5$};

\foreach \i in {1,2,3,4,5} {
	\draw[] (0,\i*\s-\s) rectangle (3*\scale,\i*\s); 
}
\foreach \j in {1,2,3,4,5}{
	\draw[] (\j*\c -\c,0) rectangle (\j*\c,-2*\scale);
}

}

\newcommand{\mypartialLS}[3]{
	\def\scale{0.7};
	
	
	\def\s{-0.4*\scale}
	
	\def\c{0.6*\scale}
	
	
	\node[] at (\c*1.5,0.2) () {\footnotesize #1};
	\node[] at (\c*2.5,0.2) () {\footnotesize #2};
	\node[] at (\c*3.5,0.2) () {\footnotesize #3};

	\foreach \i in {2,3,4} {
		\fill[RoyalBlue!20] (\c,\i*\s-\s) rectangle (3*\scale-\c,\i*\s); 

	}
	\foreach \j in {2,3,4}{
		\fill[RoyalBlue!20]  (\j*\c -\c,0) rectangle (\j*\c,-2*\scale);
		\draw[] (\j*\c -\c,0) rectangle (\j*\c,-2*\scale);
	}
	\foreach \i in {2,3,4} {
	\draw[] (\c,\i*\s-\s) rectangle (3*\scale-\c,\i*\s); 
	}

}

\newcommand{\mycoloredLS}[5]{
\def\scale{1};
\def\c{0.6*\scale};
\def\s{-0.4*\scale};

\draw[fill=#1] (0,1*\s-\s) rectangle (3*\scale,1*\s); 
\draw[fill=#2] (0,2*\s-\s) rectangle (3*\scale,2*\s); 
\draw[fill=#3] (0,3*\s-\s) rectangle (3*\scale,3*\s); 
\draw[fill=#4] (0,4*\s-\s) rectangle (3*\scale,4*\s); 
\draw[fill=#5] (0,5*\s-\s) rectangle (3*\scale,5*\s); 

\draw[] (1*\c -\c,0) rectangle (1*\c,-2*\scale);
\draw[] (2*\c -\c,0) rectangle (2*\c,-2*\scale);
\draw[] (3*\c -\c,0) rectangle (3*\c,-2*\scale);
\draw[] (4*\c -\c,0) rectangle (4*\c,-2*\scale);
\draw[] (5*\c -\c,0) rectangle (5*\c,-2*\scale);

}


\newcommand{\mysimpletreestructure}{
\node[draw,circle,RoyalBlue, line width=2pt,minimum size=1cm] (a) at (0,0) {};
\node[draw,circle,Gray!80, line width=2pt,minimum size=1cm,below left=0.5cm and 2.1cm of a] (b1) {};
\node[draw,circle,Gray!80, line width=2pt,minimum size=1cm,below right=0.5cm and 2.1cm of a] (b2) {};

\node[draw,circle,Orange, line width=2pt,minimum size=1cm,below left=0.5cm and 1.5cm of b1] (c1) {};
\node[draw,circle,Orange, line width=2pt,minimum size=1cm,below right=0.5cm and 1.5cm of b1] (c2) {};

\node[draw,circle,Orange, line width=2pt,minimum size=1cm,below left=0.5cm and 1.5cm of b2] (c3) {};
\node[draw,circle,Orange, line width=2pt,minimum size=1cm,below right=0.5cm and 1.5cm of b2] (c4) {};

\draw[Gray!80, line width=2pt,->] (a) -- node[sloped, anchor=center,above,black] {} (b1);
\draw[Gray!80, line width=2pt,->] (a) -- node[sloped, anchor=center,above,black] {} (b2);
\draw[Gray!80, line width=2pt,->] (b1) -- node[sloped, anchor=center,above,black] {} (c1);
\draw[Gray!80, line width=2pt,->] (b1) -- node[sloped, anchor=center,above,black] {} (c2);
\draw[Gray!80, line width=2pt,->] (b2) -- node[sloped, anchor=center,above,black] {} (c3);
\draw[Gray!80, line width=2pt,->] (b2) -- node[sloped, anchor=center,above,black] {} (c4);

\draw[Gray!80, line width=2pt] (6.2,0.5) -- node[right,black] {Internal nodes} (6.2,-1.5);
\draw[Orange, line width=2pt] (6.2,-1.8) -- node[right,black] {Terminal nodes}  (6.2,-3);

\draw[RoyalBlue, line width=2pt] (-0.5,0.7) -- node[above,black] {Root}  (0.5,0.7);

\node[black]  at (a) { $t_0$}; 
\node[black]  at (b1) { $t_1$}; 
\node[black] at (b2) { $t_2$};

\node[black] at (c1) { $t_3$};
\node[black] at (c2) { $t_4$};
\node[black] at (c3) { $t_5$};
\node[black] at (c4) { $t_6$};
}

\newcommand{\mysimplebinaryclassificationtreeexample}{
\node[draw,circle,Gray!80, line width=2pt,minimum size=1cm] (a) at (0,0) {};
\node[draw,circle,Gray!80, line width=2pt,minimum size=1cm,below left=0.5cm and 2.1cm of a] (b1) {};
\node[draw,circle,Gray!80, line width=2pt,minimum size=1cm,below right=0.5cm and 2.1cm of a] (b2) {};

\node[draw,circle,Gray!80, line width=2pt,minimum size=1cm,below left=0.5cm and 1.5cm of b1] (c1) {};
\node[draw,circle,Gray!80, line width=2pt,minimum size=1cm,below right=0.5cm and 1.5cm of b1] (c2) {};

\node[draw,circle,Gray!80, line width=2pt,minimum size=1cm,below left=0.5cm and 1.5cm of b2] (c3) {};
\node[draw,circle,Gray!80, line width=2pt,minimum size=1cm,below right=0.5cm and 1.5cm of b2] (c4) {};

\draw[Gray!80, line width=2pt,->] (a) -- node[sloped, anchor=center,above,black] {\small $top$} (b1);
\draw[Gray!80, line width=2pt,->] (a) -- node[sloped, anchor=center,above,black] {\small $bottom$} (b2);
\draw[Gray!80, line width=2pt,->] (b1) -- node[sloped, anchor=center,above,black] {\footnotesize $left$} (c1);
\draw[Gray!80, line width=2pt,->] (b1) -- node[sloped, anchor=center,above,black] {\footnotesize $right$} (c2);
\draw[Gray!80, line width=2pt,->] (b2) -- node[sloped, anchor=center,above,black] {\footnotesize $left$} (c3);
\draw[Gray!80, line width=2pt,->] (b2) -- node[sloped, anchor=center,above,black] {\footnotesize $right$} (c4);

\node[ line width=2pt,minimum size=1cm,below =0.5cm of c1] (y1) {$\hat{y}_3=\qquad$};
\node[ line width=2pt,minimum size=1cm,below =0.5cm of c2] (y2) {$\hat{y}_4=\qquad$};
\node[ line width=2pt,minimum size=1cm,below =0.5cm of c3] (y3) {$\hat{y}_5=\;$};
\node[ line width=2pt,minimum size=1cm,below =0.5cm of c4] (y4) {$\hat{y}_6=\qquad$};	

\draw[dashed,ForestGreen,line width=2pt,->] (c1) -- (y1);
\draw[dashed,RoyalBlue,line width=2pt,->] (c2) -- (y2);
\draw[dashed,Red,line width=2pt,->] (c3) -- (y3);
\draw[dashed,Orange,line width=2pt,->] (c4) -- (y4);	

\node[fill=ForestGreen, right=-0.75cm of y1] {};
\node[fill=RoyalBlue, right=-0.75cm of y2] {};
\node[fill=Red, right=-0.1cm of y3] {};
\node[fill=Orange, right=-0.75cm of y4] {};

\draw[Gray!0!white, line width=2pt, dashed] (6.2,-3.3) -- node[right,black] {Predictions}  (6.2,-4.5);
\draw[ForestGreen, line width=2pt] (6.2,-3.4) -- (6.2,-3.55);
\draw[RoyalBlue, line width=2pt] (6.2,-3.7) -- (6.2,-3.85);
\draw[Red, line width=2pt] (6.2,-4.0) -- (6.2,-4.15);
\draw[Orange, line width=2pt] (6.2,-4.3) -- (6.2,-4.45);

\node[] at (0,-0.9) {$Row?$}; 
\node[] at (-2.85,-2.1) {\small $Column?$}; 
\node[] at (+2.85,-2.1) {\small $Column?$}; 

\node[black, above right=-0.1cm and -0.05cm of a] {\footnotesize $t_0$}; 
\node[black, above left=-0.05cm and -0.05cm of b1] {\footnotesize $t_1$}; 
\node[black, above right=-0.05cm and -0.05cm of b2] {\footnotesize $t_2$};

\node[black, below left=-0.05cm and -0.05cm of c1] {\footnotesize $t_3$};
\node[black, below left=-0.05cm and -0.05cm of c2] {\footnotesize $t_4$};
\node[black, below right=-0.05cm and -0.05cm of c3] {\footnotesize $t_5$};
\node[black, below right=-0.05cm and -0.05cm of c4] {\footnotesize $t_6$};

\begin{scope}[]
	\node[fill=ForestGreen] at (-0.15,0.15) {};
	\node[fill=RoyalBlue] at (+0.15,0.15) {};
	\node[fill=Red] at (-0.15,-0.15) {};
	\node[fill=Orange] at (0.15,-0.15) {};
\end{scope}

\begin{scope}[shift={(b1)}]
	\node[fill=ForestGreen] at (-0.15,0.15) {};
	\node[fill=RoyalBlue] at (+0.15,0.15) {};
	\node[fill=Red!10] at (-0.15,-0.15) {};
	\node[fill=Orange!10] at (0.15,-0.15) {};
\end{scope}

\begin{scope}[shift={(b2)}]
	\node[fill=ForestGreen!10] at (-0.15,0.15) {};
	\node[fill=RoyalBlue!10] at (+0.15,0.15) {};
	\node[fill=Red] at (-0.15,-0.15) {};
	\node[fill=Orange] at (0.15,-0.15) {};
\end{scope}

\begin{scope}[shift={(c1)}]
	\node[fill=ForestGreen] at (-0.15,0.15) {};
	\node[fill=RoyalBlue!10] at (+0.15,0.15) {};
	\node[fill=Red!10] at (-0.15,-0.15) {};
	\node[fill=Orange!10] at (0.15,-0.15) {};
\end{scope}

\begin{scope}[shift={(c2)}]
	\node[fill=ForestGreen!10] at (-0.15,0.15) {};
	\node[fill=RoyalBlue] at (+0.15,0.15) {};
	\node[fill=Red!10] at (-0.15,-0.15) {};
	\node[fill=Orange!10] at (0.15,-0.15) {};
\end{scope}

\begin{scope}[shift={(c3)}]
	\node[fill=ForestGreen!10] at (-0.15,0.15) {};
	\node[fill=RoyalBlue!10] at (+0.15,0.15) {};
	\node[fill=Red] at (-0.15,-0.15) {};
	\node[fill=Orange!10] at (0.15,-0.15) {};
\end{scope}

\begin{scope}[shift={(c4)}]
	\node[fill=ForestGreen!10] at (-0.15,0.15) {};
	\node[fill=RoyalBlue!10] at (+0.15,0.15) {};
	\node[fill=Red!10] at (-0.15,-0.15) {};
	\node[fill=Orange] at (0.15,-0.15) {};
\end{scope}

\begin{scope}[shift={(-3.4,1.6)}]
	\draw[step=0.30] (-0.30,-0.30) grid (0.30,0.30);
	\node[] at (-0.8,0) {$\mathcal{X}=$};
\end{scope}

\begin{scope}[shift={(-2.2,0.7)}]
	\node[] at (-0.8,0) (Y1) {$\mathcal{Y}=\{\quad,\quad,\quad,\quad\}$};
	\node[fill=ForestGreen, below right= 0cm and -1.25cm of Y1] at (-0.15,0.15) {};
	\node[fill=RoyalBlue, below right= 0cm and -0.75cm of Y1] at (-0.15,0.15) {};
	\node[fill=Red, below right= 0cm and -0.2cm of Y1] at (-0.15,0.15) {};
	\node[fill=Orange, below right= 0cm and 0.35cm of Y1] at (-0.15,0.15) {};
\end{scope}
}

\newcommand{\myclassificationexample}{
		\begin{axis}[legend pos=north east,    
	x label style={at={(axis description cs:0.5,-0.075)},},
	y label style={at={(axis description cs:-0.05,.5)},rotate=0},
	xlabel={$X_1$},
	ylabel={$X_2$},
	legend style={,draw=gray!20},
	ymin=0,ymax=2.1,
	xmin=0,xmax=3.1]
	
	\addplot[scatter,only marks,scatter src=explicit symbolic,
	scatter/classes={
		a={mark=*,RoyalBlue,thick},
		b={mark=square*,Orange,thick}
	}]
	table[x=x,y=y,meta=label,row sep=crcr]{
		x	y	label\\
		0.7	1.2	a\\
		0.3	1.8 a\\
		1.1	0.7 a\\
		1.5 0.2 a \\
		1.7 0.7 a\\
		2.2 1.25 a\\
		2.5 1.1 a \\
		2.7 1.4 a\\
		0.3  0.7 b\\
		0.7  0.3 b\\
		0.9  0.9 b\\
		1.1	1.5 b\\
		1.7 1.1 b\\
		2.3 0.8 b\\
		2.5 0.2 b\\
		2.8 0.9 b\\
	};
	\legend{Class 1,Class 2}
\end{axis}
}
\newcommand{\myclassificationexampleexplained}{		
	\begin{axis}[legend pos=north east,    
		x label style={at={(axis description cs:0.5,-0.075)},},
		y label style={at={(axis description cs:-0.05,.5)},rotate=0},
		xlabel={$X_1$},
		ylabel={$X_2$},
		legend style={,draw=gray!20},
		ymin=0,ymax=2.1,
		xmin=0,xmax=3.1]
		\addplot[mark=none, black, dashed, samples=2,forget plot] {1};
		\addplot +[mark=none,dashed,black,forget plot] coordinates {(1, 0) (1, 3)};
		\addplot +[mark=none,dashed,black,forget plot] coordinates {(2, 0) (2, 3)};
		\addplot[scatter,only marks,scatter src=explicit symbolic,
		scatter/classes={
			a={mark=*,RoyalBlue,thick},
			b={mark=square*,Orange,thick}
		}]
		table[x=x,y=y,meta=label,row sep=crcr]{
			x	y	label\\
			0.7	1.2	a\\
			0.3	1.8 a\\
			1.1	0.7 a\\
			1.5 0.2 a \\
			1.7 0.7 a\\
			2.2 1.25 a\\
			2.5 1.1 a \\
			2.7 1.4 a\\
			0.3  0.7 b\\
			0.7  0.3 b\\
			0.9  0.9 b\\
			1.1	1.5 b\\
			1.7 1.1 b\\
			2.3 0.8 b\\
			2.5 0.2 b\\
			2.8 0.9 b\\
		};
		\legend{Class 1,Class 2}
\end{axis}
}

\newcommand{\mybinaryclassificationtreeexample}{
	\def\d{1.2}



\foreach \xi in {0,1,2,3} {\draw[Gray!50,dashed] (-2.5,-0.5*\d-\xi * \d) -- (-2.5,0+0.5*\d-\xi * \d) -- (4.5,+0.5*\d-\xi * \d) -- (4.5,-0.5*\d-\xi * \d);};
\draw[Gray!50,dashed] (-2.5,-3.5*\d) -- (4.5,-3.5*\d);


\node[Gray!80,draw,circle, line width=2pt,] (a) at (0,0) {\textcolor{red}{$X_2\le 1$}};

\node[Gray!80,draw,circle, line width=2pt,fill=white] (b1) at (-1.2,-\d) {\textcolor{red}{$X_1\le 1$}};
\node[Gray!80,draw,circle, line width=2pt,fill=white] (b2) at (+1.2,-\d) {\textcolor{red}{$X_1\le 1$}};

\node[draw,Orange, line width=2pt,fill=Orange!20] (c1) at (-1.8,-2*\d) {};
\node[Gray!80,draw,circle, line width=2pt,fill=white] (c2) at (-0.6,-2*\d) {\textcolor{red}{$X_1\le 2$}};
\node[draw,RoyalBlue, line width=2pt,fill=RoyalBlue!20] (c3) at (+0.6,-2*\d) {};
\node[Gray!80,draw,circle, line width=2pt,fill=white] (c4) at (+1.8,-2*\d) {\textcolor{red}{$X_1\le 2$}};

\node[draw,RoyalBlue, line width=2pt,fill=RoyalBlue!20] (d3) at (-0.9,-3*\d) {};
\node[draw,Orange, line width=2pt,fill=Orange!20] (d4) at (-0.3,-3*\d) {};
\node[draw,Orange, line width=2pt,fill=Orange!20] (d7) at (1.5,-3*\d) {};
\node[draw,RoyalBlue, line width=2pt,fill=RoyalBlue!20] (d8) at (2.1,-3*\d) {};

\draw[Gray!80, line width=2pt,->] (a) -- node[left,pos=0.1,black] {\footnotesize $\le 1$} (b1);
\draw[Gray!80, line width=2pt,->] (a) -- node[right,pos=0.1,black] {\footnotesize $>1$} (b2);

\draw[Orange, line width=2pt,->] (b1) -- node[left,pos=0.1,black] {\footnotesize $\le 1$} (c1);
\draw[Gray!80, line width=2pt,->] (b1) -- node[right,pos=0.1,black] {\footnotesize $>1$} (c2);
\draw[RoyalBlue, line width=2pt,->] (b2) -- node[left,pos=0.1,black] {\footnotesize $\le 1$} (c3);
\draw[Gray!80, line width=2pt,->] (b2) -- node[right,pos=0.1,black] {\footnotesize $>1$} (c4);

\draw[RoyalBlue, line width=2pt,->] (c2) -- node[left,pos=0.1,black] {\footnotesize $\le 2$} (d3);
\draw[Orange, line width=2pt,->] (c2) -- node[right,pos=0.1,black] {\footnotesize $>2$} (d4);
\draw[Orange, line width=2pt,->] (c4) -- node[left,pos=0.1,black] {\footnotesize $\le 2$} (d7);
\draw[RoyalBlue, line width=2pt,->] (c4) -- node[right,pos=0.1,black] {\footnotesize $>2$} (d8);

\begin{scope}[shift={(2.8,-0.47)},scale=0.37]
	\begin{axis}[legend pos=north east,    
		x label style={at={(axis description cs:1.1,+0.1)},},
		y label style={at={(axis description cs:-0.05,.5)},rotate=0},
		xlabel={$X_1$},
		ylabel={$X_2$},
		legend style={,draw=gray!20},
		ymin=0,ymax=2.1,
		xmin=0,xmax=3.1]
		
		\addplot[mark=none, red, samples=2,forget plot,line width=2pt ] {1};
		
		\addplot[scatter,only marks,scatter src=explicit symbolic,
		scatter/classes={
			a={mark=*,RoyalBlue,thick},
			b={mark=square*,Orange,thick}
		}]
		table[x=x,y=y,meta=label,row sep=crcr]{
			x	y	label\\
			0.7	1.2	a\\
			0.3	1.8 a\\
			1.1	0.7 a\\
			1.5 0.2 a \\
			1.7 0.7 a\\
			2.2 1.25 a\\
			2.5 1.1 a \\
			2.7 1.4 a\\
			0.3  0.7 b\\
			0.7  0.3 b\\
			0.9  0.9 b\\
			1.1	1.5 b\\
			1.7 1.1 b\\
			2.3 0.8 b\\
			2.5 0.2 b\\
			2.8 0.9 b\\
		};
	\end{axis}
\end{scope}

\begin{scope}[shift={(2.8,-0.47-\d)},scale=0.37]
	\fill[white!10] (0,0) rectangle (3.44,1.36);
	\fill[white!10] (0,1.36) rectangle (3.44,2.84);
	\begin{axis}[legend pos=north east,    
		x label style={at={(axis description cs:1.1,+0.1)},},
		y label style={at={(axis description cs:-0.05,.5)},rotate=0},
		xlabel={$X_1$},
		ylabel={$X_2$},
		legend style={,draw=gray!20},
		ymin=0,ymax=2.1,
		xmin=0,xmax=3.1]
		\addplot[mark=none, black, samples=2,forget plot,line width=1pt ] {1};
		\addplot +[mark=none,red,forget plot,line width=2pt] coordinates {(1, 0) (1, 3)};
		\addplot[scatter,only marks,scatter src=explicit symbolic,
		scatter/classes={
			a={mark=*,RoyalBlue,thick},
			b={mark=square*,Orange,thick}
		}]
		table[x=x,y=y,meta=label,row sep=crcr]{
			x	y	label\\
			0.7	1.2	a\\
			0.3	1.8 a\\
			1.1	0.7 a\\
			1.5 0.2 a \\
			1.7 0.7 a\\
			2.2 1.25 a\\
			2.5 1.1 a \\
			2.7 1.4 a\\
			0.3  0.7 b\\
			0.7  0.3 b\\
			0.9  0.9 b\\
			1.1	1.5 b\\
			1.7 1.1 b\\
			2.3 0.8 b\\
			2.5 0.2 b\\
			2.8 0.9 b\\
		};
	\end{axis}
\end{scope}

\begin{scope}[shift={(2.8,-0.47-2*\d)},scale=0.37]
	\fill[Orange!10] (0,0) rectangle (1.11,1.36);
	\fill[RoyalBlue!10] (0,1.36) rectangle (1.11,2.84);
	\fill[white!10] (1.11,0) rectangle (3.44,1.36);
	\fill[white!10] (1.11,1.36) rectangle (3.44,2.84);		
	
	\begin{axis}[legend pos=north east,    
		x label style={at={(axis description cs:1.1,+0.1)},},
		y label style={at={(axis description cs:-0.05,.5)},rotate=0},
		xlabel={$X_1$},
		ylabel={$X_2$},
		legend style={,draw=gray!20},
		ymin=0,ymax=2.1,
		xmin=0,xmax=3.1]
		\addplot[mark=none, black, samples=2,forget plot,line width=1pt ] {1};
		\addplot +[mark=none,black,forget plot,line width=1pt] coordinates {(1, 0) (1, 3)};
		\addplot +[mark=none,red,forget plot,line width=2pt] coordinates {(2, 0) (2, 3)};
		\addplot[scatter,only marks,scatter src=explicit symbolic,
		scatter/classes={
			a={mark=*,RoyalBlue,thick},
			b={mark=square*,Orange,thick}
		}]
		table[x=x,y=y,meta=label,row sep=crcr]{
			x	y	label\\
			0.7	1.2	a\\
			0.3	1.8 a\\
			1.1	0.7 a\\
			1.5 0.2 a \\
			1.7 0.7 a\\
			2.2 1.25 a\\
			2.5 1.1 a \\
			2.7 1.4 a\\
			0.3  0.7 b\\
			0.7  0.3 b\\
			0.9  0.9 b\\
			1.1	1.5 b\\
			1.7 1.1 b\\
			2.3 0.8 b\\
			2.5 0.2 b\\
			2.8 0.9 b\\
		};
	\end{axis}
\end{scope}

\begin{scope}[shift={(2.8,-0.47-3*\d)},scale=0.37]
	\fill[Orange!10] (0,0) rectangle (1.11,1.36);
	\fill[RoyalBlue!10] (0,1.36) rectangle (1.11,2.84);
	\fill[RoyalBlue!10] (1.11,0) rectangle (2.22,1.36);
	\fill[Orange!10] (1.11,1.36) rectangle (2.22,2.84);		
	\fill[Orange!10] (2.22,0) rectangle (3.44,1.36);
	\fill[RoyalBlue!10] (2.22,1.36) rectangle (3.44,2.84);				
	
	\begin{axis}[legend pos=north east,    
		x label style={at={(axis description cs:1.1,+0.1)},},
		y label style={at={(axis description cs:-0.05,.5)},rotate=0},
		xlabel={$X_1$},
		ylabel={$X_2$},
		legend style={,draw=gray!20},
		ymin=0,ymax=2.1,
		xmin=0,xmax=3.1]

		\addplot[scatter,only marks,scatter src=explicit symbolic,
		scatter/classes={
			a={mark=*,RoyalBlue,thick},
			b={mark=square*,Orange,thick}
		}]
		table[x=x,y=y,meta=label,row sep=crcr]{
			x	y	label\\
			0.7	1.2	a\\
			0.3	1.8 a\\
			1.1	0.7 a\\
			1.5 0.2 a \\
			1.7 0.7 a\\
			2.2 1.25 a\\
			2.5 1.1 a \\
			2.7 1.4 a\\
			0.3  0.7 b\\
			0.7  0.3 b\\
			0.9  0.9 b\\
			1.1	1.5 b\\
			1.7 1.1 b\\
			2.3 0.8 b\\
			2.5 0.2 b\\
			2.8 0.9 b\\
		};
	\end{axis}
\end{scope}
}

\newcommand{\mymultiwayclassificationtreeexample}{

	\def\d{1.2}



\foreach \xi in {0,1,2} {\draw[Gray!50,dashed] (-2.5,-0.5*\d-\xi * \d) -- (-2.5,0+0.5*\d-\xi * \d) -- (4.5,+0.5*\d-\xi * \d) -- (4.5,-0.5*\d-\xi * \d);};
\draw[Gray!50,dashed] (-2.5,-2.5*\d) -- (4.5,-2.5*\d);


\node[draw,circle, Gray!80, line width=2pt,] (a) at (0,0) {\textcolor{red}{$\,X_2\,$}};

\node[Gray!80, draw,circle, line width=2pt,fill=white] (b1) at (-1.2,-\d) {\textcolor{red}{$\,X_1\,$}};
\node[Gray!80, draw,circle, line width=2pt,fill=white] (b2) at (+1.2,-\d) {\textcolor{red}{$\,X_1\,$}};

\node[draw,Orange, line width=2pt,fill=Orange!20] (c1) at (-2,-2*\d) {};
\node[draw,RoyalBlue, line width=2pt,fill=RoyalBlue!20] (c2) at (-1.2,-2*\d) {};
\node[draw,Orange, line width=2pt,fill=Orange!20] (c3) at (-0.4,-2*\d) {};
\node[draw,RoyalBlue, line width=2pt,fill=RoyalBlue!20] (c4) at (+0.4,-2*\d) {};
\node[draw,Orange, line width=2pt,fill=Orange!20] (c5) at (+1.2,-2*\d) {};
\node[draw,RoyalBlue, line width=2pt,fill=RoyalBlue!20] (c6) at (+2,-2*\d) {};	

\draw[Gray!80, line width=2pt,->] (a) -- node[left,pos=0.1,black] {\footnotesize $\le 1$} (b1);
\draw[Gray!80, line width=2pt,->] (a) -- node[right,pos=0.1,black] {\footnotesize $>1$} (b2);

\draw[Orange, line width=2pt,->] (b1) -- node[above,sloped,pos=0.5,black,minimum height = 1em] {\tiny $\le 1$} (c1);
\draw[RoyalBlue, line width=2pt,->] (b1) -- node[above,sloped,pos=0.5,black,minimum height = 1em] {\tiny $1 < x \le 2$} (c2);
\draw[Orange, line width=2pt,->] (b1) -- node[above,sloped,pos=0.5,black,minimum height = 1em] {\tiny $2< x \le 3$} (c3);
\draw[RoyalBlue, line width=2pt,->] (b2) -- node[above,sloped,pos=0.5,black,minimum height = 1em] {\tiny $\le 1$} (c4);
\draw[Orange, line width=2pt,->] (b2) -- node[above,sloped,pos=0.5,black,minimum height = 1em] {\tiny $1 <x \le 2$} (c5);
\draw[RoyalBlue, line width=2pt,->] (b2) -- node[above,sloped,pos=0.5,black,minimum height = 1em] {\tiny $2 < x \le 3$} (c6);

\begin{scope}[shift={(2.8,-0.47)},scale=0.37]
	\begin{axis}[legend pos=north east,    
		x label style={at={(axis description cs:1.1,+0.1)},},
		y label style={at={(axis description cs:-0.05,.5)},rotate=0},
		xlabel={$X_1$},
		ylabel={$X_2$},
		legend style={,draw=gray!20},
		ymin=0,ymax=2.1,
		xmin=0,xmax=3.1]
		
		\addplot[mark=none, red, samples=2,forget plot,line width=2pt ] {1};
		
		\addplot[scatter,only marks,scatter src=explicit symbolic,
		scatter/classes={
			a={mark=*,RoyalBlue,thick},
			b={mark=square*,Orange,thick}
		}]
		table[x=x,y=y,meta=label,row sep=crcr]{
			x	y	label\\
			0.7	1.2	a\\
			0.3	1.8 a\\
			1.1	0.7 a\\
			1.5 0.2 a \\
			1.7 0.7 a\\
			2.2 1.25 a\\
			2.5 1.1 a \\
			2.7 1.4 a\\
			0.3  0.7 b\\
			0.7  0.3 b\\
			0.9  0.9 b\\
			1.1	1.5 b\\
			1.7 1.1 b\\
			2.3 0.8 b\\
			2.5 0.2 b\\
			2.8 0.9 b\\
		};
	\end{axis}
\end{scope}

\begin{scope}[shift={(2.8,-0.47-\d)},scale=0.37]
	\fill[white!10] (0,0) rectangle (3.44,1.36);
	\fill[white!10] (0,1.36) rectangle (3.44,2.84);
	\begin{axis}[legend pos=north east,    
		x label style={at={(axis description cs:1.1,+0.1)},},
		y label style={at={(axis description cs:-0.05,.5)},rotate=0},
		xlabel={$X_1$},
		ylabel={$X_2$},
		legend style={,draw=gray!20},
		ymin=0,ymax=2.1,
		xmin=0,xmax=3.1]
		\addplot[mark=none, black, samples=2,forget plot,line width=1pt ] {1};
		\addplot +[mark=none,red,forget plot,line width=2pt] coordinates {(1, 0) (1, 3)};
		\addplot +[mark=none,red,forget plot,line width=2pt] coordinates {(2, 0) (2, 3)};
		\addplot[scatter,only marks,scatter src=explicit symbolic,
		scatter/classes={
			a={mark=*,RoyalBlue,thick},
			b={mark=square*,Orange,thick}
		}]
		table[x=x,y=y,meta=label,row sep=crcr]{
			x	y	label\\
			0.7	1.2	a\\
			0.3	1.8 a\\
			1.1	0.7 a\\
			1.5 0.2 a \\
			1.7 0.7 a\\
			2.2 1.25 a\\
			2.5 1.1 a \\
			2.7 1.4 a\\
			0.3  0.7 b\\
			0.7  0.3 b\\
			0.9  0.9 b\\
			1.1	1.5 b\\
			1.7 1.1 b\\
			2.3 0.8 b\\
			2.5 0.2 b\\
			2.8 0.9 b\\
		};
	\end{axis}
\end{scope}

\begin{scope}[shift={(2.8,-0.47-2*\d)},scale=0.37]
	\fill[Orange!10] (0,0) rectangle (1.11,1.36);
	\fill[RoyalBlue!10] (0,1.36) rectangle (1.11,2.84);
	\fill[RoyalBlue!10] (1.11,0) rectangle (2.22,1.36);
	\fill[Orange!10] (1.11,1.36) rectangle (2.22,2.84);		
	\fill[Orange!10] (2.22,0) rectangle (3.44,1.36);
	\fill[RoyalBlue!10] (2.22,1.36) rectangle (3.44,2.84);				
	
	\begin{axis}[legend pos=north east,    
		x label style={at={(axis description cs:1.1,+0.1)},},
		y label style={at={(axis description cs:-0.05,.5)},rotate=0},
		xlabel={$X_1$},
		ylabel={$X_2$},
		legend style={,draw=gray!20},
		ymin=0,ymax=2.1,
		xmin=0,xmax=3.1]

		\addplot[scatter,only marks,scatter src=explicit symbolic,
		scatter/classes={
			a={mark=*,RoyalBlue,thick},
			b={mark=square*,Orange,thick}
		}]
		table[x=x,y=y,meta=label,row sep=crcr]{
			x	y	label\\
			0.7	1.2	a\\
			0.3	1.8 a\\
			1.1	0.7 a\\
			1.5 0.2 a \\
			1.7 0.7 a\\
			2.2 1.25 a\\
			2.5 1.1 a \\
			2.7 1.4 a\\
			0.3  0.7 b\\
			0.7  0.3 b\\
			0.9  0.9 b\\
			1.1	1.5 b\\
			1.7 1.1 b\\
			2.3 0.8 b\\
			2.5 0.2 b\\
			2.8 0.9 b\\
		};
	\end{axis}
\end{scope}

}


\newcommand{\mysplitselection}{
	\draw[draw=none, fill=red!20] (1.7-0.5,0) rectangle (2.3-0.5,-0.05);

\draw[] (0,0) to (3,0) node[right] () {$X_1$};
\draw[] (0.5,0) -- +(0,-0.1) node[below] () {$1$}; 
\foreach \x in {-3,...,23}{
	\draw[] (0.5+\x*0.1,0) -- +(0,-0.05) node[below] () {}; 
}

\draw[] (1.5,0) -- +(0,-0.1) node[below] () {$2$}; 
\draw[] (2.5,0) -- +(0,-0.1) node[below] () {$3$}; 

\def\ls{+0.15};
\def\ts{+0.30};

\node[] at (0.2,\ls) () {$\mathbf{LS}$};
\draw[dashed] (0,0.225) -- (3,0.225);
\node[] at (0.2,\ts) () {$\mathbf{TS}$};

\draw[red, line width=2pt] (2-0.5,0) -- +(0,0.6) node[above] () {$\tau=2$}; 
\draw[red, line width=1pt] (1.7-0.5,0) -- +(0,0.5);
\draw[red, line width=1pt] (2.3-0.5,0) -- +(0,0.5);
\draw[<->, line width=2pt, red] (1.7-0.5,0.45) -- (2.3-0.5,0.45);

\draw[ForestGreen, line width=2pt] (1.8-0.5,0) -- +(0,0.7) node[above] () {$\tau=1.8$}; 
\draw[ForestGreen, line width=2pt] (2.2-0.5,0) -- +(0,0.7) node[above] () {$\tau=2.2$}; 

\node[circle, fill=RoyalBlue] () at (1.1-0.5,\ls) {};
\node[circle, fill=RoyalBlue] () at (1.5-0.5,\ls) {};
\node[circle, fill=RoyalBlue] at (1.7-0.5,\ls) () {};
\node[circle, draw=RoyalBlue, line width=2pt] at (1.9-0.5,\ts) () {};

\node[, fill=Orange,] at (2.3-0.5,\ls) {};
\node[, fill=Orange] at (2.5-0.5,\ls) {};
\node[, fill=Orange] at (2.8-0.5,\ls) {};
\node[, draw=Orange, line width=2pt] at (2.1-0.5,\ts) {};
}


\newcommand{\mybinaryclassificationtreeinterpretation}{
	\def\d{1}

\node[fill=RoyalBlue!20, path picture={
	\fill[Orange!20,draw=Gray!80,line width=1pt]
	($(path picture bounding box.north west)!.0!(path picture bounding box.north east)$)
	rectangle
	($(path picture bounding box.south west)!.5!(path picture bounding box.south east)$);
},draw=Gray!80,circle, line width=2pt,] (a) at (0,0) {};
\node[] at ($(a) + (+0.25,0)$) () {\textcolor{RoyalBlue}{8}};
\node[] at ($(a) + (-0.25,0)$) () {\textcolor{Orange}{8}};


\node[fill=RoyalBlue!20, path picture={
	\fill[Orange!20,draw=Gray!80,line width=1pt]
	($(path picture bounding box.north west)!.0!(path picture bounding box.north east)$)
	rectangle
	($(path picture bounding box.south west)!.66!(path picture bounding box.south east)$);
},draw=Gray!80,circle, line width=2pt] (b1) at (-1.2,-\d) {};
\node[] at ($(b1) + (+0.25,0)$) () {\textcolor{RoyalBlue}{3}};
\node[] at ($(b1) + (-0.25,0)$) () {\textcolor{Orange}{6}};

\node[fill=Orange!20, path picture={
	\fill[RoyalBlue!20,draw=Gray!80,line width=1pt]
	($(path picture bounding box.north west)!.0!(path picture bounding box.north east)$)
	rectangle
	($(path picture bounding box.south west)!.28!(path picture bounding box.south east)$);
},draw=Gray!80,circle, line width=2pt] (b2) at (+1.2,-\d) {};
\node[] at ($(b2) + (+0.25,0)$) () {\textcolor{Orange}{2}};
\node[] at ($(b2) + (-0.25,0)$) () {\textcolor{RoyalBlue}{5}};

\node[minimum size=1cm,draw,Orange, line width=2pt,fill=Orange!20] (c1) at (-1.8,-2*\d) {\textcolor{Orange}{3}};
\node[fill=RoyalBlue!20, path picture={
	\fill[Orange!20,draw=Gray!80,line width=1pt]
	($(path picture bounding box.north west)!.0!(path picture bounding box.north east)$)
	rectangle
	($(path picture bounding box.south west)!.5!(path picture bounding box.south east)$);
},draw=Gray!80,circle, line width=2pt] (c2) at (-0.6,-2*\d) {};
\node[] at ($(c2) + (+0.25,0)$) () {\textcolor{RoyalBlue}{3}};
\node[] at ($(c2) + (-0.25,0)$) () {\textcolor{Orange}{3}};

\node[minimum size=1cm,draw,RoyalBlue, line width=2pt,fill=RoyalBlue!20] (c3) at (+0.6,-2*\d) {\textcolor{RoyalBlue}{2}};
\node[fill=RoyalBlue!20, path picture={
	\fill[Orange!20,draw=Gray!80,line width=1pt]
	($(path picture bounding box.north west)!.0!(path picture bounding box.north east)$)
	rectangle
	($(path picture bounding box.south west)!.4!(path picture bounding box.south east)$);
},draw=Gray!80,circle, line width=2pt] (c4) at (+1.8,-2*\d) {};
\node[] at ($(c4) + (+0.25,0)$) () {\textcolor{RoyalBlue}{3}};
\node[] at ($(c4) + (-0.25,0)$) () {\textcolor{Orange}{2}};

\node[minimum size=1cm,draw,RoyalBlue, line width=2pt,fill=RoyalBlue!20] (d3) at (-1.2,-3*\d) {\textcolor{RoyalBlue}{3}};
\node[minimum size=1cm,draw,Orange, line width=2pt,fill=Orange!20] (d4) at (-0,-3*\d) {\textcolor{Orange}{3}};
\node[minimum size=1cm, draw,Orange, line width=2pt,fill=Orange!20] (d7) at (1.2,-3*\d) {\textcolor{Orange}{2}};
\node[minimum size=1cm,draw,RoyalBlue, line width=2pt,fill=RoyalBlue!20] (d8) at (2.4,-3*\d) {\textcolor{RoyalBlue}{3}};

\draw[Gray!80, line width=2pt,->] (a) -- node[left,pos=0.1,black] {} (b1);
\draw[Gray!80, line width=2pt,->] (a) -- node[right,pos=0.1,black] {} (b2);

\draw[Orange, line width=2pt,->] (b1) -- node[left,pos=0.1,black] {} (c1);
\draw[Gray!80, line width=2pt,->] (b1) -- node[right,pos=0.1,black] {} (c2);
\draw[RoyalBlue, line width=2pt,->] (b2) -- node[left,pos=0.1,black] {} (c3);
\draw[Gray!80, line width=2pt,->] (b2) -- node[right,pos=0.1,black] {} (c4);

\draw[RoyalBlue, line width=2pt,->] (c2) -- node[left,pos=0.1,black] {} (d3);
\draw[Orange, line width=2pt,->] (c2) -- node[right,pos=0.1,black] {} (d4);
\draw[Orange, line width=2pt,->] (c4) -- node[left,pos=0.1,black] {} (d7);
\draw[RoyalBlue, line width=2pt,->] (c4) -- node[right,pos=0.1,black] {} (d8);
}

\newcommand{\myxorvariancelsone}{
	\begin{scope}[scale=0.8,x=1cm,y=1cm]
		\draw[ForestGreen,line width=2pt] (3.37,2.80) circle (0.2);
		\begin{axis}[legend pos=north east,    
			x label style={at={(axis description cs:0.5,0.05)},},
			y label style={at={(axis description cs:0.05,.5)},rotate=0},
			xlabel={$X_1$},
			ylabel={$X_2$},
			legend style={,draw=gray!20},
			ymin=-0.1,ymax=1.1,
			xmin=-0.1,xmax=1.1]
			
			\addplot[scatter,only marks,scatter src=explicit symbolic,
			scatter/classes={
				a={mark=*,RoyalBlue,thick},
				b={mark=square*,Orange,thick}
			}]
			table[x=x,y=y,meta=label,row sep=crcr]{
				x	y	label\\
				0.49 0.49 b \\
				0.274406751964 0.357594683186 b \\
				0.272441591498 0.211827399669 b \\
				0.218793605631 0.445886500391 b \\
				0.191720759413 0.395862519041 b \\
				0.284022280547 0.462798319146 b \\
				0.0435646498508 0.51010919872 a \\
				0.389078375475 0.935006074123 a \\
				0.399579282108 0.730739681126 a \\
				0.0591372129345 0.819960510664 a \\
				0.472334458525 0.760924160875 a \\
				0.632277806052 0.387116844717 a \\
				0.784216974434 0.00939490021818 a \\
				0.806047861361 0.308466998437 a \\
				0.840910149552 0.179753950287 a \\
				0.848815597964 0.0301127358146 a \\
				0.835318934809 0.105191280537 a \\
				0.657714175462 0.681855385471 b \\
				0.719300756731 0.99418691903 b \\
				0.604438378047 0.580654758942 b \\
				0.62664580127 0.733155386428 b \\
				0.579484791823 0.555187570582 b \\
			};
			\legend{Class 1,Class 2}
		\end{axis}
	\end{scope}
}
\newcommand{\myxorvariancelstwo}{
	\begin{scope}[shift={(3.5,0)},scale=0.8,x=1cm,y=1cm]
		\draw[ForestGreen,line width=2pt] (3.37,2.80) circle (0.2);
		\begin{axis}[legend pos=north east,    
			x label style={at={(axis description cs:0.5,0.05)},},
			y label style={at={(axis description cs:0.05,.5)},rotate=0},
			xlabel={$X_1$},
			ylabel={$X_2$},
			legend style={,draw=gray!20},
			ymin=-0.1,ymax=1.1,
			xmin=-0.1,xmax=1.1]
			
			\addplot[scatter,only marks,scatter src=explicit symbolic,
			scatter/classes={
				a={mark=*,RoyalBlue,thick},
				b={mark=square*,Orange,thick}
			}]
			table[x=x,y=y,meta=label,row sep=crcr]{
				x	y	label\\
				0.51 0.51 b \\
				0.274406751964 0.357594683186 b \\
				0.272441591498 0.211827399669 b \\
				0.218793605631 0.445886500391 b \\
				0.191720759413 0.395862519041 b \\
				0.284022280547 0.462798319146 b \\
				0.0435646498508 0.51010919872 a \\
				0.389078375475 0.935006074123 a \\
				0.399579282108 0.730739681126 a \\
				0.0591372129345 0.819960510664 a \\
				0.472334458525 0.760924160875 a \\
				0.632277806052 0.387116844717 a \\
				0.784216974434 0.00939490021818 a \\
				0.806047861361 0.308466998437 a \\
				0.840910149552 0.179753950287 a \\
				0.848815597964 0.0301127358146 a \\
				0.835318934809 0.105191280537 a \\
				0.657714175462 0.681855385471 b \\
				0.719300756731 0.99418691903 b \\
				0.604438378047 0.580654758942 b \\
				0.62664580127 0.733155386428 b \\
				0.579484791823 0.555187570582 b \\
			};
			\legend{Class 1,Class 2}
		\end{axis}
	\end{scope}
}
\newcommand{\myxorvariancetreeone}{
	\begin{scope}[scale=0.3,x=1cm,y=1cm, minimum size=0.2cm]
		\def\d{3}
		\def\s{3}
		
		\node[draw=Gray!80,circle, line width=2pt,] (a) at (0,0) {$X_1$};
		
		\node[draw=Gray!80,circle, line width=2pt,fill=white] (b1) at (-\s,-\d) {$X_2$};
		\node[draw=Gray!80,circle, line width=2pt,fill=white] (b2) at (+\s,-\d) {$X_2$};
		
		\node[draw=Gray!40,, line width=2pt,fill=white] (c1) at (-\s-0.6*\s,-2*\d) {};
		\node[draw=Gray!40,circle, line width=2pt,fill=white] (c2) at (-\s+0.6*\s,-2*\d) {};
		\node[draw=Gray!40, line width=2pt,fill=white] (c3) at (+\s-0.6*\s,-2*\d) {};
		\node[draw=Gray!40,circle, line width=2pt,fill=white] (c4) at (+\s+0.6*\s,-2*\d) {};	
		
		\node[draw=Gray!40,, line width=2pt,fill=white] (d3) at (-\s+0.6*\s-0.3*\s,-3*\d) {};
		\node[draw=Gray!40,, line width=2pt,fill=white] (d4) at (-\s+0.6*\s+0.3*\s,-3*\d) {};	
		\node[draw=Gray!40,, line width=2pt,fill=white] (d7) at (+\s+0.6*\s-0.3*\s,-3*\d) {};
		\node[draw=Gray!40,, line width=2pt,fill=white] (d8) at (+\s+0.6*\s+0.3*\s,-3*\d) {};

		\draw[Gray!80, line width=2pt,->] (a) -- node[left,pos=0.1,black] {} (b1);
		\draw[Gray!80, line width=2pt,->] (a) -- node[right,pos=0.1,black] {} (b2);
		
		\draw[Gray!40, line width=2pt,->] (b1) -- node[left,pos=0.1,black] {} (c1);
		\draw[Gray!40, line width=2pt,->] (b1) -- node[right,pos=0.1,black] {} (c2);
		\draw[Gray!40, line width=2pt,->] (b2) -- node[left,pos=0.1,black] {} (c3);
		\draw[Gray!40, line width=2pt,->] (b2) -- node[right,pos=0.1,black] {} (c4);
		
		\draw[Gray!40, line width=2pt,->] (c2) -- node[left,pos=0.1,black] {} (d3);
		\draw[Gray!40, line width=2pt,->] (c2) -- node[right,pos=0.1,black] {} (d4);
		\draw[Gray!40, line width=2pt,->] (c4) -- node[left,pos=0.1,black] {} (d7);
		\draw[Gray!40, line width=2pt,->] (c4) -- node[right,pos=0.1,black] {} (d8);
		\node[draw=white, line width=2pt,fill=white] (e1) at (+\s-0.6*\s+0.3*\s-0.3*\s,-4*\d) {};	
	\end{scope}
}
\newcommand{\myxorvariancetreetwo}{
\begin{scope}[scale=0.3,x=1cm,y=1cm, minimum size=0.2cm]
	\def\d{3}
	\def\s{3}
	
	\node[draw=Gray!80,circle, line width=2pt,] (a) at (0,0) {$X_2$};
	
	\node[draw=Gray!80,circle, line width=2pt,fill=white] (b1) at (-\s,-\d) {$X_1$};
	\node[draw=Gray!80,circle, line width=2pt,fill=white] (b2) at (+\s,-\d) {$X_1$};
	
	\node[draw=Gray!40,, line width=2pt,fill=white] (c1) at (-\s-0.6*\s,-2*\d) {};
	\node[draw=Gray!40,, line width=2pt,fill=white] (c2) at (-\s+0.6*\s,-2*\d) {};
	\node[draw=Gray!40,circle, line width=2pt,fill=white] (c3) at (+\s-0.6*\s,-2*\d) {};
	\node[draw=Gray!40,circle, line width=2pt,fill=white] (c4) at (+\s+0.6*\s,-2*\d) {};

	\node[draw=Gray!40,, line width=2pt,fill=white] (d5) at (+\s-0.6*\s-0.3*\s,-3*\d) {};
	\node[draw=Gray!40,circle, line width=2pt,fill=white] (d6) at (+\s-0.6*\s+0.3*\s,-3*\d) {};
	\node[draw=Gray!40,, line width=2pt,fill=white] (d7) at (+\s+0.6*\s-0.3*\s,-3*\d) {};
	\node[draw=Gray!40,, line width=2pt,fill=white] (d8) at (+\s+0.6*\s+0.3*\s,-3*\d) {};
	
	\node[draw=Gray!40, line width=2pt,fill=white] (e1) at (+\s-0.6*\s+0.3*\s-0.3*\s,-4*\d) {};	
	\node[draw=Gray!40, line width=2pt,fill=white] (e2) at (+\s-0.6*\s+0.3*\s+0.3*\s,-4*\d) {};	
	
	\draw[Gray!80, line width=2pt,->] (a) -- node[left,pos=0.1,black] {} (b1);
	\draw[Gray!80, line width=2pt,->] (a) -- node[right,pos=0.1,black] {} (b2);
	
	\draw[Gray!80, line width=2pt,->] (b1) -- node[left,pos=0.1,black] {} (c1);
	\draw[Gray!80, line width=2pt,->] (b1) -- node[right,pos=0.1,black] {} (c2);
	\draw[Gray!80, line width=2pt,->] (b2) -- node[left,pos=0.1,black] {} (c3);
	\draw[Gray!80, line width=2pt,->] (b2) -- node[right,pos=0.1,black] {} (c4);
	
	\draw[Gray!40, line width=2pt,->] (c3) -- node[left,pos=0.1,black] {} (d5);
	\draw[Gray!40, line width=2pt,->] (c3) -- node[right,pos=0.1,black] {} (d6);
	\draw[Gray!40, line width=2pt,->] (c4) -- node[left,pos=0.1,black] {} (d7);
	\draw[Gray!40, line width=2pt,->] (c4) -- node[right,pos=0.1,black] {} (d8);
	
	\draw[Gray!40, line width=2pt,->] (d6) -- node[left,pos=0.1,black] {} (e1);
	\draw[Gray!40, line width=2pt,->] (d6) -- node[right,pos=0.1,black] {} (e2);

\end{scope}
}
\newcommand{\myxorvarianceimpone}{
	\begin{scope}[scale=0.8,x=1cm,y=1cm]
		\begin{axis}[legend pos= north west,
			y=1.5cm,
			x=2.5cm,
			ybar,
			bar width=20pt,
			xlabel={Features},
			ylabel={Sum of $\Delta i$},
			ymin=0,
			ytick=\empty,
			xtick={$X_1$,$X_2$},
			axis x line=bottom,
			axis y line=left,
			enlarge x limits=0.7,
			symbolic x coords={$X_1$,$X_2$},
			xticklabel style={anchor=base,yshift=-\baselineskip},
			nodes near coords={\pgfmathprintnumber\pgfplotspointmeta}
			]
			\addplot[fill=RoyalBlue] coordinates {
				($X_1$,0.00596978852304)
			};
			\addplot[fill=Orange] coordinates {
				($X_2$,1.988060423)
			};
			\legend{Root,Depth = 1}
		\end{axis}
	\end{scope}	
}
\newcommand{\myxorvarianceimptwo}{
	\begin{scope}[scale=0.8,x=1cm,y=1cm]
		\begin{axis}[legend pos= north east,
			y=1.5cm,
			x=2.5cm,
			ybar,
			bar width=20pt,
			xlabel={Features},
			ylabel={Sum of $\Delta i$},
			ymin=0,
			ytick=\empty,
			xtick={$X_1$,$X_2$},
			axis x line=bottom,
			axis y line=left,
			enlarge x limits=0.7,
			symbolic x coords={$X_1$,$X_2$},
			xticklabel style={anchor=base,yshift=-\baselineskip},
			nodes near coords={\pgfmathprintnumber\pgfplotspointmeta}
			]
			\addplot[fill=RoyalBlue] coordinates {
				($X_2$,0.00596978852304)
			};
			\addplot[fill=Orange] coordinates {
				($X_1$,1.988060423)
			};
			\legend{Root,Depth = 1}
		\end{axis}
	\end{scope}
}
\newcommand{\myensemblerf}{
	
  \node[] at (1.2,-0.3) (X1) {$\mathbf{x}$};
\begin{scope}[scale=0.5, xshift=0,yshift=0,every node/.append style={transform shape}]

	\mytreeinabox
	\draw[->,RoyalBlue, line width=2pt] (X1) -- (a); 
	\node[fill, RoyalBlue, circle, radius=0.1] at (\x,\y) (a) {};  
	\node[fill, RoyalBlue, circle, radius=0.1] at (\x+\xx,\y-1) (b2) {};
	\node[fill, RoyalBlue,circle, radius=0.1] at (\x+\xx-\xxx,\y-2) (c3) {};
	\node[fill, RoyalBlue,circle, radius=0.1] at (\x+\xx-\xxx+\xxxx,\y-3) (d4) {};
	\node[fill, RoyalBlue,circle, radius=0.1] at (\x+\xx-\xxx+\xxxx+\xxxx,\y-4) (e3) {};
	\draw[RoyalBlue, line width=2pt] (a) -- (b2) -- (c3) -- (d4) -- (e3);
	
	\coordinate (mid1) at (4,-6.5);
\end{scope}

\node[] at (5.7,-0.3) (X2) {$\mathbf{x}$};
\begin{scope}[scale=0.5, xshift=9cm,yshift=0,every node/.append style={transform shape}]
	\mytreeinaboxtwo
	\draw[->,RoyalBlue, line width=2pt] (X2) -- (a); 
	\node[fill, RoyalBlue, circle, radius=0.1] at (\x,\y) (a) {};  
	\node[fill, RoyalBlue, circle, radius=0.1] at (\x-\xx,\y-1) (b2) {};
	\node[fill, RoyalBlue,circle, radius=0.1] at (\x-\xx-\xxx,\y-2) (c3) {};
	\node[fill, RoyalBlue,circle, radius=0.1] at (\x-\xx-\xxx+\xxxx,\y-3) (d4) {};
	\draw[RoyalBlue, line width=2pt] (a) -- (b2) -- (c3) -- (d4);
	\coordinate (mid2) at (4,-6.5);
\end{scope}

\begin{scope}[scale=2, xshift=4.65cm,yshift=-22,every node/.append style={transform shape}]
	\node[] at (0,0) {$\dots$};
\end{scope}

\node[] at (11.3,-0.3) (X3) {$\mathbf{x}$};
\begin{scope}[scale=0.5, xshift=20.2cm,yshift=0,every node/.append style={transform shape}]
	\mytreeinaboxthree
	\draw[->,RoyalBlue, line width=2pt] (X3) -- (a); 
	\node[fill, RoyalBlue, circle, radius=0.1] at (\x,\y) (a) {};  
	\node[fill, RoyalBlue, circle, radius=0.1] at (\x+\xx,\y-1) (b2) {};
	\node[fill, RoyalBlue,circle, radius=0.1] at (\x+\xx+\xxx,\y-2) (c3) {};
	\node[fill, RoyalBlue,circle, radius=0.1] at (\x+\xx+\xxx+\xxxx,\y-3) (d4) {};
	\draw[RoyalBlue, line width=2pt] (a) -- (b2) -- (c3) -- (d4);
	
	\coordinate (mid3) at (4,-6.5);
\end{scope}
	
	\node[] at (2.7,-3) (y1) {$T_1(\mathbf{x})=\hat{y}_1$};
	\node[] at (2.2+4.5,-2.6) (y2) {$T_2(\mathbf{x})=\hat{y}_2$};
	\node[] at (2.7+10.1,-2.6) (y3) {$T_{N_T}(\mathbf{x})=\hat{y}_{N_T}$};	  
	 
	 \node[draw,rectangle,rounded corners=5pt,fill=ForestGreen!20] at (7.5,-4.5) (agg) {\textit{Aggregation}};
	 
	 \draw [->,line width=2pt] (mid1) to [out=270,in=180] node [above] {$\hat{y}_1$} (agg.west);
	 \draw [->,line width=2pt] (mid2) to [out=270,in=90] node [above,pos=0.7] {$\hat{y}_2$} (agg.north);
	 \draw [->,line width=2pt] (mid3) to [out=270,in=0] node [above] {$\hat{y}_{N_T}$} (agg.east);
	 
	 \draw[->,line width=2pt] (agg) -- +(0,-1);
	 
	 \draw[->,line width=2pt] (agg) -- +(0,-1) node[below] () {$\mathbf{T}(\mathbf{x}) = \hat{y}$};
}

\newcommand{\myensemblerfshifted}{
	\def\s{0.5cm}
	
	

		\begin{scope}[xshift=6cm,yshift=-5cm,every node/.append style={transform shape}]
		\fill[white,rounded corners=15pt] (0,0) rectangle (8,-6);
		\mytreeinaboxtwo
		\node[] at (\x+1.2,\y+0.6) (X) {$\mathbf{x}$};
		\draw[->,RoyalBlue, line width=2pt] (X) -- (a); 
		\node[fill, RoyalBlue, circle, radius=0.1] at (\x,\y) (a) {};  
		\node[fill, RoyalBlue, circle, radius=0.1] at (\x+\xx,\y-1) (b2) {};
		\node[fill, RoyalBlue,circle, radius=0.1] at (\x+\xx+\xxx,\y-2) (c3) {};
		\node[fill, RoyalBlue,circle, radius=0.1] at (\x+\xx+\xxx+\xxxx,\y-3) (d4) {};
		\node[fill, RoyalBlue,circle, radius=0.1] at (\x+\xx+\xxx+\xxxx+\xxxx,\y-4) (e3) {};
		\draw[RoyalBlue, line width=2pt] (a) -- (b2) -- (c3) -- (d4) -- (e3);
		\node[] at (\x+\xx+\xxx+\xxxx+\xxxx-1,\y-4-0.5) (Y) {$T_{N_T}(\mathbf{x})=\hat{y}_{N_T} $};
		
		\end{scope}

	\foreach \i in {3,2,1}
	{
		\begin{scope}[xshift=\i*\s,yshift=-\i*\s,every node/.append style={transform shape}]
			\draw[->, line width=2pt] (4,-5.8) -- +(0,-1.2cm+\i*\s);
			\fill[white,rounded corners=15pt] (0,0) rectangle (8,-6);
			\mytreeinabox
			
		\end{scope}
		
	}

	\begin{scope}[xshift=0,yshift=0,every node/.append style={transform shape}]
		
		
		
		\fill[white,rounded corners=15pt] (0,0) rectangle (8,-6);
		
		\mytreeinabox
		
		\node[] at (\x+1.2,\y+0.6) (X) {$\mathbf{x}$};
		\draw[->,RoyalBlue, line width=2pt] (X) -- (a); 
		\node[fill, RoyalBlue, circle, radius=0.1] at (\x,\y) (a) {};  
		\node[fill, RoyalBlue, circle, radius=0.1] at (\x+\xx,\y-1) (b2) {};
		\node[fill, RoyalBlue,circle, radius=0.1] at (\x+\xx-\xxx,\y-2) (c3) {};
		\node[fill, RoyalBlue,circle, radius=0.1] at (\x+\xx-\xxx+\xxxx,\y-3) (d4) {};
		\node[fill, RoyalBlue,circle, radius=0.1] at (\x+\xx-\xxx+\xxxx+\xxxx,\y-4) (e3) {};
		\draw[RoyalBlue, line width=2pt] (a) -- (b2) -- (c3) -- (d4) -- (e3);
		\node[] at (\x+\xx-\xxx+\xxxx+\xxxx+0.5,\y-4-0.5) (Y) {$T_1(\mathbf{x})=\hat{y}_1 $};
		
	\end{scope}
}

\newcommand{\mybootstrapfigure}{
		\def\x{0.85}
\def\y{-0.75}


\draw [
thick,
decoration={
	brace,
	raise=0.5cm
},
decorate
] (-0.5,0) -- (\x*10-0.5,0) 
node [pos=0.5,anchor=south,yshift=+0.55cm] {$N$ \textit{samples}}; 

\draw [
thick,
decoration={
	brace,
	mirror,
	raise=0.5cm
},
decorate
] (-0.5,\y*6) -- (\x*10-0.5,\y*6) 
node [pos=0.5,anchor=north,yshift=-0.55cm] {$N$ \textit{samples}}; 

\draw [
thick,
decoration={
	brace,
	mirror,
	raise=0.5cm
},
decorate
] (-0.5+\x*11,\y*6) -- (\x*15-0.5,\y*6) 
node [pos=0.5,anchor=north,yshift=-0.55cm] {\textit{oob samples}};

\node[] at (-1,0) {$\mathbf{LS}$};
\node[draw,black, line width=1pt,] (a) at (\x*0,0) {$x^1$};
\node[draw,black, line width=1pt,] (a) at (\x*1,0) {$x^2$};
\node[draw,black, line width=1pt,] (a) at (\x*2,0) {$x^3$};
\node[draw,black, line width=1pt,] (a) at (\x*3,0) {$x^4$};
\node[draw,black, line width=1pt,] (a) at (\x*4,0) {$x^5$};
\node[draw,black, line width=1pt,] (a) at (\x*5,0) {$x^6$};
\node[draw,black, line width=1pt,] (a) at (\x*6,0) {$x^7$};
\node[draw,black, line width=1pt,] (a) at (\x*7,0) {$x^8$};
\node[draw,black, line width=1pt,] (a) at (\x*8,0) {$x^9$};
\node[draw,black, line width=1pt,] (a) at (\x*9,0) {$x^{10}$};


\node[] at (-1,\y*2) {$\mathbf{LS}_1^B$};
\node[draw,black, line width=1pt,] (a) at (\x*0,\y*2) {$x^2$};
\node[draw,black, line width=1pt,] (a) at (\x*1,\y*2) {$x^{10}$};
\node[draw,black, line width=1pt,] (a) at (\x*2,\y*2) {$x^9$};
\node[draw,black, line width=1pt,] (a) at (\x*3,\y*2) {$x^5$};
\node[draw,black, line width=1pt,] (a) at (\x*4,\y*2) {$x^{4}$};
\node[draw,black, line width=1pt,] (a) at (\x*5,\y*2) {$x^1$};
\node[draw,black, line width=1pt,] (a) at (\x*6,\y*2) {$x^1$};
\node[draw,black, line width=1pt,] (a) at (\x*7,\y*2) {$x^4$};
\node[draw,black, line width=1pt,] (a) at (\x*8,\y*2) {$x^1$};
\node[draw,black, line width=1pt,] (a) at (\x*9,\y*2) {$x^{9}$};

\node[] at (\x*15+0.4,\y*2) {$\mathbf{LS}_1^{oob}$};
\node[draw,black, line width=1pt,] (a) at (\x*11,\y*2) {$x^3$};
\node[draw,black, line width=1pt,] (a) at (\x*12,\y*2) {$x^6$};
\node[draw,black, line width=1pt,] (a) at (\x*13,\y*2) {$x^7$};
\node[draw,black, line width=1pt,] (a) at (\x*14,\y*2) {$x^8$};		

\node[] at (-1,\y*3) {$\mathbf{LS}_2^B$};
\node[draw,black, line width=1pt,] (a) at (\x*0,\y*3) {$x^5$};
\node[draw,black, line width=1pt,] (a) at (\x*1,\y*3) {$x^{10}$};
\node[draw,black, line width=1pt,] (a) at (\x*2,\y*3) {$x^6$};
\node[draw,black, line width=1pt,] (a) at (\x*3,\y*3) {$x^5$};
\node[draw,black, line width=1pt,] (a) at (\x*4,\y*3) {$x^6$};
\node[draw,black, line width=1pt,] (a) at (\x*5,\y*3) {$x^4$};
\node[draw,black, line width=1pt,] (a) at (\x*6,\y*3) {$x^5$};
\node[draw,black, line width=1pt,] (a) at (\x*7,\y*3) {$x^9$};
\node[draw,black, line width=1pt,] (a) at (\x*8,\y*3) {$x^8$};
\node[draw,black, line width=1pt,] (a) at (\x*9,\y*3) {$x^{1}$};

\node[] at (\x*15+0.4,\y*3) {$\mathbf{LS}_2^{oob}$};
\node[draw,black, line width=1pt,] (a) at (\x*11,\y*3) {$x^2$};
\node[draw,black, line width=1pt,] (a) at (\x*12,\y*3) {$x^3$};
\node[draw,black, line width=1pt,] (a) at (\x*13,\y*3) {$x^7$};

\node[] at (-1,\y*4) {$\mathbf{LS}_3^B$};
\node[draw,black, line width=1pt,] (a) at (\x*0,\y*4) {$x^1$};
\node[draw,black, line width=1pt,] (a) at (\x*1,\y*4) {$x^2$};
\node[draw,black, line width=1pt,] (a) at (\x*2,\y*4) {$x^3$};
\node[draw,black, line width=1pt,] (a) at (\x*3,\y*4) {$x^4$};
\node[draw,black, line width=1pt,] (a) at (\x*4,\y*4) {$x^5$};
\node[draw,black, line width=1pt,] (a) at (\x*5,\y*4) {$x^6$};
\node[draw,black, line width=1pt,] (a) at (\x*6,\y*4) {$x^7$};
\node[draw,black, line width=1pt,] (a) at (\x*7,\y*4) {$x^8$};
\node[draw,black, line width=1pt,] (a) at (\x*8,\y*4) {$x^9$};
\node[draw,black, line width=1pt,] (a) at (\x*9,\y*4) {$x^{10}$};

\node[] at (\x*15+0.4,\y*4) {$\mathbf{LS}_3^{oob}$};
\node[draw,black, line width=1pt,] (a) at (\x*11,\y*4) {$x^1$};
\node[draw,black, line width=1pt,] (a) at (\x*12,\y*4) {$x^2$};
\node[draw,black, line width=1pt,] (a) at (\x*13,\y*4) {$x^3$};
\node[draw,black, line width=1pt,] (a) at (\x*14,\y*4) {$x^4$};		

\node[] at (-1,\y*5) {$\mathbf{LS}_4^B$};
\node[draw,black, line width=1pt,] (a) at (\x*0,\y*5) {$x^1$};
\node[draw,black, line width=1pt,] (a) at (\x*1,\y*5) {$x^1$};
\node[draw,black, line width=1pt,] (a) at (\x*2,\y*5) {$x^1$};
\node[draw,black, line width=1pt,] (a) at (\x*3,\y*5) {$x^8$};
\node[draw,black, line width=1pt,] (a) at (\x*4,\y*5) {$x^9$};
\node[draw,black, line width=1pt,] (a) at (\x*5,\y*5) {$x^{10}$};
\node[draw,black, line width=1pt,] (a) at (\x*6,\y*5) {$x^3$};
\node[draw,black, line width=1pt,] (a) at (\x*7,\y*5) {$x^8$};
\node[draw,black, line width=1pt,] (a) at (\x*8,\y*5) {$x^8$};
\node[draw,black, line width=1pt,] (a) at (\x*9,\y*5) {$x^{5}$};

\node[] at (\x*15+0.4,\y*5) {$\mathbf{LS}_4^{oob}$};
\node[draw,black, line width=1pt,] (a) at (\x*11,\y*5) {$x^2$};
\node[draw,black, line width=1pt,] (a) at (\x*12,\y*5) {$x^4$};
\node[draw,black, line width=1pt,] (a) at (\x*13,\y*5) {$x^6$};
\node[draw,black, line width=1pt,] (a) at (\x*14,\y*5) {$x^7$};		

\node[] at (-1,\y*6) {$\mathbf{LS}_5^B$};
\node[draw,black, line width=1pt,] (a) at (\x*0,\y*6) {$x^{10}$};
\node[draw,black, line width=1pt,] (a) at (\x*1,\y*6) {$x^6$};
\node[draw,black, line width=1pt,] (a) at (\x*2,\y*6) {$x^8$};
\node[draw,black, line width=1pt,] (a) at (\x*3,\y*6) {$x^2$};
\node[draw,black, line width=1pt,] (a) at (\x*4,\y*6) {$x^3$};
\node[draw,black, line width=1pt,] (a) at (\x*5,\y*6) {$x^{10}$};
\node[draw,black, line width=1pt,] (a) at (\x*6,\y*6) {$x^5$};
\node[draw,black, line width=1pt,] (a) at (\x*7,\y*6) {$x^{10}$};
\node[draw,black, line width=1pt,] (a) at (\x*8,\y*6) {$x^9$};
\node[draw,black, line width=1pt,] (a) at (\x*9,\y*6) {$x^{7}$};

\node[] at (\x*15+0.4,\y*6) {$\mathbf{LS}_5^{oob}$};
\node[draw,black, line width=1pt,] (a) at (\x*11,\y*6) {$x^1$};
\node[draw,black, line width=1pt,] (a) at (\x*12,\y*6) {$x^4$};
}

\newcommand{\oldmybaggingfigure}{
	
\node[] at (2,1) (Z) {$\mathbf{LS}_1^B$};
\draw[->,line width=2pt] (Z) -- + (0,-0.9);

\node[] at (6.39,1) (Z) {$\mathbf{LS}_2^B$};
\draw[->,line width=2pt] (Z) -- + (0,-0.9);

\node[] at (12.1,1) (Z) {$\mathbf{LS}_T^B$};
\draw[->,line width=2pt] (Z) -- + (0,-0.9);

\myrandomforest
}

\newcommand{\mybaggingfigure}{
	%
	%

	\coordinate (LS) at (7.5,4.5);
	\coordinate (LS1) at (2,2.8);
	\coordinate (LS2) at (6.5,2.8);
	\coordinate (LS3) at (12.1,2.8);
	\draw [->,line width=2pt] (LS.south) to [out=230,in=90] node [above] {} (LS1.north);
	\draw [->,line width=2pt] (LS.south) to [out=270,in=90] node [above] {} (LS2.north);
	\draw [->,line width=2pt] (LS.south) to [out=310,in=90] node [above] {} (LS3.north);
	
	
	\begin{scope}[xshift=6cm,yshift=6.5cm,every node/.append style={transform shape}]
		\mycoloredLS{RoyalBlue!20}{Orange!20}{ForestGreen!20}{Red!20}{DarkOrchid!20}
		\node[] at (1.45,0.3) () {$\mathbf{LS}$};
		\node[] at (-0.2,-0.5*0.4) () {\small $x^1$};
		\node[] at (-0.2,-1.5*0.4) () {\small $x^2$};
		\node[] at (-0.2,-2.5*0.4) () {\small $x^3$};
		\node[] at (-0.2,-3.5*0.4) () {\small $x^4$};
		\node[] at (-0.2,-4.5*0.4) () {\small $x^5$};
	\end{scope}
	
	\begin{scope}[xshift=0.5cm,yshift=2.8cm,every node/.append style={transform shape}]
		\mycoloredLS{RoyalBlue!20}{Orange!20}{ForestGreen!20}{ForestGreen!20}{DarkOrchid!20}
		\node[] at (2.7,0.3) () {$\mathbf{LS}^B_1$};
		\node[] at (-0.2,-0.5*0.4) () {\small $x^1$};
		\node[] at (-0.2,-1.5*0.4) () {\small $x^2$};
		\node[] at (-0.2,-2.5*0.4) () {\small $x^3$};
		\node[] at (-0.2,-3.5*0.4) () {\small $x^3$};
		\node[] at (-0.2,-4.5*0.4) () {\small $x^5$};
		\draw[->,line width=2pt]  (1.5,-2) -- +(0,-0.8); 
	\end{scope}
	
	\begin{scope}[xshift=5cm,yshift=2.8cm,every node/.append style={transform shape}]	
		\mycoloredLS{Orange!20}{Orange!20}{ForestGreen!20}{Red!20}{Red!20}
		\node[] at (2.7,0.3) () {$\mathbf{LS}^B_2$};
		\node[] at (-0.2,-0.5*0.4) () {\small $x^2$};
		\node[] at (-0.2,-1.5*0.4) () {\small $x^2$};
		\node[] at (-0.2,-2.5*0.4) () {\small $x^3$};
		\node[] at (-0.2,-3.5*0.4) () {\small $x^4$};
		\node[] at (-0.2,-4.5*0.4) () {\small $x^4$};
		\draw[->,line width=2pt]  (1.5,-2) -- +(0,-0.8); 
	\end{scope}
	
	\begin{scope}[xshift=10.6cm,yshift=2.8cm,every node/.append style={transform shape}]	
		\mycoloredLS{RoyalBlue!20}{RoyalBlue!20}{Orange!20}{DarkOrchid!20}{DarkOrchid!20}
		\node[] at (2.7,0.3) () {$\mathbf{LS}^B_{N_T}$};
		\node[] at (-0.2,-0.5*0.4) () {\small $x^1$};
		\node[] at (-0.2,-1.5*0.4) () {\small $x^1$};
		\node[] at (-0.2,-2.5*0.4) () {\small $x^2$};
		\node[] at (-0.2,-3.5*0.4) () {\small $x^5$};
		\node[] at (-0.2,-4.5*0.4) () {\small $x^5$};
		\draw[->,line width=2pt]  (1.5,-2) -- +(0,-0.8); 
	\end{scope}

	\myrandomforest

%
	
}

\newcommand{\myRSfigure}{
%
%

\coordinate (LS) at (7.05,5.6);
\coordinate (LS1) at (2,4.4);
\coordinate (LS2) at (6.4,4.4);
\coordinate (LS3) at (12.1,4.4);
\draw [->,line width=2pt] (LS.south) to [out=230,in=90] node [above] {} (LS1.north);
\draw [->,line width=2pt] (LS.south) to [out=270,in=90] node [above] {} (LS2.north);
\draw [->,line width=2pt] (LS.south) to [out=310,in=90] node [above] {} (LS3.north);


\begin{scope}[xshift=6cm,yshift=7cm,every node/.append style={transform shape}]
	\myLS{0}{0}{0}{0}{0}{0}{0}{0}{0}
	
\end{scope}

\begin{scope}[xshift=0.95cm,yshift=4cm,every node/.append style={transform shape}]
	\myLS{0}{0}{0}{0}{1}{0}{3}{0}{5}
\end{scope}

\begin{scope}[xshift=0.95cm,yshift=1.8cm,every node/.append style={transform shape}]
	\mypartialLS{$X_1$}{$X_3$}{$X_5$}
\end{scope}

\begin{scope}[xshift=5.35cm,yshift=4cm,every node/.append style={transform shape}]	
	\myLS{0}{0}{0}{0}{0}{2}{0}{4}{5}
\end{scope}

\begin{scope}[xshift=5.35cm,yshift=1.8cm,every node/.append style={transform shape}]
	\mypartialLS{$X_2$}{$X_4$}{$X_5$}
\end{scope}

\begin{scope}[xshift=11.05cm,yshift=4cm,every node/.append style={transform shape}]	
	\myLS{0}{0}{0}{0}{1}{2}{3}{0}{0}
\end{scope}

\begin{scope}[xshift=11.05cm,yshift=1.8cm,every node/.append style={transform shape}]
	\mypartialLS{$X_1$}{$X_2$}{$X_3$}
\end{scope}

\myrandomforest

\draw[->, line width=2pt] (2,2.6) -- +(0,-0.4);
\draw[->, line width=2pt] (6.4,2.6) -- +(0,-0.4);
\draw[->, line width=2pt] (12.1,2.6) -- +(0,-0.4);

\draw[->, line width=2pt] (2,0.4) -- +(0,-0.4);
\draw[->, line width=2pt] (6.4,0.4) -- +(0,-0.4);
\draw[->, line width=2pt] (12.1,0.4) -- +(0,-0.4);

}

\chapter{Introduction}\label{ch:introduction}

\section{Motivation}

From Alan Turing and Claude Shannon in the 1940's and the birth of computer science to the recent breakthroughs in the Internet of Things and in Artificial Intelligence (AI), the scientific and technological worlds of data collection and computing have tremendously evolved.  
In the last 20 years, this phenomenon has been accelerating significantly. There was a real "boom" in terms of new discoveries and breakthroughs. Among those recent and popular successes, many were made in the field of Machine Learning (ML). This field unifies all researches that aim at equipping machines (high performance computer grids, robots, cars, smart-phones, etc.) with the ability of learning a new task, and then improving their performances, by the mere fact of exploiting more data. Let us mention for example  the famous softwares of Google, AlphaGo and AlphaGoZero, that learned how to play and even become champion of the game of Go as well as several other highly complex boardgames. Progresses in ML are either dedicated to help researchers to exploit growing empirical datasets in their fields (e.g., physics, medicine, environmental sciences, social sciences, linguistics... ) or to improve day-to-day life. ML applications include sorting incoming e-mails, translating text (e.g., Google translate, DeepL), understanding and producing spoken language (e.g., Siri from Apple, Ok Google, Alexa from Amazon), and even self-driving cars and autonomous robots.

Following the main trend of the ML domain, those applications are constantly improved with the avowed goal of always achieving better performances and reducing the costs. In machine learning, significant improvements are usually achieved at the price of an increasing computational complexity and thanks to bigger datasets.  Currently, cutting-edge models built by the most advanced machine learning algorithms are commonly seen as black-boxes because they are either too complex to be comprehensible, or because they are kept secret by their owners. 

In the future, there will be countless new ML applications in which human health and life are at stake. Making a diagnosis (i.e., identification of a disease), estimating a prognosis (i.e., predicting the expected development of a disease), personalising a medical treatment and so many other medical decisions are already available or currently developed. It is obvious that one will not blindly believe everything coming out of those machines. The failure of Google Flu (predicting flu pandemics) illustrates that machines are not always infallible, but they may be of great help. To gain trust in machine learning based solutions, it is and will remain crucial to understand these algorithms and the reasons behind their decisions or predictions in a given application, e.g., examine choices of (military) autonomous drones and self-driving cars and knowing why Google Death forecasts someone's near death.
That is one of the reasons motivating a second trend in ML focusing on the interpretability of models rather than on their mere predictive and computational performances only (see, e.g., \citep{lipton2016mythos,doshi2017towards}). An interpretable model means that one understands the problem that is modelled and apprehends the underlying inference mechanism. Therefore, in some circumstances, the preference is for an interpretable model, that is not necessarily the most accurate or the fastest one but that manages to extract relevant knowledge from the data.

Performances and interpretability are typically not concomitant and a trade-off between those two properties is usually a desirable feature for a ML method.  Works are then made to improve the interpretability of existing black-box approaches while others focus on boosting the performances of already interpretable models.\\

Among the broad set of existing machine learning methods, this thesis only considers tree-based models. Within that kind of methods, single decision trees are very popular method and considered as highly interpretable. The model takes the form of a tree-structured graph representing a sequential reasoning to take a complex decision. 
The interest of the model is that it follows the reasoning everyone can make to handle difficult problems. 
However, this approach often provides highly variable models (because of the greedy nature of the approach) which in turn leads to rather modest levels of accuracy. In this thesis, special attention will be given to tree-based ensemble methods, also known as random forests (RF). While improving significantly the accuracy with respect to single trees, they unfortunately provide also much less interpretable models. With an ensemble of trees, many different explanations for a single decision are aggregated and interpreting the resulting prediction is not possible any more.

As mentioned, some efforts are usually made to interpret accurate models and, in this case, to recover some of the interpretability of a single decision tree. This can be done by identifying the constitutive elements (variables) of the model and their relative importances. For example, trying to predict someone's wine taste, we could determine that the wine colour is quite important and plays a decisive role in the wine taste discovery. In the literature on random forests, several different so-called `feature-importance' measures have been proposed in order to restore some interpretability, and also in order to help selecting relevant subsets of features, whenever this is useful.

Despite their success, RF methods and in particular importance measures derived from these models still contain some grey areas:
\begin{enumerate}
	\item Parameters of the methods have been usually studied with the scope of maximising the model performances. How do these parameters impact the quality of importance measures? Are optimal values for performances similar to those providing the best understanding of the problem?
	\item What is actually measured by an importance measure? Is it its usefulness in the model? 
	Does the importance evaluate the contribution of the variable in the model? How is defined the contribution of a variable? 
	\item Are those importance measures consistent? Are all variables equally treated when their importance is evaluated? 
	\item For a given importance measure, one can retrieve a numerical score for each variable. Is this sufficient to interpret all kinds of data structures, such as interacting features? 
\end{enumerate}
Along this thesis, we focus on answering some of those important questions in the light of our own work and of major contributions from the literature. We also propose some improvements to respond to some of the main limitations of the importance measures. 

\section{Outline of the manuscript}

The first part of this manuscript aims at summarising important notions about supervised learning, feature selection and tree-based methods. 
In particular, Chapter \ref{ch:background} describes the different natures and roles of variables and how they may interact together to form complex structures. Then, in the context of supervised learning, the interest of a variable is formalised by various notions of relevance and redundancy. This chapter is concluded by a description of feature selection problems and methods,  that aim at using a dataset to find the most relevant features in order to improve performances of machine learning models and/or to improve their interpretability. 
Chapter \ref{ch:trees} introduces tree-based models: from the single decision tree algorithm to state-of-the-art tree-based ensemble methods. Some key points or methods are highlighted for a better understanding of the subsequent chapters.\\

The second part of this manuscript is dedicated to the most popular importance measures derived from tree-based ensembles. 
In particular, Chapter \ref{ch:importances} reviews the main literature on that topic, with a focus on theoretical analyses. Chapter \ref{ch:mdi} then focuses on one subclass of these importance measures (known as the Mean Decrease of Impurity (MDI)) and studies it in greater details, both theoretically and practically.\\

The third and last part collects several contributions made in order to improve existing importance measures and/or in the context of some specific applications. 
In particular, Chapter \ref{ch:context} proposes an extension of the MDI importance measure to take into account different contexts in which the problem can be put, so as to provide further insight into the feature importances. Chapter \ref{ch:SRS} describes a new method using tree-based ensembles to perform feature selection under memory constraints. Finally, Chapter \ref{ch:connectomics} considers the network inference application. Its first part describes a tree-based solution and highlights some of the limitations of the method facing some challenges of network inference. The second part focuses on a network inference challenge and a non tree-based approach that we have designed in order to win the competition.

\section{Publications}

This dissertation summarises several contributions to tree-based importance measures. Publications that are directly related to this work include:

\begin{itemize}
\item \bibentry{louppe2013understanding}

\textit{This publication is of interest in Chapters \ref{ch:importances} and \ref{ch:mdi}.}

\item \bibentry{sutera2015simple}

\textit{Chapter \ref{ch:connectomics} is the result of that publication.}

\item \bibentry{sutera2016context}

\textit{Chapter \ref{ch:context} is the result of that publication.}

\item \bibentry{sutera2017simple}

\textit{Second version of \citep{sutera2015simple}.}

\item \bibentry{sutera2018random}

\textit{Chapter \ref{ch:SRS} is the result of the methodological part of that publication.} \textit{Theoretical results of that publication are also of interest in Chapters \ref{ch:importances} and \ref{ch:mdi}.}

\end{itemize}

During the course of this thesis, several fruitful collaborations have also
led to the following publications. These are not discussed within
this dissertation.

\begin{itemize}
\item \bibentry{taralla2016decision}
\item \bibentry{olivier2018phase}
\item \bibentry{wehenkel2018random}
\end{itemize}

\part{Background}
\chapter{Machine Learning and Feature Selection}\label{ch:background}

\begin{overview}
The goal of this chapter is to provide some general background in supervised machine learning and feature selection. We start with a short motivation about the role of machine learning in the context of artificial intelligence. Then we discuss data structures and supervised learning problems. The bulk of the chapter focuses on feature selection methods. Along the way, we also introduce terminology, and some related mathematical notions and notations.  
\end{overview}

\epi{Can machines think?}{Alan Turing, 1950}

\section{Machine learning vs Artificial Intelligence}

By studying the possibility of a machine to think, which led to his famous test to establish human level  intelligence of a machine, \cite{turing1950computing} laid the foundation stone for a new field of research, called \textit{Artificial Intelligence} (AI). Since then, in their quest to give a sort of intelligence to machines (and most prominently to computers), scientists have developed theories and algorithms to enable computers to learn from examples. This topic forms a sub-domain of AI called \textit{machine learning} (ML). The goal of machine learning is to allow a machine to progressively improve its ability to solve some tasks by exploiting some relevant \textit{data} collected over time. This contrasts with the habit of classical programming that implements computer programs based on a frozen set of human-based knowledge. Learning algorithms may actually allow a machine to discover knowledge that was missed by human experts or that is too complex to be discovered by them. Thus, the purpose of ML methods is dual. On the one hand, ML methods aim at producing models derived from data that allow for accurate \textit{predictions}, e.g. to take decisions or to guess not yet observed values. On the other hand, those models need to be \textit{interpretable} in order to help humans to explore data and understand complex systems. Both goals however equally require the same thing: (a lot of) data. That is why the next section presents the notion of data and its constituent elements known as observations and features.

\section{What is data?}

In the context of this thesis, a \textit{dataset} $\mathbf{D}$ is a collection of data and is organised as a set of $N$ observations $\{\mathbf{o}^i\}_{i=1}^N$. An \textit{observation} $\mathbf{o}^i$, also called \textit{sample} or \textit{example}, is a (line)vector of $p$ values $\mathbf{o}^i=(o^i_1,\dots,o_p^i)$, where the element $o^i_j$ corresponds to the value of the feature $j$. 
A \textit{feature} (or equivalently a \textit{variable}\footnote{Both terms will be used in this thesis without distinction.}) is a function taking as argument an object (belonging to some underlying set of possible objects) and whose values belong to a certain domain. 

A dataset of $N$ observations described by $p$ features is usually represented by a matrix of size $N \times p$.  

A \textit{large dataset}  refers to a dataset where $N$ is very large while a \textit{small dataset} refers to a dataset  where $N$ is small. A high-dimensional (respectively low-dimensional) dataset corresponds to the case where $p$ is very large (respectively small), while a big dataset corresponds to the case where $N \times p$ is very large. 
From a statistical viewpoint, the number $N$ of samples should ideally be (much) larger than the number $p$ of features in order to cover sufficiently well all possible combinations of features values. In practice, datasets with $N\ll p$ are often encountered and they indeed raise important challenges in the learning process \citep{kuncheva2018feature}.

\subsection{Nature of features} \label{sec:nature}

In machine learning, a feature encodes some observed information by taking a value from its domain. The number of possible values and the relationship between them allow to define several types of features, listed hereunder.

\begin{description}
\item[Continuous] A feature is \textit{continuous} if it can take any value within an interval of $\mathbb{R}$.
This results in an uncountable number of possibilities, and one can always find a new value between two other ones as close as they can be.
A continuous feature is also \textit{ordered}: its values are inherently numerical and hence they are (logically) ordered.

A few examples of continuous features are \textit{height} (domain is $\mathbb{R}^+$), \textit{weight} ($\mathbb{R}^+$), \textit{time} ($\mathbb{R}^+$),  \textit{speed} ($\mathbb{R}$), \textit{flow} ($\mathbb{R}$), \textit{correlation score} ($[-1,1]$), \textit{error rate} ($[0,1]$).

\begin{sidenote}{Rescaling of continuous feature values} 
	A continuous feature may be \textit{rescaled} without loss of information by mapping its domain to $[-1,1]$ or $[0,1]$ for instance. 
	 In the same machine learning application, ranges of different continuous features may vary widely from each other and some machine learning algorithms (e.g., artificial neural networks or support vector machines) might require to rescale all continuous features to the same range to work properly (e.g., by helping or speeding up optimisation) or to compare features with each others (e.g., in k-nearest-neighbours so that all features can contribute equally) . 
\end{sidenote}

\item[Discrete] 
A feature is \textit{discrete} when it takes its values in a set of at most a countable (and usually finite) number of values. Its values can either be numerical or categorical, ordered or not. The number of possible values defines the \textit{cardinality} of such a feature. A \textit{m-ary feature} (i.e., a feature of cardinality $m$) can take  $m$ different values. In particular, a feature of cardinality two is a \textit{binary feature} and its set of possible values is typically represented as $\{0,1\}$ or $\{-1,1\}$.

Usually, discrete features are divided in three sub-types:

	\begin{enumerate}[$\bullet$]
		\item A \textit{numerical discrete feature} takes on numerical values from a countable or finite subset of $\mathbb{N}$ or $\mathbb{R}$. Its values are thus naturally \textit{ordered}. 

		Examples of numerical discrete features usually refer to counts or proportions of indivisible elements: \textit{the number of children}, \textit{the number of passengers}, \textit{the proportion of expensive of cars}, \textit{etc.} 

		\item An \textit{ordinal discrete feature} takes on values that are not numerical but are still following a logical order.
		
		Examples of ordinal discrete features usually refer to a scale, a degree of magnitude and can often straightforwardly be replaced by numerical values if necessary: \textit{position} $\{first, second, third\}$, \textit{the degree of severity (of a car accident, a disease)} $\{low, intermediate, high\}$, \textit{the coffee strength} $\{mild, strong\}$, \textit{etc.}
		
\begin{sidenote}{Numerical encoding of categorical features} \label{sn:encoding}

Some methods (e.g., neural networks, support vector machines) are not able to handle features with non-numerical values (i.e., ordinal and categorical discrete features). Values of such features thus need to be encoded, converted into numerical values. 
	
With an ordered feature, one can easily attribute a numerical value to each possible values while respecting the logical order between them (e.g., $\{low, middle, high\}$ into $\{1,2,3\}$ and $low<high$ is preserved through $1 < 3$). 
Similarly, numerical values can be assigned to each class of a categorical feature. For example, let us take a categorical feature representing the eye colour with possible classes $\{blue, brown, green\}$. A classical numerical encoding would give $\{blue=1, brown=2, green=3\}$. However, this introduces an order between the classes that was not originally there. Having $blue$ eyes is not "lower" than having $brown$ eyes but assigned numerical values ($1$ and $3$) induce a spurious ordering.

Another encoding consists in replacing a categorical feature by several binary features $B$.
Two binary variables are enough to perfectly encode a variable with four different classes ($x$ binary variables give up to $2^x$ combinations). However, all binary variables are required to unambiguously retrieve the value. This is the binary equivalent of the classical encoding.

One-hot encoding associates one binary feature ${h}_{i}$ to each possible value of the original feature such that the binary value is equal to $1$ only if the original feature has the corresponding class (e.g., $h_1$ corresponding to $blue$). In this case, a larger number of binary features are required to represent all possible values of the original feature but there is no ordering implied by this encoding. 

A summary is made in Table~\ref{table:encoding-categorical-variables}.

\begin{center}
	\centering
	\begin{tabular}{c|c|cc|ccc}\hline
		Eye colour & Classical Enc. & \multicolumn{2}{c|}{Binary Enc.} & \multicolumn{3}{c}{One-Hot Enc.}\\
		& & $b_1$ & $b_2$ & $h_1$ & $h_2$ & $h_3$\\ \hline
		$blue$ & 	1 & 0 & 0 & 1 & 0 & 0 \\ \hdashline 
		$brown$ & 2 & 0 & 1       & 0 & 1 & 0 \\ \hdashline 
		$green$&   3 & 1 & 1 & 0 & 0 & 1 \\ \hline
	\end{tabular}
	\captionof{table}{Example of different encodings of a categorical variable}
	\label{table:encoding-categorical-variables}
\end{center}

\end{sidenote}
		
		\item A \textit{categorical discrete feature} (also known as \textit{nominal discrete feature}) takes on values from an finite set of elements without logical order. Values, referred to as \textit{classes} or \textit{categories}, are \textit{unordered}. 
		
		Examples are \textit{eye colour} taking values in $\{blue, brown, green\}$, \textit{mood} $\in \{happy, sad\}$, \textit{etc.}. While they are not ordered, they can however be encoded as numerical values if necessary (see side note on page \pageref{sn:encoding}).

	\end{enumerate}

\end{description}

When only the existence of a logical order between the feature values is of interest, continuous, numerical or ordinal discrete features are united as \textit{ordered} features and, conversely, categorical discrete features are \textit{unordered features}.

\subsection{Interactions between features}\label{sec:role}

Beyond their individual natures, the relations between features may also play a key role. 
Indeed, features can be seen as individual entities that carry some information (e.g., a value), but to consider features to their full extent, they need to be seen in the context of other features possibly interacting with them. 

In what follows, we first define a model of interacting features and then focus on the interactions between variables.

\subsubsection*{A (causal) model of interacting features}

Following \cite{pearl2009causality}'s definition, a \textit{(causal) model} is a triple $M=(U,V,F)$ \citep{white2011linking}, where \begin{itemize}
	\item  $U$ is a set of \textit{background} variables $\{u_1, \dots, u_m\}$ that are determined outside the model. Such variables are also called \textit{exogenous}. 
	\item $V$ is a set of variables $\{v_1, \dots, v_n \}$ that are determined within the model. Such variables are also called \textit{endogenous}.
	\item $F$ is a set of functions $\{f_1, \dots, f_n\}$ specifying how each endogenous variable is determined by other variables of the model. More precisely, each $f_i$  provides the value of $v_i$ given the values of a subset of all other variables $U\cup V^{-i}$ where $V^{-i}$ is the set $V$ without the variable $v_i$ (i.e., $V^{-i} = V\setminus \{v_i\}$). 
\end{itemize}

The structure of such a model may be represented in the form of a directed graph, where each vertex corresponds to one of the (exogenous or endogenous) variables, and where for each endogenous variable $v_{i}$ there is an edge pointing to its vertex from each one of the vertices corresponding to the other variables actually intervening in the function $f_{i}$. More details about the associated graph and uniqueness are given in \citep{pearl2009causality}.

	\begin{sidenote}{Extending the model} Some exogenous variables may become endogenous if one extends the (causal) model by adding new features ($\not\in V \cup U$). In some way, the characterisation associated to one feature will depend on the considered model.

	\end{sidenote}

Based on this characterisation, endogenous and exogenous variables are particularly interesting in terms of interactions between variables. In the following section, we characterise some of those interactions.

\subsubsection*{Direct, indirect and confounded interactions between variables}

From the previous section, it appears that variables may interact with each others. An endogenous variable $v_i$ is determined by (potentially) all other variables in $V^{-i}$. It means that some variables in $V^{-i}$ interact with $v_i$ to determine its value. Let us notice that exogenous variables interact with endogenous variables asymmetrically. Indeed, they can influence the value of variables in $V$ but their values can not be determined, as defined, by variables in $V$. 
On the other side, interactions implying endogenous variables can be symmetrical because one endogenous variable $v_i$ may influence and be influenced by the value of another endogenous variable $v_j$.
 
Let us extend the characterisation of interacting variables to include indirect influences of variables. 

\begin{description}
	\item[Intermediate variable] is a variable providing a (causal\footnote{Causality is not specifically addressed in this thesis (see reference text book \citep{pearl2009causality} for more details on causality). Many scientific fields, such as medicine or economy, are however interested in causal mechanisms and study the effect of intermediary variables and confounders (see, e.g.,  \citep{pearl2001direct,pearl2009simpson,deng2013identifiability,ananth2017confounding}).}) link between two other variables\footnote{Such a variable is also known as an intervening, mediating or intermediary variable.}. 
	Let us consider two variables $x$ and $y$. There may be a (causal) path going from $x$ (a cause) directly to $y$ (an effect), or indirectly through some intermediate step(s).  A variable (i.e., the intermediate step) on the pathway from $x$ to $y$ is an \textit{intermediate variable}. An intermediate variable \textit{mediates} the effect of $x$ on $y$. Figure \ref{fig:intermediate} shows an example of model with an intermediate variable $z$ between $x$ and $y$.
	Practically, the starting point (the source) $x$ may be a treatment or an exposure and the ending point $y$ may be a survival status or a disease \citep{deng2013identifiability}. For example, let us associate $x$ with a certain drug that affects the heartbeat, $y$ with the survival status of a patient. One may observe that the drug have a positive effect on the survival of the patient. However, the drug does not directly modify the survival status. Actually, the drug helps to regulate the heartbeat which in turn may improve the survival expectation of the patient. In this example, the heartbeat is an intermediate variable between the treatment and the outcome \citep{deng2013identifiability}. A more trivial example is the relationship between the income and the life expectancy. One can not actually "buy" a longer life but money can contribute to better medical care that help to live longer. In this case, the quality of medical care is the intermediate variable.
	
	From that, we can define the  \textit{direct effect} as the influence of $x$ on $y$ that is not mediated by other variables \citep{pearl2001direct}. Conversely, the \textit{indirect effect} is the influence of $x$ on $y$ that is mediated by other variables.

	\usetikzlibrary{positioning,arrows}
	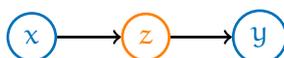
\begin{figure}[htbp]
		\centering
		\begin{tikzpicture}
			\node[circle,draw,RoyalBlue,line width=1pt] (a) at (0,0) {$x$};
			\node[circle,draw,orange,line width=1pt] (b) at (1.5,0) {$z$};
			\node[circle,draw,RoyalBlue,line width=1pt] (c) at (3,0) {$y$};
			\draw[->,line width=1pt] (a) -- (b);
			\draw[->,line width=1pt] (b) -- (c);
		\end{tikzpicture}
		\caption{Example of model with an intermediary variable $z$ in the pathway from the cause $x$ to the effect $y$.}
		\label{fig:intermediate}
	\end{figure}
	
	\begin{sidenote}{Direct and indirect effects simultaneously}
		There may be several paths from $x$ to $y$ and so $x$ may have simultaneously direct and indirect (through intermediates variables) effects on $y$. Figure \ref{fig:direct+indirect} illustrates two paths: a direct one and another that goes through an intermediate variable $z$. In this case, the indirect effect is meant to quantify the influence $x$ through indirect paths only. One may notice that this is not practically possible to block paths (i.e., holding a set of variables constant) such that the direct pathway would be circumvented. More thorough definitions of direct and indirect effects are given in \citep{pearl2001direct}. 		
		\begin{center}
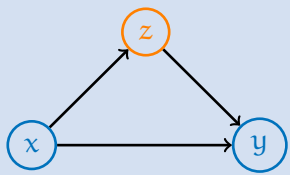

			\centering
			\begin{tikzpicture}
			\node[circle,draw,RoyalBlue,line width=1pt] (a) at (0,0) {$x$};
			\node[circle,draw,orange,line width=1pt] (b) at (1.5,1.5) {$z$};
			\node[circle,draw,RoyalBlue,line width=1pt] (c) at (3,0) {$y$};
			\draw[->,line width=1pt] (a) -- (c);
			\draw[->,line width=1pt] (a) -- (b);
			\draw[->,line width=1pt] (b) -- (c);
			\end{tikzpicture}
			\captionof{figure}{Framework where there is a direct path from $x$ to $y$ and an indirect path from $x$ going through $z$ to $y$.}
			\label{fig:direct+indirect}
		\end{center}
	\end{sidenote}

	\item[Confouding variable or confounder] is a (unstudied, exogenous) variable, say $z$, which influences two other variables $x$ and $y$ (conditionally or not to $x$), and tends to confound our reading of the effect of $x$ on $y$ \citep{pearl2009simpson,li2011ccsvm}. Figure \ref{fig:confounder} gives a possible model where $x$ and $y$ are confounded by a third variable $z$ that influences both $x$ and $y$ (conditionally or not to $x$). As illustrative example, let us examine an example proposed by \cite{kamangar2012confounding}: the risk of Down's syndrome\footnote{The Down's syndrome is a genetic disorder caused by the presence of an extra copy of human chromosome 21 \citep{patterson2009molecular}.} for a newborn baby. Let us associate $x$ with the parity (i.e., mother's number of pregnancies), $y$ with the Down's syndrome (i.e., whether or not the baby is affected by the syndrome), and $z$ with the maternal age (i.e., mother's age when giving birth to the baby). 
	Researches that only consider parity and the risk of Down's syndrome tend to show that the risk for a baby to be affected is associated with the number of his/her mother's pregnancies. For instance, the first-born has lower risk to be affected by the Down's syndrome than the fifth one. However, one needs to take the maternal age into account to determine the real association between the parity and the risk of Down's syndrome. The fifth children of a young 30-year-old mother has actually lower risk of getting affected than the first baby of a 40-year-old mother. In this case, the mother's age is a confounder\footnote{Let us note that a confounder is not on the path and can not be an intermediate variable. The number of pregnancies of a woman does not influence her age.} that accentuates the effect of parity on the risk of being affected by the syndrome.
	Many studies (e.g., in bioinformatics \citep{li2011ccsvm}, in ecology \citep{ewers2006confounding}, in medicine \citep{moller2000comet,del2012recommendations,ananth2017confounding,wu2018fair}) focus on the effect of confouding factors as a way of taking another look at previous observations.

\begin{figure}[htbp]
		\centering
		\begin{tikzpicture}
		\node[circle,draw,RoyalBlue,line width=1pt] (a) at (0,0) {$x$};
		\node[circle,draw,orange,line width=1pt] (b) at (1.5,1.5) {$z$};
		\node[circle,draw,RoyalBlue,line width=1pt] (c) at (3,0) {$y$};
		\draw[->,line width=1pt] (a) -- (c);
		\draw[->,line width=1pt] (b) -- (a);
		\draw[->,line width=1pt] (b) -- (c);
		\end{tikzpicture}
		\caption{Example of model with a confounding variable $z$ for $x$ and $y$.}
		\label{fig:confounder}
	\end{figure}
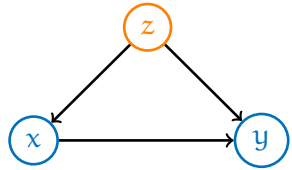

	When the confounding bias comes from contextual elements (e.g., the specific conditions in which an experiment is made), these circumstances are assumed to be encoded by a specific context variable, further referred as a \textit{contextual variable}.
	When taking into account the context, some feature dependencies may be accentuated or toned down while other may be unchanged being non-contextual (see side note about Simpson's paradox on page \pageref{sn:simpson}). 
		
		\end{description}

\begin{sidenote}{Simpson's paradox 
		\hypersetup{citecolor=white}\citep{simpson1951interpretation,pearl2009simpson}\hypersetup{citecolor=webgreen}
	} \label{sn:simpson}
	
	The Simpson's paradox \citep{simpson1951interpretation} refers to a setting where there is a trend in a given population $p$, and, at the same time, this trend disappears or reverses in every subpopulation of $p$. \cite{pearl2009simpson} formalises\footnote{\cite{pearl2009simpson} consciously chooses letters $C$ and $E$ to connote with \textit{cause} and \textit{effect}.} it as follows: 
	
	\begin{displayquote} \textit{
	``An event $C$ increases the probability of $E$ in a given population $p$ and, at the same time, decreases the probability of $E$ in every subpopulation of $p$. In other words, if $F$ and $\neg F$ are two complementary properties describing two subpopulations\footnote{Symbol $\neg$ is the logical \textit{not} operator. $\neg F$ refers to the complementary value of $F$, i.e., $not\ F$.}, we might well encounter the inequalities:
	\begin{align}
	&P(E|C) > P(E|\neg C) \label{eqn:simpson1}\\
	&P(E|C,F) < P(E|\neg C, F)\label{eqn:simpson2}\\
	&P(E|C,\neg F) < P(E|\neg C, \neg F)\label{eqn:simpson3}
	\end{align}
	[...] For example, if we associate $C$ with taking a certain drug, $E$ with recovery, and $F$ with being a female then - under the causal interpretation of Equations \ref{eqn:simpson2} and \ref{eqn:simpson3} - the drug seems to be harmful to both males and females yet beneficial to the population as a whole (Equation \ref{eqn:simpson1}). Intuition deems such a result impossible, and correctly so.''
}\end{displayquote}
	
	Such paradoxical setting - yet surprising - shows that this is possible to have a certain effect (or no effect in case of equality) without considering an external factor (here, $F$) and opposite effects when taking into account this factor (see Chapter~\ref{ch:context} in which variable $F$ will refer to some contextual conditions, i.e., a contextual variable).

\end{sidenote}

\section{Supervised learning}

In all generality, machine learning consists in learning models from data. This learning can be supervised when used data is \textit{labelled}, i.e., where each sample is associated with a label or a specific value. \textit{Supervised learning} thus focuses on learning a model from a learning set (i.e., labelled data) that can be used to predict the label of new (unseen) objects.

A \textit{learning set} $LS$ is a collection of input-output pairs  \citep{liu2012supervised,schrynemackers2015supervised}

$$\mathbf{LS} = \{(\mathbf{x}^1,y^1), \dots, (\mathbf{x^N},y^N)\} \in (\mathcal{X} \times \mathcal{Y})^N$$
where $\mathcal{X}$ and $\mathcal{Y}$ are respectively the input and output spaces, 
$\mathbf{x}^i = \{x^i_1,\dots,x^i_p\}\in \mathcal{X}$ is the vector of the $i^{th}$ sample made of $p$ input variable values and $y^i \in \mathcal{Y}$ is the corresponding output\footnote{Typically, there is only one output to predict as it will be the case in this thesis. However, sometimes applications require to predict several outputs simultaneously (e.g., the full state of a system in power system management). Learning with more than one output is called \textit{multi-output learning} (see, e.g., \citet{joly2017exploiting}).} (label).

From a learning set, a supervised learning algorithm $\mathcal{A}$ aims at finding a function $f: \mathcal{X} \rightarrow \mathcal{Y}$ that expresses the relationship between the inputs and the output. Such a model is able to provide a prediction $f(\mathbf{x})=\hat{y}$ approximating the true value $y$ for a new input vector $\mathbf{x}$.

Section \ref{sec:predictions} focuses on the prediction of a supervised model. {Section \ref{sec:modelassessment} defines the relevant notions of error for model assessment and selection.} Section \ref{sec:otherlearnings} briefly presents other forms of learning but only supervised learning is considered in the rest of this thesis.

\subsection{Predictions} \label{sec:predictions}

The \textit{output variable}, also known as a \textit{target variable}, can be either continuous or discrete and the learning algorithm must take this nature into account. The learnt model thus differs depending on the nature of the variable to predict.
The performance of the model (i.e., the quality of its predictions) is usually measured by means of a loss function $L : \mathcal{Y} \times \mathcal{Y} \rightarrow \mathbb{R}^+$ (see Section \ref{sec:modelassessment}). It provides a numerical score based on the comparison of the predictions with the targeted (actual) values.
\\ 

\noindent Two kinds of models are defined:
\begin{description}
	\item[A classification model] predicts the value of a discrete output. This model typically chooses its prediction from a set of pre-defined values (e.g., usually output values in the learning set) and is thus unable to predict an unseen value (e.g., predicting $yellow$ if only $blue$, $brown$ and $green$ have been observed in $\mathbf{LS}$). 
	A typical loss function for a classification model is the \textit{zero-one loss}  $L^{0-1}(f(\mathbf{x}),y)=\mathbb{1}(f(\mathbf{x})\ne y)$ which is equal to $1$ if the condition is verified (i.e., if the prediction is wrong and differs from the real value) and otherwise equal to zero.

	\item[A regression model] predicts the value of a continuous output. This model is usually able to produce new output values different from those found in the learning set (e.g., by averaging subsets of these latter values). 
	A typical loss function for regression is the \textit{squared error} (SE) $L^{se}(f(\mathbf{x}),y)=(y-f(\mathbf{x}))^2$ which computes the difference between the prediction and the real values exaggerating large deviations by taking the square of the difference. Another common loss function is the absolute error $L^{ae}(f(\mathbf{x}),y)=|y-f(\mathbf{x})|$.
		
\end{description}
The model and the loss functions must be chosen accordingly with the considered application. Let us note that when trying to predict the value of an ordered discrete variable, one can also use a regression model. Given the logical order between the values, even an unseen predicted value can be related to the others.

	\subsection{Model assessment and selection}\label{sec:modelassessment}

	In this section, we focus on the assessment of the prediction performance of a model $f$. Let us consider a set of input variables $X = \{x_1, x_2,\dots , x_p\}$ and an output variable $y$. We denote $P_{x_1,x_2,\dots,x_p,y}$, or equivalently $P_{X,y}$, the joint probability density of variables $x_1, x_2, \dots, x_p, y$ and $P_{y|x_1,x_2,\dots,x_p}$, or equivalently $P_{y|X}$, the conditional density of $y$ given variables $x_1, x_2, \dots, x_p$.
	
	Given a loss function $L$ (e.g., $L^{0-1}$,$L^{se}$,$L^{ae}$), the goal of supervised learning is to find a model $f$ which minimises the prediction error over an independent test set (usually drawn from the same distribution than the learning set), and defined as follows:
	\begin{definition}\label{def:generror}
		The \textbf{generalisation error} (a.k.a., \textbf{test error} or \textbf{expected prediction error}) is the expected\footnote{$\mathbb{E}_{X}\{f(X)\}$ denotes the expectation of a function $f(\cdot)$ with respect to the distribution $P_X$ of a set of random variables $X$ and defined as follows: $$\mathbb{E}_X\{f(X)\} = \sum_{x\in \mathcal{X}} P_X(x) f(x).$$} value of the loss function 
		\begin{eqnarray}\label{eqn:generror}
			Err(f) = \mathbb{E}_{X,y} \{ L(f(X),y)\}
		\end{eqnarray}
		over $X$ and $y$ randomly drawn from their joint distribution $P_{X,y}$.
	\end{definition}
	Given a model $\hat{f}_{\mathbf{LS}}$ learnt from a learning set $\mathbf{LS}$, its generalisation error is 
	\begin{eqnarray}\label{def:generrorls}
	Err (\hat{f}_{\mathbf{LS}}) = \mathbb{E}_{X,y} \{L(\hat{f}_{\mathbf{LS}}(X),y)\}.
	\end{eqnarray}
	Another quantity of interest is the \textit{expected generalisation error} $\mathbb{E}_{\mathbf{LS}}\{Err({\hat{f}_{\mathbf{LS}}})\}$ over random learning sets of size $N$. Typically, $Err (\hat{f}_{\mathbf{LS}})$ is used for model assessment and selection while $\mathbb{E}_{\mathbf{LS}}\{Err({\hat{f}_{\mathbf{LS}}})\}$ is useful to characterise a learning algorithm.\\
	
	From the distribution $P_{X,y}$ of a given problem and for a given loss function, it is actually possible the derive analytically and independently of any learning set the best possible model. First, let us rewrite the generalisation error by conditioning on $X$:
	\begin{eqnarray}Err(f) = \mathbb{E}_{X,y} \{L(X),y)\} = \mathbb{E}_{X} \{ \mathbb{E}_{y|X} \{L(f(X),y) \}\}.
	\end{eqnarray}
	From that, let us define the best possible model as follows:
	\begin{definition}
		The best possible model $f_B$, known as the \textbf{Bayes model}, that minimises $Err(f_B)$ is the one that minimises the inner expectation at each point $\mathbf{x}$ of the input space, that is:
		\begin{eqnarray}
		f_B(X) = \argmin_{y'\in \mathcal{Y}} \mathbb{E}_{y|X} \{L(y',y)\}.
		\end{eqnarray} 
	\end{definition}\noindent
	The generalisation error $Err(f_B)$ of the Bayes model is referred to as the \textit{residual error}.\\
	
	However, the joint distribution $P_{X,y}$ is usually unknown in practice and one needs to estimate the generalisation error from available data. Let us define the \textit{average prediction error} as the average loss over a set $\mathbf{LS}'$ of $N'$ observations (possibly different from the learning set $\mathbf{LS}$ used to learn $\hat{f}_{\mathbf{LS}}$), that is, 
	\begin{eqnarray}
	\widehat{Err}(\hat{f}_{\mathbf{LS}},\mathbf{LS}') = \dfrac{1}{N'} \sum_{(\mathbf{x}^i,y^i) \in \mathbf{LS}'}L(\hat{f}_{\mathbf{LS}}(\mathbf{x}^i),y^i).
	\end{eqnarray}
	When $\mathbf{LS}'$ is identical to the learning set $\mathbf{LS}$ used to learn the model, $\widehat{Err}(\hat{f}_{\mathbf{LS}},\mathbf{LS})$ is known as the \textit{training error} or \textit{empirical risk}. Another approach, known as the \textit{test set method}, consists in dividing the available learning set in two disjoint sets  $\mathbf{LS}_{train}$ (\textit{training set}) and $\mathbf{LS}_{test}$ (\textit{test set}) that are respectively use to learn the model and estimate the generalisation error\footnote{Let us note that $\widehat{Err}(\hat{f}_{\mathbf{LS}_{train}},\mathbf{LS}_{test})$ estimates the generalisation error conditional on the learning set while other approaches such as \textit{cross-validation} actually estimate the expected generalisation error.}. Similarly, the \textit{$K$ fold cross-validation} (CV) consists in dividing the available learning set in $K$ disjoint sets and learn in turn on $K-1$ folds and estimate the error on the remaining fold. When the number of folds $K$ corresponds to the number of samples, this method is then known as the \textit{leave-one-out cross validation}.

\subsection{Other forms of learning}\label{sec:otherlearnings}

Only one facet of machine learning is considered in this thesis, however many other forms of machine learning have been developed. This section is a brief summary of these other forms of learning.

 \paragraph{Unsupervised learning} differs from supervised learning by the absence of (labelled) outputs. Since, there are not outputs or targets to supervise the learning process, this part of machine learning focus on extracting informations from data (see, e.g., \textit{PCA}, \textit{ICA}, \textit{Gaussian mixture models}). Gathering similar samples together by making clusters is one way to get some information from unlabelled data. \textit{Clustering} is one of the most known unsupervised approaches and aims to gather similar samples into \textit{clusters} (see, e.g., \textit{k-means} and \textit{k-metroids}). 

\paragraph{Semi-supervised learning} is halfway between supervised and unsupervised learnings. In this case, some of the samples in the training data are not labelled. Semi-supervised techniques aim at using those additional unlabelled data to better characterise the underlying data distribution than what could be done using only labelled data. \textit{Active learning} is a particular case in which the learning algorithm can interact with the user in order to improve the quality of the learning process, e.g. by asking for a label.

\paragraph{Transfer learning} differs from other kinds of learning by the fact that the underlying distribution is not the same in the training data and in the testing data. Therefore, transfer learning mainly consists in learning a model and then apply it on a different but related application. 

\paragraph{Transductive learning} basically consists in transferring the information retrieved from labelled examples to unlabelled ones (see \citep{bousquet2002transductive} for details). The purpose is not to generate a model but only to label unlabelled samples. \textit{Transfer transductive learning} is a particular case considering transfer learning in a transductive setting \citep{arnold2007comparative,rohrbach2013transfer}.  In this setting, {the learning process can use labelled training data} but the test set is unlabelled on the target domain (which is different than the training domain as in the transfer learning) but can be seen during training.

\paragraph{Reinforcement learning} is apart from previously described forms of learning because it does not only rely on data. Indeed, the goal is not to discover an underlying distribution or mechanism but to determine an optimal control policy (i.e., the strategy that guides (future) chosen actions) from interaction with a system or from observations of a system \citep{ernst2005tree}.

\section{Feature selection for supervised learning}

Machine learning problems in bioinformatics, neuroimaging, engineering, psychology (and many others) have in common that their typical dimensions have increased very significantly within the last two decades \citep{guyon2003introduction,saeys2007review}. Such applications usually go with high-dimensional datasets that are characterised by a large number of input features. Exploring the whole input space in such applications often requires to consider hundreds of thousands of variables. However, many supervised learning techniques were originally designed to cope with only a few tens or hundreds of variables. Furthermore, most practical supervised learning algorithms decrease in
performances when facing many features that are not useful for the prediction of the output \citep{kohavi1997wrappers,blum1997selection}. 

Therefore, reducing the input data dimension, e.g., by selecting a subset of the original features \citep{liu2005toward}, has become a real prerequisite in such applications. In this context, the task of \textit{feature selection} mainly consists in finding as small as possible subsets of features that are \textit{sufficient} to build accurate predictors \citep{guyon2003introduction}), or alternatively in finding the subset of all \textit{informative} features, i.e., all those that are somehow related to the output variable \citep{nilsson2007consistent,paja2018decision}.

In addition to a dimensionality reduction, feature selection comes along with many potential benefits in terms of interpretability and  performances. 

\paragraph{Improving interpretability} Identifying and focusing on (the most) informative or useful features gives insight of the features involved in the underlying mechanism behind the data and facilitates the data understanding and data visualisation \citep{guyon2003introduction,saeys2007review}. 

Unlike feature extraction or construction techniques (e.g, principal component analysis \citep{jolliffe2011principal} or partial least squares \citep{wold1984collinearity}), feature selection preserves original features and thus resulting selected subsets of features remain interpretable by a domain expert \citep{kohavi1997wrappers,saeys2007review,wehenkel2018characterization}.

\paragraph{Increasing performances} 

The dimensionality reduction helps to overcome the curse of dimensionality and to avoid overfitting \citep{guyon2003introduction,saeys2007review}. Smaller data dimensions also reduce storage and computation requirements by providing faster and more cost-effective models \citep{guyon2003introduction, saeys2007review}.
In presence of many input features that are not necessary for predicting the output, performances of most practical algorithms decrease \citep{kohavi1997wrappers} and this can be toned down by removing irrelevant features (i.e., not related at all with the output). For example, feature selection often increases the prediction accuracy in supervised learning and often improves the quality of clustering in the case of unsupervised learning \citep{saeys2007review}.\\

So far, feature selection has been summarized as finding a subset of features. In what follows, we refine this concept by first characterising the relevance of a feature which quantifies the amount of information provided about the target variable. Then we define the usefulness of a feature which is its contribution for a given learning algorithm in prediction accuracy and therefore allows one to define what would be an optimal subset of features.
Then we describe the two flavours of feature selection mentioned in this introduction, namely the \textit{all-relevant} and the \textit{minimal-optimal} problems. While the first problem consists in finding all relevant features in the sense of all features that are somehow related with the output variable, the second problem aims at identifying the smallest subset that yields similar (or better) accuracy performances than any other subset of features. 

In the rest of this section, we review some concepts needed for our later developments while abstracting away from the fact that in practice we need to use a finite (and often small) learning set to identify suitable subsets of features for a given problem. We thus use concepts from probability theory and information theory, such as (conditional) independance, Markov boundary, and mutual information to characterize notions such as the relevance and optimality of input features and subsets of input features in the task of predicting the value of a particular output variable. 

The notions of Markov boundary and redundancy motivate the fact that all relevant features are not necessary to capture all the information about the target output. Some particular settings that limits the feature selection (or the interpretation that can be retrieved from) will also be reviewed in this chapter such as the multiplicity of Markov boundaries, the difficulty to distinguish direct from indirect effects as well as contextual effects.

\subsection{Relevance of features}\label{sec:relevance}

\subsubsection*{Notational conventions} 
\textit{In the present and subsequent sections we use uppercase letters to denote both individual random variables and sets of random variables, and we reserve lower case letters to denote values of variables or configurations of subsets of variables. In order to lighten the presentation, we assume that all considered random variables are discrete unless explicitly specified differently. We denote the joint probability density of variables $X, Y, Z$ by $P_{X,Y,Z}$ and its value for a combination of values of these variables by $P_{X,Y,Z}(x,y,z)$, and by $P_{X,Y |Z}$ (resp. $P_{X,Y |Z}(x,y|z)$) the conditional joint density of $X$ and $Y$ given $Z$ (respectively its value).}\\

Let us denote by $V$ the set of all original input variables, with $|V|=p$, and by
$Y$ the target output variable. Let $V^{-m}$ be the subset of $V$ excluding the input feature $X_m \in V$ (i.e.,$V^{-m}=V\setminus\{X_m\}$). 

One facet of feature selection is concerned about the
identification in $V$ of the (most) relevant variables.  
Many definitions of relevance have been proposed in the literature over the years \citep{gennari1989models,almuallim1991learning,kohavi1997wrappers,blum1997selection,guyon2006introduction} (usually incompatible with each other \citep{kohavi1997wrappers,kursa2011all}).  A common and popular set 
of relevance notions that we retain has been proposed by \cite{kohavi1997wrappers} and is as follows:
\begin{definition}\label{def:relevance}
	A variable $X_m\in V$ is \textbf{relevant} with respect to the output $Y$ iff there exists a subset $B\subset
	V^{-m}$ such that $X_m\nindep Y|B$. A variable is \textbf{irrelevant} if it is
	not relevant.
\end{definition}

In this definition the notation ``$X_m\nindep Y|B$'' indicates (probabilistic) conditional dependence and is equivalent (in the case of discrete variables) to saying that 
$$ \exists b, x_{m}, y : \begin{tabular}[t]{c} such that~ $P_{B}(b) > 0$ \mbox{~and~} \\ $P_{X_{m}, Y |B}(x_{m},y |b) \neq P_{X_{m} |B}(x_{m} |b)P_{Y |B}(y |b).$\end{tabular}$$

When the subset $B$ is empty, features are relevant by themselves:
\begin{definition}\label{def:marginallyrelevant}
	A variable $X_m \in V$ is \textbf{marginally relevant} with respect to the output $Y$ iff $X_m\nindep Y$.
\end{definition}

Relevant variables can be further divided into two categories:
\begin{definition}\label{def:strongrelevance}
	A variable $X_m$ is \textbf{strongly relevant}  with respect to the output $Y$ iff $Y\nindep
	X_m|V^{-m}$.
\end{definition}
\begin{definition}\label{def:weakrelevance}
	A variable $X_m$ is \textbf{weakly relevant} with respect to the output $Y$ if it is
	relevant but not strongly relevant.
\end{definition}

This definition is characterised by two degrees of relevance\footnote{\cite{kohavi1997wrappers} showed that earlier definitions were not consistent to identify relevance in the case of a Correlated XOR problem (i.e., where the target $Y$ is such that $Y = X_1 \oplus X_2$, where $\oplus$ denotes a logical XOR) with five boolean features $X_1,\dots,X_5$ and correlated/redundant features ($X_2$ and $X_4$ that are such that $X_4 = \overline{X_2}$) and that two degrees of relevance are required to achieve that. With respect to $Y$, $X_1$ is a strongly relevant feature, $X_2$ and $X_4$ are weakly relevant features due to their correlation/redundancy and $X_3$ and $X_5$ are irrelevant features.} in order to cope with particular settings such as features that are relevant but not marginally (e.g., a XOR problem) \citep{nilsson2007consistent}. Strongly relevant variables are thus variables that convey
information about the output that no other variable (or combination of
variables) in $V$ conveys \citep{nilsson2007consistent}. 
Figure \ref{fig:relevantsets} is a graphical representation of features in $V$ according to the type of relevance with respect to $Y$. It shows that the subset of relevant features is made of all weakly relevant features and all strongly relevant ones. Let us note that a system can be constructed so that it contains relevant but no strongly relevant features \citep{kursa2011all}.

Alternative, strictly equivalent, definitions of relevance can be formulated using the notion of conditional mutual informations\footnote{See Appendix \ref{app:information_theory}, for notations and definitions of several measures from information theory, including the conditional mutual information.} (see \citep{meyer2008information,louppe2013understanding}):

\begin{definition} \label{def:relevanceMI}
	A variable $X_m\in V$ is \textbf{relevant} to $Y$ iff there exists a subset $B\subset
	V$ such that $I(X_m;Y|B) > 0$. A variable is called \textbf{irrelevant} if it is
	not relevant.
\end{definition}

\begin{definition} \label{def:strongweakrelevanceMI}
	A variable $X_m$ is \textbf{strongly relevant} to $Y$ iff $I(X_m;Y|V^{-m}) > 0$. A variable $X_m$ is \textbf{weakly relevant} if it is
	relevant but not strongly relevant.
\end{definition}

The equivalence between these definitions and Definitions \ref{def:relevance}, \ref{def:strongrelevance}, and \ref{def:weakrelevance}, follows from the equivalence between zero (conditional) mutual information and (conditional) independence\footnote{$X\nindep Y|Z$ and  $X\indep Y|Z$ are equivalent to  $I(X;Y|Z) > 0$ and $I(X;Y|Z)=0$ respectively \citep{cover2012elements}.}.

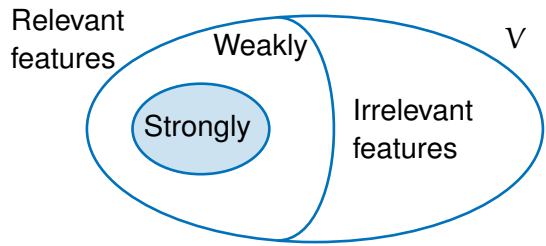
\begin{figure}[htbp]
	\centering
	\begin{tikzpicture}[line width=0.35mm, color=RoyalBlue]
	\draw [] (0,0) ellipse (3cm and 1.5cm);
	\coordinate (A) at (-0.5,+1.48);
	\coordinate (B) at (-0.5,-1.48);
	\coordinate (u) at (1,0);
	\coordinate (v) at (1,0);
	\draw[] (A) .. controls +(u) and +(v) .. (B);
	\draw [] (-1.5,0) ellipse (0.9cm and 0.6cm);
	\fill[opacity=0.2] (-1.5,0) ellipse (0.9cm and 0.6cm);
	\node[black, text width=1cm] at (3,1.2) {$V$};
	\node[black, text width=2cm] at (1.5,0) {Irrelevant features};
	\node[black, text width=2cm] at (-3,1.2) {Relevant features};
	\node[black, text width=2cm] at (-0.32,1.08) {Weakly};
	\node[black, text width=2cm] at (-1.25,0) {Strongly};
	
	\end{tikzpicture}
	\caption{Graphical decomposition of the set of input variables $V$ according to the feature relevance. The subset of relevant features can be further refined into two degrees of relevance: weak and strong relevances.}
	\label{fig:relevantsets}
\end{figure}

\subsubsection{On the quantitative measure of irrelevance}\label{sec:quantifirrelevance}

In relation to Definition \ref{def:relevanceMI}, several authors (eg., \citep{bell2000formalism,guyon2006introduction,meyer2008information}) proposed to use the notion of (conditional) mutual information to assess the level of relevance/irrelevance of a feature.

For example, \cite{guyon2006introduction} define a notion of ``approximate irrelevance'' as follows:
\begin{definition}\label{def:epsilon-irrelevance}
	A variable $X_m$ is \textbf{approximately irrelevant at level $\epsilon$} if for all\footnote{including the empty subset and the set $V^{-m}$ itself.} subsets of features $B \subseteq V^{-m}$, $I(X_m;Y|B) \leq \epsilon$.
\end{definition}

They further say that a variable $X_m$ is \textbf{surely irrelevant} if it is \textbf{approximately irrelevant at level $\epsilon=0$}. Notice that this notion is equivalent to the previously introduced notion of irrelevance (Definition \ref{def:relevanceMI}).

Let us finally mention that \cite{guyon2006introduction} claim that one single notion of (ir)relevance is enough if one simultaneously considers the notion of \textit{sufficient feature subset}, while \cite{kohavi1997wrappers} preferred  two degrees to characterise relevance. In addition to \citep{kohavi1997wrappers,guyon2006introduction}, we also further refer to \citep{bell2000formalism} for a review on relevance.

\subsection{Markov boundary} \label{sec:MB}

In this subsection, we introduce the notions of \textit{Markov blanket} and \textit{Markov boundary} that will be of interest in the rest of this chapter.

Let us consider a set of features $V$ and a target variable $Y$, Markov blanket and Markov boundaries are defined as follows \citep{pearl1988probabilistic,tsamardinos2003towards,statnikov2013algorithms}:

\begin{definition}\label{def:markovblanket}
	A \textbf{Markov blanket} of variable $Y$ relative to $V$ is a subset $M \subseteq V$ such $Y \indep V \setminus M | M$.
\end{definition}

\begin{definition}\label{def:markovboundary}
 A \textbf{Markov boundary} of variable $Y$ relative to $V$ is a Markov blanket of $Y$ relative to $V$ such that no proper subset of $M$ is also a Markov blanket of $Y$ relative to $V$.
\end{definition}

Trivially, the set of all input features $V$ is a Markov blanket of $Y$ and a given Markov blanket can be arbitrarily extended by adding features (even irrelevant ones with respect to $Y$) \citep{statnikov2013algorithms}. That is why minimal Markov blankets - Markov boundaries - are of greater interest in the context of feature selection\footnote{In computational biology, Markov boundaries are also known as (molecular) \textit{signatures}, which are minimal subset of features that are of best interest to predict the value (i.e., the phenotypic response) of a target variables\citep{statnikov2010analysis,geurts2011exploring}. In this context, the non-uniqueness of Markov boundaries is known as \textit{signature multiplicity}. Those two concepts are equivalent as it has been shown that maximally predictive and non-redundant molecular signatures are the Markov boundaries and vice-versa \citep{statnikov2010analysis}.} \citep{margaritis2000bayesian,tsamardinos2003towards,aliferis2003hiton,hardin2004theoretical,nilsson2007consistent,statnikov2013algorithms}. Figure \ref{fig:relevantsets-MB} shows how Markov boundaries relate with subsets of relevant features. As shown formally below, any  Markov blanket (and hence any Markov boundary) includes all strongly relevant features, and no Markov boundary can contain any irrelevant feature. On the other hand, some weakly relevant features may belong to some Markov boundaries. 

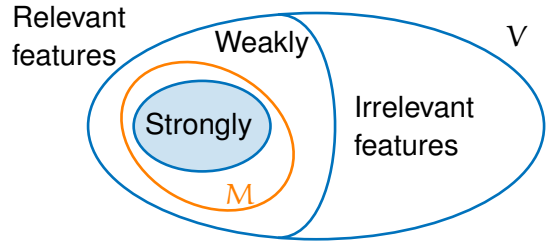
\begin{figure}[htbp]
	\centering
	\begin{tikzpicture}[line width=0.35mm, color=RoyalBlue]
	\draw [] (0,0) ellipse (3cm and 1.5cm);
	\coordinate (A) at (-0.5,+1.48);
	\coordinate (B) at (-0.5,-1.48);
	\coordinate (u) at (1,0);
	\coordinate (v) at (1,0);
	\draw[] (A) .. controls +(u) and +(v) .. (B);
	\draw [] (-1.5,0) ellipse (0.9cm and 0.6cm);
	\fill[opacity=0.2] (-1.5,0) ellipse (0.9cm and 0.6cm);
	\node[black, text width=1cm] at (3,1.2) {$V$};
	\node[black, text width=2cm] at (1.5,0) {Irrelevant features};
	\node[black, text width=2cm] at (-3,1.2) {Relevant features};
	\node[black, text width=2cm] at (-0.32,1.08) {Weakly};
	\node[black, text width=2cm] at (-1.25,0) {Strongly};
	\draw [rotate=60,color=orange] (-0.83,1.17) ellipse (0.9cm and 1.2cm);
	\node[orange, text width=2cm] at (-0.2,-0.9) {$M$};
	\end{tikzpicture}
	\caption{Graphical decomposition of the set of input variables $V$ according to the feature relevance. A Markov boundary $M$ in all generality gathers all strongly relevant features and some weakly relevant ones.}
	\label{fig:relevantsets-MB}
\end{figure}

A target variable $Y$ may have several Markov boundaries, for example because of redundancies between features \citep{statnikov2010analysis,geurts2011exploring,statnikov2013algorithms}. However, the intersection of all Markov boundaries always includes the set of strongly relevant features.

Indeed we have the following property \citep{tsamardinos2003towards}:

\begin{property}
	Let us consider a set $V$ of input features and an output $Y$. If $M$ is Markov blanket of $Y$, and $X_{m}$ is a strongly relevant feature, then $X_{m}\in M$. Therefore, any Markov boundary of $Y$, as well as the intersection of all these Markov boundaries, contains all strongly relevant features.
\end{property}

\begin{proof}	
	Consider  some subset $M$ of $V$ which is a Markov blanket of $Y$; thus 
	\begin{equation}Y \indep V \setminus M | M.\label{truc}\end{equation}
	Then consider some variable $X_m \in V\setminus M$; thus   \eqref{truc} may be rewritten as
	\begin{equation}Y \indep (\{X_m\} \cup (V \setminus (M \cup \{X_m\}))) | M.\label{baz}\end{equation}	
	The weak union property ($X\indep (Y\cup W) | Z \Rightarrow X\indep Y| (Z \cup W)$, see side note on page \pageref{sn:distributionproperties}) applied to \eqref{baz} yields 
	\begin{equation} 
 Y \indep X_m | M \cup (V\setminus (M\cup \{X_m\})), \mbox{~i.e.~} Y\indep X_m | V \setminus \{X_m\}.\end{equation}
Therefore $X_{m}$ is not strongly relevant.
\end{proof}
Furthermore, a Markov boundary of $Y$ never contains irrelevant features:
\begin{property}
Let us consider a set $V$ of input features and an output $Y$. If $M$ is a Markov boundary of $Y$, and $X_{i}$ is an irrelevant input feature, then $X_{i}\not\in M$. 
\end{property}
\begin{proof}
	Consider a Markov blanket $M$ of $Y$ containing an irrelevant variable $X_{i}$. Then, rewriting $M$ as $M^{-i} \cup \{X_i\}$, we have
	\begin{eqnarray}\label{eqn:demoirrMB-MBi}
	Y \indep V \setminus (M^{-i} \cup \{X_i\}) | (M^{-i} \cup \{X_i\}).
	\end{eqnarray}
Since $X_i$  is irrelevant with respect to $Y$, we also have
	\begin{eqnarray}\label{eqn:demoirrMB-irrelevance}
		Y \indep X_i | M^{-i}.
	\end{eqnarray}
Using the contraction property (i.e., $X\indep Y |Z \textit{ and } X\indep W |(Z\cup Y) \Rightarrow X \indep (Y\cup W)|Z$, see side note on page \pageref{sn:distributionproperties}) between Equations \ref{eqn:demoirrMB-irrelevance} and \ref{eqn:demoirrMB-MBi}, we thus have 
	\begin{eqnarray}
	&&Y \indep \{X_i\} \cup (V \setminus (M^{-i} \cup \{X_i\})) |M^{-i}\\
	\Leftrightarrow && Y \indep V \setminus M^{-i} |M^{-i}. \label{eqn:demoirrMB-imp}
	\end{eqnarray}
	Equation \ref{eqn:demoirrMB-imp} implies that $M^{-i}$ is also a Markov blanket of $Y$, so that $M$ can not be a Markov boundary of $Y$. 
\end{proof}

Following \citep{nilsson2007consistent}, let us define a {\textit{strictly
positive density} $P_{V}$ over the full set of input variables $V$ as a density such
that $P_{V}(v) > 0$ for all configurations $v$  of the variables
in $V$.} When {$P_{V}$}  is strictly positive\footnote{Equivalently, for any distributions satisfying the intersection property, there is a unique Markov boundary \citep{pearl1988probabilistic,statnikov2013algorithms}. Strictly positive distributions always verifies the intersection property \citep{nilsson2007consistent}.This also holds for faithful distributions (to some Bayesian network) satisfying the intersection property and being strictly positive \citep{tsamardinos2003towards,tsamardinos2003time,aliferis2010local,statnikov2013algorithms}.} (see side note on page \pageref{sn:distributionproperties}), the Markov boundary {of $Y$} is unique and it contains {only}  strongly relevant features \citep{tsamardinos2003towards,hardin2004theoretical,nilsson2007consistent,sutera2018random}.

Figure \ref{fig:relevantsets-MBfaithful} illustrates the relation between the concept of Markov boundary and relevance. Figure \ref{fig:relevantsets-MBfaithful(a)} shows that the (unique) Markov boundary coincides with the set of strongly relevant features when the distribution verifies the intersection property (proof in \citep[Theorem 10]{nilsson2007consistent}). Figure \ref{fig:relevantsets-MBfaithful(b)} illustrates the fact  that when the Markov boundary is not unique, the intersection of all Markov boundaries (or blankets) yields the set of strongly relevant features \citep{tsamardinos2003towards}.
The composition property prevents features to be irrelevant for some B but relevant when considered together for the same B.

\begin{figure}[htbp]
	\centering
	\subfloat[Distribution satisfying the intersection property]{
	\begin{tikzpicture}[line width=0.35mm, color=RoyalBlue]
	\draw [] (0,0) ellipse (3cm and 1.5cm);
	\coordinate (A) at (-0.5,+1.48);
	\coordinate (B) at (-0.5,-1.48);
	\coordinate (u) at (1,0);
	\coordinate (v) at (1,0);
	\draw[] (A) .. controls +(u) and +(v) .. (B);
	\draw [] (-1.5,0) ellipse (0.9cm and 0.6cm);
	\fill[opacity=0.2] (-1.5,0) ellipse (0.9cm and 0.6cm);
	\node[black, text width=1cm] at (3,1.2) {$V$};
	\node[black, text width=2cm] at (1.5,0) {Irrelevant features};
	\node[black, text width=2cm] at (-3,1.2) {Relevant features};
	\node[black, text width=2cm] at (-0.32,1.08) {Weakly};
	\node[black, text width=2cm] at (-1.25,0) {\hspace{0.1em}Strongly\\ \hspace{0.5em} \textcolor{orange}{= $M$}};
	
	\draw [orange,dashed] (-1.5,0) ellipse (0.9cm and 0.6cm);
	\end{tikzpicture}\label{fig:relevantsets-MBfaithful(a)}}
	\subfloat[Distribution \textbf{not} satisfying the intersection property]{
	\def\firstcircle {(0.2,0)    circle (1 cm)}
	\def\secondcircle{(45:0.5cm) circle (1 cm)}
	\def\thirdcircle {(0:0.5cm)  circle (1 cm)}
	\def\fourthcircle{(-45:0.5cm)circle (1 cm)}
	\begin{tikzpicture}[line width=0.35mm, color=RoyalBlue]
	\draw [] (0,0) ellipse (3cm and 1.5cm);
	\coordinate (A) at (-0.5,+1.48);
	\coordinate (B) at (-0.5,-1.48);
	\coordinate (u) at (1,0);
	\coordinate (v) at (1,0);
	\draw[] (A) .. controls +(u) and +(v) .. (B);
	\node[black, text width=1cm] at (3,1.2) {$V$};
	\node[black, text width=2cm] at (1.5,0) {Irrelevant features};
	\node[black, text width=2cm] at (-3,1.2) {Relevant features};
	
	\begin{scope}[scale=1.0]
	\node[black, text width=2cm] at (-0.18,1.12) {\small Weakly};
	\node[black, text width=2cm] at (-1.17,0) {\small Strongly};
	\end{scope}

	\begin{scope}[shift={(-1.82cm,0)},scale=0.83]
	\draw [] (0.33,0) ellipse (0.8cm and 0.5cm);
	\fill[opacity=0.2] (0.33,0) ellipse (0.8cm and 0.5cm);
	\draw[RedOrange] \firstcircle;
	\node[RedOrange, text width=2cm] at (-0.23,0) {\small $M_1$};
	\draw[orange] \secondcircle;
	\node[orange, text width=2cm] at (2.5,0.9) {\small $M_2$};
	\draw[ForestGreen] \thirdcircle;
	\node[ForestGreen, text width=2cm] at (2.75,0) {\small $M_3$};
	\draw[Orchid] \fourthcircle;
	\node[Orchid, text width=2cm] at (2.5,-0.9) {\small $M_4$};
	\end{scope}
	\end{tikzpicture}\label{fig:relevantsets-MBfaithful(b)}}	
	\caption{Correspondance between relevance and Markov boundaries in case of a distribution (a) satisfying the intersection property with a unique Markov boundary $M$ and (b) not satisfying the intersection property with four Markov boundaries $M_1,M_2,M_3,M_4$ whose the intersection is the set of strongly relevant features.}
	\label{fig:relevantsets-MBfaithful}
\end{figure}
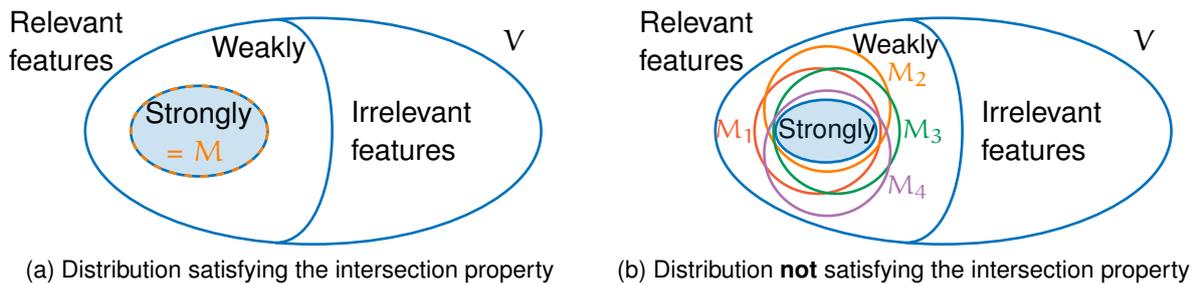

\begin{sidenote}{Distribution properties}
	\label{sn:distributionproperties}
Let $X$,$Y$,$Z$ and $W$ be any four subsets of features from $V$ and $T_i \in V$ be a single variable. Any distribution verifies the following properties \citep{pearl1988probabilistic,nilsson2007consistent,statnikov2013algorithms}:
\begin{enumerate}
	\item[$\bullet$] \textbf{Symmetry:} $X\indep Y | Z \Leftrightarrow Y \indep X | Z$,
	\item[$\bullet$] \textbf{Decomposition:} $X\indep(Y\cup W) | Z \Rightarrow X \indep Y | Z \textit{ and } X\indep W | Z$,
	\item[$\bullet$] \textbf{Weak union:} $X\indep (Y\cup W) | Z \Rightarrow X\indep Y| (Z \cup W)$,
	\item[$\bullet$] \textbf{Contraction:} $X\indep Y |Z \textit{ and } X\indep W |(Z\cup Y) \Rightarrow X \indep (Y\cup W)|Z$,
	\item[$\bullet$] \textbf{Self-conditioning:} $X\indep Z |Z$.	
\end{enumerate}
Strictly positive distributions (P) also satisfy \citep{pearl1988probabilistic, nilsson2007consistent,statnikov2013algorithms}:
\begin{enumerate}
	\item[$\bullet$] \textbf{Intersection:} $X\indep Y | (Z\cup W) \textit{ and } X\indep W|(Z\cup Y) \Rightarrow X \indep (Y \cup W)|Z$.
\end{enumerate}
\cite{nilsson2007consistent} also consider two additional classes of distributions:  strictly positive distributions that satisfy the composition property (PC):
\begin{enumerate}
	\item[$\bullet$] \textbf{Composition:} $X\indep Y|Z \textit{ and } X\indep W |Z \Rightarrow X\indep (Y\cup W)|Z$,
\end{enumerate}
and strictly positive distributions that satisfy both composition and weak transitivity (PCWT): 
\begin{enumerate}
	\item[$\bullet$] \textbf{Weak transitivity:} $X\indep Y |Z \textit{ and }  X\indep Y | R\cup \{T_i\} \Rightarrow X\indep \{T_i\} |Z      \textit{ and } \{T_i\} \indep Y | Z$.
\end{enumerate}
A more restricted class of distributions is strictly positive distributions that are DAG-faithful (PD) (i.e., faithful to some Bayesian network \citep{tsamardinos2003towards,statnikov2013algorithms}). PD is included in PCWT \citep{nilsson2007consistent} and verifies all its properties. While PD distributions offer some information about the causal structure (i.e., the Markov boundary of feature $Y$ is the set of direct causes, direct effects, and direct causes of direct effects (i.e., spouses) of $Y$), PCWT (including in particular jointly Gaussian distributions \citep{studeny2006probabilistic}) is claimed to be more realistic \citep{nilsson2007consistent}.
\end{sidenote}

\subsection{Redundancy} \label{sec:background-redundancy}

In many applications, and in particular in high-dimensional settings, the information about the output $Y$ to predict is shared and sometimes replicated among several input variables.
In neuroimaging for instance, one often observes a strong spatial correlation between voxels (i.e., pixels in 3D image) implying that neighbouring voxels are likely to be exchangeable when it comes to predict the output class \citep{wehenkel2018random}. The fact that the same information about the output is held by several features is called \textit{redundancy}. It can be \textit{total}, i.e., several features carry exactly the same information about the output and are exchangeable, or \textit{partial}, i.e., several features carry some of the same information about the target. 

From a feature selection point of view, features that share similar information about the target, such as neighbouring voxels in neuroimaging, are relevant but not necessarily useful together for a learning algorithm. Taking into account redundancy in feature selection may thus help to reduce the number of selected features.

In this section, we first review and refine formal definitions of feature redundancy
, propose a quantitative measure of redundancy
, and discuss the relation between redundancy and relevance 
and between redundancy and correlation 
.

\subsubsection{\citet{yu2004efficient}'s redundancy}\label{sec:yuetalredundancy}

Using the concept of Markov blankets, \citet{yu2004efficient} define the following notion of redundancy:

\begin{definition}\label{def:redundancy}
	Let us consider a subset $B\subset V$ of features and a variable $X_i \in V \setminus B$. We say that $X_i$ is {\textbf{redundant} to the set $B$ with respect to the target $Y$} iff (i) $X_i$ is weakly relevant with respect to $Y$ and (ii) there exists a subset $M \subseteq B$ such that $X_i \indep (\{Y\} \cup (V \setminus (M \cup \{X_i\})) | M$.
\end{definition}

Condition (i) excludes irrelevant features from consideration, since they are anyhow not useful to predict $Y$. Condition (ii) implies that $B$ contains a Markov blanket $M$ of variable $X_i$ relative to all other features including the target $Y$. This subset of variables can thus replace $X_i$ without loss of information, both about $Y$ and about any variables from $V$ not in the set $B$. According to this definition, and as expected, a strongly relevant feature can thus never be redundant to any subset because it conveys information about $Y$ that can not be found in other features and thus condition (ii) can not be satisfied. 

This definition was proposed by \citet{yu2004efficient} to identify features that can be safely ignored when $B$ is an intermediate approximate solution in the search for a Markov boundary of the target $Y$. A relaxed definition could have been adopted by changing condition (ii) simply into $Y\indep X_i | B$, but this would have excluded features $X_i$ that might bring complementary information about $Y$ with respect to $B$ when combined with some other features from $V\setminus B$.

\subsubsection{Total redundancy}\label{sec:totalredundancy}

\citet[Definition 7.1]{louppe2014understanding} defines
  \textit{totally redundant features} as pairs of features $X_i$ and
  $X_j$ such that\begin{eqnarray}\label{eqn:totallyred-gilles}
  H(X_i|X_j)=H(X_j|X_i)=0.
\end{eqnarray}
Note that an asymmetrical version\footnote{$X_i$ is defined as redundant with respect to $X_j$ if $H(X_i|X_j)=0$, which does not imply $H(X_j|X_i)=0$ and the redundancy of $X_j$ with respect to $X_i$.} of Equation \ref{eqn:totallyred-gilles} has also been proposed to define redundancy (e.g., \citep{meyer2008information}). One limitation of these definitions is that they do not involve the output variable $Y$. Therefore, based on \citep[Lemma 7.1]{louppe2014understanding}, let us define \textit{total redundancy with respect to $Y$} as follows:

\begin{definition}\label{def:totalred}
  $X_i$ and $X_j$ are \textbf{totally redundant} 
  variables with respect to the target $Y$  if for any conditioning set $B\subseteq V^{-i,j}(= V \setminus \{X_{i},X_{j}\})$, we have:
  \begin{eqnarray}\label{def:totalred:eqn1}
    Y\indep X_i|B\cup\{X_j\} &\mbox{ and }&Y\indep X_j|B\cup\{X_i\}
  \end{eqnarray}
\end{definition}

Equation \ref{def:totalred:eqn1} states that $X_i$ provides no additional information about the output once $X_{j}$ is given, whatever the context $B$, and vice versa. A direct consequence of this definition is that for all $B\subseteq V^{-i,j}$, we have\footnote{This is an immediate consequence of Equations \ref{eqn:proofpartially:red-chainrule1}-\ref{eqn:proofpartially:red-chainrule2} and the fact that $Y\indep X_i|B\cup\{X_j\} \Rightarrow I(X_i;Y|B\cup\{X_j\})=0$ and $Y\indep X_j|B\cup\{X_i\} \Rightarrow I(X_j;Y|B\cup\{X_i\})=0$.}
\begin{equation}\label{def:totalred:eqn2}
  I(X_i;Y|B) = I(X_j;Y|B),
\end{equation}
ie., $X_i$ and $X_j$ are equally informative about $Y$ in all circumstances. Total redundancy defines the ability of one feature to replace entirely the other in any context without loss of information about the output. Two totally redundant features are such that one is irrelevant iff the other is irrelevant. Obviously, none of them can be strongly relevant since Equation \ref{def:totalred:eqn1} for $B=V^{-i,j}$ gives $Y \indep X_i|V^{-i,j}\cup\{X_j\} \Leftrightarrow Y\indep X_i|V^{-i}$ and also $Y\indep X_j|V^{-j}$.

Note that Equation \ref{eqn:totallyred-gilles}, which implies that $X_{i}$ and $X_{j}$ are copies of each other, implies Definition \ref{def:totalred} (see \cite[Lemma 7.1]{louppe2014understanding} for a proof and \cite[Equations 3.7-3.9]{meyer2008information} for a proof in the asymmetrical case) but the converse is not true. Two features might be totally redundant with respect to the target, while not explaining perfectly each other. As defined, total redundancy and \citet{yu2004efficient}'s redundancy (Definition \ref{def:redundancy}) are also different concepts. Given two totally redundant features $X_i$ and $X_j$, we do not have necessarily that $X_i$ is redundant with respect to the subset $B=\{X_j\}$ according to Definition \ref{def:redundancy}. There might indeed exist a distinct feature $X_k\in V$ such that $X_i \nindep X_k| X_j$ and thus condition (ii) in Definition \ref{def:redundancy} might not be satisfied. It would be always satisfied however if using \citet{louppe2014understanding}'s definition of total redundancy (Equation \ref{eqn:totallyred-gilles}).

\subsubsection{Asymmetric and partial redundancies}\label{sec:partialredundancy}

In this section, we propose and discuss two relaxations of the definitions of redundancy given in the two previous sections.

First, while total redundancy as defined in Definition \ref{def:totalred} is symmetric, one can also define total redundancy in an asymmetric way:
\begin{definition}\label{asymmetrictotallyredundant}
$X_i$ is \textbf{totally redundant} to $X_j$ with respect to $Y$ if $\forall B\subseteq V^{-i,j}$, $X_i\indep Y|B\cup X_j$.
\end{definition}
In other words, $X_i$ is totally redundant to $X_j$ if it never brings any additional information about $Y$ when $X_j$ is known. $X_i$ and $X_j$ are thus totally redundant if they are totally redundant to each other.

\textit{Total} redundancy means that $X_i$ is always useless for predicting the output when $X_j$ is known. A notion of \textit{partial} redundancy could also be defined that relaxes this constraint.
\begin{definition}\label{partiallyredundant}
  $X_i$ is \textbf{partially redundant} to $X_j$ with respect to $Y$ if (i) $\exists B\subseteq V^{-i,j}$ such that $X_i\nindep Y|B\cup X_j$ and (ii) $\forall B\subseteq V^{-i,j}$ such that $X_i\nindep Y|B\cup X_j$:
  \begin{equation}\label{eqn:partialredundancy}
    I(X_i;Y|B) > I(X_i;Y|B\cup \{X_j\}).
  \end{equation}
\end{definition}
Condition (i) excludes $X_i$ from being totally redundant to $X_j$. Condition (ii) means that the information that $X_i$ brings about the output is always reduced when $X_j$ is known. Having instead $I(X_i;Y|B) < I(X_i;Y|B\cup \{X_j\})$ would mean that $X_i$ is more complementary than redundant to $X_j$. Note that the equality is impossible since $X_i\nindep Y|B\cup X_j$ implies that $I(X_i;Y|B\cup \{X_j\})>0$.

Interestingly, Definition \ref{partiallyredundant} implies that $X_i$ and $X_j$ are both relevant to $Y$. 
 \begin{property}
 	If $X_i$ is partially redundant to $X_j$ with respect to $Y$, then $X_i$ and $X_j$ are both relevant with respect to the output $Y$. 
  \end{property}
\begin{proof}
  By definition of partial redundance, there exists at least one $B$ such that $X_i\nindep Y|B\cup X_j$. For one such $B$, condition (ii) implies that:
	\begin{eqnarray} \label{eqn:proofpartially:red-partially}
	I(X_i;Y|B) > I(X_i;Y|B\cup \{X_j\}) > 0.
	\end{eqnarray}
	From Equation \ref{eqn:proofpartially:red-partially}, we directly have that
	\begin{eqnarray}
		I(X_i;Y|B) > 0
	\end{eqnarray}
	implying that $X_i$ is relevant with respect to $Y$.
	
	Then, the first inequality of Equation \ref{eqn:proofpartially:red-partially} is equivalent to 
	\begin{eqnarray} \label{eqn:proofpartially:red-first}
	I(X_i;Y|B) - I(X_i;Y|B\cup \{X_j\}) > 0.
	\end{eqnarray}
	The chain rule ($I(X_1,X_2,\dots,X_n;Y) = \sum_{i=1}^n I(X_i;Y|X_{i-1},\dots,X_1)$) applied to the mutual information between both features $X_i,X_j$ and $Y$ yields
	\begin{eqnarray}
	I(X_i,X_j;Y|B) &=& I(X_i;Y|B) + I(X_j;Y|B\cup\{X_i\}) \label{eqn:proofpartially:red-chainrule1}\\
	&=& I(X_j;Y|B) +   I(X_i;Y|B\cup\{X_j\}) \label{eqn:proofpartially:red-chainrule2}
	\end{eqnarray}
	where Equations \ref{eqn:proofpartially:red-chainrule1} and \ref{eqn:proofpartially:red-chainrule2} depend on the order in which $X_i$ and $X_j$ are used.
	By rearranging terms in \ref{eqn:proofpartially:red-chainrule1} and \ref{eqn:proofpartially:red-chainrule2}, we have 
	\begin{eqnarray}
	I(X_i;Y|B)  - I(X_i;Y|B\cup\{X_j\})  =  I(X_j;Y|B)   -  I(X_j;Y|B\cup\{X_i\}).
	\end{eqnarray}
	Since the left member is strictly positive given Equation \ref{eqn:proofpartially:red-first}, we thus have
	\begin{eqnarray} \label{eqn:proofpartially:posxj}
		I(X_j;Y|B)   -  I(X_j;Y|B\cup\{X_i\}) > 0
	\end{eqnarray}
	which implies that  $I(X_j;Y|B) > 0$ because $I(X_j;Y|B\cup\{X_i\}) \ge 0$ (positivity of conditional mutual information) and $I(X_j;Y|B)  > I(X_j;Y|B\cup\{X_i\})$. 
	Therefore $X_j$ is also relevant with respect to $Y$.  
	
\end{proof}

The proof of the previous theorem shows that Equation \ref{eqn:partialredundancy} is equivalent to Equation \ref{eqn:proofpartially:posxj}. In consequence, if $X_j$ reduces the information brought by $X_i$ about $Y$, then $X_i$ also reduces the information brought by $X_j$ about $Y$. Nevetheless, partial redundancy is not symmetric because the sets $B$ such that $X_i\nindep Y|B\cup X_j$ do not necessarily coincide with the sets $B$ such that $X_j\nindep Y|B\cup X_i$.

\subsubsection{Quantitative measure of redundancy}\label{sec:quantitative}


A measure of redundancy among $p$ random variables $X_1,\dots,X_p$ can be defined as follows (see, e.g.,  \citep{mcgill1954multivariate,watanabe1960information,wienholt1996determine,jakulin2003quantifying,meyer2008information}):
	\begin{eqnarray}\label{eqn:redmcgill}
	R(X_1;X_2;\dots;X_p) = \sum_{i=1}^{p} H(X_i) - H(X_1,X_2,\dots,X_p)
	\end{eqnarray}
where $H(X_i)$ and $H(X_1,X_2,\dots,X_p)$ are respectively the entropy of $X_i$ and the joint entropy of $X_1, X_2, \dots, X_p$ (see Appendix \ref{app:information_theory}). However, like total redundancy (Definition \ref{def:totalred}), this measure does not involve the output variable
 
	$Y$ \citep{meyer2008information}. Therefore, following the use of $I(X_m;Y|B)$ to quantify feature relevance (see Section \ref{sec:quantifirrelevance}), one could use similarly \textit{multivariate mutual information} \citep{mcgill1954multivariate} to quantify redundancy.

\textit{Multivariate mutual information} is usually defined as follows : 
\begin{eqnarray} \label{eqn:multivariateMI}
I(X;Y;Z) &= & I(X;Y,Z) -  I(X;Z|Y) - I(X;Y|Z)
\end{eqnarray}
It can be shown that $I(X;Y;Z)$ is symmetric with respect to a permutation of the roles of $X,Y,Z$ (e.g., $I(X;Y;Z) = I(X;Z;Y)$) and, applying the chain rule on $I(X;Y,Z)$, that 
\begin{eqnarray}\label{eqn:decomp-multivariateMI}
I(X;Y;Z)=I(X;Y)-I(X;Y|Z). \end{eqnarray}
Unlike standard (conditional) mutual information, $I(X;Y;Z)$ can be negative as $I(X;Y)$ can be increased by conditioning on $Z$. \cite{mcgill1954multivariate} sees $I(X;Y;Z)$ (Equation~\ref{eqn:decomp-multivariateMI}) as the the gain (or loss) of common information between two variables (i.e., $X$ and $Y$) due to the additional knowledge of a third one (i.e., $Z$). A negative value is therefore due to an increase of the dependence between $X$ and $Y$  knowing $Z$. Noting the symmetry, $I(X;Y;Z)$ (Equation~\ref{eqn:multivariateMI}) can also be seen intuitively as a generalisation of the mutual information common to three random variables \citep{cover2012elements}.

The degree of redundancy between two features (in a given context $B$)
could then be defined as follows:
\begin{definition}\label{def:measureofred}
	For a given conditioning set $B \subseteq V^{-i,j}$, the \textbf{degree of redundancy} between $X_i$ and $X_j$ with respect to $Y$ is measured by
	\begin{eqnarray}
	I(X_i;X_j;Y|B) = I(X_i;Y|B) - I(X_i;Y|B\cup\{X_j\}).
	\end{eqnarray}	
\end{definition}

$I(X_i;X_j;Y|B)$  has several desirable properties as a measure of the degree of redundancy:
\begin{itemize}
\item It is positive as soon as $I(X_i;Y|B)>I(X_i;Y|B\cup\{X_j\})$ or equivalently $I(X_i;Y|B)>I(X_i;Y|B\cup\{X_j\})$, which corresponds precisely to condition (ii) of partial redundancy (Definition \ref{partiallyredundant}).
\item It is equal to zero when $I(X_i;Y|B)=I(X_i;Y|B\cup \{X_j\})$, which corresponds to $X_j$ not impacting the information brought by $X_i$ about the output.
\item It is negative when $X_i$ and $X_j$ are complementary. For instance, in the case of a XOR problem,  $X_i$ and $X_j$ are marginally irrelevant but together perfectly explain the output $Y$. Mathematically, we have in this case $I(X_i;Y) = I(X_j;Y) = 0$ and $I(X_i;Y|X_j) = I(X_j;Y|X_i) = H(Y)$, which is strictly greater than 0 unless $Y$ is constant. Therefore, $I(X_i;Y|X_j)  > I(X_i;Y)$ and thus $I(X_i;X_j;Y)<0$.
\item It is maximal and equal to $I(X_i;Y|B)=I(X_j;Y|B)$ when $X_i$ and $X_j$ are totally redundant, as in this case $I(X_i;Y|B\cup\{X_i\})=I(X_i;Y|B\cup\{X_i\})=0$.
\end{itemize}

Note that several authors have proposed to use the opposite of Equation \ref{def:measureofred} to quantify the synergy or the complementarity between two features, which is indeed the opposite of redundancy. This measure can also be generalised to more than two features. See, e.g., \citep{meyer2013information} for a review of these measures.

\subsubsection{Redundancy and relevance}\label{sec:redundancyrelevance}

Like relevance, redundancy characterises the interest of (de)selecting features. Depending on how the feature selection problem is formulated (see Section \ref{sec:fsproblems}), it is often desirable not to select totally redundant features that convey the exact same information about the output as other features. By definition, strongly relevant features always contain some unique information and thus only weakly relevant features can be considered as (totally) redundant with respect to some other features. Figure \ref{fig:relevantsets-red} (adapted from \cite{yu2004efficient}) illustrates that input features can be divided into four categories: irrelevant, strongly relevant, non-redundant and redundant weakly relevant features. Non-redundant and redundant features are such that the redundant ones are redundant to both the non-redundant ones and the strongly relevant features with respect to the target (according for example to Definition \ref{asymmetrictotallyredundant} extended to sets of features). Since redundancy is a relative notion that is defined for pairs of features (or sets of features), the division of the weakly relevant features is typically not unique. For instance, if two copies of the same (relevant) feature are present, each one of them could play the role of the redundant one to the other leading to at least two divisions.

\begin{figure}[htbp]
	\centering
	\begin{tikzpicture}[line width=0.35mm, color=RoyalBlue]
	\draw [] (0,0) ellipse (3cm and 1.5cm);
	\coordinate (A) at (-0.5,+1.48);
	\coordinate (B) at (-0.5,-1.48);
	\coordinate (u) at (1,0);
	\coordinate (v) at (1,0);
	\draw[] (A) .. controls +(u) and +(v) .. (B); 
	\draw[] (-1.3,+1.35) .. controls +(1,0.2) and +(1,-0.2) .. (-1.3,-1.35); 
	\draw[] (-2.2,+1.02) .. controls +(1,0.5) and +(1,-0.5) .. (-2.2,-1.02); 
	\node[black, text width=1cm] at (3,1.2) {$V$};
	\node[black, text width=2cm] at (1.5,0) {Irrelevant features};
	\node[black, text width=2cm] at (-3,1.2) {Relevant features};
	
	\node[black] at (-2.2,0) {Strongly};
	\node[black,text width=1cm] at (-0.95,-0.7) {\footnotesize \baselineskip=10pt Not Red.\par};	
	\node[black] at (-0.2,-0.7) {\footnotesize Red.};	

	\fill[white] (-0.58,-0.2) rectangle (-0.53,0.2);
	\node[black] at (-0.6,0) {Weakly};
	
	\begin{scope} 
	\clip [] (0,0) ellipse (3cm and 1.5cm);
	\clip[] (A) .. controls +(u) and +(v) .. (B) -- (-2,-1.48) -- (-2,1.48) -- (A) ;
	\clip[] (-1.3,+1.35) .. controls +(1,0.2) and +(1,-0.2) .. (-1.3,-1.35) -- (-1.3,-1.48) -- (0.5,-1.48) -- (0.5,+1.48) -- (-1.3,+1.48) -- (-1.3,+1.35);
	\fill [RedOrange,opacity=0.2] (0,0) ellipse (3cm and 1.5cm);
	\end{scope}
	
	\begin{scope} 
	\clip [] (0,0) ellipse (3cm and 1.5cm);
	\clip[] (-2.2,+1.02) .. controls +(1,0.5) and +(1,-0.5) .. (-2.2,-1.02) -- (-3,-1.02) -- (-3,+1.02) -- (-2.2,+1.02); 
	\fill [RoyalBlue,opacity=0.2] (0,0) ellipse (3cm and 1.5cm);
	\end{scope}
	
	\end{tikzpicture}
	\caption{Graphical decomposition of the set of input variables $V$ according to the feature relevance. The subset of relevant features can be refined into two degrees of relevance: weak and strong relevance. Weakly relevant features can furthermore be divided into completely redundant (with respect to non-redundant features) and non-redundant features.}
	\label{fig:relevantsets-red}
\end{figure}
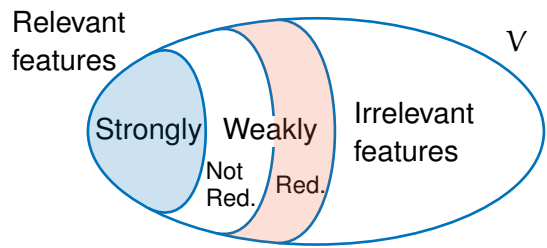

\subsubsection{Redundancy and correlation}\label{sec:background-redcorr}

Correlation is a statistical measure of the dependence between two numerical random variables. The most common measure of correlation is the Pearson correlation coefficient defined for two random variables $\mathcal{A}$ and $\mathcal{B}$ as \citep{pearson1896mathematical,lee1988thirteen,guyon2006introduction}:
\begin{eqnarray}
\rho(A,B) = \dfrac{cov(A,B)}{\sigma_A \sigma_B}
\end{eqnarray}
where $cov(A,B) = \mathbb{E}\left \{ (A-\mu_A)(B-\mu_B)\right \}$ is the covariance between both variables and where $\mu$ and $\sigma$ denote respectively the mean and the standard deviation. When the values of both variables move in the same direction (resp. opposition direction) in a similar fashion (i.e., by keeping a fixed distance), they are perfectly correlated (resp. anti-correlated) and this corresponds to $\rho=1$ (resp. $\rho=-1$).

Correlation and redundancy are different notions. We saw that duplicated (relevant) features are subsequently totally redundant with respect to the target. Intuitively, one may expect that a high correlation (or anti-correlation) between the values of two features suggests that those features are also redundant. However, correlation does not imply redundancy \citep{guyon2006introduction}. Figure \ref{fig:corrred} gives examples (inspired from \citep{guyon2006introduction}) showing that highly correlated features are not necessary redundant. But, if $cov(X_{i}, X_{j}) = \pm 1$ then Equation \ref{eqn:totallyred-gilles} holds and thus $X_{i}$ and  $X_{j}$ are totally redundant with respect to any target $Y$.

\begin{figure}[htbp]
	\centering
	\subfloat[Features are correlated ($\rho(X_1,X_2)=0.94$) and not redundant.]{\label{fig:corrred(a)}\includegraphics[width=0.48\linewidth]{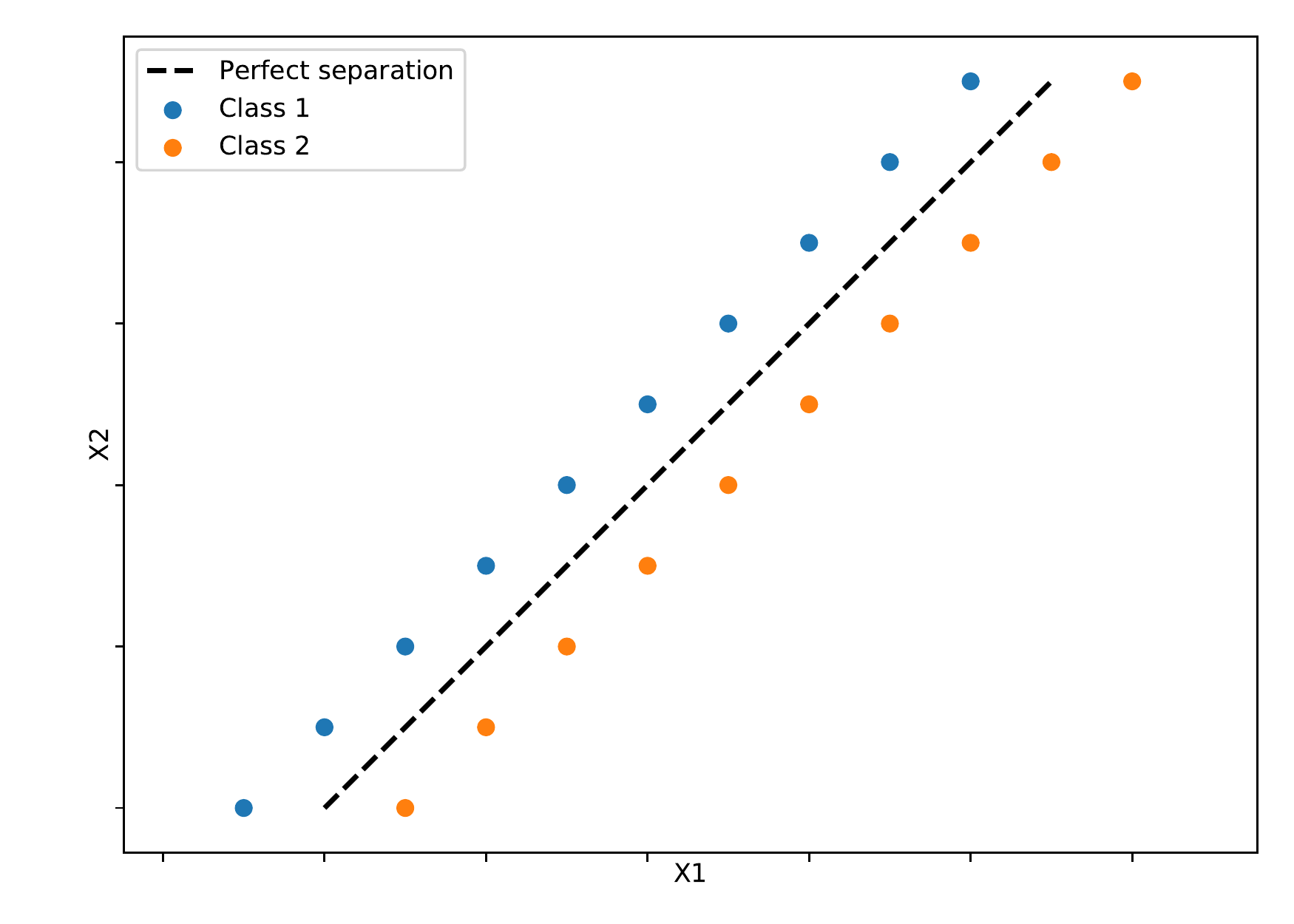}}\hfill
	\subfloat[Features are anti-correlated ($\rho(X_1,X_2)=-0.94$) and not redundant ]{\label{fig:corrred(b)}\includegraphics[width=0.48\linewidth]{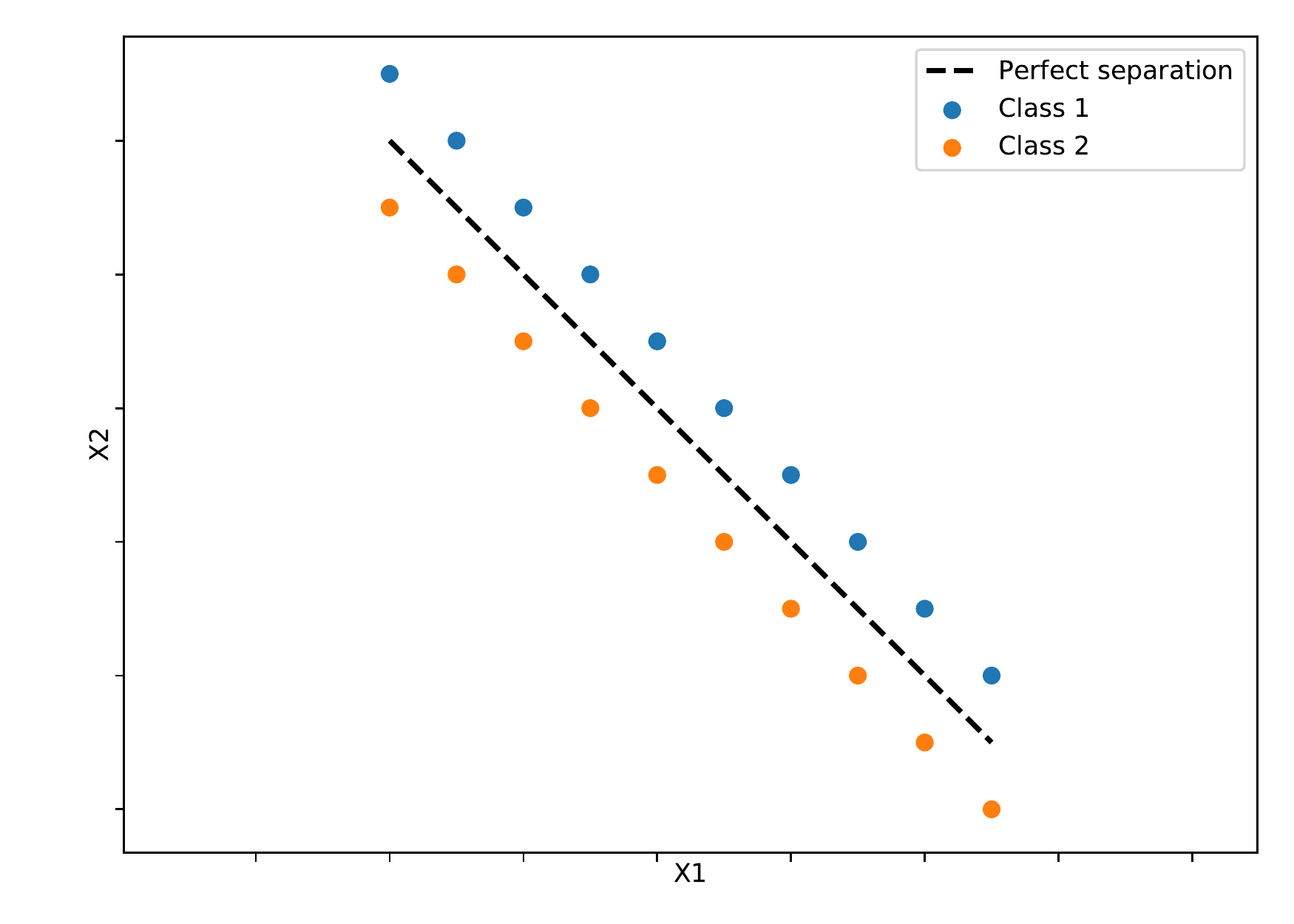}}\\
	\subfloat[Features are correlated ($\rho(X_1,X_2)=0.99$) and not redundant.]{\label{fig:corrred(c)}\includegraphics[width=0.48\linewidth]{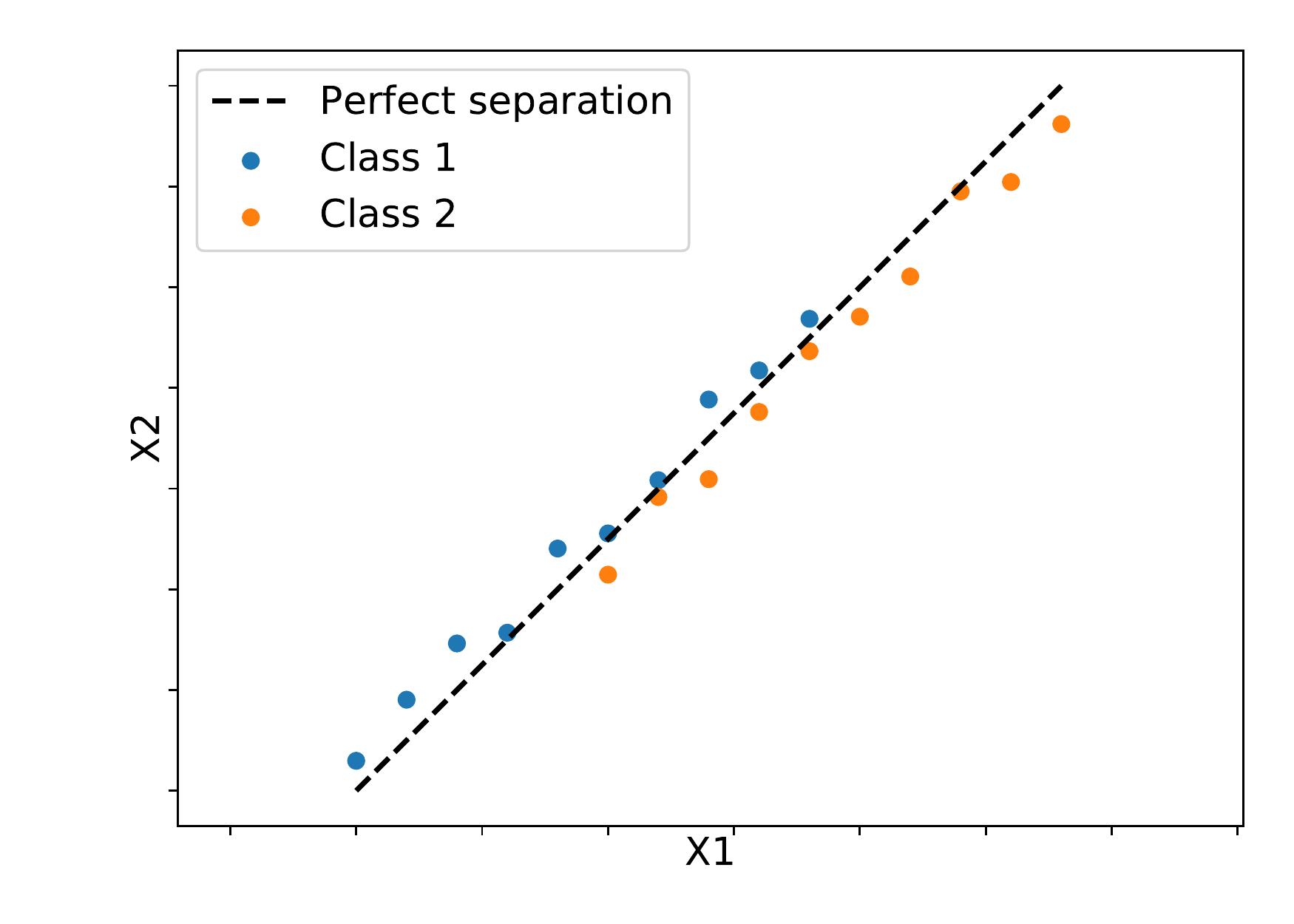}}\hfill
	\subfloat[Features are anti-correlated ($\rho(X_1,X_2)=-0.99$) and not redundant.]{\label{fig:corrred(d)}\includegraphics[width=0.48\linewidth]{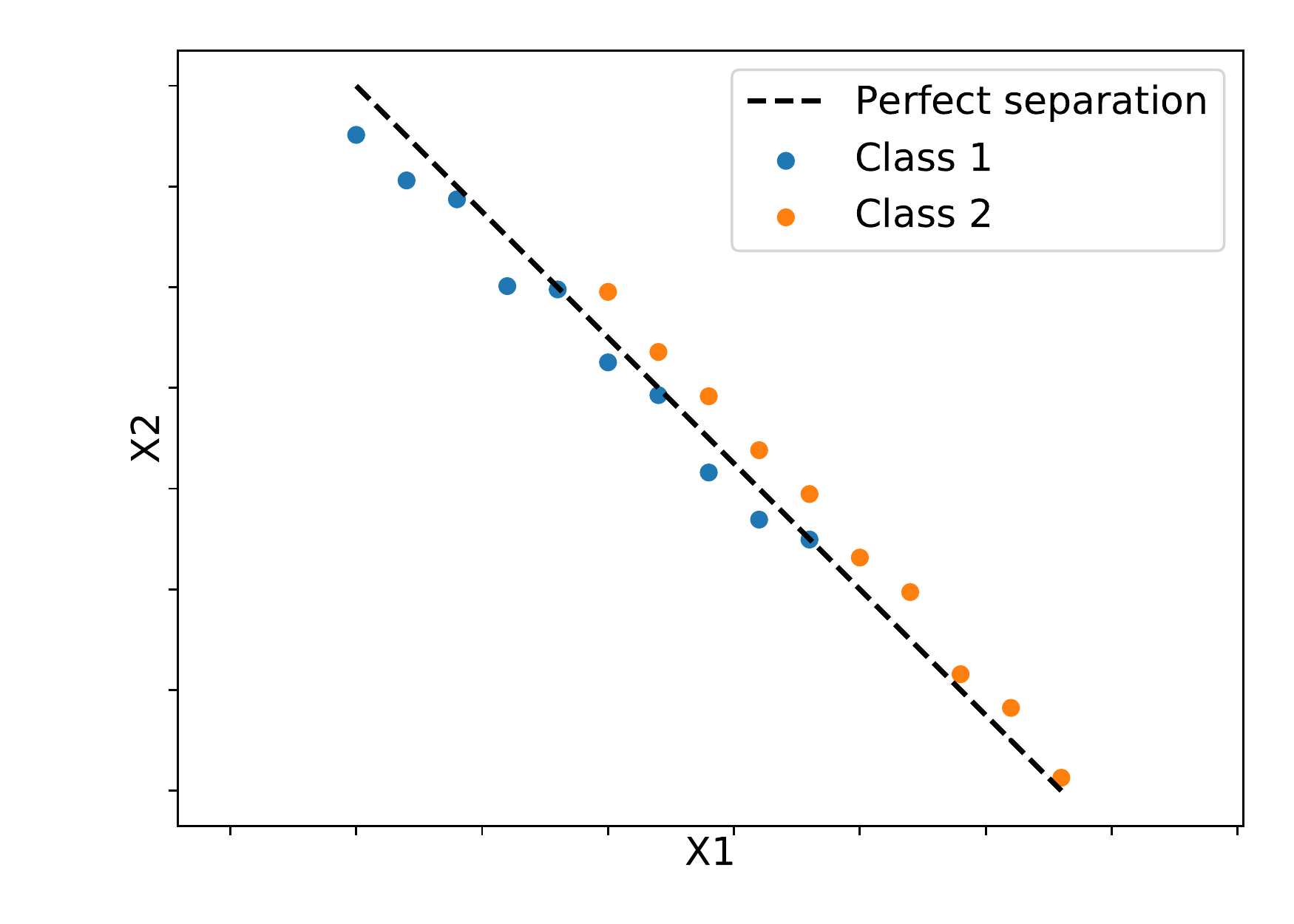}}\\
	\subfloat[Features are correlated ($\rho(X_1,X_2)=1$) and indeed redundant.]{\label{fig:corrred(e)}\includegraphics[width=0.48\linewidth]{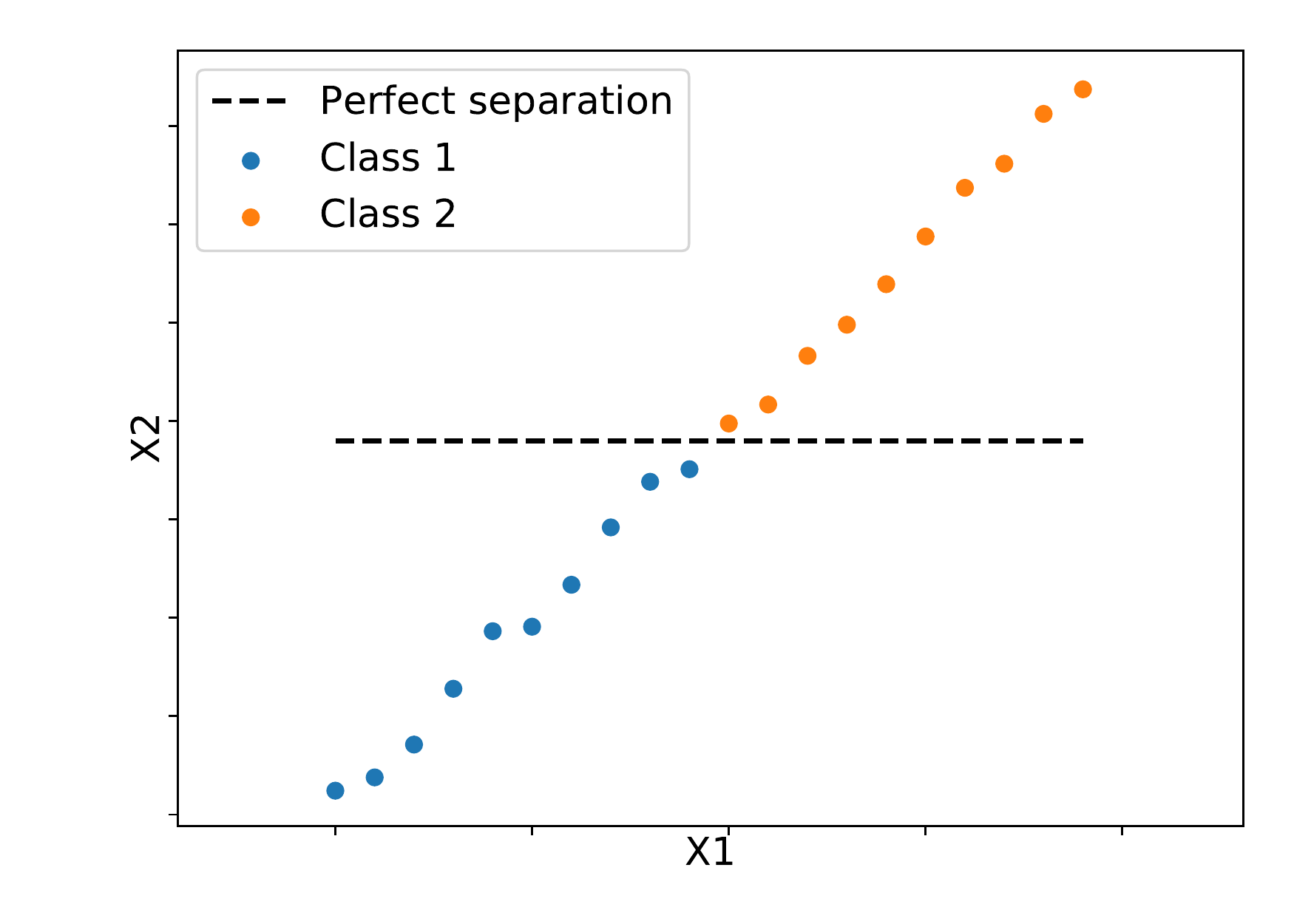}}\hfill
	\subfloat[Features are anti-correlated ($\rho(X_1,X_2)=-1$) and indeed redundant.]{\label{fig:corrred(f)}\includegraphics[width=0.48\linewidth]{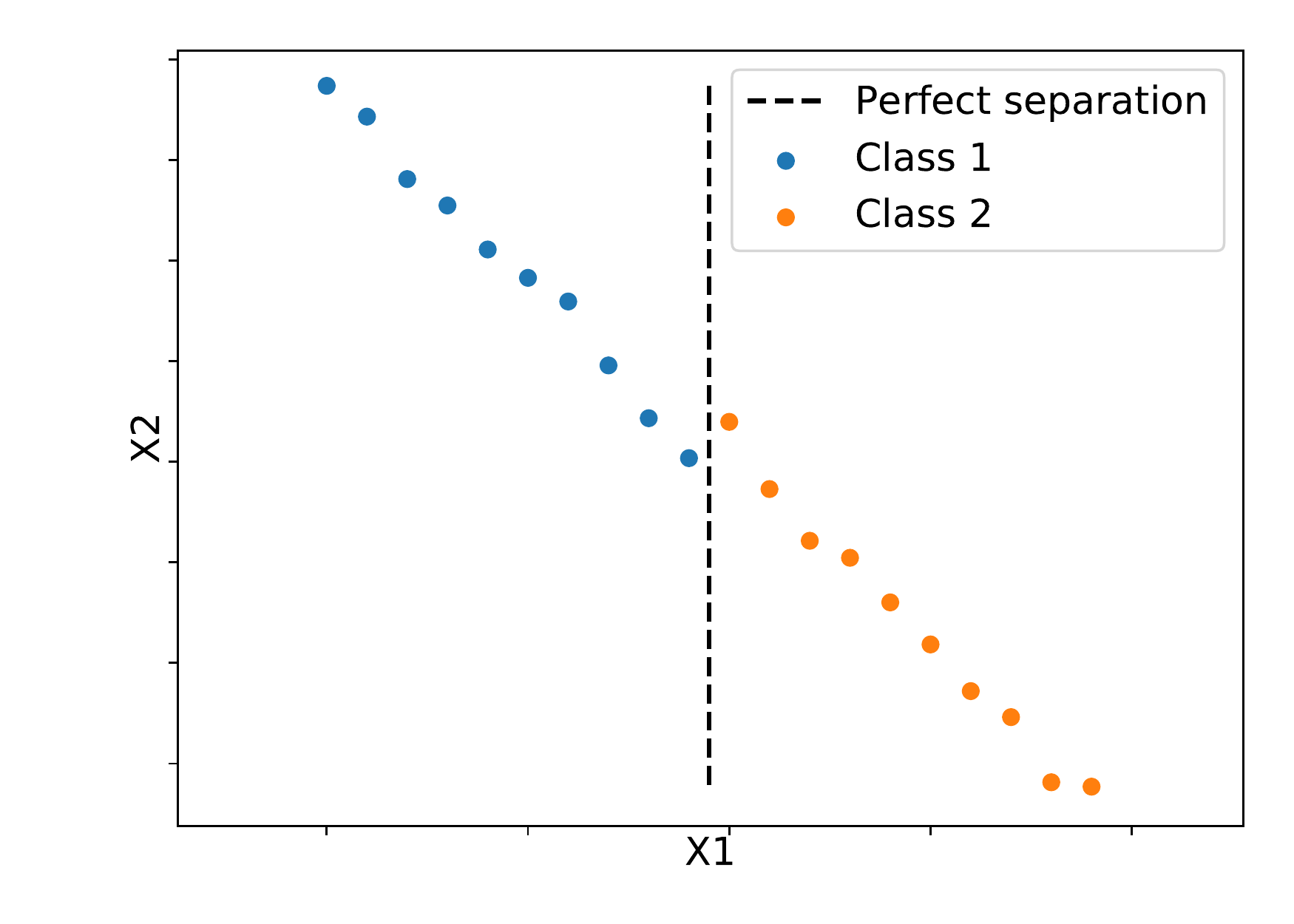}}
	\caption{Illustrating examples where correlation does not necessary imply redundancy. Figures (a) to (d) show that features can be highly correlated while being not redundant as both features are required to achieve a perfect separation between the two classes. Figures (e) and (f) show that correlated features can indeed be redundant as one feature out of the two is enough to perfectly separate classes. Let us note that in both last examples, both features can individually lead to a perfect separation.}
	\label{fig:corrred}
\end{figure}

\subsection{Feature selection problems}\label{sec:fsproblems}

Besides the objective of size reduction, the problem of feature selection usually can take two flavours \citep{guyon2003introduction,nilsson2007consistent,genuer2010variable,kursa2011all}. Typically, those side-objectives guide the feature selection and determine the subset of features that end up being selected. 

Many studies (e.g., with microarray gene-expression data \citep{ambroise2002selection} or in drug discovery application \citep{janecek2008relationship}) showed that dealing with small sets of relevant features usually gives better results and facilitate learning accurate classifiers \citep{guyon2003introduction,nilsson2007consistent,kursa2011all}.

In presence of many features, it is common that a large number of features are either irrelevant or redundant (to the target). Such variables are in principle not necessary to predict the output and computational performances of supervised learning algorithms can often be optimised by discarding them \citep{yu2004efficient}. Discarding some non-redundant features (with respect to those that are kept), may however be detrimental in terms of accuracy. Furthermore, for some specific learning algorithms, it may actually be beneficial in terms of accuracy to keep some redundant features. Moreover, when sample sizes are small compared to the number of features, it may even become beneficial (in terms of accuracy), to discard some non-redundant features (to decrease overfitting). Usually, as the number of selected features grows, it is expected that the performances of a learning algorithm increases and then decreases. The optimal size for the feature subset being the one that maximises the accuracy \citep{hua2004optimal}, the \textit{minimal-optimal problem} is the first problem of feature selection and consists in finding the smallest optimal subset for a given learning algorithm and a given dataset.

When only accuracy of the learnt predictor is used a criterion to select an optimal subset of features, many weakly relevant features (and sometimes even some strongly relevant one)  might be discarded.  
There is however an interest of identifying all features that are somehow related to the target in order to get a full understanding of the underlying mechanism (e.g., in gene expression analysis \citep{golub1999molecular}). The \textit{all-relevant} problem is the second approach of feature selection and consists in finding all relevant features.

Those two approaches are usually complementary for a given application. Let us take the example of a medical diagnosis that consists in predicting a disease. The doctor has to evaluate a given number of factors before making his diagnosis. The number of factors has to be as a small as possible to save time and money. Hence, one would want to identify a small set of features that provides the best possible diagnosis. The minimal-optimal approach aims at providing such a feature subset.
In different circumstances, for research purposes for instance, the all-relevant approach may be more appropriate. One may want to identify all factors that are related to the output even if some of them are redundant with respect to other. 

Both feature selection approaches are further described below.

\paragraph{All-relevant problem}

The all-relevant problem is defined as follows \citep{nilsson2007consistent,kursa2011all}:
\begin{definition} 
	The all-relevant feature selection problem consists in finding all relevant features. The solution to this problem is the set of all strongly and weakly relevant features.
\end{definition}
The solution of this problem is in principle unique, as suggested by Figure \ref{fig:relevantsets-allrel}. One further step, in such an analysis, would be to also distinguish between strongly and weakly relevant features. 

\begin{figure}[htbp]
	\centering

	\begin{tikzpicture}[line width=0.35mm, color=RoyalBlue]
\draw [] (0,0) ellipse (3cm and 1.5cm);
\coordinate (A) at (-0.5,+1.48);
\coordinate (B) at (-0.5,-1.48);
\coordinate (u) at (1,0);
\coordinate (v) at (1,0);
\draw[] (A) .. controls +(u) and +(v) .. (B); 
\draw[] (-1.3,+1.35) .. controls +(1,0.2) and +(1,-0.2) .. (-1.3,-1.35); 
\draw[] (-2.2,+1.02) .. controls +(1,0.5) and +(1,-0.5) .. (-2.2,-1.02); 
\node[black, text width=1cm] at (3,1.2) {$V$};
\node[black, text width=2cm] at (1.5,0) {Irrelevant features};
\node[black, text width=2cm] at (-3,1.2) {Relevant features};

\node[black] at (-2.2,0) {Strongly};
\node[black,text width=1cm] at (-0.95,-0.7) {\footnotesize \baselineskip=10pt Not Red.\par};	
\node[black] at (-0.2,-0.7) {\footnotesize Red.};	

\fill[white] (-0.58,-0.2) rectangle (-0.53,0.2);
\node[black] at (-0.6,0) {Weakly};

\begin{scope} 
\draw[orange] (A) .. controls +(u) and +(v) .. (B); 
\clip[] (A) .. controls +(u) and +(v) .. (B) -- (-0.5,-1.5) -- (-3.2,-1.5) -- (-3.2,1.5) -- (-0.5,+1.5) -- (A) ;
\draw [orange] (0,0) ellipse (3cm and 1.5cm);

\fill [orange,opacity=0.2] (0,0) ellipse (3cm and 1.5cm);

\end{scope}
\node[orange] at (-0.4,1) {$S$};
\end{tikzpicture}
	\caption{The solution $S$ to the all-relevant problem is the union of weakly relevant and strongly relevant features to the target variable. This set includes all relevant features even if there is redundant information about the target.}
	\label{fig:relevantsets-allrel}
\end{figure}
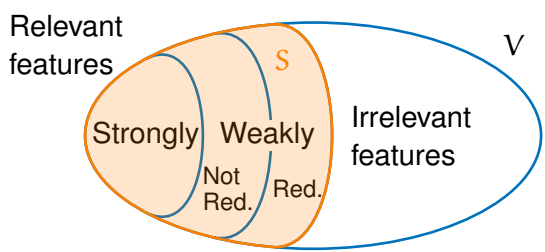

\paragraph{Minimal-optimal problem}

In terms of learning algorithm performances, the minimal-optimal problem is usually defined as follows \citep{kohavi1997wrappers,nilsson2007consistent}:

\begin{definition}
	Let $\mathcal{A}$ be a learning algorithm, $V$ the set of input features and $Y$ be the target feature. The minimal-optimal feature selection problem consists in finding a subset of $V$ of minimal size that minimises the generalisation error of $\mathcal{A}$.
\end{definition}

A solution of this problem is usually a subset of all relevant features, even if for some very specific combinations of problems and algorithms, including irrelevant features may actually be beneficial from the viewpoint of accuracy \citep{kohavi1997wrappers}.

{For regression and \textit{calibrated}\footnote{A classification problem which requires the exact distribution of predictions of $Y$ and not only the most probable class of $Y$ is said to be calibrated \citep{tsamardinos2003towards}. Such problems correspond for instance to classification problems where the mean squared loss is used instead of the zero-one loss.} classification tasks, \cite{tsamardinos2003towards} showed that a Markov boundary of minimal size is a solution to the minimal-optimal problem (see \citep[Proposition 3]{tsamardinos2003towards} for more details and see side note on page \pageref{sn:markov} for a word on Markov blanket discovery algorithms). Therefore, a solution of the minimal-optimal problem is a set made of all strongly relevant and a maximal subset of non-redundant\footnote{Features providing non-redundant information about the output but that are redundant with some of non-selected features. In other words, relevant features that are included in a Markov boundary but that are not strongly relevant with respect to the output.} weakly relevant features \citep{kursa2011all}. Let us however note that when the zero-one loss is used (i.e., only the most probable class of $Y$ is required), \cite{tsamardinos2003towards} state that only some features of the Markov boundary are required or features that do not belong to the Markov boundary.}

\begin{sidenote}{Markov blanket discovery algorithms} \label{sn:markov}
	Markov blanket and boundary discovery algorithms constitute another broad family of feature selection techniques (see, e.g., \citep{guyon2003introduction,tsamardinos2003time,aliferis2010local,statnikov2013algorithms,tsamardinos2003algorithms}). They are usually independent of any learning algorithm and are mainly based on graph theory and related to causality. They are however not addressed in this thesis. 
	
\end{sidenote}

{Resulting of the multiplicity of Markov boundaries, the minimal-problem problem does not have a unique solution in general.} For example, in presence of two totally redundant features, each can be kept (without the other one) giving two valid options. 
For strictly positive distributions however, the Markov boundary $M$ of $Y$ is
unique and corresponds to the set of all
strongly relevant variables \citep{nilsson2007consistent}.

Figure \ref{fig:relevantsets-minopt} shows typical solutions to the minimal-optimal problem with respect to the relevance of features. In the case of a strictly positive distribution, Figure \ref{fig:relevantsets-minopt(a)} gives the unique solution to the minimal-optimal problem which is the set of strongly relevant features. In the case of non-strictly positive distribution, Figure \ref{fig:relevantsets-minopt(b)} illustrates a solution to the minimal-optimal problem which includes in all generality some weakly relevant features and all strongly relevant ones. 

Finding an optimal subset is usually intractable because some distributions may require an exhaustive search of all possible subsets to guarantee optimality \citep{cover1977possible, kohavi1997wrappers,blum1997selection,yu2004efficient,nilsson2007consistent}. With $p$ features, there are $2^p$ possible subsets which is clearly impractical, especially for high-dimensional datasets.
However, letting this search be guided by a heuristic (see Section \ref{sec:fs-search} and \citep{guyon2003introduction}) or considering only strictly positive distributions \citep{nilsson2007consistent} make this problem more tractable computationally.

\begin{figure}[htbp]
	\centering
	\subfloat[Distribution satisfying the intersection property]{

			\begin{tikzpicture}[line width=0.35mm, color=RoyalBlue]
		\draw [] (0,0) ellipse (3cm and 1.5cm);
		\coordinate (A) at (-0.5,+1.48);
		\coordinate (B) at (-0.5,-1.48);
		\coordinate (u) at (1,0);
		\coordinate (v) at (1,0);
		\draw[] (A) .. controls +(u) and +(v) .. (B); 
		\draw[] (-1.3,+1.35) .. controls +(1,0.2) and +(1,-0.2) .. (-1.3,-1.35); 
		\draw[] (-2.2,+1.02) .. controls +(1,0.5) and +(1,-0.5) .. (-2.2,-1.02); 
		\node[black, text width=1cm] at (3,1.2) {$V$};
		\node[black, text width=2cm] at (1.5,0) {Irrelevant features};
		\node[black, text width=2cm] at (-3,1.2) {Relevant features};
		
		\node[black] at (-2.2,0) {Strongly};
		\node[black,text width=1cm] at (-0.95,-0.7) {\footnotesize \baselineskip=10pt Not Red.\par};	
		\node[black] at (-0.2,-0.7) {\footnotesize Red.};	
		
		\fill[white] (-0.58,-0.2) rectangle (-0.53,0.2);
		\node[black] at (-0.6,0) {Weakly};
		
	\begin{scope} 
	\draw[orange] (-2.2,+1.02) .. controls +(1,0.5) and +(1,-0.5) .. (-2.2,-1.02);
	\clip[] (-2.2,+1.02) .. controls +(1,0.5) and +(1,-0.5) .. (-2.2,-1.02) -- (-2.2,-1.1) -- (-3.2,-1.1) -- (-3.2,+1.1) -- (-2.2,+1.1) -- (-2.2,+1.02); 
	\draw [orange] (0,0) ellipse (3cm and 1.5cm);
	\clip [] (0,0) ellipse (3cm and 1.5cm);
	
	\fill [orange,opacity=0.2] (0,0) ellipse (3cm and 1.5cm);
	\end{scope}

		\node[orange] at (-2.2,-0.5) {$=S$};
		\end{tikzpicture}
		\label{fig:relevantsets-minopt(a)}}
	\subfloat[Distribution \textbf{not} satisfying the intersection property]{

			\begin{tikzpicture}[line width=0.35mm, color=RoyalBlue]
\draw [] (0,0) ellipse (3cm and 1.5cm);
\coordinate (A) at (-0.5,+1.48);
\coordinate (B) at (-0.5,-1.48);
\coordinate (u) at (1,0);
\coordinate (v) at (1,0);
\draw[] (A) .. controls +(u) and +(v) .. (B); 
\draw[] (-1.3,+1.35) .. controls +(1,0.2) and +(1,-0.2) .. (-1.3,-1.35); 
\draw[] (-2.2,+1.02) .. controls +(1,0.5) and +(1,-0.5) .. (-2.2,-1.02); 
\node[black, text width=1cm] at (3,1.2) {$V$};
\node[black, text width=2cm] at (1.5,0) {Irrelevant features};
\node[black, text width=2cm] at (-3,1.2) {Relevant features};

\node[black] at (-2.2,0) {Strongly};
\node[black,text width=1cm] at (-0.95,-0.7) {\footnotesize \baselineskip=10pt Not Red.\par};	
\node[black] at (-0.2,-0.7) {\footnotesize Red.};	
\draw[orange] (-1.3,+1.35) .. controls +(1,0.2) and +(1,-0.2) .. (-1.3,-1.35);
\fill[white] (-0.58,-0.2) rectangle (-0.53,0.2);

\begin{scope} 
\clip[] (-1.3,+1.35) .. controls +(1,0.2) and +(1,-0.2) .. (-1.3,-1.35) -- (-1.3,-1.4) -- (-3.2,-1.4)-- (-3.2,1.4)-- (-1.3,1.4) -- (-1.3,1.35);
\draw [orange] (0,0) ellipse (3cm and 1.5cm);
\clip [] (0,0) ellipse (3cm and 1.5cm);
\fill [orange,opacity=0.2] (0,0) ellipse (3cm and 1.5cm);
\end{scope}

\node[orange] at (-1,+1) {$S$};
\node[black] at (-0.6,0) {Weakly};
\end{tikzpicture}
\label{fig:relevantsets-minopt(b)}}
	\caption{Typical solutions $S$ to the minimal-optimal problem for distributions satisfying the intersection property or not. Solutions are Markov boundaries of $Y$ with respect to $V$.}
	\label{fig:relevantsets-minopt}
\end{figure}
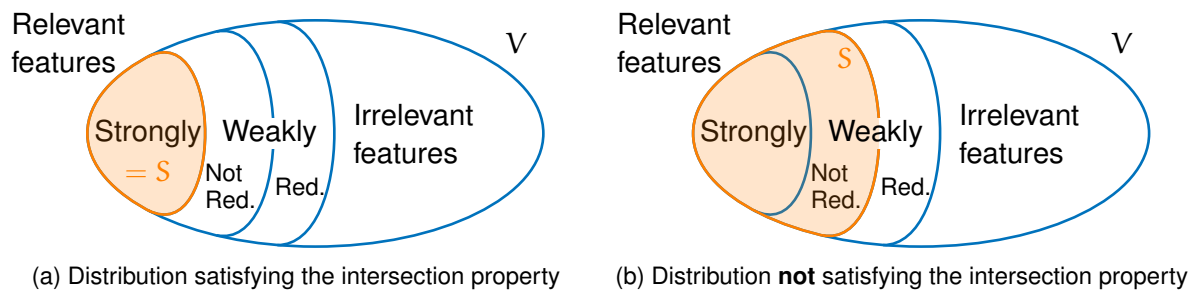

\subsection{Feature selection methods}\label{sec:fsmethods}

Feature selection methods are usually classified in three categories depending on how they interact with the learning algorithm: filters, wrappers and embedded methods \citep{blum1997selection,guyon2003introduction,tsamardinos2003towards,saeys2007review}. 

\paragraph{Filters} 

\textit{A filter approach} aims at selecting features independently of the learning algorithm (i.e., without optimising its performance) \citep{kohavi1997wrappers,blum1997selection,guyon2003introduction,guyon2006introduction,saeys2007review,brown2012conditional,chandrashekar2014survey}. 
A filter tries to assess the interest of keeping features solely from the data in order to then filter out irrelevant features. As a pre-processing step that selects inputs, any learning algorithm can thus be combined with a filtering feature selection.

A common filter method is feature ranking\footnote{Feature ranking is sometimes referred to as feature weight based approach as a weight is assigned to each feature \citep{blum1997selection,kira1992feature}.} \citep{stoppiglia2003ranking,blum1997selection,guyon2003introduction,chandrashekar2014survey} and consists in ordering features according to a suitable ranking criterion. Any feature relevance measure providing a numerical score can be used (e.g., correlation \citep{guyon2003introduction} or mutual information \citep{blum1997selection,brown2012conditional} with the target, decision tree\footnote{Feature selection using tree-based models will be discussed in Chapter \ref{ch:importances}.} \citep{cardie1993using}, ...). Then the top $k$ features (i.e., with highest value) are selected \citep{blum1997selection,saeys2007review,chandrashekar2014survey}. 
The number $k$ of selected features is determined using an (arbitrary or not) threshold value or based on a \textit{random probe} (i.e., a random variable is introduced in the process in order to determine which features  are statistically better than an artificial irrelevant feature and thus relevant) \citep{stoppiglia2003ranking}. 


Filter techniques are usually computationally fast and scale very well to high-dimensional datasets \citep{saeys2007review}. As they are independent of the learning algorithm, they only need to be performed once and for all whatever what follows.

A downside of this independence is that filter techniques totally ignore the performance of the learning algorithm with the selected subset \citep{kohavi1997wrappers}. 
In the filtering approach, most proposed techniques (e.g., correlation and mutual information) are univariate: each feature is considered individually and feature dependencies and redundancies are not taken into account \citep{guyon2003introduction,saeys2007review,chandrashekar2014survey}. Ignoring such effects may lead to a selected set of features that yields poor performance when compared to other types of (multivariate) feature selection techniques \citep{saeys2007review}. Besides, a subset of the selected set of features may be sufficient in presence of redundancy \citep{chandrashekar2014survey}. Consequently, multivariate criteria have been proposed to integrate feature dependencies (e.g., based on mutual information \citep{peng2005feature,meyer2008information,frenay2013mutual,meyer2013information} or based on Markov blankets \citep{koller1996toward}, and see \citep{brown2012conditional} for a unifying framework based on conditional likelihood maximisation) but at the cost of scalability and computational speed \citep{saeys2007review}.


In the light of the feature selection problems introduced in Section \ref{sec:fsproblems}, two filter approaches are of interest:
\begin{description}
	\item[The FOCUS algorithm] \citep{almuallim1991efficient,almuallim1991learning,almuallim1994learning} conducts an \textbf{exhaustive search} among all possible subsets of features for the minimal one providing a perfect discrimination (or the best possible) of the target values \citep{koller1996toward,kohavi1997wrappers}.  It has a preference for a small set of features and suffer from the so-called \textit{Min-features bias} \citep{almuallim1991efficient} which may lead to a poor feature selection (see side note on page \pageref{sn:minfeaturebias}).
	Nevertheless, it is expected that the selected set of features includes all strongly relevant features and some weakly relevant ones. 

\begin{sidenote}{Min-features bias} \label{sn:minfeaturebias}
	By chance, an irrelevant feature could be sufficient to perfectly determine the target value in the training data (e.g., a unique sample ID such as the social security number in a medical dataset). A learning algorithm receiving such a feature would surely overfit the training data leading to poor performances in generalisation \citep{kohavi1997wrappers}. A preference towards small set of features - the Min-features bias - would choose that variable as the best subset in comparison with other subsets made of a single variable.
\end{sidenote}

\item[The Relief algorithm] \citep{kira1992feature,kira1992practical} is an instance of \textbf{feature ranking} and aims at assigning a relevance score to each feature\footnote{In \cite{kira1992feature}, the relevance level of the $j^{th}$ feature of the $i^{th}$ sample, denoted $x_j^i$, is based on two distances: (i) the difference $c_j^i$ between values of $x_j^i$ and $x_j^c$ where $x_j^c$ is the value of the same feature for a sample $s$ which is the closest one with the same class (i.e., $y^j = y^c$) as sample $i$; (ii) the difference $d_j^i$ between values of  $x_j^i$  and $x_j^d$ where $x_j^d$ is the value of the same feature for a sample $d$ which is the closest one with a different class (i.e., $y^j \ne y^d$). The relevance level of a feature $X_j$ is based on an average over all samples of the square of those two distances.}. The selection is then made by considering as relevant (and thus to be kept) features with a relevance score above a given threshold (determined for instance by a statistical method of interval estimation).  The selected subset of features is expected to be the set of all relevant features (weak and strong ones) including redundant features. 


\end{description}
Figure \ref{fig:relevantsets-focus/relief} illustrates a typical solution according to the relevance for both algorithms. One can see that Focus algorithm aims to solve the minimal-optimal problem (although ignoring the usefulness of the selected set of features) and that Relief algorithm aims to solve the all-relevant problem \citep{kohavi1997wrappers}. 

\begin{figure}[htbp]
	\centering

	\subfloat[Focus ($F_F$) aims to solve the minimal-optimal problem]{
	
		\begin{tikzpicture}[line width=0.35mm, color=RoyalBlue]
		\draw [] (0,0) ellipse (3cm and 1.5cm);
		\coordinate (A) at (-0.5,+1.48);
		\coordinate (B) at (-0.5,-1.48);
		\coordinate (u) at (1,0);
		\coordinate (v) at (1,0);
		\draw[] (A) .. controls +(u) and +(v) .. (B); 
		\draw[] (-1.3,+1.35) .. controls +(1,0.2) and +(1,-0.2) .. (-1.3,-1.35); 
		\draw[] (-2.2,+1.02) .. controls +(1,0.5) and +(1,-0.5) .. (-2.2,-1.02); 
		\node[black, text width=1cm] at (3,1.2) {$V$};
		\node[black, text width=2cm] at (1.5,0) {Irrelevant features};
		\node[black, text width=2cm] at (-3,1.2) {Relevant features};
		
		\node[black] at (-2.2,0) {Strongly};
		\node[black,text width=1cm] at (-0.95,-0.7) {\footnotesize \baselineskip=10pt Not Red.\par};	
		\node[black] at (-0.2,-0.7) {\footnotesize Red.};	
		\draw[orange] (-1.3,+1.35) .. controls +(1,0.2) and +(1,-0.2) .. (-1.3,-1.35);
		\fill[white] (-0.58,-0.2) rectangle (-0.53,0.2);

		\begin{scope} 
		\clip[] (-1.3,+1.35) .. controls +(1,0.2) and +(1,-0.2) .. (-1.3,-1.35) -- (-1.3,-1.4) -- (-3.2,-1.4)-- (-3.2,1.4)-- (-1.3,1.4) -- (-1.3,1.35);
		\draw [orange] (0,0) ellipse (3cm and 1.5cm);
		\clip [] (0,0) ellipse (3cm and 1.5cm);
		\fill [orange,opacity=0.2] (0,0) ellipse (3cm and 1.5cm);
		\end{scope}

		\node[orange] at (-1.3,+1) {$F_F$};
		\node[black] at (-0.6,0) {Weakly};
		\end{tikzpicture}
}	
	\subfloat[Relief ($F_R$) aims to solve the all-relevant problem]{
	\begin{tikzpicture}[line width=0.35mm, color=RoyalBlue]
\draw [] (0,0) ellipse (3cm and 1.5cm);
\coordinate (A) at (-0.5,+1.48);
\coordinate (B) at (-0.5,-1.48);
\coordinate (u) at (1,0);
\coordinate (v) at (1,0);
\draw[] (A) .. controls +(u) and +(v) .. (B); 
\draw[] (-1.3,+1.35) .. controls +(1,0.2) and +(1,-0.2) .. (-1.3,-1.35); 
\draw[] (-2.2,+1.02) .. controls +(1,0.5) and +(1,-0.5) .. (-2.2,-1.02); 
\node[black, text width=1cm] at (3,1.2) {$V$};
\node[black, text width=2cm] at (1.5,0) {Irrelevant features};
\node[black, text width=2cm] at (-3,1.2) {Relevant features};

\node[black] at (-2.2,0) {Strongly};
\node[black,text width=1cm] at (-0.95,-0.7) {\footnotesize \baselineskip=10pt Not Red.\par};	
\node[black] at (-0.2,-0.7) {\footnotesize Red.};	

\fill[white] (-0.58,-0.2) rectangle (-0.53,0.2);
\node[black] at (-0.6,0) {Weakly};

\begin{scope} 
\draw[orange] (A) .. controls +(u) and +(v) .. (B); 
\clip[] (A) .. controls +(u) and +(v) .. (B) -- (-0.5,-1.5) -- (-3.2,-1.5) -- (-3.2,1.5) -- (-0.5,+1.5) -- (A) ;
\draw [orange] (0,0) ellipse (3cm and 1.5cm);

\fill [orange,opacity=0.2] (0,0) ellipse (3cm and 1.5cm);

\end{scope}
\node[orange] at (-0.4,1) {$F_R$};
\end{tikzpicture}\label{fig:relevantsets-relief}}

	\caption{Expected set of features selected by Focus ($F_F$) and Relief ($F_R$) algorithm according to the feature relevance.}
	\label{fig:relevantsets-focus/relief}
\end{figure}
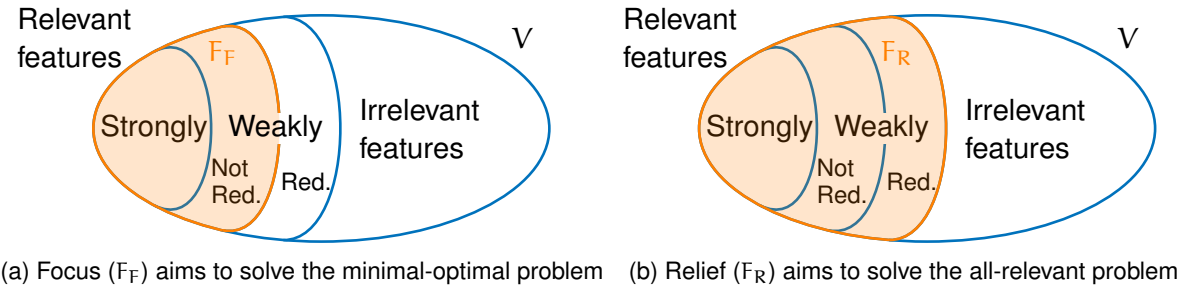

\paragraph{Wrappers} A \textit{wrapper method} aims at selecting a set of features using a learning algorithm as a "black box" \citep{kohavi1997wrappers,blum1997selection,guyon2003introduction, saeys2007review}. A set of features is presented to a learning algorithm and the corresponding accuracy performances is used as an estimation of the relative usefulness of the given set of features. The search over all possible subsets is usually guided by a search algorithm (see Section \ref{sec:fs-search} for more details) "wrapped" around the learning algorithm \citep{saeys2007review}. At the end, the best set of features is then selected as the one leading to the best performances of the given learning algorithm. 

In the wrapper approach, the optimal feature subset search is carried out in interaction with a specific learning algorithm $\mathcal{A}$. The resulting selected set of features is therefore the most useful for $\mathcal{A}$ but also tailored to it. Wrapper methods benefit from learning algorithm characteristics (e.g., feature dependencies) but depend on its complexity implying a high computational cost. 
In contrast with filter techniques, the feature selection is coupled with the learning algorithm performances increasing the risk of overfitting. {Examples of wrapper methods (e.g., sequential feature selection and sequential backward elimination) are given in Section \ref{sec:fs-search}.}

\paragraph{Embedded methods} An \textit{embedded method} of feature selection is comprised in the learning algorithm \citep{blum1997selection,guyon2003introduction,geurts2006extremely,saeys2007review, chandrashekar2014survey}. Similarly to wrapper methods, the selected set of features is specific to the learning algorithm. However, the feature subset search and evaluation are incorporated in the training algorithm \citep{guyon2006introduction} and thus embedded methods are usually less computationally expensive than wrapper methods \citep{saeys2007review,chandrashekar2014survey}. Two examples of embedded methods are two regularised linear regressions known as Lasso and Ridge regressions. Both methods construct a linear model that minimises its error for a given loss function and uses a subset of the variables while including a penalty term that limits the number of variables used. Similarly with wrappers, the selected set of features may depend on the considered embedded method. Indeed, Lasso and Ridge regressions use different penalisation terms and therefore may select different set of features.

\subsection{Feature subset search algorithms}\label{sec:fs-search}
	
	Given $p$ features, the number of possible feature subsets (i.e., equal to $2^p$) grows exponentially with the number of features making the feature selection space (i.e., the space of all possible subsets of features) very large. Several approaches have been proposed to explore this space. An exhaustive search is optimal but computationally intensive. Heuristic searches have been introduced to explore this space more efficiently. In the rest of this section, we describe well-known search algorithms that will be of interest in this thesis.
	
	\begin{description}
	\item[Exhaustive search] \citep{kira1992feature} consists in exploring the whole feature selection space. 
	All possible subsets are evaluated and the smallest one that maximises a given criterion (which can be a relevance index for a filter approach or the accuracy for a wrapper method for example) is selected. The optimal subset is thus always found at the expense of computational efficiency. For example, the Focus algorithm \citep{almuallim1991efficient,almuallim1991learning,almuallim1994learning} examines subsets by increasing order of size and stops as soon as an optimal subset is found. This approach limits the computational burden while preserving optimality \citep{kira1992feature}.
	
	\item[Heuristic search] explores more efficiently the search space while trying to find the best (possible) subsets of features. Several approaches aim at reducing the number of subsets to evaluate. A first way consists in limiting the maximal size to $d \le p$. Only subsets with $d$ or less features are considered but this requires an explicit value of $d$ which is in practice unknown \citep{kira1992feature,devijver1982pattern}.
	
	The \textit{Sequential Feature Selection} (SFS, also known as \textit{Forward Selection}) \citep{whitney1971direct,miller1990subset,kira1992feature,blum1997selection,jain2000statistical,guyon2003introduction,reunanen2003overfitting,chandrashekar2014survey} starts by selecting the single feature that maximises the given criterion and then sequentially adds one feature at a time. At each step, each remaining feature is evaluated in combination with already selected features and the best one is permanently added to the current subset. The process stops when all features have been added or when the required size of subset is reached.  
	Conversely, the \textit{Sequential Backward Elimination} (SBE, also known as \textit{Sequential Backward Selection}) \citep{marill1963effectiveness,kira1992feature,pudil1994floating,kohavi1997wrappers,blum1997selection,jain2000statistical,guyon2003introduction,chandrashekar2014survey} starts with all features and then evaluates shrinking feature sets. At each step, the less promising feature (i.e., the one whose removal is the less penalising according to the criterion) is removed, one at a time, until the required subset size is reached.
	
	SFS is more computationally advantageous than SBE as first evaluated subsets are made of few features \citep{kohavi1997wrappers} (see side node on page \pageref{sn:sfssbe}).
	Feature dependency is not taken into account as some features may not be very useful individually while being highly informative together \citep{kira1992feature,chandrashekar2014survey}. However, the backward elimination strategy can theoretically capture feature interactions \citep{kohavi1997wrappers}. 
	Both approaches do not examine all possible subsets and yield nested feature subsets in the sense that a selected (respectively removed) feature can not be removed (respectively re-selected) even if it would lead to a better subset of features \citep{pudil1994floating,guyon2003introduction,reunanen2003overfitting,chandrashekar2014survey}. 
	Therefore, optimality of the selected subset can not be guaranteed \citep{pudil1994floating,jain1997feature}.	
	More complexed algorithms have been proposed in order to overcome nested subsets. The approach \textit{"Plus-$l$-Minus-$r$"} consists in combining the forward selection and backward elimination in selecting at each step the $l$ most promising features and removing the $r$ less promising ones \citep{stearns1976selecting,kittler1978feature}. Parameters $l$ and $r$ need however to be fixed. \textit{Sequential Floating Forward Selection} (SFFS) follows the sequential search procedure but includes a potential feature elimination at each step \citep{pudil1994floating,somol1999adaptive}. 
	Similarly, \textit{Sequential Floating Backward Elimination} (SFBE) includes a potential feature selection at each step \citep{pudil1994floating,somol1999adaptive}. 
	\textit{Adaptive Sequential Forward Floating Selection} (ASFFS) generalises above-mentioned approaches with an adaptive determination of $l$ and $r$ at each step \citep{somol1999adaptive}.
	\end{description}

\begin{sidenote}{A word on the complexity of SFS and SBE} \label{sn:sfssbe}
			 Let us consider a dataset $\mathbf{D}$ made of a set $F$ of $p$ features and $N$ samples. We want to solve the minimal-optimal problem. Thus, we need to identify the best feature subset of $F$ according to a function $J(E)$ that evaluates a feature subset $E \subseteq F$ on the dataset $\mathbf{D}$. The evaluating function $J$ can either be an independent criterion (in a filtering approach) or an induced model (in a wrapper approach). In both cases, $J$ returns a score that assess the quality of the selected subset $E$ and has a computational cost $\mathcal{O}(J)$ (e.g., that may be the computation cost of the model). In the case of an exhaustive search, all $2^p$ subsets must be examined in order to find the best one. The overall complexity is therefore $\mathcal{O}(2^p) \mathcal{O}(J)$ but guarantees optimality. This can be conceivably performed, if $p$ is not too large \citep{guyon2003introduction} but otherwise it is computationally intractable. {In the case of a heuristic search\footnote{{Let us notice that complexities of those sequential search algorithms given in \citep{kira1992feature} rather correspond to approaches that exhaustively consider all subsets of sizes lower than (resp. greater or equal) $d\le p$ for the sequential forward (respectively backward) selection.}} (either SFS or SBE), the complexity\footnote{{Note that the size of the selected feature set can be fixed ($d \le p$) to stop earlier these sequential searches, but this does not change the complexity.}} is $\mathcal{O}(p^2)\mathcal{O}(J)$, which is much less than $\mathcal{O}(2^p)\mathcal{O}(J)$.}
			Both are then much more efficient than exhaustive search but may not yield optimal results as the procedure may miss some feature interactions. It should be stressed that the evaluating function cost may overburden the overall search complexity and therefore it is important to have an efficient and reliable evaluation of each subset. 
			In Chapter~\ref{ch:SRS}, we propose to use a computationally inexpensive model (i.e., a randomised tree, see Chapter~\ref{ch:trees} for a definition) to perform a multivariate sequential feature selection.
			Let us also mention that \cite{nilsson2007consistent} showed that if the distribution is restricted to be strictly positive (i.e., all weakly relevant variables are necessarily redundant and thus can be ignored): the minimal-optimal problem can be solved in polynomial time in the number of features and SBE approaches become consistent (but SFS ones do not).
			
\end{sidenote}

\subsection{Discussion}\label{sec:background-discussion}

This section aims at reviewing some of the main limitations of feature selection and open problems that motivate some of the  research questions considered in the rest of this thesis.

\paragraph{Feature relevance in the context of others}

Multivariate approaches are usually preferred over univariate ones because they take into account feature dependencies even though they are computationally less efficient. It shows that feature dependencies is crucial in many applications. 
Relevant features (even strongly) can be marginally irrelevant while being (highly) relevant in combination with other features \citep{domingos1996exploiting,guyon2003introduction}. A well-known example is the exclusive-OR (XOR) structure \citep{guyon2003introduction,kohavi1997wrappers}. 
Redundancy (another form of feature dependency) may tone relevance or usefulness of features down. Consequently a feature may be not selected (or identified as relevant) while being highly marginally relevant.

The problem of finding all relevant features requires thus to carefully take into account feature dependencies and the only way to do so is to perform an exhaustive search \citep{nilsson2007consistent}, especially to identify weakly relevant features. {
For example, a sequential forward selection would systematically fail in the identification of relevant features structured such as \textit{cliques}, i.e. all features are relevant together but are irrelevant in any subset of the clique. Indeed, let us for instance consider a clique made of two features, i.e. an XOR structure. SFS evaluates the relevance of features in the context of already selected ones. In our example, if both features are marginally irrelevant, then none will be selected preventing also the identification of the other feature. SFS is thus unable to identify features that are only relevant in the context of non-selected ones.
Nevertheless, features that make other features relevant need to be relevant as well and thus may be end up being selected \citep{sutera2018random}. However, if such structures are excluded (e.g., by considering only PCWT distributions), the exhaustive search is not required any more and the all-relevant problem can be solved efficiently \citep{nilsson2007consistent}.} Complex feature structures such as the clique are studied in Chapter~\ref{ch:SRS} in the context of tree-based feature selection. 

Last but not least, the confounding effect is an indirect feature interaction.  One input feature may seem irrelevant to the target but another feature, an external feature known as a confounding factor, provides the key to understand the relationship between the input feature and the output. This confounding effect can be enlarged to features that appear at first sight to be irrelevant but, taking into account the context, are indeed relevant. Such feature interactions are studied in Chapter \ref{ch:context}.

\paragraph{Feature ranking is limited for interpretation}

Feature ranking is extremely limited as it only provides a single ordering of features. This ranking can not render the full complexity of feature interactions or the multiplicity of optimal subsets of features. {The subset evaluation function is also critical, e.g. a univariate criterion will only rank features according to their marginal relevance missing potential interactions}.

In all generality, the most relevant features are not necessary the best ones (or the only ones) to select \citep{guyon2003introduction}. The top-ranked feature may be a rather good feature to predict the target but some other features with lower ranks may perfectly discriminate the target together. Redundancy may have lowered the rank of redundant but highly relevant features \citep{guyon2003introduction}. 
Selecting a top-ranked feature may also be counter-productive, {e.g. selecting only one feature of a clique is not interesting without all the rest of the clique (which might typically be much lower in the ranking).} Although very useful, feature selection/ranking methods however only provide very limited information about the often very complex input-output relationships that can be modelled by supervised learning methods. There is no information about feature dependencies in a classical feature ranking. In case of a contextual effect, two similar ranked features may have totally different roles. One may be always relevant while the relevance of the other one depends on the context. The interpretation is totally different but the rank similarity seems to indicate that they are similarly relevant as well.

Feature ranking does not allow to 
distinguish among features that are directly related to the output and those that influence it only indirectly. 
Applications focusing on direct links (e.g., network inference \citep{de2010advantages,huynh2010inferring,altay2011differential,marbach2012wisdom}, see also Chapter \ref{ch:connectomics}), must therefore filter out the indirect component from feature selection methods.

 There is thus a high interest in designing new techniques to extract more complete information about input-output relationships than a single global feature subset or feature ranking.  A first step towards more interpretable results could be to derive more than one (relevance) score to capture the interest of a feature in several settings. Chapter~\ref{ch:context} extends classical tree-based feature ranking to incorporate a contextual analysis.

\paragraph{Finite sample size makes feature selection more difficult}

High dimensionality together with small sample-size are nowadays typical in many application domains and it poses a great challenge for classical machine learning techniques \citep{raudys1991small,braga2004cross,molinaro2005prediction,saeys2007review} and in particular for feature selection \citep{sima2006should,saeys2008robust,meinshausen2010stability,kuncheva2007stability,bolon2015recent,kuncheva2018feature}. In such conditions, feature selection is however all the more interesting and may help, for example, to counter-balance the disadvantageous features/samples ratio by reducing the number of variables. Nevertheless, studies show that selecting features in such datasets (e.g., micro-arrays of gene-expression) is less reliable \citep{jain1997feature,sima2006should}. In this case, feature selection methods may not necessarily provide a close-to-optimal feature set (i.e., whose  error is close to the minimal achievable error) \citep{sima2006should,hua2009performance}. They also may be unable to find a satisfying feature subset and this does not imply either that an optimal subset does not exist \citep{sima2006should,hua2009performance}. 

Despite an expensive computational cost, the evaluation function must be properly (cross-)validated\footnote{\cite{meinshausen2010stability} however claim that cross-validation may fail for high-dimensional data and alternatively propose a stability selection based on subsampling in combination with selection algorithms.} to avoid the risk of overfitting and overestimated accuracy performances (known as the so-called "\textit{peeking phenomenom}"\footnote{It occurs when data dedicated for testing the model is already used in a pre-processing stage such as feature selection. This results in an optimistically biased estimation of accuracy performances for the selected model \citep{diciotti2013peeking,kuncheva2018feature}.} or as "\textit{selection bias}" problem \citep{ambroise2002selection}) \citep{reunanen2003overfitting,smialowski2009pitfalls,pereira2009machine,diciotti2013peeking,kuncheva2018feature}. Hybrid data (i.e., coexistence of categorical and numerical data) are also worthy of attention \citep{wang2016efficient,jiang2016efficient}. 

In small sample-size conditions, small changes  (e.g., addition/removal of samples or noise added to features \citep{saeys2008robust}) may have a strong influence on the selected feature subset\footnote{Let us note that the existence of multiple sets that are equally good may also lead to some instability in selected feature sets \citep{he2010stable}.}. For the sake of interpretation for instance, one would usually prefer some stability in the outcomes of feature selection algorithm. In a cross-validation feature selection, this would be highly undesirable to have tremendously different selected feature sets from two folds drawn from the same dataset. Stability of feature selection with respect to sampling variation have drawn researchers' attention as another step towards a more robust feature selection \citep{kuncheva2007stability,kalousis2007stability,saeys2008robust,saeys2008towards,abeel2009robust,he2010stable}. 

In small sample-size conditions, irrelevant variables may seem relevant due to random fluctuations. Indeed, the risk of having spurious associations between irrelevant features and the output increases with a decreasing sample-size, especially if the number of features is large \citep{kursa2011all}. Discerning barely but truly relevant from falsely relevant features is a common issue in feature selection with high-dimensional datasets. Solutions, such as introducing an artificial random contrast variable \citep{stoppiglia2003ranking,tuv2006feature,rudnicki2006statistical,kursa2011all, huynh2012statistical} or using dimensionality reduction techniques (by random projections, see random subspace method \citep{ho1998random} in Chapters \ref{ch:trees} and \ref{ch:SRS}), are required to do so. \\

\begin{summary}
Supervised machine learning aims at exploiting a learning set to gain understanding about the interactions among input features and a target output and to build models to make as accurate as possible predictions of the target based on a subset of the inputs. When considering the relation between the input features and the target output, several notions of relevance and redundancy have been defined in the literature and are of interest. These notions may be exploited in many different ways in order to propose feature ranking and feature selection algorithms. Feature selection is often paramount in order to optimize the accuracy of machine learning algorithms, specially in the context of small sample-size and/or high-dimensionality. More and more practical applications are concerned. 
\end{summary}


\chapter{Decision trees and ensemble methods}
\label{ch:trees}

\begin{overview}
In this chapter we explain the essential ideas of tree-based supervised learning methods. We focus on classification problems, i.e. supervised learning problems where the target variable $Y$ takes a finite number of unordered values called classes. Occasionally we however mention how presented ideas would carry over to the case of regression trees. Our goal is to provide the required notions used in subsequent chapters, while also providing an intuitive understanding of the main features tree-based supervised learning. After a brief introduction, Section \ref{sec:DT} provides the main building blocks, namely single decision trees and their greedy recursive partitioning based learning algorithm. Then, in section \ref{sec:RF}, we consider tree-based ensemble methods, and more particularly those used in the subsequent chapters. 
\end{overview}

\epianonymous{May the forest be with you.}{}

\section{Introduction}

A popular and classical approach to solve a complex problem is the divide-and-conquer strategy. It consists in (recursively) dividing the problem into several sub-problems easier to solve. The solution of the original problem is then a combination of the sub-problem solutions. Based on that strategy, the \textit{recursive partitioning method} aims at simplifying a task to carry out on a set of elements (e.g., sorting, labelling, \dots ) by recursively dividing the set into smaller and smaller subsets in such a way that doing this task is easier in each subset than in the original set.
For example, sorting can be achieved efficiently using this strategy: the \textit{merge-sort algorithm} recursively divides the list of elements into smaller and smaller groups until each sub-group is easy (or trivial) to sort, and then combines sorted sub-lists.

The decision tree algorithm successfully applies this method to provide a supervised learning model that partitions the input space into distinct (smaller) subspaces \citep{breiman1984classification,quinlan1986induction,quinlan2014c4}. As a sub-problem, an output value is then assigned to each subspace. From there, the prediction of a new object simply consists in identifying the subspace in which it falls to retrieve its predicted output value.

Single decision trees are simple and consistent supervised models making them easy to use and to understand. They however suffer from variance, and their accuracy performances are consequently affected. 

In order to circumvent variance issues and thus improve model performances, \cite{ho1998random,dietterich2000experimental,breiman2001random} were among the firsts to propose to grow an ensemble of trees instead of settling for a single one. Making a prediction by letting every tree vote and then aggregating these votes results in significant improvement in accuracy.  Many state-of-the-art algorithms stemmed from that idea, including random forests and boosting methods. In particular, a random forest is an ensemble (i.e., a forest) of randomised trees and is at the centre of this thesis. Randomisation is introduced to create some diversity between trees of the same ensemble. The motivating assumption of this approach is that the prediction of an ensemble of weak models is better than the prediction of a single (supposedly stronger) model. 

Furthermore, the success of tree-based methods is also explained by their following common characteristics \citep{geurts2002contributions,louppe2014understanding}:
\begin{description}
	\item[non-parametric nature] 
	by not requiring a priori assumptions on the relationships between inputs and output,
	\item[ability to handle heteregeneous data] 
	by handling learning sets made of a mix of continuous, discrete (ordered or not), and categorical variables (but not necessarily fairly, see Section \ref{sec:biases-cardinality} for more details),
	\item[robustness to outliers or errors in labels] by usually avoiding to completely modify the model to fit a few spurious values in the data, 
	\item[robustness to irrelevant or noisy variables] 
	by automatically selecting the most useful (and relevant) features to build the tree structure (at least to some extent, see Chapters \ref{ch:importances} and \ref{ch:mdi} for more details),
	\item[interpretability] by providing a decision path (with decision trees) or an importance degree for used features (with ensemble methods, see Chapters \ref{ch:importances} and \ref{ch:mdi} for more details),
\end{description}

In this chapter, Section \ref{sec:DT} describes the decision tree algorithm. Then, Section \ref{sec:RF} presents ensemble methods as a way of circumventing the high variance of decision trees.

\section{Supervised Learning with Decision trees}\label{sec:DT}

\subsection{Semantics of tree based prediction models}

\subsubsection{From graph theory to decision tree terminology}

In all generality, let $G=(V,E)$ be a graph where $V$ is a finite set of \textit{nodes} $t$ (also denoted as \textit{vertices} in graph theory), and $E \subset V \times V$ is the set of \textit{edges}. The graph is called \textit{undirected}, if $(t_{i},t_{j}) \in E$ implies that also  $(t_{j},t_{i}) \in E$. In graph theory, a tree is an undirected graph in which any two vertices are connected by exactly one (undirected) path.

We use the term \textit{tree structure} to denote a directed graph obtained from a tree by choosing a node as the root (denoted $t_0$), and by directing all edges `away' from this root (see Figure \ref{fig:treestructuregraph} for an illustrative example). A \textit{branch} $(t_i,t_{i+1})$ is an edge going from $t_i$ towards $t_{i+1}$ where $t_i$ is called \textit{the parent} of $t_{i+1}$, and $t_{i+1}$ is \textit{a  child} of $t_i$. A node is \textit{internal} if it has at least one child, and \textit{terminal} (also known as \textit{leaf node} in the tree terminology) if it has no children.\footnote{Internal nodes generally have several children, while every node has exactly one parent. The number of branches of a tree structure is always equal to its number of nodes minus 1.}

Figure \ref{fig:treestructuregraph} gives an example of a tree structure (i.e., a tree-structured graph). It is represented with the (internal) root node $t_0$ on top and such that nodes at the same depth (i.e., distance with respect to the root node) are horizontally aligned. Nodes $t_1$ and $t_2$ are also internal because they respectively have the children $t_3,t_4$ and $t_5,t_6$. Here $t_3,t_4, t_5,t_6$ are the leaves of the tree.

The following section describes a decision tree model: a tree structure with an additional layer of information.

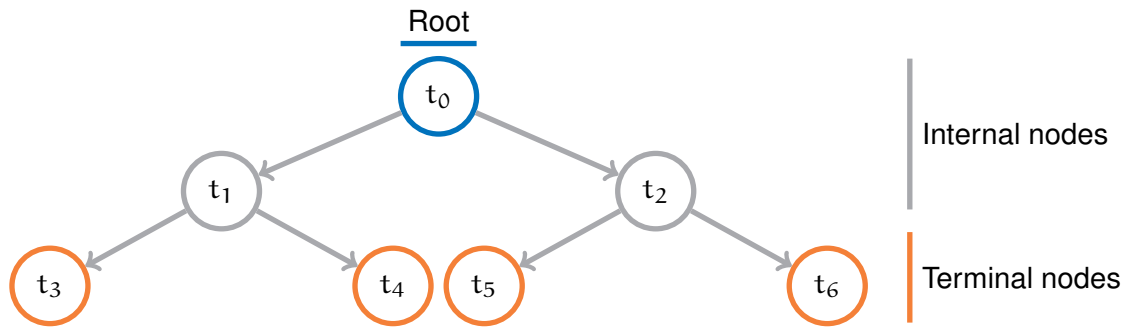
\begin{figure}[htbp]
	\centering

	\begin{tikzpicture}[node distance=0.5cm]
	\mysimpletreestructure

	\end{tikzpicture}
	\caption{Example of a tree-structure. }
	\label{fig:treestructuregraph}
	
\end{figure}

\subsubsection{A tree  structure shaped by the features}

A tree structure $G_T$ recursively partitions the input space $\mathcal{X}$ into subsets where each node $t$ is associated to one specific subset $\mathcal{X}_t$. The subsets corresponding to the terminal nodes are disjoint and such that their union is the original input space $\mathcal{X}$, i.e., $\cup_{t\, \text{is terminal}} \mathcal{X}_t = \mathcal{X}$. The subset corresponding to an internal node is the union of the subsets attached to its children; hence the subset corresponding to the root is always the whole input space. To define all these subsets the tree structure uses features as building blocks. Each internal node typically uses one specific feature, in order to partition its own subset into the subsets corresponding to its children. 

In all generality, a \textit{split} $s$ is a partition of a set $\mathcal{L}$ into a finite number of non-empty and disjoint subsets $\mathcal{L}^i$.\footnote{i.e. such that $\forall i : \mathcal{L}^i \neq \emptyset$, $\forall i \neq j:  \mathcal{L}^i \cap \mathcal{L}^j = \emptyset$, and $\cup_i \mathcal{L}^i = \mathcal{L}$.} In other words, every element of $\mathcal{L}$ belongs to one and only one $\mathcal{L}^i$.
A split on a node $t$, also known as a \textit{test} and denoted by $s_t$, is a split of $\mathcal{X}_t$ using the value of a feature to compute the partition. A \textit{split variable} $v(s_t)$ is the variable on which the test $s_t$ is based and is the one that corresponds to node $t$ in the tree structure.

The cardinality of a split $s_t$, denoted $|s_t|$, corresponds to the number of created subsets, or equivalently the number of possible test outcomes. Cardinalities may or may not be the same for all $t$. The cardinality $|s_t|$ also determines the number of children of node $t$ (i.e., the \textit{node cardinality}) and may depend on the number of possible values for the split variable (the \textit{variable cardinality}). 

A split is said to be \textit{binary} if exactly two subsets are created. However, a node can be divided in more than two by a so-called \textit{multiway} splits. A multiway split is said to be \textit{exhaustive} if the split cardinality is equal to the number of values of the split variable (i.e., one value per branch).

Some authors have looked at more exotic splits. An \textit{oblique split} is made by using a linear combination of several numerical features to create the partition.\footnote{Such splits are said to be \textit{oblique} because they produce separating hyperplanes that are not axis-parallel like classical splits made on a single numerical feature. They lead to shorter trees but are more complex to learn \citep{heath1993induction,murthy1995growing,rokach2008data}.} In an even more general framework, \textit{multivariate splits} also consider complex models (e.g., a decision tree \citep{botta2013walk}) as separating functions, extending axis-parallel and oblique splits \citep{gama2004functional}. \textit{Fuzzy trees} do not longer consider disjoint subsets for children but take advantage of the fuzzy logic to allow some (uncertain) samples to be in several terminal nodes \citep{janikow1998fuzzy,olaru2003complete}.

\begin{figure}[htbp]
	\centering
	
	\begin{tikzpicture}[node distance=0.5cm]
	\mysimplebinaryclassificationtreeexample

	\end{tikzpicture}
	\caption{Example of a decision tree model: a tree-structure that recursively splits the $2\times 2$ input space with four colours.}
	\label{fig:binarytreeexample}
	
\end{figure}
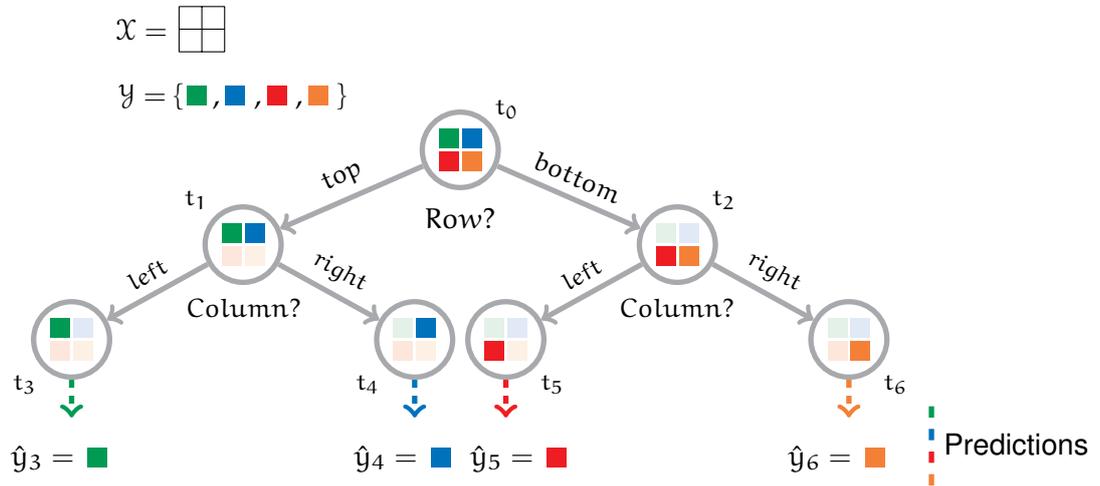

\subsubsection{Decision tree models}

A \textit{decision tree model} $T: \mathcal{X} \rightarrow \mathcal{Y}$ recursively partitions the input space $\mathcal{X}$ into subspaces to provide an input-output model in the form of a tree structure (see Figure \ref{fig:binarytreeexample} for an illustrative example). The model is such that 
\begin{enumerate}
	\item each node $t$ corresponds to one subset $\mathcal{X}_t\subseteq \mathcal{X}$, 
	in particular the one associated to the root node is the input space $\mathcal{X}$ itself,
	
	\item each internal node $t$  is labelled with a split $s_t$, 
	
	\item each branch going from an internal node $t$ indicates one possible outcome $i$ of the split $s_t$, and leads to one child $c_i$ of $t$ such that its subset is  $\mathcal{X}_{c_i} = \mathcal{X}_t \cap \mathcal{X}^i$ where $\mathcal{X}^i \subset \mathcal{X}$ is the subset of inputs satisfying outcome $i$, 
	\item all terminal nodes $t$ have their subsets (called \textit{terminal subsets}) assigned to a predicted value $\hat{y}_t\in \mathcal{Y}$; $\hat{y}_t$ is also called the label of the leaf $t$. 
\end{enumerate}

Figure \ref{fig:binarytreeexample} shows a decision tree model that decomposes an input space of two dimensions (represented by a $2\times 2$ matrix) with four possible output values (i.e., \textit{green}, \textit{blue}, \textit{red}, or \textit{orange}) using the tree-structure of Figure \ref{fig:treestructuregraph}.
The root node $t_0$ corresponds to the complete input space $\mathcal{X}$. Its split is made on the vertical axis ("Which row?") and gives two children ($t_1$ and $t_2$) corresponding to the two possible outcomes (i.e., top or bottom). Each child has its own subset that is still made of two colours. By splitting them on the horizontal axis ("Which column?"), we obtain four terminal nodes, each with a subset of only one colour. At this point, there is no interest in further partitioning these subsets. The output label associated to each terminal node is immediate and corresponds to the remaining colour. The prediction of the model for a new input value $\mathbf{x}$ is the associated value of the terminal node reached by $\mathbf{x}$.
Let us observe that, for each internal node $t$, the input subsets of its children are disjoint and their union is the subset of that node $t$, i.e., $\mathcal{X}_t = \cup_{i=1}^{|s_t|} \mathcal{X}_{c_i}$.\\

Let us consider a more realistic classification problem, described in Example \ref{ex:example_tree}, that will be used to illustrate the two following decision tree models.  
\begin{example}\label{ex:example_tree}
	Let us consider a classification problem with two input variables $X_1$ and $X_2$ with two possible output classes $c_1$ and $c_2$. Figure \ref{fig:example_data} illustrates the learning set where each input variable corresponds to one dimension. At first sight, based on Figure \ref{fig:example_data(a)}, this is not straightforward to give a model that will perfectly separate the two classes. For the sake of illustration, Figure \ref{fig:example_data(b)} gives a decomposition of the input space that provides a perfect separation between objects of different classes. 
\end{example}

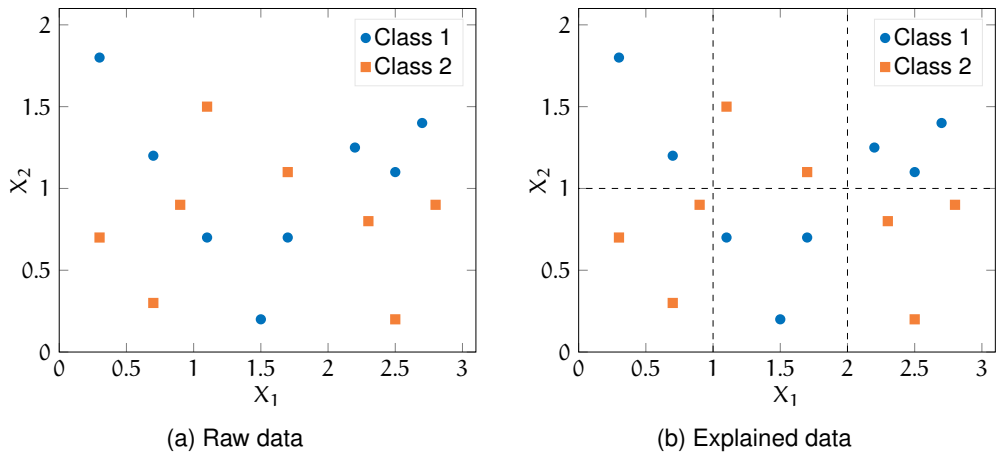
\begin{figure}[htbp]
	
	\subfloat[Raw data]{
		\begin{tikzpicture}[x=2cm,y=2cm,scale=0.8]
		\myclassificationexample
		
		\end{tikzpicture}\label{fig:example_data(a)}}\hspace{1em}
	\subfloat[Explained data]{
		\begin{tikzpicture}[x=2cm,y=2cm,scale=0.8]
		\myclassificationexampleexplained
		\end{tikzpicture}\label{fig:example_data(b)}}
	\caption{Example of a classification problem with two input variables $X_1$ and $X_2$ and two possible values of the output $y$ ($c_1$ and $c_2$). Blue dots correspond to objects class $c_1$ while orange squares correspond to objects class $c_2$. On the right figure, the underlying decomposition of the input space is explicitly given.}
	\label{fig:example_data}
\end{figure}

\begin{definition}
		A \textbf{binary decision tree} is a decision tree model in which all internal nodes have exactly two children.
\end{definition}

This is the case when all splits are binary, that is to say, when there are only two possible outcomes (e.g., \textit{true} or \textit{false}, \textit{yes} or \textit{no}), or when all input features are binary.

 A split $s$ divides the input space between the part that satisfies the test $\mathcal{X}^s$ and the rest $\mathcal{X}^{\bar{s}}$.
 Therefore, the input subspace of the left child $c_l$ of $t$  (i.e., satisfying the test) is $\mathcal{X}_{c_l} = \mathcal{X}_t \cap \mathcal{X}^s$
, and the input subspace of the right child $c_r$ is $\mathcal{X}_{c_r} = \mathcal{X}_t \cap \mathcal{X}^{\bar{s}} = \mathcal{X}_t \cap (\mathcal{X} \setminus  \mathcal{X}^s)$.

Figure \ref{fig:example_bintree} shows a binary classification tree applied on Example \ref{ex:example_tree}.\\

\begin{figure}[htbp]
	\centering
	\scalebox{0.9}{	
	\begin{tikzpicture}[x=2cm,y=2cm,minimum size=0.8cm]
	\mybinaryclassificationtreeexample
	\end{tikzpicture}
	}
	
	\caption{Example of a binary classification tree applied on Example \ref{ex:example_tree}.}
	\label{fig:example_bintree}
\vspace*{10mm}
	
	\scalebox{0.9}{
	\begin{tikzpicture}[x=2cm,y=2cm,minimum size=0.8cm]
	\centering
	
	\mymultiwayclassificationtreeexample
	
	\end{tikzpicture}
	}
	\caption{Example of a multiway classification tree applied on Example \ref{ex:example_tree}.}
	\label{fig:example_multitree}
\end{figure}

Decision trees are typically binary but they can also be built using multiway splits. 
Figure \ref{fig:example_multitree} illustrates a multiway decision tree applied on Example \ref{ex:example_tree}. In comparison with the binary decision tree of Figure \ref{fig:example_bintree}, threeway splits are used at the second level to 
create three children corresponding to three intervals of values of $X_{1}$ (in the example, intervals $[0,1],]1,2],]2,3]$). Notice that while the 
binary tree uses two more splits, it manages to find the same final partition.

When all input variables are categorical, let us define a decision tree using multiway exhaustive splits: 
\begin{definition}
	Let all input variables $V=\{X_1, \dots, X_p\}$ be categorical. A \textbf{multiway exhaustive decision tree} is a decision tree model in which splits on feature $X_{i}$ yield exactly $|X_i|$ children, namely one for each possible value of the split variable $X_i$.
\end{definition}

Multiway exhaustive splits\footnote{In the rest of this thesis, multiway splits on categorical features will always be exhaustive, i.e., one child for each value and not for only a subset of values. Therefore, the term "exhaustive" is sometimes omitted.} are typically of various cardinalities as they depend on the number of possible values of each split variable. 
Notice that for such a tree, the maximal depth is limited by the number of features, as each feature can be used at most once along a path.

\subsection{Learning a decision tree model from data}

The tree model aims at fitting at best the partition induced by $\mathcal{Y}$ over $\mathcal{X}$ and thus approximating the Bayes model (i.e., the optimal model yielding the lowest error rate). 
In practice, the partition induced by $\mathcal{Y}$ over $\mathcal{X}$ is unknown and the input space is only partially observed through a learning set. Given a learning set $\mathbf{LS}$, a decision tree model $T^{\mathbf{LS}}$ is learnt on $\mathbf{LS}$ and provides a partitioning of $\mathbf{LS}$, denoted $\varphi$. While growing the decision tree, the objective is to find the partitioning $\varphi$ that provides the lowest possible error rate, the optimal induced partitioning $\varphi^*$. Assuming that the learning set represents faithfully the input space, $\varphi^*$ should be close to the partition induced by $\mathcal{Y}$ over $\mathcal{X}$.

The tree learning algorithms that we consider in this thesis (and which have become a standard in supervised learning) proceed in a top-down fashion, by starting with the root node and progressively developing the tree structure, while at each step choosing a node to split and a way to split the node, until the tree fits the learning sample sufficiently well (see side note on page \pageref{sn:tree-growing-algo}). 

\begin{sidenote}{Generic top-down decision tree growing algorithm} \label{sn:tree-growing-algo}
\begin{itemize}
\item \emph{Initialization:} create the root node of the tree, attach the whole learning set to this node, and set the list of open nodes to contain only this node.
\item \emph{Recursion:} {until the list of open nodes is empty, remove a node from the list of open nodes {(following a given growing strategy\footnote{{Well-known strategies are \textit{depth-first}, \textit{breadth-first} or \textit{best-first}. Each strategy may yield different decision trees if the stop splitting rule is global, i.e. based on the whole tree.}})}, and decide whether this node should be split:} 
\begin{itemize}
\item If yes, the node becomes a test node, and a good split for it is determined and used to split the learning set of the node into two or more subsets. For each subset a child node is created and inserted in the list of open nodes. 
\item If no, the node becomes a leaf and a class label is assigned to it based on its learning subset.
\end{itemize}
\end{itemize}
\end{sidenote}

This procedure 
aims at finding a suitable tree structure, and at associating the right class label to each one of its terminal nodes. 
This thought has been summarised by \cite{breiman1984classification} as follows:

\begin{displayquote} \textit{
		``It turns out that the class assignment problem is simple. The whole story [of the construction of a tree] is in finding good splits and in knowing when to stop splitting.'' \citep{breiman1984classification}
}\end{displayquote}

The three next sections are dedicated to a detailed description of these three key steps of a decision tree learning procedure. In Section \ref{sec:splittingrules}, we describe how to find the variable (and the associated test) that provides a "good" split for a learning subset. In Section \ref{sec:stoppingcriteria}, we review some stopping criteria that define the end of the building process. In Section \ref{sec:treeprediction}, how to choose the labels attached to leaves and used for making predictions.

\subsubsection{Splitting rules} \label{sec:splittingrules}

\paragraph{The impurity framework} The \textit{growing/learning procedure} of a decision tree model recursively divides the learning set in subsets of learning samples $\mathbf{LS}_t$ where $\mathbf{LS}_t$ is the set of all objects reaching node $t$ (i.e., $\mathbf{LS}_t = \{(\mathbf{x},y)| \mathbf{x} \in \mathcal{X}_t\}$). For a given node $t$ and its set of learning samples $\mathbf{LS}_t$, let us define $p(c_j|t)$ as the proportion of samples in $\mathbf{LS}_t$ such that $y=c_j,\, c_j\in\mathcal{Y}$. The sum of $p(c_j|t)$ for all $c_j\in \mathcal{Y}$ is $1$. Based on $\mathbf{LS}$, the learnt model tries to mimic the optimal induced partitioning $\varphi^*$. 

A good decision tree is one that minimises the generalisation error while minimising some complexity criterion of three, e.g., the size of the tree. Even though several trees can equivalently represent the optimal partitioning $\varphi^*$, the shorter tree is usually the easiest to interpret and consequently the best one. Naively, one can generate all possible decision trees in order to keep the best one (minimising a criterion depending on the accuracy performances and the complexity of the model). However, even if the number of trees may be finite when the number of (discrete/categorical) features is limited, this number can increase exponentially and becomes intractable from a computational point of view when considering a large number of (continuous) features. 

Circumventing the intractability of an exhaustive search for the optimal tree model (giving $\varphi^*$), the idea of \cite{breiman1984classification}'s heuristic algorithm is to keep splitting nodes until they are (almost\footnote{The purity of a node is a natural stopping criterion, but some other criteria exist and may stop the growing process before having pure nodes. See Section \ref{sec:stoppingcriteria} for more details.}) \textit{pure}. The resulting partitioning is expected to be close to $\varphi^*$. A node $t$ is \textit{pure} when all learning samples reaching that node ($\mathbf{LS}_t$) are of the same class label $c_j$ ($p(c_j|t)=1,\, c_j \in \mathcal{Y}$ and $p(c_i|t)=0$ for all $c_i\ne c_j$,$c_i\in \mathcal{Y}$) (see terminal nodes of Figures \ref{fig:binarytreeexample}, \ref{fig:example_bintree} and \ref{fig:example_multitree}). Hereafter, we refer to the output distribution of a pure node as a \textit{pure distribution}.
A pure node is always terminal because there is no gain in splitting more its samples. Conversely, the impurity of a node is the largest when all class labels are equally likely ($p(c_j|t) = p(c_i|t)$ for all $c_i,c_j \in \mathcal{Y}$). 

From that, one can logically assume that the purer a node is, the more striking is the majority class making the prediction easier and usually better.

Following the framework of \cite{breiman1984classification}, let us define an \textit{impurity measure} $i(t)$ as a non-negative function $\phi$ that evaluates the purity of a node $t$ from the vector of class proportion samples $\pi$ (where the $j^{th}$ term of $\pi$, $\pi^j = p(c_j|t)$) and verifies the following three properties \citep{breiman1984classification,joly2017exploiting}:
\begin{enumerate}
	\item $i(t)$ is minimal (typically equal to $0$) when the node $t$ is pure, i.e., $p(c_j|t) = 1$ for some $c_{j} \in \mathcal{Y}$ and $\forall  c_{i} \neq c_{j} : p(c_i|t)=0$, 
	\item $i(t)$ is maximal only when the distribution of output values in $\mathbf{LS}_t$ is uniform, i.e.  such that $p(c_i|t) = \frac{1}{|\mathcal{Y}|}$ for any $c_i \in \mathcal{Y}$,
	\item $i(t)$ is not biased towards some output values (symmetrical with respect to the class proportion samples), e.g., the impurity measures of two nodes $t_1$ and $t_2$ are the same if $\pi_2$ is a permutation\footnote{The same numerical values but not necessarily in the same order.} of $\pi_1$.
\end{enumerate}

\paragraph{The goodness of a split}
A good split is one that reduces the impurity $i(t)$ of a node $t$, i.e., such that children of $t$ are purer than $t$ itself. The goodness of a split dividing a node $t$ in two\footnote{For the sake of clarity, only binary splits are considered hereafter but one can naturally generalise what follows for multiway splits by considering $|s_t|$ children instead of two.} can be formalised using the impurity measure as follows:
\begin{definition}
	Let $s$ be a binary split that divides a node $t$ into a left node $t_L$ and a right node $t_R$. 	
	The decrease of impurity is
	\begin{eqnarray}
	{\Delta {i(s, t)}} &=& i(t) - \dfrac{N_{t_L}}{N_t} i(t_L) - \dfrac{N_{t_R}}{N_t} i(t_R)\\
	&=& i(t) - p_{t_L} i(t_L) - p_{t_R} i(t_R)
	\end{eqnarray}
	where $N_t$ is the number of learning samples in node $t$, $N_{t_L}$ and $p_{t_L}$ (respectively, $N_{t_R}$ and $p_{t_R}$) are the number of samples and the proportion of samples that fall into $t_L$ (resp., $t_R$).
\end{definition}

We will discuss later on several impurity measures that may be used for growing decision trees. Once the impurity is chosen, the greedy procedure for growing a decision tree consists in searching at each node for the split that yields locally the largest decrease of impurity $\Delta i(s,t)$ among all valid splits.

\paragraph{Candidate splits for different types of features} 
{Let $S_{t,m}$ be the set of all candidate splitting functions for node $t$ on feature $X_m$, consisting of all candidate ways to divide $\mathcal{X}_{t,m}$ in two non-empty subsets, where $\mathcal{X}_{t,m}$ denotes the set of all values of $X_m$ observed in the learning sample of node $t$.

\paragraph{}If $X_m$ is an \textit{unordered variable}, defining a split amounts to find two non-empty subsets $\mathcal{X}_{t_L,m}$ and $\mathcal{X}_{t_R,m}$ such that every element of $\mathcal{X}_{t,m}$ is in one and only one of them, i.e., $\mathcal{X}_{t,m}= \mathcal{X}_{t_L,m}\cup \mathcal{X}_{t_R,m}$ and $ \mathcal{X}_{t_L,m} \cap \mathcal{X}_{t_R,m}=\emptyset $. In that case, $S_{t,m}$ can be formally defined as follows:
\begin{eqnarray} S_{t,m} = \{s(\mathbf{x}) = \mathbb{1}(x_m \in \mathcal{X}_{t_L,m})  | \mathcal{X}_{t_L,m} \subset \mathcal{X}_{t,m} \}
\end{eqnarray}  
where $\mathbf{x}$ is a vector of input values and $x_m$ is the value of $X_m$. All splits guide samples whose value $x_m$ is in $\mathcal{X}_{t_L,m}$ in the left child, while all others go in the right child.
Let us note that $\mathcal{X}_{t_L,m}$ must be non-empty, and a proper subset of $\mathcal{X}_{t,m}$ to ensure that $\mathcal{X}_{t_R,m}$ is also non-empty.\footnote{In practice, it prevents one of the child nodes from having zero learning samples (i.e., $N_{t_L} = 0$ or $N_{t_R} = 0$) which corresponds to a split devoid of interest.} 
A combinatorial analysis gives that the number of possible splits $|S_{t,m}|$ is equal to $2^{|X_{t,m}|-1}- 1$ where $|X_{t,m}|$ is the cardinality of $X_{t,m}$.\footnote{Taking into account the fact that exchanging $\mathcal{X}_{t_L,m}$ with $\mathcal{X}_{t_R,m}$ leads to an equivalent split.}}

\paragraph{}If $X_m$ is an \textit{ordered variable}, the logic between values should be preserved by the split. Consequently, the two disjoint non-empty subspaces $\mathcal{X}_{t_L,m}$ and $\mathcal{X}_{t_R,m}$ must be such that every element in one subspace has a split variable value strictly lower than the split variable value of any element from the other subspace, i.e., $\mathbf{x}_{t_L,m} < \mathbf{x}_{t_R,m}$ for all pairs $(\mathbf{x}_{t_L}, \mathbf{x}_{t_R}) \in \mathcal{X}_{t_L,m} \times \mathcal{X}_{t_R,m}$. 
{An equivalent way to fulfil that condition is to determine a threshold value $\tau$ (also called cut-point), and  to assign every value below $\tau$ to the left child and to the right child otherwise, i.e.: 
\begin{eqnarray}
S_{m} = \{s(\mathbf{x}) = \mathbb{1}(x_m \le \tau)  | \tau \in \mathcal{X}_{m} \}
\end{eqnarray} where $\tau$ is a threshold value referred to as the \textit{cut-point} of the split. 

In practice, it suffices to consider a single candidate cut-point between each pair of successive values of the concerned feature observed in the learning subset of the node $t$ (in most implementations it is the mid-point). Indeed, different cut-points between a given pair of such successive values yield the same partition of the learning sample of the considered node, and are thus equivalent from the viewpoint of impurity reduction. The number of different splits to consider is thus $|S_{t,m}| = |X_{t,m}| - 1$.  Let us however notice that all cut-points between two successive values (as observed in the learning set) are not necessarily equivalent outside the learning set (see Figure \ref{fig:splitselection} for an illustrative example).}

	\begin{figure}[hbtp]
	\centering
	\begin{tikzpicture}[x=4cm,y=4cm]
	
	\mysplitselection
	
	\end{tikzpicture}
	
	\caption{Split selection. Projection on the $X_1$ axis of samples reaching the second node that splits on $X_1$ (i.e., $X_1\le 2$) on the left branch (i.e., $X_2\le 1$) of the decision tree of Figure \ref{fig:example_bintree}. Filled circles and squares are samples from the learning set $\mathbf{LS}$ and non-filled ones are samples from the testing set $\mathbf{TS}$ (unknown in the learning phase). In practice, all cut-points in the red zone (i.e., between two successive values $]1.7,2.3[$) are equivalent on the learning set and $\tau=2$ was chosen in Figure \ref{fig:example_bintree}. However, other values such as $\tau=1.8$ or $\tau=2.2$ also perfectly separate classes in the learning but not on the test set.}
	\label{fig:splitselection}
\end{figure}
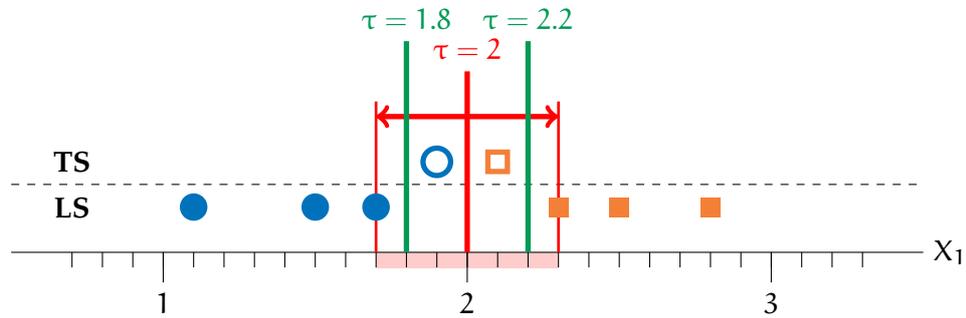

Let $S_t$ be the set of splits on all $p$ features and such that $S_t= \cup_{j=1}^{p} S_{t,j}$. 
The best split $s_{t}^*$ is therefore 
\begin{eqnarray}\label{eqn:s_t^*}
{s_t^* = \argmax_{s\in S_t} \Delta i(s,t).}
\end{eqnarray}

In practice, Equation \ref{eqn:s_t^*} is solved by exhaustively considering all features and all possible splits on those features (either all cut-points $\tau$ or all subsets $\mathcal{X}_{t_L,m}$). This approach however only optimises the split for the current node. Growing a decision tree while foreseeing some future splits is known as (limited) lookahead search and has been shown to provide shorter but not significantly better trees while being computationally more costly \citep{murthy1995lookahead,louppe2014understanding}.

\paragraph{Suitable impurity measures}

Any function satisfying the three properties of an impurity measure can be plugged in the decision tree algorithm. Classical impurity measures used for classification problems\footnote{The above ideas have also been extended to regression problems, where the (empirical) variance is typically used to measure impurity \cite{breiman1984classification}.} are the Shannon entropy and the Gini index.  

\begin{definition}
	The impurity function $i_h(t)$ of a node $t$ derived from Shannon entropy \citep{shannon1949mathematical} is 
	\begin{eqnarray}
	i_h(t) = -\sum_{j=1}^C p(c_j|t) \log_2(p(c_j|t))
	\end{eqnarray}
	where $C$ is the number of possible classes. 
\end{definition}

Shannon entropy quantifies the uncertainty of a discrete random variable based on its probability density. It is non-negative, maximal for a uniform density, and equal to zero (hence minimal) when only one value has a strictly positive probability. Notice that the entropy-based impurity reduction $\Delta i_h(s,t)$ is actually an estimation, based on the learning subset reaching the node $t$, of the mutual information between the split outcome and the output $t$. This impurity reduction is also non-negative, and equal to zero only if the class proportions in the two subsets are identical.

\begin{definition}
	The impurity function $i_g(t)$ of a node $t$ derived from Gini index \citep{gini1912variabilita} is 
	\begin{eqnarray}
	i_g(t) = \sum_{j=1}^C p(c_j|t) (1-p(c_j|t))
	\end{eqnarray}
	where $C$ is the number of possible classes. 
\end{definition}

The Gini index quantifies the dispersion of a distribution. The gini-based impurity $i_g(t)$ aims at evaluating the error rate of a random labelling of objects from $\mathbf{LS}_t$ following the distribution of labels within node $t$, $p(y|t)$. That is, the probability of labelling an object with class $c_j$ is given by the probability $p(c_j|t)$ while $1-p(c_j|t) = \sum_{i\ne j}^C p(c_i|t)$ is the probability of error when labelling an object $c_j$. Similarly to Shannon entropy, $i_g(t)$ is non-negative, maximal for a uniform distribution, and equal to zero and hence minimal for a pure distribution. The resulting impurity reduction is also non-negative, and equal to zero only if the class proportions in the two subsets are identical.

\paragraph{Extension to regression trees}

In order to extend the tree growing algorithm to the case where the output is numerical (i.e. for regression), various alternative goodness of split measures have been defined in the literature. In particular, for 'least squares regression', a natural way to do this is to use the same approach as above while using as  ``impurity'' measure the variance of the output $Y$ estimated from a learning subset \citep{breiman1984classification}. 

\begin{definition}
	The ``impurity'' function $i_v(t)$ of a node $t$ derived from the variance is 
	\begin{eqnarray}
	i_v(t) = \dfrac{1}{N_t} \sum_{y\in \mathcal{Y}_t} (y-\bar{y}_t)^2
	\end{eqnarray}
	where $N_t$ is the number of learning samples in node $t$ and $\bar{y}_t = \dfrac{1}{N_t} \sum_{y\in\mathcal{Y}_t} y$ is the average of $y$ in $\mathbf{LS}_t$.
\end{definition}
The variance estimate $i_v(t)$ is non-negative and equal to zero when all samples have the same target value (equal to the mean value). It also leads to an impurity reduction measure that is non-negative.

\subsubsection{Stopping rules and pruning} \label{sec:stoppingcriteria}

In the previous section, we described how to develop a tree by starting with its root node and splitting its nodes so as to maximise at every step the impurity reduction. 

Given the recursive nature of the growing process, there comes a stage when it is no longer possible to further divide a sample set. The splitting process then has no choice but to stop if there is no more valid splits for the node. It occurs in the two following situations, seen as inherent stopping criteria:
\begin{enumerate}
	\item \textit{Constant output value}: all learning observations reaching the node have the same output value, meaning that the impurity of the learning subset is already equal to zero and hence can not be further reduced,  
	\item \textit{Constant input values}: all learning observations reaching the node have the same value for every input feature, so that the set of available candidate splits is empty. 
\end{enumerate} 

Let us note that all learning samples may have the same input values (case (b)) while not having the same output value. 

\begin{definition}
	A decision tree is said to be \textbf{fully developed} if all learning subsets corresponding to its leaves  have either a constant output  (case (a)) or constant inputs (case (b)) and consequently none of the leaves could have been split in a meaningful way. 
\end{definition}

Fully developed trees are often overfitting the training data. To limit this phenomenon, additional criteria  for stopping to split have been imposed. 
\begin{enumerate}
	\item \textit{Complexity-based stopping criteria} aim at preventing the decision tree from becoming too complex. Typical complexity measures are the total number of nodes or the maximal (or average) depth of the tree.
	
	\item \textit{Impurity-based stopping criteria} stops the growing procedure when the possible impurity reduction is not significant anymore. Indeed, since the growing procedure recursively splits the learning set, the number of learning samples reaching deeper nodes decreases typically rather quickly with the tree depth. Deeper nodes therefore typically yield impurity reductions that are less and less significant from a statistical point of view. Thus it has been proposed to stop splitting if 
	\begin{enumerate}
	\item  the size of the learning subset of a node is below a given threshold, or if learning subset sizes of its child nodes would be below a given threshold, 
	\item  if the best achievable impurity reduction is too small given the size of the learning subset. 
	Instead of setting explicitly a threshold, some statistical measures (e.g., a $\chi^2$ test or a permutation test) can associate a split impurity reduction to a significance level (e.g., a p-value) for which it is easier to find an interpretable threshold value.\end{enumerate}
\end{enumerate}

It should be noted that a single criterion may be sufficient to stop the construction of a tree although several can be combined.
In practice, all criteria are defined by a hyper-parameter whose value must be carefully chosen. By being too restrictive with their values, these criteria would result in a shallow tree that potentially misses some information about the output in the dataset (i.e., a situation of under-fitting). On the other hand, choosing parameter values that are too permissive would not limit the size of the tree enough, causing over-fitting and sub-optimal performances (in terms of generalisation error). All parameters must therefore be carefully tuned in order to achieve the best trade-off for the size of the tree.

Although those stopping criteria may give in practice good results, they may also lead to sub-optimal trees. A few nodes more or less might indeed sometimes produce a significantly better tree. Another way of finding the best model is to first build a fully developed tree and then choose one of its subtrees a posteriori. Techniques following this approach are known as \textit{post-pruning methods}. In practice, a post-pruning method consists in finding the best subtree $T^* \subseteq T$, obtained by contracting an internal node of the fully developed tree $T$  (i.e., replacing it by a terminal node and dropping all its descendent nodes), say one which minimises a given criterion such as the error rate on a independent test set for example.

Therefore, stopping criteria that preventively control the growing of the tree are usually referred to as \textit{pre-pruning methods}.

\subsubsection{Labeling the leaves} \label{sec:treeprediction}

The prediction $T(\mathbf{x})=\hat{y}(\mathbf{x})$ for an input vector $\mathbf{x}$ is obtained by propagating $\mathbf{x}$ through the tree (following branches according to its values) and then returning the prediction (or label) $\hat{y}_t$ associated to the terminal node reached by $\mathbf{x}$.

During the learning stage, each terminal node $t$ must thus receive a label $\hat{y}_t \in \mathcal{Y}$. The choice of $\hat{y}_{t}$ of course aims at maximizing accuracy and hence essentially depends on the nature of the output variable and on the loss function used to measure accuracy. In practice the output label values found in the learning subset of each leaf are used to choose a label such that in the end the total loss is minimised over the learning set.

\paragraph{For classification trees and zero-one loss} Let us consider a decision tree model to predict $\mathcal{Y} = \{c_1,\dots, c_J\}$. If the goal is to minimise the probability of mis-classification, the label $\hat{y}_t$ associated to a terminal node $t$ is chosen as the most frequent class (output value) among objects reaching node $t$. That is
\begin{eqnarray}
\hat{y}_t = \argmax_{c_j} p(c_j|t).
\end{eqnarray}

{Indeed, in classification tasks, the commonly used loss is the zero-one loss, which for a decision tree and its learning set sums up to $$L^{0-1}= \sum_{t}\sum_{(\mathbf{x}^i, y^i)\in \mathbf{LS}_t} \mathbb{1}(y^i\ne \hat{y}_t),$$
where the outer sum is over all leaves of the tree. And thus, choosing for each leaf its label as the most frequent class in its learning subset $\mathbf{LS}_t$ therefore minimises the total zero-one loss over the complete learning set.}

\paragraph{For regression trees and square loss} Let us consider a regression tree model ($\mathcal{Y} \in \mathbb{R}$). If the goal is to minimise the expected square error, the label $\hat{y}_{t}$ associated to a terminal node $t$ is chosen as the average of all output values of objects reaching this terminal node. That is
\begin{eqnarray}
\hat{y}_t = \dfrac{1}{N_t} \sum_{y_t \in \mathcal{Y}_t} y_t.
\end{eqnarray}

{Indeed, in regression tasks, the commonly used loss is the square loss, which for a regression tree and its learning set sums up to  $$L^{se}=\sum_{t}\sum_{(\mathbf{x}^i,y^i) \in \mathbf{LS}_t} (y^i - \hat{y}_t)^2.$$ And thus, choosing for each leaf $t$ its label as the average of all $y^i$ values in $\mathbf{LS}_t$ therefore minimises the total square loss over the complete learning set.}

\subsection{Interpretability of decision tree models}\label{sec:dt-interpretation}
One of the main strengths of decision tree models is their interpretability \citep{hastie2005elements}. A decision tree model can be naturally represented in the form of a tree-structured graph or seen as a set of mutually exclusive rules. It recursively partitions the input space into subregions.  Each of these regions is described by a sequence of feature-based tests.

A decision tree model also helps to fully understand the reasons for a prediction. By following the path of a sample from the root to the terminal node, one can directly retrieve the explanation for the predicted value. This property is desirable in many domains and in particular in medical applications where a model can provide sensitive results such as a diagnosis or a prognosis. In such cases, understanding the reasons driving the model to some conclusions is crucial as wrong decisions might have severe consequences. 

In practice, the tree structure gives all features that are involved in the model. More specifically, the followed branch gives the features used for the prediction in particular and the sequential order in which they are used. In addition to that, one can follow the progress of a prediction by tracking the evolution of output values (i.e., class proportions or output averaged value) within nodes in the path. Figure \ref{fig:classproportions} is another graphical representation of the classification tree shown in Figure \ref{fig:example_bintree} which  highlights class proportions within nodes. Note that sometimes left and right nodes are rearranged so that the left child always corresponds to an increase of the same class (even if the splitting function must be reversed). However, it can be laborious to understand each decision/node of a decision tree, especially if it is large or deep (see \citep{luvstrek2016makes} for a study of factors impacting the interpretability of a decision tree). 

\usetikzlibrary{calc}
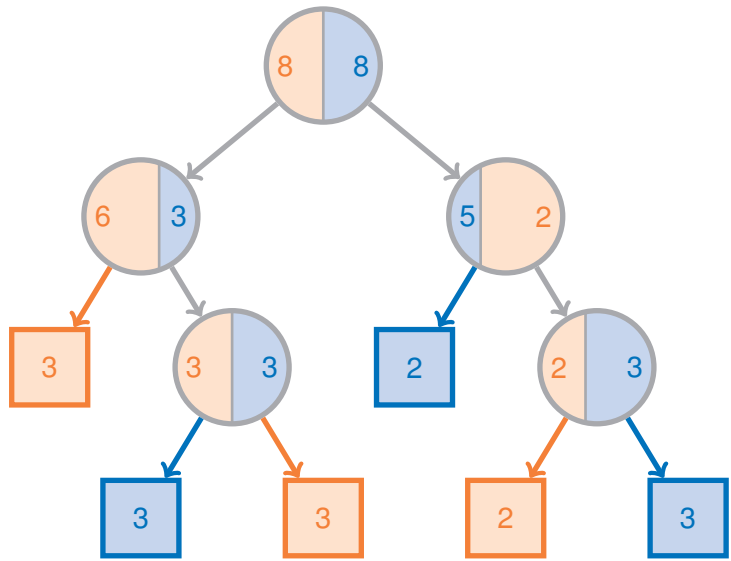
\begin{figure}[htbp]
	\centering
	\begin{tikzpicture}[x=2cm,y=2cm,node distance=0.5cm,  minimum size=1.5cm]
	\mybinaryclassificationtreeinterpretation	
	\end{tikzpicture}
	\caption{Another representation of the binary classification tree in Figure \ref{fig:example_bintree}. In each node, class proportions are represented by the part of the circle filled with the class colour and number of samples of each class are given.}
	\label{fig:classproportions}
	
\end{figure}

Furthermore, one may exploit the impurity reductions computed when growing the tree in order to measure the ``relevance'' of the different input features (see e.g. \citep{breiman1984classification}). Since we will focus on this idea in the subsequent chapters of this thesis, we do not elaborate too much on it here.

On the other hand, an important caveat concerning interpretability stems from the high learning variance of the decision tree growing algorithms \citep{geurts2002contributions} and the so-called ``masking effect'' \citep{breiman1984classification}. A high learning variance means that small changes to the learning set may lead to large changes in the learnt model. The masking effect denotes situations where several candidate splits on different features yield roughly the same impurity reduction, but one of the features is always slightly better so that none of the other ones has a chance to be selected by the tree-growing algorithm. We highlight both effects on the ``XOR'' example explained in Figure \ref{fig:xor-variance}.

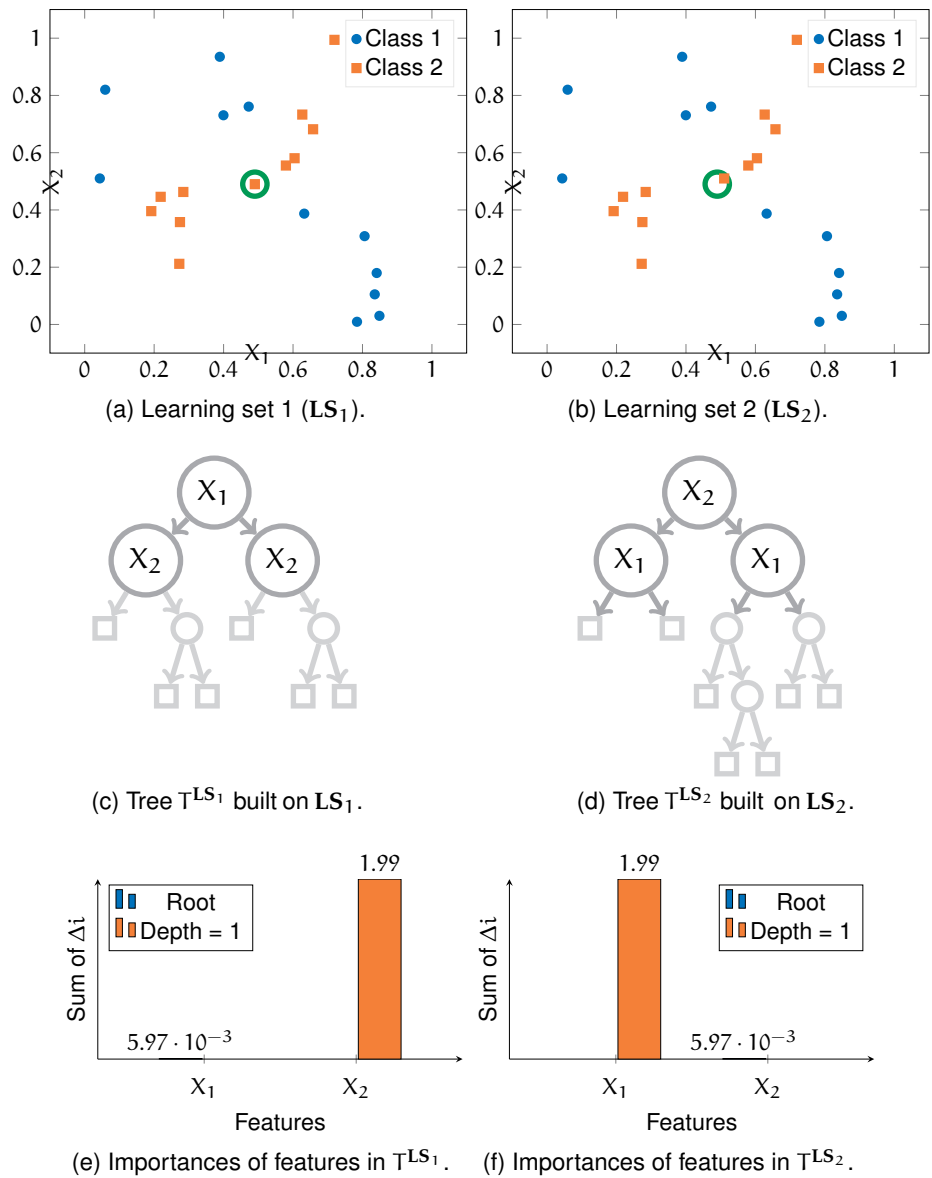
\begin{figure}[htbp]
	\centering
	\subfloat[Learning set 1 ($\mathbf{LS}_1$).\label{fig:xor-variance-a}]{\begin{tikzpicture}[x=1cm,y=1cm]
	\myxorvariancelsone
	\end{tikzpicture}}
	\subfloat[Learning set 2 ($\mathbf{LS}_2$).\label{fig:xor-variance-b}]{\begin{tikzpicture}[x=1cm,y=1cm]
	\myxorvariancelstwo
	\end{tikzpicture}}\\

	\subfloat[Tree~$T^{\mathbf{LS}_1}$~built~on~$\mathbf{LS}_1$.\label{fig:xor-variance-c}]{\begin{tikzpicture}[x=1cm,y=1cm]
	\myxorvariancetreeone
	\end{tikzpicture}} \hspace{2.8cm}
	\subfloat[Tree~$T^{\mathbf{LS}_2}$~built~ on~$\mathbf{LS}_2$.\label{fig:xor-variance-d}]{\begin{tikzpicture}[x=1cm,y=1cm]
	\myxorvariancetreetwo
	\end{tikzpicture}}\\

	\pgfplotsset{
		compat=newest,
		xlabel near ticks,
		ylabel near ticks
	}

\subfloat[Importances of features in $T^{\mathbf{LS}_1}$.\label{fig:xor-variance-e}]{\begin{tikzpicture}[x=1cm,y=1cm]
\myxorvarianceimpone
\end{tikzpicture}}
\subfloat[Importances of features in $T^{\mathbf{LS}_2}$.\label{fig:xor-variance-f}]{
	\begin{tikzpicture}[x=1cm,y=1cm]
\myxorvarianceimptwo
\end{tikzpicture}}
	
\caption{{Let us consider two highly similar datasets $\mathbf{LS}_1$ and $\mathbf{LS}_2$ made of a set of input features $V$ and a binary output (of two classes). Two features $X_1 \in V$ and $X_2 \in V$ (represented in Figures \ref{fig:xor-variance-a} and \ref{fig:xor-variance-b}) form a $XOR$ structure that determines the output, i.e. all points with ($X_1 \le 0.5$ and $X_2\le 0.5$), or ($X_1 > 0.5$ and $X_2 > 0.5$) belong to the first class, and to the second class otherwise. Both datasets are identical except one sample (surrounded by a green circle) that has been slightly moved in $\mathbf{LS}_2$. Figures \ref{fig:xor-variance-c} and \ref{fig:xor-variance-d} show trees built  on each learning set respectively. For sake of simplicity, let us assume that $X_1$ and $X_2$ are used on top of the tree and each split has a cut-point at $0.5$. In $\mathbf{LS}_1$, $X_1$ is slightly better than $X_2$ (masking $X_2$) and thus selected first, while in $\mathbf{LS}_2$, the situation is reversed ($X_1$ is now masked by $X_2$) and $X_2$ is selected first. The small change only is enough to completely change the (top of the) tree (i.e., the order in which $X_1$ and $X_2$ are used) and potentially all the rest of the tree, symbolised by shaded different sub-trees (see \citep[Figure 5.8]{breiman1984classification} for a complete example). Figures \ref{fig:xor-variance-e} and \ref{fig:xor-variance-f} show the importances of $X_1$ and $X_2$ computed as the (unweighted) sum of Shannon impurity decreases.}}
\label{fig:xor-variance}
\end{figure}

\section{Tree-based ensembles} \label{sec:RF}
Decision trees are simple and interpretable models but fail to compete with other machine learning algorithms in terms of accuracy. This lack of performances is mostly caused by their very high variance \citep{geurts2002contributions}.

This variability stems from the strong sensitivity of the decision tree algorithm to the variability of the learning dataset. Indeed, a small change in the learning set (e.g., due to sampling or noise) may cause significant differences between induced models such as the split choices, the branch depths or the distributions of samples in terminal nodes \citep{breiman1996heuristics,geurts2002contributions}.  Any modification has a strong impact on all following decisions because of the recursive nature of the algorithm, resulting in a greatly modified tree structure \citep{dietterich1995machine,schrynemackers2015supervised}.  In addition, the choice of splits or predictions in deep nodes are made with only few training samples and hence are expected to be of very high variance \citep{dietterich1995machine,geurts2002contributions}.
Ultimately, the high variance of a decision tree model penalises both its accuracy and its interpretability (at least to some extent).

As a way of increasing the performances, \textit{ensemble learning} is a technique that is particularly adapted for variance reduction in the context of decision tree models \citep{louppe2014understanding}. Based on the idea of \cite{kwok1990multiple}'s \textit{'Multiple decision trees'}, the principle of this approach consists in combining several different models to achieve better performances than individual ones by aggregating their predictions \citep{hastie2005elements}.
Base models of an ensemble are usually built independently of each other and their predictions are either averaged (for a regression task) or aggregated by majority vote (for a classification task). 
In the same vein, \textit{boosting methods} do not build independent individual predictors but rather build a sequence of models in which each step builds a predictor trying to refine the predictions of its predecessors.

In what follows, we focus on the first family of methods, usually referred to as \textit{averaging methods}, where models are built independently and usually differ from each other because of some randomisation introduced in one way or another. We generically denote these methods by ``Random forest type of method'' to distinguish the family from its particular well-known instance proposed by Leo Breiman and called ``Random forests''.

\subsection{Random forest type of methods} \label{sec:trees-methods}

\textit{Random forest type of methods} refers to several tree-based ensemble learning methods based on the idea of randomisation and aggregation. The main common principle is to generate an ensemble of randomised trees (i.e., a \textit{forest}) in which each individual tree is induced by a randomised version of the classical decision tree growing algorithm, and to combine in a suitable way the predictions of all the elements of this ensemble. Formally, a random forest consists of a collection of $N_T$ tree-structured models 
$\mathbf{T}=\{T_{i} | i = 1, ..., N_T\}$ used together in the way suggested by Figure \ref{fig:ensembleoftrees} in order to make predictions. 

	\begin{figure}[hbtp]
	\centering
	\begin{tikzpicture}[x=1cm,y=1cm]

	\myensemblerf
	
	\end{tikzpicture}
	
	\caption{Principle of the random forests method. The model $\mathbf{T}$ consists of an ensemble of $N_T$ (different) trees. The model prediction $\mathbf{T}(\mathbf{x})=\hat{y}$ is the aggregation of the predictions of every individual decision tree model.}
	\label{fig:ensembleoftrees}
	\end{figure}
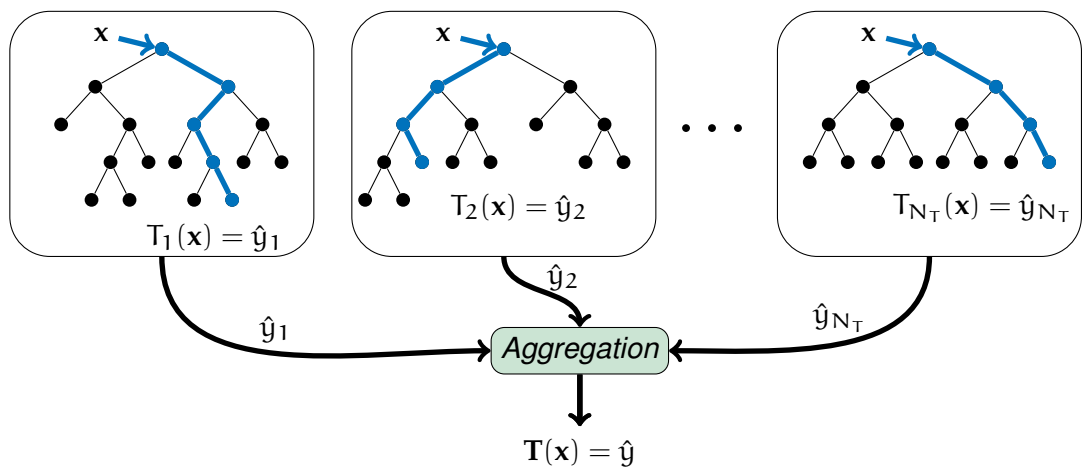

The goal of introducing randomisation is to generate diverse tree models, i.e., models whose errors are as much as possible uncorrelated. Indeed, for a given average behavior of the members of the ensemble, the more diverse they are, the smaller is the variance of the ensemble model and the higher is its accuracy (see side note on page \pageref{sn:diversityoftrees} and in particular \citep{hastie2005elements,louppe2014understanding,joly2017exploiting} for more details).

\begin{sidenote}{Number and diversity of trees in an ensemble} \label{sn:diversityoftrees}
	
	\cite{hastie2005elements} motivate the aggregation of several models by giving the variance of the average of :
	\begin{enumerate}
		\item $N_T$ independent and identically distributed (i.i.d.) random variables, each with a variance of $\sigma^2$, is
		\begin{eqnarray}
		\dfrac{1}{N_T} \sigma^2.
		\end{eqnarray}
		As the number of random variables $N_T$ increases, the variance tends to disappear.
		\item $N_T$ identically distributed (but not independent) (i.d.) random variables, each with a  variance of $\sigma^2$ and a positive pairwise correlation of $\rho$, is
		\begin{eqnarray}
		\rho \sigma^2 + \dfrac{1-\rho}{N_T} \sigma^2.
		\end{eqnarray}
		Similarly to the first case, the second term disappears with an increasing $N_T$. The first term however is independent of $N_T$ but decreases as the variables are de-correlated (i.e., lowering the value of $\rho$). 
	\end{enumerate}
	Both examples show that trees must as numerous and diverse (i.e., de-correlated) as possible to decrease the variance. It motivates the use of randomisation to generate trees for an ensemble. We refer to \cite{louppe2014understanding} for a detailed bias-variance decomposition of an ensemble of trees.
\end{sidenote}

In addition to a potential increase of performances, let us note that building a random forest is usually advantageous from a computational point of view. Indeed, the randomisation often cuts the complexity down as it removes heavy computations or reduces the dimensionality of the problem. In addition, the bulk of the  learning of a random forest can be parallelised by growing the individual trees independently and exploiting several computers to do so.

Several random forest type of methods have been proposed over the years. They all apply the \textit{'perturb and combine' paradigm} and essentially differ from each other only in the way the decision tree procedure is perturbed \citep{geurts2002contributions}. The random perturbation can be introduced in several parts of the algorithm (mainly where the variability is observed), namely at the level of:
\begin{enumerate}
	\item \textit{the learning set}: As discussed in the context of the high variance of decision trees, models are expected to vary if they are built on different learning sets \citep{breiman1996bagging};
	\item \textit{the split variable selection}, i.e., \textit{features that are considered at each tree node}: not considering all features at each node allows sometimes alternative (e.g. masked) features to be selected;
	\item \textit{the split value selection}: the cut-point for numerical features or the binary splitting function for categorical features is chosen at each node at random rather than being optimised in terms of impurity reduction for the learning subset of that node.
\end{enumerate}

Below we explain the involved randomization mechanism of the main random forest type of methods published in the literature \footnote{See e.g. \cite{louppe2014understanding} for a more exhaustive list of random forests methods.}.

\begin{description}
	\item[Bagging] -- \textit{tree-wise learning set randomization}\\
	\textit{Bagging}, standing for \textit{bootstrap aggregating}  \citep{breiman1996bagging}, consists in  growing each tree of the ensemble from a bootstrap replicate of the learning set. Given a learning set $\mathbf{LS}$ of $N$ samples, a \textit{bootstrap sample} $\mathbf{LS}^B$ is obtained by sampling $n$ samples from $\mathbf{LS}$ at random and with replacement \citep{efron1994introduction}. Let us note that some samples of $\mathbf{LS}$ may appear multiple times in $\mathbf{LS}^B$ or not at all. On average, around 37\% of original samples are not represented in the bootstrap sample \citep{louppe2014understanding}, this will be of interest in Section \ref{sec:oob}. Figure \ref{fig:bootstrap} sketches the principle of generating bootstrap copies of a learning set,  for an ensemble of 5 copies gotten from a learning set of ten samples. Figure \ref{fig:bagging} illustrates the Bagging approach.

	\begin{figure}[hbtp]
		\centering
		
		\begin{tikzpicture}[x=1cm,y=1cm,minimum width=0.85cm]
		\mybootstrapfigure
		\end{tikzpicture}
		\caption{Example of five bootstrap replicates of a learning set $\mathbf{LS}$ of $N=10$ samples. Each $x^i$ represents a sample $(\mathbf{x}^i,y^i)$ of the learning set ($y^i$ is omitted for sake of clarity). On the left, five bootstrap replicates $\mathbf{LS}^{B}_{1},\mathbf{LS}^{B}_{2}, \dots, \mathbf{LS}^{B}_{5}$ of $\mathbf{LS}$ are shown. On the right, sets $\mathbf{LS}^{oob}_{1},\mathbf{LS}^{oob}_{2}, \dots, \mathbf{LS}^{oob}_{5}$ of (out-of-bag) samples that are not used in the corresponding bootstrap samples are highlighted. Sizes of oob sample sets are not necessarily the same. (Figure inspired from \cite{raschka2016model}).}
		\label{fig:bootstrap}
	\end{figure}
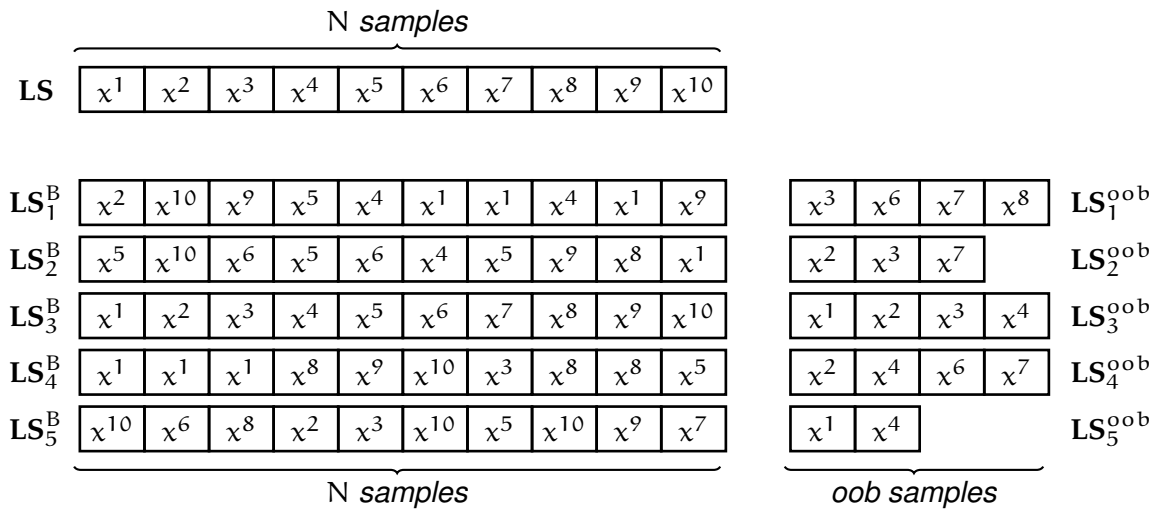
	
	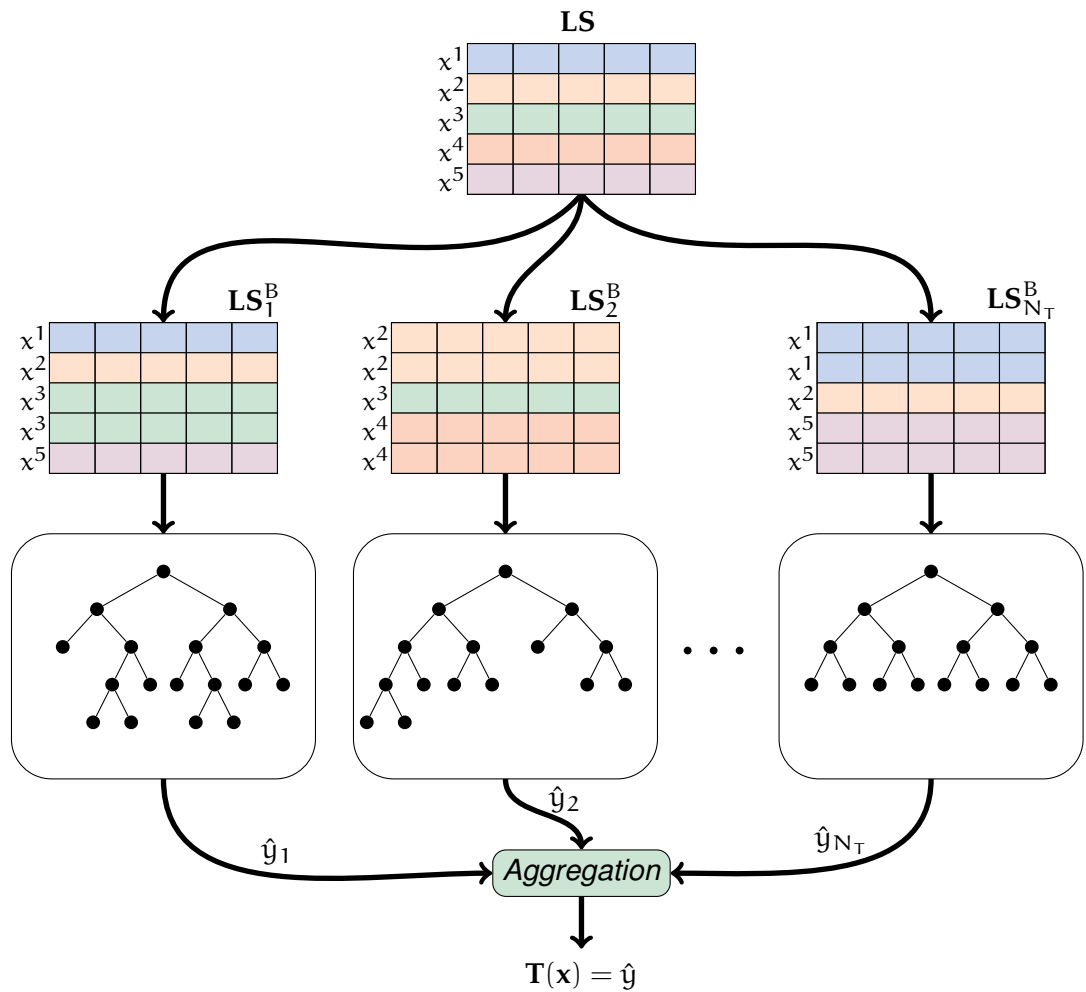
\begin{figure}[hbtp]

	\begin{tikzpicture}[x=1cm,y=1cm]
	
	\mybaggingfigure
	
	\end{tikzpicture}
	\caption{Bagging method. It consists of an ensemble of $N_T$ trees, each built on bootstrap replicates of $\mathbf{LS}$. Classically, the prediction $\hat{y}$ of the bagging model is the aggregation (majority vote or average) of every individual predictions $\hat{y}_i$.}
	\label{fig:bagging}
	\end{figure}
	
	\item[Randomized Trees] -- \textit{node-wise randomized split selection among best ones}\\
	 With this first randomised version of the decision tree algorithm itself, \cite{dietterich1995machine} extend the idea of \cite{kwok1990multiple} and propose to randomise the choice of the split for each node. For a given node $t$, instead of selecting the best split $s^*_t$, one of the $20$ best splits of node $t$ is selected uniformly at random.

	\item[Random Feature Subset] -- \textit{node-wise variable randomization}\\
	When the number of variables $p$ is large (e.g., in a handwritten character recognition application), the number of potential splits at each node is typically very large too. In order to avoid a search for the best split among too many possibilities, \cite{amit1997shape} propose to limit the search for the best split among a random subset of only $K$ variables chosen at each node. 
	
	\item[Random Subspace] -- \textit{tree-wise variable randomization}\\
	\cite{ho1998random} propose to  grow each tree of the ensemble on a \textit{random subspace}, i.e., a learning set in which only $K$ ($\le p$) features have been randomly chosen. Figure \ref{fig:RS} illustrates this approach.
	This method appears as similar to the ``Random feature subset'' approach, but here one particular tree of the ensemble faces the same subset of features at all its nodes.

\item[Random Patches] -- \textit{tree-wise variable and learning set randomization}\\ \cite{louppe2012ensembles} propose to build an ensemble of trees on \textit{random patches} where, before building a tree, both a subset of (say $K$) features and a subset of (say $L$) learning samples is selected at random. This allows to handle very big datasets and adapt to different types of problems by tuning $K$ and $L$ while keeping $K\times L$ compatible with memory capacity. 
	
	\begin{figure}[hbtp]

	\begin{tikzpicture}[x=1cm,y=1cm]

	\myRSfigure

	\end{tikzpicture}

	\caption{Building an ensemble of trees with the random subspace method. Given $p=5$ features,  each individual tree is learnt on an input subspace made of $K=3$ features that have been randomly sampled.}
	\label{fig:RS}
	\end{figure}
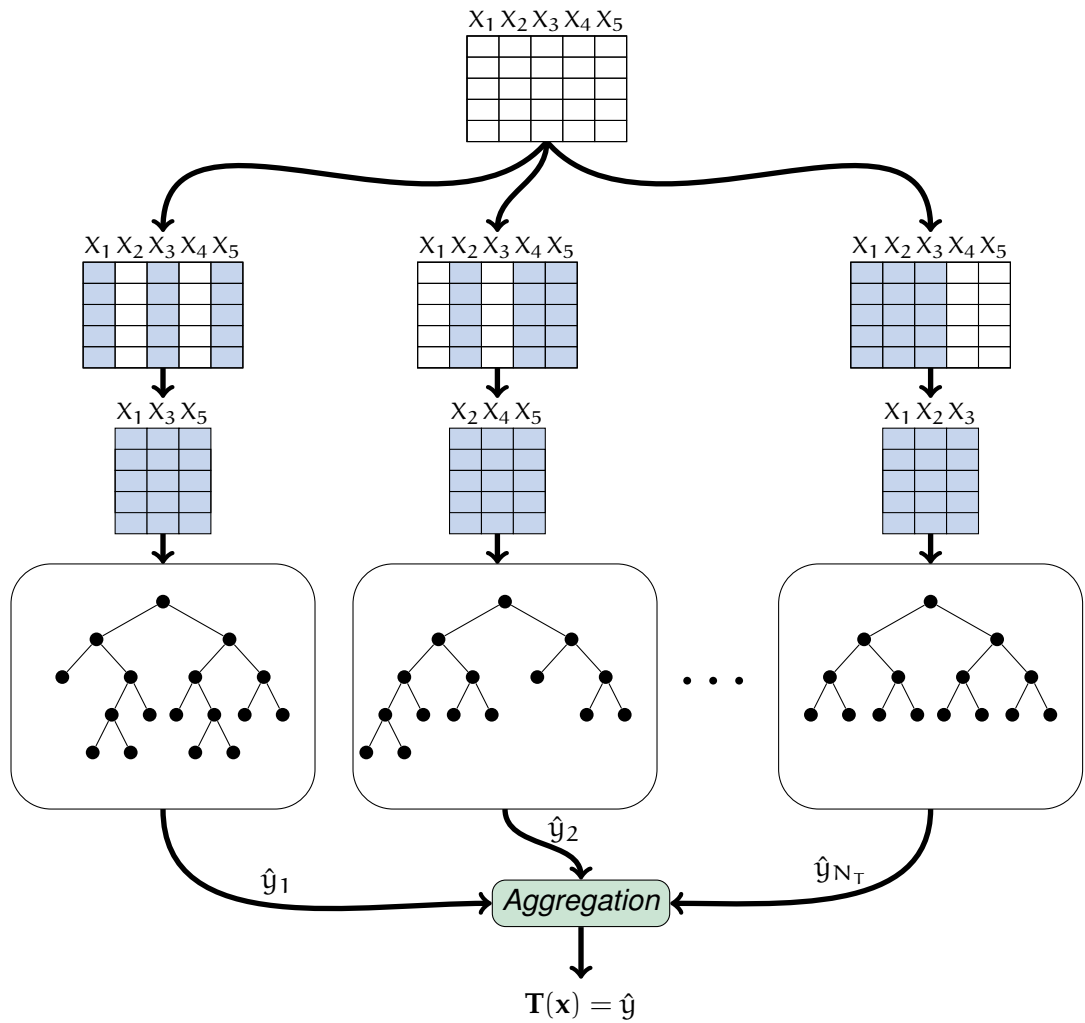

	\item[Random Forests] -- \textit{tree-wise learning set, node-wise variable randomization}\\ 
	With \textit{Random Forests} (RFs), \cite{breiman2001random} combines his idea of bagging with the random feature subset at each node of \cite{amit1997shape} in order to differentiate even more trees by perturbing them in two simultaneous ways. This is undoubtedly the most well known and used version of the random forests methods and more details are given in the following section.
	
	\item[Perfect Random Tree Ensembles] -- \textit{node-wise split randomization}\\
	The novelty of the \textit{Perfect Random Tree Ensembles} (PERT) proposed by \citep{cutler2001pert} is to combine a feature selection totally at random, similar to the random feature subset approach with only one feature considered at each node (i.e., $K=1$), and then a random split on that feature. Given an ordered split variable $X_m$ and a node $t$, two samples of different output values (classes) in $\mathbf{LS}_t$ are selected, say $(\mathbf{x}^i,y^i)$ and $(\mathbf{x}^j,y^j)$ with $y^i\ne y^j$, and the cut-point $\tau$ (the split value) is found as follows $\tau = \alpha x^i_m + (1-\alpha) x^j_m$ where $\alpha$ is drawn uniformly at random between $[0,1]$, $x^i_m$ and $x^j_m$ are respectively the values of variable $X_m$ for samples $\mathbf{x}^i$ and $\mathbf{x}^j$.  
	
	\item[Extra-Trees] -- \textit{node-wise candidate variable and split randomization}.\\
        The method of \textit{Extremely Randomized Trees} or \textit{Extra-Trees} (ETs) \cite{geurts2002contributions,geurts2006extremely} draws a random subset of $K$ variables at each node (as the ``Random feature subset method'') and for each one a single random split, and selects among these $K$ candidate splits the one yielding the largest impurity reduction to split a node. In this method, the cut-point selected for a numerical feature is drawn at each node according to a uniform distribution between the minimum and maximum values of that feature as observed in the local learning subset.
	
	\item[Totally randomized Trees] -- \textit{node-wise split randomization}\\
	The method of \textit{Totally Randomized Trees} (TRTs) is a variant of ``Extremely randomized trees'' maximising the randomization \cite{geurts2002contributions,geurts2006extremely}. Concretely, it consists in building ETs with $K=1$. Node splitting is thus carried independently of the output variable. The method of ``Totally randomized trees'' is especially of interest in theoretical analyses in the rest of this thesis, in particular in Chapters \ref{ch:importances} and \ref{ch:mdi}.
	
\end{description}

Without further explanation, let us also mention the \textit{Rotation Forests} method \citep{rodriguez2006rotation} which exploits feature extraction principle to build an ensemble of trees on different learning sets.

\subsection{Random Forests and Extra-Trees: parameters, properties, interpretability} \label{sec:rf-prop}

Among all methods, \cite{breiman2001random}'s Random Forests is certainly the most widely known. 
It was implemented from the very beginning in a freely available and well documented library \citep{breiman2002manual,breiman2003random}. Today, it is available within ``R'' and in the \textit{Scikit-learn} open-source platform (one of the most used machine learning libraries) which proposes a very efficient and simple to use implementation of both Random Forests and Extra-Trees \citep{pedregosa2011scikit}.
From a theoretical viewpoint, several authors studied the consistency (i.e., theoretical guarantees that the model converges towards optimality given asymptotic conditions, including a learning set of infinite size) of the method (see, e.g., \citep{zhao2000new,breiman2000some,breiman2004consistency,biau2008consistency,biau2012analysis,denil2014narrowing,scornet2015consistency}). In conclusion, all the results point in the direction that random forests methods work well in practice (see \cite{louppe2014understanding} for a review).

In this section, we first go through the different parameters of the Random Forest and Extra-Trees methods and then describe some of their properties that allow us to go beyond a simple predictor, and to some extent interpret the model.

\subsubsection{Parameters}

In this section, we discuss the common parameters of the Random Forest and the Extra-Tree methods. Specific parameters of other random forest type of methods are not mentioned here.

\begin{enumerate}
	\item \textit{Randomisation parameter} $K$: It concerns the number of features considered at each node as split variable candidates. Usually given as a function of the number of features, it directly impacts the degree of randomisation of the tree-based model. With $p$ features, typical default values for this parameter are $K=\sqrt{p}$, $K=\log_2{p}$ or $K=p$. Experimentally, it has been shown that $\sqrt{p}$ is usually an appropriate choice for classification tasks, while $K=p$ is often a better choice in case of regression \citep{hastie2005elements,geurts2006extremely}. The minimal value, $K=1$, implies a maximal randomisation. It may be of interest when all features are a priori known to be more or less  equally informative, while large values of $K$ are preferable when a large proportion of irrelevant variables is suspected.
	
	\item \textit{Number of trees} $N_T$: It defines the number of trees in the ensemble. Intuitively and theoretically, it seems that the number of trees should not be limited as it does not cause over-fitting \citep{hastie2005elements}, but performance stabilises after a certain number of trees depending on the problem considered. However, the number of trees should not be too small either as it has been shown that a certain number of trees is required to achieve the best prediction accuracy or to capture the whole problem structure \citep{latinne2001limiting,genuer2010variable,wehenkel2018characterization}. One usually needs to find a good trade-off for the number of trees to achieve good performance while not being too costly in terms of memory or computational resources.
	
	\item \textit{Individual tree complexity}: This parameter, unlike the first two, is not only defined by a single value. Several criteria, including of course a simple constraint on the maximal tree depth $d$, aim at limiting the complexity of the trees. As this corresponds to pre-prune the tree, we retrieve parameters that correspond to the stopping criteria that were discussed in Section \ref{sec:stoppingcriteria}. In addition to a maximal depth parameter $d$, $n_{min}$ and $n_{leaf}$ control the growing process of a branch and respectively define the minimal number of samples required to split a node and the minimal number of samples required in child nodes after the split. $\Delta i_{min}$ and $i_{min}$ respectively prevent the splitting of a node if the impurity reduction is not large enough or if the node has low impurity (i.e., pure enough). $N_{nodes}$ and $N_{leaf}$ control the overall complexity of the tree by defining a maximal number of nodes or leaves. 
\end{enumerate}

Let us mention that the choice of the impurity function (typically, Gini or Shannon) for classification tasks is usually left to the discretion of the user.

\subsubsection{Variable importances}

The decision tree model is interpretable. From this model, one can directly read the tree structure giving  features that have been used to build the model and how they are split, and the reasons behind a prediction. This was however limited by the high variance of the decision tree model. 

When taking an ensemble of trees, the resulting model is indeed more accurate in general but the multiplicity of trees it contains makes it difficult to read and synthesise the information provided by this model. Moreover, because of randomisation, every individual tree structure is also less relevant. 

In order to recover some interpretability, the random forest type of algorithms however offer, similarly to single decision trees, the possibility to derive a numerical ``importance'' value for each feature. This score aims at evaluating the contribution of a feature in the model. Reviewing, studying, and assessing such variable importances derived from tree-based ensemble models is the focus of Chapters \ref{ch:importances} and \ref{ch:mdi}. More specifically, Chapter \ref{ch:importances} revisits the main variable importance measures, while Chapter \ref{ch:mdi} is devoted to a detailed analysis of one of these measures in particular, namely the mean decrease of impurity, on which we have focused our research.

\subsubsection{Out-of-bag samples and estimates}\label{sec:oob}

In methods using bootstrapping such as Bagging or Random Forests, for each tree model, there are some samples that have not been used for construction. Given a bootstrap sample set $\mathbf{LS}^{B}_{i}$ used for tree $i$, left-out samples $\mathbf{LS}^{oob}_{i} = \mathbf{LS}\setminus \mathbf{LS}^{B}_{i}$ are said to be \textit{out-of-bag} (OOB) for tree $i$ (see Figure \ref{fig:bootstrap}). 
These OOB samples can be used to estimate important statistics of the ensemble of trees such as the generalisation error or variable importances (see Section \ref{sec:MDA} of Chapter \ref{ch:importances}).

For each training sample $(\mathbf{x}^j,y^j) \in \mathbf{LS}$, some trees are built on bootstrap samples that did not include sample $j$. Let us denote this subset of trees as $\mathbf{T}^{-j}=\{T^{-j}_i|i=1,\dots,N_T^{-j}\}$ where $N_T^{-j}$ is the number of such trees.
The \textit{out-of-bag error estimate} at $(\mathbf{x}^j,y^j)\in \mathbf{LS}$ consists in evaluating the prediction $\mathbf{T}^{-j}(\mathbf{x}_j)$ of the ensemble of trees $\mathbf{T}^{-j}$ for the input $\mathbf{x}^j$. Mathematically, the out-of-bag error estimate over all the learning set is computed as follows
\begin{eqnarray}
\widehat{Err}^{oob}  = \dfrac{1}{N} \sum_{(\mathbf{x}^j,y^j)\in \mathbf{LS}} L( \mathbf{T}^{-j}(\mathbf{x}^j) , y^j )
\end{eqnarray}
where $N$ is the number of samples in $\mathbf{LS}$. In classification, $L$ and $\mathbf{T}^{-j}(\mathbf{x}_j)$ are respectively the zero-one loss and the result of a majority vote between all individual predictions $\{T^{-j}_i(\mathbf{x}_j) |i=1,\dots, N_T^{-j}\}$. In regression, $L$ and $\mathbf{T}^{-j}(\mathbf{x}_j)$ are respectively the MSE loss and the average of all individual  prediction, i.e., $\frac{1}{N_{T}^{-j}} \sum_{i=1}^{N_T^{-j}} T^{-j}_i(\mathbf{x}_j)$.

The out-of-bag error estimate provides an accurate approximation of the generalisation error (compared to one resulting from a test set of the same size as the training set \citep{breiman1996out} and from a K-fold cross validation\footnote{K-fold cross validation consists in dividing the learning set into $K$ folds (subsets) of same size and then learning a model on K-1 folds in turn and testing it on the remaining fold.} \citep{wolpert1999efficient}). Let us note that the out-of-bag error estimate requires only one ensembles of $N_T$ trees while K-fold cross validation needs to learn $K$ ensemble of $N_T$ trees.

\subsubsection{Proximity measure}

As another by-product, the Random Forests algorithm offers a \textit{proximity measure} between samples from which a \textit{proximity matrix} can be derived from the tree-based model \citep{breiman2002manual,breiman2003random}. Given a set of $N$ samples, each element $(i,j)$ of the matrix $N\times N$ is the proximity value between samples $(\mathbf{x}^i,y^i)$ and $(\mathbf{x}^j,y^j)$ which corresponds to the fraction of trees in which both samples fall in the same terminal node. 
The intuition is that samples sharing regularly the same terminal node (and thus the same prediction) are close to each other from the point of view of the random forests model. This also provides a comparison of samples that may of high dimensionality and/or made of mixed variables. 

This proximity measure can be used to identify structures in the data or for unsupervised learning (see for more details and examples of proximity plot, e.g., \citep{breiman2002manual,liaw2002classification,breiman2003random,hastie2005elements,louppe2014understanding,scornet2016random}).\\

\begin{summary}
Tree based supervised learning methods have been proposed several decades ago, and studied and used extensively since. With respect to other supervised learning methods, these algorithms are highly scalable, provide interpretable information, but are often suboptimal from an accuracy point of view. In the last twenty years, a significant body of research has been carried out in the machine learning community in order to understand the theoretical features of these methods, and to find out how to improve them. This work has culminated with the idea of building ensembles of randomized trees, rather than one single fully optimized tree. Many different variants of this idea have been proposed over the years, the two most widely used ones being ``Random Forests'' and ``Extremely Randomized Trees''. These methods have shown to be very effective in terms of accuracy (among the best general purpose supservised learning algorithms). On the other hand, they lead to less easily interpretable models than the original single decision trees. 
\end{summary}

\part{Characterisation of importance measures}

\chapter{A survey of the literature about tree-based feature importance measures}
\label{ch:importances}
\begin{overview}
In  this chapter we review the literature  on tree-based feature importance measures. We start by an intuitive description of this notion and then focus on the two most popular feature importance measures, namely the Mean Decrease of Impurity (MDI) and the Mean Decrease of Accuracy (MDA). We  examine theoretical and empirical analyses carried out on those measures and discuss their limitations and biases. Finally, the last parts of this chapter focus on practical applications of those measures, and in particular how to distinguish relevant from irrelevant features based on their importance scores. 
\end{overview}

%

Tree-based ensemble methods are known to be powerful methods for modelling complex systems while providing accurate predictions \citep{auret2011empirical}. In many problems, including for example micro-array studies  \citep{archer2008empirical} or medical prognosis \citep{wehenkel2017tree}, a black-box that only provides predictions is however not enough, or even not the main goal. Such applications require indeed to understand how the model is built, to allow some interpretation of results and predictions so as to gain insights on the underlying problem structure \citep{archer2008empirical}. However, at first sight, tree based ensemble models are not directly interpretable as the number of trees and the introduction of perturbations in the growing process make their individual interpretation difficult and certainly unreliable \citep{auret2011empirical}. Indeed, two questions are raised among others: \begin{quote}
``Is a feature used at the top of only one tree necessarily important?'' 

``What about features that are only used in a few trees of the ensemble, are they necessarily useless?''\end{quote}

Anticipating this need for interpretability, the Random Forests algorithm (presented in Section \ref{sec:RF}) was proposed together with several built-in measures of feature importance \citep{breiman2001random,breiman2002manual,breiman2003random}. 
Identifying the constitutive elements of the forest model (and their relative importance) is a way to interpret it, and so to gain insight about the underlying problem. Indeed, the variable importance is often presented as a robust statistic to assess the feature contribution in the random forests model of the underlying data generating mechanism \citep{archer2008empirical}. Furthermore,   these importance measure give an aggregated information, contrasting with the local interpretation of each individual tree. 

Concretely, given an ensemble of trees, the principle of \textit{feature importance evaluation} is to derive a numerical score that reflects the ``(relative) contribution''  of the different candidate features in the learnt model. Based on those scores, one can now evaluate the usefulness of a feature and compare the contributions of two features, whatever the way they are used in the individual trees. A feature having a larger importance score than another one indicates that it is more useful in the learnt model than the other one \citep{archer2008empirical}. Conversely, a feature with a very low importance score is not really useful in the learnt model. In addition, ordering all features according to their importance scores provides a feature ranking \citep{guyon2006introduction} that may be exploited in different ways.

In this chapter we focus on the subclass of tree-based ensemble methods where all trees are drawn from the same distribution and independently of the others. This choice corresponds, for example, to Tree Bagging, Random Forests, and Totally or Extremely Randomised Trees; but it excludes, for example, Tree Boosting\footnote{Let us note that feature importance can also be derived from ensembles of boosted trees (see, e.g., \citep{auret2011empirical} for a study).} or non-tree-based supervised learning methods. Whenever suitable, we will indicate how the discussed methods could apply to other types of predictors. 

Section \ref{sec:imp-imp} gives an intuitive discussion of the contribution of a feature in a tree-based ensemble model. Section \ref{sec:mdivsmda} provides the definitions of the MDA and MDI measures, the two most used ones, while Sections \ref{sec:th-imp} and \ref{sec:imp-emp} summarise the main theoretical and empirical studies on these measures reported in the literature. Then, the last sections aim at reviewing the main use of those importance measures. In particular, Section \ref{sec:threshold} focuses on techniques to distinguish important features from non-important ones based on their importance scores. Section \ref{sec:imp-ext} describes several machine learning methods exploiting importance measures or extending them. Section \ref{sec:imp-new} is dedicated to other importance measures that have been proposed in the literature. Finally, Section \ref{sec:imp-app} aims at describing some practical applications using successfully tree-based feature importance measures.

{\textit Remark: in order to make this chapter self-consistent and as complete as possible, we have included in our review results that will be discussed in more details in subsequent chapters of this thesis (and published in \citep{sutera2016context,sutera2018random}).}

\section{Contribution of a feature to a tree-based  model} \label{sec:imp-imp}

In this  section, we discuss several possible indicators to evaluate the contribution of a feature in a tree-based predictor. We first look at the role of a feature inside a single decision tree built by the classical CART approach \citep{breiman1984classification} and then consider the case of randomised tree ensembles.

\paragraph{Position of feature splits in the tree}
Intuitively, the position in the tree structure of the splits using a given feature gives an indication on the importance of that feature: splits close to the root should be more important than those used deeper in the tree. Indeed, in ordre to produce simple trees, the tree growing procedure first considers the most useful splits (corresponding to largest decreases of node impurity) and then refines the model by using less useful ones. 

However, this intuitive principle can not be directly transposed to ensemble of randomised trees. In all generality, a feature is used in more than one tree. Instead of a single position, the same feature may be at several (and different) positions in the different trees and one would need to take all of these positions into account to determine which features are the most important ones. For example, a feature might be used deeper in a tree because it has some redundant information with other variables used higher in that tree. Such a feature could be seen as important despite its deep positions in some tree.
The randomised nature of the growing procedure (e.g., at the level of split variable selection\footnote{See Section \ref{sec:RF} for the other mechanisms.}) also disrupts the intuitive order in which features are used in the tree. A feature may be used in the top of a tree while being barely useful or relevant, e.g., if the split variable selection is randomised, this feature may be considered simultaneously with a lot of noisy irrelevant variables and be the best choice among them. 

\paragraph{Feature selection frequency}
When extended to an ensemble of randomised trees, the position in a tree does not longer reflect the importance of a feature. If we put the node position aside, the decision tree growing procedure still naturally performs a feature selection by selecting the best feature in each node except for the most randomised variant of random forests methods. Intuitively, irrelevant features are not supposed to be selected, or only a very limited number of times by chance, because there is no interest of using them anywhere in the tree.  Conversely, relevant features are statistically related to the output and therefore should be regularly used in the model \citep{konukoglu2014approximate}. A feature can therefore be seen as important if it is used frequently in many trees. From there, the most straightforward way - although naive - to measure the importance of a feature is to simply count the number of times a feature is used as split variable in all individual trees in the ensemble \citep{strobl2007bias, konukoglu2014approximate, lundberg2017consistent, lundberg2018consistent}. 

Although it is sometimes not done in the literature, we prefer to normalise the ``feature selection importance'' by the total number of test nodes of all the trees composing the ensemble, in the following fashion:

\begin{definition}
Let us consider an ensemble $\mathbf{T}=\{T_1,\dots,T_{N_T}\}$ of $N_T$ trees using a set of input features $V$ to predict an output variable $Y$. The \textbf{feature selection frequency importance measure} $Imp^{freq}$ of  $X_m \in V$ in $\mathbf{T}$ is the proportion of nodes of the tree ensemble in which $X_m$ has been used as split variable, i.e.,
\begin{eqnarray}
  Imp^{freq}(X_m) = \frac{\sum_{i=1}^{N_T} \sum_{t \in T_i} \mathbb{1}(v(s_t) = X_m)}{\sum_{i=1}^{N_T} \sum_{t\in T_i} 1}
\end{eqnarray}
where a node is denoted $t$ and associated to a split $s_t$ with a split variable $v(s_t)$.
\end{definition}

Despite its intuitive interest, this importance measure is biased towards features used deeply in trees. Indeed, being selected at the root node only counts for one, while the same feature can be used multiple times deeper in the trees. For example, a barely important feature always selected in each last node of a branch (and providing only marginal impurity reductions) would outscore a feature selected only once at each root node. Moreover, the actual contributions of two features with the same importance (i.e., used the same number of times in the forest model) can be completely different if one yields much larger decreases of impurity than the other. Indeed, some features can be seen many times despite their irrelevance (e.g., because of randomisation) while relevant features are missed because of some undesirable effects (e.g., a masking effect of another feature, see Section \ref{sec:biases} for other examples), impacting directly their importance.

To address those limitations, other criteria of feature importance taking into account the actual contribution of a feature in the learnt predictor should be considered.

\paragraph{Two ways for evaluating the actual contribution of a feature to a decision tree prediction}

As presented in Section \ref{sec:splittingrules}, the tree growing procedure aims at splitting nodes until all terminal nodes are pure. To that end, each split is optimised by selecting as split variable the feature yielding locally the largest decrease of impurity. The construction of a model is thus completely based on the notion of impurity decrease, and in the eyes of the learning algorithm, a variable is indeed important if it provides a large decrease of impurity. Based on that observation, it makes sense to integrate the amount of impurity decrease obtained thanks to all the splits using a particular feature, in order to evaluate its contribution to making predictions. This rationale leads to the Mean Decrease of Impurity (MDI) importance measure.

Beyond its specific mechanism, the purpose of supervised learning is to enable accurate predictions of the target variable. In this respect, the importance of a feature should be directly related to its contribution to the predictive accuracy of the learnt predictor, or in other ways how this accuracy is affected by not using the concerned feature. This rationale leads to the Mean Decrease of Accuracy (MDA) feature importance measure.

Notice that these two importance measures are not equivalent, since reducing impurity on a learning sample does not necessarily imply increasing accuracy out of the learning sample. 

Section \ref{sec:MDI} describes the first importance measure based on the contribution of a feature in the building mechanism while Section \ref{sec:MDA} presents the second importance measure that associates the contribution of a feature to the impact of its removal on the prediction accuracy.

\section{MDI and MDA feature importance measures}\label{sec:mdivsmda}

In this section, we present the two importance measures, each considering a different aspect of the contribution of features. Section \ref{sec:MDI} introduces the \textit{Mean Decrease of Impurity} (MDI) that assesses the importance of a feature based on its average contribution in the impurity reduction in the tree-ensemble growing procedure. Section \ref{sec:MDA} defines the \textit{Mean Decrease of Accuracy} (MDA) that evaluates the contribution  a feature in terms of its impact on predictive accuracy.
 Anticipating on the rest of this chapter, let us notice the parallel that can be made with the two feature selection problems (described in Section \ref{sec:fsproblems}). The minimal-optimal approach focuses on selecting features that provide the highest accuracy. The all-relevant approach aims at identifying all features that are relevant to the target variable.

\subsection{MDI importance measure}\label{sec:MDI}

Used as splitting criterion in decision tree growing \citep{breiman1984classification} and then in tree-based ensemble methods \citep{breiman2001random}, the computation of impurity and impurity reductions is at the heart of these supervised learning algorithms. Taking advantage of these computations of  impurity reductions, \cite{breiman2002manual} proposed to evaluate the importance of an input feature $X_m$ for predicting the output $Y$ by its \textit{Mean Decrease of Impurity} (MDI), also presented as the empirical improvement in the splitting criterion \citep{strobl2007bias,friedman2001greedy}\footnote{Let us note that the sum of all impurity decreases provided by a feature was already proposed by \cite{breiman1984classification} as an importance measure for that feature in a single decision tree.}. Concretely, it consists in summing all impurity decreases due to $X_m$, weighted by the size of the node (in terms of the relative number of observations reaching that node) and divided by the number of trees composing the ensemble model. For a forest made out of $N_T$ trees, the MDI importance measure is computed as follows:

\begin{definition} \label{imp:eqn:mdi-imp}
	The \textbf{Mean Decrease of Impurity importance} $Imp^{mdi}$ of a feature $X_m \in V$ about the output $Y$ is
	\begin{eqnarray}
	Imp^{mdi}(X_m) = \dfrac{1}{N_T}\sum_{T} \sum_{t\in T: v(s^{*}_t)=X_m} p(t) \Delta i (s^{*}_t,t) \label{eqn:MDI}
	\end{eqnarray}
	where $p(t)$ is the ratio $N_t/N$ between samples reaching node $t$ $(N_t)$ and the total number of samples $(N)$, and $v(s_t)$ is the split variable of $s_t$.
\end{definition}

This definition of the MDI importance can be applied with any impurity measure, including Gini impurity and Shanon entropy used for decision tree growing, and variance used for regression tree growing (see Section \ref{sec:splittingrules}). 

\paragraph{Discussion}

The underlying assumption of MDI is that all relevant features, i.e., related to the output and thus important, will show up to be useful to discriminate $Y$ at some point of the ensemble learning, and thus yield  a high enough decrease of impurity to lead to their selection as split variable, while, on the contrary, irrelevant features are expected to provide no (too small) impurity decrease in any context, and so will be selected only with very low probability as split variable when growing a tree. It may occur that some noisy features yield (e.g., at nodes with a small number of samples) are still selected,  but their (low) impurity decrease should be toned down by the weighting mechanism. Let us however note that the MDI importance can not be negative as a split never increases the impurity of a node, i.e., $\Delta i(s^{*}_t,t) \ge 0$.
\cite{konukoglu2014approximate} see the MDI importance as an extension of the selection frequency importance where the split count is weighted by the actual contribution of the feature, i.e., $\Delta i(s^{*}_t,t)$. The size of the node $p(t)$ is moreover taken into account to balance deep and shallow nodes. There are more deep nodes than shallow ones but usually with less samples. 

One of the main advantages of this measure is its computational efficiency. MDI computation is indeed a direct byproduct of the ensemble learning: all impurity decreases are already computed in order to build the tree ensemble \citep{breiman2003random}. 
However, it does not explicitly take into account the quality of the generated model, while being important according to MDI in a poor model does not imply much.

\subsection{MDA importance measure}\label{sec:MDA}

In tree ensemble learning methods using bootstrapping (Bagging, Random Forests), a tree of the ensemble does not use all samples for its construction. Using these out-of-bag samples, \cite{breiman2001random} proposed to evaluate the importance of an input feature $X_m$ by its \textit{Mean Decrease of Accuracy} (MDA) based the out-of-bag (OOB) error estimate. To this end, the contribution of a feature in a particular tree is evaluated by the impact of its removal on the OOB error-rate for that tree (which is expected to increase for an important feature). The removal of the feature is simulated by permuting in a random fashion its values in the OOB sample, and by evaluating the impact of this on the prediction accuracy of the tree estimated over its OOB sample.\footnote{Therefore, the MDA importance is also known in the literature as the \textit{permutation importance}.} The contribution of a feature for the whole forest is then obtained by averaging this measure over all trees.\footnote{Notice that the original definition of MDA importance derived from a Random Forest,  as introduced in \cite{breiman2001random}, is quite different from the current one adopted later on by several authors (e.g., \citep{friedman2009elements,genuer2010variable,biau2016random,gregorutti2017correlation}); in the original definition, the impact of removing a feature on the accuracy of the whole ensemble model was evaluated, instead of the now used average impact on the accuracy of the individual terms of the ensemble model. It is the more recent interpretation to which we refer in our work.}

To formalize this idea, let us first consider a given predictor $f(\cdot) \in {\cal Y}^{\cal X}$ and a given sample $\cal D$ of input-output pairs $(x,y)$ and some loss function $L$. Let us denote by $\tilde{{\cal D}}_{m}$ a modified sample obtained from $\cal D$ by permuting the values of the variable $X_{m}$ randomly (and thus independently of the values of $Y$ and all other input features), and define the MDA-estimate   of  $X_{m}$  (in $f$)  over $\cal D$ by 
\begin{equation} 
Imp_{f}^{{mda}}(X_{m}, f, {\cal D}, \tilde{{\cal D}}_{m}) = \frac{1}{|{\cal D}|}\left(\sum_{(x,y) \in \tilde{\cal D}_{m}} L(f(x),y) - \sum_{(x,y) \in {\cal D}}L(f(x),y) \right).\label{eqn:emp-mda-f}
\end{equation}
This quantity is an empirical estimate, based on the sample $\cal D$, of how much the ``removal'' of variable $X_{m}$ influences the accuracy of $f$ as a predictor of $y$. Its value depends on the particular permutation  $\tilde{\cal D}_{m}$ used. This dependence can be factored out by averaging over a uniform distribution of permutations, yielding
\begin{equation} 
Imp_{f}^{{mda}}(X_{m}, f, {\cal D}) = \mathbb{E}_{ \tilde{{\cal D}}_{m}}\{Imp_{f}^{{mda}}(X_{m}, f, {\cal D}, \tilde{{\cal D}}_{m})\}
\label{eqn:emp-mda-f-exp}.
\end{equation}

Now, consider a  learning set $\mathbf{LS}$ of input-output pairs and a tree growing algorithm $Algo$. Denote by $\mathbf{T}=\{T_1,\dots, T_{N_T}\}$ an ensemble of trees where each tree $T_i$ is grown by  $Algo$ on a bootstrap replicate $\mathbf{LS}_{i}$ of $\mathbf{LS}$, and evaluated on the corresponding OOB sample ($\mathbf{LS}^{oob}_{i} = \mathbf{LS} \setminus \mathbf{LS}_{i}$). The MDA importance of a feature $X_m$ derived from $Algo$ is defined as follows:
\begin{definition}\label{def:empiricalMDA}
	The \textbf{Mean Decrease of Accuracy Importance} $Imp_{Algo}^{mda}$ of a feature $X_m$ about the output $Y$ derived from a bagged version of $Algo$ applied on the learning sample $\mathbf{LS}$ is
	\begin{eqnarray}\label{eqn:emp-mda}
	Imp_{Algo}^{mda}(X_m, Algo, \mathbf{LS}) = \dfrac{1}{N_T} \sum_{i=1}^{N_T} Imp_{f}^{{mda}}(X_{m}, T_{i}, \mathbf{LS}^{oob}_{i}, \widetilde{\mathbf{LS}}^{oob}_{i,m}).
	\end{eqnarray}
\end{definition}

\paragraph{Discussion}

The underlying assumption of MDA is that all important features are related to the output $Y$, and thus contribute to the ability of the model to predict $Y$. The permutation of the values of a feature $X_m$ breaks the statistical link between $X_m$ and $Y$, and thus mimics predictions made without using feature $X_m$, which are expected to be worse if $X_m$ is an important feature. 

A high (and positive) importance value indicates that the variable is important and its removal strongly reduces the accuracy of the tree ensemble-based predictor. 

Contrary to MDI, MDA importances can take negative values \citep{genuer2010variable}.

\subsection{Discussion of MDI versus MDA}

Both methods can be used for classification and regression problems. MDA depends explicitly on the loss function used, whereas MDI depends explicitly on the impurity measure used. Both $Imp^{mdi}$ and $Imp^{mda}$ are random quantities depending on the random learning sample and on the tree ensemble randomisation; $Imp^{mda}$ further depends on the random permutations of the values of $X_{m}$.
While MDI is defined only for tree-based models, MDA can be used with any bagged supervised learning algorithm, and with slight modification in the loss-estimation method with any supervised learning algorithm.

\section{Theoretical analyses} \label{sec:th-imp}

Supported by the broad success of tree-based methods in applied research (see, e.g.,   \citep{svetnik2003random,diaz2006gene,cutler2007random,statnikov2008comprehensive,ghimire2010contextual,zaklouta2011traffic,nayak2016brain,belgiu2016random}), many authors studied tree-based variable importances to increase their understanding of the methods. Some theoretical analyses about the consistency of the Random Forests algorithm were already mentioned in Section \ref{sec:rf-prop}. But only a few works focused on tree-based variable importances from a theoretical point of view and this section aims at summarising these results and at providing the reader with a better understanding of their theoretical properties. 

Mechanisms for building a tree-based ensemble, and consequently to derive importance measures, are highly complex because of their randomisation and their data-dependent nature. For that reason, theoretical studies on MDI and MDA usually deal with that complexity by considering either a simplified version of the tree-based algorithm \citep{ishwaran2007variable}, an asymptotic setting \citep{louppe2013understanding,louppe2014understanding,sutera2018random}, or even a specific class of supervised learning problems \citep{gregorutti2017correlation}. 

In the present section, we first review the main known theoretical properties of the importance measures focusing on so-called asymptotic conditions, i.e., when the ensemble of trees and the training sample are both assumed to be of infinite sizes. We then discuss theoretical analyses studying the impact of feature correlation or redundancy on importance measures. Empirical analyses of these measures in real settings are discussed in the next section.

\subsubsection*{Notational conventions}

 In the present and subsequent sections, MDI and MDA importances derived in asymptotic conditions, i.e. their population versions, are respectively denoted  $Imp_{\infty}^{mdi}$ and $Imp^{mda}_{\infty}$. Additional parameters are specified as subscript or superscripts when they have an influence on the importance measure.

\label{sec:theory}

\subsection{Asymptotic properties of MDA \label{sec:th-mda}}

Following \cite{gregorutti2017correlation}, let us introduce the population version of the MDA importance measure (Equation \ref{eqn:emp-mda-f}) in the context of least-squares regression problems. Denote by $P(Y,X)$ the joint distribution of inputs and all outputs, and by $\tilde{P}_{m}(Y,X)$ the joint distribution obtained by replacing in $P$ the factor $P(X_{m}|Y, X^{-m})$ by the marginal distribution of $P(X_{m})$, i.e. by breaking any link between $X_{m}$ with the output and all other input features will leaving the marginal distribution of $X_{m}$ unchanged. Denote also by $f_{B}$ the Bayes model with respect to the original distribution $P$ and the square loss-function (i.e. $L(y,y') = (y - y')^{2}$): $$f_{B}(X) =  \mathbb{E}_{P}\left \{Y | X \right \},$$
where the subscript $P$ indicates the distribution used for computing the conditional expectation. 
Then the population version of MDA introduced by \cite{gregorutti2017correlation} is defined as follows 
\begin{eqnarray}\label{eqn:pop-mda}
Imp_\infty^{mda} (X_m) = \mathbb{E}_{\tilde{P}_{m}}\left \{ ( Y - f_{B}(X) )^2 \right \} - \mathbb{E}_{P}\left \{ ( Y - f_{B}(X)  )^2 \right \}.
\end{eqnarray}
Notice that this quantity is non-negative, since $f_{B}$ is the Bayes model with respect to the original distribution\footnote{{More formally, we can rewrite the first term of \ref{eqn:pop-mda} as
  $$\mathbb{E}_{\tilde{P}_{m}}\left \{ ( Y - f_{B}(X) )^2 \right \} = \mathbb{E}_{P}\left \{ \mathbb{E}_{\tilde{X}_m\sim P(X_m)}\left \{ (Y-f_{B,\tilde{X}_m}(X))^2\right\}\right\},$$
  where $f_{B,\tilde{X}_m}(X)$ returns the value of $f_{B}$ at $\tilde{X}^m$ obtained from $X$ by replacing $X_m$ by $\tilde{X}_m$ and leaving all other features unchanged. Inverting the two expectations, one gets:
  $$\mathbb{E}_{\tilde{X}_m\sim P(\tilde{X}_m)} \left \{ \mathbb{E}_{P} \left \{ (Y-f_{B,\tilde{X}_m}(X))^2\right\}\right\}.$$
  By definition of $f_B$, the inner expectation, and thus also the outer expectation, is greater or equal to $\mathbb{E}_{P}\left\{(Y-f_B(X))^2\right\}$, which proves that $Imp_\infty^{mda} (X_m)$ is non-negative.}}.

Obviously\footnote{The two terms in Equation \ref{eqn:emp-mda-f} are indeed unbiased and consistent sample estimates of the two population mean square errors in \ref{eqn:pop-mda}.},

both $Imp_{f}^{mda} (X_m, f_{B}, {\cal D}, \tilde{\cal D}_{m})$ (Equation \ref{eqn:emp-mda-f}) and $Imp_{f}^{mda} (X_m, f_{B}, {\cal D})$ (Equation \ref{eqn:emp-mda-f-exp}) are unbiased and consistent finite sample estimates of  $Imp_\infty^{mda} (X_m)$.

On the other hand, while Equation \ref{eqn:pop-mda} only depends on the joint distribution between $Y$ and $X$, the ``Bagging'' estimate  of Equation \ref{eqn:emp-mda} also depends on the base learner $Algo$ used.  The consistency of $Imp_{Algo}^{mda}$ with respect to $Imp_{\infty}^{mda}$ thus depends on the properties (and obviously the consistency) of the base learner. In particular, 
\citep{gregorutti2017correlation} note that this consistency was shown by \cite{zhu2015reinforcement} under several hypotheses, including the use of purely random forests \cite{biau2008consistency} and the independence between features\footnote{This assumption is quite strong and excludes works on correlated features for instance.}. 

\subsubsection*{Additive regression model.}

To handle the complexity of the theoretical analysis of the MDA importance measure, \citep{gregorutti2017correlation} consider the particular case of a joint distribution $P$ satisfying the following additive regression model
\begin{eqnarray}\label{eqn:add}
Y = \sum_{j=1}^p f_j(X_j) + \epsilon
\end{eqnarray} where $\epsilon$ is such that $\mathbb{E}\{\epsilon| X\}=0$ and $\mathbb{E}\{\epsilon^2 |X\}$ is finite (and where all functions $f_j$ are measurable) implying that $f_{B}(\mathbf{x}) = \sum_{j=1}^p f_j(x_j)$. 

In this setting, \cite{gregorutti2017correlation} show that the MDA importance of a variable $X_m$ is 
\begin{eqnarray} \label{eqn:pop-mda-add}
Imp^{mda}_\infty (X_m) = 2\, var\{f_m(X_m)\}.
\end{eqnarray}
Equation \ref{eqn:pop-mda-add} states that the MDA importance of a feature is (twice) the variance of the contribution $f_m(X_m)$ of $X_m$ in the additive Bayes model (Equation \ref{eqn:add}). 
{In the classification setting, \cite{gregorutti2017correlation} show that this result is not valid with zero-one loss in the case of an additive logistic regression model, as they note that $Imp^{mda}_\infty(X_{m}) > 0$ only if the contribution of $X_{m}$ to $P(Y|X)$ is large enough to change the predicted class.}

{\cite{zhu2015reinforcement} use a slightly different notion of population importance, which is a normalised version of $$\mathbb{E}\{(f_{B}(X)-f_{B}(\tilde{X}^{m}))^{2}\}$$ where $\tilde{X}^{m}$ denotes the vector of inputs where the $m$th coordinate was replaced by an independent copy of $X_{m}$ and the expectation is taken with respect to the joint distribution of $Y$, the original inputs $X$, and the independent copy of $X_{m}$. Under the above additive model, this definition actually coincides with the former notion introduced above, as shown by \citep{gregorutti2017correlation}.}

\subsubsection*{Simplified permutation scheme.}

Instead of considering a specific model and still circumventing the complexity of the permutation scheme, \cite{ishwaran2007variable} study a variant of MDA importance sharing similar key properties but implementing another permutation scheme. Instead of permuting the values of a feature $X_m$ in oob samples, \cite{ishwaran2007variable} propose to "noise up" the feature $X_m$ by ignoring all nodes coming after one splitting on $X_m$. In practice, it comes to a random left-right assignment of samples in all ignored nodes. The beginning of the tree however remains unchanged. For this setting and assuming that the model can provide a good approximation\footnote{In details, in asymptotic conditions, the tree-based model must be able to provide a good approximation of the true inputs-output function which implies the consistency of the model and the piecewise constance of the regression function \citep{ishwaran2007variable}.}, the asymptotic behaviour of this variant can be derived. 

In particular, \citep{ishwaran2007variable} focus on the \textit{position bias} and show that variables split close to the root node tend to have a stronger effect on the predictive accuracy than other variables. It seems reasonable that the model performances are highly impacted as most of the tree is ignored when evaluating the importance of a feature close to the root. A similar behaviour is expected in the classical MDA importance. Indeed, the relation between features used at the top of the tree structure and their expected usefulness is obvious.  

Nevertheless, some irrelevant features may appear as important in this variant because of the feature noising. Since all nodes are ignored after one splitting on the evaluated feature $X_m$, the observed decreases in predictive accuracy is not only due to $X_m$ but also to all features used in deeper nodes. Therefore, the importance of $X_m$ reflects both the actual contribution of $X_m$ and the contribution of all split variables of ignored nodes. The importance of $X_m$ can thus be strictly positive even if $X_m$ is irrelevant. In response to that, \cite{ishwaran2007variable} suggest that non-informative features are more likely used down in trees and thus spurious importance scores should be limited.
He also claims that noising up only the right node (i.e., the one using $X_m$ to split)  is too difficult to be theoretically analysed without additional assumptions.

\subsection{Asymptotic properties of MDI} \label{sec:th-mdi}

\subsubsection*{Regression tree-based models}

{According to \cite{friedman2001greedy}, the MDI importance measure is an approximated measure of the relative influence of variables. In the context of regression problems, let us consider a given predictor $f(\cdot) \in \mathcal{Y}^{\mathcal{X}}$. Following \cite{friedman2001greedy}, the relative importance of an input variable $X_j$ in the predictor $f$ is its relative influence on the variation of $f$ over the joint input variable distribution and computed as follows
	\begin{eqnarray}\label{def:friedman-influence}
	Imp_{f}^{infl}(X_j,f) = \sqrt{ \mathbb{E}_{X} \left \{ \left( \dfrac{\partial f(X)}{\partial X_j} \right)^2\right \} .\;  var_{X} \{X_j\} }.
	\end{eqnarray}  
\cite{friedman2001greedy} note that Equation \ref{def:friedman-influence} does not strictly exist for piecewise constant functions such as produced by regression tree-based models. \cite{friedman2001greedy} therefore suggest that the MDI importance measure\footnote{Actually, the MDI importance computed as the sum of empirical improvement in squared error over all nodes splitting on $X_j$ in a given tree and its average over all trees.} of $X_j$ was proposed as a surrogate measure to approximate Equation \ref{def:friedman-influence} for piecewise constant functions and shown to be consistent with expected feature influences in the case of linear relationships between inputs and output variable \citep{friedman2001greedy}.}\\

Beyond this intuitive motivation, we now turn to classification problems, and analyse the main properties of the  MDI importance measure when it is based on the Shannon entropy as an impurity measure.

\subsubsection*{Totally randomized decision-tree based ensembles with categorical input features and multiway exhaustive splits.}

Following \cite{louppe2013understanding,louppe2014understanding}, let us consider a set $V=\{X_1,\dots, X_p\}$ of categorical input features and a categorical output $Y$. For the sake of simplicity, only the
Shannon impurity is considered below but most results can be go generalised to other impurity measures \citep{louppe2013understanding, louppe2014understanding}. Let us also consider totally randomized trees (defined in Section \ref{sec:RF}) with multiway exhaustive splits (see Section \ref{sec:DT}). In case of categorical variables, each node $t$ is split into $|X_i|$ sub-trees, i.e., one for each possible value of $X_i$.  It implies that features can only be used once and thus limits the depth of a branch to $p$. 

In this setting, the MDI importance of feature $X_m \in V$ for $Y$ computed in asymptotic conditions\footnote{Infinite learning sample size, infinite ensemble of fully developed (ie., unpruned) totally randomised trees.} is given by \citep{louppe2013understanding}: 
\begin{eqnarray}\label{eqn:mdiXm}
Imp^{mdi}_{\infty}(X_m) = \sum_{k=0}^{p-1} \dfrac{1}{C_p^k} \dfrac{1}{p-k} \sum_{B\in \mathcal{P}_k(V^{-m})}I(X_m;Y|B)
\end{eqnarray}
where $V^{-m}$ denotes the subset of features $V \setminus \{X_m\}$, $\mathcal{P}_k(V^{-m})$ is the set of subsets of $V^{-m}$ of cardinality $k$, and $I(X_m;Y|B)$ is the conditional mutual information of $X_m$ and $Y$ given the variables in the conditioning set $B$. Additionally, \cite{louppe2013understanding} show that 
\begin{eqnarray}\label{eqn:mdiall}
\sum_{m=1}^p Imp^{mdi}_{\infty}(X_m) = I(X_1,\dots,X_p;Y)
\end{eqnarray}
where  $I(X_1,\dots,X_p;Y)$ is the joint mutual information between all features in $V$ and the output $Y$.

Equation \ref{eqn:mdiXm} shows that each importance can be divided along the interaction degree $k$, i.e., the number of features in the conditioning set $B$, and along the combinations of $B$ of fixed size of $k$ features.

Equation \ref{eqn:mdiall} states that all the information $I(X_1,\dots,X_p;Y)$ contained in the set of input variables $V$ about the output $Y$ can be decomposed between the importance of all features.  The equality of Equation \ref{eqn:mdiall} induces that the sum of all importances equals a fixed value (of the joint mutual information). It implies that the increase or decrease of one feature importance is made to the detriment of other importances. 

Let us mention that any (conditional) mutual information term involving $Y$ (of the form $I(X;Y|B)$ or $I(X_1,\cdots,X_q;Y|B)$ with $B$ potentially empty) is upper bounded by $H(Y)$. It gives in particular that $\sum_{m=1}^p Imp^{mdi}_{\infty}(X_m) \le H(Y)$ where the equality indicates that $Y$ is perfectly explained by $V$ (i.e., $I(X_1,\dots,X_p;Y)=H(Y)$).

\cite{louppe2013understanding} show also that the form of these expressions remains valid for any impurity measure leading to non negative impurity decreases, including obviously all classical impurity measures such as Shannon-, Gini-, and variance-based ones.

\subsubsection*{Non-totally randomized trees with multiway exhaustive splits and categorical input features.}

Beyond its asymptotic behaviour, \citep{louppe2013understanding, louppe2014understanding} establish a relationship between relevance and MDI importance. This relationship follows from the definition of relevance in terms of mutual information (see Definitions \ref{def:relevanceMI} and \ref{def:strongweakrelevanceMI} in Section \ref{sec:relevance}).

In what follows, results can be extended to MDI importances derived from non-totally randomised trees (i.e., with $K>1$). Thus, let us denote the MDI importance computed with totally or non-totally randomized trees depending on the value of $K$ as $Imp^{mdi,1}_{\infty}$ and $Imp^{mdi,K}_{\infty}$ respectively.

In this context, a feature $X$ which is irrelevant for $Y$ with respect to $V$ always verifies $Imp^{mdi,K}_{\infty}(X) = 0$ \citep{louppe2013understanding,sutera2018random}. In case of totally randomised trees ($K=1$), a null score is only associated to an irrelevant feature and consequently all relevant features (strongly and weakly) have strictly positive MDI importance scores. Additionally, this result implies that irrelevant features do not impact importance scores of other features. Consequently, the relevant feature MDI importances are thus independent of the number of irrelevant features.

On the contrary, with non-totally randomised trees ($K>1$), some relevant features can also have a zero importance score due to the effect of $K$ on the tree construction. \cite{sutera2018random} show that only strongly relevant features are guaranteed to have strictly positive MDI importance score as they convey information about the output that no other variable (or combination of variables) in $V$ conveys
Depending on the value of $K$, some weakly relevant features may have a zero importance score. The randomisation parameter $K$ (when $>1$) thus affects the number and nature of relevant variables that can be found. 
 
 In the same conditions, \citep{louppe2013understanding} also show that the MDI importance derived from pruned trees (i.e., built up to a depth $q<p$) is equivalent to the ones obtained from unpruned trees built on random subspaces of $q$ variables randomly drawn from $V$.

\subsection{Correlated and redundant features}

By definition, totally redundant features share exactly the same information about the target variable $Y$, while correlated features often share information without necessarily being totally redundant with respect to $Y$. Tree-based or model-based importance measures described so far evaluate the contribution of a feature in the tree-based predictor or in the Bayes model. In the presence of redundant or correlated features, the sum of all contributions can no longer be shared unequivocally between all features. For example, the same "piece" of contribution might be attributed to several totally redundant features as they are interchangeable in the eyes of the model. The rest of this section describes works focusing on that aspect of importance measures. 

\subsubsection{MDA} \label{sec:th-mda-correlation}

\subsubsection*{Additive regression model with centred $f_j(X_j)$ functions.}

\cite{gregorutti2017correlation} continue their theoretical study of the additive model, by analysing the population version of the MDA importance in terms of feature correlations, assuming in addition that all $f_j(X_j)$ functions have zero mean. Under these conditions, Equation \eqref{eqn:add} becomes\footnote{See \citep[Proposition 2]{gregorutti2017correlation} for a proof; the zero-mean assumption is not essential but simplifies the reading of the expression.}
\begin{eqnarray}
Imp^{mda}_\infty (X_m) = 2 cov\{Y,f_m(X_m)\} - 2 \sum_{k\ne m} cov\{f_m(X_m), f_k(X_k)\}
\end{eqnarray}
where $cov$ denotes the covariance function. In this alternative formulation, interactions between input features are explicitly shown in the second term. 

\subsubsection*{Additive regression model and a normal distribution.}

\cite{gregorutti2017correlation} further consider the case of normal joint distribution $P_{V,Y}\sim \mathcal{N}_{p+1}\left ( 0 , Z \right )$ with a group $C$ of $c$ features $\{X_1,\dots,X_c\}$ equally correlated with each other and with the output. In order to highlight relationships between block of features, the covariance matrix $Z$ can be expressed as follows
 \begin{eqnarray}\label{eqn:setting-corr}
 Z = \begin{pmatrix}
 Z_V & \mbox{\boldmath$\tau$}^T\\
 \mbox{\boldmath$\tau$} & \sigma_y^2
 \end{pmatrix}  = \begin{pmatrix}
\mbox{\boldmath$\rho$} & 0 & \mbox{\boldmath$\tau$}_{\in C}^T\\
0 & \mathbb{1} & \mbox{\boldmath$\tau$}_{\not\in C}^T\\
\mbox{\boldmath$\tau$}_{\in C} & \mbox{\boldmath$\tau$}_{\not\in C} & \sigma_y^2
\end{pmatrix} 
\end{eqnarray}
where \begin{enumerate}[\indent $\color{Gray!80} \bullet$]
	\item $Z_V$ is the covariance between input features;
	\item $\mbox{\boldmath$\rho$}$ is the covariance sub-matrix $c\times c$ of features in the correlated group such that $cov(X_i,X_i) = 1$ and $cov(X_i,X_j)=\rho$ for all $1\le i,j\le c$, i.e. $\mbox{\boldmath$\rho$}=(1-\rho)I_c + \rho \mathbb{1}\mathbb{1}^T$;
	\item $\mbox{\boldmath$\tau$}_{\in C}$ is a (line)vector of $c$ elements $\mbox{\boldmath$\tau$}_{\in C} =\{\tau_C, \dots, \tau_C\}$, i.e. $cov\{X_m,Y\} = \tau_C$ with $0<m\le c$;
	\item $\mbox{\boldmath$\tau$}_{\not\in C}$ is a (line)vector of $(p-c)$ elements $\mbox{\boldmath$\tau$}_{\not\in C} =\{\tau_{c+1}, \dots, \tau_p\}$, i.e. $cov\{X_j,Y\} = \tau_j$ with $c<j<p$; 
	\item and $\sigma_y^2$ is the variance of $Y$.
\end{enumerate}
 In this setting, \cite{gregorutti2017correlation} specify the MDA importances as follows:
\begin{eqnarray}\label{eqn:mda_normal}
Imp^{mda}_\infty(X_m) = 2 \alpha_m^2 var\{X_m\} = 2 \alpha_m cov\{X_m,Y\} - 2\alpha_m \sum_{k\ne m} \alpha_k cov\{X_m,X_k\}
\end{eqnarray}
where $\alpha$'s are deterministic coefficient\footnote{See \citep[Proposition 3]{gregorutti2017correlation} for a proof.} equal to $\alpha_m = [Z_V^{-1}  \mbox{\boldmath$\tau$}]_m$.\\

\noindent
For a feature $X_j \not\in C$, Equation \ref{eqn:mda_normal} becomes
	\begin{eqnarray}\label{eqn:mda_noncorr}
	Imp^{mda}_{\infty}(X_j) = 2 \tau_j ^2
	\end{eqnarray}
	where $\tau_j$ corresponds to $cov\{X_j,Y\}$.\\

\noindent
For a feature $X_i \in C$, Equation \ref{eqn:mda_normal} becomes
	\begin{eqnarray} \label{eqn:MDAcorroutC} 
	Imp^{mda}_{\infty}(X_i) = 2 \left (\dfrac{\tau_C}{1-\rho+c\rho} \right )^2
	\end{eqnarray} 
	where $\tau_C = cov\{X_i,Y\}$ and $\rho = cov\{X_i,X_k\}$ with $0 \le k \le c,k\ne i$.
	In the particular case of two copies of the same feature, i.e. $c=2$ and $\rho=1$, it is
	\begin{eqnarray}\label{eqn:MDAcorrinC}
	Imp^{mda}_{\infty}(X_i) = 2 \left ( \dfrac{\tau_C}{2} \right )^2 = \dfrac{\tau_C^2}{2}.
	\end{eqnarray}

Equation \ref{eqn:mda_noncorr} states that the importance of a non-correlated feature is not impacted by potential correlation between other features. Equation \ref{eqn:MDAcorroutC} shows that the importance of a feature correlated with others is influenced by $\rho$ and $c$.
A large number of correlated features $c$ or a strong correlation, i.e. $c$ close to $1$, decrease the MDA importance of each individual feature. Combining Equations \ref{eqn:mda_noncorr} and \ref{eqn:MDAcorroutC} suggests that $X_j$ may appear more important, i.e. corresponds to a higher MDA importance, than $X_i$ even if $\tau_j < \tau_C$ if $\rho$ is large enough. Conversely, anti-correlation $\rho<0$ tends to increase the MDA importance. 

\subsubsection{MDI} \label{sec:th-mdi-redundant}

\subsubsection*{Totally randomized trees with multiway exhaustive splits and categorical input features.}

Let $X_j \in V$ be a relevant variable with respect to $Y$ and $V$ and let $X'_j \not\in V$ be a new variable such that $X_j$ and $X'_j$ are totally redundant with respect to $Y$ (see Definition \ref{def:totalred:eqn1}).  \cite{louppe2014understanding} extends the analytical formulation of the MDI importances of $X_j$ and any non-redundant variable $X_l \in V^{-j}$ in order to show the impact of the addition of $X'_j$. For sake of clarity, only one pair of totally redundant featuresis considered but see \citep{louppe2014understanding} for a generalisation to $c$ such features.\\

\noindent
The asymptotic importance of variable $X_j$ as computed from an ensemble built on $V\cup\{X'_j\}$ is\footnote{See \citep[Proposition 7.2]{louppe2014understanding} for a proof.}:
\begin{eqnarray}\label{eqn:MDIredin}
Imp_{\infty}^{mdi,1} (X_j) = \sum_{k=0}^{p-1}\dfrac{p-k}{p+1} \dfrac{1}{C^k_p} \dfrac{1}{p-k} \sum_{B\in \mathcal{P}_k(V^{-j})} I(X_j;Y|B)
\end{eqnarray}
For any other variable $X_l$ from $V^{-j}$, the importance becomes\footnote{See \citep[Proposition 7.4]{louppe2014understanding} for a proof.} 
\begin{eqnarray}\label{eqn:MDIredout}
\begin{split}
Imp^{mdi,1}_{\infty}(X_l) = \sum_{k=0}^{p-2} \dfrac{p-k}{p+1} \dfrac{1}{C^k_p} \dfrac{1}{p-k} \sum_{B \in \mathcal{P}_k(V^{-l}\setminus \{X_j\})} I(X_l;Y|B) \\+ \sum_{k=0}^{p-2} \left [ \sum_{k'=1}^{2} \dfrac{C^{k'}_{2}}{C^{k+k'}_{p+1}} \dfrac{1}{p+1-(k+k')} \right ] \sum_{B\in \mathcal{P}_k(V^{-l}\setminus \{ X_j\})} I(X_l;Y|B\cup \{X_j\})
\end{split}
\end{eqnarray}

A comparison of Equations \ref{eqn:MDIredin} and \ref{eqn:mdiXm} shows that the introduction of a variable $X'_j$ totally redundant with $X_j$ decreases the importance of $X_j$. Indeed, with respect to \ref{eqn:mdiXm}, all terms of the sum in \ref{eqn:MDIredin} are multiplied by a factor $\dfrac{p-k}{p+1} <1$. Intuitively, this is a consequence of the fact that both $X_j$ and $X'_j$ convey the exact same information about the output and they now both compete to explain the output, as the sum of all importances is not affected by the introduction of $X'_j$. Indeed, $X'_j$ does not bring any new information about the output with respect to $X_j$ (by definition) and therefore the right side of Equation \ref{eqn:mdiall} is unchanged. Although we obviously have $Imp^{mdi,1}_{\infty}(X_j) = Imp^{mdi,1}_{\infty}(X'_j)$ by symmetry, notice that the importance of $X_j$ is not simply divided by a factor 2 since the importances of the other variables are also affected by the introduction of $X'_j$, as shown in Equation \ref{eqn:MDIredout}.

Equation \ref{eqn:MDIredout} shows that the impact of the introduction of $X'_j$ on the importances of the variables in $V^{-j}$ is the combination of two effects. The first sum in \ref{eqn:MDIredout} is over all $B$ composed of variables from $V^{-j}$. With respect to the corresponding terms in \ref{eqn:mdiXm}, each term is multiplied by a factor  $\dfrac{p-k}{p+1}$ strictly lower than 1. The second sum in \ref{eqn:MDIredout} is over all conditionings including $X_j$ and the weights of the corresponding terms are now increased with respect to similar terms in \ref{eqn:mdiXm}. Whether or not the importance of $X_l$ will increase will thus depend on the way $X_l$ interacts with $X_j$. If the mutual informations $I(X_l;Y|B,X_j)$ are large ($X_l$ and $X_j$ are complementary), then adding $X_j$ will reinforce these terms and the net effect could be an increase of the importance of $X_l$. On the other hand, if these mutual informations are small ($X_l$ and $X_j$ are redundant), the net effect could be a decrease of the importance of $X_l$.

\section{Empirical analyses} \label{sec:imp-emp}
\newcommand{\com}[1]{\textcolor{blue}{#1}\\}

In the previous section, we studied theoretically both importance measures in asymptotic conditions. Although those results are helpful to better understand the mechanisms of MDA and MDI importance measures, they do not provide insights on how they actually behave in practice. In the light of their expected behaviours, the goal of this section is to analyse those two measures in a more realistic setting, i.e. with finite sample size and number of trees. To do so, we review many empirical analyses of their practical behaviours in numerous settings. In particular, we aim at highlighting the main biases and practical limitations of MDI and MDA importance measures in several view angles.

\subsection{Soundness}

Variable importance measures derived from tree-based ensemble methods have been suggested for the identification and selection of relevant features in numerous applications, e.g. gene selection in micro-array data \citep{huang2005comparative,diaz2006gene,pang2006pathway,rodenburg2008framework}, SNPs in large-scale/genome-wide association study data (GWAS) \citep{lunetta2004screening,bureau2005identifying,botta2014exploiting}, proteins \citep{qi2006evaluation}, or, more recently, brain regions involved in neuronal disease in neuroimaging data \citep{wehenkel2018characterization}. Along with this wide practical use, some works have tried to assess the quality of this identification. 

\cite{archer2008empirical,gromping2009variable} show that MDI and MDA feature importance measures manage to identify true predictors in different settings, and results are usually in agreement with other machine learning methods.

In presence of feature interactions, it was also noted that these measures provide interesting alternatives to classical statistical tests because they do not require explicit modelling or assumptions on the problem (e.g., gaussianity, (non-)linearity, or independence) and naturally handle feature interactions \citep{gromping2009variable,geurts2009supervised}. Differences between univariate approaches and tree-based importance scores may additionally be indicative of multivariate interactions \citep{rodenburg2008framework,auret2011empirical}. For example, \cite{lunetta2004screening} show that selections of relevant genetic markers (SNPs) provided by random forest feature importance measures outperform those obtained from a standard univariate screening method (i.e., Fisher Exact test), especially in presence of many interacting features.

In presence of correlated features, \cite{archer2008empirical} showed, in a setting similar to \cite{gregorutti2017correlation}'s (i.e., one group of correlated and equally predictive features, see Section \ref{sec:th-mda-correlation}), that both Gini MDI and MDA importance measures manage to identify most predictive features in many settings. They however noted that in case of strong correlation ($\rho$ close to $1$), the highest importance score may be associated to one feature correlated with the most predictive one. When there were more than one group of predictive correlated features or uncorrelated predictive features, some experiments show that both importance measures are sensitive to correlation structures and this may sometimes impact the reliability and stability of importance scores \citep{strobl2008conditional,nicodemus2009predictor,tolocsi2011classification,auret2011empirical}. Depending on tree parameters and correlation structures, empirical observations seems to diverge. Therefore, a more detailed analysis of those experimental results will be the focus of Section \ref{sec:biases-correlation}. 

From another point of view, \cite{lundberg2017consistent,lundberg2018consistent} claim that MDI importance measure is not "consistent" in the case of a (non-randomised) single tree. In the chosen example of two equally relevant features, increasing the predictive contribution of one does not necessarily correspond to an increase of its MDI importance. Conversely, the MDA importance measure appears to be "consistent" in this example. 

\subsection{Split randomisation parameter $K$}\label{sec:empirical-K}

In random forest methods, $K$ is the number of features considered at each node as split variable candidates. A low value of $K$ (e.g., $K=1$) maximises randomisation as one feature is selected totally at random without optimising the node impurity reduction. Consequently, all features can be selected and all relevant features may be identified. In contrast, high values (e.g., $K=p$) induce more optimised trees and only strongly relevant features are guaranteed to be identifiable. 

The interaction between $K$ and feature importance measures is not clear. For several authors \citep{auret2011empirical, strobl2008conditional,nicodemus2010behaviour}, importance measures are more accurate when derived from ensemble of trees built with large $K$ values. In these studies, experiments are carried out on simulated data where the output is a linear combination of several features, i.e. $Y = \alpha_1 X_1 + \alpha_2 X_2 + \dots + \alpha_p X_p$ where non-zero coefficients correspond to predictive features while zero coefficients refer to non-predictive ones. Additionally, some features may be correlated, possibly in a strong fashion. In this setting, a feature importance measure is said to be inaccurate if it provides importance scores that do not comply with the $\alpha_{i}$  coefficients of the true model. Below, we argue that feature importance measures should not be necessarily considered as less accurate for low values of $K$ because importance scores do not align with these coefficients, especially when correlated features are not equally contributive in the linear combination as it is the case in their analyses. In particular, as explained above, low values of $K$ might be more appropriate to address the all-relevant problem, even if this leads to importances that do not match coefficients $\alpha$. Results in these papers are also of interest to discuss biases in feature importance measures due to correlation and we analyse them with this different angle in Section \ref{sec:biases-correlation}. 

It should be noted that a non predictive feature $X_i$ ($\alpha_i=0$) that is strongly correlated with a predictive feature $X_j$ ($\alpha_j > 0$) may therefore be weakly relevant to the target as it may provide part of the information of $X_j$ about $Y$. \cite{nicodemus2010behaviour} characterised such features that appears to be predictive as long as some other features are not included in the model as ``spurious correlation". In our terminology, feature $X_i$ is weakly relevant and totally redundant to $X_j$ with respect to the target. Consequently, coefficients $\alpha$ do not reflect the actual contribution of each feature in a tree-based model.

Authors adopting the minimal-optimal point of view for feature selection (like those mentioned above) concentrate their efforts on identifying only a part of relevant features (i.e., strongly relevant features and a maximal subset of non-redundant ones). It therefore makes sense that redundant features are expected not to be identified as important. However, except in trees built without node-wise split randomisation (i.e., $K=p$), even totally redundant and weakly relevant features can be selected in tree models if they do not compete at some nodes with features that are most useful (and eventually provide the same information). This explains why \citep{auret2011empirical, strobl2008conditional,nicodemus2010behaviour} observe that feature importance measures seem more accurate for high values of $K$ even if low values of $K$ would be more appropriate when interested in solving the all-relevant problem. In such cases, theoretical results (from \citep{sutera2018random} and summarised in Section \ref{sec:th-mdi}) confirm that high values of $K$ imply that redundant features are more frequently masked by strongly relevant features (with positive coefficients) and therefore importance scores are more similar to coefficients $\alpha$. Strongly (``truly'') relevant features are also expected to be used more often and to recover most of the importance in the tree model. \cite{genuer2010variable} indeed observed experimentally that higher values of $K$ increase the importance of truly important variables.

 In addition, low-sample conditions imply that only few variables can be evaluated before reaching nodes with too few samples for an accurate impurity estimation (see Section \ref{sec:biases-estimation}). Increasing the value of $K$ may actually improve importance scores for relevant features that are more often chosen near the root. They are estimated more often and with more samples, potentially making them more stable and more accurately estimated.

Similarly, in presence of many irrelevant features, using a small value of $K$ may induce that numerous splits are made on irrelevant features (because all split variable candidates are irrelevant) . On one hand, such splits do not provide information about the target. On the other hand, a feature is selected based on its spurious relationship with the output and is unfairly credited of some importance for it. Less randomised trees (i.e., $K$ close to $p$) are therefore preferable in such situations. In contrast, if all features are assumed to be equally relevant, then more randomised trees ($K$ close to $1$) are more suitable because they consider all features and not just some of them. 

From all those observations, a trade-off for the value of $K$ needs to be found in order to identify the right set of relevant features while taking into account the nature of the problem.

\subsection{Feature ranking stability and number of trees} \label{sec:importances:stability}

Typically, the number of trees necessary for good performances grows with the number of features \citep{liaw2002classification}. There is no need to grow more trees when the predictions of a subset of the forest are as good as the predictions of the whole forest. This approach however requires to build an unnecessary large number of trees. Therefore, several works propose simple procedure to determine a priori the number of trees for stable and accurate predictions \citep{latinne2001limiting,hernandez2013large}. However, these only concern the predictive ability of tree-based ensemble and the number of trees may not be optimal with respect to the feature importance measures.  In \citep{huynh2012statistical,paul2012stability}, experiments show that the numbers of required trees yielding stable feature selection and predictive performances differ from several orders of magnitude.  

Theoretically, feature importance measures only attribute zero importance scores for irrelevant or masked features. However, in practice, this property relies on one fundamental principle: the number of trees is large enough. Indeed, as pointed in \citep{wehenkel2018random}, in case of too small trees and/or high-dimensional datasets ($p\gg N$), some features may have a zero importance value because they never have been considered during the tree growing process. Additionally, some feature importance may have been evaluated in too few occasions to fairly represent its true contribution. For example, two features forming a XOR structure need to be used at least two times such that both features can be used once before each other. Ultimately, one expect that their averaged importances over a sufficient number of evaluations is the same for both features. In that context, \cite{wehenkel2018characterization} uses the idea of the so-called coupon collector's problem and derives a minimal number of trees (for given parameters $N$, $p$ and $K$) that should be built to have some minimum guarantee that all features are seen at least once.

Even if all features have been considered and receive an importance score, the interpretation of feature importance measures is only possible if results are stable enough, i.e., do not vary significantly if a few additional trees are taken into account, for another ensemble of same size  or if small changes are made to the dataset \citep{strobl2008conditional,saeys2008robust}. Typically, it has been suggested and observed that increasing the number of trees in the forest improves the stability of feature importance measures \citep{liaw2002classification,archer2008empirical,genuer2010variable,paul2012stability}. 
In practice, \cite{liaw2002classification} however observed that importance scores may vary from one ensemble to another while ranking of importances is usually more stable for the same number of trees. In a discussion about stability of ranked gene lists (which aims at identifying a short-list of genes of interest for further analyses), \cite{boulesteix2009stability} state that the rank of a particular feature is usually as important as its value from a practical point of view. \cite{saeys2008robust} note that the analyses of selected features typically require much effort and time and this stresses the need for a stable feature ranking and robust feature selection techniques, especially for model interpretation in biomedical applications \citep{tolocsi2011classification}. 

Assuming enough trees and a stabilised feature ranking, it appears in several data sets that the most important features have typically the highest importance scores \citep{auret2011empirical}. This also suggests that efficient feature selection can be performed by selecting the best $k$ features, where $k$ can be determined by selecting a judicious importance thresholds so as to minimise the number of selected irrelevant features (false positive). Section \ref{sec:threshold} focuses on approaches proposed in the literature to determine this threshold. However, a stable feature ranking does not imply that importance scores are reliable, i.e. that one feature better ranked than another is not necessarily more important. Feature importance measures may be sensitive to different factors, such as the presence of correlated features, and provides unfair importance scores. In Section \ref{sec:biases}, we review the main sources of unfairness (biases) that have been studied in literature.

\subsection{Importance measures vs prediction performances}

Tree-based feature importance measure is usually seen as a side-product of the random forest model. However, a model optimised so as the maximise its performances is typically not adjusted for measuring feature importances \citep{vanderlaan2006statistical}. For example, \cite{paul2012stability} show that the number of trees yielding stable prediction performances is smaller of several orders of magnitude than what is required for a stable feature selection. The number of trees should then be carefully chosen. 
In relation with Section \ref{sec:empirical-K}, randomisation parameter $K$ is usually considered as crucial to obtain good accuracy performances, by controlling the randomisation of the model (and thus the bias-variance trade-off). In classification (respectively, in regression), empirical studies typically suggest that $K=\sqrt{p}$ (resp., $K=p$) is an appropriate and often optimal value with respect to prediction accuracy \citep{geurts2006extremely,strobl2008conditional}. It has however been noticed that model performances is  usually not related to the goodness of tree ensemble parameters for variable importance purposes \citep{auret2011empirical,huynh2012statistical}. Theoretical results suggest that low values of $K$ are more suitable for feature importance measures as $K=1$ is the only way to guarantee that all relevant features can be identified, but this usually requires a larger number of trees to consider all features. Conversely, higher values of $K$ tend to focus more on strongly relevant features. In terms of prediction accuracy, larger values (e.g., $K=\sqrt{p}$ or $K=p$) are more suitable, especially in presence of many irrelevant features, to avoid useless but will definitely prevent some weakly relevant features to be identified. As a result of this discussion, one should carefully choose tree-based parameters and find an appropriate trade-off between feature importance measures (selection or ranking) and prediction performances.

\subsection{Biases} \label{sec:biases}

In what follows, we discuss some experimental results that reveal the presence of biases that affect one or both importance measures. In this work, an importance measure is \textit{biased} if its use in practical conditions differs from its expected and theoretical behaviour. In particular, it is biased if it does not equally treat similar variables, i.e. it does not attribute the same importance score to all features that are equally relevant (or irrelevant) \citep{dobra2001bias}. For example, let us consider two features that are completely independent of the output and are thus irrelevant. An unbiased measure would attribute the same score for both variables while a biased one may have a systematic preference for one of them resulting in a higher importance score. 

We refer to an importance score over-estimation (respectively, under-estimation) as a \textit{positive bias} (resp. \textit{negative} bias). For example, an importance measure that gives a positive score to an irrelevant feature, that should receive a zero importance, is positively biased. 

As a preamble, let us note that MDA feature importance measure relies on the tree structure that has been induced using an impurity criterion. Therefore, some biases that affect impurity measures and thus MDI importance measures, may sometimes also affect MDA. For example, if a feature is never selected because it produces for some reasons no impurity decrease, its permutation does not change the accuracy performances of the model. Conversely, it is also possible that MDA importance measure reduces the importance of features that have been unfairly selected. For the sake of example, let us imagine a bias favouring the selection of redundant features, each providing strictly positive impurity decreases (partly due to noise). Permuting the value of one variable may be ineffective on the prediction of the model, yielding to a null MDA importance scores while the corresponding MDI value might be slightly higher.

\subsubsection{Bias due to masking effect}
\subsubsection*{Source of bias: tree-based method randomisation parameter $K$.}

Masking effect was already mentioned in several occasions in this thesis as a consequence of non-totally randomised split variable selections. In Section \ref{sec:dt-interpretation}, we showed that the inversion between masked and masking features by the means of a small change in the learning set can induce totally different decision tree model (and non randomised), illustrating the high variance of the decision tree algorithms. In Sections \ref{sec:th-mdi} and \ref{sec:empirical-K}, we highlighted that large values of $K$ increase the range of masking effect, resulting in giving preference to strongly relevant features that can not be masked to the detriment of weakly relevant features. The masking effect is maximal when $K=p$. In this section, we discuss the impact of the masking effect on importance measures.

The masking effect denotes situations where several candidate splits on different variables yield roughly the same impurity reduction, but one is always slightly better so that none of the other ones has a chance to be selected by the tree-growing algorithm. Concretely, some branches are never explored as splits are never selected. This induces a positive bias for importances of masking features as they are more frequently selected and their contributions is prioritised over features carrying similar information about the target, i.e., in case of two redundant features with one masking the other, the first one always receives credit for its information because the second one is never selected before. 
In contrast, importance of masked features are negatively biased and under-estimated. Let us note that this bias impacts both importance measures as it affects the building of the tree models.  

A straightforward way to reduce this bias is to reduce the value of $K$. This bias can be totally removed by using totally randomised trees ($K=1$) but this usually requires to increase the number of trees and might jeopardise the predictive performance of the model in presence of many irrelevant features. However, in order to reach global optimality of the ensemble \citep{strobl2008conditional}, it may also be necessary to unveil some feature interactions (e.g., cliques where features are marginally irrelevant and thus unlikely to be selected at first sight) or feature importances (e.g., the second feature in an imbalanced XOR\footnote{An imbalanced XOR is the example used in Section \ref{sec:dt-interpretation}. Two features form a XOR but one is always slightly more marginally relevant and is thus always selected first, obtaining therefore a lower importance score than the other one.}). 

\subsubsection{Bias due to correlation}\label{sec:biases-correlation}
\subsubsection*{Source of bias: presence of correlated features in learning samples.}

Random forest methods are popular in many scientific fields for their ability to handle high-dimensional datasets, as it is particularly the case in biomedical applications. In addition, it is quite common in biomedical studies that features are strongly correlated with each other and this strong correlation usually has a biological explanation. For example, co-regulated genes in expression data are expected to be similar as they relate to the same molecular pathway \citep{tolocsi2011classification}. Neighbouring pixels/voxels in biomedical images are likely associated to the same biological entities (e.g., neurons) implying a spatial correlation \citep{wehenkel2018random}. These examples have motivated several empirical studies of feature importance measures in presence of correlated features.

We however need to distinguish two different biases due to correlation that have been identified in the literature: a preference for correlated features with respect to uncorrelated ones and a preference for correlated groups of smaller sizes. In what follows, let us note that the correlation structure is not the same in both parts. All features in a group share the same predictive power to study the effect of the size of correlated feature groups \citep{tolocsi2011classification} while features within the same group can vary in their information about the target in order to highlight preference for correlated features \citep{strobl2008conditional,nicodemus2010behaviour}.

\paragraph{Preference for (un)correlated features}

In their experimental studies, \cite{strobl2008conditional,nicodemus2009predictor,nicodemus2010behaviour} analyse feature importance measures in presence of correlated features that are not equally contributive in the prediction of the output. Several effects are observed in those studies.

Gini MDI importance measure appears to be biased in the presence of correlation \citep{nicodemus2009predictor}. \cite{strobl2008conditional} observe that correlated features are positively biased with MDA feature importance measure. \cite{strobl2008conditional,nicodemus2010behaviour} report that correlated features are more frequently selected at the first split of the tree (when $K>1$). Nevertheless, across all splits, \cite{nicodemus2010behaviour} observe a slight preference for selection of uncorrelated features. Most of these results are studied for different values of $K$, including totally randomised trees with $K=1$ but excluding non-randomised trees with $K=p$. A first observation is that a correlated feature with zero coefficient in the generating model (see Section \ref{sec:empirical-K} for the description) ends up with larger importance than uncorrelated features with zero coefficient. For \cite{strobl2008conditional}, this phenomenon is due to a spurious correlation that makes a zero coefficient feature marginally informative but conditionally useless. However, such feature carrying redundant information is actually weakly relevant and thus might be selected and contribute to the model. Because of randomisation, it may occur that those features are evaluated without being in competition with their correlated features and end up being selected at some nodes. Such situations are expected to be less likely when the level of randomisation decreases, as observed in those studies with an increasing $K$. Moreover, correlation does not necessarily imply redundancy (as shown in Section \ref{sec:background-redcorr} and in \citep{guyon2006introduction}) and it may slightly increase the predictive contribution of some correlated features with respect to uncorrelated ones with similar coefficients, making them more frequently selected. Simultaneously, non-predictive features that are weakly relevant because of correlation necessarily provide redundant informations. If those features are selected, it reduces the potential interest of selecting correlated features in subsequent nodes in favour of uncorrelated features.

In conclusion, we believe that some of these observations are not actually directly due to the presence of correlation but consequences of masking effect (and the preference for strongly relevant features with high $K$ values) and weakly relevance of features with zero coefficient that benefits from their correlation with highly informative features. Furthermore, \cite{nicodemus2009predictor} noticed that pre-pruning trees by limiting node-size tends to reduce the effect of bias. Therefore, part of observed effects may actually be due to other reasons, such as empirical impurity misestimations in nodes with too few samples. 

\paragraph{Preference for smaller groups of correlated features}

In many biomedical applications, all features within a correlated group are roughly equivalent (e.g., neighbouring voxels in neuroimaging) and can typically be used interchangeably yielding equally performing tree-based models. One can thus associate a group of correlated features with a certain contribution in the prediction of the output. 

Theoretical results, especially MDI importance of totally redundant features (see Section \ref{sec:th-mdi-redundant}), suggest that if features are equivalent\footnote{They are assumed to be strictly equivalent and not masked, or equivalently, that $K=1$. Moreover, let us consider that they are also identical on other aspects, such as their cardinalities, to prevent other biases.}, they are expected to be equally informative and the importance corresponding to the group contribution is equally shared between all correlated features. This implies that features belonging to larger groups receive smaller importance scores compared to a equally informative group but with less correlated features. \cite{tolocsi2011classification} refer to this phenomenon as the correlation bias and noted that if the group is large enough, all features may appear as irrelevant (because of their low importance scores), even if they are highly informative about the output. 

Let us however mention that due to the masking effect can counter-balance this bias as only some features of the group may collect the whole group importance, implying that some other features are masked and so of lower importances.

\subsubsection{Bias due to number of categories and scale of measurement}\label{sec:biases-cardinality}
\subsubsection*{Source of bias: features of various natures and different cardinalities.}

It is known for a long time that the Gini impurity is biased in favour of features of higher cardinalities which thus offer more potential splits \cite{breiman1984classification,kim2001classification}. This phenomenon is usually referred to as the so-called ``bias selection''. Since (Gini) MDI importance measure is directly derived from impurity decreases within trees, it suffers from the same bias towards features of higher cardinalities and numerous studies reported this selection bias for the MDI importance measure (see, e.g., \citep{dobra2001bias,strobl2007unbiased,boulesteix2012overview}). \citep{strobl2007bias} noted that features with high cardinality (i.e., categorical features with a large number of categories or continuous ones)  offer more potential cut-points (splits on that feature) and are thus more likely to provide a good split with respect to features of lower cardinalities. Consequently, the number of categories and the scale of measurement affects the feature and some features might be more frequently selected by a (Gini-based) impurity criterion yielding biased MDI importance scores and misleading feature ranking. 

In contrast, it has been observed that this bias does not impact MDA importance measure \citep{strobl2007bias,boulesteix2012overview}. The explanation given is that a feature that is more frequently selected does not necessarily improves the oob accuracy and thus may receive low MDA importance scores despite being often used in the model. This however increase the variance of MDA importances \citep{boulesteix2012overview}.  

Let us note that comparison between continuous and discrete features (i.e., with different domain size) is not specific to trees and has been studied in other context (see, e.g., \citep{jiang2016efficient}).

\cite{louppe2014understanding} however suggests that the observed bias in \cite{strobl2007bias}'s study is mainly due to empirical misestimations (see Section \ref{sec:biases-estimation}). Indeed, this bias was also observed when no feature or split value selections are performed (e.g., for Extra-Trees with $K=1$ or for totally randomised trees). This suggests that the bias is not only caused by a preference for features of higher cardinalities.

\subsubsection{Bias due to the category frequencies}
\subsubsection*{Source of bias: features of various category frequencies.}

In genetic epidemiology, single nucleotide polymorphisms (SNPs), i.e., variation of a single nucleotide that occurs at a specific position in the genome, are of interest to study some diseases and personalised medicine \citep{carlson2008snps} and are known to interact with each others. In the context of genetic association studies, all SNPs have the same number of categories but vary in their category frequencies. Experiments reveal that both importance measures prefer informative SNPs with larger \textit{minor allele frequency}\footnote{It refers to the frequency of the second most frequent allele value.} (MAF) with respect to informative SNPs with lower MAF and (Gini) MDI importance measure is still biased in case of non-informative SNPs \citep{nicodemus2011letter,boulesteix2011random}. 

This phenomenon, known as the minor allele frequency bias, highlight the bias due to an unbalance in the category frequencies, or more generally in the value distribution.  Let us note that the presence of missing values modifies the actual value distribution and therefore may also impact importance measures.

\subsubsection{Bias due to empirical impurity estimations} \label{sec:biases-estimation}
\subsubsection*{Source of bias: number of learning samples $N$.}

In the beginning of this chapter, the size of the learning set was never actually taken into account. In all theoretical analyses, a learning set of infinite size (in asymptotic conditions) assumes that the joint probability density $P_{V,Y}$ is known. Similarly, most empirical studies consider artificial datasets and thus the generating model was also known. In practice however, the learning set size is finite and this may cause empirical misestimations. For multiway splits, \cite{louppe2014understanding} observes that misestimation bias in (Shannon) MDI importance relates to the misestimations of the mutual information terms $\Delta i (s,t) \approx I(X_i;Y|t)$. For independent random variables $X_i$ and $Y$, the mean of the distribution of finite sample size estimates of their mutual information is proportional to the cardinalities $|X_i|$ and $|Y|$ and inversely proportional to $N_t$, the number of samples in node $t$. This explains why MDI importance measures tend to positively bias importance of features of higher cardinalities. We refer to \citep{louppe2014understanding} for a detailed analysis.  
The case of binary splits is discussed in Section \ref{sec:biases-binary}. 

Notice that the estimation of impurity measures and impurity decreases, in particular Shannon entropy and mutual information, has been widely studied in general frameworks that are not directly related to tree-based methods (see, e.g., \citep{moddemeijer1989estimation,beirlant1997nonparametric,paninski2003estimation,schurmann2004bias}).

\subsubsection{Bias due to binary splits and split value selection} \label{sec:biases-binary}
\subsubsection*{Source of bias: tree-based algorithms.}

Unlike multiway splits, binary splits do not fully exploit a variable. A binary split only discretises the information contained in a variable and therefore the same variable (if not binary) can be reused several times in the same branch. Therefore, binary splits $\Delta i(s,t)$ actually estimates the mutual information between the output and the split outcome (as mentioned in Section \ref{sec:splittingrules}) while multiway splits would provide an estimate of the mutual information between the split variable and the outcome. As a consequence, the estimated mutual information $I(X_i;Y)$ is actually a collection of potentially biased estimates provided by all binary splits \citep{louppe2014understanding}. From a different angle, explored branches are not equivalent in binary and multiway trees. A feature can be used several times in binary trees but only once in multiway tree because branches correspond to single value of the split variable. Feature importance scores are therefore not computed from the same sequence of impurity terms and can therefore be different. 
\cite{louppe2014understanding} gives an illustrative example of two features whose importance scores are different depending on the kind of tree used to compute them. Moreover, the discretisation directly depends on the split value selection and thus the chosen strategy may have an impact on the feature importance scores. For example, a random split value selection such as in Extra-Trees may induce more splits on the same variable and thus more impurity terms, each providing part of the information contained in the feature, compared to an optimal split value such as in Random Forest that may yield all the information contained in the feature in only one split. Feature importance scores obtained with one or the other technique can thus also differ \citep{louppe2014understanding}.

\subsubsection{Bias due to bootstrapping}
\subsubsection*{Source of bias: tree-based algorithms and number of learning samples $N$.}

\cite{strobl2007bias} observe that the bootstrap sampling increases the bias due to the cardinality and therefore suggest not to use bootstrap. Moreover, it has been shown experimentally in \citep{louppe2012ensembles} that bootstrapping is rarely crucial for random forest to obtain good accuracy. 

The second observation is not directly due to the bootstrap mechanism but related to the number of OOB samples. The number of samples $N$ has a direct impact on the resolution of MDA importance measure. On average, around $37\%$ of original samples are not represented in the bootstrap sample. Therefore, when computing the MDA importance score on those samples, only granular values of accuracy change can be obtained when $N$ is small because to resolution is limited to approximately $3/N$, yielding over- and under-estimations of true feature importances \citep{archer2008empirical}.

\section{Meaningful thresholds on feature importances}\label{sec:threshold}

Feature importance measures can be used to rank features in order to facilitate the identification of a useful subset of important features. In this way those features having an importance below some threshold would be considered as unimportant and thus eliminated from further consideration. Unfortunately, there is no natural way to choose a ``good'' threshold on importances \citep{janitza2015computationally}. Therefore, in practice, performing feature selection from such a ranking consists in selecting the $k$ top features (i.e., with highest importance scores). This then reduces to the determination of a ``good'' value of $k$. This may be trivial if one observes a huge gap between relevant and irrelevant features, however in practice, such differences are not common and importance scores are usually smoothly decreasing when going down in the ranking. In such cases, distinguishing when features are no longer informative and when their importances are due to random fluctuations or some undesirable effects, is much more complicated. In this section, we give a non-exhaustive list of several approaches that allow to find either a threshold separating importance scores of relevant features from irrelevant, or propose to use or derive some statistical measure scores for which thresholds are usually more interpretable \citep{konukoglu2014approximate}. Let us note that methods that are not specific to tree-based methods are asterisked. 

\begin{mldescription}

	\mlitem{Random probe* \citep{stoppiglia2003ranking}} In the probe feature method, the key idea is to introduce a random feature in the feature ranking technique. This probe is expected to be ranked similarly as other irrelevant features and all features ranked below the probe should be naturally discarded. However, this probe rank can actually be seen as a random variable and its cumulative distribution function can be computed exactly or estimated (through the generation of several realisations of that random variable). One can then choose an acceptable value of risk and derive the corresponding rank position (and the corresponding threshold importance value) in order to discriminate relevant from irrelevant features.   
	\mlitem{Artificial contrast variables \citep{tuv2006feature}} Similarly to random probes, \cite{tuv2006feature} propose to introduce $M$ \textit{contrast features} that are known to be truly independent of the output and to generate them by randomly permuting values of $M$ input features. By the means of a t-test and a significance level, this allows to identify relevant features as those with importance scores significantly better than those of contrast features. Additionally, they propose to estimate split weights from oob samples and to introduce the mechanism of contrast features in an procedure building iteratively ensemble of trees on kept features and a residual of the target.
	\mlitem{Feature importance as a real-valued parameter* \citep{vanderlaan2006statistical}} The principle of their approach is to define the wished feature importance measure (in particular, in prediction tasks) as a real-valued parameter and propose estimators for those feature importance parameters, accompanied with a p-value and confidence interval.
	\mlitem{MDA Z-score \citep{breiman2008random}}
As defined in Section \ref{sec:MDA}, the MDA importance score for a feature $X_m$ derived from an ensemble $\mathbf{T}$ of $N_T$ trees  consists of the average impact of removing a feature on the accuracy of every tree in the forest.
 In contrast to this ``raw'' MDA importance score of a feature, \citep{breiman2008random} propose a ``scaled" version for which the raw importance score is divided by its standard error. This importance measure is usually referred to as the \textit{z-score} of a feature. If all individual importance scores have the same standard deviation $\sigma$, the standard error of the mean of those individual scores is $\sigma / \sqrt{N_T}$ \citep{strobl2008danger}. The z-score of $X_m$ is therefore given by 
 \begin{eqnarray}\label{eqn:zscore}
 Imp^{mda}_z(X_m, rf, \mathbf{LS}) = \dfrac{Imp^{mda}_{rf}(X_m,rf, \mathbf{LS})}{\frac{\sigma}{\sqrt{N_T}}},
 \end{eqnarray}
 where $rf$ is a random forest algorithm. Assuming that individual importance scores are independent because they are computed from independent bootstrap samples \citep{strobl2008danger}, then Equation \ref{eqn:zscore} tends towards a normal distribution by the central limit theorem. Therefore, a statistical test can be conducted to check whether
 the null hypothesis of zero importance for variable $X_m$ (i.e., corresponding to an irrelevant variable $X_m$) is true or not for a given significance level.
However, \citep{strobl2008danger} find out that the power of this test based on z-scores decreases with an increasing sample size and increases boundlessly with the number of trees and claim that these are undesirable properties for an importance measure.

	\mlitem{Feature set permutation scheme \citep{tang2009identification}} Instead of permuting a single feature, \cite{tang2009identification} propose to permute a set of features. In their application, each gene corresponds to a set of SNPs. Permuting all SNPs corresponding to the same gene allows to make a gene-permutation that directly evaluates the importance of the gene.
	\mlitem{Label permutation scheme \citep{altmann2010permutation}}\label{ref:label-permutation-scheme}
	In their work, they use a permutation test to obtain a threshold for the selection of relevant features. Firstly, un-permuted feature importance scores are computed. Secondly, $m$ permutations are generated by randomly permuting the labels and then, for each permutation, ``permuted" feature importance scores are computed. From that, p-values can be determined by the fraction of permuted importances that are larger than the un-permuted importances and then a threshold can be chosen from a given significance level. \cite{rodenburg2008framework} also suggest a second approach that consists in keeping all features whose importance scores are larger than the mean value of maximal permuted importances. This approach however appears to be very restrictive. Alternatively, \cite{altmann2010permutation} propose to fit a parametrised probability distribution on permuted importance scores. 
	\mlitem{Conditional permutation scheme \citep{strobl2008conditional}} is an alternative permutation scheme aiming at measuring the impact of a feature on the output conditionally to other features in comparison with the classical permutation scheme, and so to correct for the bias towards correlated features.  See side note on page \pageref{sn:nullhyp} for details on this permutation scheme.
	\mlitem{Separate feature permutation scheme \citep{hapfelmeier2013new}} Instead of permuting labels or group of features, \cite{hapfelmeier2013new} propose to permute feature individually while keeping the output and all other features unchanged. The proposed new permutation scheme aims at measuring only the impact of a feature on the output. 
	
	\begin{sidenote}{About permutation schemes}\label{sn:nullhyp}
		Let us consider a set $V$ of input features and an output $Y$. Following \citep{strobl2008conditional,hapfelmeier2013new}, we detail hereafter permutation schemes that have been proposed to evaluate the MDA importance of a feature $X_m \in V$. 
		
		The classical permutation scheme, as described in Section \ref{sec:MDA}, consists in permuting $X_m$ against both the output $Y$ and the remaining features $V^{-m}$. It therefore simulates the independence between $X_m$ and both $Y$ and $V^{-m}$. Mathematically, the evaluated independence (null hypothesis) is $$X_m\indep (Y\cup V^{-m})
	 \Rightarrow  X_m\indep Y \textit{ and } X_m \indep V^{-m} \; \text{(decomposition property)}.$$
	 Note that the converse ($\Leftarrow$) is also verified if the composition property is satisfied (e.g., for a strictly positive distribution). Consequently, a deviation yielding a positive importance can result of a violation of the independence either between $X_m$ and $Y$, or $X_m$ and $V^{-m}$.

	The label permutation scheme (proposed by \citep{altmann2010permutation}, see page \pageref{ref:label-permutation-scheme}) consists in permuting the output value. On one hand, this breaks all relationships between $X_m$ and $Y$, but on the other hand it also breaks any relationships between any input feature in $V^{-m}$ and $Y$. Therefore, the evaluated independence is $$(X_m \cup V^{-m}) \indep Y $$ and would wrongly attribute to $X_m$ the importance of all input features. Permuting the output values is therefore equivalent to permuting all input feature values jointly (i.e., $v^{-m}$ of $V^{-m}$). Instead of permuting all input features, \cite{tang2009identification} suggest to only permute a group of features $P$ including $X_m$. In this work focusing on identifying relevant SNPs (input features) in GWAS\footnote{Genome wide association studies.}, they propose to simultaneously permute all SNPs which belong to the same gene. Within this permutation scheme, the evaluated independence is 
	$$P \indep (Y \cup V\setminus P)$$
	 which does not allow to evaluate the importance of the single feature $X_m$. On the contrary this gives the importance of the group to every feature within this group.
	 \cite{hapfelmeier2013new} argue that each feature needs to be permuted separately in order to correctly estimate the importance of a single variable which is not possible by means of a label permutation.
	 	  
	 With the conditional permutation scheme, \cite{strobl2008conditional} suggest to permute $X_m$ only within groups of observations with $V^{-m}=v^{-m}$ in order to preserve the relationships between $X_m$ and all features in $V^{-m}$ while destroying the link with $Y$. It corresponds to the following evaluated independence 
	 $$X_m\indep Y | V^{-m}$$
	 which highlights the conditioning on $V^{-m}$. Interestingly, it corresponds to the definition of strongly relevant features. The conditional permutation scheme may therefore miss some weakly relevant features that are independent of $Y$ knowing all other features (e.g., redundant features).
	 
	\end{sidenote}
	
	\mlitem{Approximate false positive rate control \citep{konukoglu2014approximate}} Permutation techniques can be intractable for high-dimensional datasets and therefore \cite{konukoglu2014approximate} propose an approach to determine thresholds and control the false positive rate in random forest method at no additional computational cost. Based on the feature selection frequency importance measure (see Section \ref{sec:imp-imp}), they rely their approach on the estimation of the probability that a feature is selected $k$ times in a tree ensemble if it is assumed to be irrelevant to the output. They propose an approximate model for selection frequency in random forest from which one can determine a desired level of false positive rate and obtain an optimal threshold on the selection frequency importance scores.
	
	\mlitem{Conditional error rate* \citep{huynh2008exploiting,huynh2012machine}} In order to overcome limitations of classical permutation-based techniques of false positive rate estimation, \cite{huynh2008exploiting} propose the conditional error rate (CER) as an alternative measure to be associated with each importance threshold $\tau_i$. It estimates the probability to include an irrelevant feature when selecting all features (assumed to be relevant) with an importance score greater or equal to $\tau_i$.
 	
	Note that \citep{huynh2012statistical,wehenkel2018characterization} review statistical interpretation of (tree-based) feature importance scores, including random probe techniques and conditional error rate.  

	\mlitem{Rank-based conditional error rate \citep{wehenkel2017tree}} While the CER is based on the importance scores, \cite{wehenkel2017tree} propose an adaptation of CER based on rank for group of features. Let us assume an original order of feature groups ranked by order of decreasing importance scores. The key principle is that a relevant group should not be as well or better ranked than originally once all statistical links within this group and in all groups ranked below (in the original order) are broken. \cite{wehenkel2018characterization} note that this variant is less restrictive than the original method.

	\mlitem{Subsampling and delete-d jacknife \citep{ishwaran2018standard}} Recently, \cite{ishwaran2018standard} study several sampling approaches for estimating MDA importance measure variance, such as double-bootstrap, subsampling and delete-d jacknife algorithm. They additionally propose a subsampling approach that can be used to estimate the standard error of MDA importance measure and for defining confidence intervals. 
	
\end{mldescription}

\section{Extensions and derivations} \label{sec:imp-ext}
In this section, we briefly review some methodologies that exploit feature importance measures or derive their use to perform new tasks.  

\begin{mldescription}
	\mlitem{Recursive Feature Elimination \citep{diaz2006gene}} Their approach is an instance of the Sequential Back Elimination (SBE, see \ref{sec:fs-search}) that recursively removes features with the smallest importance scores computed with a tree-based ensemble method.  
	\mlitem{Enriched random forests \citep{amaratunga2008enriched}} In presence of many irrelevant features, many splits can be made on irrelevant features because all split variables candidates were irrelevant. In order to circumvent that, \cite{amaratunga2008enriched} propose a weighted random sampling in each node instead of a uniform one. They suggest to determine weights as the p-value of a t-test.  
	\mlitem{Guided regularized random forest \citep{deng2013gene}} Similarly to \citep{amaratunga2008enriched}, their approach first builds a classical random forest and then use feature importance scores to guide the feature selection process in a second model (i.e., regularized random forest \citep{deng2012feature}). 
	\mlitem{Variable importance-weighted feature selection \citep{liu2017variable}} Similarly with the previous approach, instead of selecting split variable candidates at random, \cite{liu2017variable} propose to sample features according to their importance scores in order to focus on informative features. 
	\mlitem{Random Subspace for feature selection \citep{ho1998random, lai2006random}} Inspired from the Random Subspace method proposed by \citep{ho1998random}, this approach consists in growing each tree of the ensemble on a random subspace of $K$ ($\le p$) features randomly chosen. Similarly to a classical forest, feature importance scores are then computed for each tree and then aggregated with the difference that at least $p-K$ features have necessarily a zero importance for each tree. One can however expect that all available features can be considered in the tree, even if it is made of only few nodes. Let us note that this approach is also compatible with the Random patches method \citep{louppe2012ensembles}.
	\mlitem{Sequential Random Subspace \citep{sutera2018random}}	In this sequential variant of the random subspace method, the key ideas are that (i) some relevant features may be difficult to identify because they need to be considered conditionally to some other features (ii) which are necessarily relevant. Therefore, the principle is to reuse more frequently features that have already been identified as relevant in order to make the detection of other relevant features easier. In contrast with approaches such as variable importance-weighted feature selection \citep{liu2017variable}, one can force the method to keep a part of exploration to discover masked features for instance. 
	 
	\mlitem{Feature selection with a knock-out strategy \citep{ganz2015relevant}} This approach is interesting in several respects. Uncommonly, they consider the frequency selection as importance measure on which they apply a false positive rate control (see Section \ref{sec:threshold} and \citep{konukoglu2014approximate}). Moreover, at each iteration, the identified set of relevant features are removed (``knocked out") in order to force the algorithm to identify remaining relevant features since already identified are no longer available. This method has the merit of looking for all relevant features without taking care of accuracy performances. However, \citep{sutera2018random} show that relevant features may be required to reveal some others that are more difficult to detect (e.g., a clique) but this may be circumvented by the use of frequency selection instead of other importance measures.
	
	\mlitem{Representative feature(s) \citep{tolocsi2011classification}} Proposed as a way to reduce the correlation bias, the idea developed in \citep{tolocsi2011classification} is to group several ``similar'' features into \textit{representative feature(s)} that can then be used as input features for the model. At the end, the importance scores of the original features can be retrieved as the importance of the representative feature (or the average in case of several representatives). 
	
	\cite{wehenkel2018characterization} reviewes some approaches to determine the representative features and discusses those based on a priori knowledge (e.g., atlas for brain regions \citep{wehenkel2018characterization}, self-organizing maps for genes \citep{rodenburg2008framework}) and on neighbouring positions. Let us also mention that similar features can also be identified with techniques such as hierarchical clustering \citep{rodenburg2008framework}.
	
	\mlitem{Group importance smoothing \citep{wehenkel2018characterization}} Because of masking effect or some other biases, similar features may receive different importance scores. \citep{wehenkel2018characterization} proposes two ways to post-process importance scores in order to rebalance more fairly importance scores among similar features. The first approach consists in sharing the importance score of a feature with its neighbours. The second approach consists in assigning all features of a group (e.g., based on a priori knowledge) the same ``group" importance scores that have been derived from the distribution of all importance scores within the group. A group importance is then derived using either the average, the sum, or the maximum of the individual importance scores within the group.
		
\end{mldescription}
\section{Other importance measures} \label{sec:imp-new}

Previous sections show in several respects that MDI and MDA feature importance measures are not perfect and can not address all needs. Thus, several other importance measures have been proposed in the literature. Some of them are described in this section.
Note that we exclude from the following list \textit{local feature importance measures}, such as \textit{Shapley values} \citep{lundberg2017consistent,lundberg2018consistent}, that evaluate feature importances for a given input vector $x$, although global feature importance measures can be obtained from such local measures by aggregating them over a sample of input vectors.

\begin{mldescription}
	
		\mlitem{Cross-validated MDA feature importance \citep{janitza2015computationally}} Firstly, they propose an alternative approach to compute the MDA importance measures of cross-validated subsets instead of oob samples. The principle is similar: the accuracy is estimated on samples that have not been used to learn the model, i.e., the remaining fold. Secondly, they propose a new variable importance test that is computationally more efficient than traditional permutation schemes discussed in Section \ref{sec:threshold}.
		
		\mlitem{Contextual importance measures \citep{sutera2016context}} MDI importance measures are extended to identify and characterise features whose relevance is context-dependent (i.e., varying depending on the context) or context-indepen\-dent. 
		
		\mlitem{AUC-based permutation importance measure \citep{janitza2013auc}} To overcome the sub-optimality of random forest methods in presence of strongly unbalanced data, \cite{janitza2013auc} propose to use an AUC-based criterion instead of an error-rate-based one for the MDA importance measure. 
		
		\mlitem{Change in class vote distribution \citep{paul2013identification}} In this work, a new feature importance index is proposed that uses a statistical test to determine whether permuting a variable significantly influences the class vote distribution of the forest. This new importance measure correlates well with MDA importance and has the advantage of providing directly a p-value.

\end{mldescription}

Without more explanations, let us however mention two works proposing bias-corrected impurity importance measures: \cite{sandri2008bias} add uninformative features (e.g., permutation of original features) among input ones, and \cite{nembrini2018revival} propose an efficient procedure that does not require permutation and is feasible for extremely large datasets.

\section{Some applications exploiting feature importances} \label{sec:imp-app}

To conclude this chapter, we briefly mention in this section two applications of feature importance measures in the biomedical domain.

\begin{mldescription}
	\mlitem{Gene network inference} In genomics, regulatory gene network inference consists in the identification of all gene-to-gene interactions from their expression level and reconstruct a network with these interactions. Concretely, one needs to infer a (un)directed graph where each nodes is a biological entity (e.g., a gene) and edges connecting two nodes represent an interaction between them. In all generality, the GENIE3 method aims at inferring a network of $p$ nodes by decomposing it into $p$ independent supervised learning problems. Each feature is in turn considered as the target to predict from all $p-1$ other features. When these sub-problems are solved by the means of tree-based methods, feature importance measures can be derived and seen as indications of the degree of association between input features and the target. Concretely, in a model predicting $X_j$ from $V^{-j}$, the importance score of a feature $X_i \in V^{-j}$ is used a the degree of association between node $i$ and node $j$. Once all sub-problems are solved, the ranking of all gene-gene pairs can be used to reconstruct the global network (e.g., by selecting the stronger interactions). Chapter \ref{ch:connectomics} focuses on that application.
	\mlitem{Neuroimaging} Random forest methods are able to handle high-dimensional data\-set ($p\gg N$), such as neuroimaging datasets, and therefore constitute interesting alternatives to SVM and deep learning methods in the context of neuroimaging datasets. For example, in the context of fMRI datasets, \citep{langs2011detecting} use Gini MDI importance to identify interacting brain regions that are activated under experimental stimuli, and \citep{richiardi2010brain} exploit tree-based feature importance measures to determine relevant brain region connections. In the particular case of Alzheimer's disease, \cite{wehenkel2018random,wehenkel2018characterization} exploit feature importance measures to identify important (group of) voxels from Positron Emission Tomography (PET) images in order to identify brain regions involved in the prognosis of the disease.

\end{mldescription}

\begin{summary}

In the litterature, several measures have been proposed to quantify the importance of features from tree-based ensemble models. Because of their ability to handle feature interactions and non-linearities, these measures are interesting alternatives to classical statistical tests. Driven by many successful applications (notably in the biomedical domain), several studies have been carried out to analyse these measures from different perspectives that have revealed several biases from which these measures suffer. Recent theoretical works also give new insight on previous empirical results. One major drawback of standard importance measures is that they lack a statistical interpretation that would allow to naturally determine a threshold value to distinguish truly important from non-important features. Several techniques, mostly based on random permutations, have however been proposed in the literature to address this issue. In addition, new tree-based importance measures have been designed to go further in the exploitation and interpretation of tree-based ensemble models.

\end{summary}

\chapter{Characterisation of MDI importance measure}
\label{ch:mdi}
\begin{overview}

In this chapter, we characterise the Mean Decrease of Impurity (MDI) feature importance measure as computed by an ensemble of randomised trees. First, in asymptotic conditions, we derive a multi-level decomposition of the information jointly provided by all input features about the output, with a particular attention on the link between importance and relevance of features. We also extend the characterisation to take into account the presence of redundant features. We then analyse importance measure properties in the case of non-totally randomised, non-fully developed and binary trees, respectively. Finally, we discuss how these properties may change in the finite case, in particular in the number of trees and samples.\\

\textbf{\textcolor{RoyalBlue}{References:}} This chapter presents results that were published in the following publications:
\begin{enumerate}[\textcolor{Gray!80}{$\bullet$}]
	\item \bibentry{louppe2013understanding};
	\item \bibentry{sutera2018random}.
\end{enumerate}
Note that proofs from the first publication, which were also part of \citep{louppe2014understanding}, are not reproduced below.
\end{overview}

Nowadays, most of state-of-the-art supervised learning algorithms typically provide a black-box model able to accurately predict the output. In many applications, a particular attention is paid to an understanding of the modelled system, which is typically not possible with a black-box model. Random forest methods, by the means of importance measures, allow to identify important features which are the key elements of the model. This interpretation provides insights to understand the underlying mechanism. 
Concretely, given an ensemble of trees, one may derive a numerical score for each feature that assesses its importance in the tree-based model. \cite{breiman2001random,breiman2003random} proposed two importance measures\footnote{Note that both importance measures are described in Chapter \ref{ch:importances}.}. Firstly, the Mean Decrease of Accuracy aims at evaluating the contribution of a feature for predicting the output as the change in accuracy of the model when this feature is permuted.
Secondly, the Mean Decrease of Impurity (MDI) relies on the impurity criterion used to grow trees. In this chapter, we only focus on that particular importance measure. It adds up the weighted impurity decreases $\Delta i(s,t)$ over all nodes $t$ in a tree $T$ where the variable $X_m$ is used to split and then averages this quantity over all trees in the ensemble, i.e.\footnote{From now on, this thesis only focuses on the MDI importance measure and the notation is thus simplified accordingly, i.e., $Imp(X)$ is equivalent to $Imp^{mdi}(X)$.} : 
\begin{eqnarray}\label{eqn:mdi-mdi-empirical}
Imp(X_m) = \dfrac{1}{N_T} \sum_{T} \sum_{t\in T: v(s^*_t) = X_m} p(t) \Delta i(s_t^*,t)\\
\nonumber \text{with}\; \Delta i(s,t) = i(t) - \dfrac{p(t_L)}{p(t)} i(t_L) -\dfrac{p(t_R)}{p(t)}i(t_R)
\end{eqnarray}
 where $i$ is the impurity measure (introduced in Section \ref{sec:splittingrules}), $p(t)$ is the proportion of samples reaching node $t$, $v(s_t)$ is the variable used in the split $s_t$, at node $t$, and $t_L$ and $t_R$ are the left and right successors of $t$ after the split. 
 
In Chapter \ref{ch:importances}, we outlined a theoretical analysis of both importance measures and then focused on their practical uses, in particular biases that may provide misleading interpretations of feature ranking and importance scores. We also consider several extensions, derivations and applications in which feature importance are typically used. 

Despite these numerous works, only few studied theoretically feature importance measures from a theoretical point of view \citep{ishwaran2007variable,louppe2013understanding,louppe2014understanding,zhu2015reinforcement,gregorutti2017correlation,sutera2018random}. In order to go one step further in the understanding of this measure, this chapter aims at providing an in-depth theoretical analysis of the MDI importance derived from ensembles of randomised trees in an infinite sample setting. We also discuss how it may change in the case of finite sample and tree ensemble size conditions.

As a preambule, Section \ref{sec:degreedef} first defines the degree of a relevant variable and provide two propositions that characterize minimal conditionings that make relevant variables dependant of the output. Section \ref{sec:mdi-totally} then provides a theoretical characterisation of MDI importance measures in asymptotic conditions in the case of totally randomized and fully developed and presents an interpretable decomposition of the information jointly provided by all input features about the output at several levels of feature interactions. Section \ref{sec:mdi-rel} shows that MDI importance measures can be used to identify relevant features. 
Section \ref{sec:mdi-red} extends the characterisation of MDI importance measures to highlight the impact of the presence of redundant features. Sections \ref{sec:mdi-K} and \ref{sec:mdi-pruning} consider respectively non-totally randomised and non-fully developed trees and analyse to what extent MDI properties are still verified. Section \ref{sec:mdi-binary} examines trees made with binary splits and aims at extending the characterisation of multiway trees to binary ones that are more common in practice. Finally, Section \ref{sec:mdi-finite} discusses the finite case and in particular considers a finite number of trees (Section \ref{sec:mdi-finite-trees}) and a finite number of samples (Section \ref{sec:mdi-finite-samples}).

\subsection*{Notational conventions}
\textit{For sake of clarity, the setting under study is reminded  at the beginning of the sections and summarised by some of the following parameters (described in Chapter \ref{ch:trees}): the split selection randomisation parameter $K$ of the random forest algorithm, the maximal depth $D$ of the tree structure, the split cardinality\footnote{$|s_t|=|v(s_t)|$ denotes a tree built with multiway exhaustive splits as the split cardinality equals the number of values of split variable $v(s_t)$.} $|s_t|$ used in decision trees, the number of trees $N_T$ in the ensemble, and the number of samples $N$ of the learning set. For the sake of completeness, subscripts and superscripts will be used to specify the parameter values of the tree-based method used to derive importance scores: $Imp^{K,D}_{N,N_T}$ corresponds to the importance measure computed with an ensemble of $N_T$ trees built with a split randomisation parameter $K$ and a maximal depth $D$ on a dataset of $N$ samples.}

\section{Degree of relevant variables}\label{sec:degreedef}
 
In addition to the definitions of relevance provided in Section \ref{sec:relevance} (Definitions \ref{def:strongrelevance}, \ref{def:marginallyrelevant} and \ref{def:weakrelevance} in terms of conditional independences and Definitions \ref{def:relevanceMI} and \ref{def:strongweakrelevanceMI} in terms of mutual informations), for some results derived below, we need to qualify relevant variables according to their degree:
\begin{definition}{\citep[Definition 3]{sutera2018random}}\label{def:mdi:degree}
The \textbf{degree} of a relevant variable $X$, denoted $deg(X)$, is
defined as the minimal size of a subset $B\subseteq V$ such that
$Y\nindep X|B$.
\end{definition}
Relevant variables $X$ of degree 0, i.e. such that $Y\nindep X$
unconditionally, will be called \textbf{marginally relevant}.

We will say that a subset $B$ such that $Y\nindep
X|B$ is \textbf{minimal} if there is no proper subset $B'\subseteq B$
such that $Y\nindep X|B'$. The following two propositions give a
characterisation of these minimal subsets.

\begin{proposition}{\citep[Proposition 1]{sutera2018random}} \label{prop:mdi-only-relevant}
A minimal subset $B$ such that $Y\nindep X|B$ for a relevant
  variable $X$ contains only relevant variables. 
\end{proposition}
\begin{proof}
  Let us assume that $B$ contains an irrelevant variable $X_i$. Let us
  denote by $B^{-i}$ the subset $B\setminus\{X_i\}$. Since $X_i$ is
  irrelevant, we have $Y\indep X_i|B^{-i}\cup\{X\}$. Given that $B$ is
  minimal we furthermore have $Y\indep X|B^{-i}$ where
  $B^{-i}=B\setminus\{X_i\}$. By using the contraction property of any
  probability distribution (see side note on page
  \pageref{sn:distributionproperties}), one can then conclude from
  these two independences that $Y\indep \{X,X_i\}|B^{-i}$ and, by
  using the weak union property, that $Y\indep X|B$, which proves the
  theorem by contradiction.
\end{proof}

\begin{proposition}{\citep[Proposition 2]{sutera2018random}} \label{prop:only-degree}
Let $B$ denote a minimal subset such that $Y\nindep X|B$ for a
relevant variable $X$. For all $X'\in B$, $deg(X')\leq |B|$. 
\end{proposition}
\begin{proof}
      If we reduce the set of features $V$ to a new set $V'=B\cup\{X\}$,
$X$ will remain relevant, as well as all features in $B$, given
Proposition \ref{prop:mdi-only-relevant}. So, for any feature $X'$ in
$B$, there exists a subset $B'= B\cup\{X\}\setminus X'$ such
that $Y \nindep X'|B'$ and the degree of $X'$ is therefore $\leq
|B|$.
\end{proof}

These two propositions show that a minimal conditioning $B$ that makes a
variable dependent on the output is composed of only relevant variables whose
degrees are all smaller or equal to the size of $B$. Let us note that we will provide in Section \ref{sec:convergence} a more stringent characterisation of variables in minimum conditionings in
the case of specific classes of distributions.\\

\section{Totally randomised and totally developed trees}\label{sec:mdi-totally}

\subsection*{Setting of this section: $K=1, D=p, |s_t|=|v(s_t)|, N_T \rightarrow\infty, N\rightarrow\infty$}

Let us assume a set $V={X_1,...,X_p}$ of categorical input variables and a categorical output $Y$. Let us consider a joint probability density $P_{V,Y}$ of $X_1,\dots,X_p,Y$ and a learning set $\mathbf{LS}$ of $N$ observations of $X_1,\dots,X_p,Y$ independently drawn from that distribution. From $\mathbf{LS}$, an infinitely large ensemble of totally randomised, multiway and fully developed trees is inferred. As a reminder of Chapter \ref{ch:trees}, such trees are built such that, for each node $t$, a split variable $X_i$ is selected totally at random among those not yet picked and used to split the node $t$ into $|X_i|$ branches (i.e., one for each value of $X_i$), until there is no more remaining unused features. Let us note that all branches have the same depth $p$, because each feature is used once along each branch. For sake of simplicity, we only consider Shannon impurity to evaluate the importances, but results can be extended to some extent to other impurity measures as shown in \citep[Appendix I]{louppe2013understanding}. Note that in the totally randomized setting, the tree structure does not depend on the impurity measure, but the MDI importance measure derived from this structure obviously does.

In that context, let us consider the MDI importance as defined by Equation \ref{eqn:mdi-mdi-empirical} computed by this ensemble of trees. 

\begin{theorem}{\citep[Theorem 1]{louppe2013understanding}}\label{thm:mdi-totally-imp}
The MDI importance of $X_m \in V$ for $Y$ as computed
with an infinite ensemble of fully developed totally randomized trees and an
infinitely large learning set is:
  \begin{equation}\label{eqn:mdi-totally-imp}
  Imp_{\infty,\infty}^{1,p}(X_m)=\sum_{k=0}^{p-1} \frac{1}{C_p^k} \frac{1}{p-k} \sum_{B \in {\cal P}_k(V^{-m})} I(X_m;Y|B),
  \end{equation}
\end{theorem}
\noindent where $V^{-m}$ denotes the subset $V \setminus \{X_m\}$, ${\cal
P}_k(V^{-m})$ is the set of all subsets of cardinality $k$ of  $V^{-m}$, and
$I(X_m;Y|B)$ is the conditional mutual information of $X_{m}$ and $Y$ given the
variables in $B$. A setting when both learning set and tree ensemble sizes are assumed to be infinitely large is further referred to as \textit{asymptotic conditions}.

\begin{proof}
	See \citep[Appendix B]{louppe2013understanding} for a proof.
\end{proof}
\begin{theorem}{\citep[Theorem 2]{louppe2013understanding}}\label{thm:mdi-totally-sum}
For any ensemble of fully developed trees in asymptotic learning sample size
conditions (e.g., in the same conditions as those of Theorem~\ref{thm:mdi-totally-imp}), we
have that
\begin{equation}\label{eqn:mdi-totally-sum}
\sum_{m=1}^{p}Imp^{1,p}_{\infty,\infty}(X_m) = I(X_{1}, \ldots, X_{p} ; Y).
\end{equation}
\end{theorem}
\begin{proof}
	See \citep[Appendix C]{louppe2013understanding} for a proof.
\end{proof}

In Theorem \ref{thm:mdi-totally-sum} , the term $I(X_{1}, \ldots, X_{p} ; Y)$ denotes the information contained in the set of input variables about the output variable and can be computed for a given joint probability density $P_{V,Y}$. Let us notice that this property actually holds for every single tree, and consequently also for any ensemble of $N_T$ trees, and in particular when $N_T$ goes to infinity. Given that $I(X_{1}, \ldots, X_{p} ; Y)$ is fixed for a given problem, Theorem \ref{thm:mdi-totally-imp} shows that an increase of the importance of  one feature will always come with a decrease of the importance of another feature.

Combining Theorems \ref{thm:mdi-totally-imp} and
  \ref{thm:mdi-totally-sum} in the context of ensemble of trees, the
  information contained in the set of inputs variables can be decomposed into the following three-level nested sums:
	\begin{eqnarray}\label{eqn:mixsumimp}
	I(X_1,\dots,X_p:Y) &=& \sum_{m=1}^p \sum_{k=0}^{p-1} \frac{1}{C_p^k} \frac{1}{p-k} \sum_{B \in {\cal P}_k(V^{-m})} I(X_m;Y|B)\\
	\end{eqnarray}

The first sum is over the variables, the second sum over the degrees $k$ of the interaction terms, and the third sum over all conditioning subsets $B$ of size $k$. Equivalently, the first two sums can be swapped to yield the following decomposition of $I(X_1,\dots,X_p;Y)$:
\begin{eqnarray}
  I(X_1,\dots,X_p:Y) &=& \sum_{k=0}^{p-1}  \sum_{m=1}^p  \frac{1}{C_p^k} \frac{1}{p-k} \sum_{B \in {\cal P}_k(V^{-m})} I(X_m;Y|B).\label{eqn:mixsumimp2}.
\end{eqnarray}
While Equations \ref{eqn:mdi-totally-sum} and \ref{eqn:mixsumimp} divide the total output information between the features, computing each term of the outer sum in Equation \ref{eqn:mixsumimp} will give a decomposition of $I(X_1,\dots,X_p;Y)$ per interaction degree, which highlights how important feature interactions are for predicting the output.

Table \ref{fig:mdi-imp} illustrates these two ways of decomposing $I(X_1,\dots,X_p;Y)$ in the context of the digit recognition problem of \citep{breiman1984classification} (see Appendix \ref{app:digit} for a description of this problem). We can observe that almost all inner sum terms $\sum_B I(X_m;Y|B)$ are strictly positive implying that large conditioning sets $B$ (corresponding to deep nodes in the tree) still contribute to the total variable importance. In this example, importances monotonically decrease with the degree of interaction $k$, but this is not always the case (e.g., with XOR-like structures)

\begin{table}[htbp]
	\centering
	\begin{tabular}{c|lllllll||c} \hline
		$ $&$k=0$&$k=1$&$k=2$&$k=3$&$k=4$&$k=5$&$k=6$&$\sum_k$\\ \hline
$X_1$&0.103&0.085&0.068&0.053&0.042&0.033&0.029&0.413\\
$X_2$&0.139&0.126&0.105&0.082&0.060&0.042&0.029&0.582\\
$X_3$&0.103&0.091&0.081&0.073&0.066&0.061&0.057&0.531\\
$X_4$&0.126&0.114&0.097&0.077&0.058&0.042&0.029&0.542\\
$X_5$&0.139&0.123&0.106&0.090&0.076&0.065&0.057&0.657\\
$X_6$&0.067&0.056&0.043&0.031&0.020&0.010&0.000&0.226\\
$X_7$&0.126&0.098&0.070&0.045&0.025&0.010&0.000&0.372\\\hline
$\sum_m$&0.802&0.692&0.568&0.450&0.347&0.262&0.200&3.322
	\end{tabular}

\caption{Feature importances as computed with an ensemble of totally randomised trees. Last row ($\sum_m$) corresponds to importances per interaction degree (i.e., summed over over all features, see Equation \ref{eqn:mixsumimp2}) while last column ($\sum_k$) corresponds to importances per feature (i.e., summed over all interaction degrees, see Equation \ref{eqn:mixsumimp}).
Let us note that the sum of all importances is equal to $I(X_1,\dots,X_7;Y)=\log_2(10)=3.322$.
}
\label{fig:mdi-imp}
\end{table} 


The last sum in Equations \ref{eqn:mixsumimp} or \ref{eqn:mixsumimp2} includes all interaction terms of a given degree and it is weighted in a way that depends only on the combinatorics of possible interaction terms. Interestingly, the weight $\frac{1}{C^k_p} \frac{1}{p-k}$ in front of each such sum perfectly counter-balances the change in the size of $\mathcal{P}_k(V^{-m})$ with $k$, since we have
$$\dfrac{|\mathcal{P}_k(V^{-m})|}{C^k_p (p-k)} =
\dfrac{C^{k}_{p-1}}{C^k_p (p-k)} = \dfrac{1}{p},$$ which is
independent of $k$. This result is illustrated numerically for several
values of $p$ in Figure \ref{fig:evolutionofweights}. Given that each mutual information term $I(X_m;Y|B)$ is upper bounded by $H(Y)$, each term of the sum over $k$ in Equation \ref{eqn:mixsumimp} is upper bounded by $\frac{1}{p} H(Y)$, which does not depend on $k$.  It shows that
importance measures are inherently unbiased with respect to interaction degrees.

\begin{figure}
	\centering
	\subfloat[Evolution of $\frac{1}{C^k_p (p-k)}$ with respect to $k$. Note the symmetry.\label{fig:evolutionofweights:a}]{
		\begin{tikzpicture}
		\begin{axis}[width=\textwidth*0.5,ylabel near ticks,legend style={at={(1.45,0.5)},anchor=east},
		xlabel={$k$},
		ylabel={$(C^k_p (p-k))^{-1}$}
		]
		\addplot[Orange,mark=*] coordinates {
			(0,0.5) (1,0.5) 
		};
		
		\addplot[RoyalBlue,mark=*]  coordinates{
			(0,0.333333333333) (1,0.166666666667) (2,0.333333333333) 
		};
		
		\addplot[ForestGreen,mark=*] coordinates{
			(0,0.166666666667) (1,0.0333333333333) (2,0.0166666666667) (3,0.0166666666667) (4,0.0333333333333) (5,0.166666666667) 
			
		};
		
		\addplot[Orange,mark=square*] coordinates{
			(0,0.125) (1,0.0178571428571) (2,0.00595238095238) (3,0.00357142857143) (4,0.00357142857143) (5,0.00595238095238) (6,0.0178571428571) (7,0.125) 
		};
		
		\addplot[RoyalBlue,mark=square*]  coordinates{
			(0,0.1) (1,0.0111111111111) (2,0.00277777777778) (3,0.00119047619048) (4,0.000793650793651) (5,0.000793650793651) (6,0.00119047619048) (7,0.00277777777778) (8,0.0111111111111) (9,0.1) 
		};
		\legend{$p=2$,$p=3$,$p=6$,$p=8$,$p=10$}
		\end{axis}
		\end{tikzpicture}
	}\\
	\subfloat[Evolution of $C^k_{p-1}$ with respect to $k$. Note the symmetry.\label{fig:evolutionofweights:b}]{
		\begin{tikzpicture}
		\begin{axis}[width=\textwidth*0.5,ylabel near ticks,legend style={at={(1.45,0.5)},anchor=east},
		xlabel={$k$},
		ylabel={$C^k_{p-1}$}
		]
		\addplot[Orange,mark=*] coordinates {
			(0,1.0) (1,1.0) 
		};
		
		\addplot[RoyalBlue,mark=*]  coordinates{
			(0,1.0) (1,2.0) (2,1.0) 
		};
		
		\addplot[ForestGreen,mark=*] coordinates{
			(0,1.0) (1,5.0) (2,10.0) (3,10.0) (4,5.0) (5,1.0) 
			
		};
		
		\addplot[Orange,mark=square*]  coordinates{
			(0,1.0) (1,7.0) (2,21.0) (3,35.0) (4,35.0) (5,21.0) (6,7.0) (7,1.0) 
		};
		
		\addplot[RoyalBlue,mark=square*] coordinates{
			(0,1.0) (1,9.0) (2,36.0) (3,84.0) (4,126.0) (5,126.0) (6,84.0) (7,36.0) (8,9.0) (9,1.0) 
		};
		\legend{$p=2$,$p=3$,$p=6$,$p=8$,$p=10$}
		\end{axis}
		\end{tikzpicture}
	}\\
	\subfloat[Evolution of $\frac{C^k_{p-1}}{C^k_p (p-k)}$ with respect to $k$. Note that for a given $p$, all values are equal.\label{fig:evolutionofweights:c}]{
		\begin{tikzpicture}
		\begin{axis}[width=\textwidth*0.5,ylabel near ticks,legend style={at={(1.45,0.5)},anchor=east},
		xlabel={$k$},
		ylabel={$\dfrac{C^k_{p-1}}{C^k_p (p-k)}$}
		]
		\addplot[Orange,mark=*] coordinates {
			(0,0.5) (1,0.5) 
		};
		
		\addplot[RoyalBlue,mark=*]  coordinates{
			(0,0.333333333333) (1,0.333333333333) (2,0.333333333333) 
		};
		
		\addplot[ForestGreen,mark=*] coordinates{
			(0,0.166666666667) (1,0.166666666667) (2,0.166666666667) (3,0.166666666667) (4,0.166666666667) (5,0.166666666667) 
			
		};
		
		\addplot[Orange,mark=square*] coordinates{
			(0,0.125) (1,0.125) (2,0.125) (3,0.125) (4,0.125) (5,0.125) (6,0.125) (7,0.125)  
		};
		
		\addplot[RoyalBlue,mark=square*]  coordinates{
			(0,0.1) (1,0.1) (2,0.1) (3,0.1) (4,0.1) (5,0.1) (6,0.1) (7,0.1) (8,0.1) (9,0.1) 
		};
		\legend{$p=2$,$p=3$,$p=6$,$p=8$,$p=10$}
		\end{axis}
		\end{tikzpicture}
	}
	\caption{Interpreting the weights in the three-level decomposition of total importance in Equation \ref{eqn:mixsumimp}. Figure \ref{fig:evolutionofweights:a} shows how the weights of the second level of decomposition evolve with respect to $k$ for several number of features $p$. Figure \ref{fig:evolutionofweights:b} shows the number of combinations $B$ in the third level of decomposition. Figure \ref{fig:evolutionofweights:c} combines both decompositions and shows that sub-importance terms corresponding to every interaction degree equally contribute to the total importance.
		\label{fig:evolutionofweights}}
\end{figure}

\section{Importances of relevant and irrelevant variables}\label{sec:mdi-rel}
\subsection*{Setting of this section: $K=1, D=p, |s_t|=|v(s_t)|,N_T \rightarrow\infty, N\rightarrow\infty$}

The following theorems characterise the importances of relevant and irrelevant variables. These results can be derived from the equivalence between condition independance and zero conditional mutual information (see Section \ref{sec:th-mdi}).

\begin{theorem}{\citep[Theorem 3]{louppe2013understanding}}\label{thm:mdi-totally-irrelevant}
  $X_i \in V$ is irrelevant to $Y$ with respect to $V$ if and only if  its
  infinite sample size importance as computed with an infinite ensemble of fully
  developed totally randomized trees built on $V$ for $Y$ is 0.
\end{theorem}
\begin{proof}
	See \citep[Appendix D]{louppe2013understanding} for a proof.
\end{proof}

\begin{corollary}\label{cor:K1allrelevant}
	$Imp_{\infty,\infty}^{1,p}(X_m)  > 0 \mbox{ iff }X_m \in V \mbox{ is relevant with respect to } Y.$
\end{corollary} 
\begin{proof}
It directly stems from Theorem \ref{thm:mdi-totally-irrelevant}.
\end{proof}

\begin{lemma}{\citep[Lemma 4]{louppe2013understanding}}\label{lemma:mdi-adding-irrelevant}
  Let $X_i \notin V$ be an irrelevant variable for $Y$ with respect to $V$. The infinite
  sample size importance of $X_m \in V$ as computed with an infinite
  ensemble of fully developed totally randomized trees built on $V$ for $Y$ is the
  same as the importance derived when using $V\cup \{X_i\}$ to build the ensemble of trees for $Y$.
\end{lemma}
\begin{proof}
	See \citep[Appendix E]{louppe2013understanding} for a proof.
\end{proof}

\begin{theorem}{\citep[Theorem 5]{louppe2013understanding}}\label{thm:mdi-totally-relevant}
  Let $V_R \subseteq V$ be the subset of all variables in $V$ that are relevant with respect to $Y$. The infinite sample size importance of any variable $X_m \in
  V_R$ as computed with an infinite ensemble of fully developed totally randomized
  trees built on $V_R$ for $Y$ is the same as its importance computed in the same conditions by using all variables in $V$. That is:
    \begin{equation}
      \begin{aligned}
      Imp(X_m)&=\sum_{k=0}^{p-1} \frac{1}{C_p^k} \frac{1}{p-k} \sum_{B \in {\cal P}_k(V^{-m})} I(X_m;Y|B)\\
              &=\sum_{l=0}^{r-1} \frac{1}{C_r^l} \frac{1}{r-l} \sum_{B \in {\cal P}_l(V_R^{-m})} I(X_m;Y|B)\\
      \end{aligned}
    \end{equation}
  where $r$ is the number of relevant variables in $V_R$.
\end{theorem}
\begin{proof}
	See \citep[Appendix F]{louppe2013understanding} for a proof.
\end{proof}

Theorem~\ref{thm:mdi-totally-irrelevant} shows that only irrelevant features have a zero importance. They can thus be distinguished from relevant ones based solely on their importance scores. In addition, Lemma \ref{lemma:mdi-adding-irrelevant} points out that they do no affect the importance scores of relevant variables and the addition or the removal of irrelevant features have no effect which implies that only relevant features are required to compute importances (Theorem \ref{thm:mdi-totally-relevant}). Intuitively, splitting on an irrelevant feature $X_i$ instead of a relevant feature $X_m$ at node $t$ only postpones the attribution of the local importance of $X_m$ into the child nodes $t_L$ and $t_R$, but do not actually change its total importance. Indeed, on one hand, if $X_m$ was used at node $t$, then the local importance of $X_m$ would be proportional to $p(t)$ (i.e., $p(t)\Delta i(s,t) $). On the other hand, splitting on $X_i$ at node $t$ does not actually change the distribution of samples in $t_L$ and $t_R$. Therefore, splitting then on $X_m$ at $t_L$ and $t_R$ would provide the sum of local importances $p(t_L) \Delta i(s,t_L) + p(t_R) \Delta i(s,t_R)$. Given that $\Delta i(s,t)=\Delta i(s,t_L) = \Delta(i,t_R)$ because node sample distributions are unchanged by the split on $X_i$, we have that $(p(t_L) + p(t_R)) \Delta i(s,t) = p(t) \Delta i(s,t)$ which shows that splitting on $X_i$ first does not change anything. Similarly, one can recursively apply this reasoning if $X_m$ was used deeper in the tree (i.e., at descendant nodes of $t_L$ or $t_R$). Let us however note that this result may actually be due to the fact that total importance of a feature $X_m$ is the sum of all local importances in nodes where $X_m$ is used weighted by the number of samples reaching this node $p(t)$. \cite{louppe2014understanding} suggests that importances computed with another approach consisting in summing local importances over all nodes (e.g., using surrogate splits) would necessarily depend on the total number of nodes in a tree, which depends on the number of features $p$ and not only on the number of relevant features $r$.

In conclusion, in our opinion, theorems~\ref{thm:mdi-totally-irrelevant} and \ref{thm:mdi-totally-relevant} exhibit two desirable and sound properties for a feature importance measure.

\section{Impact of redundant variables}
\label{sec:mdi-red}
\subsection*{Setting of this section: $K=1, D=p, |s_t|=|v(s_t)|, N_T \rightarrow\infty, N\rightarrow\infty$}

Let us consider redundant variables as defined in Section \ref{sec:background-redundancy} and in particular totally redundant variables from Definition \ref{def:totalred}. In this section, we analyse  how feature importance scores are affected by the presence of (totally) redundant variables.

\begin{proposition}{\citep[Proposition 7.2]{louppe2014understanding}}\label{prop:mdi-red-Xm}
	Let $X_j\in V$ be a relevant variable with respect to $Y$ and $V$ and let
	$X_j^\prime \notin V$ be a totally redundant variable with respect to $X_j$.
	The infinite sample size importance of $X_j$ as computed with an infinite
	ensemble of fully developed totally randomized trees built on $V\cup
	\{X_j^\prime\}$ is
	\begin{equation}\label{eqn:mdi-red-Xm}
	\text{Imp}_{\infty,\infty}^{1,p}(X_j) = \sum_{k=0}^{p-1} \frac{p-k}{p+1} \frac{1}{C_p^k} \frac{1}{p-k} \sum_{B \in {\cal P}_k(V^{-j})} I(X_j;Y|B)
	\end{equation}
\end{proposition}
\begin{proof}
	See \citep[Page 147]{louppe2014understanding} for a proof.
\end{proof}

As observed in Theorem \ref{eqn:mdi-totally-sum}, the sum of all importance scores is equal to $I(X_1,\dots,X_p;Y)$. The addition of $X_j'$ does not actually modify $I(X_1,\dots,X_j,\dots,X_p;Y) $ which is equal to $I(X_1,\dots,X_j,X_j',\dots,X_p;Y)$\footnote{It can be shown by applying chain rule ($I(X_1,X_2,\dots,X_p;Y) = \sum_{i=1}^p I(X_i;Y|X_{i-1},\dots,X_1)$) $I(X_1,\dots,X_j,X_j',\dots,X_p;Y)$ while finishing by $X_j'$. Therefore, the last term is $I(X_j' ;Y|X_1,\dots,X_j,\dots,X_p)$ which is, by definition of total redundancy, equal to zero. Then by applying the chain rule backward, we obtain $I(X_1,\dots,X_j,\dots,X_p;Y)$.}. All importances, including those of non-redundant features, are therefore modified so that the sum of all importances remains the same

Equation \ref{eqn:MDIredin} shows that the importance of a variable decreases if it is totally redundant with other features. Indeed, the addition of a new feature increase the number of feature combinations $B$ and thus the number of terms ($I(X_j;Y|B)$) in the sum. This reflects in the weights of the outer sum of Equation \ref{eqn:mdi-red-Xm}. 
Indeed, all weights $\dfrac{1}{C_p^k (p-k)}$ are multiplied by a factor  $\dfrac{C^k_p(p-k)}{C^k_{p+1}} = \dfrac{p-k}{p+1} <1$ that updates weights to take into account the new feature, i.e. the ensemble of trees is now built on $p+1$ variables instead of $p$.
Mathematically, the importance of $X_i$ however decreases. By definition of total redundancy, $X_j$ becomes useless if $X_j'$ is given making all those new terms where $X_j'$ is included in $B$ equal to zero. Moreover, $X_j'$ does not either increase the information conveyed by $X_j$ about the target and thus all terms $I(X_j;Y|B)$ where $X_j'$ is not included in $B$ are unchanged.  
One may notice that the impact of the addition of a totally redundant feature is not simply a division of the original importance score of $X_j$ into $X_j$ and $X_j'$.

\begin{proposition}{\citep[Proposition 7.4]{louppe2014understanding}}\label{prop:mdi-red-Xj}
	Let $X_j\in V$ be a relevant variable with respect to $Y$ and $V$ and let
	$X_j^\prime \notin V$ be a totally redundant variable with respect to $X_j$.
	The infinite sample size importance of $X_l \in V^{-j}$ as computed with an infinite
	ensemble of fully developed totally randomized trees built on $V\cup
	\{X_j^\prime\}$ is
	\begin{align}\label{prop:red:other:eqn}
	\text{Imp}_{\infty,\infty}^{1,p}(X_l) &= \sum_{k=0}^{p-2} \frac{p-k}{p+1} \frac{1}{C_p^k} \frac{1}{p-k} \sum_{B \in {\cal P}_k(V^{-l} \setminus X_j)} I(X_l;Y|B) + \\
	& \hookrightarrow \sum_{k=0}^{p-2}  \left[ \sum_{k'=1}^2 \frac{C^{k'}_2}{C_{p+1}^{k+k'}} \frac{1}{p+1-(k+k')} \right]  \sum_{B \in {\cal P}_k(V^{-l}\setminus X_j)} I(X_l;Y|B\cup X_j). \nonumber
	\end{align}
\end{proposition}
\begin{proof}
	See \citep[Pages 148-149]{louppe2014understanding} for a proof.
\end{proof}

First, let us note that $X_j$ and $X_j'$ are identical and thus they can be used interchangeably or together without modifying the link between other features and the target. Mathematically, for any conditioning set $B$ and for any variable $X_l$, we have that $I(X_l;Y|B,X_i) = I(X_l;Y|B, X_j) = I(X_l;Y|B, X_i, X_j)$. It is the reason why Equation \ref{prop:red:other:eqn} is divided in two parts: those terms that do not involve either $X_j$ nor $X_j'$ and those whose $B$ necessarily includes $X_j$ (which is equivalent to include $X_j'$ or both).  

Equation \ref{prop:red:other:eqn} shows the impact on a non-redundant variable $X_j$. The first part concerns all $B$ made without $V^{-i}$. Corresponding conditional mutual information terms are decreased by a factor $$\dfrac{C^k_p (p-k)}{C^k_{p+1}(p+1-k)} = \dfrac{p-k}{(p+1)}\dfrac{1}{p+1-k} < 1.$$ 
Similarly to Equation \ref{eqn:mdi-red-Xm}, the first sub-factor updates weights to take into account the additional feature. 

The second part concerns all $B$ involving either $X_i$ or $X'_i$ and it shows that the corresponding conditional mutual information weights are accentuated, implying an increase of importances. Indeed, because of $I(X_l;Y|B,X_i) = I(X_l;Y|B, X_j) = I(X_l;Y|B, X_i, X_j)$, the same $B$ is actually taken into account several times (two more in this case). The net effect of those two parts on the importance of $X_j$ is a trade-off between those two antagonist effects which depends on the interaction of $X_j$ with $X_i$. Indeed, features that are positively affected by the presence of $X_i$, e.g. features such that $I(X_j;Y|X_i) > I(X_j;Y)$,  may end up with increased importances while importances of features that are either not or negatively impacted by $X_i$ will accordingly decrease (because the fixed value for the sum of all importances).

Without further proof, \cite{louppe2014understanding} extends Proposition \ref{prop:mdi-red-Xm} and \ref{prop:mdi-red-Xj} to consider the addition of $N_c$ totally redundant features with $X_j$ with respect to $Y$. Concretely, the effects given above are the same but amplified by the presence of $N_c$ totally redundant features instead of two. 

\begin{proposition}{\citep[Proposition 7.5]{louppe2014understanding}}\label{prop:red:general}
	Let $X_j\in V$ be a relevant variable with respect to $Y$ and $V$ and let
	$X_j^c \notin V$ (for $c=1,\dots,N_c$) be $N_c$ totally redundant variables with respect to $X_j$.
	The infinite sample size importances of $X_j$ and $X_l \in V$ as computed with an infinite
	ensemble of fully developed totally randomized trees built on $V\cup
	\{X_j^1,\dots,X_j^{N_c}\}$ are
	\begin{align*}
	\text{Imp}_{\infty,\infty}^{1,p}(X_j)  &= \sum_{k=0}^{p-1} \left[ \frac{C^k_p (p-k)}{C^k_{p+N_c}(p+N_c-k)}  \right] \frac{1}{C_p^k} \frac{1}{p-k} \sum_{B \in {\cal P}_k(V^{-j})} I(X_j;Y|B), \\
	\text{Imp}_{\infty,\infty}^{1,p}(X_l) &= \sum_{k=0}^{p-2} \left[ \frac{C^k_p (p-k)}{C^k_{p+N_c}(p+N_c-k)}  \right] \frac{1}{C_p^k} \frac{1}{p-k} \sum_{B \in {\cal P}_k(V^{-l} \setminus X_j)} I(X_l;Y|B) + \\
	& \hookrightarrow \sum_{k=0}^{p-2}  \left[ \sum_{k'=1}^{N_c+1} \frac{C^{k'}_{N_c+1}}{C_{p+N_c}^{k+k'}} \frac{1}{p+N_c-(k+k')} \right]  \sum_{B \in {\cal P}_k(V^{-l}\setminus X_j)} I(X_l;Y|B\cup X_j).
	\end{align*}
\end{proposition}

\section{Non-totally randomised trees}
\label{sec:mdi-K}
\subsection*{Setting of this section: $K>1, D=p, |s_t|=|v(s_t)|, N_T \rightarrow\infty, N\rightarrow\infty$}

In practice, random forest methods (e.g., Random Forest \citep{breiman2001random} or Extra-Trees \citep{geurts2006extremely}) are rarely built with $K=1$ because the growing procedure is then made independently of the data, and may lead to useless tree structures especially if the number of irrelevant features is large. Note that in the case of infinite ensemble size, and assuming that ties are broken deterministically, trees built with $K=p$ (i.e., the maximal value) amount to build classical single trees in a deterministic way.

In contrast with totally randomised trees (with $K=1$), masking effects may appear when trees are built with $K>1$. The masking effect denotes situations where several candidate splits on different features yield roughly the same impurity reduction, but one of the features is always slightly better so that none of the other ones has a chance to be selected by the tree-growing algorithm. Note that with multiway splits in particular, each feature is associated to one potential impurity decrease. Some variables may never be selected because some other variables always yield larger impurity decreases, and may thus be ``masked". Such effects tend to use first the best variables (in the sense of those yielding the largest impurity decrease at first) while pushing the least promising (i.e., yielding small impurity decreases in comparison to the best ones) towards the leaves. This implies that all feature combinations are no longer considered: best features are considered alone or conditioned only with the best others used before while the least promising ones are only considered conditioned on most of all other variables. As a result, some branches are never explored and the importance of a variable no longer decomposes into a sum including all $I(X_m;Y|B)$ terms. 

	To make things clearer, let us consider a simple example. Let $X_1$ be a variable that perfectly explains $Y$
and let $X_2$ be a slightly noisy copy of $X_1$ (i.e., $I(X_1;Y)
\approx I(X_2;Y)$, $I(X_1;Y|X_2)=\epsilon$ and $I(X_2;Y|X_1)=0$). Using totally
randomized trees, the importances of $X_1$ and $X_2$ are nearly equal -- the
importance of $X_1$ being slightly higher than the importance of $X_2$:
\begin{eqnarray*}
	Imp^{1,p}_{\infty,\infty}(X_1) &=& \frac{1}{2} I(X_1;Y) + \frac{1}{2} I(X_1;Y|X_2) = \frac{1}{2} I(X_1;Y) + \frac{\epsilon}{2}\\
	Imp^{1,p}_{\infty,\infty}(X_2) &=& \frac{1}{2} I(X_2;Y) + \frac{1}{2} I(X_2;Y|X_1) = \frac{1}{2} I(X_2;Y) + 0
\end{eqnarray*}
In non-totally randomized trees, for $K=2$, $X_1$ is always selected at the root
node and $X_2$ is always used in its children. Also, since $X_1$ perfectly
explains $Y$, all its children are pure and the reduction of entropy when
splitting on $X_2$ is null. As a result, $Imp^{K=2,p}_{\infty,\infty}(X_1) = I(X_1;Y)$ and
$Imp^{K=2,p}_{\infty,\infty}(X_2) = I(X_2;Y|X_1) = 0$. Masking effects are here
clearly visible: the true importance of $X_2$ is masked by $X_1$ as if $X_2$
were irrelevant, while it is only a bit less informative than $X_1$.
In the same
way, it can also be shown that the importances become dependent on the number of
irrelevant variables. Let us indeed consider the following example: let
us add in the previous example an irrelevant variable $X_i$ with respect to
$\{X_1, X_2\}$ and let us keep $K=2$. The probability of selecting $X_2$ at the
root node now becomes positive, which means that $Imp^{K=2,p}_{\infty,\infty}(X_2)$ now includes
$I(X_2;Y)>0$ and is therefore strictly larger than the importance computed
before. For $K$ fixed, adding irrelevant variables dampens masking effects,
which thereby makes importances indirectly dependent on the number of irrelevant
variables.

Consequently, non-totally randomised trees may be unable to identify all relevant features unlike totally randomised trees (see Corollary \ref{cor:K1allrelevant}). The following proposition however guarantees that all strongly relevant features will still be identified.

\begin{proposition}{\citep{sutera2018random}} \label{th:mdi:K:stronglypruned}
	\begin{eqnarray} \forall K,X_m\in V: \quad  X_m\mbox{ strongly relevant }\Rightarrow Imp^{K,p}_{\infty,\infty}(X_m) > 0. \nonumber \end{eqnarray}
\end{proposition}

\begin{proof}
	See proof of Theorem \ref{th:mdi-stronglyrelpruned} with the particular case of $q=p$.
\end{proof}

There is thus no masking effect possible for the strongly relevant features
when $K>1$. For a given $K$, the features found will thus include all strongly
relevant variables and some (when $K>1$) or all (when $K=1$) weakly relevant
ones. It is easy to show that increasing $K$ can only decrease the number of
weakly relevant variables found. Using $K=1$ will thus provide a solution for
the \textbf{all-relevant} problem, while increasing $K$ will provide a better
and better approximation of the \textbf{minimal-optimal} problem in the case of
strictly positive distributions (see Section \ref{sec:relevance}).

While strongly relevant variables can not be masked, their importances are not necessarily higher than the importances of weakly relevant variables, i.e., $X_i$ strongly relevant and $X_j$ weakly relevant does not imply that $Imp^{K,p}_{\infty,\infty}(X_i)\geq Imp^{K,p}_{\infty,\infty}(X_j)$. Example \ref{ex:stronglylowerthanweakly} illustrates this. Unfortunately, strongly relevant variables can thus not be distinguished from weakly relevant ones only using importances.
	
\begin{example}\label{ex:stronglylowerthanweakly}
Let us consider a problem defined by three binary input variables $X_1$,$X_2$ and $X_3$, and a binary output $Y$. 

\begin{figure}[H]
	\centering
	\begin{tikzpicture}[minimum size=2.3em]
	\node[circle,draw,RoyalBlue,line width=1pt] (a) at (0,0.5) {$X_1$};
	\node[circle,draw,orange,line width=1pt] (c) at (1.5,0) {$Y$};
	\node[circle,draw,RoyalBlue,line width=1pt] (b) at (0,-0.5) {$X_2$};
	\node[circle,draw,RoyalBlue,line width=1pt] (d) at (3,0) {$X_3$};
	\draw[->,black,line width=1pt] (a) -- (c);
	\draw[->,black,line width=1pt] (b) -- (c);
	\draw[->,black,line width=1pt] (c) -- (d);
	\node[] () at (1.5,1) {$Y=X_1\oplus X_2$}; 
	\node[] () at (2.2,-0.3) {$\alpha$}; 
	\end{tikzpicture}
\end{figure} 

The relationships between input and output variables are the following:
\begin{enumerate}[$\bullet$]
	\item $Y=X_1 \oplus X_2$ and $Y$ is therefore completely determined by $X_1$ and $X_2$;
	\item $X_3 = Y$ with probability $\alpha$ and its value is randomly chosen otherwise (i.e., $X_3=0$ with probability $(1-\alpha)/2$ and $X_3=1$ with probability $(1-\alpha)/2$).
\end{enumerate}
In this case, $X_1$ and $X_2$ are strongly relevant with respect to $Y$ while $X_3$ is only weakly relevant because it is useless when $X_1$ and $X_2$ are both known.

For $\alpha=0.8$, we can compute that $Imp(X_1) = Imp(X_2) = 0.296$ and $Imp(X_3) = 0.408$. Let us note that for small values of $\alpha$ (e.g., $\alpha=0.2$), $Imp(X_3)<Imp(X_1)=Imp(X_2)$.
  \end{example}

In conclusion, the importances as derived from trees with non-totally randomised split selection do not possess the same properties as those computed with totally randomised trees. The ability to identify all relevant features and the independence with respect to the addition or removal of irrelevant features are both lost. Asymptotically, the use of totally randomised trees seems more appropriate for assessing the importance of features. 

But in a finite setting  (i.e., a limited number of samples and a limited number of trees), $I(X_m;Y|B)$ terms are not all considered neither for all $X_m$ nor for all $B$, and/or need to be empirically estimated. Therefore, the use of non-totally randomised trees may help to focus on informative features providing better trees and splits on those features with more samples. Let us note that it could also be of interest in order to avoid useless splits on irrelevant features.  Assessing feature importances  with $K>1$ therefore remains a sound strategy in practice even if some features might be missed and the resulting importances may be biased.

\section{Non-fully developed trees}
\label{sec:mdi-pruning}
\subsection*{Setting of this section: $K>1, D=q\, (<p), |s_t|=|v(s_t)|, N_T \rightarrow\infty, N\rightarrow\infty$}

One key assumption of Theorem \ref{thm:mdi-totally-imp} was that all features are used once in every branch of the tree. However, when trees are no longer fully developed and say limited to a maximal depth $q$ ($<p$), all combinations are no longer explored and therefore we investigate in this section the ability of identifying relevant features with importance scores derived from pruned trees.

\begin{proposition}{\citep[Proposition 6]{louppe2013understanding}}\label{proposition:pruning}
	The importance of $X_m \in V$ for $Y$ as computed with an
	infinite ensemble of pruned totally randomized trees built up to depth $q \leq p$ and an
	infinitely large training sample is:
	\begin{equation}\label{eqn:imp-pruning}
	Imp_{\infty,\infty}^{1,q}(X_m)=\sum_{k=0}^{q-1} \frac{1}{C_p^k} \frac{1}{p-k} \sum_{B \in {\cal P}_k(V^{-m})} I(X_m;Y|B)
	\end{equation}
\end{proposition}
\begin{proof}
	See \citep[Appendix G]{louppe2013understanding} for a proof.
\end{proof}

\begin{proposition}{\citep[Proposition 7]{louppe2013understanding}}\label{proposition:imp-subspaces}
	The importance of $X_m \in V$ for $Y$ as computed with an infinite ensemble   of
	pruned totally randomized trees built up to depth $q \leq p$ and an infinitely
	large training sample is identical to the importance as computed  for $Y$ with an
	infinite ensemble of fully developed totally randomized trees built on random
	subspaces of $q$ variables drawn from $V$.
\end{proposition}
\begin{proof}
	See \citep[Appendix G]{louppe2013understanding} for a proof.
\end{proof}

Given Proposition
\ref{prop:mdi-only-relevant}, the degree of a variable $X$ can not be larger than
$r-1$ and thus as soon as $r\leq q$, we have the guarantee that all relevant variables can be identified with totally randomised trees ($K=1$). 

\begin{proposition}\label{th:mdi-relpruned}
	If $r\leq q$: \quad $Imp_{\infty,\infty}^{1,q}(X_m)>0 \mbox{ iff }X\mbox{ is relevant}.$
\end{proposition}

\begin{proof} Given Proposition \ref{prop:mdi-only-relevant}, for all, and only, the relevant variables, there exists at least one subset $B$ of size $|B|<r$ such that $I(X_m;Y|B)>0$. The proposition then follows from the fact that Equation \ref{eqn:imp-pruning} contains all conditional mutual information terms $I(X_m;Y|B)$ with $|B|<r$ when $r\leq q$.
\end{proof}

In the case of non-totally randomized trees ($K>1$), we lose the guarantee to
find all relevant variables even when $r\leq q$. Indeed, there is potentially a
masking effect due to $K>1$ that might prevent the conditioning needed for a
given variable to be relevant to appear in a tree branch.  However, we have the
following general result:

\begin{theorem}\label{th:mdi-stronglyrelpruned}
	$\forall K$, if $r\leq q$:\quad
	$X_m\mbox{ strongly relevant} \Rightarrow Imp_{\infty,\infty}^{K,q}(X_m)>0$ 
\end{theorem}
\begin{proof}
  
	By definition, $Imp_{\infty,\infty}^{K,q}(X_m)>0$ means that there is
	at least one tree (grown with parameters $q$ and $K$) in which $X_m$
	receives a strictly positive score for its split, i.e. such that $Y$ depends on
	$X_m$ conditionally to the variable assignment defined by the path from the root
	node to the node where $X_m$ is used to split. Let us show that one such tree always
	exists whatever $K$ when $X_m$ is strongly relevant and $r\leq q$.

        Within the infinite ensemble, let us consider only the trees such that irrelevant variables
	are tested in each branch only when all relevant variables (including $X_m$) are
	exhausted. These trees are always explored whatever the value of $K$. This
	derives from the fact that a relevant variable can always be picked with non
	zero probability at any tree node, except if all relevant variables have been
	tested above that node. Indeed, except in this latter case, the $K$ tested
	variables can always include at least one relevant variable. If some relevant
	variable gets a non zero score, one relevant variable will be automatically
	used to split since irrelevant variables can only get zero scores. Even when
	all tested relevant variables get a zero score, one of them can still be
	selected instead of an irrelevant one given that ties are resolved by
	randomisation.
        
	Let us denote by $\tau_R$ the set of trees as just defined and let us show that
	$X_m$ gets a non zero score in at least one tree in $\tau_R$.
	
	By definition of relevance and proposition \ref{prop:mdi-only-relevant}, $X_m$ strongly relevant implies that there exists
	at least one assignment of values to all relevant variables but $X_m$ such that
	conditionally to this assignment, $Y$ is dependent on $X_m$. In each tree in
	$\tau_R$, there is a path from the root node to a node where $X_m$ is used to
	split that is compatible with this assignment. Let us assume that $X_m$ always
	gets a zero score in all these compatible paths and show that this leads to a
	contradiction.
	
	If all relevant variables are tested above $X_m$ in a compatible path
        then $X_m$ should receive a non zero score at its node, which would
        contradict our hypothesis. Thus, $X_m$ can only be tested in a
        compatible path before all relevant variables have been tested. Given
        our hypothesis that $X_m$ only gets zero scores, if $X_m$ is used to
        split in one compatible path, then there exists another tree in
        $\tau_R$ with the same splits above $X_m$ in the compatible path and
        with the split on $X_m$ replaced by a split on another relevant
        variables (because of tie randomization or because of the randomisation
        due to the use of a $K<p$). In this new tree, $X_m$ is thus used to
        split at least one level below in the compatible path. Applying this
        argument recursively, one can thus show that there is at least one tree
        in $\tau_R$ where $X_m$ is the last variable used to split in the
        compatible path. In this tree, $X_m$ thus gets a non zero score, which
        contradicts the hypothesis and therefore concludes the theorem.
\end{proof}

There is thus no masking effect possible for the strongly relevant features
when $K>1$ as soon as the number of relevant features is lower than $q$.

When $q<r$, we do not have the guarantee any more to explore all minimal conditionings required to find all (strongly or not) relevant variables, whatever the values of $K$. We nevertheless still have the guarantee to find all (strongly) relevant variables of degree lower than $q$ (proofs are straightforward from proofs of Proposition \ref{th:mdi-relpruned} and Theorem \ref{th:mdi-stronglyrelpruned}):

\begin{proposition} \label{mdi:pop:deg:k1:relevantimp}\begin{eqnarray} 
	& X_m \in V\mbox{ relevant with respect to $Y$ and }  deg(X_m)<q \Rightarrow Imp^{1,q}_{\infty,\infty}(X_m) >0. \nonumber\end{eqnarray}
\end{proposition}

\begin{proposition}\label{mdi:pop:deg:stronglyrelevantnotmasked} \begin{eqnarray} \forall K: \quad  X_m \in V\mbox{ strongly relevant with respect to $Y$ and } deg(X_m)<q \Rightarrow Imp^{K,q}_{\infty,\infty} >0. \nonumber \end{eqnarray}
\end{proposition}

\section{Binary trees}
\label{sec:mdi-binary}

\subsection*{Setting of this section: $|s_t|=2, N_T \rightarrow\infty, N\rightarrow\infty$}

The last simplifying assumption on tree model is the number of nodes created when splitting a node. 
So far, we considered multiway trees (i.e., with exhaustive splits) where one branch was created for each value of the split variable. 
This way of growing trees allows to consider a variable in a branch only once and limits the maximal depth of a tree to the number of features. It also implies that once a variable is used for splitting in a node, all subsequent nodes have access to all the information (about the target) held by this variable. 

However, binary trees are most often used in practice. Instead of creating a branch per value, only two branches are created regardless of the cardinality of the split variable. 
The splitting rule is from now on of the form of a boolean condition (e.g., "\textit{less than a given threshold value}" or not, "\textit{in a subset of values}" or not) where samples verifying this condition go in one branch while the others necessarily go in the other branch.
As a consequence, a variable can now be used several times in a given tree branch, since a variable potentially only partially delivers its information at each split. There are also now several ways to split a node on the basis of a categorical variable of cardinality greater than two. When growing a tree, a binary split can be determined for such variable either by identifying among a set of predefined candidate binary splits the one that maximizes the impurity reduction (as in the standard Random Forests method) or by picking one binary split at random among these candidates (as in the Extra-Trees method).

As a consequence of these changes, one can not expect that Theorem~\ref{thm:mdi-totally-imp} and formula \ref{eqn:mdi-totally-imp} that were derived in the case of multiway trees will remain valid in the case of binary trees (except, trivially, if all variables are binary). And indeed, Example \ref{ex:mdi:bin}, taken from \cite{louppe2014understanding}, shows that importances computed from binary trees can be different from importances computed from multiway trees.

\begin{example}\label{ex:mdi:bin}
We present here the example as it is given in \citep{louppe2014understanding} and we refer the reader to the original source for more details on the exact computation of importance scores. Let us consider two ordered input variables of different cardinalities: $X_1$ is a ternary variable (i.e., its cardinality is $3$) and $X_2$ is a binary variable. The output variable $Y$ is defined as $Y=X_1<1$ and as a copy of $X_2$. The possible combinations of values are given in Table \ref{tab:mdi:bin:ex}.
	
	\begin{table}[h!]
		\centering
		\begin{tabular}{cc|c}
			$X_1$&$X_2$&$Y$\\ \hline
			$0$ & $0$ & $0$ \\
			$1$ & $1$ & $1$ \\
			$2$ & $1$ & $1$
		\end{tabular}
	\caption{Possible combinations of values for $X_1$,$X_2$ and $Y$.}
	\label{tab:mdi:bin:ex}
	\end{table}
	Only two totally randomised trees with multiway splits can be built from this setting as a single node split is sufficient to exhaust a variable of any cardinality (either $X_1$ or $X_2$) and to fully determine the output value. Importances derived from such trees (in asymptotic conditions) are as follows: 
	\begin{eqnarray*}
		Imp(X_1) &=& \dfrac{1}{2} I(X_1;Y) = \dfrac{1}{2} H(Y) = 0.459;\\
		Imp(X_2) &=& \dfrac{1}{2} I(X_2;Y) = \dfrac{1}{2} H(Y) = 0.459.
	\end{eqnarray*}
        Despite different trees, features are used in exactly half of the trees with the same usefulness and thus their importances are logically identical. Note that since both features perfectly explain the output, their importances do not depend on $K$.

        On the other hand, a binary split can not exhaust $X_1$ all at once. Using ordered binary splits, four possible decision trees can now be constructed. Assuming that the Extra-Trees split randomization is used and that $K$ is set to 1, the importances of $X_1$ and $X_2$ are respectively (in asymptotic conditions):
	\begin{eqnarray*}
		Imp(X_1) &=& \dfrac{1}{4} I(X_1\le 1;Y) + \dfrac{1}{8}P(X_1\le 1) I(X_1\le 0;Y|X_1\le 1) + \dfrac{1}{4}I(X_1\le0;Y)\\
		&=& 0.375,\\
		Imp(X_2) &=& \dfrac{1}{2}I(X_2;Y) + \dfrac{1}{8}P(X_1\le 1) I(X_2;Y|X_1\le 1)\\
		&=& 0.541,
	\end{eqnarray*}
which are strictly different from the importance scores derived from multiway splits. 
\end{example}

In this section, our aim is to revisit some of our previous results in the context of binary trees. In Section~\ref{sec:mdi:bin:splits}, we discuss different ways to generate binary splits for unordered and ordered categorical variables, focusing only on sets of candidate binary splits that are totally redundant with the original variable.
Example \ref{ex:mdi:bin} shows that importance scores computed with binary trees can be different from those computed with multiway trees. In Section \ref{sec:mdi:bin:relevance}, we show that the links between variable relevances and variable importances that were highlighted in Sections \ref{sec:mdi-rel} and \ref{sec:mdi-K} are preserved despite this difference. In Section \ref{sec:mdi:bin:experiments}, we illustrate further how binary splits influence variable importance scores on Breiman's digit recognition problem. 

\subsection{Binary splits}\label{sec:mdi:bin:splits}

As defined in Section \ref{sec:splittingrules}, binary splits may or may not take into account the value logic, i.e., a potential ordering between the values. An unordered split simply divides all the values into two disjoint sets, while an ordered split creates two partitions consisting of all the values that are respectively either lower or equal, or greater than a given threshold.

In the case of binary variables, both ways of splitting are strictly equivalent. In the case of variables of higher cardinality, they lead to different numbers of candidate splits. For example, there are only two possible ways of splitting a ternary variable of values $\{1,2,3\}$ while preserving the order (i.e., $(\{1\},\{2,3\})$ and $(\{1,2\},\{3\})$). By contrast, there are three possible ways of making two disjoint sets of values if the order is not taken into account (the split $(\{1,3\}$,$\{2\})$ being the additional binary partition that does not preserve the order). In general, a categorical variable of cardinality $m$ will lead to $2^{m-1}-1$ candidate unordered binary splits and to $m-1$ candidate ordered binary splits.

In addition to these two kinds of binary splits, let us also mention a third one based on the principle of "one value vs. all", where each binary split isolates one value of the variable in one branch and all the others in the other branch. In the case of a ternary variable, it provides the same candidate splits as the unordered binary splits (i.e., $(\{1\},\{2,3\})$, $(\{2\},\{1,3\})$, $(\{3\},\{1,2\})$) but for variables of higher cardinalities, less splits are considered than in the unordered case (see e.g., Figures \ref{fig:mdi:bin-oh:a} and \ref{fig:mdi:bin-unord:a}). For a variable of cardinality $m$, it leads to $m$ candidate binary splits.

All three ways of defining binary splits actually replace a categorical variable $X_m$ by a set of new binary variables, each one corresponding to a candidate binary split defined on $X_m$. Let us denote by $T_m=\{T_{m,1},\ldots,T_{m,|T_m|}\}$ the set of binary variables of size $|T_m|$ defined by one of these three families of binary splits. Figures \ref{fig:mdi:bin-ord:b} ,\ref{fig:mdi:bin-oh:b} and \ref{fig:mdi:bin-unord:b} illustrate the three sets of binary variables corresponding to the different ways of defining binary splits described above, and Figures \ref{fig:mdi:bin-ord:a}, \ref{fig:mdi:bin-oh:a} and \ref{fig:mdi:bin-unord:a} illustrate all possible splits, in the case of a quaternary variable $X_m$.

In all three cases, it is easy to show that $T_m$ and $X_m$ are totally redundant with respect to the target $Y$, i.e., mathematically (see Definition \ref{def:totalred}):
\begin{eqnarray}
\forall B\subseteq V^{-m},\quad X_m \indep Y | B\cup T_m \quad and \quad T_m \indep Y|B\cup\{X_m\}.
\end{eqnarray}
Thus, collectively, variables in $T_m$ convey the exact same information about the output as the original variable $X_m$ from which they are derived. There is thus no loss in information when replacing multiway splits with binary splits in all three cases. Note that in the case of unordered and one-value-vs-all splits, there are redundancy in $T_m$ in the sense that some variables can be removed from $T_m$ without impacting its total redundancy with $X_m$.

  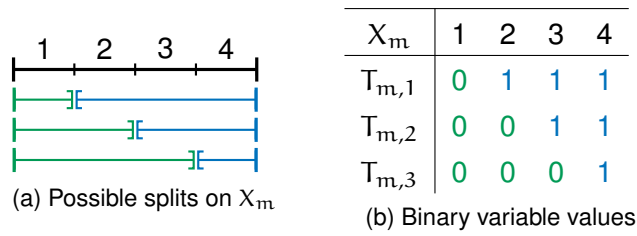
\begin{figure}[htbp]
  	\centering
	\subfloat[Possible splits on $X_m$\label{fig:mdi:bin-ord:a}]{

		\adjustbox{valign=c}{\begin{tikzpicture}[y=1cm, x=1cm,thick]
			\begin{scope}[scale=0.8]
			\draw[line width=1.2pt, -, ](0,0) -- (4,0) node[right] {}; 
			
			\foreach \x in {1,...,4} \draw (\x,0.1) -- (\x,-0.1) node[below] {}; 
			
			\foreach \x in {1,...,4} \node[above] at (\x-0.5,0) {\x}; 
			\foreach \x in {0,4} \draw[line width=1.2pt] (\x,0.2) -- (\x,-0.2); 
			\foreach \x in {0,1,2} \draw[ForestGreen,line width=1.2pt] (0.01,0.2-0.5-\x/2) -- (0.01,-0.2-0.5-\x/2); 
			\foreach \x in {0,1,2} \draw[RoyalBlue,line width=1.2pt] (3.99,0.2-0.5-\x/2) -- (3.99,-0.2-0.5-\x/2); 
			
			\foreach \x in {0} \draw[ForestGreen, [-|,] (\x+1-0.02,-0.5-\x/2) -- (0,-0.5-\x/2)  {};
			\foreach \x in {0} \draw[RoyalBlue, [-|,] (\x+1+0.02,-0.5-\x/2) -- (4,-0.5-\x/2)  {};
			
			\foreach \x in {1} \draw[ForestGreen, [-|,] (\x+1-0.02,-0.5-\x/2) -- (0,-0.5-\x/2)  {};
			\foreach \x in {1} \draw[RoyalBlue, [-|,] (\x+1+0.02,-0.5-\x/2) -- (4,-0.5-\x/2)  {};
			
			\foreach \x in {2} \draw[ForestGreen, [-|,] (\x+1-0.02,-0.5-\x/2) -- (0,-0.5-\x/2)  {};
			\foreach \x in {2} \draw[RoyalBlue, [-|,] (\x+1+0.02,-0.5-\x/2) -- (4,-0.5-\x/2)  {};

			\end{scope}
			\end{tikzpicture}}

	}
	\qquad
	\subfloat[Binary variable values\label{fig:mdi:bin-ord:b}]{\begin{tabular}{c | c c c c} \hline
	$X_m$ & 1 & 2 & 3& 4\\ \hline
	$T_{m,1}$ & \textcolor{ForestGreen}{0} & \textcolor{RoyalBlue}{1} & \textcolor{RoyalBlue}{1} & \textcolor{RoyalBlue}{1} \\
	$T_{m,2}$ & \textcolor{ForestGreen}{0} & \textcolor{ForestGreen}{0} & \textcolor{RoyalBlue}{1} & \textcolor{RoyalBlue}{1} \\
	$T_{m,3}$ & \textcolor{ForestGreen}{0} & \textcolor{ForestGreen}{0} & \textcolor{ForestGreen}{0} & \textcolor{RoyalBlue}{1} \\
	\end{tabular}\hspace{1em}}
	\caption{Set $T_{num}$ of binary variables corresponding to possible ordered splits, i.e. between two successive values of $X_m$. Each colour is associated to one of the two branches leaving the node after the spit. For instance, intervals of values in green correspond to the left branch whereas intervals in blue correspond to the right one.}
	\label{fig:mdi:bin-ord}
\end{figure}

  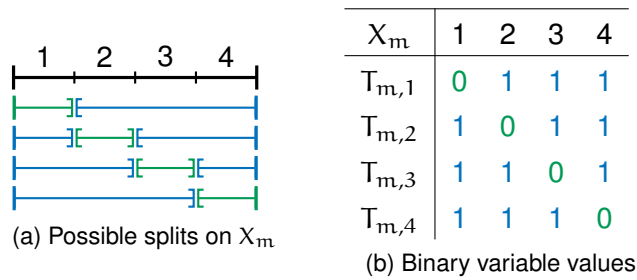
\begin{figure}[htbp]
	\centering
	\subfloat[Possible splits on $X_m$\label{fig:mdi:bin-oh:a}]{

		\adjustbox{valign=c}{\begin{tikzpicture}[y=1cm, x=1cm,thick]
			\begin{scope}[scale=0.8]
			\draw[line width=1.2pt, -, ](0,0) -- (4,0) node[right] {}; 
			
			\foreach \x in {1,...,4} \draw (\x,0.1) -- (\x,-0.1) node[below] {}; 
			
			\foreach \x in {1,...,4} \node[above] at (\x-0.5,0) {\x}; 
			\foreach \x in {0,4} \draw[line width=1.2pt] (\x,0.2) -- (\x,-0.2); 
	
			\foreach \x in {0} \draw[ForestGreen,line width=1.2pt] (0.01,0.2-0.5-\x/2) -- (0.01,-0.2-0.5-\x/2); 
			\foreach \x in {1,2,3} \draw[RoyalBlue,line width=1.2pt] (0.01,0.2-0.5-\x/2) -- (0.01,-0.2-0.5-\x/2); 
			\foreach \x in {0,1,2} \draw[RoyalBlue,line width=1.2pt] (3.99,0.2-0.5-\x/2) -- (3.99,-0.2-0.5-\x/2); 
			\foreach \x in {3} \draw[ForestGreen,line width=1.2pt] (3.99,0.2-0.5-\x/2) -- (3.99,-0.2-0.5-\x/2); 

			\foreach \x in {0} \draw[ForestGreen, [-|,] (\x+1-0.02,-0.5-\x/2) -- (0,-0.5-\x/2)  {};
			\foreach \x in {0} \draw[RoyalBlue, [-|,] (\x+1+0.02,-0.5-\x/2) -- (4,-0.5-\x/2)  {};
			
			\foreach \x in {1} \draw[ForestGreen, {[}-{]},] (\x+1-0.02,-0.5-\x/2) -- (\x+0.02,-0.5-\x/2)  {};
			\foreach \x in {1} \draw[RoyalBlue, [-|,] (\x+1+0.02,-0.5-\x/2) -- (4,-0.5-\x/2)  {};
			\foreach \x in {1} \draw[RoyalBlue, [-|,] (\x-0.02,-0.5-\x/2) -- (0,-0.5-\x/2)  {};

			\foreach \x in {2} \draw[ForestGreen, {[}-{]},] (\x+1-0.02,-0.5-\x/2) -- (\x+0.02,-0.5-\x/2)  {};
			\foreach \x in {2} \draw[RoyalBlue, [-|,] (\x+1+0.02,-0.5-\x/2) -- (4,-0.5-\x/2)  {};
			\foreach \x in {2} \draw[RoyalBlue, [-|,] (\x-0.02,-0.5-\x/2) -- (0,-0.5-\x/2)  {};

			\foreach \x in {3} \draw[ForestGreen, |-{]},] (\x+1,-0.5-\x/2) -- (\x+0.02,-0.5-\x/2)  {};
			\foreach \x in {3} \draw[RoyalBlue, [-|,] (\x-0.02,-0.5-\x/2) -- (0,-0.5-\x/2)  {};

			\end{scope}
			\end{tikzpicture}}
	}
	\qquad
	\subfloat[Binary variable values\label{fig:mdi:bin-oh:b}]{\begin{tabular}{c | c c c c} \hline
			$X_m$ & 1 & 2 & 3& 4\\ \hline
			$T_{m,1}$ & \textcolor{ForestGreen}{0} & \textcolor{RoyalBlue}{1} & \textcolor{RoyalBlue}{1} & \textcolor{RoyalBlue}{1} \\
			$T_{m,2}$ & \textcolor{RoyalBlue}{1} & \textcolor{ForestGreen}{0} & \textcolor{RoyalBlue}{1} & \textcolor{RoyalBlue}{1} \\
			$T_{m,3}$ & \textcolor{RoyalBlue}{1} & \textcolor{RoyalBlue}{1} & \textcolor{ForestGreen}{0} & \textcolor{RoyalBlue}{1} \\
			$T_{m,4}$ & \textcolor{RoyalBlue}{1} & \textcolor{RoyalBlue}{1} & \textcolor{RoyalBlue}{1} & \textcolor{ForestGreen}{0} \\
			
	\end{tabular}\hspace{1em}}
	\caption{Set $T_{oh}$ of binary variables corresponding to possible unordered "one value vs. all" splits, i.e. one-hot encoding of values of $X_m$. Each colour is associated to one of the two branches leaving the node after the spit. For instance, intervals of values in green correspond to the left branch whereas intervals in blue correspond to the right one.}
	\label{fig:mdi:bin-oh}
\end{figure}

  \begin{figure}[htbp]
	\centering
	\subfloat[Possible splits on $X_m$\label{fig:mdi:bin-unord:a}]{

		\adjustbox{valign=c}{\begin{tikzpicture}[y=1cm, x=1cm,thick]
			\begin{scope}[scale=0.8]
			\draw[line width=1.2pt, -, ](0,0) -- (4,0) node[right] {}; 
			
			\foreach \x in {1,...,4} \draw (\x,0.1) -- (\x,-0.1) node[below] {}; 
			
			\foreach \x in {1,...,4} \node[above] at (\x-0.5,0) {\x}; 
			\foreach \x in {0,4} \draw[line width=1.2pt] (\x,0.2) -- (\x,-0.2); 
			\foreach \x in {0,1,2} \draw[ForestGreen,line width=1.2pt] (0.01,0.2-0.5-\x/2) -- (0.01,-0.2-0.5-\x/2); 
			\foreach \x in {0,1,2} \draw[RoyalBlue,line width=1.2pt] (3.99,0.2-0.5-\x/2) -- (3.99,-0.2-0.5-\x/2); 
			
			\foreach \x in {0} \draw[RoyalBlue, |-{]},]   (0,-0.5-\x/2) -- (1-0.02,-0.5-\x/2)  {};
			\foreach \x in {0} \draw[ForestGreen, [-|,] 		  (1+0.02,-0.5-\x/2) -- (4,-0.5-\x/2)  {};
			
			\foreach \x in {1} \draw[ForestGreen, |-{]},]   (0,-0.5-\x/2) -- (1-0.02,-0.5-\x/2)  {};
			\foreach \x in {1} \draw[RoyalBlue, {[}-{]},]   (1+0.02,-0.5-\x/2) -- (2-0.02,-0.5-\x/2)  {};
			\foreach \x in {1} \draw[ForestGreen, {[}-|,]   (2+0.02,-0.5-\x/2) -- (4,-0.5-\x/2)  {};

			\foreach \x in {2} \draw[RoyalBlue, |-{]},]   (0,-0.5-\x/2) -- (2-0.02,-0.5-\x/2)  {};
			\foreach \x in {2} \draw[ForestGreen, {[}-|,]   (2+0.02,-0.5-\x/2) -- (4,-0.5-\x/2)  {};
			
			\foreach \x in {3} \draw[ForestGreen, |-{]},]   (0,-0.5-\x/2) -- (2-0.02,-0.5-\x/2)  {};
			\foreach \x in {3} \draw[RoyalBlue, {[}-{]},]   (2+0.02,-0.5-\x/2) -- (3-0.02,-0.5-\x/2)  {};
			\foreach \x in {3} \draw[ForestGreen, {[}-|,]   (3+0.02,-0.5-\x/2) -- (4,-0.5-\x/2)  {};

			\foreach \x in {4} \draw[RoyalBlue, |-{]},]   (0,-0.5-\x/2) -- (1-0.02,-0.5-\x/2)  {};
			\foreach \x in {4} \draw[ForestGreen, {[}-{]},]   (1+0.02,-0.5-\x/2) -- (2-0.02,-0.5-\x/2)  {};
			\foreach \x in {4} \draw[RoyalBlue, {[}-{]},]   (2+0.02,-0.5-\x/2) -- (3-0.02,-0.5-\x/2)  {};
			\foreach \x in {4} \draw[ForestGreen, {[}-|,]   (3+0.02,-0.5-\x/2) -- (4,-0.5-\x/2)  {};

			\foreach \x in {5} \draw[ForestGreen, |-{]},]   (0,-0.5-\x/2) -- (1-0.02,-0.5-\x/2)  {};
			\foreach \x in {5} \draw[RoyalBlue, {[}-{]},]   (1+0.02,-0.5-\x/2) -- (3-0.02,-0.5-\x/2)  {};
			\foreach \x in {5} \draw[ForestGreen, {[}-|,]   (3+0.02,-0.5-\x/2) -- (4,-0.5-\x/2)  {};

			\foreach \x in {6} \draw[RoyalBlue, |-{]},]   (0,-0.5-\x/2) -- (3-0.02,-0.5-\x/2)  {};
			\foreach \x in {6} \draw[ForestGreen, {[}-|,]   (3+0.02,-0.5-\x/2) -- (4,-0.5-\x/2)  {};
			
			\foreach \x in {7} \draw[ForestGreen, |-{]},]   (0,-0.5-\x/2) -- (3-0.02,-0.5-\x/2)  {};
			\foreach \x in {7} \draw[RoyalBlue, {[}-|,]   (3+0.02,-0.5-\x/2) -- (4,-0.5-\x/2)  {};

			\foreach \x in {8} \draw[ForestGreen, {[}-{]},]   (1+0.02,-0.5-\x/2) -- (3-0.02,-0.5-\x/2)  {};
			\foreach \x in {8} \draw[RoyalBlue, |-{]},]   (0,-0.5-\x/2) -- (1-0.02,-0.5-\x/2)  {};
			\foreach \x in {8} \draw[RoyalBlue, [-|,] 		  (3+0.02,-0.5-\x/2) -- (4,-0.5-\x/2)  {};
			
			\foreach \x in {9} \draw[ForestGreen, |-{]},]   (0,-0.5-\x/2) -- (1-0.02,-0.5-\x/2)  {};
			\foreach \x in {9} \draw[ForestGreen, {[}-{]},]   (2+0.02,-0.5-\x/2) -- (3-0.02,-0.5-\x/2)  {};
			\foreach \x in {9} \draw[RoyalBlue, {[}-{]},]   (1+0.02,-0.5-\x/2) -- (2-0.02,-0.5-\x/2)  {};
			\foreach \x in {9} \draw[RoyalBlue, [-|,] 		  (3+0.02,-0.5-\x/2) -- (4,-0.5-\x/2)  {};
			
			\foreach \x in {10} \draw[ForestGreen, {[}-{]},]   (2+0.02,-0.5-\x/2) -- (3-0.02,-0.5-\x/2)  {};
			\foreach \x in {10} \draw[RoyalBlue, |-{]},]   (0,-0.5-\x/2) -- (2-0.02,-0.5-\x/2)  {};
			\foreach \x in {10} \draw[RoyalBlue, [-|,] 		  (3+0.02,-0.5-\x/2) -- (4,-0.5-\x/2)  {};
			
			\foreach \x in {11} \draw[ForestGreen, |-{]},]   (0,-0.5-\x/2) -- (2-0.02,-0.5-\x/2)  {};
			\foreach \x in {11} \draw[RoyalBlue, {[}-|,]   (2+0.02,-0.5-\x/2) -- (4,-0.5-\x/2)  {};

			\foreach \x in {12} \draw[ForestGreen, {[}-{]},]   (1+0.02,-0.5-\x/2) -- (2-0.02,-0.5-\x/2)  {};
			\foreach \x in {12} \draw[RoyalBlue, |-{]},]   (0,-0.5-\x/2) -- (1-0.02,-0.5-\x/2)  {};
			\foreach \x in {12} \draw[RoyalBlue, {[}-|,]   (2+0.02,-0.5-\x/2) -- (4,-0.5-\x/2)  {};

			\foreach \x in {13} \draw[ForestGreen, |-{]},]   (0,-0.5-\x/2) -- (1-0.02,-0.5-\x/2)  {};
			\foreach \x in {13} \draw[RoyalBlue, {[}-|,]   (1+0.02,-0.5-\x/2) -- (4,-0.5-\x/2)  {};

			\end{scope}
			\end{tikzpicture}}
	}
	\qquad
	\subfloat[Binary variable values.\label{fig:mdi:bin-unord:b}]{\begin{tabular}{c | c c c c} \hline
			$X_m$ & 1 & 2 & 3& 4\\ \hline
			$T_{m,1}$ & \textcolor{RoyalBlue}{1} & \textcolor{ForestGreen}{0} & \textcolor{ForestGreen}{0} & \textcolor{ForestGreen}{0}\\
			$T_{m,2}$ & \textcolor{ForestGreen}{0} & \textcolor{RoyalBlue}{1} & \textcolor{ForestGreen}{0} & \textcolor{ForestGreen}{0}\\
			$T_{m,3}$ & \textcolor{RoyalBlue}{1} & \textcolor{RoyalBlue}{1} & \textcolor{ForestGreen}{0} & \textcolor{ForestGreen}{0}\\
			$T_{m,4}$ & \textcolor{ForestGreen}{0} & \textcolor{ForestGreen}{0} & \textcolor{RoyalBlue}{1} & \textcolor{ForestGreen}{0}\\
			$T_{m,5}$ & \textcolor{RoyalBlue}{1} & \textcolor{ForestGreen}{0} & \textcolor{RoyalBlue}{1} & \textcolor{ForestGreen}{0}\\
			$T_{m,6}$ & \textcolor{ForestGreen}{0} & \textcolor{RoyalBlue}{1} & \textcolor{RoyalBlue}{1} & \textcolor{ForestGreen}{0}\\
			$T_{m,7}$ & \textcolor{RoyalBlue}{1} & \textcolor{RoyalBlue}{1} & \textcolor{RoyalBlue}{1} & \textcolor{ForestGreen}{0}\\
	\end{tabular}\hspace{1em}
\begin{tabular}{c | c c c c} \hline
	$X_m$ & 1 & 2 & 3& 4\\ \hline
			$T_{m,8}$ & \textcolor{ForestGreen}{0} & \textcolor{ForestGreen}{0} & \textcolor{ForestGreen}{0} & \textcolor{RoyalBlue}{1}\\
			$T_{m,9}$ & \textcolor{RoyalBlue}{1} & \textcolor{ForestGreen}{0} & \textcolor{ForestGreen}{0} & \textcolor{RoyalBlue}{1}\\
			$T_{m,10}$ & \textcolor{ForestGreen}{0} & \textcolor{RoyalBlue}{1} & \textcolor{ForestGreen}{0} & \textcolor{RoyalBlue}{1}\\
			$T_{m,11}$ & \textcolor{RoyalBlue}{1} & \textcolor{RoyalBlue}{1} & \textcolor{ForestGreen}{0} & \textcolor{RoyalBlue}{1}\\
			$T_{m,12}$ & \textcolor{ForestGreen}{0} & \textcolor{ForestGreen}{0} & \textcolor{RoyalBlue}{1} & \textcolor{RoyalBlue}{1}\\
			$T_{m,13}$ & \textcolor{RoyalBlue}{1} & \textcolor{ForestGreen}{0} & \textcolor{RoyalBlue}{1} & \textcolor{RoyalBlue}{1}\\
			$T_{m,14}$ & \textcolor{ForestGreen}{0} & \textcolor{RoyalBlue}{1} & \textcolor{RoyalBlue}{1} & \textcolor{RoyalBlue}{1}\\
\end{tabular}}
	\caption{Set $T_{bp}$ of binary variables corresponding to possible unordered splits, i.e. all binary partitions of values of $X_m$. Each colour is associated to one of the two branches leaving the node after the spit. For instance, intervals of values in green correspond to the left branch whereas intervals in blue correspond to the right one.}
	\label{fig:mdi:bin-unord}
\end{figure}
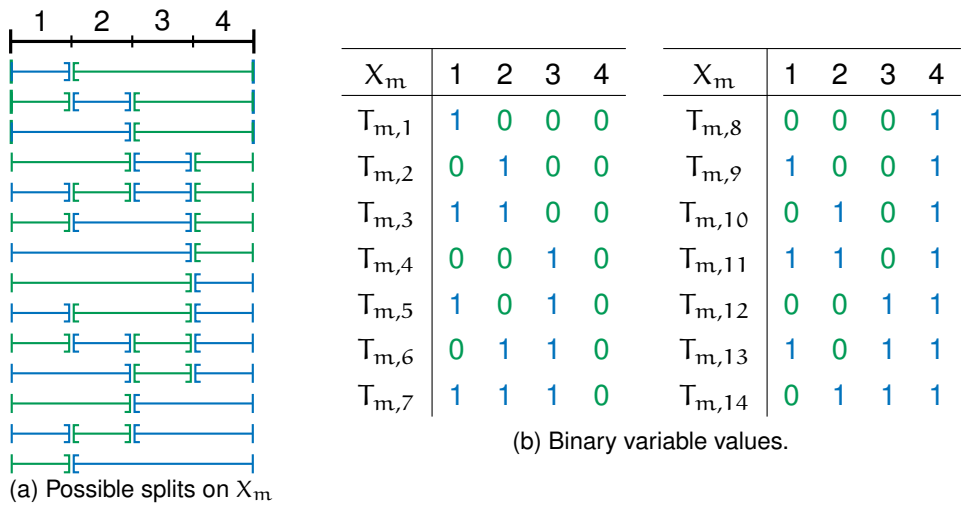

\subsection{Relevance in binary trees} \label{sec:mdi:bin:relevance}

In this section, we assume that a binary tree is grown from a set of categorical variables using the Extra-Trees split randomization, i.e., by randomly selecting $K$ variables at each node, picking for each of them a random binary split in its set of candidate binary splits, and finally using the split among $K$ that leads to the most important decrease of impurity (breaking ties at random). The importance of a variable $X_m$ is then obtained by summing total impurity reductions at all nodes where a binary split has been performed on $X_m$.

In this setting, we would like first to check whether Theorem \ref{thm:mdi-totally-irrelevant}, stating that a variable is irrelevant if and only if its infinite sample size importance is 0, remains valid when using (fully developed totally randomized) binary trees instead of multiway ones.

Let us denote by $X_m$ a categorical variables of cardinality greater than 2 and by $T_m$ a set of totally redundant binary variables corresponding to the candidate binary splits used for this variable during tree growing. The following theorem first shows that $X_m$ is relevant if and only if at least one variable in $T_m$ is relevant.
\begin{proposition} \label{sec:mdi:bin:selfrelevance}
  Let $X_m\in V$ be an input variable and let $T_m=\{T_{m,1},\ldots,T_{m,|T_m|}\}$ be a set of binary variables such that $X_m$ and $T_m$ are totally redundant with respect to $Y$. There exists a subset $B\subseteq V^{-m}$ such that $I(X_m;Y|B)>0$ if and only if there exists a subset $B\subseteq V^{-m}$, a variable $T_{m,j}\in T_m$ and a subset $T'\subseteq T_m\setminus\{T_{m,j}\}$ such that $I(T_j;Y|B\cup{T'})>0$. 
\end{proposition}
\begin{proof}
  \textbf{Necessary condition}: 
  ($\exists B: I(X_m;Y|B) > 0 \Rightarrow \exists B, T_{m,j}, T': I(T_{m,j};Y|B\cup T')>0$)
  
	As a consequence of the total redundancy between $T_m$ and $X_m$, we directly have that $$I(T_m;X_m|B) = I(X_m;Y|B) > 0.$$ Applying the chain rule on $I(T_m;X|B) = I(T_{m,1},\dots,T_{m,|T_m|};Y|B)$, we have that $$I(T_m;X_m|B)=\sum_{i=1}^{|T_m|} I(T_{m,i};Y|B\cup\{T_{m,1},\dots,T_{m,i-1}\}) >0$$ which implies that a least one term of the sum should be strictly positive. That is, $$\exists T_{m,j}:\, I(T_{m,j};Y|B\cup T') >0.$$
	where $T'=\{T_{m,1},\dots,T_{m,j-1}\}$.\\[2ex]

	\textbf{Sufficient condition}:  ($\exists B, T_{m,j}, T': I(T_{m,j};Y|B\cup T')>0 \Rightarrow \exists B: I(X_m;Y|B) > 0$)

        Given $I(T_{m,j};Y|B\cup T')>0  $, the proof is a direct consequence of the chain rule where variables in $T'$ are used first and then $T_{m,j}$. Indeed, $$I(T_m;Y|B)=\sum_{j=1}^{|T_m|} I(T_{m,j};Y|B\cup\{T_{m,1},\dots,T_{m,j-1}\})$$ is therefore necessarily strictly positive and thus $$I(T_m;Y|B) = I(X_m;Y|B)>0$$ by total redundancy. 
\end{proof}

The following proposition further shows that variables whose relevance is conditioned on $X_m$ will remain relevant conditionally to some variables in a totally redundant set $T_m$.

\begin{proposition}\label{prop:mdi:bin:otherelevance}
	For any relevant variable $X_i\in V^{-m}$ with respect to $Y$, there exists a subset $B$ such that $I(X_i;Y|B\cup X_m)>0$ if and only if there exists a subset $T' \subseteq T$ such that $I(X_i;Y|B\cup T')$ where $T$ is a set of binary variables which is totally redundant with $X_m$ with respect to $Y$.
\end{proposition}
\begin{proof}
	The proof is a direct consequence of the total redundancy between $X_m$ and $T_m$.
\end{proof}

Propositions \ref{sec:mdi:bin:selfrelevance} and \ref{prop:mdi:bin:otherelevance} can be combined to show that Theorem \ref{thm:mdi-totally-irrelevant} remains valid in the case of fully developed totally randomized binary trees, when candidate binary splits are totally redundant with the original variables.

\begin{theorem}\label{thm:mdi-totally-irrelevant-binary}
Let us assume binary trees constructed by using totally redundant candidate binary splits and the Extra-Trees split randomization. Then, $X_i\in V$ is irrelevant to $Y$ with respect to $V$ if and only if its infinite sample size importance as computed with an infinite ensemble of fully developed totally randomized binary trees built on $V$ for $Y$ is 0.
\end{theorem}
We do not provide a formal proof of this theorem to not overload the text. Intuitively, the theorem can be proven by noting that when $K=1$ and with split randomization, the importance of a variable $X_m$ is a weighted sum of all possible terms $I(T_{m,i};Y|B)$, where $T_{m,i}$ is a binary split based on $X_m$ and $B$ is a subset of binary splits defined on all features (including $X_m$). Given Propositions \ref{sec:mdi:bin:selfrelevance} and \ref{prop:mdi:bin:otherelevance}, at least one such term is strictly positive if and only if $X_m$ is relevant.

In the case of multiway trees, Proposition \ref{th:mdi:K:stronglypruned} shows that strongly relevant variables will be always found whatever the value of $K$. A similar result can be shown in the case of binary trees.

Let us first characterize the relevance of binary variables in $T_m$ with respect to the relevance of $X_m$. 
The following corollary of Proposition \ref{sec:mdi:bin:selfrelevance} first shows that if $X_m$ is only weakly relevant, no variable in a totally redundant set $T_m$ can be strongly relevant.
\begin{corollary}\label{cor:mdi:bin:weakrelevance}
If $X_m$ is weakly relevant with respect to $Y$, then each $T_{m,j}\in T_m$, with $X_m$ and $T_m$ totally redundant with respect to $Y$, is either irrelevant or weakly relevant with respect to $Y$.
\end{corollary}
\begin{proof}
	The relevance of some $T_{m,j} \in T_m$ directly results from Proposition \ref{sec:mdi:bin:selfrelevance}. No $T_{m,j}$ can however be strongly relevant. Indeed, if $X_m$ is weakly relevant with respect to $Y$, we have that $X_m\indep Y|V^{-m}$ which is equivalent to $T_m \indep Y|V^{-m} \Leftrightarrow T_{m,1},\dots,T_{m,|T_m|}\indep Y | V^{-m}$, given the total redundancy between $T_m$ and $X_m$. By weak union, the latter independence implies that:  $$T_i\indep Y |V^{-m} \cup T^{-i}$$ for all $T_i\in T$.
\end{proof}

$X_m$ strongly relevant does not ensure that a variable in a totally redundant set $T_m$ will be strongly relevant (Surely, this can not be the case if $T_m$ contains redundant features), which would have sufficed to show that Proposition \ref{th:mdi:K:stronglypruned} remains valid for binary trees. However, the following results show that at least one $T_{m,i}\in T_m$ can not be masked by variables in $V^{-m}$.

\begin{proposition}
	Let $X_m$ be a strongly relevant variable with respect to $Y$ and let $T_m=\{T_{m,1},\ldots,T_{m,|T_m|}\}$ be a set of binary variables such that $X_m$ and $T_m$ are totally redundant with respect to $Y$. There exists at least one variable $T_{m,i}\in T_m$ such that $T_{m,i}\nindep Y|V^{-m}\cup T'$ for at least one subset $T'\subseteq T_m\setminus\{T_{m,i}\}$.
\end{proposition}
\begin{proof}

  Let us assume that one such $T_{m,i}$ does not exist and show that this leads to a contradiction. For all $T_{m,i}\in T_m$ and all $T'\subseteq T_m \setminus\{T_{m,i}\}$ (possibly empty), we thus have $T_{m,i} \indep Y | V^{-m} \cup T'$.
  Let us consider any ordering of the variables in $T_m$ and let us recursively apply the contraction property. We then have the following sequence of independences: $T_{m,1}\indep Y | V^{-m}$, $T_{m,2}\indep Y | V^{-m}\cup T_{m,1}$ and $T_{m,1} \indep Y | V^{-m}$ gives $\{T_{m,1},T_{m,2}\} \indep Y | V^{—m}$, $\dots$, $T_{m,|T_m|} \indep Y |V^{-m} \cup \{T_{m,1},\dots,T_{m,|T_m|-1}\}$ and $\{T_{m,1},\ldots,T_{m,|T_m|-1}\}\indep Y |V^{—m}$ gives $\{T_{m,1},\dots,T_{m,|T_m|}\}\indep Y |V^{-m}$. The latter independence is impossible because of the strong relevance of $X_m$ that implies that $X_m\nindep Y|V^{-m}$ and thus $T_m\nindep Y|V^{-m}$, given that $T_m$ and $X_m$ are totally redundant.
	\end{proof}

Using this result, one can adapt the proof of Proposition \ref{th:mdi:K:stronglypruned} in a straightforward way to show the following result (provided without proof):

\begin{proposition}\label{th:mdi:K:stronglypruned_binary}
  Let us assume binary trees constructed by using totally redundant candidate binary splits and the Extra-Trees split randomization.
	\begin{eqnarray} \forall K,X_m\in V: \quad  X_m\mbox{ strongly relevant }\Rightarrow Imp^{K,p}_{\infty,\infty}(X_m) > 0. \nonumber \end{eqnarray}
\end{proposition}

Theorem \ref{thm:mdi-totally-irrelevant-binary} and Proposition \ref{th:mdi:K:stronglypruned_binary} thus show that using binary instead of multiway splits fortunately does not affect the ability of variable importances to identify the relevant features and filter out the irrelevant ones.

\subsection{Importance scores in binary trees}\label{sec:mdi:bin:experiments}

Through Example \ref{ex:mdi:bin}, we already know that importance scores are expected to be different in binary trees compared to multiway trees. In this section, we further illustrate this difference in more details by computing variable importance scores in various settings on \citet{breiman1984classification}'s digit recognition problem (see Appendix \ref{app:digit} for a description of this problem).

\begin{figure}[htbp]
	
  \subfloat[$K=1$\label{fig:mdi:bin:k1}]{\includegraphics[width=\linewidth]{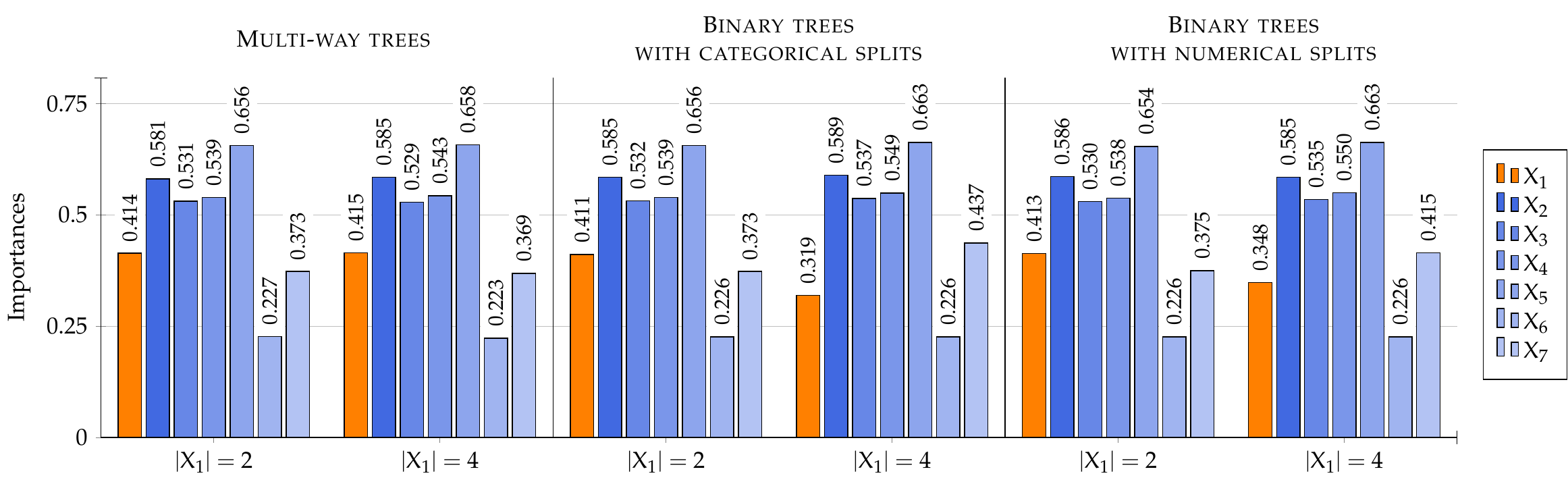}}\\
  \subfloat[$K=p$\label{fig:mdi:bin:kp}]{\includegraphics[width=\linewidth]{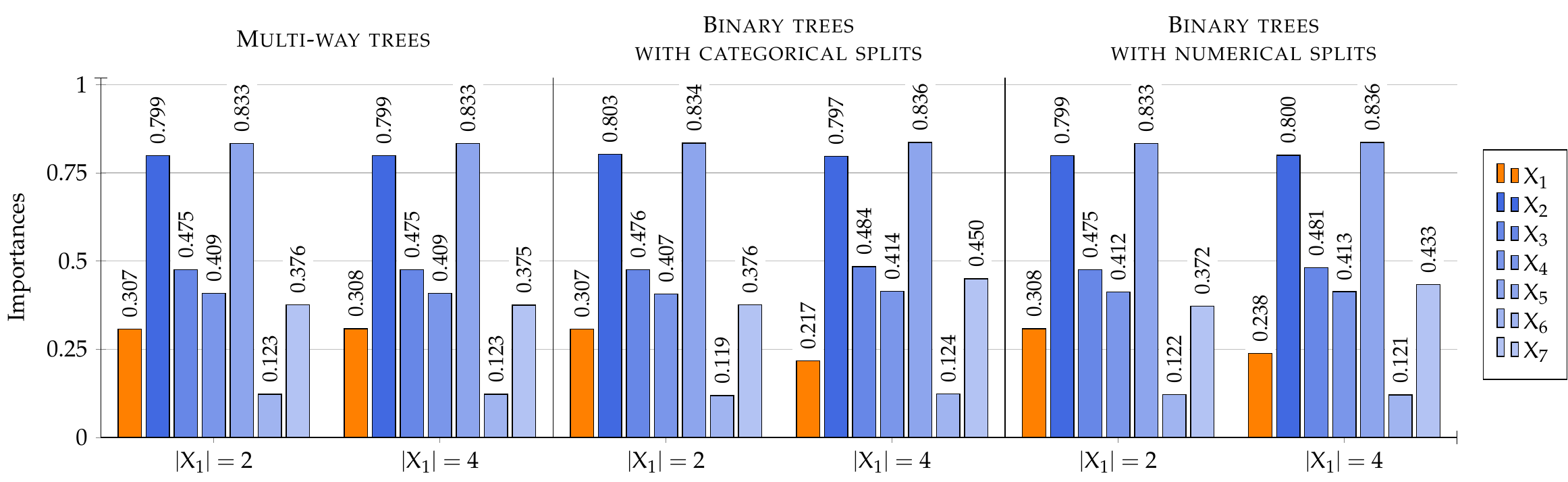}}\\
	\caption{Importance scores as computed by an ensemble of $10000$ trees with $K=1$ (top) and $K=p$ (bottom), with multiway trees (left), binary trees with (unordered) categorical splits (center), and binary trees with ordered splits (right). The considered problem is the digit recognition problem of \citep{breiman1984classification}. $|X_1|=2$ corresponds to the original problem with only binary variables while $|X_1|=4$ corresponds to the same problem where the cardinality of $X_1$ has been artificially increased to $4$ (with both values $0$ and $1$ splitted each into two new values, $\{0,1\}$ and $\{2,3\}$ respectively, with equal probability). The six other variables $X_2,\dots,X_7$ are left unchanged.}
	\label{fig:mdi:bin}
\end{figure}

Figures \ref{fig:mdi:bin:k1} and \ref{fig:mdi:bin:kp} show the importance scores computed respectively from totally randomised (i.e., $K=1$) and non-totally randomised (i.e., $K=p$) Extra-Trees, in which the split is randomly selected (split-wise randomisation). Both figures compare multiway trees and binary trees with either unordered or ordered binary splits. While all variables are binary in the original problem, we artificially increased the cardinality of variable $X_1$ from 2 to 4 by splitting both values $0$ and $1$ of this variable each into two new values, respectively $\{1,2\}$ and $\{3,4\}$, with equal probability. This transformation does not change the information brought by $X_1$ about $Y$ but it allows us to illustrate the effect of the different binary split strategies on importance scores.

When $|X_1|=2$, all tree growing methods lead to the same importance scores for all variables as expected (the slight differences are due to the use of a finite number of trees). The importance of $X_1$ is nevertheless decreased when $K$ goes from 1 to $p$, due to masking effects. When the cardinality of $X_1$ is increased to 4, we notice that the three splitting strategies lead to different importance scores. With ordered splits (see Figure \ref{fig:mdi:bin-ord} for all candidate splits), all candidate splits are somehow useful because they all provide part (or all for the mid-split) of the information content of $X_1$. By contrast, there are much more candidate unordered splits (see Figure \ref{fig:mdi:bin-unord} for all of them) including several ones that do not provide any information about $Y$ (e.g., the split $(\{1,3\},\{2,4\})$ does not change the distribution of $Y$). With $K=1$, split variables are selected totally at random. In the case of ordered splits, any variable except $X_1$ that is used is granted for all its information while $X_1$ only receives its full importance in one third of all splits. In the case of unordered splits, the chance of $X_1$ to be granted of its full importance is even smaller because of the useless splits. In both cases, this gives more opportunity to another variable to capture part of the information contained in $X_1$ about $Y$ and hence leads to a reduction of the importance of $X_1$ in the case of binary trees (with respect to multiway trees). A similar effect is observed when $K=p$. Because of the split randomisation, some splits on $X_1$ will be uninformative and in such case, $X_1$ will not be chosen to split the node to the benefit of another variable, leading to an overall decrease of the importance of $X_1$. Interestingly, the importances of all variables except $X_7$ are mostly unchanged whatever the splitting strategy. The importance of $X_7$ is however increased when going from multiway to binary trees with $|X_1|=4$. This is a consequence of the high redundancy between variables $X_1$ and $X_7$: they are equal for all digits except 7 (see Appendix \ref{app:digit}). $X_7$ is thus the variable which benefits the most from the irrelevant splits on $X_1$ introduced by the binary trees.

Note that importance scores would be different if splits were optimized, instead of randomized, for each variable, as in the standard Random Forests method. In the case of our example, multiway and binary trees would have given the exact same importance scores for all variables even when $|X_1|=4$, since the optimal split would always be the split $(\{1,2\},\{3,4\})$. It is possible however to design problems where the Random Forests node splitting strategy will make importance scores derived from multiway trees different from importance scores derived from binary trees.

\section{In non-asymptotic conditions}\label{sec:mdi-finite}

\subsection*{Setting of this section: $N_T \not\rightarrow\infty, N \not\rightarrow\infty$}
	
	From now on, we do no longer consider asymptotic conditions. This section aims at examining the importance measure in finite settings and investigate results of this chapter in this context.
	Section \ref{sec:mdi-finite-trees} considers a finite number of trees. This suggests that all possible branches (i.e., not masked) are not necessarily explored and/or fairly taken into account. Section \ref{sec:mdi-finite-samples} considers a finite number of samples. It implies that $I(X_m;Y|B)$ can not be computed exactly and must be empirically estimated from samples. 

	\subsection{With a finite number of trees}\label{sec:mdi-finite-trees}

	As mentioned in Section \ref{sec:importances:stability}, in practice, the number of trees in a Random Forest ensemble should be as large as possible in order to achieve the best predictive performances. At some point however, a plateau should be reached and adding more trees will not increase significantly the performance. The impact on feature importance is usually not taken into account however. In this section, we still assume a learning sample of infinite size and study the impact of the number of trees on the properties highlighted so far. We only examine fully developed trees but results in this section can be easily generalised to non-fully developed trees given the analysis in Section \ref{sec:mdi-pruning}.
	
	As presented in Equation \ref{eqn:mdi-mdi-empirical}, the importance of a feature is computed over all trees and over all nodes of all trees. With an infinite number of trees, we saw in Theorem \ref{thm:mdi-totally-imp} that the relationship between $X_m$ and $Y$ is evaluated for all combinations $B$ in such a way that all terms equally contribute to the total importance. When only a finite number $N_T$ of trees is constructed, some conditionings $B$ (branches) can be missed and thus the importance will only contain a subset of all $I(X_m;Y|B)$ terms. However, we have the following general result:
	\begin{proposition}\label{mdi:relevance:finitenumberoftrees}
		$\forall K,q$, $Imp^{K,q}_{\infty,N_T} > 0 \Rightarrow X_m\quad \text{is relevant}$.
	\end{proposition}
	Features with strictly positive importance scores are necessarily relevant, since it implies that at least one term $I(X_m;Y|B) >0$. However, a relevant feature does not necessarily have positive importance score, even a strongly relevant one. In all generality, Proposition \ref{th:mdi:K:stronglypruned} is thus not valid with a finite number of trees. To give an example, let us consider a XOR scenario with two features $X_1$ and $X_2$ such that $I(X_1;Y)=I(X_2;Y)=0$ and $I(Y;X_1,X_2)>0$. If a single tree is grown, only the feature tested at the second level will receive a non-zero importance, while both features are (strongly) relevant.

        This observation suggests that an undesirable effect of using a finite number of trees is that features that are not examined (or not with the right conditioning set $B$) have zero importances. \textit{Unseen} features therefore wrongly appear as irrelevant with respect to $Y$, like masked features or those with too high degree.
        Note however that if the composition property is verified, then a single tree (with $K=p$) can identify all strongly relevant features because strongly relevant features can not be masked.

        \begin{theorem}
          If $K=p$ and if $P_{V,Y}(V,Y)$ verifies the composition property:  $\; X_m \in V$ strongly relevant $\Rightarrow Imp_{\infty,1}^{p,q} (X_m)>0$.
        \end{theorem}

\begin{proof}
      	We want to show that a single tree that is fully developed with $K=q$ is sufficient to give to all strongly relevant features a strictly positive importance score when the distribution over all variables verifies the composition property. 
        Since the tree is fully developed, all features are exhausted in each branch and each leaf corresponds to a possible assignment $v$ to all input features $V$. Let us assume that a strongly relevant variable $X_m$ does not have a strictly positive importance score and show that this leads to a contradiction.

        If $X_m$ does not receive a strictly positive importance score, it means that $X_m$ is never used in a terminal node corresponding to a configuration $v^{-m}$ such that $X_m\nindep Y | V^{-m}=v^{-m}$ or in an internal node corresponding to a configuration $b$ of $B\subset V^{-m}$ such that $X_m\nindep Y | B=b$. By definition of strong relevance, we have $X_m\nindep Y | V^{-m}$ which implies that there exists at least one configuration $V^{-m}=v^{-m,*}$ such that $X_m\nindep Y | V^{-m}=v^{-m,*}$. Let us consider the path in the tree from the root node to a node where $X_m$ is tested that matches the values in $v^{-m,*}$. $X_m$ can not be tested at the end of such path because otherwise it would have got a strictly positive importance score. The node $X_m$ is thus necessarily used in the path in a node corresponding to a configuration $B=b^*$ that matches for some variables $B\subset V^{-m}$ the configuration $v^{-m,*}$ and such that $X_m\indep Y|B=b^*$. In the same conditioning $B=b^*$, all features in $R=V^{-m} \setminus B$ are also independent of $Y$ conditionally to $B=b^*$, otherwise one of them would have been preferred to $X_m$ to split the node (since $K=p$ means that they were all evaluated when splitting the node). Given the composition property, we thus have that $(X_m\cup R)\indep Y|B=b^*$. The weak union property then implies that $X_m\indep Y |(B=b^*)\cup R$, which means that $X_m\indep Y|(B=b^*)\cup(R=r)$ for all configurations $r$ of the variables in $R$. This is thus also true for the configuration $r^*$ that matches the configuration $v^{-m,*}$, which shows that $X_m \indep Y|V^{-m}=v^{-m,*}$. This is however impossible by definition of $v^{-m,*}$.
        
\end{proof}

	    In the same vein, \cite{wehenkel2018characterization} computed analytically the minimum number of trees such that all features are at least seen once (among the $K$ features selected at a given node) for a given value $K$. This analysis showed that many trees are needed, in particular when $K$ is small and individual decision trees are small. Note that having seen all features once is obviously not enough to identify all relevant variables, as they need to be tested at least in one of their minimal conditioning sets $B$ and furthermore not to be masked in such case by other variables. The number of trees given in \citep{wehenkel2018characterization} is thus a very minimal bound on the number of trees really needed to find all relevant variables.

	Moreover, let us note that even if the number of trees is large enough to consider all possible branches, computed importances with a finite forest are most likely different from theoretical asymptotic importances because all $B$ may not be fairly considered in the forest.
	
	\subsection{With a finite number of samples}\label{sec:mdi-finite-samples}

	In all analyses carried out so far, assuming a sample set of infinite size actually corresponds to know the data distribution and therefore to compute with exactitude all measures, e.g. node impurity $i(t)$ and node decrease $\Delta i(s,t)$. However, in practice, impurity measurements are estimated from a finite sample set and therefore suffer from an empirical misestimation bias. Concretely, it means that Equation \ref{eqn:mdi-totally-imp} of Theorem \ref{thm:mdi-totally-imp} becomes, if we still assume an infinite number of trees, 
	\begin{eqnarray}\label{eqn:mdi:finitesample:totally-imp}
	Imp_{N,\infty}^{1,p}(X_m) = \sum_{k=0}^{p-1} \dfrac{1}{C_p^k}\dfrac{1}{p-k} 
	\sum_{B\in \mathcal{P}_k (V^{-m})} \hat{I}(X_m;Y|B)
	\end{eqnarray}
	where $\hat{I}(X_m;Y|B)$ are estimated mutual informations. 
	
	Among other authors, \cite{goebel2005approximation} show that mutual information estimation between two independent variables is positively biased. That is, let us consider two independent discrete random variables $X$ and $Y$ of probability density $P_X(X)$ and $P_Y(Y)$ respectively and such that $I(X;Y) = 0$, their finite sample size estimates $\hat{I}(X;Y)$ are expected to be strictly positive, i.e.,
	\begin{eqnarray}
	\mathbb{E}\{\hat{I}(X;Y)\} = \dfrac{(|X|-1)(|Y|-1)}{2N \ln(2)} > 0
	\end{eqnarray}
	 where $N$ is the number of  observed samples, $|X|$ and $|Y|$ are respectively the cardinalities of $X$ and $Y$. In contrast with Theorem \ref{thm:mdi-totally-irrelevant}, this however suggests that irrelevant features never have zero importances. \cite{louppe2014understanding} stresses the linear dependence with variable cardinalities and the inverse dependence of the number of samples, and relates with many empirical studies that observe a bias towards feature of large number of categories and cardinalities \citep{strobl2007bias}. 
	
	In details, given three random variables $X$,$Y$,$Z$ of probability densities $P_X(X)$, $P_Y(Y)$, $P_Z(Z)$ respectively, \cite{goebel2005approximation} show that the estimator for conditional mutual information $\hat{I}(X;Y|Z)$ is approximately gamma distributed  
	\begin{eqnarray}
		\hat{I}(X;Y|Z) \sim \Gamma \left ( \dfrac{|Z|}{2} (|X|-1) (|Y|-1),  \dfrac{1}{N\ln(2)}\right )
	\end{eqnarray}
	 where $\Gamma(k,\theta)$ is the gamma distribution with a shape parameter $k$ and a scale parameter $\theta$. Let us note that a random variable $W$ such that $W \sim \Gamma(v/2,2c)$ with $c>0$, then $W$ also follows a $\chi^2$ (chi-square) distribution\footnote{\cite{saporta2006probabilites} define a $\chi^2$ law as follows: Let $U_1,U_2,\dots,U_p$ be $p$ independent variables, each following $\mathcal{N}(0,1)$, then the chi-square law with $p$ degrees of freedom, denoted $\chi^2_p$, is the law of the variable $\sum_{i=1}^p U_i^2$.} with $v$ degrees of freedom. In the case of $\hat{I}(X;Y|Z)$, it then follows a $\chi^2$ distributions of $|Z|(|X|-1) (|Y|-1)$ degrees of freedom (and $c=\frac{1}{2N\ln(2)}$). That is, we have that $2N\ln(2) \hat{I}(X;Y|Z)$ converges asymptotically towards a $\chi^2$ distribution with $|Z|(|X|-1) (|Y|-1)$ degrees of freedom, that only depends on feature cardinalities. One can then use a chi-square based statistical test on the mutual information between two features to determine if their are independent. Let us once again note that the number of degrees of freedom increase with features cardinalities.

To avoid false positives, all those results suggest to combine non-totally developed trees, in order not to estimate mutual informations from too few samples at deep nodes, with non-totally randomised trees ($K>1$), in order to avoid splitting on irrelevant features at the top nodes, which would unnecessarily reduce the size of the learning sample. Unfortunately, as the previous analyzes show, decreasing tree depth or increasing $K$ will however increase the number of false negatives. There is thus a tradeoff to be found in practice between these two antagonistic effects.

\section{Result summary}

The following table summarises the main results exposed in this chapter, with references to the main theorems.

\newpage

\renewcommand\tabularxcolumn[1]{m{#1}}

{\footnotesize 
\begin{tabularx}{\linewidth}{X|m{6em}|m{5.2em}:m{6em}|m{5.2em}:m{5.2em}}
	\multicolumn{2}{c|}{}& \multicolumn{2}{c|}{\hspace{-0.4em}$K=1$} & \multicolumn{2}{c}{$K>1$}\\
	\multicolumn{2}{c|}{}&	$D=p$ & $D<p$ & $D=p$ & $D<p$\\ \hline
	
	\multirow{2}{*}{\parbox[t][6em][c]{5em}{Theoretical\\results}} & \parbox[t][3em][t]{6em}{Analytical\\formulation} & Thm. \ref{thm:mdi-totally-imp} & \parbox[t][3em][t]{4.2em}{Prop.~\ref{proposition:pruning} \\ Prop.~\ref{proposition:imp-subspaces}} & $-$&$-$ \\\cdashline{2-6}
	
	& \parbox[t][3em][t]{5em}{Sum of\\ importances}    & Thm. \ref{thm:mdi-totally-sum} & $-$ & $-$&$-$ \\\hline
														   
	\multirow{4}{*}{\parbox[t][16em][c]{6em}{Importance vs. \\ Relevance \\ in \\asymptotic\\ conditions}}  & \parbox[c][4em][c]{5em}{Irrelevant\\ variables} & \parbox[c][3em][c]{5em}{$\Leftrightarrow Imp=0$\\ Thm. \ref{thm:mdi-totally-irrelevant}} & $\Rightarrow Imp=0$ &  $\Rightarrow Imp=0$ & $\Rightarrow Imp=0$ \\ \cdashline{2-6}

  & \parbox[c][4em][c]{6em}{Relevant\\variables} & \multirow{2}{*}{\parbox[c][8em][c]{5em}{$\Leftrightarrow Imp>0$\\ Cor. \ref{cor:K1allrelevant}}} &  \multirow{2}{*}{\parbox[c][8em][c]{6em}{$\Rightarrow Imp>0$\footnote{When $r>q$, this is only valid for relevant variables $X_m$ such that $deg(X_m) < q$ (see Prop. \ref{mdi:pop:deg:k1:relevantimp}). }\\ if $r\le  q$\\ Prop. \ref{th:mdi-relpruned}}} & \multicolumn{2}{c}{\parbox[c][3em][c]{8em}{\centering $\Rightarrow Imp\ge 0$ \\ Masking effect}}  \\ \cdashline{2-2}\cdashline{5-6}

  & \parbox[t][4em][t]{6em}{Strongly\\relevant\\variables} & & & \parbox[c][3em][c]{5em}{$\Rightarrow Imp>0$ \\ Prop. \ref{th:mdi:K:stronglypruned}} & \parbox[c][8em][c]{5em}{$\Rightarrow Imp>0$\\ if $r\le  q$\footnote{When $r>q$, only strongly relevant variables $X_m$ such that $deg(X_m)<q$ can not be masked (see Prop. \ref{mdi:pop:deg:stronglyrelevantnotmasked}).}\\ Thm. \ref{th:mdi-stronglyrelpruned}}\\ \cdashline{2-6}

 & \parbox[t][4em][t]{6em}{Presence of \\ irrelevant variables} & \parbox[t][4em][t]{6em}{No effect\\ Lem.~\ref{lemma:mdi-adding-irrelevant},\\ Thm. \ref{thm:mdi-totally-relevant}} &  $-$ & Dampens masking effect  & $-$\\ \hline
 Binary splits & &\multicolumn{4}{c}{Same relevance but different importance scores}\\ \hline
 \multirow{2}{*}{Finite settings} & \parbox[c][3em][c]{6.2em}{Finite number\\ of trees} & \multicolumn{4}{c}{\parbox[c][4em][c]{20em}{\centering$\forall K,D\quad Imp^{K,D}_{\infty,N_T}(X_m) > 0 \Rightarrow X_m$ is relevant \\ Prop. \ref{mdi:relevance:finitenumberoftrees}}} \\ \cdashline{2-6}
 & \parbox[c][3em][c]{6.2em}{Finite number\\ of samples} & \multicolumn{4}{c}{misestimation bias ($Imp^{K,D}_{N,\infty}>0 \not\Rightarrow$ relevance)} \\
	 
\end{tabularx}}

\begin{summary}
	In asymptotic conditions, MDI feature importances can be derived analytically and provide an understandable decomposition of the total information conveyed by input features about the target by feature, by cardinality of the interaction term, and by interaction term. Additionally, the sum of all importances is fixed. The introduction of redundant feature tends to modify all importance scores and not only those of features that conveyed redundant informations. When trees are built totally at random, zero importances are only associated to irrelevant features. When trees are not totally random, the masking effect prevents some weakly relevant feature to be identified and only strongly relevant features are ensured to have positive importance scores. When trees are non-totally developed, interaction terms of larger cardinalities are no longer evaluated and therefore do not enter into account in the importance scores. However, guarantees can be preserved by restricting the number of relevant features or the feature degree. In more realistic and practical settings (i.e., binary trees, finite sample set, finite number of trees), those desirable properties are however usually not preserved.	 
\end{summary}

\part{Extensions and derivations of importance measures}

\chapter{With a contextual effect}\label{ch:context}

\begin{overview}

In this chapter, we extend the random forest feature importances framework to perform a contextual analysis. For many problems, feature selection is often more complicated than identifying a single subset of input features that would together explain the output, as described in Chapter \ref{ch:background} especially. There may be interactions that depend on contextual information, i.e., variables that reveal to be relevant only in some specific circumstances. We briefly discussed in Section \ref{sec:background-discussion} that such feature interactions must be taken into account but a single feature ranking provides only very limited information about such complex relationships. In this setting, our contribution is to extend the MDI feature importance measure (i) to identify variables whose relevance is context-dependent, and (ii) to characterise as precisely as possible the effect of contextual information on the importance of these variables. \\

\textbf{\textcolor{RoyalBlue}{References:}} This chapter is an adapted version of the following publication: \\[2mm]
\bibentry{sutera2016context}.\\[2mm]
Terminology and notations have been slightly adjusted for the sake of consistency with the rest of this manuscript. The text has also been processed to minimize overlap with respect to previous chapters.

\end{overview}

\section{Motivation}
Supervised learning finds applications in many domains such as
medicine, economics, computer vision, or bioinformatics. Given a
sample of observations of several inputs and one output variable, the
goal of supervised learning is to learn a model for predicting the
value of the output variable given any values of the input
variables. Another common side objective of supervised learning is to
bring as much insight as possible about the relationship between the
inputs and the output variable. One of the simplest ways to gain such
insight is through the use of feature selection or ranking methods
that identify the input variables that are the most decisive or
relevant for predicting the output, either alone or in combination
with other variables. Among feature selection/ranking methods, one
finds variable importance scores derived from random forest models
that stand out from the literature mainly because of their
multivariate and non parametric nature and their reasonable
computational cost.  Although very useful, feature selection/ranking
methods however only provide very limited information about the often
very complex input-output relationships that can be modeled by
supervised learning methods. There is thus a high interest in
designing new techniques to extract more complete information about
input-output relationships than a single global feature subset or
feature ranking.

In this chapter, we specifically address the problem of the
identification of the input variables whose relevance or irrelevance
for predicting the output only holds in specific circumstances, where
these circumstances are assumed to be encoded by a specific context
variable. This context variable can be for example a standard input
variable, in which case, the goal of contextual analyses is to better
understand how this variable interacts with the other inputs for
predicting the output. The context can also be an external variable
that does not belong to the original inputs but that may nevertheless
affect their relevance with respect to the output. Practical
applications of such contextual analyses are numerous. In some
applications, one may be interested in finding variables that are both
relevant and independent of the context.  For example, in medical
studies \citep[see, e.g.,][]{geissler2000risk}, one is often interested
in finding risk factors that are as independent as possible of
external factors, such as the sex of the patients, their origins or
the data cohort to which they belong. By contrast, in some other
cases, one may be interested in finding variables that are relevant
but dependent in some way on the context. For example, in systems
biology, differential analysis~\citep{ideker2012differential} aims at
discovering genes or factors that are relevant only in some specific
conditions, tissues, species or environments.

Our contribution in this chapter is two-fold. First, starting from
common definitions of feature relevance in the literature, we propose
a formal definition of context-dependent variables and provide a
complete characterization of these variables depending on how their
relevance is affected by the context variable. Second, we extend the
random forest variable importances framework in order to identify and
characterize variables whose relevance is context-dependent or
context-independent. Building on existing theoretical results for
standard importance scores, we propose asymptotic guarantees for the
resulting new measures with respect to the formal definitions.

The chapter is structured as follows. In
Section~\ref{sec:contextual-relevance}, we first lay out our formal
framework defining context-dependent variables and describing how the
context may change their relevance.  We describe in
Section~\ref{sec:contextual-relevance-trees} how random forest
variable importances can be used for identifying context-dependent
variables and how the effect of contextual information on these
variables can be highlighted. Our results are then illustrated in
Section~\ref{sec:contextexperiments} on representative problems. Finally,
conclusions and directions of future works are discussed in
Section~\ref{sec:contextconclusions}.

\section{Context-dependent feature selection and characterization}
\label{sec:contextual-relevance}

\paragraph{Context-dependence.}

Let us consider a set $V = \{X_1,\dots,X_p\}$ of $p$ input variables and an
output $Y$ and let us denote by $V^{-m}$ the set $V\setminus \{X_m\}$. All
input and output variables are assumed to be categorical, not necessarily
binary\footnote{The case of a non categorical output will be discussed in
	Section~\ref{sec:generalisation}.}. Let us reconsider the definitions of relevant,
irrelevant, and marginally relevant variables based on their mutual information
$I$ (as defined in Definitions \label{ref:relevanceMI} and \ref{def:strongweakrelevanceMI}).

Let us now assume the existence of an additional (observed) context variable
$X_c\notin V$, also assumed to be categorical.

Inspired by the notion of relevant and irrelevant variables, we propose to
define context-dependent and context-independent variables as follows:
\begin{definition}\label{defcontextdependent}
A variable $X_m \in V$ is \textbf{context-dependent to $Y$ with respect to $X_c$}
iff there exists a subset $B \subseteq V^{-m}$ and some values $x_c$ and $b$
such that\footnote{In this definition and all definitions that follow, we will
  assume that the events on which we are conditioning have a non-zero
  probability and that if such event does not exist then the condition of the
  definition is not satisfied.}:
\begin{equation}\label{eqn:context-dependent}
 I(Y;X_m|B=b,X_c=x_c) \ne I(Y;X_m|B=b).
\end{equation}
\end{definition}
\begin{definition}\label{defcontextindependent}
A variable $X_m \in V$ is \textbf{context-independent to $Y$ with respect to $X_c$} iff for all subsets $B \subseteq V^{-m}$ and for all values $x_c$ and $b$, we have:
\begin{equation}\label{eqn:context-independent}
 I(Y;X_m|B=b,X_c=x_c) = I(Y;X_m|B=b).
\end{equation}
\end{definition}
Context-dependent variables are thus the variables for which there exists a
conditioning set $B$ in which the information they bring about the output is
modified by the context variable. Context-independent variables are the
variables that, in all conditionings $B=b$, bring the same amount of
information about the output whether the value of the context is known or not.
This definition is meant to be as general as possible. Other more specific
definitions of context-dependence are as follows:

\begin{eqnarray}
&&\begin{split}
&\hspace{-3em}\exists B \subseteq V^{-m}, b, x^1_c, x^2_c:\\
&\hspace{-2em}I(Y;X_m|X_c=x^1_c,B=b) \neq I(Y;X_m|X_c=x^2_c,B=b),
 \end{split}
 \label{altcond1}\\
&&\begin{split}
&\hspace{-3em}\exists B \subseteq V^{-m},x_c:\\
&\hspace{-2em}I(Y;X_m|X_c=x_c,B) \neq I(Y;X_m|B),
\end{split}
\label{altcond2}\\
&&\begin{split}
&\hspace{-3em}\exists B\subseteq V^{-m},b:\\
&\hspace{-2em}I(Y;X_m|X_c,B=b) \neq I(Y;X_m|B=b),
\end{split}
\label{altcond3}\\
&&\begin{split}
&\hspace{-3em}\exists B\subseteq V^{-m}:\\
&\hspace{-2em}I(Y;X_m|X_c,B) \neq I(Y;X_m|B).
\end{split}
\label{altcond4}
\end{eqnarray}

These definitions all imply context-dependence as defined in Definition~\ref{defcontextdependent} but the converse is in general not true. For example, Definition~(\ref{altcond1}) misses problems where the context makes some otherwise irrelevant variable relevant but where the information brought by this variable about the output is exactly the same for all values of the context. A variable that satisfies Definition~(\ref{eqn:context-dependent}) but not Definition~(\ref{altcond2}) is given in example~\ref{example1}. This example can be easily adapted to show that both Definitions~(\ref{altcond3}) and (\ref{altcond4}) are more specific than Definition~(\ref{eqn:context-dependent}) (by swapping the roles of $X_c$ and $X_2$).

\begin{example}\label{example1}
  This artificial problem is defined by two input variables $X_1$ and $X_2$, an output $Y$, and a context $X_c$. $X_1$, $X_2$, and $X_c$ are binary variables taking their values in $\{0,1\}$, while $Y$ is a quaternary variable taking its values in $\{0,1,2,3\}$. All combinations of values for $X_1$, $X_2$, and $X_c$ have the same probability of occurrence $0.125$ and the conditional probability $P(Y|X_1,X_2,X_C)$ is defined by the two following rules:
  \begin{itemize}
    \item If $X_2=X_c$ then $Y=X_1$ with probability 1.
    \item If $X_2\neq X_c$ then $Y=2$ with probability $0.5$ and $Y=3$ with probability $0.5$.
  \end{itemize}
The corresponding data table is given in Appendix \ref{app:ex1}. For
this problem, it is easy to show that $I(Y;X_1|X_2=0,X_c=0)=1$ and
that $I(Y;X_1|X_2=0)=0.5$, which means
condition~(\ref{eqn:context-dependent}) is satisfied and $X_1$ is thus
context-dependent to $Y$ with respect to $X_c$ according to our
definition. On the other hand, we can show that:
\begin{eqnarray*}
&& I(Y;X_1|X_c=x_c)=I(Y;X_1)=0.5\\
&& I(Y;X_1|X_2,X_c=x_c)=I(Y;X_1|X_2)=0.5,
\end{eqnarray*}
for any $x_c\in\{0,1\}$, which means that condition~(\ref{altcond2}) can not be satisfied for $X_1$.

\end{example}

To simplify the notations, the context variable was assumed to be a
separate variable not belonging to the set of inputs $V$. It can
however be considered as an input variable, whose own relevance to $Y$
(with respect to $V\cup \{X_c\}$) can be assessed as for any other
input.  Let us examine the impact of the nature of this variable on
context-dependence. First, it is interesting to note that the
definition of context-dependence is not symmetric. A variable $X_m$
being context-dependent to $Y$ with respect to $X_c$ does not imply
that the variable $X_c$ is context-dependent to $Y$ with respect to
$X_m$.\footnote{This would be the case however if we had adopted the
  definition (\ref{altcond4}).} Second, the context variable does not
need to be marginally relevant for some variable to be
context-dependent, but it needs however to be relevant to $Y$ with
respect to $V$. Indeed, we have the following theorem:

\begin{theorem}\label{theo1} $X_c$ is irrelevant to $Y$ with respect to $V$ iff all variables in $V$ are context-independent to $Y$ with respect to $X_c$ (and $V$) and $I(Y;X_c)=0$.
\end{theorem}

\begin{proof}
	See Appendix \ref{app:theo1}.
\end{proof}

As a consequence of this theorem, there is no interest in looking for
context-dependent variables when the context itself is not relevant\footnote{{This is consistent with Proposition \ref{prop:mdi-only-relevant}. All features in a minimal conditioning subset of $B$ are necessarily relevant, including any contextual features. }}.

\paragraph{Characterizing context-dependent variables.}

Contextual analyses need to focus only on context-dependent variables since,
by definition, context-independent variables are unaffected by the context:
their relevance status (relevant or irrelevant), as well as the information
they contain about the output, remain indeed unchanged whatever the
context.

Context-dependent variables may be affected in several directions by
the context, depending both on the conditioning subset $B$ and on the
value $x_c$ of the context. Given a context-dependent variable $X_m$,
a subset $B$ and some values $b$ and $x_c$ such that
$I(Y;X_m|B=b,X_c=x_c)\neq I(Y;X_m|B=b)$, the effect of the context can
either be an increase of the information brought by $X_m$
($I(Y;X_m|B=b,X_c=x_c) > I(Y;X_m|B=b)$) or a decrease of this
information ($I(Y;X_m|B=b,X_c=x_c) < I(Y;X_m|B=b)$). Furthermore, for
a given variable $X_m$, the direction of the change can differ from
one context value $x_c$ to another (at fixed $B$ and $b$) but also
from one conditioning $B=b$ to another (for a fixed context
$x_c$). Example~\ref{example2} below illustrates this latter
case. This observation makes a global characterization of the effect
of the context on a given context-dependent variable difficult. Let us
nevertheless mention two situations where such global characterization
is possible:

\begin{definition}
A context-dependent variable $X_m\in V$ is \textbf{context-complementary} (in a context $x_c$) iff for all
$B \subseteq V^{-m}$ and $b$, we have $I(Y;X_m|B=b,X_c=x_c) \ge
I(Y;X_m|B=b)$.
\end{definition}
\begin{definition}
A context-dependent variable $X_m\in V$ is \textbf{context-redundant} (in a context $x_c$)  iff for all
$B \subseteq V^{-m}$ and $b$, we have $I(Y;X_m|B=b,X_c=x_c) \le
I(Y;X_m|B=b)$.
\end{definition}

Context-complementary and redundant variables are variables that
always react in the same direction to the context and thus can be
characterized globally without loss of
information. Context-complementary variables are variables that bring
complementary information about the output with respect to the
context, while context-redundant variables are variables that are
redundant with the context. Note that context-dependent variables that
are also irrelevant to $Y$ are always context-complementary, since the
context can only increase the information they bring about the
output. Context-dependent variables that are relevant to $Y$ however
can be either context-complementary, context-redundant, or
uncharacterized. A context-redundant
variable can furthermore become irrelevant to $Y$
(with respect to $V\cup \{X_c\}$)
as soon as $I(Y;X_m|B=b,X_c=x_c)=0$ for all $B$, $b$, and $x_c$.

\begin{example}\label{example2}
As an illustration, in the problem of Example~\ref{example1}, $X_1$
and $X_2$ are both relevant and context-dependent variables. $X_1$ can
not be characterized globally since we have simultaneously:
\begin{eqnarray*}
  I(Y;X_1|X_2=0,X_c=x_c)&>&I(Y;X_1|X_2=0)\\
  I(Y;X_1|X_2=1,X_c=x_c)&<&I(Y;X_1|X_2=1),
\end{eqnarray*}
for both $x_c=0$ and $x_c=1$. $X_2$ is however context-complementary
as the knowledge of $X_c$ always increases the information it contains
about $Y$.
\end{example}

\paragraph{Related works.}

Several authors have studied interactions between variables in the context of
supervised learning. They have come up with various interaction definitions and
measures, e.g., based on multivariate mutual information
\citep{mcgill1954multivariate,jakulin2003analyzing}, conditional mutual
information \citep{jakulin2005machine, van2011two}, or variants thereof
\citep{brown2009new, brown2012conditional}. There are several differences
between these definitions and ours. In our case, the context variable has a
special status and as a consequence, our definition is inherently asymmetric,
while most existing variable interaction measures are symmetric. In addition,
we are interested in detecting any information difference occurring in a given
context (i.e., for a specific value of $X_c$) and for any conditioning subset
$B$, while most interaction analyses are interested in average and/or
unconditional effects. For example, \citep{jakulin2003analyzing} propose as a
measure of the interaction between two variables $X_1$ and $X_2$ with respect
to an output $Y$ the multivariate mutual information, which is defined as
$I(Y;X_1;X_2)=I(Y;X_1)-I(Y;X_1|X_2)$. Unlike our definition, this measure can
be shown to be symmetric with respect to its arguments. Adopting this measure
to define context-dependence would actually amount at using
condition~(\ref{altcond4}) instead of condition~(\ref{defcontextdependent}),
which would lead to a more specific definition as discussed earlier in this
section.

The closest work to ours in this literature is due to
\cite{turney1996identification}, who proposes a definition of
context-sensitivity that is very similar to our definition of
context-dependence. Using our notations, \cite{turney1996identification}
defines a variable $X_m$ as weakly context-sensitive to the variable $X_c$ if
there exist some subset $B\subseteq V^{-m}$ and some values $y$, $x_m$, $b$,
and $x_c$ such that these two conditions hold:
\begin{align*}
p(Y=y|X_m=x_m,X_c=x_c,B=b)&\neq p(Y=y|X_m=x_m,B=b),\\
p(Y=y|X_m=x_m,X_c=x_c,B=b)&\neq p(Y=y|X_c=x_c,B=b).
\end{align*}
$X_m$ is furthermore defined as strongly context-sensitive to $X_c$ if $X_m$ is
weakly sensitive to $X_c$, $X_m$ is marginally relevant,and $X_c$ is not
marginally relevant. These two definitions do not exactly coincide with ours
and they have two drawbacks in our opinion. First, they do not consider that a
perfect copy of the context is context-sensitive, which we think is
counter-intuitive. Second, while strong context-sensitivity is asymmetric, the
constraints about the marginal relevance of $X_m$ and $X_c$ seems also
unnatural.

Our work is also somehow related to several works in the graphical model
literature that are concerned with context-specific independences between
random variables \citep[see e.g.][]{boutilier1996,zhang99}. \cite{boutilier1996}
define two variables $Y$ and $X_m$ as contextually independent given some
$B\subseteq V^{-m}$ and a context value $x_c$ as soon as
$I(Y;X_m|B,X_c=x_c)=0$. When $B\cup\{X_m,X_c\}$ are the parents of node $Y$ in
a Bayesian network, then such context-specific independences can be exploited
to simplify the conditional probability tables of node $Y$ and to speed up
inferences. \cite{boutilier1996}'s context-specific independences will be captured by our
definition of context-dependence as soon as $I(Y;X_m|B)>0$. However, our
framework is more general as we want to detect any context dependencies, not
only those that lead to perfect independences in some context.

\section{Context analysis with random forests}
\label{sec:contextual-relevance-trees}

In this section, we show how to use variable importances derived from Random
Forests first to identify context-dependent variables
(Section~\ref{sec:identification}) and then to characterize the effect of the
context on the relevance of these variables
(Section~\ref{sec:charact}). Derivations in this section are based on the
theoretical characterization of variable importances provided in
\citep{louppe2013understanding}, which is briefly reminded in
Section~\ref{backgroundtrees}. Section~\ref{sec:in-practice} discusses
practical considerations and Section~\ref{sec:generalisation} shows how to
generalize our results to other impurity measures.

\subsection{Variable importances \protect \footnote{This section is a reminder of the MDI importance measure and its asymptotic characterisation. See Section \ref{sec:mdi-totally} for more details.}
}\label{backgroundtrees}

Within the random forest framework,
\cite{breiman2001random} proposed to evaluate the importance of
a variable $X_m$  for predicting  $Y$ by adding up the weighted impurity decreases
for all nodes $t$ where $X_m$ is used, averaged over all $N_T$ trees
in the forest:
\begin{equation} \label{eqn:impfini}
  Imp(X_m) = \frac{1}{N_T} \sum_{T} \sum_{t\in T:v(s_t)=X_m} p(t) I(Y;X_m|t)
\end{equation}
where $v(s_t)$ is the variable used in the split $s_t$ at node $t$, $p(t)$ is the proportion of samples reaching $t$ and $I$ is the mutual information.

According to \cite{louppe2013understanding}, for any ensemble of fully
developed trees in asymptotic learning sample size conditions, the Mean Decrease Impurity (MDI)
importance~(\ref{eqn:impfini}) can be shown to be equivalent to
\begin{equation} \label{eqn:impasymp}
  Imp(X_m) = \sum_{k=0}^{p-1} \frac{1}{C^k_p}\frac{1}{p-k} \sum_{B\in{\cal P}_k(V^{-m})} I(Y;X_m|B),
\end{equation}
 where $V^{-m}$ denotes the subset $V \setminus \{X_m\}$, ${\cal P}_k(V^{-m})$
 is the set of subsets of $V^{-m}$ of size $k$.
where  ${\cal P}_k(V^{-m})$ denotes
the set of subsets of $V^{-m}$ of size $k$.
Most notably, it can be shown \citep{louppe2013understanding}  that this
measure is zero for a variable $X_m$ iff $X_m$ is irrelevant to $Y$ with
respect to $V$. It is therefore well suited for identifying relevant features.

\subsection{Identifying context-dependent variables}\label{sec:identification}

Theorem~\ref{theo1} shows that if the context variable $X_c$ is irrelevant,
then it can not interact with the input variables and thus modify their
importances. This observation suggests to perform, as a preliminary test, a
standard random forest variable importance analysis using all input variables
and the context in order to check the relevance of the latter. If the context
variable does not reveal to be relevant, then, there is no hope to find
context-dependent variables.

Intuitively, identifying context-dependent variables seems similar to
identifying the variables whose importance is globally modified when the
context is known.
Therefore, one first straightforward approach to identify context-dependent
variables is to build a forest per value $X_c=x_c$ of the context variable, i.e.,
using only the data samples for which $X_c=x_c$ , and also
globally, i.e. using all samples and not including the context among the
inputs. Then it consists in deriving from these models an importance score for
each value of the context, as well as a global importance score.
Context-dependent variables are then the variables whose global importance score
differs from the contextual importance scores for at least one value of the
context.

More precisely, let us denote by $Imp(X_m)$ the global score of a variable
$X_m$ computed using (\ref{eqn:impfini}) from all samples and by
$Imp(X_m|X_c=x_c)$ its importance score as computed according to
(\ref{eqn:impfini}) using only those samples such that $X_c=x_c$. With this
approach, a variable would be declared as context-dependent as soon as there
exists a value $x_c$ such that $Imp(X_m)\neq Imp(X_m|X_c=x_c)$.

Although straightforward, this approach has several drawbacks. First, in the
asymptotic setting of Section~\ref{backgroundtrees}, it is not
guaranteed to find all context-dependent variables. Indeed, asymptotically, it
is easy to show from (\ref{eqn:impasymp}) that $Imp(X_m)-Imp(X_m|X_c=x_c)$ can
be written as:

\begin{eqnarray}
  Imp^{x_c}(X_m) &\triangleq& Imp(X_m)-Imp(X_m|X_c=x_c)\\
  &=&\nonumber \sum_{k=0}^{p-1} \frac{1}{C_k^p} \frac{1}{p-k} \sum_{B\in{\cal P}_k(V^{-m})}  (I(Y;X_m|B)-I(Y;X_m|B,X_c=x_c)). \label{mdidiffcontext}
\end{eqnarray}

Example~\ref{example1} shows that $I(Y;X_m|B)$ can be equal to
$I(Y;X_m|B,X_c=x_c)$ for a context-dependent variable. Therefore we have the
property that if there exists an $x_c$ such that $Imp^{x_c}(X_m)\neq 0$, then
the variable is context-dependent but the opposite is unfortunately not true.
Another drawback of this approach is that in the finite case, we do not have
the guarantee that the different forests will have explored the same
conditioning sets $B$ and therefore, even assuming that the learning sample is
infinite (and therefore that all mutual informations are perfectly estimated),
we lose the guarantee that $Imp^{x_c}(X_m)\neq 0$ for a given $x_c$ implies
context-dependence.

To overcome these two problems, we propose the following new importance score to identify context-dependent variables:

\begin{equation}\label{impabs}
  Imp^{|x_c|}(X_m) \triangleq  \frac{1}{N_T} \sum_T \sum_{t\in T:v(s_t)=X_m} p(t)  |I(Y;X_m|t)-I(Y;X_m|t,X_c=x_c)| 
\end{equation}

This score is meant to be computed from a forest of totally randomized trees
built from all samples, not including the context variable among the inputs. At
each node $t$ where the variable $X_m$ is used to split, one needs to compute
the absolute value of the difference between the mutual information between $Y$
and $X_m$ estimated from all samples reaching that node and the mutual
information between $Y$ and $X_m$ estimated only from the samples for which
$X_c=x_c$. The same forest can then be used to compute $Imp^{|x_c|}(X_m)$ for
all $x_c$. A variable $X_m$ is then declared context-dependent as soon
as there exists an $x_c$ such that $Imp^{|x_c|}(X_m)>0$.

Let us show that this measure is sound. In asymptotic conditions,
i.e., with an infinite number of trees, one can show from
(\ref{impabs}) that $Imp^{|x_c|}(X_m)$ becomes:
{
\begin{equation}\label{eqn:impabsinfinite}
	\begin{split}
	Imp^{|x_c|}(X_m)  = &\sum_{k=0}^{p-1} \frac{1}{C^k_p} \frac{1}{p-k} \sum_{B\in{\cal P}_k(V^{-m})} \sum_{b\in {\cal B}} P(B=b)  \\
	&\hspace{-1.5em}\hookrightarrow  \left | I(Y;X_m|B=b)  - I(Y;X_m|B=b;X_c=x_c) \right | .
	\end{split}
\end{equation}}

Asymptotically, this measure has now the very desirable property to not miss any context-dependent variable as formalized in the next theorem:

\begin{theorem}\label{theo2}
A variable $X_m \in V$ is context-independent to $Y$ with respect to $X_c$ iff $Imp^{|x_c|}(X_m)=0$ for all $x_c$.
\end{theorem}

\begin{proof}
	See Appendix \ref{app:theo2}.
\end{proof}

Given that the absolute differences are computed at each tree node, this measure also continues to imply context-dependence in the case of finite forests and infinite learning sample size. The only difference with the infinite forests is that only some conditionings $B$ and values $b$ will be tested and therefore one might miss the conditionings that are needed to detect some context-dependent variables.

\subsection{Characterizing context-dependent variables}\label{sec:charact}

Besides identifying context-dependent variables, one would want to characterize
their dependence with the context as precisely as possible. As discussed in
Section \ref{sec:contextual-relevance-trees}, irrelevant variables (i.e, such that
$Imp(X_m)=0$) that are detected as context-dependent do not need much effort to
be characterized since the context can only increase their importance. All
these variables are therefore context-complementary.

Identifying the context-complementary and context-redundant variables among the
relevant variables that are also context-dependent can in principle be done by
simply comparing the absolute value of $Imp^{x_c}(X_m)$ with
$Imp^{|x_c|}(X_m)$, as formalized in the following theorem:

\begin{theorem}\label{theo3} If $|Imp^{x_c}(X_m)| = Imp^{|x_c|}(X_m)$ for a context-dependent variable $X_m$, then $X_m$ is context-complementary if $Imp^{x_c}(X_m) < 0$ and context-redundant if $Imp^{x_c}(X_m) > 0$.
\end{theorem}

\begin{proof}
	The absolute value of a sum is less than or equal the sum of the absolute value of each terms. The equality is only verified when all terms are of the same sign. Therefore, the sign of $Imp^{x_c}(X_m)$ indicates the sign of all terms and thus verify either the context-complementarity if all terms are negative or the context-redundancy if all terms are positive.
\end{proof}

This result allows to identify easily the context-complementary and
context-redundant variables. In addition, if, for a context-redundant variable
$X_m$, we have $Imp^{|x_c|}(X_m)=Imp^{x_c}(X_m)=Imp(X_m)$, then this variable is irrelevant in
the context $x_c$.

Then it remains to characterize the context-dependent variables that are neither
context-complementary nor context-redundant. It would be interesting to be able
to also characterize them according to some sort of average effect of the
context on these variables. Similarly as the common use of importance $Imp(X_m)$ to
rank variables from the most to the less important, we propose here to use the
importance $Imp^{x_c}(X_m)$ to characterize the average global effect of
context $x_c$ on the variable $X_m$. Given the asymptotic formulation of this
importance in Equation~(\ref{mdidiffcontext}), a negative value of
$Imp^{x_c}(X_m)$ means that $X_m$ is essentially complementary with the
context: in average over all conditionings, it brings more information about
$Y$ in context $x_c$ than when ignoring the context. Conversely, a positive
value of $Imp^{x_c}(X_m)$ means that the variable is essentially redundant with
the context: in average over all conditionings, it brings less information
about $Y$ than when ignoring the context. Ranking the context-dependent
variables according to $Imp^{x_c}(X_m)$ would then give at the top the
variables that are the most complementary with the context and at the bottom
the variables that are the most redundant.

Note that, like $Imp^{|x_c|}(X_m)$, it is preferable to estimate
$Imp^{x_c}(X_m)$ by using the following formula rather than to estimate it from
two forests by subtracting $Imp(X_m)$ and $Imp(X_m|X_c=x_c)$:
\begin{equation}\label{eqn:impxcnodebynode}
 Imp_s^{x_c}(X_m)=  \frac{1}{N_T} \sum_T \sum_{t\in T:v(s_t)=X_m} p(t)  (I(Y;X_m|t)-I(Y;X_m|t,X_c=x_c))
\end{equation}

This estimation method has the same asymptotic form as
$Imp(X_m)-Imp(X_m|X_c=x_c)$ given in Equation~(\ref{mdidiffcontext}) but, in
the finite case, it ensures that the same conditionings are used for both
mutual information measures. Note that in some applications, it is interesting
also to have a global measure of the effect of the context. A natural adaptation
of (\ref{eqn:impxcnodebynode}) to obtain such global measure is as follows:
\begin{equation*}
Imp^{X_c}(X_m) \triangleq  \frac{1}{N_T} \sum_T \sum_{t\in T:v(s_t)=X_m} p(t)   (I(Y;X_m|t)-I(Y;X_m|t,X_c))
\end{equation*}

which, in asymptotic sample and ensemble of trees size conditions, gives the following formula:
\begin{equation*}
 Imp^{X_c}(X_m) = \sum_{k=0}^{p-1} \frac{1}{C_k^p} \frac{1}{p-k} \sum_{B\in{\cal P}_k(V^{-m})}   (I(Y;X_m|B)-I(Y;X_m|B,X_c)).
\end{equation*}

If $Imp^{X_c}(X_m)$ is negative then the context variable $X_c$ makes variable
$X_m$ globally more informative ($X_c$ and $X_m$ are complementary with
respect to $Y$ and $V$). If $Imp^{X_c}(X_m)$ is positive, then the context
variable $X_c$ makes variable $X_m$ globally less informative ($X_c$ and $X_m$
are redundant with respect to $Y$ and $V$).

\subsection{In practice}
\label{sec:in-practice}

As a recipe when starting a context analysis, we suggest first to build a single forest using all input variables $X_m$ (but not the context $X_c$) and then to compute from this forest all importances defined in the previous section: the global importances $Imp(X_m)$ and the different contextual importances, $Imp_s^{x_c}(X_m)$, $Imp^{|x_c|}(X_m)$, and $Imp^{X_c}(X_m)$, for all variables $X_m$ and context values $x_c$.

Second, variables satisfying the context-dependence criterion, i.e., such that 

\noindent$Imp^{|x_c|}(X_m)>0$ for at least one $x_c$, can be identified from the other variables. Among context-dependent variables, an equality between $|Imp_s^{x_c}(X_m)|$ and $Imp^{|x_c|}(X_m)$ highlights that the context-dependent variable $X_m$ is either context-complementary or context-redundant (in $x_c$) depending on the sign of $Imp_s^{x_c}(X_m)$. Finally, the remaining context-dependent variables can be ranked according to $Imp_s^{x_c}(X_m)$ (or $Imp^{X_c}(X_m)$ for a more global analysis).

Note that, because mutual informations will be estimated from finite training sets, they will be generally non zero even for independent variables, leading to false positives in the identification of context-dependent variables. In practice, one could instead identify context-dependent variables by using a test $Imp^{|x_c|}(X_m)> \epsilon$ where $\epsilon$ is some cut-off value greater than 0. In practice, the determination of this cut-off can be very difficult. In our experiments, we propose to turn the importances $Imp^{|x_c|}(X_m)$ into $p$-values by using random permutations. More precisely, 1000 scores $Imp^{|x_c|}(X_m)$ will be estimated by randomly permuting the values of the context variable in the original data (so as to simulate the null hypothesis corresponding to a context variable fully independent of all other variables). A $p$-value will then be estimated by the proportion of these permutations leading to a score $Imp^{|x_c|}(X_m)$ greater than the score obtained on the original dataset.

   \begin{table}[htbp]
   	\centering
  {\scriptsize
   \centering
\begin{tabular}{c|ccc|c} \hline
$X_c$ & $X_1$ & $X_2$ & $X_3$ & $Y$\\ \hline
    0   &   0\g &   0   &   0   & \g2\\
    0   &   0\g &   0   &   1   & \g2\\
    0   &   0\g &   1   &   0   & \g2\\
    0   &   0\g &   1   &   1   & \g2\\
    0   &   1   & \g0   &   0   & \g0\\
    0   &   1   & \g0   &   1   & \g0\\
    0   &   1   & \g1   &   0   & \g1\\
    0   &   1   & \g1   &   1   & \g1\\ \hline
    1   &   0\g &   0   &   0   & \g2\\
    1   &   0\g &   0   &   1   & \g2\\
    1   &   0\g &   1   &   0   & \g2\\
    1   &   0\g &   1   &   1   & \g2\\
    1   &   1   &   0   & \g0   & \g0\\
    1   &   1   &   0   & \g1   & \g1\\
    1   &   1   &   1   & \g0   & \g0\\
    1   &   1   &   1   & \g1   & \g1\\ \hline
\end{tabular}}
\captionof{table}{Problem 1: Values of $X_c$, $X_1$, $X_2$, $X_3$, $Y$.}
\label{table:data1}
  \end{table}

   \begin{table}[htbp]\centering
{\scriptsize
    \centering
\begin{tabular}{l|lll} \hline
        & $X_1$ & $X_2$ & $X_3$ \\ \hline
$Imp(X_m)$        & 1.0 & 0.125 & 0.125\\ 
  $Imp(X_m|X_c=0)$    & 1.0 & 0.5 & 0.0\\ 
  $Imp(X_m|X_c=1)$    & 1.0 & 0.0 & 0.5\\ 
$Imp^{|0|}(X_m)$  & 0.0 & 0.375 & 0.125 \\ 
  $Imp^{0}(X_m)$    & 0.0 & -0.375& 0.125 \\ 
  $Imp^{|1|}(X_m)$  & 0.0 & 0.125 & 0.375 \\ 
  $Imp^{1}(X_m)$  & 0.0 & 0.125 & -0.375 \\ 
  $Imp^{X_c}(X_m)$  & 0.0 & -0.125  & -0.125  \\ \hline 
\end{tabular}}
\captionof{table}{Problem 1: Variable importances as computed analytically using asymptotic formulas. Note that $X_1$ is context-independent
  and $X_2$ and $X_3$ are context-dependent.}\label{table:impexp1} 

  \end{table}

\begin{table*}[t]
{\scriptsize
\begin{tabular}{l|llllllll} \hline
						& $X_1$ & $X_2$ & $X_3$	& $X_4$ & $X_5$ & $X_6$ & $X_7$ & $X_8$ \\ \hline
$Imp(X_m)$				&0.5727		&0.7514		&0.5528		&0.687 		&0.1746		&0.0753		&0.1073		&0.0  \\
$Imp(X_m|X_c=0)$		&0.4127 	&0.5815		&0.5312		&0.5421		&0.6566		&0.2258		&0.372		&0.0\\
$Imp(X_m|X_c=1)$		&0.6243		&0.8057		&0.5577		&0.7343		&0.0		&0.0		&0.0		&0.0\\ \hline 
$Imp^{|0|}(X_m)$ 		&0.2263		&0.2431		&0.1181		&0.2241		&0.4139		&0.1961		&0.2861		&0.0\\
$Imp^{|1|}(X_m)$ 		&0.0987		&0.0611		&0.021		&0.0736		&0.1746		&0.0753		&0.1073		&0.0\\ \hline
$Imp^{0}(X_m)$			&0.2179		&0.2422		&0.1111		&0.2190		&-0.3839 	&-0.1389	&-0.2346	&0.0\\ 
$Imp^{1}(X_m)$			&-0.0516	&-0.0543	&-0.0049	&-0.0473	&0.1746		&0.0753		&0.1073		&0.0\\ \hline 
\end{tabular}}

\caption{Problem 2: Variable importances as computed analytically using the asymptotic formulas for the different importance measures. 
}\label{table:impexp2}

\end{table*}

\subsection{Generalization to other impurity measures}\label{sec:generalisation}

All our developments so far have assumed a categorical output $Y$ and
the use of Shannon's entropy as the impurity measure. Our framework
however can be carried over to other impurity measures and thus in
particular also to a numerical output $Y$. Let us define a generic
impurity measure $i(Y|t)\geq 0$ that assesses the impurity of the
output $Y$ at a tree node $t$. The corresponding impurity decrease at
a tree node is defined as:
\begin{equation}\label{eqn:generic_reduction}
G(Y;X_m|t) = i(Y|t) - \sum_{{x_m}\in{\cal X}_m} p(t_{x_m})i(Y|t_{x_m})
\end{equation}
with $t_{x_m}$ denoting the successor node of $t$ corresponding to value
$x_m$ of $X_m$. By analogy with conditional entropy and mutual
information, let us define the population based measures $i(Y|B)$ and
$G(Y;X_m|B)$ for any subset of variables $B\subseteq V$ as follows:
\begin{eqnarray*}
  i(Y|B) &=&\sum_b P(B=b) i(Y|B=b)\\
  G(Y;X_m|B) &= &i(Y|B)-i(Y|B,X_m),
\end{eqnarray*}
where the first sum is over all possible combinations $b$ of values
for variables in $B$. Now, substituting mutual information $I$ for the
corresponding impurity decrease measure $G$, all our results above
remain valid, including Theorems 1, 2, and 3 (proofs are omitted for
the sake of space). It is important however to note that this
substitution changes the notions of both variable relevance and
context-dependence. Definition \ref{defcontextdependent} indeed
becomes:
\begin{definition}\label{defcontextdependentgeneral}
 A variable $X_m\in V$ is \textbf{context-dependent to $Y$ with respect
   to $X_c$} iff there exists a subset $B \subseteq V^{-m}$ and some
 values $x_c$ and $b$ such that $$G(Y;X_m|B=b,X_c=x_c) \ne
 G(Y;X_m|B=b).$$
\end{definition}
When $Y$ is numerical, a common impurity measure is variance, which
defines $i(Y|t)$ as the empirical variance $\mbox{var}[Y|t]$ computed
at node $t$. The corresponding $G(X_m;Y|B=b)$ and $G(X_m;Y|B=b,X_c=x_c)$
in Definition (5) are thus defined respectively as:
\begin{eqnarray*}
  &var\{Y|B=b\}-\mathbb{E}_{X_m|B=b}\{var\{Y|X_m,B=b\}\}
\end{eqnarray*}
and
\begin{equation*}
   var\{Y|B=b,X_c=x_c\} -\mathbb{E}_{X_m|B=b,X_c=x_c}\{var\{Y|X_m,B=b,X_c=x_c\}\} .
\end{equation*}
We will illustrate the use of our framework in a regression
 setting with this measure in the next section.

\section{Experiments}
\label{sec:contextexperiments}

We first illustrate the different importance measures defined in
Section~\ref{sec:contextual-relevance-trees} on two artificial problems and
then exploit them on two real bio-medical datasets.

\subsubsection*{Problem 1.}

The purpose of this first problem is to illustrate the different
measures introduced earlier. This artificial problem is defined by
three binary input variables $X_1$, $X_2$, and $X_3$, a ternary output
$Y$, and a binary context $X_c$. All samples are enumerated in
Table~\ref{table:data1} and are supposed to be equiprobable. By
construction, the output $Y$ is defined as $Y=2$ if $X_1=0$, $Y=X_2$
if $X_c=0$ and $X_1=1$, and $Y=X_3$ if $X_c=1$ and $X_1=1$.

Table~\ref{table:impexp1} reports all importance scores for
the three inputs. These scores were computed analytically using the
asymptotic formulas, not from actual experiments. Considering the
global importances $Imp(X_m)$, it turns out that all variables are
relevant, with $X_1$ clearly the most important variable and $X_2$ and
$X_3$ of smaller and equal importances. According to $Imp^{|0|}(X_m)$
and $Imp^{|1|}(X_m)$, $X_1$ is a context-independent variable, while
$X_2$ and $X_3$ are two context-dependent variables. This result is as
expected given the way the output is defined. For $X_2$ and $X_3$, we
have furthermore $Imp^{|x_c|}(X_m)=|Imp^{|x_c|}(X_m)|$ for both values
of $x_c$. $X_2$ is therefore context-complementary when $X_c=0$ and
context-redundant when $X_c=1$. Conversely, $X_3$ is context-redundant
when $X_c=0$ and context-complementary when $X_c=1$. $X_2$ is
furthermore irrelevant when $X_c=1$ (since
$Imp^{1}(X_2)=Imp^{|1|}(X_2)=Imp(X_2)$) and $X_3$ is irrelevant when
$X_c=0$ (since $Imp^{0}(X_3)=Imp^{|0|}(X_3)=Imp(X_3)$). The values of
$Imp^{X_c}(X_2)$ and $Imp^{X_c}(X_3)$ suggest that these two variables
are in average complementary.

\subsubsection*{Problem 2.}
\label{sec:s4.2}

This second experiment is based on an adaptation of the digit recognition
problem initially proposed in \cite{breiman1984classification} and reused in
\cite{louppe2013understanding} (see Appendix \ref{app:digit} for a detailed description). The original problem contains 7 binary
variables ($X_1$,\ldots,$X_7$) and the output $Y$ takes its values in $\{0,1,\ldots,9\}$. Each input
represents the on-off status of one lightning segment of a
seven-segment indicator and is determined univocally from $Y$. To create an
artificial (binary) context, we created two copies of this dataset, the first
one corresponding to $X_c=0$ and the second one to $X_c=1$. The first dataset
was unchanged, while in the second one variables $X_5$, $X_6$, and $X_7$ were
turned into irrelevant variables. In addition, we included a new variable
$X_8$, irrelevant by construction in both contexts. The final dataset contains
320 samples, 160 in each context.

Table~\ref{table:impexp2} reports possible importance scores for all the
inputs. Again, these scores were computed analytically using the asymptotic
formulas. As expected, variable $X_8$ has zero importance in all cases. Also as
expected, variables $X_5$, $X_6$, and $X_7$ are all context-dependent
($Imp^{|x_c|}(X_m)>0$ for all of them). They are context-redundant (and even
irrelevant) when $X_c=1$ and complementary when
$X_c=0$. More surprisingly, variables $X_1$, $X_2$, $X_3$, and $X_4$ are also
context-dependent, even if their distribution is independent from the
context. This is due to the fact that these variables are complementary with
variables $X_5$, $X_6$, and $X_7$ for predicting the output. Their
context-dependence is thus a consequence of the context-dependence of $X_5$,
$X_6$, $X_7$. $X_1$, $X_2$, $X_3$, and $X_4$ are all almost redundant when
$X_c=0$ and complementary when $X_c=1$, which expresses the fact that they
provide more information about the output when $X_5$, $X_6$ and $X_7$ are
irrelevant ($X_c=1$) and less when $X_5$, $X_6$, and $X_7$ are relevant
($X_c=0$). Nevertheless, $X_8$ remains irrelevant in every situation.

\subsubsection*{Problem 3.} \label{sec:context-problem3}

As a third experiment, we consider bio-medical data from the \textit{Primary tumor} dataset. The objective of the
corresponding supervised learning problem is to predict the location
of a primary tumor in patients with metastases. It was downloaded from
the UCI repository \citep{Lichman2013uci} and was collected by the
University Medical Center in Ljubljana, Slovenia. We restrict our
analysis to 132 samples without missing values. Patients are
described by 17 discrete clinical variables (listed in the first
column of Table~\ref{table:forest-topresult}) and the output is chosen among 22
possible locations. For this analysis, we use the patient gender as the
context variable.

Table~\ref{table:forest-topresult} reports variable importances computed
with 1000 totally randomized trees and their corresponding p-values.
According to the p-values of $Imp^{|x_c|}(X_m)$, two variables are
clearly emphasized for each context: importances of \textit{histologic-type} and
\textit{neck} both significantly decrease in the first context ($female$) and
importances of \textit{peritoneum} and \textit{abdominal} both significantly
decrease in the second context ($male$). While the biological relevance of
these finding needs to be verified, such dependences could not have
been highlighted from standard random forests importances.

Note that the same importances computed using the asymptotic formulas
are provided in Table~\ref{table:asymptotic-realresult}.  Importance
values are very similar, highlighting that finite forests provide good
enough estimates for this problem.

\begin{table*}[htbp]
	{\tiny
		\begin{tabular}{c|c|c|cc|cc:cc|cc:cc} \hline
			& &             $Imp(X_m)$  & \multicolumn{2}{c|}{$Imp(X_m|X_c=x_c)$}   &  \multicolumn{4}{c|}{$Imp^{|x_c|}(X_m)$} & \multicolumn{4}{c}{$Imp_s^{x_c}(X_m)$}\\
			m &             & -     & $x_c=0$   & $x_c=1$             &  $x_c=0$ & pval   & $x_c=1$ & pval  & $x_c=0$ & pval & $x_c=1$ & pval \\ \hline
			0 & age              & 0.2974 & 0.2942 & 0.2900 & 0.1505 &   0.899 & 0.1717 &   0.417 &  0.0032  &   0.938 &  0.0074 &   0.846 \\
			1 & histologic-type  & 0.3513 & 0.1354 & 0.4005 & 0.2265 &\g 0.000 & 0.1183 &   0.121 &  0.2159  &\g 0.000 & -0.0492 &   0.331 \\
			2 & degree-of-diffe  & 0.4415 & 0.3725 & 0.4070 & 0.1827 &   0.680 & 0.1724 &   0.689 &  0.0690  &   0.102 &  0.0345 &   0.398 \\
			3 & bone             & 0.2452 & 0.2342 & 0.2220 & 0.1088 &   0.396 & 0.0845 &   0.904 &  0.0110  &   0.717 &  0.0232 &   0.410 \\
			4 & bone-marrow      & 0.0188 & 0.0190 & 0.0131 & 0.0128 &   0.892 & 0.0105 &   0.980 & -0.0001  &   0.994 &  0.0057 &   0.682 \\
			5 & lung             & 0.1677 & 0.1837 & 0.1420 & 0.1134 &   0.448 & 0.1079 &   0.397 & -0.0160  &   0.605 &  0.0257 &   0.373 \\
			6 & pleura           & 0.1474 & 0.1132 & 0.1127 & 0.0613 &   1.000 & 0.1026 &   0.097 &  0.0342  &   0.179 &  0.0348 &   0.165 \\
			7 & peritoneum       & 0.3171 & 0.2954 & 0.2084 & 0.0939 &   0.968 & 0.1516 &\g 0.000 &  0.0216  &   0.710 &  0.1087 &\g 0.000\\
			8 & liver            & 0.2300 & 0.1844 & 0.2784 & 0.0888 &   0.966 & 0.1382 &   0.053 &  0.0456  &   0.134 & -0.0483 &   0.100 \\
			9 & brain            & 0.0466 & 0.0334 & 0.0566 & 0.0403 &   0.173 & 0.0279 &   0.814 &  0.0131  &   0.693 & -0.0101 &   0.751 \\
			10 & skin            & 0.0679 & 0.0310 & 0.0786 & 0.0426 &   0.922 & 0.0420 &   0.841 &  0.0369  &   0.107 & -0.0107 &   0.663 \\
			11 & neck            & 0.2183 & 0.0774 & 0.2255 & 0.1562 &\g 0.000 & 0.0710 &   0.575 &  0.1409  &\g 0.000 & -0.0071 &   0.764 \\
			12 & supraclavicular & 0.1701 & 0.1807 & 0.1344 & 0.0942 &   0.379 & 0.0738 &   0.884 & -0.0106  &   0.695 &  0.0357 &   0.136 \\
			13 & axillar         & 0.1339 & 0.1236 & 0.0846 & 0.0748 &   0.214 & 0.0663 &   0.388 &  0.0103  &   0.795 &  0.0493 &   0.194 \\
			14 & mediastinum     & 0.1826 & 0.1752 & 0.1613 & 0.1129 &   0.266 & 0.0867 &   0.853 &  0.0074  &   0.767 &  0.0213 &   0.404 \\
			15 & abdominal       & 0.2558 & 0.2883 & 0.1512 & 0.1419 &   0.139 & 0.1526 &\g 0.028 & -0.0325  &   0.368 &  0.1046 &\g 0.003\\ \hline
	\end{tabular}}
	\caption{Problem 3: Importances as computed with a forest of 1000 totally randomized trees. The context is defined by the binary context feature \textit{Sex} ($Sex=0$ denotes \textit{female} and $Sex=1$ denotes \textit{male}). \textit{P-values} were estimated using 1000 permutations of the context variable. Grey cells highlight \textit{p-values} under the 0.05 threshold.}
	\label{table:forest-topresult}
\end{table*}

\subsubsection*{Problem 4.} \label{sec:context-problem4}

\begin{figure*}[h]
\centering

\subfloat[$Imp^{|x_c=Mesenchymal|}$]{\label{imp4-a}\includegraphics[width=0.35\linewidth]{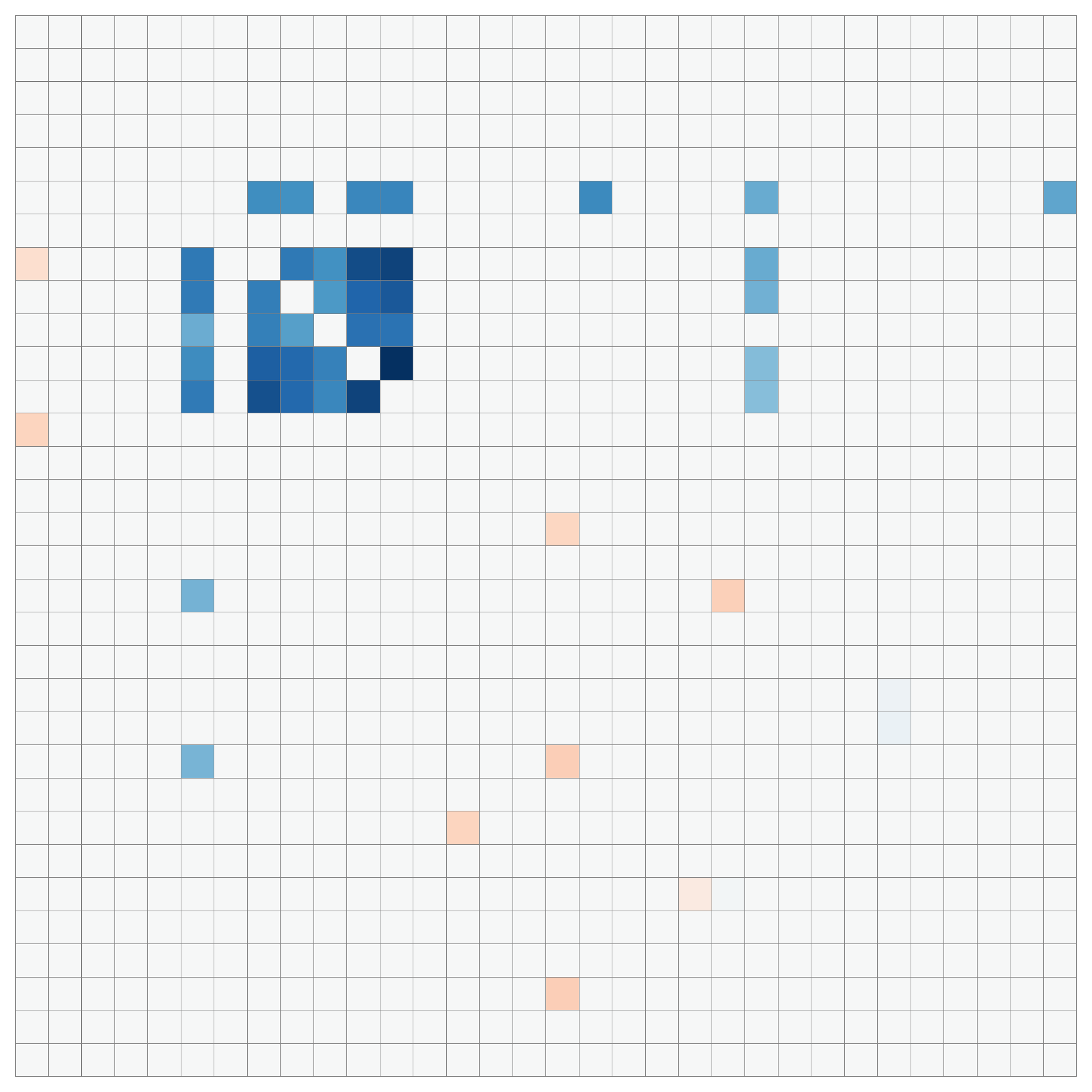}}\hspace{2em}
\subfloat[$Imp^{|x_c=Proneural|}$]{\label{imp4-b}\includegraphics[width=0.35\linewidth]{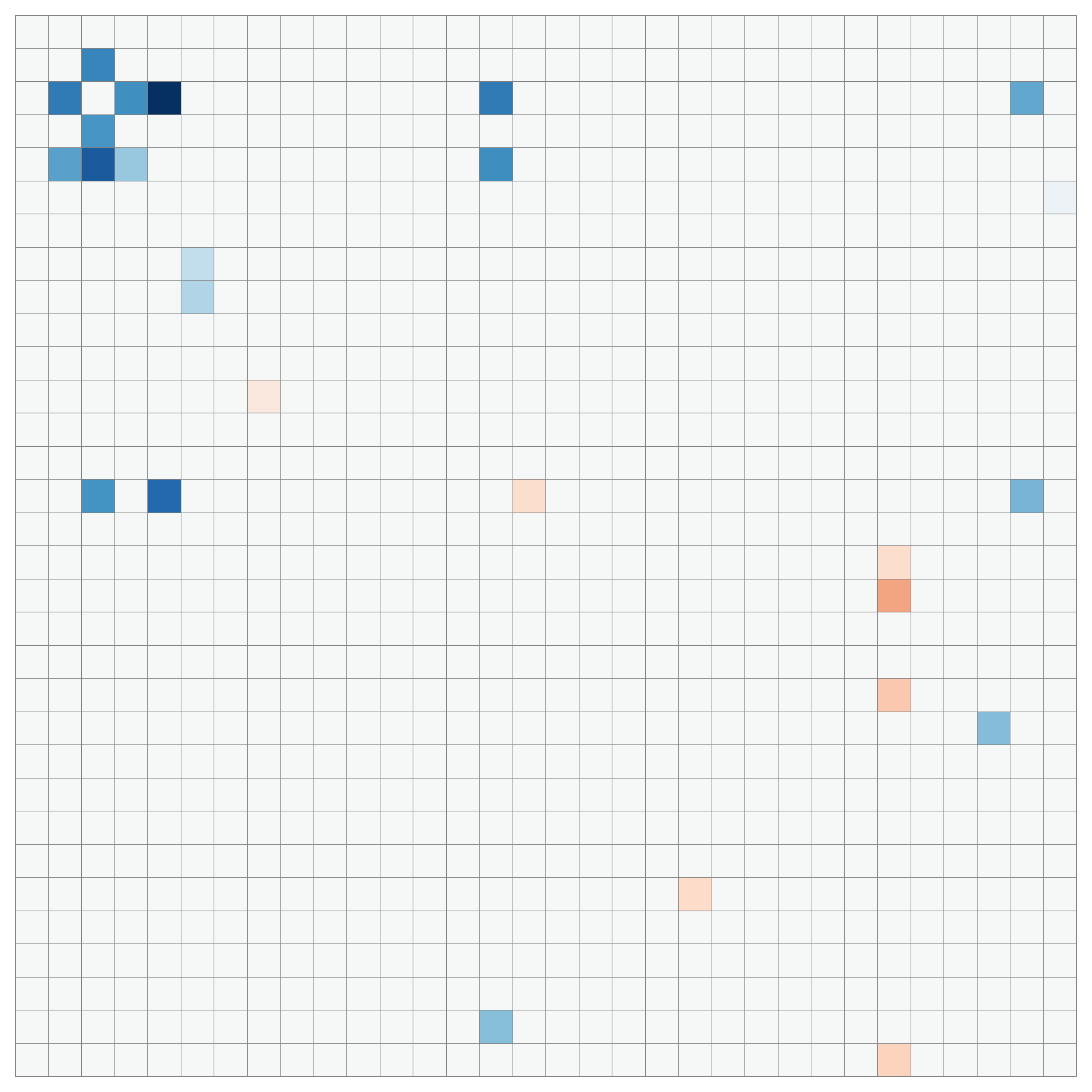}}\includegraphics[width=0.067\linewidth]{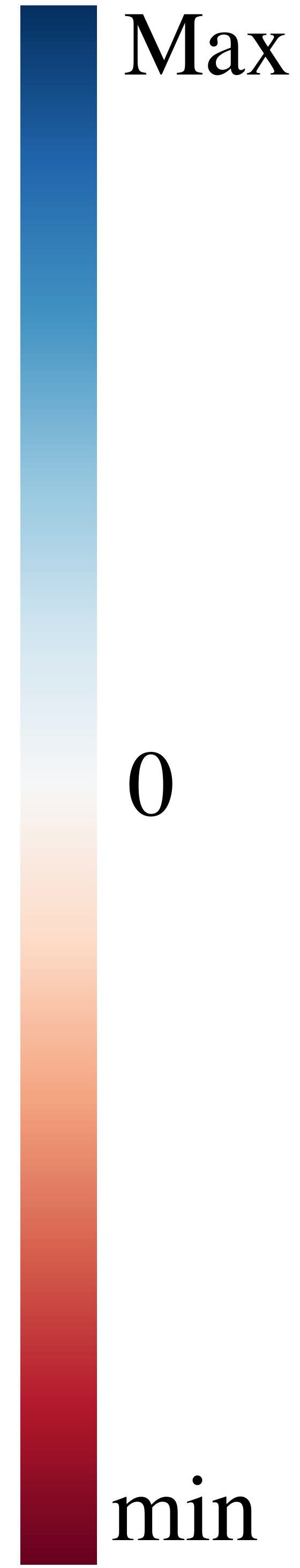}\\ 
\subfloat[$Imp^{x_c=Mesenchymal}$]{\label{imp4-c}\includegraphics[width=0.35\linewidth]{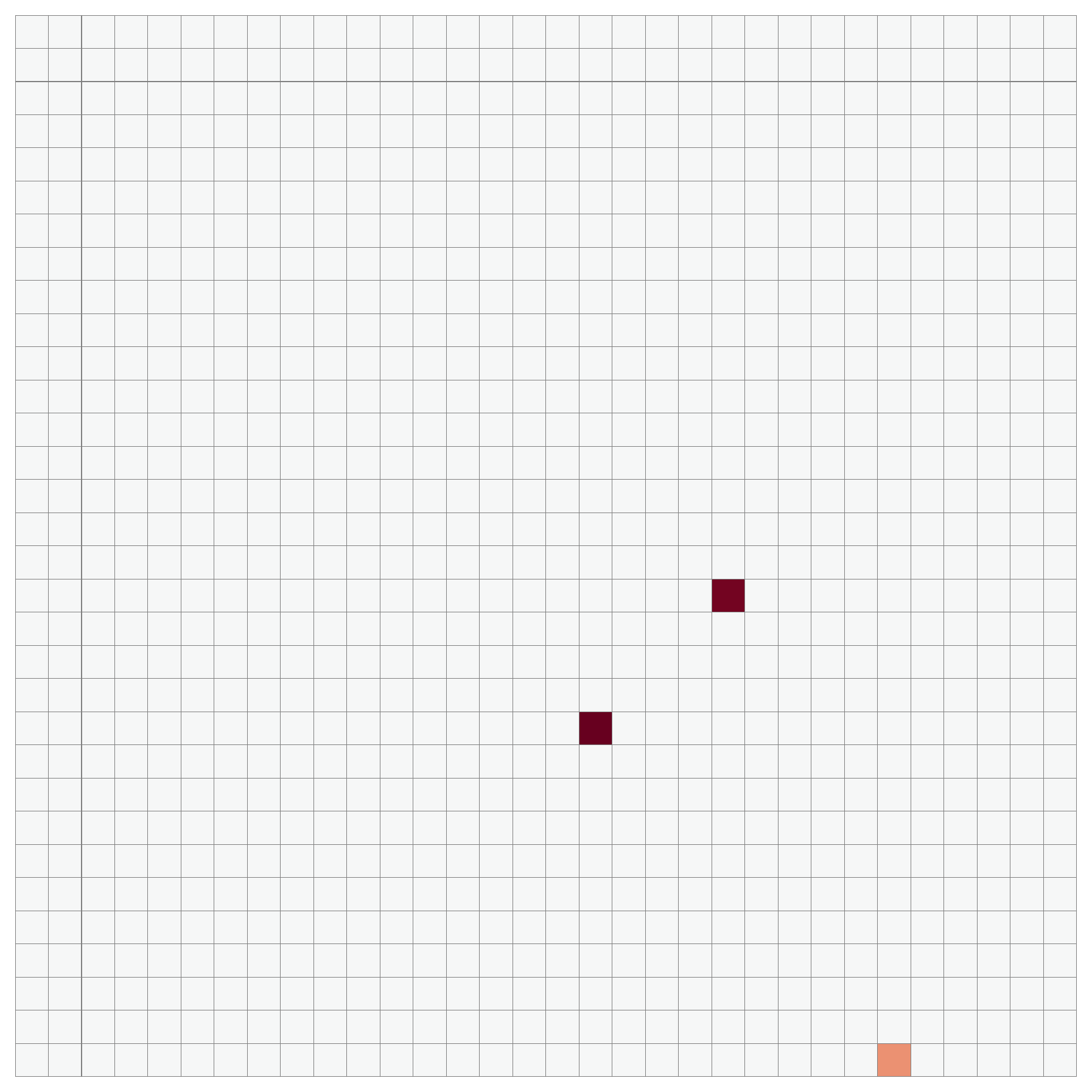}}\hspace{2em}
\subfloat[$Imp^{x_c=Proneural}$]{\label{imp4-d}\includegraphics[width=0.35\linewidth]{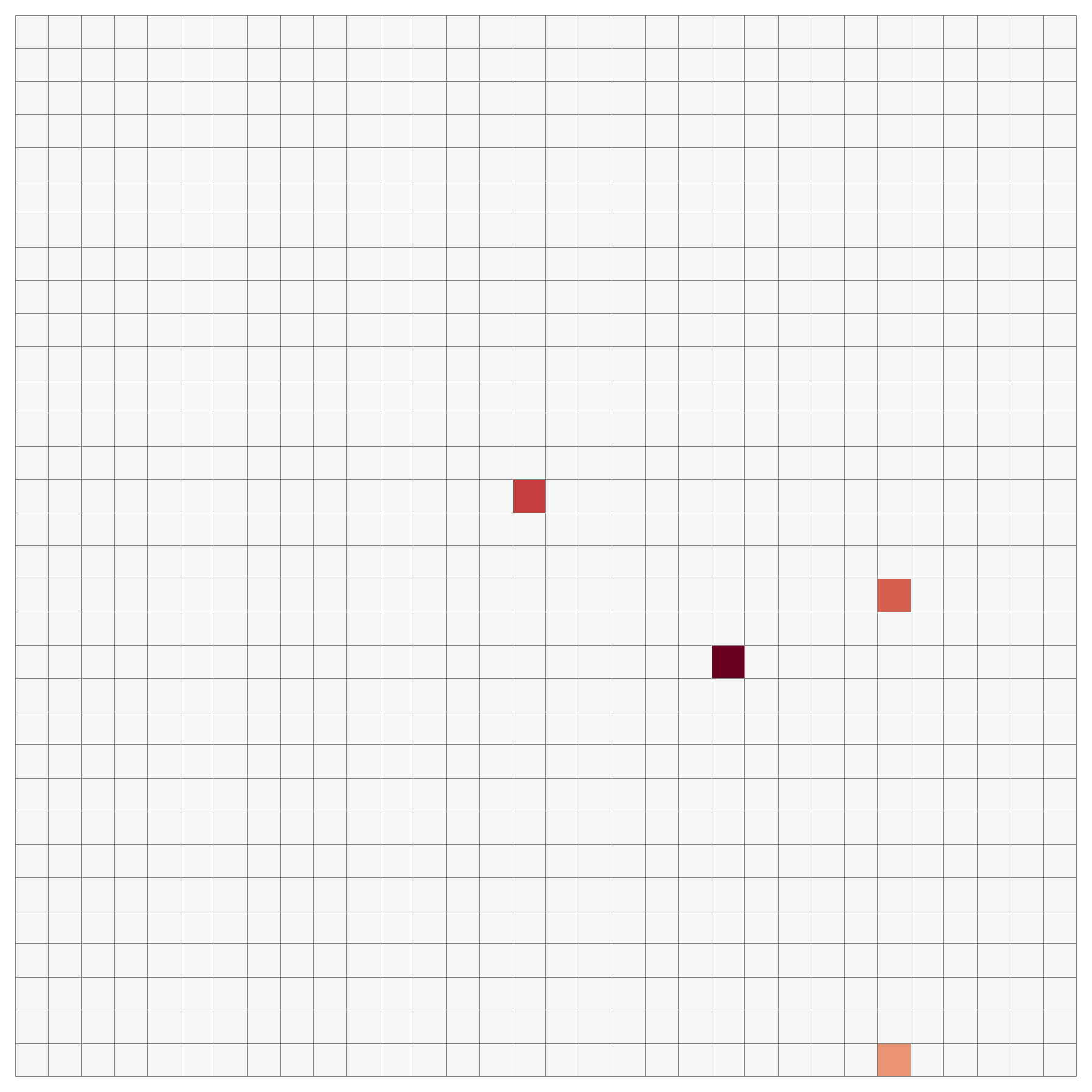}} 
\includegraphics[width=0.067\linewidth]{images/images_context/colorbar_legend.pdf}

\caption{Results for Problem 4. Each matrix represents significant
	context-dependent gene-gene interactions as found using $Imp^{|x_c|}$ in
	(a)(b) and $Imp^{x_c}$ in (c)(d), in GBM sub-type Mesenschymal in (a)(c) and
	Proneural in (b)(d). In (a) and (b), cells are colored according to
	$Imp_s^{x_c}$. In (c) and (d), cells are colored according to
	$Imp^{x_c}$. Positive (resp negative) values are in blue (resp. red) and
	highlight context-redundant (resp. context-complementary)
	interactions. Higher absolute values are darker.}
\label{fig:imp4}
\end{figure*}

As a last experiment, we consider a publicly available brain cancer
gene expression dataset \citep{verhaak2010integrated}. This dataset
collects measurements of mRNA expression levels of 11861 genes in 220
tissue samples from patients suffering from glioblastoma multiforme
(GBM), the most common form of malignant brain cancer in
adults. Samples are classified into four GBM sub-types: Classical,
Mesenchymal, Neural and Proneural. The interest of this dataset is to
identify the genes that play a central role in the development and
progression of the cancer and thus improve our understanding of this
disease. In our experiment, our aim is to exploit importance scores
to identify interactions between genes that are significantly affected
by the cancer sub-type considered as our context variable. This
dataset was previously exploited by \citet{mohan2014node}, who used it
to test a method based on Gaussian graphical models for detecting
genes whose global interaction patterns with all the other genes vary
significantly between the subtypes. This latter method can be
considered as gene-based, while our approach is link-based.

Following \citep{mohan2014node}, we normalized the raw data
using Multi-array Average (RMA) normalization. Then, the data was corrected for
batch effects using the software \textrm{ComBat} \citep{johnson2007adjusting}
and then $log_2$ transformed. Following \citep{mohan2014node}, we focused our
analysis on only two GBM sub-types, Proneural (57 tissue samples) and
Mesenchymal (56 tissue samples), and on a particular set of 32 genes, which are
all genes involved in the TCR signaling pathway as defined in the Reactome
database \citep{matthews2009reactome}. The final dataset used in the
experiments below thus contains 113 samples, 57 and 56
for both context values respectively, and 32 variables.

To identify gene-gene interactions affected by the context, we performed a
contextual analysis as described in Section
\ref{sec:contextual-relevance-trees} for each gene in turn, considering each
time a particular gene as the target variable $Y$ and all other genes as the
set of input variables $V$. This procedure is similar to the procedure adopted
in the Random forests-based gene network inference method called GENIE3
\citep{huynh2010inferring}, that was the best performer in the DREAM5 network
inference challenge \citep{marbach2012wisdom}. Since gene expressions are
numerical targets, we used variance as the impurity measure (see Section
\ref{sec:generalisation}) and we built ensembles of 1000
totally randomized trees in all experiments.

The matrices in Figure~\ref{fig:imp4} highlight context-dependent interactions
found using different importance measures (detailed below). A cell $(i,j)$ of
these matrices corresponds to the importance of gene $j$ when gene $i$ is the
output (the diagonal is irrelevant). White cells correspond to non significant
context-dependencies as determined by random permutations of the context
variable, using a significance level of 0.05. Significant context-dependent
interactions in Figures~\ref{fig:imp4}(a) and (b) were determined using the
importance $Imp^{|x_c|}$ defined in (\ref{impabs}), which is the measure we
advocate in this paper. As a baseline for comparison, Figures~\ref{fig:imp4}(c)
and (d) show significant interactions as found using the more straightforward
score $Imp^{x_c}$ defined in (\ref{mdidiffcontext}). In
Figures~\ref{fig:imp4}(a) and (b) (resp. (c) and (d)), significant cells are
colored according to the value of $Imp_s^{x_c}$ defined in
(\ref{eqn:impxcnodebynode}). In Figures~\ref{fig:imp4}(c) and (d), they are
colored according to the value of $Imp^{x_c}$ in (\ref{mdidiffcontext})
instead. Blue (resp. red) cells correspond to positive (resp. negative) values
of $Imp^{x_c}$ or $Imp_s^{x_c}$ and thus highlight context-redundant
(resp. context-complementary) interactions. The darker the color, the higher
the absolute value of $Imp^{x_c}$ or $Imp_s^{x_c}$.

Respectively 49 and 26 context-dependent interactions are found in
Figures~\ref{fig:imp4}(a) and (b). In comparison, only 3 and 4 interactions are
found respectively in Figures~\ref{fig:imp4}(c) and (d) using the more
straightforward score $Imp^{x_c}$. Only 1 interaction is common between
Figures~\ref{fig:imp4}(a) and (c), while 3 interactions are common between
Figures~\ref{fig:imp4}(b) and (d). The much lower sensitivity of $Imp^{x_c}$
with respect to $Imp^{|x_c|}$ was expected given the discussions in
Section~\ref{sec:identification}. Although more straightforward, the score $Imp^{x_c}(X_m)$,
defined as the difference $Imp(X_m)-Imp(X_m|X_c=x_c)$, indeed suffers from the
fact that $Imp(X_m)$ and $Imp(X_m|X_c=x_c)$ are estimated from different
ensembles and thus do not explore the same conditionings in finite
setting. $Imp^{x_c}$ also does not have the same guarantee as $Imp^{|x_c|}$ to
find all context-dependent variables.

\section{Conclusions and future work}
\label{sec:contextconclusions}

In this chapter, our first contribution is a formal framework defining and
characterizing the dependence to a context variable of the relationship between
the input variables and the output (Section~\ref{sec:contextual-relevance}). As
a second contribution, we have proposed several novel adaptations of random
forests-based variable importance scores that implement these definitions and
characterizations and we have derived performance guarantees for these scores
in asymptotic settings (Section~\ref{sec:contextual-relevance-trees}). The
relevance of these measures was illustrated on several artificial and real
datasets (Section~\ref{sec:contextexperiments}).

There remain several limitations to our framework that we would like to address
as future works.  All theoretical derivations in Sections
\ref{sec:contextual-relevance} and \ref{sec:contextual-relevance-trees} concern
categorical input variables. It would be interesting to adapt our framework to
continuous input variables, and also, probably with more difficulty, to
continuous context variables. Finally, all theoretical derivations are based on
forests of totally randomized trees (for which we have an asymptotic
characterization). It would be interesting to also investigate non totally
randomized tree algorithms (e.g., \cite{breiman2001random}'s standard Random
Forests method) that could provide better trade-offs in finite settings.

\begin{subappendices}
	
\section{Details of Example \ref{example1} \label{app:ex:example1}}
\label{app:ex1}

\begin{table}[h]
	\centering
	\begin{tabular}{c c c | c }
		$X_1$ & $X_2$ & $X_c$ & $Y$ \\ \hline  
		0 & 0 & 0 & 0 \\
		0 & 0 & 0 & 0 \\ 
		0 & 0 & 1 & 2 \\
		0 & 0 & 1 & 3 \\ 
		0 & 1 & 0 & 2 \\
		0 & 1 & 0 & 3 \\ 
		0 & 1 & 1 & 0 \\
		0 & 1 & 1 & 0 \\ 
		1 & 0 & 0 & 1 \\
		1 & 0 & 0 & 1 \\ 
		1 & 0 & 1 & 2 \\
		1 & 0 & 1 & 3 \\ 
		1 & 1 & 0 & 2 \\
		1 & 1 & 0 & 3 \\
		1 & 1 & 1 & 1 \\
		1 & 1 & 1 & 1 \\ \hline
	\end{tabular}
	\caption{Values of $X_1$, $X_2$, $X_c$ and $Y$.}
	\label{table:full-example1}
\end{table}

\section{Proof of Theorem \ref{theo1} \label{app:thm:theo1}}
\label{app:theo1}

\begin{theorem*}
\textit{$X_c$ is irrelevant to $Y$ with respect to $V$ iff all variables in $V$ are context-independent to $Y$ with respect to $X_c$ (and $V$) and $I(Y;X_c)=0$.}
\end{theorem*}

\paragraph{Necessary condition.}
\begin{proof}
	If $X_c$ is irrelevant to $Y$ w.r.t. $V$, we have, by definition, that $I(Y;X_c|B)=0$ for all subset $B \subseteq V$. Hence, we have $I(Y;X_c)=0$ as a special case.
	
	A variable $X_m\in V$ is context-independent if for all $B \subseteq V^{-m}$ and for all $x_c \in {\cal X}_c$, $b \in {\cal B}$, we have
	\begin{equation*}
	I(Y;X_m|B=b,X_c=x_c)-I(Y;X_m|B=b)=0.
	\end{equation*}
	
	Let us proof this:
	\begin{eqnarray*}
		&& \hspace{-1cm}  I(Y;X_m|B=b,X_c=x_c)-I(Y;X_m|B=b)\\
		&=& H(Y|B=b,X_c=x_c)-H(Y|X_m,B=b,X_c=x_c) \\ 
		&& \hspace{1cm}\hookrightarrow -H(Y|B=b)+H(Y|X_m,B=b)\\
		&=& H(Y|B=b)-H(Y|X_m,B=b)\\
		&& \hspace{1cm}\hookrightarrow -H(Y|B=b)+H(Y|X_m,B=b)\\
		&=& 0,
	\end{eqnarray*}
	where $H(Y|B=b,X_c=x_c)=H(Y|B=b)$ and $H(Y|X_m,B=b,X_c=x_c)=H(Y|X_m,B=b)$ are consequences of $I(Y;X_c|B)=0$ for all $B$ if we assume that $p(B=b)\neq 0 $ ($\forall b\in {\cal B}$) and $p(X_c=x_c,B=b)\neq 0$ ($\forall x_c\in {\cal X}_c$ and $\forall b\in {\cal B}$).
\end{proof}

\paragraph{Sufficient condition.}
\begin{proof}
	If all variables are context-independent, we have that for all
	$X_m \in V$, $B \subseteq V^{-m}$, $b \in {\cal B}$, and $x_c \in {\cal X}_c$:
	$$I(Y;X_m|B=b,X_c=x_c)=I(Y;X_m|B=b).$$
	By averaging the left- and right-hand sides of this equality over $P(B,X_c)$, we get:
	$$I(Y;X_m|B,X_c)=I(Y;X_m|B).$$
	From this, one can derive \citep{louppe2013understanding}:
	$$I(Y;X_c|B,X_m)=I(Y;X_c|B).$$ Since this equality is valid for all $B$,
	including $B=\emptyset$, and all $X_m$, we have that for all $B'\subseteq V$,
	$I(Y;X_c|B')$ can be reduced to $I(Y;X_c)$, which is equal to zero by
	hypothesis. The variable $X_c$ is thus irrelevant to $Y$ with respect to $V$.
\end{proof}

\section{Proof of Theorem \ref{theo2}}
\label{app:theo2}

\begin{theorem*}
\textit{A variable $X_m \in V$ is context-independent to $Y$ with respect to $X_c$ iff $Imp^{|x_c|}(X_m)=0$ for all $x_c$.}
\end{theorem*}

\paragraph{Necessary condition.}
\begin{proof}
	
	By definition of context-independence, we have
	\begin{equation}
	\begin{split}
	I(Y;X_m|B=b,X_c=x_c) - I(Y;X_m|B=b) = 0 \\ \qquad \forall B \subseteq V^{-m}, \forall x_c \in {\cal X}_c , \forall b \in {\cal B}. \end{split}
	\end{equation}
	
	Given that each term $$\left | I(X_m;Y|B=b) -  I(X_m;Y|B=b;X_c=x_c) \right |$$ of $Imp^{|x_c|}(X_m)$ (Equation~(\ref{eqn:impabsinfinite})) is equal to $0$, the sum is thus also equal to $0$.
\end{proof}

\paragraph{Sufficient condition.}
\begin{proof}
	
	Given the definition of $Imp^{|x_c|}(X_m)$:
	
	\begin{equation}
	\begin{split}
	Imp^{|x_c|}(X_m) = \sum_{k=0}^{p-1} \frac{1}{C^k_p} \frac{1}{p-k} & \sum_{B\in{\cal P}_k(V^{-m})} \sum_{b\in {\cal B}} P(B=b) \\
	&\hspace{-3.5cm}\hookrightarrow \left | I(X_m;Y|B=b) -  I(X_m;Y|B=b;X_c=x_c) \right |,\end{split}
	\end{equation}
	appears to be a sum of positive terms (because of the absolute value). As in Theorem~\ref{theo1}, we assume that probabilities are non-null and therefore, we have that the only way to have the sum equal to zero is to have each term of the sum equal to $0$. Hence, we have $\left | I(X_m;Y|B=b) -  I(X_m;Y|B=b;X_c=x_c) \right | = 0$ for all $x_c$, $B$ and $b$ which verifies the definition of context-independence for $X_m$.
	
\end{proof}

\section{Proof of Theorem \ref{theo3}}
\label{app:theo3}

\begin{theorem*}
\textit{If $|Imp^{x_c}(X_m)| = Imp^{|x_c|}(X_m)$ for a context-dependent variable $X_m$, then $X_m$ is context-complementary if $Imp^{x_c}(X_m) < 0$ and context-redundant if $Imp^{x_c}(X_m) > 0$.}
\end{theorem*}

\begin{proof}
	The absolute value of a sum is less than or equal the sum of the absolute value of each terms. The equality is only verified when all terms are of the same sign. Therefore, the sign of $Imp^{x_c}(X_m)$ indicates the sign of all terms and thus verify either the context-complementarity if all terms are negative or the context-redundancy if all terms are positive.
\end{proof}

\section{Results for Problem \hyperref[sec:context-problem3]{3}}
\label{app:results-problem3}

\begin{table}[htbp]
	{\small
		\begin{tabular}{c|c|c|cc|cc|cc} \hline
			& &           $Imp(X_m)$ & \multicolumn{2}{c|}{$Imp(X_m|X_c=x_c)$} &  \multicolumn{2}{c|}{$Imp^{|x_c|}(X_m)$} & \multicolumn{2}{c}{$Imp_s^{x_c}(X_m)$}\\
			m & & - & $x_c=0$ & $x_c=1$ &  $x_c=0$ & $x_c=1$  & $x_c=0$ & $x_c=1$ \\ \hline
			0 & age               & 0.2958 & 0.3386 & 0.2885 & 0.1382 & 0.1505 & -0.0095  & -0.0156 \\
			1 & histologic-type   & 0.3522 & 0.1389 & 0.4366 & 0.2087 & 0.114  &  0.1988  & -0.0569 \\
			2 & degree-of-diffe   & 0.4413 & 0.4175 & 0.4208 & 0.1653 & 0.158  &  0.0561  &  0.0157 \\
			3 & bone              & 0.2429 & 0.2502 & 0.2367 & 0.0933 & 0.0755 & -0.0043  &  0.0165 \\
			4 & bone-marrow       & 0.0192 & 0.0201 & 0.0148 & 0.0126 & 0.0101 &  0.0009  &  0.0041 \\
			5 & lung              & 0.1627 & 0.2059 & 0.1370 & 0.1038 & 0.0949 & -0.0259  &  0.0172 \\
			6 & pleura            & 0.1485 & 0.1496 & 0.1015 & 0.0590 & 0.09   &  0.0313  &  0.0234 \\
			7 & peritoneum        & 0.3184 & 0.3459 & 0.1979 & 0.0861 & 0.138  &  0.0147  &  0.0956 \\
			8 & liver             & 0.2285 & 0.2138 & 0.2630 & 0.0786 & 0.1279 &  0.0375  & -0.0602 \\
			9 & brain             & 0.0465 & 0.0349 & 0.0548 & 0.0378 & 0.0254 &  0.0114  & -0.0104 \\
			10 & skin             & 0.0677 & 0.0362 & 0.0923 & 0.0314 & 0.0403 &  0.0252  & -0.0133 \\
			11 & neck             & 0.2215 & 0.0690 & 0.2582 & 0.1466 & 0.0692 &  0.1316  & -0.0081 \\
			12 & supraclavicular  & 0.1676 & 0.1915 & 0.1448 & 0.0845 & 0.067  & -0.0198  &  0.0269 \\
			13 & axillar          & 0.1393 & 0.1457 & 0.1068 & 0.0655 & 0.0629 & -0.0067  &  0.0447 \\
			14 & mediastinum      & 0.1838 & 0.2050 & 0.1716 & 0.1016 & 0.0806 & -0.0059  &  0.0140 \\
			15 & abdominal        & 0.2553 & 0.3296 & 0.1372 & 0.1346 & 0.1379 & -0.0330  &  0.0898 \\ \hline
	\end{tabular}}
	\caption{Importances as computed analytically using asymptotic formulas. The context is defined by the binary context feature \textit{Sex} ($Sex=0$ denotes \textit{female} and $Sex=1$ denotes \textit{male}).}
	\label{table:asymptotic-realresult}
\end{table}

\end{subappendices}
\chapter{In very high dimensions} \label{ch:SRS}

\begin{overview}

Dealing with datasets of very high dimension is a major challenge in machine learning. This chapter considers the problem of feature selection in applications where the memory is not large enough to contain all features. In this setting, we propose a novel tree-based feature selection approach that builds a sequence of randomised trees on small sub-samples of variables mixing both variables already identified as relevant by previous models and variables randomly selected among the other variables. As our main contribution, we provide an in-depth theoretical analysis of this method in infinite sample setting. In particular, we study its soundness with respect to common definitions of feature relevance and its convergence speed under various variable dependence scenarios. We also provide some preliminary empirical results highlighting the potential of this approach.\\

\textbf{\textcolor{RoyalBlue}{References:}} This chapter  is an adapted version of the following publication: \\[2mm]\bibentry{sutera2018random}. \\[2mm]
We do not reproduce Section 2 of this paper which provides background material already given in the preceding chapters of this manuscript.
\end{overview}

\section{Motivation}
We consider supervised learning and more specifically feature selection in
applications where the memory is not large enough to contain all data. Such
memory constraints can be due either to the large volume of available training
data or to physical limits of the system on which training is performed (eg.,
mobile devices). A straightforward, but often efficient, way to handle such
memory constraint is to build and average an ensemble of models, each trained
on only a random subset of samples and/or features that can fit into
memory. Such simple ensemble approaches have the advantage to be applicable to
any batch learning algorithm, considered as a black-box, and they have been
shown empirically to be very effective in terms of predictive performance, in
particular when combined with trees, and even when samples and/or features are
selected uniformly at random \citep[see,
  eg.,][]{chawla2004,louppe2012ensembles}.  In particular, and independently of
any considerations about memory constraints, feature subsampling has been shown
in several works to be a very effective way to introduce randomization when
building ensembles of models \citep{ho1998random,kuncheva2010random}. The idea
of feature subsampling has also been investigated in the context of feature
selection, where several authors have proposed to repeatedly apply a
multivariate feature selection technique on random subsets of features and then
to aggregate the results obtained on these subsets \citep[see,
  eg.,][]{draminski2008monte,lai2006random,konukoglu2014approximate,nguyen2015new,draminski2016discovering}.

In this chapter, focusing on feature subsampling, we adopt a simplistic setting
where we assume that only q input features (among $p$ in total, with typically
$q \ll p$) can fit into memory. In this setting, we study ensembles of
randomized decision trees trained each on a random subset of $q$ features. In
particular, we are interested in the properties of variable importance scores
derived from these models and their exploitation to perform feature selection.
In contrast to a purely uniform sampling of the features, we propose in
Section~\ref{sec:srs} a modified sequential random subspace (SRS) approach that
biases the random selection of the features at each iteration towards features
already found relevant by previous models. As our main contribution, we perform
in Section~\ref{sec:analysis} an in-depth theoretical analysis of this method
in infinite sample size condition. In particular, we show that (i) this
algorithm provides some interesting asymptotic guarantees to find all
(strongly) relevant variables, (ii) that accumulating previously found variables
can reduce the number of trees needed to find relevant variables by several
orders of magnitudes with respect to the standard random subspace method in
some scenarios, and (iii) that these scenarios are relevant for a large class of
(PC) distributions. As an important additional contribution, our analysis also
sheds some new light on both the popular random subspace and random forests
methods that are special cases of the SRS algorithm. Finally,
Section~\ref{sec:empirical} presents some preliminary empirical results with
the approach on several artificial and real datasets.

\section{Sequential random subspace}\label{sec:srs}

In this chapter, we consider a simplistic memory-constrained setting where it is
assumed that only $q$ input features can fit into memory at once, with
typically $q$ small with respect to $p$. Under this hypothesis,
Algorithm~\ref{algo:SRS} describes the proposed sequential random subspace (SRS)
algorithm to build an ensemble of randomized trees, which generalizes the
Random Subspace (RS) method \citep{ho1998random} (presented in Section \ref{sec:trees-methods}). The idea of this method is to
bias the random selection of the features at each iteration towards features
that have already been found relevant by the previous trees. A parameter
$\alpha$ is introduced that controls the degree of accumulation of previously
identified features. When $\alpha=0$, SRS reduces to the standard RS
method. When $\alpha=1$, all previously found features are kept while when
$\alpha<1$, some room in memory is left for randomly picked features, which
ensures some permanent exploration of the feature space. Further randomization
is introduced in the tree building step through the parameter $K\in [1,q]$,
ie. the number of variables sampled at each tree node for splitting. Variable
importance is assumed to be the MDI importance. This algorithm returns both an
ensemble of trees and a subset $F$ of variables, those that get an importance
(significantly) greater than 0 in at least one tree of the ensemble. Importance
scores for the variables can furthermore be derived from the final ensemble
using Equation \ref{imp:eqn:mdi-imp}. In what follows, we will denote by $F^{K,\alpha}_{q,T}$
and $Imp^{K,\alpha}_{q,T}(X)$ resp. the set of features and the importance of
feature $X$ obtained from an ensemble grown with SRS with parameters
$K$, $\alpha$, $q$ and $T$.

\begin{algorithm}[t]
  
  \textbf{Inputs:}\\
  \underline{Data}: $Y$ the output and $V$, the set of all input variables (of size $p$).\\
  \underline{Algorithm}: $q$, the subspace size, and $T$ the number of iterations, $\alpha\in[0,1]$, the percentage of memory devoted to previously found features.\\
  \underline{Tree}: $K$, the tree randomization parameter\\
  \textbf{Output:} An ensemble of $T$ trees and a subset $F$ of features\\
  \textbf{Algorithm:}\vspace{-0mm}
  \begin{enumerate}[1.]
    \setlength\itemsep{0em}
    \item $F=\emptyset$
    \item Repeat $T$ times:
    \begin{enumerate}[(a)]

      \item Let $Q=R\cup C$, with $R$ a subset of $\min\{\lfloor \alpha
      q\rfloor, |F|\}$ features randomly picked in $F$ without
      replacement and $C$ a subset of $q-|R|$ features randomly
      selected in $V\setminus R$.
      \item Build a decision tree $\cal T$ from $Q$ using randomization parameter $K$.
      \item Add to $F$ all features from $Q$ that get an importance
      greater than zero in $\cal T$.
    \end{enumerate}
  \end{enumerate}

\captionof{algorithm}{Sequential Random Subspace algorithm}
\label{algo:SRS}

\end{algorithm}

The modification of the RS algorithm is actually motivated by Propositions
\ref{prop:mdi-only-relevant} and \ref{prop:only-degree}, stating that the relevance
of high degree features can be determined only when they are analysed jointly
with other relevant features of equal or lower degree. From this result, one
can thus expect that accumulating previously found features will fasten the
discovery of higher degree features on which they depend through some snowball
effect. In the next section, we provide a theoretical asymptotic analysis of
the SRS method that confirms and quantifies this effect.

Note that the SRS method can also be motivated from the perspective of
accuracy.  When $q\ll p$ and the number of relevant features $r$ is also much
smaller than the total number of features $p$ ($r\ll p$), many trees with
standard RS are grown from subsets of features that contain only very few, if
any, relevant features and are thus expected not to be better than random
guessing \citep{kuncheva2010random}. In such setting, RS ensembles are thus
expected not to be very accurate.
\begin{example} With $p=10000$, $r=10$ and $q=50$, the proportion of trees in a RS ensemble grown from only irrelevant variables is $C^{q}_{p-r}/C^{q}_{p} = 0.95$.
\end{example}
With SRS (and $\alpha>0$), we ensure that more and more relevant variables are given to the tree
growing algorithm as iterations proceed and therefore we reduce the chance to
include totally useless trees in the ensemble. Note however that in finite
settings, there is a potential risk of overfitting when accumulating the
variables. The parameter $\alpha$ thus controls a new bias-variance tradeoff
and should be tuned appropriately. We will study the impact of SRS on accuracy
empirically in Section \ref{sec:empirical}.

\section{Theoretical analysis}
\label{sec:analysis}
In this section, we carry out a theoretical analysis of the
proposed method when seen as a feature selection technique. This analysis is
performed in asymptotic sample size condition, assuming that all features,
including the output, are discrete, and using Shannon entropy as the impurity
measure. We proceed in two steps. First, we study the soundness of the
algorithm, ie., its capacity to retrieve the relevant variables when the number
of trees is infinite. Second, we study its convergence properties, ie. the
number of trees needed to retrieve all relevant variables in different
scenarios.

\subsection{Soundness}
\label{sec:soundness}

Our goal in this section is to characterize the sets of features
$F^{K,\alpha}_{q,\infty}$ that are identified by the SRS algorithm, depending
on the value of its parameters $q$, $\alpha$, and $K$, in an asymptotic
setting, ie. assuming an infinite sample size and an infinite forest
($T=\infty$). Note that in asymptotic setting, a variable is relevant as soon
as its importance in one of the tree is strictly greater than zero and we thus
have the following equivalence for all variables $X\in V$:
$$X\in F^{K,\alpha}_{q,\infty} \Leftrightarrow Imp^{K,\alpha}_{q,\infty}(X)>0$$
Furthermore, in infinite sample size setting, irrelevant variables always get a
zero importance and thus, whatever the parameters, we have the
following property for all $X\in V$:
$$X\mbox{ irrelevant} \Rightarrow X\notin F^{K,\alpha}_{q,\infty} \mbox{ (and }
Imp^{K,\alpha}_{q,\infty}(X)=0\mbox{)}.$$ The method parameters thus only affect
the number and nature of the relevant variables that can be found. Denoting by $r$ ($\leq p$) the
number of relevant variables, we will analyse separately the case $r\leq
q$ (all relevant variables can fit into memory) and the case $r>q$ (all relevant
variables can not fit into memory).

\paragraph{All relevant variables can fit into memory ($r\leq q$).}

Let us first consider the case of the RS method ($\alpha=0$). In this case,
\citet{louppe2013understanding} have shown the following asymptotic formula for the
importances computed with totally randomized trees ($K=1$):
\begin{eqnarray}
Imp_{q,\infty}^{1,0}(X) = \sum_{k=0}^{q-1} \dfrac{1}{C_p^k} \sum_{B\in {\cal
    P}_k(V^{-m})} I(X;Y|B),
\end{eqnarray} 
where ${\cal P}_k(V^{-m})$ is the set of subsets of $V^{-m}=V\setminus \{X_m\}$
of cardinality $k$. Given that all terms are positive, this sum will be
strictly greater than zero if and only if there exists a subset $B\subseteq
V$ of size at most $q-1$ such that $Y\nindep X|B$ ($\Leftrightarrow I(X;Y|B)>0$), or equivalently if $deg(X)< q$. When $\alpha=0$, RS with $K=1$ will thus find all and only the
relevant variables of degree at most $q-1$. Given Proposition
\ref{prop:mdi-only-relevant}, the degree of a variable $X$ can not be larger than
$r-1$ and thus as soon as $r\leq q$, we have the guarantee that RS with $K=1$
will find all and only the relevant variables. Actually, this result remains
valid when $\alpha>0$. Indeed, asymptotically, only relevant variables
will be selected in the $F$ subset by SRS and given that all relevant variables
can fit into memory, cumulating them will not impact the ability of SRS to
explore all conditioning subsets $B$ composed of relevant variables. We thus
have the following result:
\begin{proposition}
  $\forall \alpha$, if $r\leq q$: \quad $X \in F_{q,\infty}^{1,\alpha} \mbox{ iff }X\mbox{ is relevant}.$
\end{proposition}

In the case of non-totally randomized trees ($K>1$), we lose the guarantee to
find all relevant variables even when $r\leq q$. Indeed, there is potentially a
masking effect due to $K>1$ that might prevent the conditioning needed for a
given variable to be relevant to appear in a tree branch.  However, we have the
following general result:
\begin{theorem}\label{th:stronglyrelpruned}
  $\forall \alpha, K$, if $r\leq q$:\quad
$X\mbox{ strongly relevant} \Rightarrow X \in F_{q,\infty}^{K,\alpha}$ 

\end{theorem} 
\begin{proof}
See Appendix \ref{app:proof:stronglyrelpruned}
  \end{proof}

There is thus no masking effect possible for the strongly relevant features
when $K>1$ as soon as the number of relevant features is lower than $q$. For a
given $K$, the features found by SRS will thus include all strongly relevant
variables and some (when $K>1$) or all (when $K=1$) weakly relevant ones. It is easy to
show that increasing $K$ can only decrease the number of weakly relevant
variables found. Using $K=1$ will thus provide a solution for the \textbf{all-relevant}
problem, while increasing $K$ will provide a better and better approximation of
the \textbf{minimal-optimal} problem in the case of strictly positive distributions (see
Section \ref{sec:fsproblems} for definitions of those problems).

Interestingly, Theorem \ref{th:stronglyrelpruned} remains true when $q=p$, ie.,
when forests are grown without any feature sampling. It thus extends Theorem \ref{thm:mdi-totally-irrelevant} (also \citep[Theorem 3]{louppe2013understanding}) for arbitrary $K$ in the case of
standard random forests.

\paragraph{All relevant variables can not fit into memory ($r>q$).}

When all relevant variables can not fit into memory, we do not have the
guarantee anymore to explore all minimal conditionings required to find all
(strongly or not) relevant variables, whatever the values of $K$ and
$\alpha$. When $\alpha=0$, we have the guarantee however to identify the
relevant variables of degree strictly lower than $q$. When $\alpha>1$, some
space in memory will be devoted to previously found variables that will
introduce some further masking effect. We nevertheless have the following
general results (without proof):
\begin{proposition} \begin{eqnarray}
& \forall X: \quad X\mbox{ relevant and } \nonumber \\ & deg(X)<(1-\alpha)q \Rightarrow X\in F^{1,\alpha}_{q,\infty}. \nonumber\end{eqnarray}
\end{proposition}
\begin{proposition} \begin{eqnarray} \forall K,X: \quad  X\mbox{ strongly relevant and } \nonumber \\ deg(X)<(1-\alpha)q \Rightarrow X\in F^{K,\alpha}_{q,\infty}. \nonumber \end{eqnarray}
\end{proposition}

In these propositions, $(1-\alpha)q$ is simply the amount of memory that
always remains available for the exploration of variables not yet found
relevant.

\paragraph{Discussion.}

Results in this section show that SRS is a sound approach for feature selection
as soon as either the memory is large enough to contain all relevant variables
or the degree of the relevant variables is not too high. In this latter case,
the approach will be able to detect all strongly relevant variables whatever
its parameters ($K$ and $\alpha$) and the total number of features $p$. Of
course, these parameters will have a potentially strong influence on the number
of trees needed to reach convergence (see the next section) and the performance
in finite setting.

\subsection{Convergence}\label{sec:convergence}

Results in the previous section show that accumulating relevant variables has
no impact on the capacity at finding relevant variables asymptotically (when
$r\leq q$). It has however a potentially strong impact on the convergence speed
of the algorithm, as measured for example by the expected number of trees
needed to find all relevant variables. Indeed, when $\alpha=0$ and $q\ll p$,
the number of iterations/trees needed to find relevant variables of high degree
can be huge as finding them requires to sample them together with all features
in their conditioning. Given Proposition 2, we know that a minimum subset $B$
such that $X\nindep Y|B$ for a relevant variable $X$ contains only relevant
variables. This suggests that accumulating previously found relevant features
can improve significantly the convergence, as each time one relevant variable
is found it increases the chance to find a relevant variable of higher degree
that depends on it. In what follows, we will quantify the effect of
accumulation on convergence speed in different best-case and worst-case
scenarios and under some simplifications of the tree building procedure. We
will conclude by a theorem highlighting the interest of the SRS method in the
general class of PC distributions.

\paragraph{Scenarios and assumptions.}

The convergence speed is in general very much dependent on the data
distribution. We will study here the following three specific scenarios (where
features $\{X_1,\ldots,X_r\}$ are the only relevant features):\vspace{-0.5em}
\begin{itemize}
\item \textbf{Chaining:} The only and minimal conditioning that makes variable
  $X_i$ relevant is $\{X_1,\ldots,X_{i-1}\}$ (for $i=1,\ldots,r$). We
  thus have $deg(X_i)=i-1$. This scenario should correspond to the most
  favorable situation for the SRS algorithm.\vspace{-0.5em}
\item \textbf{Clique:} The only and minimal conditioning that makes variable
  $X_i$ relevant is $\{X_1,\ldots,X_{i-1},X_{i+1},\ldots,X_r\}$ (for
  $i=1,\ldots,r$). We thus have $deg(X_i)=r-1$ for all $i$. This is a rather
  defavorable case for both RS and SRS since finding a relevant variable implies
  to draw all of them at the same iteration. \vspace{-0.5em}
\item \textbf{Marginal-only:} All variables are marginally relevant. We will
  furthermore make the assumption that these variables are all strongly
  relevant. They can not be masked mutually.  This scenario is the most
  defavorable case for SRS (versus RS) since accumulating relevant variables is
  totally useless to find the other relevant variables and it should actually
  slow down the convergence as it will reduce the amount of memory left for
  exploration.
\end{itemize}
In Appendix \ref{sec:avgtime}, we provide explicit
formulation of the expected number of iterations needed to find all
$r$ relevant features in the chaining and clique scenarios both when
$\alpha=0$ (RS) and $\alpha=1$ (SRS). In Appendix \ref{sec:mc}, we
provide order 1 Markov chains that model the evolution through the
iterations of the number of variables found in the three scenarios
when $\alpha=0$ and $\alpha=1$. These chains can be used to compute
numerically the expected number of relevant variables found through
the iterations (and in the case of the marginal-only setting, the
expected number of iterations to find all variables).  These
derivations are obtained assuming $r\leq q$, $K=q$, and under the following
additional simplifying assumptions.

Below, we compute analytically the average number of trees needed to
  find all relevant variables in the chaining and clique scenarios and
  we derive transition matrices of Markov chains that model the
  evolution of the number of variables found through the iterations in
  the three scenarios. These results are obtained assuming $K=q$ and
  $r\leq q$, and with either $\alpha=0$ (RS) or $\alpha=1$ (SRS).

  To make these derivations possible and independent of a particular
  data distribution, one needs furthermore to simplify the decision tree
  growing algorithm in the case of the chaining and clique
  scenarios. In what follows, trees are thus assumed to be grown such
  that a unique variable is selected at each tree level and this
  variable is selected at random among all variables $X$ such that
  $Y\nindep X|B$ where $B$ is the set of all variables tested at
  previous levels.

  In the clique scenario, this assumption implies that only one
  variable of the clique will get a non-zero importance when all
  clique variables are selected at one iteration of RS/SRS (since only
  the last variable of the clique tested along a tree branch can get a
  non-zero score and this variable is the same in each branch given
  our tree growing assumption). This corresponds to a pessimistic
  scenario. Indeed, with standard unconstrained trees, several
  relevant variables could be found at one iteration given that the
  ordering of the variables, and thus the last variable of the clique
  tested, might differ from one tree branch to another.  As a
  consequence, the tree growing assumption will lead to an
  overestimation of the number of trees needed to reach
  convergence. In the chaining scenario, the simplified tree growing
  algorithm implies that all relevant variables selected at one
  iteration of RS/SRS together with their minimal conditioning will
  get a non-zero importance. This corresponds this time to an
  optimistic scenario, as, with unconstrained trees, such variable
  might not be detected at one iteration depending on the exact data
  distribution. This will thus lead this time to an underestimation of
  the number of trees needed to reach convergence. Note however that,
  in both cases, these over/under-estimations will affect both RS and
  SRS in the same proportion and thus our assumption will not impact
  their relative performance.

  Note that in the marginal-only scenario, given that all relevant
  variables are marginally and strongly relevant, they will always get
  a non-zero importance as soon as they are selected at one
  iteration. Our estimations below are thus not impacted by the
  simplification of the tree growing algorithm.

\paragraph{Results and discussion.}

Tables \ref{tab:timechaining}, \ref{tab:timeclique}, and \ref{tab:timemarginal}
show the expected number of iterations needed to find all relevant variables
for various configurations of the parameters $p$, $q$, and $r$, in the three
scenarios. Figure \ref{fig:evolution} plots the expected number of variables
found at each iteration both for RS and SRS in the three scenarios for some
particular values of the parameters.

\begin{table*}[htb]
\small
\centering
\caption{Expected number of iterations needed to find all relevant variables
  for various configurations of parameters $p$, $q$ and $r$ with RS
  ($\alpha=0$) and SRS ($\alpha=1$) in the three scenarios.}

\subfloat[Chaining.\label{tab:timechaining}]{
  \begin{tabular}{l|cc}
    Config (p,q,r) & RS & SRS\\
    \hline
    $10^4
    ,100,1$ & 100 & 100\\
    $10^4
    ,100,2$ & 10100 & 200\\
    $10^4
    ,100,3$ & $>10^6$ & 301\\
    $10^4
    ,100,5$ & $>10^{10}$ & 506\\
    $10^5
    ,100,3$ & $>10^9$ & 3028\\
  \end{tabular} 
}
\hspace{1.5em}
\subfloat[Clique.\label{tab:timeclique}]{
  \begin{tabular}{l|cc}
    Config (p,q,r) & RS & SRS\\
    \hline
    $10^4,100,1$ & 100 & 100\\
    $10^4,100,2$ & 30300 & 10302\\
    $10^4,100,3$ & $5\cdot 10^6$ & $10^6$\\
    $10^4,100,4$ & $9\cdot  10^{8}$ & $10^8$\\
    $10^4,10^3,4$ & $83785$ & $11635$\\
  \end{tabular}  
}
\hspace{1.5em}
\subfloat[Marginal-only.\label{tab:timemarginal}]{  
\begin{tabular}{l|cc}
    Config (p,q,r) & RS & SRS\\
    \hline
    $10^4
    ,100,10$ & 291 & 312\\
    $10^4
    ,100,50$ & 448 & 757\\
    $10^4
    ,100,90$ & 506 & 2797\\
    $10^4
    ,100,100$ & 1123 & 16187\\
    $25000
    ,100,50$ & 1123 & 1900\\
  \end{tabular}
}
\end{table*}

\begin{figure*}[htbp]
  \centering
  \includegraphics[width=0.7\linewidth]{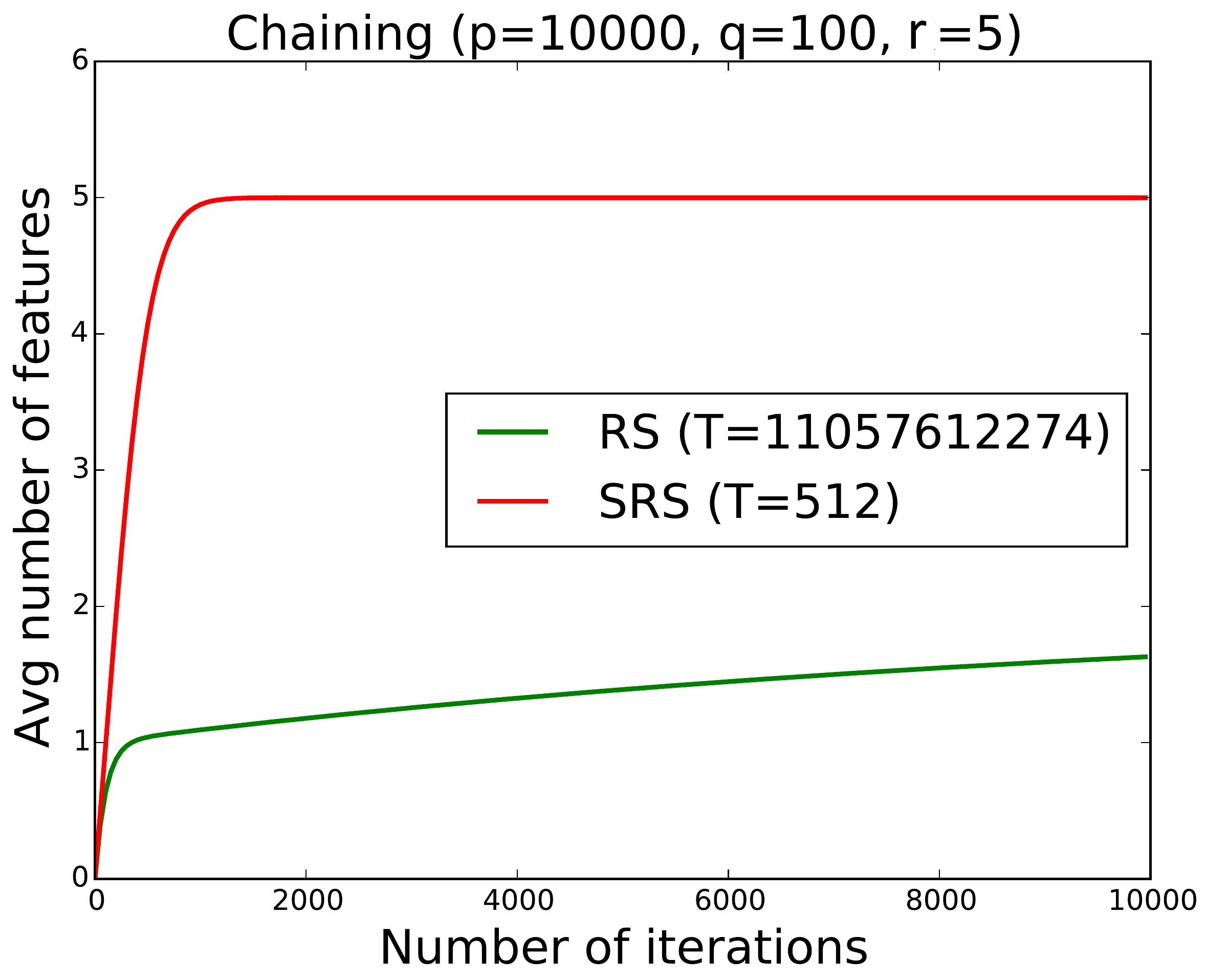}\\
  \includegraphics[width=0.7\linewidth]{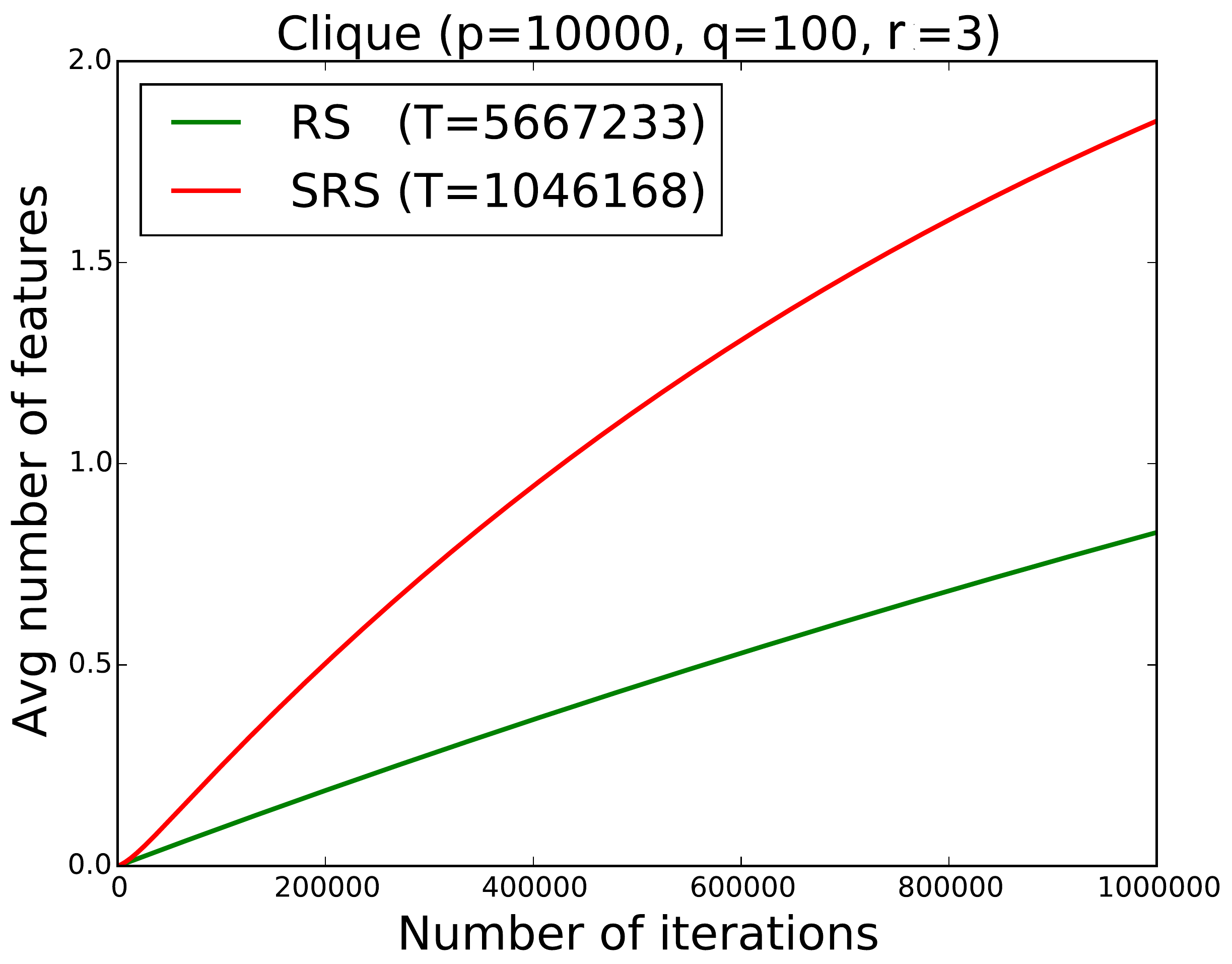}\\
 \includegraphics[width=0.7\linewidth]{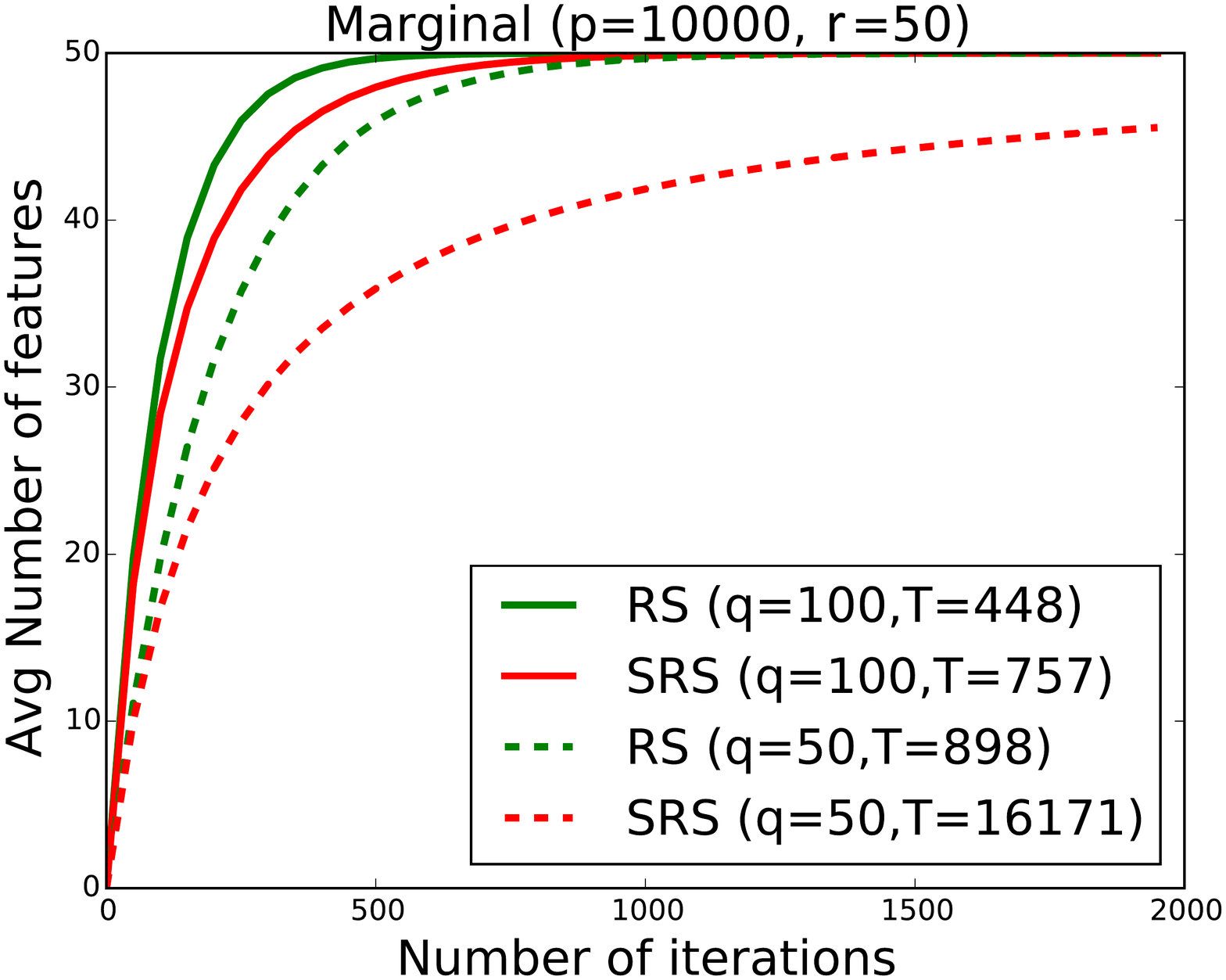}
  \caption{Evolution of the number of selected features in the different scenarios.\label{fig:evolution}}
\end{figure*}

From these results, we can draw several conclusions. In all cases, expected times (ie., number of iterations/trees to find all relevant
variables) depend mostly on the ratio $\frac{q}{p}$, not on absolute values of
$q$ and $p$. The larger this ratio, the faster the convergence. Parameter $r$
has a strong impact on convergence speed in all three scenarios.

The most impressive improvements with SRS are obtained in the \textbf{chaining}
hypothesis, where convergence is improved by several orders of magnitude (Table
\ref{tab:timechaining} and Figure \ref{fig:evolution}a) . At fixed $p$ and $q$,
the time needed by RS indeed grows exponentially with $r$ ($\simeq
(\frac{p}{q})^r$ if $r\ll q$), while time grows linearly with $r$ for the SRS
method ($\simeq r\frac{p}{q}$ if $r\ll q$) (see Eq.  (\ref{ref:naive}) and
(\ref{ref:smart}) in Appendix \ref{sec:avgtime}).

In the case of \textbf{cliques}, both RS and SRS need many iterations to find all
features from the clique (see Table \ref{tab:timeclique} and Figure
\ref{fig:evolution}b). SRS goes faster than RS but the improvement is not as
important as in the chaining scenario. This can be explained by the fact that
SRS can only improve the speed when the first feature of the clique has been
found. Since the number of iterations needed to find the $r$ features from the
clique for RS is close to $r$ times the number of iterations needed to find one
feature from the clique, SRS can only decrease at best the number of iterations
by approximately a factor $r$ (see Eq.  (\ref{eqn:naiveclique}) and
(\ref{eqn:smartclique}) in Appendix \ref{sec:avgtime}).
  
In the \textbf{marginal-only} setting, SRS is actually slower than RS because the
only effect of cumulating the variables is to leave less space in memory for
exploration. The decrease of computing times is however contained when $r$ is
not too close to $q$ (see Table \ref{tab:timemarginal} and Figure
\ref{fig:evolution}c).

Since we can obtain very significant improvement in the case of the chaining
and clique scenarios and we only increase moderately the number of iterations
in the marginal-only scenario (when $r$ is not too close from $q$), we can
reasonably expect improvement in general settings that mix these scenarios.

\paragraph{PC distributions and chaining.}

Chaining is the most interesting scenario in terms of convergence improvement
through variable accumulation. In this scenario, SRS makes it possible to find
high degree relevant variables with a reasonable amount of trees, when finding
these variables would be mostly untractable for RS. We provide below two
theorems that show the practical relevance of this scenario in the specific
case of PC distributions.

A PC distribution is defined as a strictly positive (P) distribution that
satisfies the composition (C) property stated as follows
\cite{nilsson2007consistent}:
\begin{property}
  For any disjoint sets of variables $R , S, T, U \subseteq V\cup\{Y\}$:
  $$S\indep T|R\mbox{ and }S\indep U|R \Rightarrow S\indep T\cup U | R$$
\end{property}
The composition property prevents the occurence of cliques and is preserved
under marginalization. PC actually represents a rather large class of
distributions that encompasses for example jointly Gaussian distributions and
DAG-faithful distributions \cite{nilsson2007consistent}.

The composition property allows to make Proposition \ref{prop:only-degree} more
stringent in the case of PC:
\begin{proposition} \label{prop:only-degree-pc}
Let $B$ denote a minimal subset $B$ such that $Y\nindep X|B$ for a relevant
variable $X$. If the distribution $P$ over $V\cup\{Y\}$ is PC, then for all
$X'\in B$, $deg(X')<|B|$. 
\end{proposition}

\begin{proof}
  Proposition \ref{prop:only-degree} proves that the degree of all features in
  $B$ is $\leq |B|$ in the general case. Let us assume that there exists a
  feature $X'\in B$ of degree $|B|$ in the case of PC distribution. Since this
  property remain true when the set of features $V$ is reduced to a subset
  $V'=B\cup\{X\}$, the minimal $B'$ of $X'$ can only be
  $(B\setminus\{X_i\})\cup\{X\}$. We thus have the following two properties:
    $$Y \indep X|B\setminus\{X'\}$$
    $$Y \indep X'|B'\setminus\{X\},$$
    because $B$ and $B'$ are minimal. Together, by the composition property, they should imply that
    $$Y \indep \{X,X_i\}|B\setminus\{X_i\},$$ which implies, by weak union:
    $Y\indep X|B,$ which contradicts the hypothesis.
  \end{proof}

In addition, one has the following result:
\begin{theorem}\label{th:pcchain}
  For any PC distribution, let us assume that there exists a non empty
  minimal subset $B=\{X_1,\ldots,X_k\}\subset V\setminus \{X\}$ of
  size $k$ such that $X\nindep Y |B$ for a relevant variable $X$.
  Then, variables $X_1$ to $X_k$ can be ordered into a sequence
  $\{X'_1,\ldots,X'_k\}$ such that $deg(X'_i)<i$ for all
  $i=1,\ldots,k$. 
\end{theorem}

\begin{proof}
 Let us denote by $\{X'_1,X'_2,\ldots,X'_k\}$ the variables in $B$ ordered
 according to their degree, ie., $deg(X'_i)\leq deg(X'_{i+1})$, for
 $i=1,\ldots,k-1$. Let us show that $deg(X'_i)<i$ for all $i=1,\ldots,k$. If
 this property is not true, then there exists at least one $X'_i\in B$ such
 that $deg(X'_i)\geq i$. Let us denote by $l$ the largest $i$ such that
 $deg(X_i)\geq i$. Using a similar argument as in the proof of Proposition
 \ref{prop:only-degree-pc}, there exists some minimal subset $B'\subseteq
 B\setminus\{X_l\}$ such that $Y\nindep X_l|B'$. Given that $deg(X_l)\geq l$,
 this subset $B$ should contain $l$ variables or more from $B\setminus
 \{X_l\}$. It thus contains at least one variable $X_m$ with $l<m\leq k$, and
 this variable is such that $deg(X_m)<m$. Given Proposition
 \ref{prop:only-degree-pc}, if $B'$ is minimal and contains $X_m$, then for a
 PC distribution, $deg(X_m)$ should be strictly smaller than $|B'|\geq l$, which
 contradicts the fact that $X_m$ is after $X_l$ in the ordering and proves the
 theorem.
 \end{proof}

 This theorem shows that, when the data distribution is PC, for all relevant
 variables of degree $k$, the $k$ variables in its minimal conditioning form a
 chain of variables of increasing degrees (at worst). For such distribution, we
 thus have the guarantee that SRS finds all relevant variables
 with a number of iterations that grows almost only linearly with the maximum
 degree of relevant variables (see Eq.\ref{ref:smart} in Appendix
 \ref{sec:avgtime}), while RS would be unable to find relevant variables of
 even small degree.

\section{Experiments}\label{sec:empirical}

\begin{figure}[htbp]
\centering
  \subfloat[SRS with $q=0.05 \times p $ on a dataset with $p=50000$ features and $r=20$ relevant features.]{\includegraphics[width=0.85\linewidth]{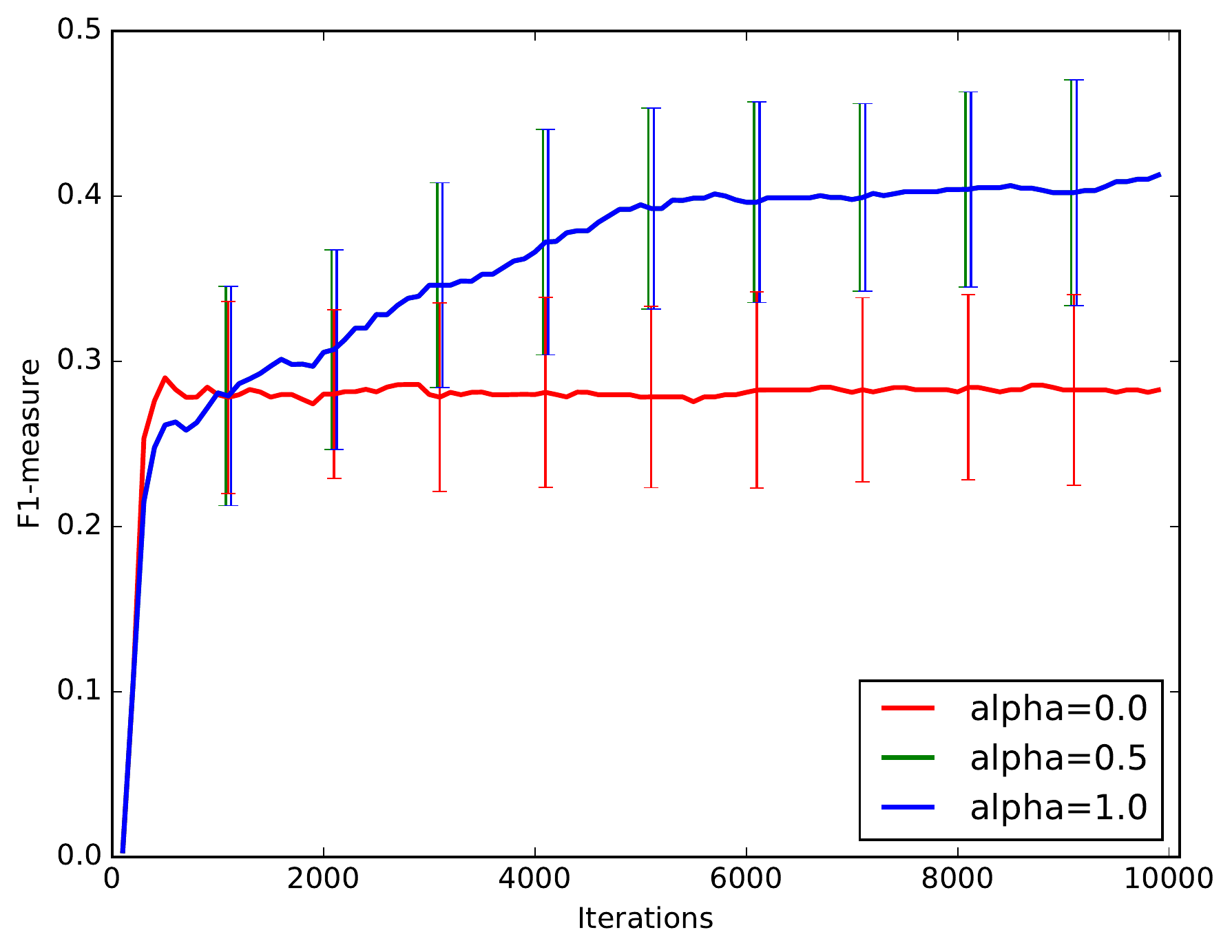}}
  \\
  \subfloat[SRS with $q=0.005 \times p$ on a dataset with $p=50000$ features and $r=20$ relevant features.]{\includegraphics[width=0.85\linewidth]{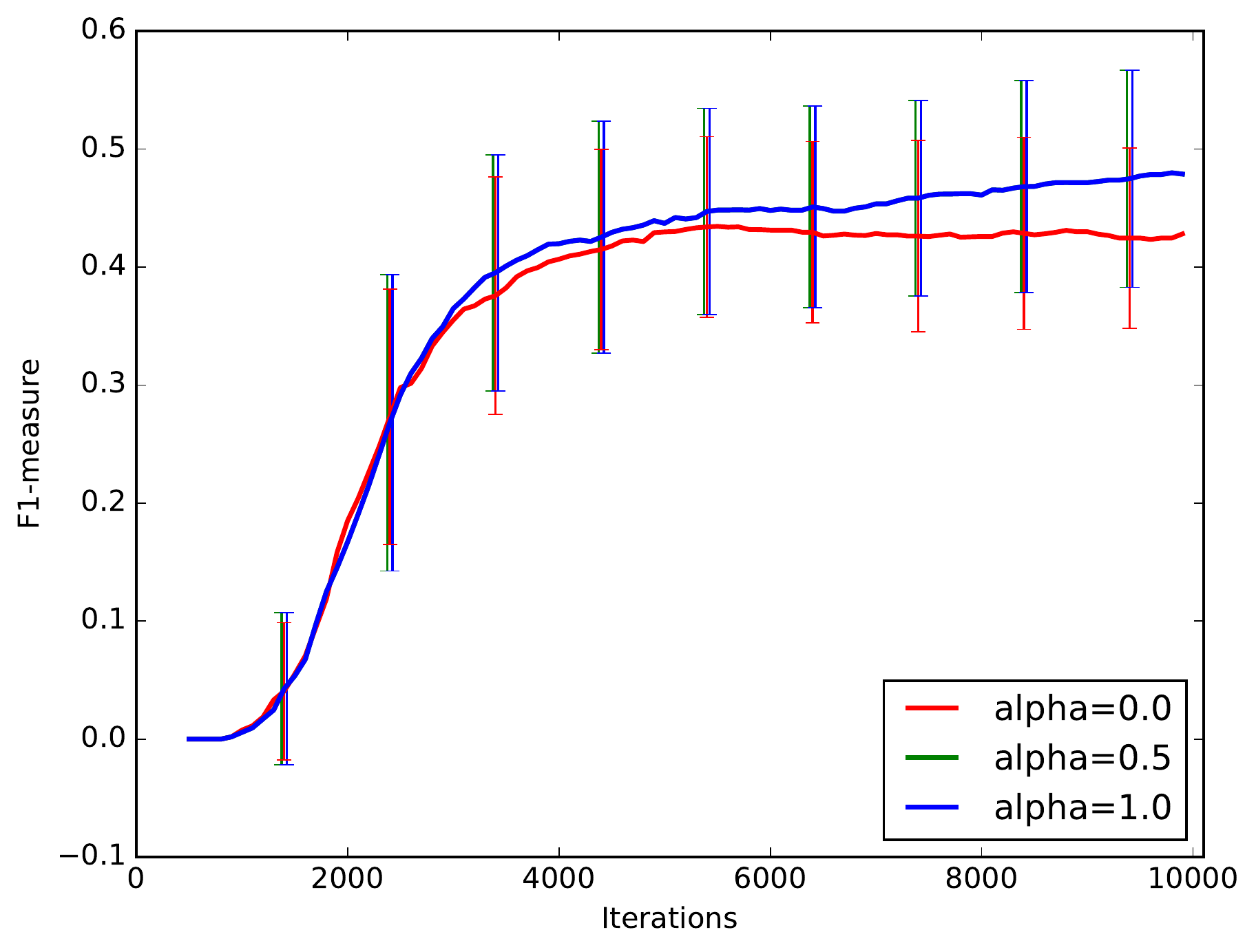}}
  \caption{Evolution of the evaluation of the feature subset found by RS and SRS using the F1-measure computed with respect to relevant features. A higher value means that more relevant features have been found. This experiment was computed on an artificial dataset (similar to madelon) of 50000 features with 20 relevant features and for two sizes of memory.}
  \label{fig:pr} \vspace{-0em}
\end{figure}

Although our main contribution is the theoretical analysis in
asymptotic setting of the previous section, we present here a few
preliminary experiments in finite setting as a first illustration of
the potential of the method.
One of the main difficulties to implement the SRS algorithm as presented in
Algorithm \ref{algo:SRS} is step 2(c) that decides which variable should be
incorporated in $F$ at each iteration. In infinite sample size
setting, a variable with a non-zero importance in a single tree is guaranteed
to be truly relevant. Mutual informations estimated from finite samples however
will always be greater than 0 even for irrelevant variables. One should thus
replace step 2(c) by some statistical significance tests to avoid the
accumulation of irrelevant variables that would jeopardize the convergence of
the algorithm. In our experiments here, we use a random probe (ie., an
artificially created irrelevant variable) to derive a statistical measure
assessing the relevance of a variable \cite{stoppiglia03}. Details about this
test are given in Appendix~\ref{app:results}.

Figure \ref{fig:pr} evaluates the feature selection ability of SRS for three values
of $\alpha$ (including $\alpha=0$) and two memory sizes (250 and 2500) on an
artificial dataset with 50000 features, among which only 20 are relevant (see
Appendix~\ref{app:results} for more details). The two plots show the evolution
of the F1-score comparing the selected features (in $F$) with the truly
relevant ones as a function of the number of iterations. As expected, SRS
($\alpha>0$) is able to find better feature subsets than RS ($\alpha=0$) for
both memory sizes and both values of $\alpha>0$.

Additional results are provided in Appendix~\ref{app:results} that compare the
accuracy of ensembles grown with SRS for different values of $\alpha$ and on 13
classification problems. These comparisons clearly show that accumulating the
relevant variables is beneficial most of the time (eg., SRS with $\alpha=0.5$
is significantly better than RS on 7 datasets, comparable on 5, and
significantly worse on only 1). Interestingly, SRS ensembles with $\alpha=0.5$ are also most of
the time significantly better than ensembles of trees grown without memory
constraint (see Appendix~\ref{app:results} for more details).

\section{Conclusions and future work} 

Our main contribution is a theoretical analysis of the SRS (and RS) methods in
infinite sample setting. This analysis showed that both methods provide some
guarantees to identify all relevant (or all strongly relevant) variables as
soon as the number of relevant variables or their degree is not too high with
respect to the memory size. Compared to RS, SRS can reduce very strongly the
number of iterations needed to find high degree variables in particular in the
case of PC distributions. We believe that our results shed some new light on random
subspace methods for feature selection in general as well as on tree-based
methods, which should help designing better feature selection procedures.

Some preliminary experiments were provided that support the
theoretical analysis, but more work is clearly needed to evaluate the approach
empirically on controlled and real high-dimensional problems. We believe
that the statistical test used to decide which feature to include in the
relevant set should be improved with respect to our first implementation based
on the introduction of a random probe. One drawback of the SRS method with
respect to RS is that it can not be parallelized anymore because of its
sequential nature. It would be interesting to design and study variants of the
method that are allowed to grow parallel ensembles at each iteration instead of
single trees. Finally, relaxing the main hypotheses of our theoretical analysis
would be also of course of great interest.

\begin{subappendices}
	\section{Proof of Theorem \ref{th:stronglyrelpruned}} \label{app:proof:stronglyrelpruned}

	\begin{theorem*}
		$\forall \alpha, K$, if $r\leq q$:\quad
		$X\mbox{ strongly relevant} \Rightarrow X \in F_{q,\infty}^{K,\alpha}$ 
		
	\end{theorem*} 
	\begin{proof}
		By definition, $X$ belonging to $F_{q,\infty}^{K,\alpha}$ means that there is
		at least one tree (grown with parameters $q$, $\alpha$ and $K$) in which $X$
		receives a strictly positive score for its split, i.e. such that $Y$ depends on
		$X$ conditionally to the variable assignement defined by the path from the root
		node to the node where $X$ is used to split. Let us show that one such tree always
		exists whatever $K$ and $\alpha$ when $X$ is strongly relevant and $r\leq q$.
		
		Within the infinite ensemble, let us consider only the trees grown using all
		$r$ relevant variables (and $q-r$ irrelevant ones randomly selected). Given that
		$r\leq q$ and given that only relevant features can be kept in memory, these
		trees are always explored whatever the value of $\alpha\geq 0$. Among these
		trees, let us furthermore only consider those such that irrelevant variables
		are tested in each branch only when all relevant variables (including $X$) are
		exhausted. These trees are always explored whatever the value of $K$. This
		derives from the fact that a relevant variable can always be picked with non
		zero probability at any tree node, except if all relevant variables have been
		tested above that node. Indeed, except in this latter case, the $K$ tested
		variables can always include at least one relevant variable. If some relevant
		variable gets a non zero score, one relevant variable will be automatically
		used to split since irrelevant variables can only get zero scores. Even when
		all tested relevant variables get a zero score, one of them can still be
		selected instead of an irrelevant one given that tie are resolved by
		randomization.
		
		Let us denote by $\tau_R$ the set of trees as just defined and let us show that
		$X$ gets a non zero score in at least one tree in $\tau_R$.
		
		By definition 2 and property 1, $X$ strongly relevant implies that there exists
		at least one assignement of values to all relevant variables but $X$ such that
		conditionally to this assignement, $Y$ is dependent on $X$. In each tree in
		$\tau_R$, there is a path from the root node to a node where $X$ is used to
		split that is compatible with this assignement. Let us assume that $X$ always
		gets a zero score in all these compatible paths and show that this leads to a
		contradiction.
		
		If all relevant variables are tested above $X$ in a compatible path then $X$
		should receive a non zero score at its node, which would contradict our
		hypothesis. Thus, $X$ can only be tested in a compatible path before all
		relevant variables have been tested. Given our hypothesis that $X$ only gets
		zero scores, if $X$ is used to split in one compatible path, then there exists
		another tree in $\tau_R$ with the same splits above $X$ in the compatible path
		and with the split on $X$ replaced by a split on another relevant variables
		(because of tie randomization or because of the randomization due to $K<q$). In
		this new tree, $X$ is thus used to split at least one level below in the
		compatible path. Applying this argument recursively, one can thus show that
		there is at least one tree in $\tau_R$ where $X$ is the last variable used to
		split in the compatible path. In this tree, $X$ thus gets a non zero score,
		which contradicts the hypothesis and therefore concludes the theorem.
	\end{proof}
	
	\section{Convergence analysis}
	
	\subsection{Simplifying assumptions}
	
	\label{app:assumption-convergence}
	
	Below, we compute analytically the average number of trees needed to
	find all relevant variables in the chaining and clique scenarios and
	we derive transition matrices of Markov chains that model the
	evolution of the number of variables found through the iterations in
	the three scenarios. These results are obtained assuming $K=q$ and
	$r\leq q$, and with either $\alpha=0$ (RS) or $\alpha=1$ (SRS).
	
	To make these derivations possible and independent of a particular
	data distribution, one needs furthermore to simplify the decision tree
	growing algorithm in the case of the chaining and clique
	scenarios. In what follows, trees are thus assumed to be grown such
	that a unique variable is selected at each tree level and this
	variable is selected at random among all variables $X$ such that
	$Y\nindep X|B$ where $B$ is the set of all variables tested at
	previous levels.
	
	In the clique scenario, this assumption implies that only one
	variable of the clique will get a non-zero importance when all
	clique variables are selected at one iteration of RS/SRS (since only
	the last variable of the clique tested along a tree branch can get a
	non-zero score and this variable is the same in each branch given
	our tree growing assumption). This corresponds to a pessimistic
	scenario. Indeed, with standard unconstrained trees, several
	relevant variables could be found at one iteration given that the
	ordering of the variables, and thus the last variable of the clique
	tested, might differ from one tree branch to another.  As a
	consequence, the tree growing assumption will lead to an
	overestimation of the number of trees needed to reach
	convergence. In the chaining scenario, the simplified tree growing
	algorithm implies that all relevant variables selected at one
	iteration of RS/SRS together with their minimal conditioning will
	get a non-zero importance. This corresponds this time to an
	optimistic scenario, as, with unconstrained trees, such variable
	might not be detected at one iteration depending on the exact data
	distribution. This will thus lead this time to an underestimation of
	the number of trees needed to reach convergence. Note however that,
	in both cases, these over/under-estimations will affect both RS and
	SRS in the same proportion and thus our assumption will not impact
	their relative performance.
	
	Note that in the marginal-only scenario, given that all relevant
	variables are marginally and strongly relevant, they will always get
	a non-zero importance as soon as they are selected at one
	iteration. Our estimations below are thus not impacted by the
	simplification of the tree growing algorithm.
	
	\subsection{Average times}
	\label{sec:avgtime}
	
	\paragraph{Chaining.}

	Let us denote by $T^{RS}_{chain}(i,p,q)$ ($1\leq i\leq r$) the
	average number of iterations needed to find the feature $X_{i}$ of
	degree $i-1$ and by $T^{SRS}_{chain}(i,p,q)$ the average number of
	iterations needed to find the same feature with the SRS algorithm
	(that forces the selection of already found relevant
	variables). Given our assumptions above, each tree will be able to
	identify all relevant variables $X$ it gets as soon as it gets also
	the relevant variables in its minimal conditioning. Note that
	$T^{RS/SRS}_{chain}(i,p,q)$ can also be interpreted as the average
	time needed to find the first $i$ relevant features, given that one
	can not find $X_i$ without finding all features $X_{j}$ with $1\leq
	j < i$. $T^{RS/SRS}_{chain}(r,p,q)$ also represents the average
	number of iterations needed to find all relevant variables under the
	chain assumption.
	
	\begin{theorem}
		Under our assumptions, the $T^{RS}_{chain}$ function can be computed as follows:
		\begin{equation}\label{ref:naive}
		T^{RS}_{chain}(i,p,q) = \prod_{l=0}^{i-1} \frac{p-l}{q-l}
		\end{equation}
	\end{theorem}
	\begin{proof}
		Indeed, $T^{RS}_{chain}(i,p,q)$ is the mean of a geometric
		distributed random variable with a probability of success defined as
		the probability of drawing the $i-1$ variables in $X_{i}$'s
		conditioning and $X_{i}$ at the same time, which is given by:
		\begin{equation}
		\frac{\binom{p-i}{q-i}}{\binom{p}{q}} = \prod_{l=0}^{i-1} \frac{q-l}{p-l}.
		\end{equation}
	\end{proof}
	
	\begin{theorem}
		Under the same assumption, $T^{SRS}_{chain}(i,p,q)$ can be computed as follows:
		\begin{equation}\label{ref:smart}
		T^{SRS}_{chain}(i,p,q) = \sum_{l=0}^{i-1} \frac{p-l}{q-l} - (i-1)
		\end{equation}
	\end{theorem}
	\begin{proof}
		Let us show this by induction on $i$. The base case corresponds to $i=1$. In this case, we have:
		$$T^{SRS}_{chain}(1,p,q) = T^{RS}_{chain}(1,p,q) =
		\frac{p}{q},$$ which satisfies Eqn (\ref{ref:smart}). Let us assume
		that Eqn. (\ref{ref:smart}) is satisfied for $i<i'$ and let us show
		that it is satisfied for $i=i'$. $T^{SRS}_{chain}(i',p,q)$ can be
		defined as follows:{\small
			\begin{eqnarray}\label{ref:recdeftsmart}
			\begin{split}
			T^{SRS}_{chain}(i',p,q) =  \frac{q}{p} T^{SRS}_{chain}(i'-1,p-1,q-1) + \\(1-\frac{q}{p}) (1 + T^{SRS}_{chain}(i',p,q)).
			\end{split}
			\end{eqnarray}}
		One can indeed distinguish two cases:
		\begin{itemize}
			\item $X_1$ is selected at the first iteration (this happens with
			probability $q/p$): the average time needed to find feature
			$X_{i'}$ of degree $i'-1$ then becomes the time needed to find a
			feature of degree $i'-2$ when one is allowed to draw $q-1$
			features among $p-1$, which is $T^{SRS}_{chain}(i'-1,p-1,q-1)$
			\item $X_1$ is not selected at the first iteration (this happens
			with probability $1-q/p$): in this case, the first iteration is
			useless and thus the number of iterations needed will be
			$1+T^{SRS}_{chain}(i',p,q)$.
		\end{itemize}
		Eqn. (\ref{ref:recdeftsmart}) can be used to compute $T^{SRS}_{chain}$ recursively:
		\begin{eqnarray}\label{ref:recdeftsmart2}
		T^{SRS}_{chain}(i',p,q)  =  T^{SRS}_{chain}(i'-1,p-1,q-1) + (\frac{p}{q}-1).
		\end{eqnarray}
		Deriving Eqn. (\ref{ref:smart}) from Eqn. (\ref{ref:recdeftsmart2}) is then straightforward, which concludes the proof by induction.
	\end{proof}
	Eqn. (\ref{ref:smart}) shows that the average time needed to find the
	$i$ first features is equal to the sum of the time needed to find all
	features individually minus the number of features. This last term
	takes into account the fact that by chance, one might find several
	features at once.

	\paragraph{Clique.}

	Let us denote by $T^{RS}_{cl}(i,p,q)$ and $T^{SRS}_{cl}(i,p,q)$, the
	average time needed to find $i$ features (among $r$) from the clique
	respectively with the RS and the SRS algorithm. Given our
	assumptions above, when the tree growing algorithm is given all $r$ relevant
	features, it will be able to identify one (and only one) feature
	from the clique at random. If it has already found $i$ features from
	the clique, the chance to get a new one, when all $r$ features are
	selected among the $q$ ones, will thus be $(r-i)/r$, i.e., the probability to test one of the $r-i$ not yet found features after all other $r$ features from the clique.
	
	\begin{theorem}
		\begin{equation}\label{eqn:naiveclique}
		T^{RS}_{cl}(i,p,q) = \left( \prod_{l=0}^{r-1} \frac{p-l}{q-l} \right) \cdot \left(  \sum_{l=0}^{i-1} \frac{r}{r-l}\right) \\
		\end{equation}
	\end{theorem}
	\begin{proof}
		The first factor in Eqn.(\ref{eqn:naiveclique}) is the inverse of
		the probability of selecting all $r$ relevant features at once. Each
		term of the sum in the second factor corresponds to the inverse of
		the probability of testing a new relevant variables, not yet found,
		at the bottom of the tree. As discussed above, this probability is
		$\frac{r-l}{r}$ when we have already found $l$ features from the
		clique.
	\end{proof}
	
	\begin{theorem}
		\begin{equation}\label{eqn:smartclique}
		T^{SRS}_{cl}(i,p,q) = \sum_{l=0}^{i-1} \frac{r}{r-l} \prod_{m=l}^{r-1} \frac{p-m}{q-m}
		\end{equation}
	\end{theorem}
	\begin{proof}
		Each term of the sum represents the average time needed to find a
		new clique feature given that we have already found $l$
		features. This time is equal to one over the probability of finding
		a new feature when we have already found $l$ of them. This latter is
		the probability of selecting among $q$ the $r-l$ missing relevant
		features (i.e., $\prod_{m=l}^r\frac{q-m}{p-m}$) times the
		probability of testing one of the missing relevant features at the
		bottom of the tree (i.e., $(r-l)/r$).
	\end{proof}
	
	When $i=1$, $T^{SRS}_{cl}(1,p,q) =
	T^{RS}_{cl}(1,p,q)$. Intuitively, it indeed takes the same time for
	the RS and the SRS algorithms to find the first relevant
	features. When $i$ increases however, the SRS algorithm becomes
	faster and faster than the RS algorithm. Indeed, the RS
	algorithm always needs to find all $r$ clique features, while the
	SRS one only needs to find the $r-i$ missing relevant features.

	\subsection{Markov chain interpretation}\label{sec:mc}
	
	Let us denote by $N^{X,Y}_{t}$ the number of variables found for $t$
	iterations, with $X=c$, $X=g$, and $X=m$ respectively for the chain
	hypothesis, the clique hypothesis and the marginal only hypothesis
	(as defined in the first section of this document) and $Y=n$ and
	$Y=s$ respectively for the RS and SRS algorithms. All these random
	variables follow order 1 Markov chains. The transition probabilities
	are provided below for each chain (without proof), under the
	assumptions given in Section~\ref{app:assumption-convergence}.
	
	\paragraph{Chain hypothesis.}
	
	{\small
		\begin{equation}
		P(N_t^{c,n}=l_1|N_t^{c,n}=l_2) = \left\{
		\begin{array}{ll}
		0 & \mbox{if } l_1<l_2\\
		\frac{\binom{p-r}{q-l_1}}{\binom{p}{q}} & \mbox{if } l_1>l_2\\
		1-\sum_{i=l_2+1}^r \frac{\binom{p-r}{q-i}}{\binom{p}{q}} &\mbox{if } l_1=l_2
		\end{array}
		\right.
		\end{equation}
		
		\begin{equation}
		P(N_t^{c,s}=l_1|N_t^{c,s}=l_2) = \left\{
		\begin{array}{ll}
		0 & \mbox{if } l_1<l_2\\
		\frac{\binom{p-r}{q-l_1}}{\binom{p-l_2}{q-l_2}} & \mbox{if } l_1>l_2\\
		1-\sum_{i=l_2+1}^r \frac{\binom{p-r}{q-i}}{\binom{p-l_2}{q-l_2}} &\mbox{if } l_1=l_2
		\end{array}
		\right.
		\end{equation}
		
		\paragraph{Clique hypothesis.}
		
		\begin{equation}
		P(N_t^{g,n}=l_1|N_t^{g,n}=l_2) = \left\{
		\begin{array}{ll}
		0 & \mbox{if } l_1<l_2\\
		1-\frac{\binom{p-r}{q-r}}{\binom{p}{q}} \frac{r-l_2}{r} &\mbox{if } l_1=l_2\\
		\frac{\binom{p-r}{q-r}}{\binom{p}{q}} \frac{r-l_2}{r} & \mbox{if } l_1=l_2+1\\
		0 &\mbox{if } l_1>l_2+1
		\end{array}
		\right.
		\end{equation}
		
		\begin{equation}
		P(N_t^{g,s}=l_1|N_t^{g,s}=l_2) = \left\{
		\begin{array}{ll}
		0 & \mbox{if } l_1<l_2\\
		1-\frac{\binom{p-r}{q-r}}{\binom{p-l_2}{q-l_2}} \frac{r-l_2}{r} &\mbox{if } l_1=l_2\\
		\frac{\binom{p-r}{q-r}}{\binom{p-l_2}{q-l_2}} \frac{r-l_2}{r} & \mbox{if } l_1=l_2+1\\
		0 &\mbox{if } l_1>l_2+1
		\end{array}
		\right.
		\end{equation}
	}
	\paragraph{Marginal only hypothesis.}
	
	\begin{equation}
	P(N_t^{m,n}=l_1|N_t^{m,n}=l_2) = \left\{
	\begin{array}{ll}
	0 & \mbox{if } l_1<l_2\\
	\frac{\binom{r-l_2}{l_1-l_2} \binom{p-r+l_2}{q-l_1+l_2}}{\binom{p}{q}} &\mbox{if } l_1>l_2\\
	\frac{\binom{p-r+l_2}{q}}{\binom{p}{q}} &\mbox{if }l_1=l_2\\
	\end{array}
	\right.
	\end{equation}
	
	\begin{equation}
	P(N_t^{m,s}=l_1|N_t^{m,s}=l_2) = \left\{
	\begin{array}{ll}
	0 & \mbox{if } l_1<l_2\\
	\frac{\binom{r-l_2}{l_1-l_2} \binom{p-r}{q-l_1}}{\binom{p-l_2}{q-l_2}} &\mbox{if } l_1>l_2\\
	\frac{\binom{p-r}{q-l_2}}{\binom{p-l_2}{q-l_2}} &\mbox{if }l_1=l_2\\
	\end{array}
	\right.
	\end{equation}

	\section{Details for Section \ref{sec:empirical}}
	\label{app:results}
	
	In this section, we give more details about our practical implementation of SRS and performed experiments.
	
	\subsection{On the use of a random probe to distinguish relevant features from irrelevant features.}
	
	As explained in Section~\ref{sec:empirical}, we add an artificial irrelevant feature in data as a random probe. By comparison with that probe of importances scores, one can distinguish relevant features (better than the probe) from irrelevant features. Through iterations, we can compute a \textit{p-value} score which is the percentage of times a variable has been better than the probe. If the \textit{p-value} is above a given threshold $\beta$ then the feature is likely relevant. Moreover, a variable has to be sampled more than $L$ times in $Q$ sets to insure that the \textit{p-value} is reliable. Then at each iteration, the variables that satisfy the two criteria are added to $F$. In the following experiments, we choose arbitrarily $L=10$ and $\beta=95\%$.
	
	\subsection{On the datasets and on the protocol}
	
	We evaluate the accuracy of all these methods on a list of both artificial and real classifications problems (all but madelon are real data) described in Table~\ref{table:datasets} and publicly available in the UCI machine learning repository \cite{Lichman2013uci}. For each dataset, we separate it into two random partitions of the same size (i.e., the same number of samples) to have a training set and a test set. There is no optimization of the parameters. For all datasets, the procedure was repeated 50 times, using the same random partitions between all methods. Following results are averages over those 50 runs. 
	
	\subsection{Detailed results}

	Table~\ref{table:accuracy} is average accuracy scores obtained on all datasets for each method for some parameters. We consider different sizes of memory (i.e., parameter $q$) and different value for the parameter $\alpha$ for the SRS algorithm. This allows to consider every behaviour of the SRS algorithm : without memory ($\alpha=0$) which is equivalent to the Random Subspace method, with a full memory ($\alpha=1$) and a non-full memory ($\alpha=0.5$). For both methods (RS and SRS), a single extra-tree is build at each iteration. The randomization parameter of the extra-tree is set to its maximal value (ie., all features).
	For the tree-based ensemble methods, we consider different values for the randomization parameter. This parameter reduces the ability to consider the whole dataset in once and in that it relates in a way to the size of the memory of SRS. We choose for that parameter values of 0.01, 0.1 and 1 corresponding to considering respectively 1\%, 10\%, 100\% of all features at each node.

	\begin{table}
		\centering
		\scriptsize
		\begin{tabular}{c|cc} \hline
			Dataset & \# samples & \# features \\ \hline
			arcene & 100 & 10000\\
			breast2 & 295 & 24496 \\
			cina0 & 16033 & 132\\
			isolet & 7797 & 617 \\
			madelon & 2000 & 500 \\
			marti0 & 500 & 1024\\
			reged0 & 500 & 999\\
			secom & 1567 & 591\\
			mnist & 70000 & 784 \\
			mnist3v8 & 13966 & 784 \\
			mnist4v9 & 13782 & 784 \\
			sido0 & 12678 & 4932 \\
			tis & 13375 & 927 \\ \hline
		\end{tabular}
		\caption{Dataset specifications}
		\label{table:datasets}
	\end{table}
	
	\begin{table*}
		\centering
		\tiny
		\begin{tabular}{|l|ccc|ccc|ccc||ccc|ccc|} \hline
			& \multicolumn{9}{c||}{SRS} & \multicolumn{6}{c|}{Tree-based ensemble methods} \\ \hline
			& \multicolumn{3}{c|}{q=0.01} & \multicolumn{3}{c|}{q=0.05} & \multicolumn{3}{c||}{q=0.1} & \multicolumn{3}{c|}{RF} & \multicolumn{3}{c|}{ET}  \\ \hline
			& \multicolumn{9}{c||}{$\alpha$} & \multicolumn{6}{c|}{Randomization parameter $K$}  \\ \hline
			&  0.0  & 0.5  & 1.0  & 0.0  & 0.5  & 1.0  & 0.0  & 0.5  &  1.0  & 0.01  &  0.1  & 1  & 0.01  & 0.1  & 1  \\ \hline
			arcene      &0.743 &0.717 &0.717 &0.743 &0.743 &0.743 &0.732 &0.732 &0.732 &0.717 &0.706 &0.678 &0.739 &0.729 &0.701 \\ 
			breast2     &0.649 &0.647 &0.647 &0.651 &0.651 &0.650 &0.654 &0.654 &0.654 &0.646 &0.649 &0.649 &0.650 &0.654 &0.651 \\ 
			cina0       &0.755 &0.755 &0.777 &0.809 &0.929 &0.873 &0.931 &0.933 &0.921 &0.933 &0.939 &0.939 &0.931 &0.934 &0.934 \\ 
			isolet      &0.906 &0.899 &0.336 &0.944 &0.945 &0.766 &0.949 &0.950 &0.817 &0.936 &0.940 &0.912 &0.943 &0.951 &0.943 \\ 
			madelon     &0.558 &0.689 &0.745 &0.639 &0.858 &0.861 &0.673 &0.845 &0.845 &0.620 &0.700 &0.754 &0.608 &0.690 &0.815 \\ 
			marti0      &0.881 &0.881 &0.881 &0.874 &0.874 &0.874 &0.870 &0.870 &0.870 &0.878 &0.870 &0.866 &0.879 &0.868 &0.854 \\ 
			reged0      &0.880 &0.966 &0.939 &0.885 &0.974 &0.974 &0.898 &0.974 &0.974 &0.882 &0.963 &0.960 &0.881 &0.948 &0.978 \\ 
			secom       &0.935 &0.935 &0.930 &0.935 &0.931 &0.931 &0.934 &0.932 &0.932 &0.935 &0.933 &0.929 &0.935 &0.930 &0.928 \\ 
			mnist       &0.564 &0.823 &0.525 &0.959 &0.966 &0.905 &0.968 &0.970 &0.938 &0.964 &0.966 &0.953 &0.966 &0.971 &0.968 \\ 
			mnist3v8    &0.910 &0.941 &0.828 &0.980 &0.986 &0.958 &0.987 &0.989 &0.975 &0.980 &0.985 &0.978 &0.981 &0.988 &0.987 \\ 
			mnist4v9    &0.889 &0.957 &0.848 &0.981 &0.986 &0.960 &0.986 &0.988 &0.974 &0.983 &0.984 &0.974*&0.985 &0.987 &0.984*\\ 
			sido0       &0.970 &0.972 &0.953 &0.973 &0.968 &0.968 &0.974 &0.969 &0.969 &0.972 &0.973 &0.973*&0.973 &0.974 &0.960*\\ 
			tis         &0.751 &0.751 &0.757 &0.753 &0.887 &0.888 &0.844 &0.917 &0.915 &0.854 &0.916 &0.913*&0.856 &0.906 &0.914*\\  \hline
		\end{tabular}
		\caption{Average accuracy scores for all methods with specified parameters on original datasets. SRS and RS were computed with 10000 iterations and RF/ET with 10000 trees.}
		\label{table:accuracy}
	\end{table*}
	
	\begin{table*}
		\tiny \centering

		\subfloat[$q=0.1 \times p $]{
			\begin{tabular}{|c| c c c c c| } \hline 
				& \multicolumn{1}{c|}{RS} & \multicolumn{2}{c|}{SRS} 
				& \multicolumn{2}{c|}{ET}\\
				& \multicolumn{1}{c|}{$q = 0.1$} & $q=0.1$ & \multicolumn{1}{c|}{$q=0.1$} & 
				$k=0.1$ & $k=1.0$\\
				$q=0.1$& \multicolumn{1}{c|}{}& $\alpha=0.5$ & \multicolumn{1}{c|}{$\alpha=1.0$} & 
				& \\ \hline
				$\text{RS}$ & $-$ & $1/5/7$ & $\mathbf{6/4/3}$ 
				& $\mathbf{7/2/4}$ & $\mathbf{6/3/4}$\\
				$\text{SRS}_{\alpha=0.5}$ & $\mathbf{7/5/1}$ & $-$ & $5/8/0$ 
				& $\mathbf{9/2/2}$ & $\mathbf{10/2/1}$\\
				$\text{SRS}_{\alpha=1.0}$ & $3/4/6$ & $0/8/5$ & $-$ 
				& $5/2/6$ & $\mathbf{6/2/5}$ \\ \hline
		\end{tabular}} \hspace{1em}
		\subfloat[$q=0.01 \times p $]{
			\begin{tabular}{|c| c c c| } \hline 
				& \multicolumn{1}{c|}{RS} & \multicolumn{2}{c|}{SRS} \\
				& \multicolumn{1}{c|}{$q = 0.01$} & $q=0.01$ & \multicolumn{1}{c|}{$q=0.01$} \\
				$q=0.01$& \multicolumn{1}{c|}{}& $\alpha=0.5$ & \multicolumn{1}{c|}{$\alpha=1.0$}\\ \hline
				$\text{RS}$ & $-$ & $2/2/9$ & $5/2/6$\\
				$\text{SRS}_{\alpha=0.5}$ & $\mathbf{9/2/2}$ & $-$ &  $\mathbf{7/3/3}$\\
				$\text{SRS}_{\alpha=1.0}$ & $\mathbf{6/2/5}$& $3/3/7$& $-$ \\ \hline
		\end{tabular}}
		
		\caption{Pairwise t-test (with a significance level of $0.05$) comparisons : each element on line \textrm{$i$} and column \textrm{$j$} of the table in terms of Win/Draw/Loss is the result of the comparison for method $i$ vs. method $j$: the tree values indicate respectively on how many datasets method $i$ is significantly better / not significantly different / significantly worse than method $j$. All methods were computed with 10000 iterations or trees on all 14 datasets (from Table~\ref{table:datasets}) with parameters specified on columns. In \textbf{bold} when the first value is greater than other values.}
		\label{table:comparison}
	\end{table*}
	
\end{subappendices}

\chapter{Network Inference and Connectomics Challenge}
\label{ch:connectomics}

\begin{overview}

This chapter considers a specific machine learning task consisting in reconstructing a network from data. We first present the principle of GENIE3, originally designed to infer gene regulatory networks from samples of gene expression levels.
Then we propose a simple yet effective solution to the problem of connectome inference in calcium imaging data. The proposed algorithm consists of two steps. First, processing the raw signals to detect neural peak activities. Second, inferring the degree of association between neurons from partial correlation statistics. Section \ref{sec:connectomics-paper} summarises the methodology that led us to win the Connectomics Challenge, proposes a simplified version of our  method, and finally compares our result with respect to other inference methods.\\

\textbf{\textcolor{RoyalBlue}{References:}} Section \ref{sec:connectomics-paper} reproduces the following publication:\\[2mm]
 \bibentry{sutera2015simple}.\\[2mm] These results have also been published afterwards as a book chapter: \bibentry{sutera2017simple}. \\[2mm]

Note that Section \ref{sec:connectomics:preamble} was not in the original publication and aims at putting the proposed method in perspective with tree-based network inference techniques.
\end{overview}

\section{Motivation}

In systems biology, networks provide a natural representation for complex feature interactions (where features are biological entities such as genes, proteins, ... ) \citep{schrynemackers2013protocols}. Network inference consists in the reconstruction of such biological networks from high-throughput data \citep{de2010advantages}. 
Concretely, given a set of $p$ input variables $V=\{X_1,\dots, X_p\}$, it aims at inferring (or completing\footnote{Some network inference techniques use a priori knowledge including known interactions.}) a directed graph with $p$ nodes, where each node represents a variable, and an edge directed from one variable $X_i$ to another variable $X_j$ indicates a direct (causal) influence of $X_i$ on $X_j$ \citep{huynh2010inferring,louppe2014understanding}.  Sometimes, targeted networks are undirected and only represent interactions (i.e., conditional dependencies) between variables without any causal interpretation of the edge direction \citep{louppe2014understanding}. Figure \ref{fig:networks} illustrates two possible networks: Figure \ref{fig:networks(a)} is a network of (causal) influences represented through a directed graph while Figure \ref{fig:networks(b)} is a network of statistical dependencies represented by an undirected graph.

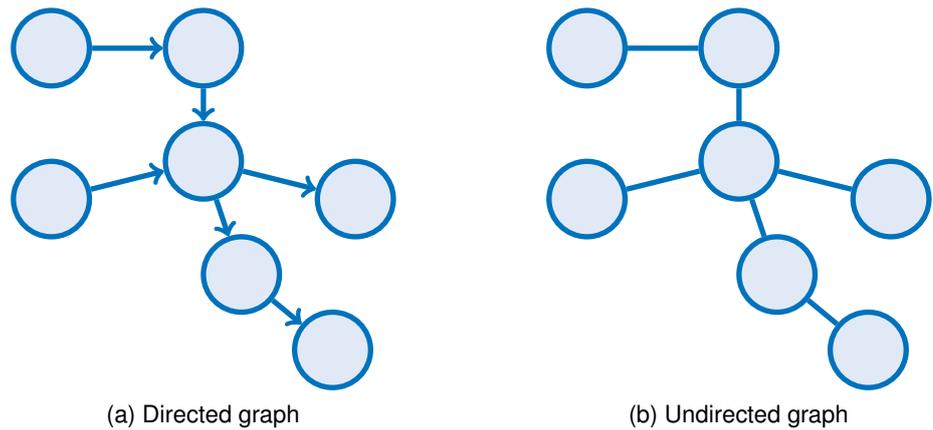
\begin{figure}
	\centering 
	\hfill
	\subfloat[Directed graph]{\begin{tikzpicture}

	 \node[circle,draw,RoyalBlue,fill=RoyalBlue!10!white, line width=2pt,minimum size=1cm] (z) at (0,0) {};
	 \node[circle,draw,RoyalBlue,fill=RoyalBlue!10!white, line width=2pt,minimum size=1cm] (a) at (2,0) {};
	 \node[circle,draw,RoyalBlue,fill=RoyalBlue!10!white, line width=2pt,minimum size=1cm] (b) at (0,-2) {};
	 \node[circle,draw,RoyalBlue,fill=RoyalBlue!10!white, line width=2pt,minimum size=1cm] (c) at (2,-1.5) {};
	 \node[circle,draw,RoyalBlue,fill=RoyalBlue!10!white, line width=2pt,minimum size=1cm] (d) at (4,-2) {};
	 \node[circle,draw,RoyalBlue,fill=RoyalBlue!10!white, line width=2pt,minimum size=1cm] (e) at (2.5,-3) {};
	 \node[circle,draw,RoyalBlue,fill=RoyalBlue!10!white, line width=2pt,minimum size=1cm] (f) at (3.7,-4) {};
	 \draw[RoyalBlue, line width=2pt,->] (z) -- (a);
	 \draw[RoyalBlue, line width=2pt,->] (a) -- (c);
	 \draw[RoyalBlue, line width=2pt,->] (b) -- (c);
	 \draw[RoyalBlue, line width=2pt,->] (c) -- (d);
	 \draw[RoyalBlue, line width=2pt,->] (c) -- (e);
	 \draw[RoyalBlue, line width=2pt,->] (e) -- (f);
	\end{tikzpicture}\label{fig:networks(a)}} \hfill
	\subfloat[Undirected graph]{\begin{tikzpicture}

	 \node[circle,draw,RoyalBlue,fill=RoyalBlue!10!white, line width=2pt,minimum size=1cm] (z) at (0,0) {};
\node[circle,draw,RoyalBlue,fill=RoyalBlue!10!white, line width=2pt,minimum size=1cm] (a) at (2,0) {};
	\node[circle,draw,RoyalBlue,fill=RoyalBlue!10!white, line width=2pt,minimum size=1cm] (b) at (0,-2) {};
	\node[circle,draw,RoyalBlue,fill=RoyalBlue!10!white, line width=2pt,minimum size=1cm] (c) at (2,-1.5) {};
	\node[circle,draw,RoyalBlue,fill=RoyalBlue!10!white, line width=2pt,minimum size=1cm] (d) at (4,-2) {};
	\node[circle,draw,RoyalBlue,fill=RoyalBlue!10!white, line width=2pt,minimum size=1cm] (e) at (2.5,-3) {};
	\node[circle,draw,RoyalBlue,fill=RoyalBlue!10!white, line width=2pt,minimum size=1cm] (f) at (3.7,-4) {};
	\draw[RoyalBlue, line width=2pt] (z) -- (a);
	\draw[RoyalBlue, line width=2pt] (a) -- (c);
	\draw[RoyalBlue, line width=2pt] (b) -- (c);
	\draw[RoyalBlue, line width=2pt] (c) -- (d);
	\draw[RoyalBlue, line width=2pt] (c) -- (e);
	\draw[RoyalBlue, line width=2pt] (e) -- (f);
	\end{tikzpicture}\label{fig:networks(b)}} \hfill
	\caption{Examples of inferred networks.}
	\label{fig:networks}
\end{figure}

Biological network inference consists in reconstructing a network in which biological entities interact (e.g., genes, proteins, cells, neurons, ...) \citep{tieri2016network} and two  applications in particular will be of interest in the rest of this chapter.

 In genomics, \textit{gene regulatory networks} represent interactions between genes and transcription factors\footnote{Transcription factors are proteins that regulates gene expression \citep{huynh2012machine}.} \citep{huynh2012machine,louppe2014understanding,tieri2016network}. The inference of such a network is based on gene expression levels. 

In neuroscience, the \textit{connectome} represents the neural connectivity, i.e., the interaction between neurons \citep{de2018connectivity,panagopoulos2018review}. Inferring the connectome from neural activity gives insights on effective brain structure. The effective brain structure gathers the structural (anatomical) connectivity (referring to physical connections between neurons, i.e., synapses) and the functional connectivity (referring to patterns of neuron activation regardless of spatiality, that are specific to a brain function and may change over time) and represent directional effects of neural elements on others \citep{sporns2007brain}.  The activation of a neuron (i.e., an action potential) is characterised by a sudden change in membrane potential by opening $Ca$ channels for instance \citep{simons1988calcium,tian2009imaging}. Calcium imaging can thus be used to record the neuronal activity by means of fluorescent marker. Calcium fluorescent levels are converted into neural activation times series that are in turn used for connectome inference \citep{panagopoulos2018review}

The interest in (biological) network inference has lead to many studies in the literature\footnote{Exhaustive lists of methods are given in \citep{de2018connectivity,panagopoulos2018review} for connectome inference and in \citep{huynh2010inferring,marbach2012wisdom} for gene regulatory inference.} involving models based on statistical measures (e.g., mutual information and (cross- and partial-)correlation) and probabilistic models (e.g., Bayesian networks and Gaussian graphical models) for example.

\section{Tree-based network inference based on variable importances}

Tree-based models have been also developed for network inference because they advantageously do not make any assumption about the target function, deal with non-linearity and take into account feature dependencies \citep{huynh2010inferring,schrynemackers2015classifying}. Both supervised and unsupervised approaches have been proposed for network inference and aim at deriving a score expressing the confidence for a pair of nodes to interact. In (tree-based) supervised approaches, a (tree-based) supervised model is usually constructed using a partial knowledge of the network and then used to assess the remaining untested pairs \citep{schrynemackers2013protocols}. In tree-based unsupervised methods, variable importances are derived and used to estimate the degree of association between two variables \citep{huynh2010inferring,louppe2014understanding}. We focus here on these latter methods.

\subsection{GENIE3} \label{sec:genie3}

\textit{GENIE3} \citep{huynh2010inferring} is an approach that aims at inferring a network of $p$ nodes by decomposing it into $p$ independent supervised learning problems.

Given a set of variables $V=\{X_1,\dots,X_p\}$, the $i^{th}$ sub-problem consists in learning a tree-based ensemble method (e.g., Random Forests or Extra-Trees) in order to predict the value 
of the variable $X_i$  from all remaining $p-1$ variables $X_j$ (with $j\ne i$). The contribution of $X_j$ in the prediction of $X_i$ gives an indication of the confidence level $p_{j,i}$ for the putative edge from $X_j$ to $X_i$ in the network (i.e., the degree of association between node $j$ and node $i$). Aggregating the confidence levels of all pairs of nodes allows to reconstruct the whole network by selecting the top-ranked interactions (i.e., above a given threshold of confidence level) for example.

In the case of tree-based ensemble models, confidence scores are given by the variable importance scores\footnote{$p_{j,i}$ is given by the importance of $X_j$ in the sub-problem in which $X_i$ is the target variable.}.  However, the aggregation of importance scores resulting from the $p$ sub-problems should be done cautiously if the variables are (i) of different scale, (ii) of different variability, or (iii) vary in the number of categories.

Indeed, \cite{huynh2010inferring} and \cite{louppe2014understanding} point out\footnote{It can also be retrieved in Chapter \ref{ch:importances}, especially from Equation \ref{eqn:mdiall}.} a positive bias in  the upper bound of (the sum of) all variable importances which depends on the target variable. In other words, if variables differ from each other on (i), (ii) or (iii), importance scores are not directly comparable without an appropriate normalisation\footnote{Let us consider a model learnt on a learning set $LS$ using Shannon entropy (respectively, variance/gini index) as impurity measure and MDI importance scores. One may normalise the target variable by its entropy (respectively, variance) estimated on $LS$  so that all variables have unit-entropy (resp., unit-variance) making importance scores comparable to each others. With MDA, one should consider a normalised accuracy metric.} before their aggregation.

\subsection{Direct interaction}

Network inference only considers direct connections between variables. Indirect effect must therefore be filtered out. However, neither relevance/usefulness nor importance score rankings may help to discriminate direct effects from indirect ones. Indeed, in all generality, variable importances do not guarantee that the importance of a feature \textit{indirectly} related to the target has a lower importance score than any other feature \textit{directly} related to the target. Moreover, one can imagine that several paths actually connect an input feature to the target (e.g., $X_1\rightarrow X_2 \rightarrow X_4$ and $X_1 \rightarrow X_4$ in Figure \ref{fig:networksdirect}) and so importance scores may reflect simultaneously direct and indirect interactions. 

Regardless of the degree of association between nodes, one should aim at discovering the Markov boundary of a node in order to identify all its direct neighbours \citep{aliferis2010local}. With strictly positive distributions (as seen in Section \ref{sec:MB}), it corresponds to identify all strongly relevant variables. 

By definition, a strongly relevant feature $X_m$ is such that $I(X_m;Y | V^{-m})>0$. Therefore, importance measures can, in theory, be adjusted to only detect strongly relevant features by only considering the deepest level of fully developed trees (corresponding to $B=V^{-m}$): \begin{eqnarray}\label{eqn:mdiXmstrong}
Imp^{mdi,last}_{\infty}(X_m) =  \sum_{v^{-m}} P(V^{-m}=v^{-m}) I(X_m;Y|V^{-m}=v^{-m})
\end{eqnarray}
where the sum of $v^{-m}$ is a sum of over all possible value configurations of the set of variables $V^{-m}$. A feature such that $Imp^{mdi,last}_{\infty}(X_m)>0$ is strongly relevant. However, the deepest level of fully developed trees is also the one where impurity decrease is estimated with the less samples and thus not reliable under other circumstances than infinite sample size.

In practice, a good heuristic to filter out as much as much as possible indirect interactions (and thus mainly focusing on strongly relevant variables) is to accentuate the masking effect (as strongly relevant variables can not be masked) by setting $K > 1$ (ideally, $K=p$) \citep{louppe2014understanding}. Additionally, an adequate stopping criterion may help to mitigate impurity miss-estimation effects (by avoiding estimation on too few samples) \citep{louppe2014understanding}.

\begin{figure}
	\centering 
	\begin{tikzpicture}

		\node[circle,draw,RoyalBlue,fill=RoyalBlue!10!white, line width=2pt,minimum size=1cm] (z) at (0,0) {$X_1$};
		\node[circle,draw,RoyalBlue,fill=RoyalBlue!10!white, line width=2pt,minimum size=1cm] (a) at (2,0) {$X_2$};
		\node[circle,draw,RoyalBlue,fill=RoyalBlue!10!white, line width=2pt,minimum size=1cm] (b) at (0,-2) {$X_3$};
		\node[circle,draw,RoyalBlue,fill=RoyalBlue!10!white, line width=2pt,minimum size=1cm] (c) at (2,-1.5) {$X_4$};
		\node[circle,draw,RoyalBlue,fill=RoyalBlue!10!white, line width=2pt,minimum size=1cm] (d) at (4,-2) {};
		\node[circle,draw,RoyalBlue,fill=RoyalBlue!10!white, line width=2pt,minimum size=1cm] (e) at (2.5,-3) {};
		\node[circle,draw,RoyalBlue,fill=RoyalBlue!10!white, line width=2pt,minimum size=1cm] (f) at (3.7,-4) {};
		\draw[RoyalBlue, line width=2pt,->] (z) -- node [above] {$p_{1,2}$} (a);
		\draw[RoyalBlue, line width=2pt,->] (a) -- node [right] {$p_{2,4}$}  (c);
			\draw[RoyalBlue, line width=2pt,->,dashed] (z) -- node [above] {$p_{1,4}$}  (c);
		\draw[RoyalBlue, line width=2pt,->] (b) -- node [above] {$p_{3,4}$}  (c);
		\draw[RoyalBlue, line width=2pt,->] (c) -- (d);
		\draw[RoyalBlue, line width=2pt,->] (c) -- (e);
		\draw[RoyalBlue, line width=2pt,->] (e) -- (f);
		\end{tikzpicture}
	\caption{Direct interaction may be outscored by indirect ones. Solid arrows represent direct interactions while dashed arrows represent indirect effects. $p_{1,3}$ may be numerically higher than $p_{3,4}$ while being associated to an indirect effect.}
	\label{fig:networksdirect}
\end{figure}
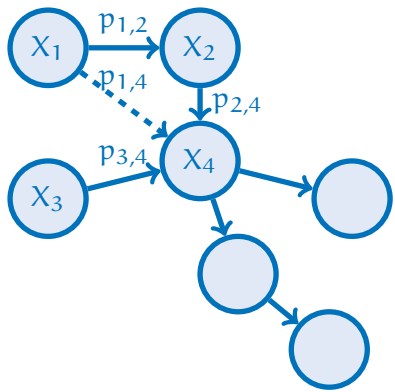

\subsection{Edge orientation}

In GENIE3, there is no explicit edge orientation despite that the importance score is usually asymmetrical ($p_{i,j} \ne p_{j,i}$) in opposition to symmetrical measures such as correlation or mutual information. In the method, only edges with confident scores above a given threshold\footnote{In their experiments, \cite{huynh2010inferring} set the threshold such that the number of inferred edges corresponds to the number of edges in the gold standard.} are considered. Only one confidence score can be above the threshold implying a seemingly edge orientation while, on the contrary, both $p_{i,j}$ and $p_{j,i}$ can be kept making the edge undirected. \cite{huynh2010inferring} further analyse the ability of GENIE3 to correctly deduce the edge orientation including a comparison between $p_{i,j}$ and $p_{j,i}$ (i.e., $p_{i,j} > p_{j,i}$ implying the edge $i\rightarrow j$, or vice versa). Despite relatively symmetrical inferred networks (i.e., only few edges are only directed), GENIE3 seems to infer fairly correctly the edge orientation at least considering $p_{i,j} > p_{j,i}$ when an edge is such that $i\rightarrow j$. More recently, \cite{bloebaum2018cause} investigate the asymmetry in the mean-squared errors of predicting the cause from the effect and the effect from the cause in order to determine the causal direction between two variables. Such researches are promising to infer edge orientation from observational data.

\section{Connectomics challenge} \label{sec:connectomics-paper}

\subsection{Preamble}\label{sec:connectomics:preamble}

In the previous section, we introduced network inference and GENIE3, a tree-based method to infer a gene regulatory network. We also presented the questions of the direct effect identification and edge orientation. 

Based on variable importances, GENIE3 provides excellent results in the context of gene regulatory network inference (best performer in the DREAM4 \textit{In Silico Multifactorial} challenge in 2009 and in the DREAM5 \textit{Network Inference} challenge in 2010). 
We however noticed that variable importances as usually used do not filter out indirect effects. Edge orientation seems promising but GENIE3-inferred networks are relatively symmetric and only few edges are undoubtedly oriented.

This section summarises contributions made in the scope of the competition "Neural Connectomics Challenge" organised in the context of 2014 ECML/PKDD conference \citep{battaglia17neuralbook}, consisting in inferring a connectome from fluorescent calcium data. In what follows, we present our solution, which was the winning solution of the challenge.

In the context of connectome inference, neural networks seem to consist of fewer edges than gene regulatory networks (proportionally to the number of nodes). Subsequently, the number of indirect effects should be higher and thus it is even more crucial to identify direct interactions (actual edges). The edge orientation is however comparable with gene regulatory network.

The GENIE3 approach suggests to decompose the inference of a network of $p$ nodes into $p$ independent sub-problems. In order to identify direct effects, one should consider ensemble methods with fully developed trees as the learning algorithm for each sub-problem.

At first sight, GENIE3 seems to be a good candidate for connectome inference. We however noticed that running the learning algorithm $p$ times (typically, $p=1000$) was computationally too expensive under time constraints pertaining to a machine learning challenge. We therefore opt for another learning algorithm - based on partial correlation - that is computationally advantageous\footnote{Especially for a fast development and parameter tuning.}. Conversely with GENIE3, partial correlation based approach aims at finding explicitly only direct interactions.

\subsection{Connectome inference}\label{sec:intro}

The human brain is a complex biological organ made of about 100 billion of
neurons, each connected to, on average, 7,000 other neurons
\citep{pakkenberg2003aging}. Unfortunately, direct observation of the
connectome, the wiring diagram of the brain, is not yet technically feasible.
Without being perfect, calcium imaging currently allows for real-time and
simultaneous observation of neuron activity from thousands of neurons,
producing individual time-series representing their fluorescence intensity.
From these data, the connectome inference problem amounts to retrieving the
synaptic connections between neurons on the basis of the fluorescence time-series. This problem is difficult to solve because of experimental issues,
including masking effects (i.e., some of the neurons are not observed or
confounded with others), the low sampling rate of the optical device with
respect to the neural activity speed, or the slow decay of fluorescence.

Formally, the connectome can be represented as a directed graph $G=(V,E)$,
where $V$ is a set of $p$ nodes representing neurons, and $E \subseteq
\left\{(i, j) \in V \times V\right\}$ is a set of edges representing direct
synaptic connections between neurons. Causal interactions are expressed by the
direction of edges: $(i, j) \in E$ indicates that the state of neuron $j$ might
be caused by the activity of neuron $i$. In those terms,  the connectome
inference problem is formally stated as follows:  \textit{Given the sampled
observations $\{ x^t_i \in \mathbb{R} | i \in V, t = 1, \dots, T \}$ of $p$
neurons for $T$ time intervals, the goal is to infer the set $E$ of connections
in $G$.}

In this section, we present a simplified - and almost as good - version of the
winning method\footnote{Code available at \url{https://github.com/asutera/kaggle-connectomics}} of the Connectomics
Challenge\footnote{\url{http://connectomics.chalearn.org}}, as a simple and
theoretically grounded approach based on signal processing techniques and
partial correlation statistics. The rest of this chapter is structured as follows:
Section~\ref{sec:filter} describes the signal processing methods applied on
fluorescent calcium time-series; Section \ref{sec:inference} then presents the
proposed approach and its theoretical properties; Section~\ref{sec:connectomicsresults}
provides an empirical analysis and comparison with other network inference
methods, while finally, in Section~\ref{sec:conclusion} we discuss our work and
provide further research directions. Additionally,
Appendix~\ref{app:optimized} further describes, in full detail, our actual
winning method which gives slightly better results than the method presented in
this paper, at the cost of parameter tuning. Appendix~\ref{app:supp} provides supplementary results on other datasets.

\subsection{Signal processing} \label{sec:filter}

Under the simplifying assumption that neurons are on-off units, characterised
by short periods of intense activity, or peaks, and longer periods of
inactivity, the first part of our algorithm consists of cleaning the raw
fluorescence data.
More specifically, time-series are processed using standard
signal processing filters in order to : (i) remove noise mainly due to fluctuations independent of calcium, calcium fluctuations independent of spiking activity, calcium fluctuations in nearby tissues that have been mistakenly captured, or simply by the imaging process ; (ii) to account for fluorescence low decay ; and (iii) to reduce the importance of
high global activity in the network. The overall process is illustrated in
Figure~\ref{fig:filtered-signal}.

\begin{figure}
\centering
\subfloat[Raw signal]{\includegraphics[width=0.45\textwidth]{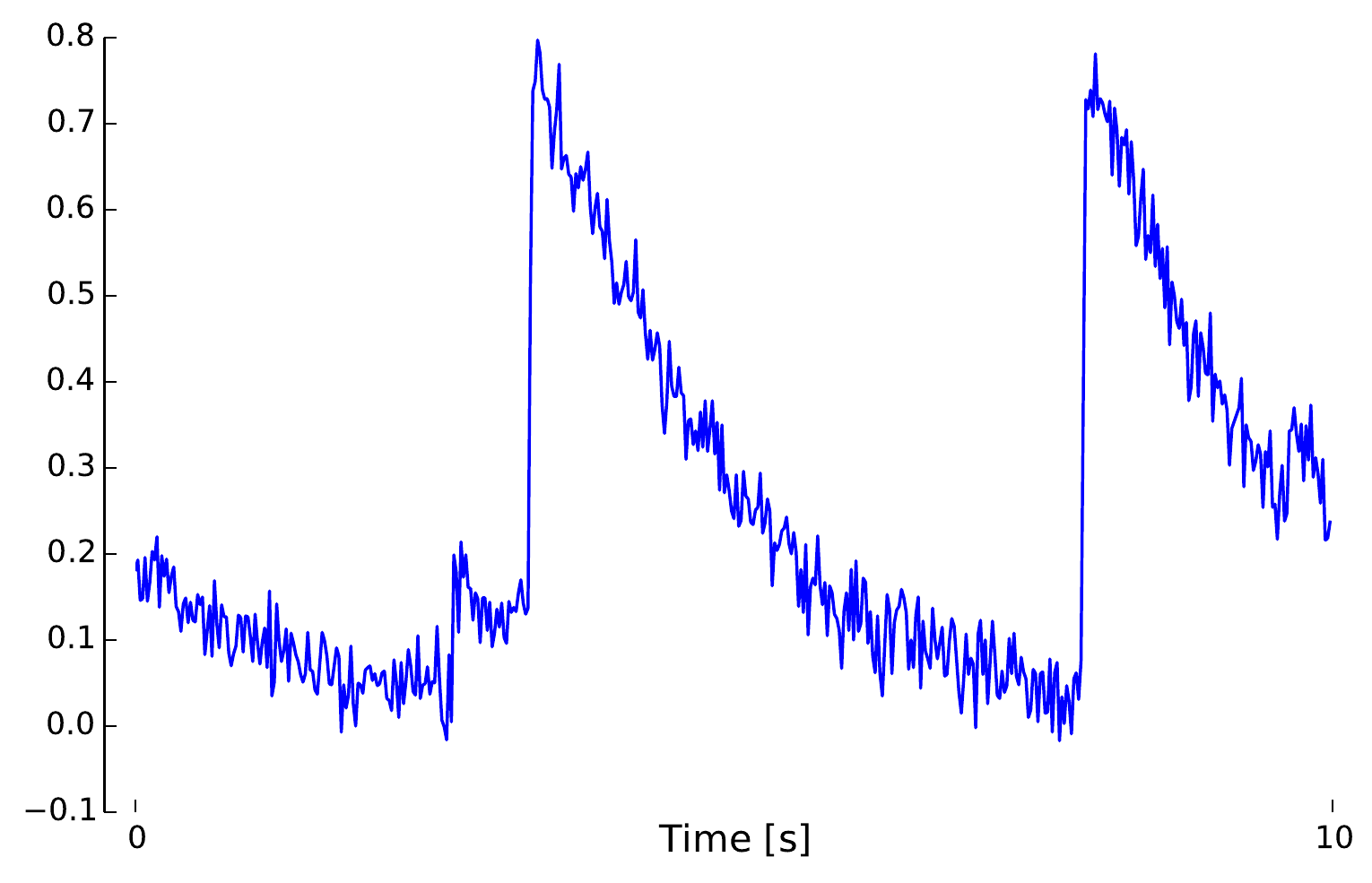} \label{fig:original_curve}}\\
\subfloat[Low-pass filter $f_1$]{\includegraphics[width=0.45\textwidth]{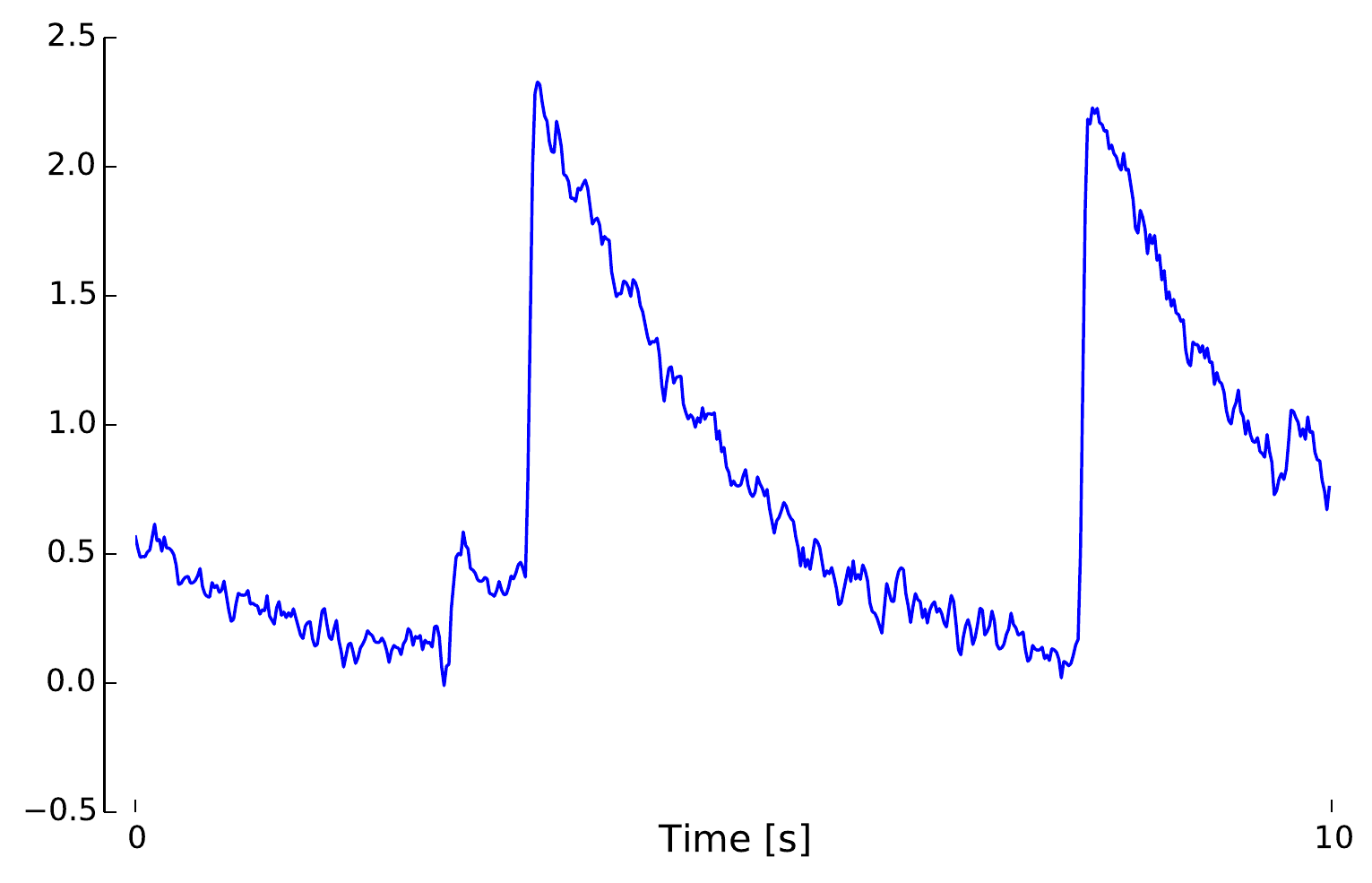} \label{fig:lp_curve}}\\
\subfloat[High-pass filter $g$]{\includegraphics[width=0.45\textwidth]{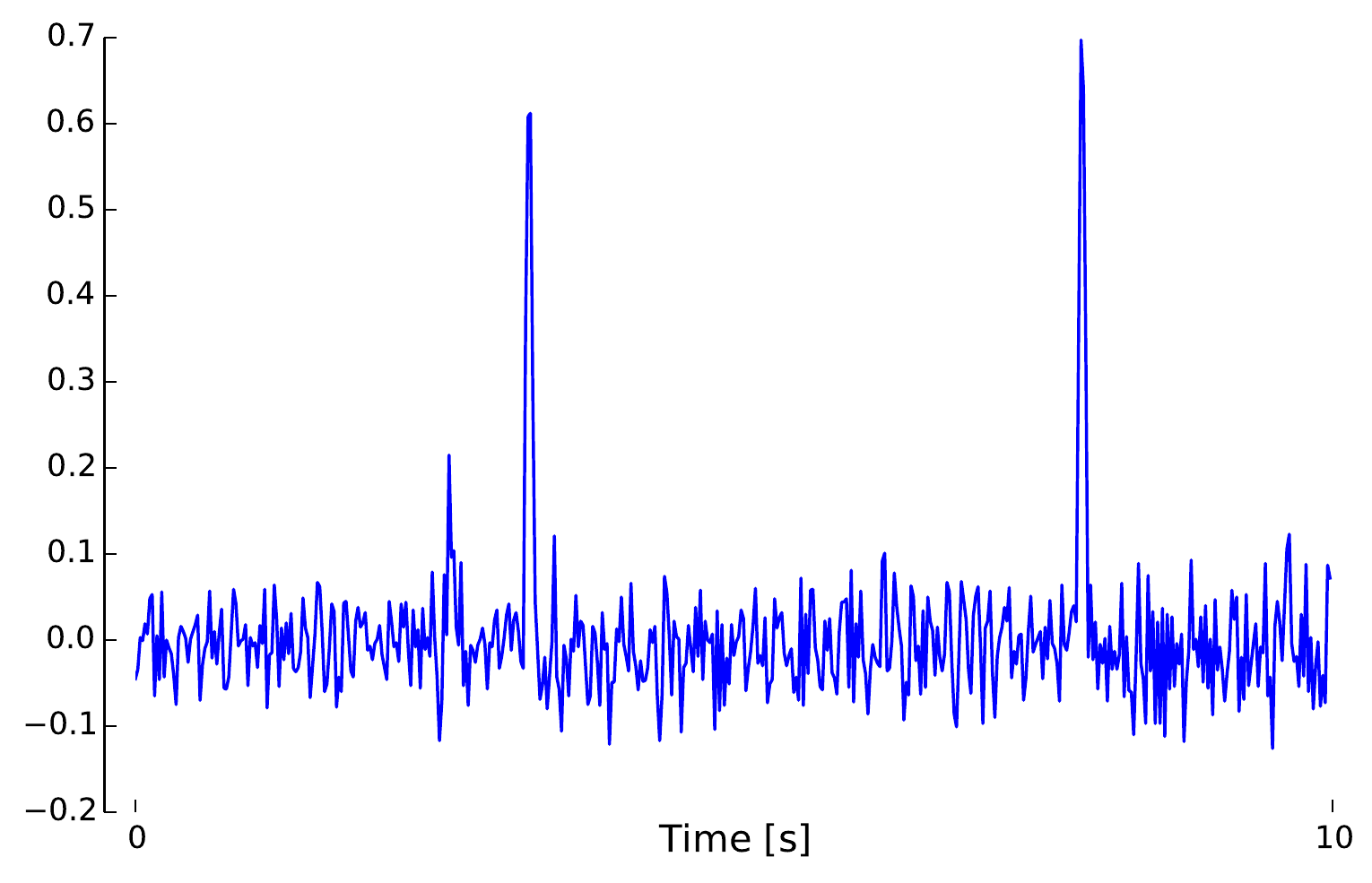} \label{fig:hp_curve}}\\
\subfloat[Hard-threshold filter $h$]{\includegraphics[width=0.45\textwidth]{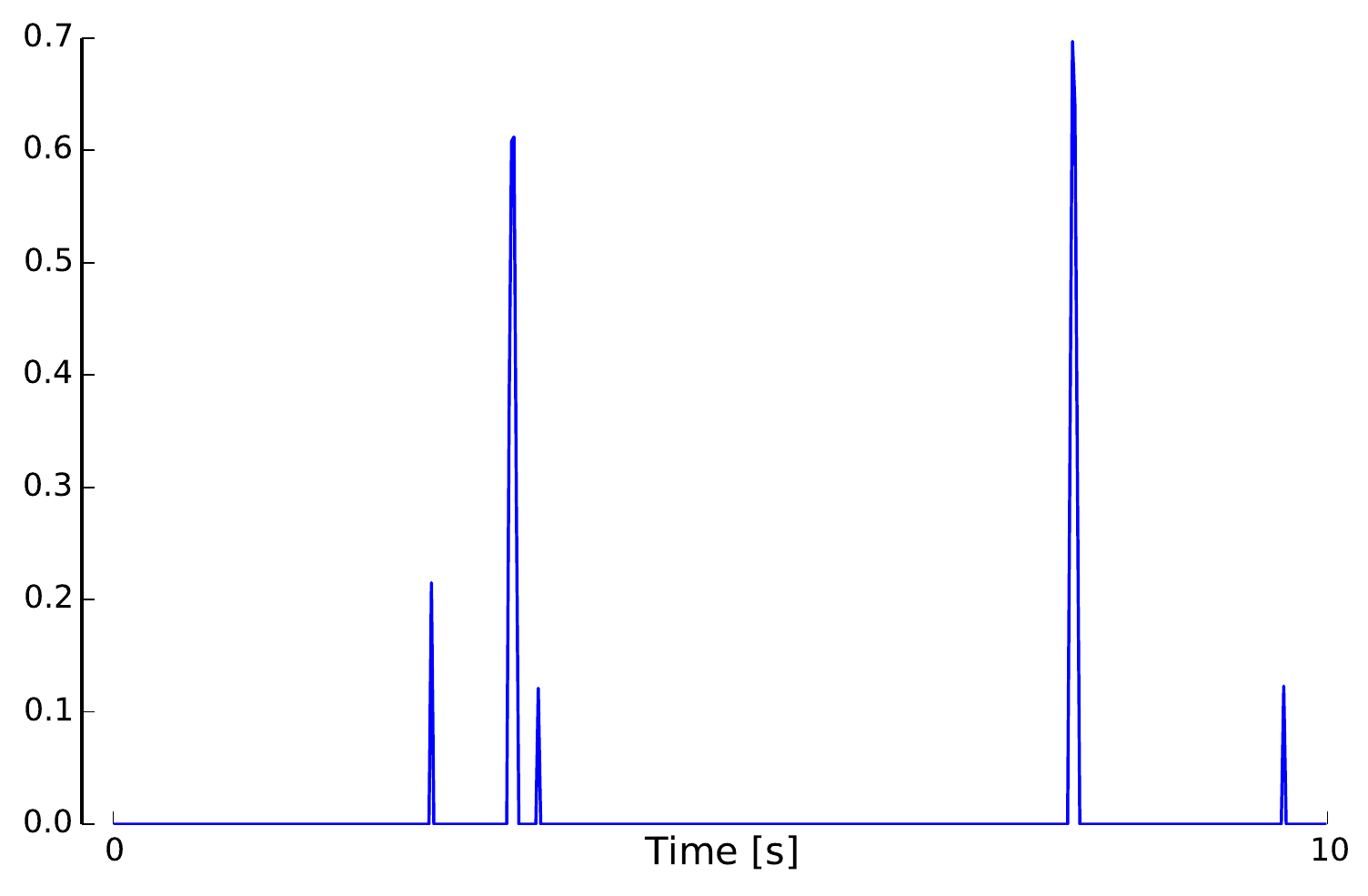} \label{fig:threshold_curve}}\\
\subfloat[Global regularization $w$]{\includegraphics[width=0.45\textwidth]{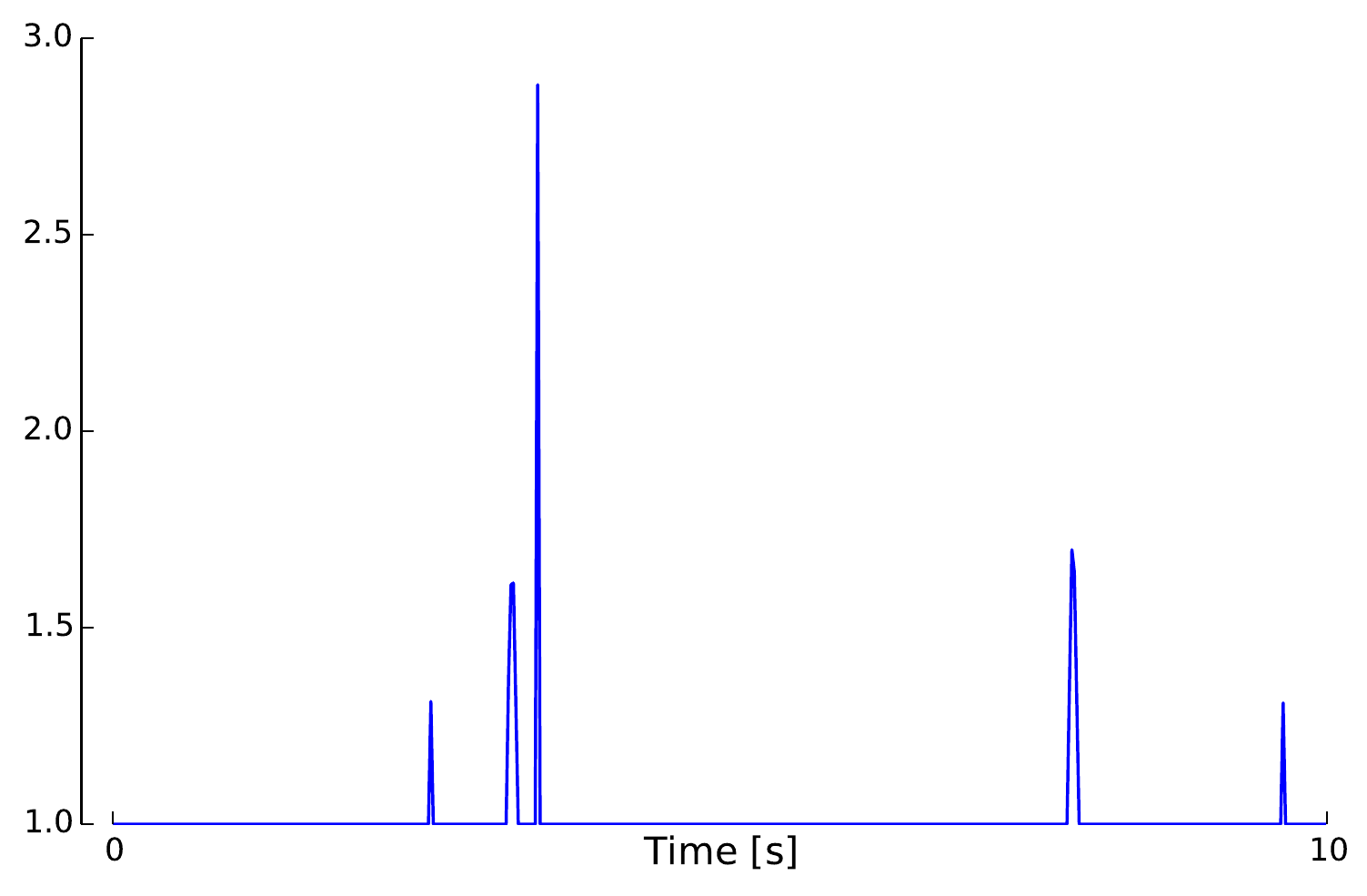} \label{fig:weight_curve}}\\
\caption{Signal processing pipeline for extracting peaks from the raw fluorescence data.}
\label{fig:filtered-signal}
\end{figure}

As Figure~\ref{fig:original_curve} shows, the raw
fluorescence signal is very noisy due to light scattering artifacts that
usually affect the quality of the recording~\citep{lichtman2011big}.
Accordingly, the first step of our pipeline is to smooth the signal, using
one of the following low-pass filters for filtering out high frequency noise:
\begin{align}
f_1(x^t_i) &= x^{t-1}_i + x^t_i + x^{t+1}_i, \label{eq:symetric-median} \\
f_2(x^t_i) &= 0.4 x^{t-3}_i + 0.8 x^{t-2}_i + x^{t-1}_i + x_i^t.
\label{eq:weighted-asymetric-median}
\end{align}
These filters are standard in the signal processing field \citep{kaiser1977data, oppenheim1983signals}. For the purposes of illustration, the effect of the filter $f_1$ on the signal
is shown in Figure \ref{fig:lp_curve}.

Furthermore, short spikes, characterized by a high
frequency, can be seen as an indirect indicator of neuron communication, while low frequencies of the signal mainly correspond to the slow
decay of fluorescence. To have a signal that only has high magnitude around instances where the spikes occur, the second step of our pipeline transforms the time-series into its backward
difference
\begin{align}
g(x^{t}_{i}) &= x^{t}_i - x^{t-1}_i, \label{eq:high-pass-filter}
\end{align}
as shown in Figure \ref{fig:hp_curve}.

To filter out small variations in the signal obtained after applying the
function $g$, as well as to eliminate negative values, we use the following
hard-threshold filter
\begin{align}\label{eqn:hfilter}
h(x^{t}_i) &= x^{t}_i \mathbb{1}(x^{t}_i \geq \tau) \text{ with } \tau > 0,
\end{align}
yielding Figure \ref{fig:threshold_curve} where $\tau$ is the threshold parameter and $\mathbb{1}$ is the indicator function.
As can be seen, the processed signal only contains clean spikes.

The objective of the last step of our filtering procedure is to decrease the
importance of spikes that occur when there is high global activity in the
network with respect to  spikes that occur during normal activity. Indeed, we
have conjectured that when a large part of the network is firing, the rate at
which observations are made is not high enough to be able to detect
interactions, and that it would therefore be preferable to lower their
importance by changing their magnitude appropriately. Additionally, it is
well-known that neurons may also spike because of a high global activity
\citep{stetter2012model}. In such  context, detecting pairwise neuron
interactions from the firing activity is meaningless. As such,
the signal output by $h$ is finally applied to the following function
\begin{align}
 w(x^{t}_i) &= (x^{t}_i + 1 )^{1 + \frac{1}{\sum_{j} x^{t}_j}}, \label{eq:magnify-filter}
\end{align}
whose effect is to magnify the importance of spikes that occur in cases of low
global activity (measured by $\sum_{j} x^{t}_j$), as observed, for instance,
around $t=4\text{s}$ in Figure~\ref{fig:weight_curve}. Note the particular case where there
is no activity, i.e., $\sum_{j} x^{t}_j = 0$, is solved by setting $w(x^{t}_i)
= 1$.

To summarise, the full signal processing pipeline of our simplified approach is defined by the composed function $w \circ h \circ g \circ
f_1$ (resp. $f_2$). When applied to the raw signal of Figure
\ref{fig:original_curve}, it outputs the signal shown in Figure
\ref{fig:weight_curve}.

\subsection{Connectome inference from partial correlation statistics}
\label{sec:inference}

Our procedure to infer connections between neurons first assumes that
the (filtered) fluorescence concentrations of all $p$ neurons at each
time point can be modelled as a set of random variables $X = \{X_1,
\dots, X_p\}$ that are independently drawn from the same time-invariant
joint probability distribution $P_X$. 
As a consequence, our inference method does not exploit the time-ordering of the observations (although time-ordering is exploited by
the filters).

Given this assumption, we then propose to use as a measure of the
strength of the connection between two neurons $i$ and $j$, the
\textit{Partial correlation} coefficient $p_{i,j}$ between their corresponding
random variables $X_i$ and $X_j$, defined by:
\begin{equation}
p_{i,j} =
-\frac{\Sigma^{-1}_{ij}}{\sqrt{\Sigma^{-1}_{ii} \Sigma^{-1}_{jj}}}, \label{eq:inverse}
\end{equation}
where $\Sigma^{-1}$, known as the precision or concentration matrix, is the inverse of the covariance matrix $\Sigma$ of $X$. 
Assuming that the distribution $P_X$ is a multivariate Gaussian
distribution ${\cal N}(\mu,\Sigma)$, it can be shown that $p_{i,j}$ is
zero if and only if $X_i$ and $X_j$ are independent given all other
variables in $X$, i.e., $X_i \perp X_j|X^{-i,j}$ where $X^{-i,j}= X
\setminus\{X_i,X_j\}$. Partial correlation (illustrated by Figure \ref{fig:partialcorrelation}) thus measures conditional
dependencies between variables ; therefore it should naturally only detect direct associations
between neurons and filter out spurious indirect effects.  The interest
of partial correlation as an association measure has already been
shown for the inference of gene regulatory networks
\citep{de2004discovery,Schafer2005shrinkage}.
Note that the partial correlation statistic is symmetric
(i.e. $p_{i,j}=p_{j,i}$). Therefore, our approach cannot identify the
direction of the interactions between neurons. We will see in
Section~\ref{sec:connectomicsresults} why this only slightly affects its
performance, with respect to the metric used in the Connectomics
Challenge. 

Practically speaking, the computation of all $p_{i,j}$ coefficients using Equation
\ref{eq:inverse} requires the estimation of the covariance matrix $\Sigma$
and then computing its inverse. Given that typically we have more
samples than neurons, the covariance matrix can be inverted in a
straightforward way. We nevertheless obtained some improvement by
replacing the exact inverse with an approximation using only the $M$
first principal components \citep{bishop2006pattern} (with
$M=0.8 p$ in our experiments, see Appendix~\ref{app:pca}). 

Finally, it should be noted that the performance of our simple method appears to
be quite sensitive to the values of parameters (e.g., choice of $f_1$ or $f_2$
or the value of the threshold $\tau$) in the combined function of the
filtering and inferring processes. One approach, further referred
to as \textit{Averaged Partial correlation} statistics, for improving
its robustness is to average correlation statistics over various
values of the parameters, thereby reducing the variance of its
predictions. Further details about parameter selection are provided in
Appendix~\ref{app:optimized}.

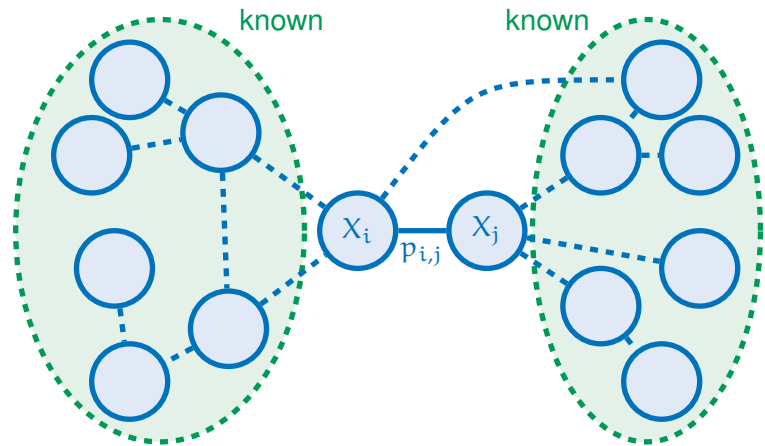
\begin{figure}
	\centering
	\begin{tikzpicture}
	\draw[line width=2pt, ForestGreen,dashed,fill=ForestGreen!10] (1.9,3) ellipse (1.9 and 2.8);
	\draw[line width=2pt, ForestGreen,dashed,fill=ForestGreen!10] (8.3,3) ellipse (1.5 and 2.8);
	\node[ForestGreen] at (3.5,5.8) {known};
	\node[ForestGreen] at (7,5.8) {known};
	\node[draw,circle,RoyalBlue,fill=RoyalBlue!10!white, line width=2pt,minimum size=1cm] (c1) at (4.5,3) {$X_i$};
	\node[draw,circle,RoyalBlue,fill=RoyalBlue!10!white, line width=2pt,minimum size=1cm] (c2) at (6.2,3) {$X_j$};
	\draw[line width=2pt,RoyalBlue] (c1) -- node[below] {$p_{i,j}$} (c2);

	\node[draw,circle,RoyalBlue,fill=RoyalBlue!10!white, line width=2pt,minimum size=1cm] (l1) at (1.5,5) {};
	\node[draw,circle,RoyalBlue,fill=RoyalBlue!10!white, line width=2pt,minimum size=1cm] (l2) at (2.7,4.3) {};
	\node[draw,circle,RoyalBlue,fill=RoyalBlue!10!white, line width=2pt,minimum size=1cm] (l3) at (1,4) {};
	\node[draw,circle,RoyalBlue,fill=RoyalBlue!10!white, line width=2pt,minimum size=1cm] (l4) at (1.3,2.5) {};
	\node[draw,circle,RoyalBlue,fill=RoyalBlue!10!white, line width=2pt,minimum size=1cm] (l5) at (2.8,1.7) {};
	\node[draw,circle,RoyalBlue,fill=RoyalBlue!10!white, line width=2pt,minimum size=1cm] (l6) at (1.5,1) {};
	\draw[dashed,line width=2pt,RoyalBlue] (l1) -- (l2);
	\draw[dashed,line width=2pt,RoyalBlue] (l3) -- (l2);
	\draw[dashed,line width=2pt,RoyalBlue] (l2) -- (c1);	
	\draw[dashed,line width=2pt,RoyalBlue] (l2) -- (l5);	
	\draw[dashed,line width=2pt,RoyalBlue] (l5) -- (c1);	
	\draw[dashed,line width=2pt,RoyalBlue] (l4) -- (l6);	
	\draw[dashed,line width=2pt,RoyalBlue] (l5) -- (l6);	
	
	\node[draw,circle,RoyalBlue,fill=RoyalBlue!10!white, line width=2pt,minimum size=1cm] (r1) at (8.5,5) {};
	\node[draw,circle,RoyalBlue,fill=RoyalBlue!10!white, line width=2pt,minimum size=1cm] (r2) at (9,4) {};
	\node[draw,circle,RoyalBlue,fill=RoyalBlue!10!white, line width=2pt,minimum size=1cm] (r3) at (7.7,4) {};
	\node[draw,circle,RoyalBlue,fill=RoyalBlue!10!white, line width=2pt,minimum size=1cm] (r4) at (9,2.5) {};
	\node[draw,circle,RoyalBlue,fill=RoyalBlue!10!white, line width=2pt,minimum size=1cm] (r5) at (7.7,2) {};
	\node[draw,circle,RoyalBlue,fill=RoyalBlue!10!white, line width=2pt,minimum size=1cm] (r6) at (8.5,1) {};
	\draw[dashed,line width=2pt,RoyalBlue] (r1) -- (r3);
	\draw[dashed,line width=2pt,RoyalBlue] (r3) -- (r2);
	\draw[dashed,line width=2pt,RoyalBlue] (r3) -- (c2);	
	\draw[dashed,line width=2pt,RoyalBlue] (c2) -- (r4);	
	\draw[dashed,line width=2pt,RoyalBlue] (c2) -- (r5);	
	\draw[dashed,line width=2pt,RoyalBlue] (r5) -- (r6);	
	\draw[dashed,line width=2pt,RoyalBlue] (c1) .. controls (6,5) .. (r1);
	\end{tikzpicture}
	\caption{Partial correlation coefficient $p_{i,j}$ measures the degree of direct association between $X_i$ and $X_j$ given all other nodes (in green areas).}
	\label{fig:partialcorrelation}
\end{figure}

\subsection{Experiments} \label{sec:connectomicsresults}
\paragraph{Data and evaluation metrics.}

We report here experiments on the \textit{normal-1,2,3}, and \textit{4}
datasets provided by the organisers of the Connectomics Challenge (see
Appendix~\ref{app:supp} for experiments on other datasets). Each of
these datasets is obtained from the simulation \citep{stetter2012model} of
different neural networks of 1,000 neurons and approximately 15,000 edges (i.e., a
network density of about 1.5\%). Each neuron is described by a calcium
fluorescence time-series of length $T=179500$. All inference methods compared
here provide a ranking of all pairs of neurons according to some association score. To assess the quality of this ranking, we compute both ROC and
precision-recall curves against the ground-truth network, which are represented by
the area under the curves and respectively  denoted AUROC and AUPRC. Only
the AUROC score was used to rank the challenge participants, but the precision-recall curve has been shown to be a more sensible metric for network
inference, especially when network density is small (see e.g.,
\cite{schrynemackers2013protocols}). Since neurons are not self-connected in
the ground-truth networks (i.e., $(i, i) \not \in E, \forall i \in V$), we
have manually set the score of such edges to the minimum possible association
score before computing ROC and PR curves.

\paragraph{Evaluation of the method.}

The top of Table \ref{tab:comparison} reports AUROC and AUPRC for all
four networks using, in each case, partial correlation with different
filtering functions. Except for the last two rows that use PCA, the
exact inverse of the covariance matrix was used in each case. These
results clearly show the importance of the filters. AUROC increases in
average from 0.77 to 0.93. PCA does not really affect AUROC scores, but
it significantly improves AUPRC scores. Taking the average over
various parameter settings gives an improvement of 10\% in AUPRC but
only a minor change in AUROC. The last row (``Full method'') shows the
final performance of the method specifically tuned for the challenge
(see Appendix \ref{app:optimized} for all details). Although this
tuning was decisive to obtain the best performance in the challenge,
it does not significantly improve either AUROC or AUPRC.

\begin{table}[t]
\centering
\small
\begin{tabular}{| l | c c c c | c c c c |}
\hline
& \multicolumn{4}{c|}{AUROC} & \multicolumn{4}{c|}{AUPRC} \\
\textit{Method} $\backslash$ \textit{normal-} & \textit{1} & \textit{2} & \textit{3} & \textit{4} & \textit{1} & \textit{2} & \textit{3} & \textit{4} \\
\hline
\hline
No  filtering       					& 0.777 & 0.767 & 0.772 & 0.774 & 0.070 & 0.064 & 0.068 & 0.072\\
$ h \circ g \circ f_1$                  & 0.923 & 0.925 & 0.923 & 0.922 & 0.311 & 0.315 & 0.313 & 0.304\\
$ w \circ h \circ g \circ f_1$          & 0.931 & 0.929 & 0.928 & 0.926 & 0.326 & 0.323 & 0.319 & 0.303\\
+ PCA         							& 0.932 & 0.930 & 0.928 & 0.926 & 0.355 & 0.353 & 0.350 & 0.333\\
Averaging           					& 0.937 & 0.935 & 0.935 & 0.931 & 0.391 &  0.390 &  0.385 & 0.375\\
Full method           					& \textbf{0.943} & \textbf{0.942} & \textbf{0.942} & \textbf{0.939} & \textbf{0.403} & \textbf{0.404} & \textbf{0.398} & \textbf{0.388}\\
\hline
PC & 0.886 & 0.884 & 0.891 &  0.877 & 0.153 & 0.145 & 0.170 & 0.132\\
GTE & 0.890 & 0.893 & 0.894 & 0.873 & 0.171 & 0.174 & 0.197 & 0.142\\
GENIE3 & 0.892 & 0.891 & 0.887 & 0.887 & 0.232 & 0.221 & 0.237 & 0.215 \\
\hline
\end{tabular}
\caption{Top: Performance on \textit{normal-1,2,3,4} with partial correlation and different filtering functions.
	Bottom: Performance on \textit{normal-1,2,3,4} with different methods.}
\label{tab:comparison}
\end{table}

\paragraph{Comparison with other methods.}

At the bottom of Table \ref{tab:comparison}, we provide as a comparison the
performance of three other methods: standard (Pearson) correlation (PC),
generalised transfer entropy (GTE), and GENIE3. ROC and PR curves on the
\textit{normal-2} network are shown for all methods in Figure~\ref{fig:curves}. Pearson correlation measures the unconditional linear
(in)dependence between variables and it should thus not be able to filter out
indirect interactions between neurons. GTE \citep{stetter2012model} was
proposed as a baseline for the challenge. This method builds on Transfer
Entropy to measure the association between two neurons. Unlike our approach, it
can predict the direction of the edges. GENIE3 \citep{huynh2010inferring} is
a gene regulatory network inference method that was the best performer in the
DREAM5 challenge \citep{marbach2012wisdom} (more details are given in Section \ref{sec:genie3}). When transposed to neural networks, this
method uses the importance score of variable $X_i$ in a Random Forest model trying to
predict $X_j$ from all variables in $X\setminus X_j$ as a confidence score for the edge going from neuron $i$ to neuron
$j$. However, to reduce the
computational cost of this method, we had to limit each tree in the
Random Forest model to a maximum depth of 3. This constraint has a potentially
severe effect on the performance of this method with respect to the use of
fully-grown trees. PC and GENIE3 were applied to the time-series filtered using the functions $w\circ h\circ g$ and $h\circ g\circ f_1$ (which
gave the best performance), respectively. For GENIE3, we built 10,000 trees per neuron and we
used default settings for all other parameters (except for the maximal tree
depth). For GTE, we reproduced the exact same setting (conditioning level and
pre-processing) that was used by the organisers of the challenge.

Partial correlation and averaged partial correlation clearly outperform all
other methods on all datasets (see Table \ref{tab:comparison} and Appendix \ref{app:supp}). The
improvement is more important in terms of AUPRC than in terms of AUROC. As
expected, Pearson correlation performs very poorly in terms of AUPRC. GTE and
GENIE3 work much better, but these two methods are nevertheless clearly below
partial correlation. Among these two methods, GTE is slightly better in terms
of AUROC, while GENIE3 is significantly better in terms of AUPRC. Given that we
had to limit this latter method for computational reasons, these results are
very promising and a comparison with the full GENIE3 approach is certainly part
of our future works.

The fact that our method is unable to predict edge directions does not seem to
be a disadvantage with respect to GTE and GENIE3. Although partial correlation
scores each edge, and its opposite, similarly, it can reach precision values
higher than 0.5 (see Figure \ref{fig:curves}(b)), suggesting that it mainly ranks high
pairs of neurons that interact in both directions.  It is interesting also to
note that, on \textit{normal-2}, a method that perfectly predicts the
undirected network (i.e., that gives a score of $1$ to each pair $(i,j)$ such that
$(i,j)\in E$ or $(j,i)\in E$, and $0$ otherwise) already reaches an AUROC as high
as $0.995$ and an AUPRC of $0.789$.
\begin{figure}[t]
\centering
\subfloat[ROC curves]{\includegraphics[width=0.49\textwidth]{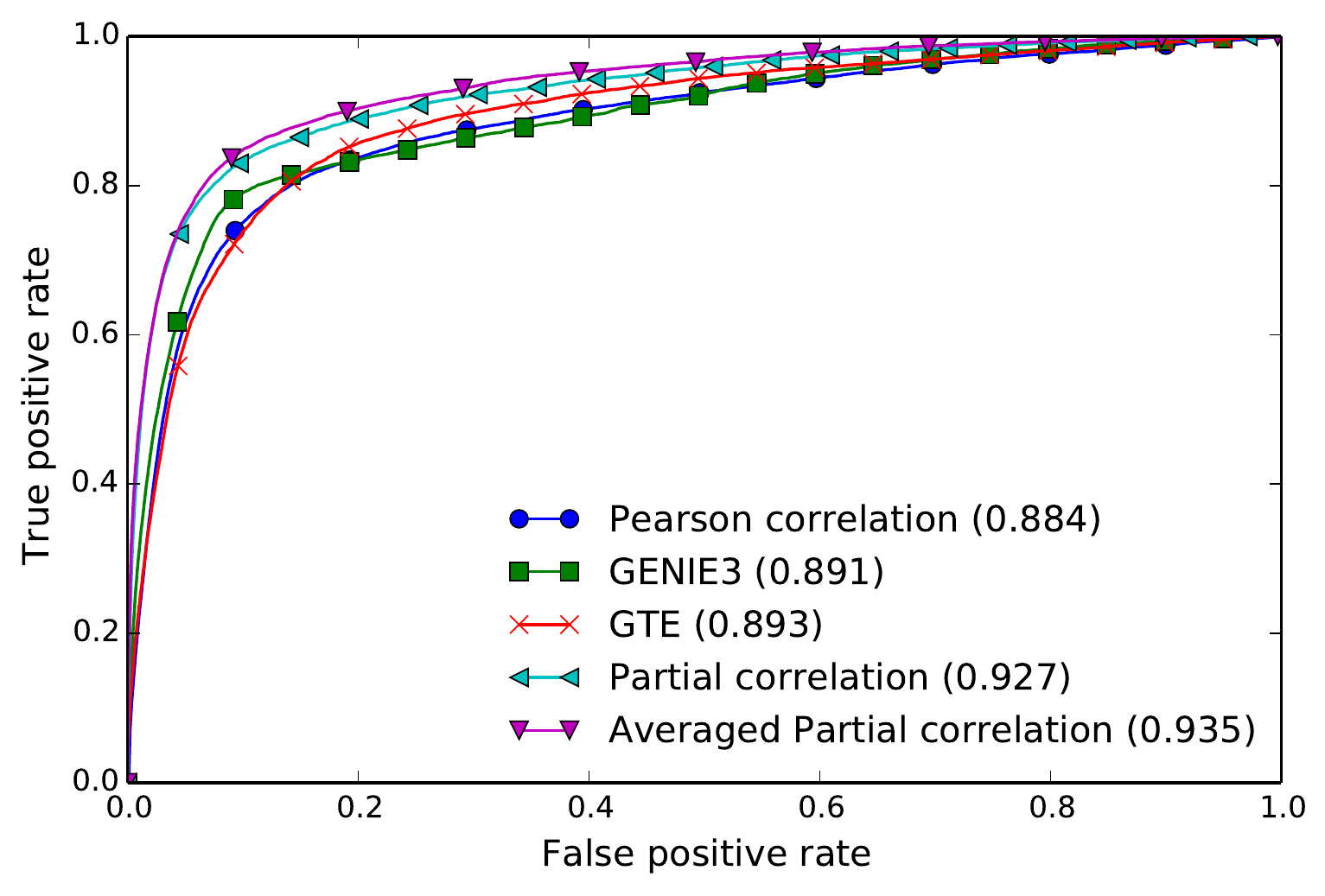} \label{fig:roc_curve}}
\subfloat[Precision-recall curves]{\includegraphics[width=0.49\textwidth]{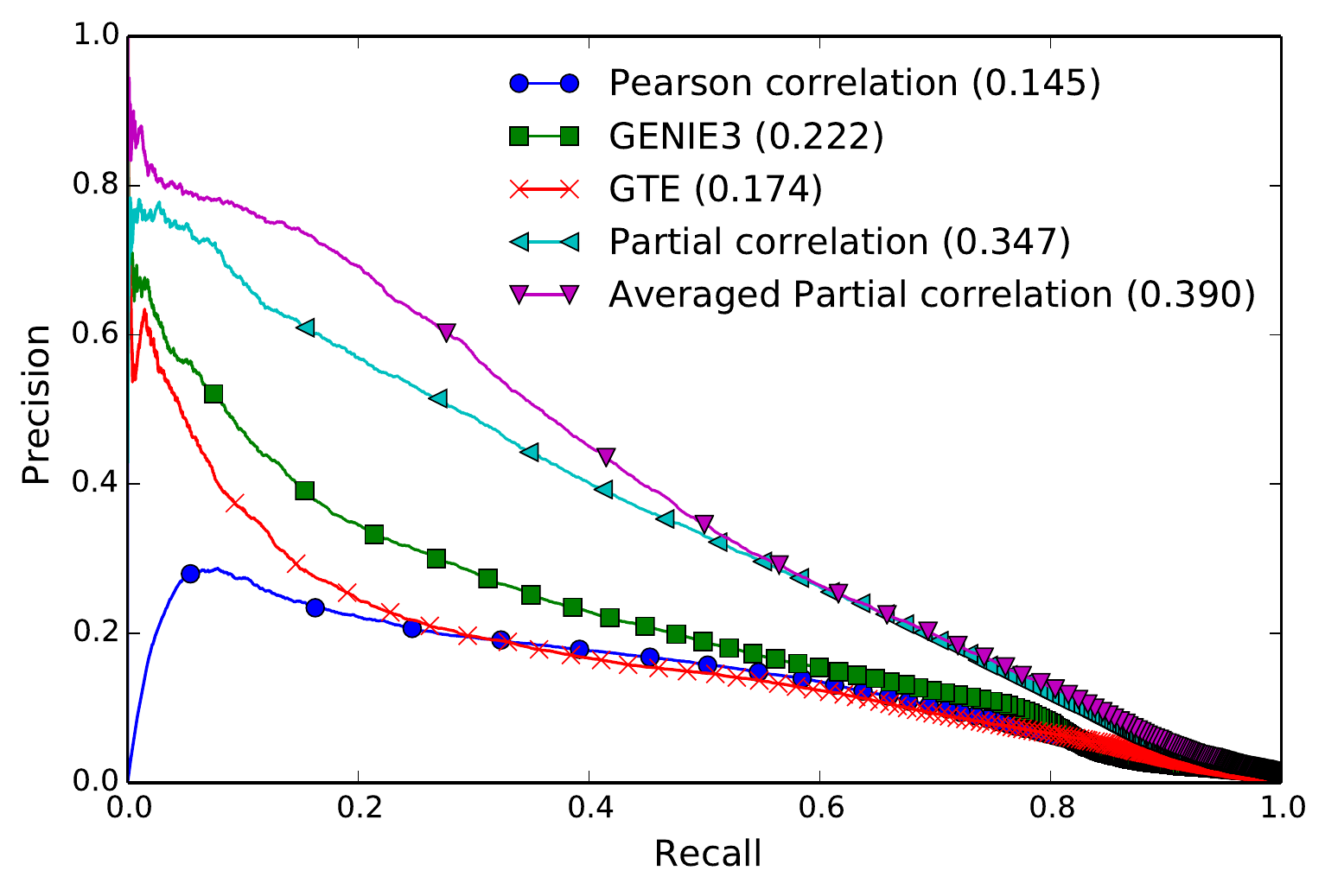} \label{fig:pr_curve}}
\caption{ROC (left) and PR (right) curves on \textit{normal-2} for the compared methods. Areas under the curves are reported in the legend.}
\label{fig:curves}
\end{figure}

\subsection{Conclusion for connectome inference} \label{sec:conclusion}

In this section, we outlined a simple but efficient methodology for the problem
of connectome inference from calcium imaging data. Our approach consists of two
steps: (i) processing fluorescence data to detect neural peak activities; and
(ii) inferring the degree of association between neurons from partial
correlation statistics. Its simplified variant outperforms other
network inference methods while its optimized version proved to be the best method
on the Connectomics Challenge. Given its simplicity and good performance, we
therefore believe that the methodology presented in this work
would constitute a solid and easily-reproducible baseline for further work in
the field of connectome inference.

\begin{subappendices}
	\section{Description  of the ``Full method''}
	\label{app:optimized}
	
	This section provides a detailed description of the method specifically tuned
	for the Connectomics Challenge. We restrict our description to the
	differences with respect to the simplified method presented in the main
	paper. Most parameters were tuned so as to maximize AUROC on the
	\textit{normal-1} dataset and our design choices were validated by monitoring
	the AUROC obtained by the 145 entries we submitted during the
	challenge. Although the tuned method performs better than the simplified one on
	the challenge dataset, we believe that the tuned method clearly overfits the
	simulator used to generate the challenge data and that the simplified method
	should work equally well on new independent datasets. We nevertheless provide
	the tuned method here for reference purposes. Our implementation of the tuned
	method is available at \url{https://github.com/asutera/kaggle-connectomics}.
	
	This appendix is structured as follows: Section~\ref{sapp:signal} describes
	the differences in terms of signal processing. Section~\ref{sapp:averaging}
	then provides a detailed presentation of the averaging approach.
	Section~\ref{sapp:connectome} presents an approach to correct the $p_{i,j}$
	values so as to take into account the edge directionality. Finally,
	Section~\ref{sapp:results} presents some experimental results to validate the
	different steps of our proposal.
	
	\subsection{Signal processing}
	\label{sapp:signal}
	
	In Section~\ref{sec:filter}, we introduced four filtering functions ($f$, $g$,
	$h$, and $w$) that are composed in sequence (i.e., $w \circ h \circ g \circ
	f$) to provide the signals from which to compute partial correlation
	statistics. Filtering is modified as follows in the tuned method:
	
	\begin{itemize}
		\item In addition to $f_1$ and $f_2$ (Equations \ref{eq:symetric-median} and
		\ref{eq:weighted-asymetric-median}), two alternative low-pass filters $f_3$
		and $f_4$ are considered:
		\begin{align}
		f_3(x^t_i) &= x^{t-1}_i + x^{t}_i + x^{t+1}_i + x^{t+2}_i, \label{eq:asymetric-median-forward} \\
		f_4(x^t_i) &=  x_i^t + x^{t+1}_i  + x^{t+2}_i + x^{t+3}_i. \label{eq:asymetric-median}
		\end{align}
		\item An additional filter $r$ is applied to smoothe differences in peak magnitudes
		that might remain after the application of the hard-threshold filter $h$:
		\begin{align}
		r(x^t_i) = (x_i^t)^c,
		\end{align}
		with $c=0.9$.
		\item Filter $w$ is replaced by a more complex filter $w^*$ defined as:
		\begin{align}
		w^*(x^{t}_i) &= {(x^{t}_i + 1 )^{\left (1 + \frac{1}{\sum_{j} x^{t}_j}\right )}}^{k(\sum_{j} x^{t}_j)}
		\end{align}
		where the function $k$ is a piecewise linear function optimised separately for
		each filter $f_1$, $f_2$, $f_3$ and $f_4$ (see the implementation for full
		details). Filter $w$ in the simplified method is a special case of $w^*$ with
		$k(\sum_j x_j^t)=1$.
	\end{itemize}
	The pre-processed time-series are then obtained by the application of the
	following function: $w^*\circ r \circ h \circ g \circ f_i$ (with $i=1$, 2, 3, or 4).
	
	\subsection{Weighted average of partial correlation statistics}
	\label{sapp:averaging}
	
	As discussed in Section \ref{sec:inference}, the performance of the method (in
	terms of AUROC) is sensitive to the value of the parameter $\tau$ of the
	hard-threshold filter $h$ (see Equation \ref{eqn:hfilter}), and to the choice
	of the low-pass filter (among $\{f_1, f_2, f_3, f_4\}$).
	As in the simplified method, we have averaged the partial correlation statistics obtained for all the pairs $(\tau,\mbox{low-pass filter}) \in \{0.100,0.101,\ldots,0.210\}\times \{f_1, f_2, f_3, f_4\}$.
	
	Filters $f_1$ and $f_2$ display similar performances and thus were given similar
	weights (i.e., resp. $0.383$ and $0.345$). These weights were chosen equal to the weights selected for the simplified method. In contrast, filters $f_3$
	and $f_4$ turn out, individually, to be less competitive and were therefore given
	less importance in the weighted average (i.e., resp. $0.004$ and $0.268$). Yet, as further shown in
	Section~\ref{sapp:results}, combining all $4$ filters proves to marginally
	improve performance with respect to using only $f_1$ and $f_2$.
	
	\subsection{Prediction of edge orientation}
	\label{sapp:connectome}
	
	Partial  correlation  statistics is  a  symmetric  measure, while  the
	connectome is a directed graph. It  could thus be beneficial to try to
	predict edge orientation. In this section, we present an heuristic that
	modifies the  $p_{ij}$ computed  by the  approach described  before which
	takes into account directionality.
	
	This approach is based on the following
	observation. The rise of fluorescence of a neuron indicates its
	activation. If another neuron is activated after a slight delay, this
	could be a consequence of the activation of the first neuron and
	therefore indicates a directed link in the connectome from the first to
	the second neuron.  Given this observation, we have computed the following term for every
	pair $(i,j)$:
	\begin{align}
	s_{i,j} = \sum_{t=1}^{T - 1} \mathbb{1}((x_j^{t+1} - x_i^t) \in \left[\phi_1, \phi_2\right])
	\end{align}
	that could be interpreted as an image of the  number of times
	that neuron $i$ activates neuron $j$. $\phi_1$ and $\phi_2$ are
	parameters whose values have been chosen in our experiments equal to
	$0.2$ and $0.5$, respectively. Their role is to
	define when the difference between $x_j^{t+1}$  and $x_i^t$ can
	indeed be assimilated to an event for which neuron $i$ activates neuron
	$j$.
	
	Afterwards, we have computed the difference between $s_{i,j}$ and
	$s_{j,i}$, that we call $z_{i,j}$, and used this difference to modify  $p_{i,j}$ and
	$p_{j,i}$ so as to take into account directionality. Naturally, if
	$z_{i,j}$ is greater  (smaller) than $0$, we may conclude that should there  be an
	edge between $i$ and $j$, then this edge would have to be oriented
	from $i$ to $j$ ($j$ to $i$).
	
	This suggests the new association matrix $r$:
	\begin{align}
	r_{i,j} =  \mathbb{1}(z_{i,j} > \phi_3)  *  p_{i,j}
	\end{align}
	where $\phi_3 >0$ is another parameter. We discovered that this new
	matrix $r$ was not providing good results, probably due to the fact that
	directivity was not rewarded well enough in the challenge.
	
	This has lead us to investigate other ways for exploiting the
	information about directionality contained in the matrix $z$. One of
	those ways that gave good performance was to use as an association
	matrix:
	\begin{align}
	q_{i,j} = weight * p_{i,j} + (1-weight) * z_{i,j}
	\label{eqn:qij}
	\end{align}
	with  $weight$ chosen close to 1 ($weight=0.997$). Note that with
	values for $weight$ close to 1,   matrix $q$ only uses the
	information to a minimum about directivity contained in $z$ to modify the  partial
	correlation matrix $p$. We tried smaller values for $weight$ but those
	provided poorer results.
	
	It was  this association matrix $q_{i,j}$ that actually led to the
	best results of the challenge, as shown in Table \ref{tab:directivity}
	of Section~\ref{sapp:results}.

	\subsection{Experiments}
	\label{sapp:results}
	
	\paragraph{On the interest of low-pass filters $f_3$ and $f_4$.}
	
	As reported in Table~\ref{tab:f3f4}, averaging over all low-pass filters leads
	to better AUROC scores than averaging over only two low-pass filters, i.e., $f_1$ and
	$f_2$. However this slightly reduces AUPRC.
	
	\begin{table}[ht]
		\caption{Performance on \textit{normal-1, 2, 3, or 4} with partial correlation with different averaging approaches.}
		\label{tab:f3f4}
		\centering
		\small
		\begin{tabular}{| l | c c c c | c c c c |}
			\hline
			& \multicolumn{4}{c|}{AUROC} & \multicolumn{4}{c|}{AUPRC} \\
			\textit{Averaging} $\backslash$ \textit{normal-} & \textit{1} & \textit{2} & \textit{3} & \textit{4} & \textit{1} & \textit{2} & \textit{3} & \textit{4} \\
			\hline
			\hline
			with $f_1$, $f_2$ & 0.937 & 0.935 & 0.935 & 0.931 & 0.391 &  \textbf{0.390} &  0.385 & \textbf{0.375}  \\
			with $f_1$, $f_2$, $f_3$, $f_4$ & \textbf{0.938} & \textbf{0.936} & \textbf{0.936} & \textbf{0.932} & 0.391 & 0.389 & 0.385 & 0.374\\
			\hline
		\end{tabular}
	\end{table}
	
	\paragraph{On the interest of using matrix $q$ rather than $p$ to take into account directivity.}
	
	Table~\ref{tab:directivity} compares AUROC and AUPRC with or without correcting the $p_{i,j}$ values according to Equation \ref{eqn:qij}. Both AUROC and AUPRC are (very slightly) improved by using information about directivity.
	
	\begin{table}[ht]
		\caption{Performance on \textit{normal-1,2,3,4} of ``Full Method'' with and
			without using information about directivity.}
		\label{tab:directivity}
		
		\centering
		\small
		\begin{tabular}{| l | c c c c | c c c c |}
			\hline
			& \multicolumn{4}{c|}{AUROC} & \multicolumn{4}{c|}{AUPRC} \\
			\textit{Full method} $\backslash$ \textit{normal-} & \textit{1} & \textit{2} & \textit{3} & \textit{4} & \textit{1} & \textit{2} & \textit{3} & \textit{4} \\
			\hline
			\hline
			Undirected & 0.943 & 0.942 & 0.942 & 0.939 & 0.403 & 0.404 & 0.398 & 0.388  \\
			Directed & \textbf{0.944} & \textbf{0.943} & 0.942 & \textbf{0.940} & \textbf{0.404} & \textbf{0.405} & \textbf{0.399} & \textbf{0.389}\\
			\hline
		\end{tabular}
	\end{table}

	\section{Supplementary results} \label{app:supp}
	
	In this appendix we report the performance of the different methods compared
	in the paper on 6 additional datasets provided by the Challenge
	organisers. These datasets, corresponding each to networks of 1,000 neurons, are similar to
	the \textit{normal} datasets except for one feature:
	\begin{description}
		\item[lowcon:] Similar network but on average with a lower number of connections per neuron.
		\item[highcon:] Similar network but on average with a higher number of connections per neuron.
		\item[lowcc:] Similar network but on average with a lower clustering coefficient.
		\item[highcc:] Similar network but on average with a higher clustering coefficient.
		\item[normal-3-highrate:] Same topology as \textit{normal-3} but with a higher firing frequency, i.e., with highly active neurons.
		\item[normal-4-lownoise:] Same topology as \textit{normal-4} but with a better signal-to-noise ratio.
	\end{description}
	
	The results of several methods applied to these 6 datasets are provided in
	Table~\ref{tab:results_appendix}. They confirm what we observed on the
	\textit{normal} datasets. Average partial correlation and its tuned variant,
	i.e.,``Full method'', clearly outperform other network inference methods on all
	datasets. PC is close to GENIE3 and GTE, but still slightly worse. GENIE3
	performs better than GTE most of the time. Note that the "Full method" reported in this table does not use Equation \ref{eqn:qij} to slightly correct the values of $p_{i,j}$  to take into account directivity.

	\begin{table}[h]
		\caption{Performance (top: AUROC, bottom: AUPRC) on specific datasets with different methods.}
		\label{tab:results_appendix}
		\centering
		\small
		\begin{tabular}{| l | c c c c c c |}
			\hline
			& \multicolumn{6}{c|}{AUROC}\\
			\textit{Method} $\backslash$ \textit{normal-} & \textit{lowcon} & \textit{highcon} & \textit{lowcc} & \textit{highcc} & \textit{3-highrate} & \textit{4-lownoise} \\
			\hline
			\hline
			Averaging     & 0.947 & 0.943 & 0.920 & 0.942 & 0.959 & 0.934 \\
			Full method   & \textbf{0.955} & \textbf{0.944} &  \textbf{0.925} & \textbf{0.946} & \textbf{0.961} & \textbf{0.941} \\
			PC & 0.782 & 0.920 &  0.846 & 0.897  & 0.898  & 0.873 \\
			GTE & 0.846 & 0.905 & 0.848 & 0.899 & 0.905 & 0.879\\
			GENIE3 & 0.781 &  0.924 & 0.879 & 0.902 & 0.886 &  0.890 \\ \hline \hline
			& \multicolumn{6}{c|}{AUPRC}\\ \hline
			Averaging     & 0.320 & 0.429 & 0.262 & 0.478 & 0.443 & 0.412 \\
			Full method   & \textbf{0.334} & \textbf{0.413} &  \textbf{0.260} & \textbf{0.486} & \textbf{0.452} & \textbf{0.432}\\
			PC & 0.074 & 0.218 & 0.082 & 0.165  & 0.193 & 0.135 \\
			GTE & 0.094 & 0.211 & 0.081 & 0.165 & 0.210 & 0.144\\
			GENIE3 & 0.128 & 0.273 & 0.116 & 0.309 & 0.256 & 0.224\\ \hline
		\end{tabular}
	\end{table}
	
	\section{On the selection of the number of principal components}
	\label{app:pca}
	
	The (true) network, seen as a matrix, can be decomposed through a singular value decomposition (SVD) or principal component
	analysis (PCA), so as to respectively determine a set of independent linear combinations of the
	variable \citep{alter2000singular}, or a reduced set of linear
	combinations combine, which then maximize the explained variance of the data
	\citep{jolliffe2005principal}. Since SVD and PCA are related, they can be defined by the same goal: both aim at finding a reduced set of neurons, known as components, whose activity can explain the rest of the network.
	
	The distribution of compoment eigen values obtained from PCA and SVD decompositions can be studied by sorting them in descending order of magnitude, as illustrated in Figure~\ref{fig:pca}. It can be seen that some component eigen values are zero, implying that the behaviour of the network could be explained by a subset of neurons because of the 
	redundancy and relations between the neurons. For all datasets, the eigen value distribution is exactly the same. 
	
	In the context of the challenge, we observe that only $800$ components seem to be necessary and we exploit this when computing partial correlation statistics. Therefore, the value of the parameter $M$ is immediate and should be clearly set to $800$ ($=0.8p$).
	
	Note that if the true network is not available, similar decomposition analysis could be carried on the inferred network, or on the data directly.

	\begin{figure}[t]
		\centering
		\subfloat[PCA]{\includegraphics[width=0.75\textwidth]{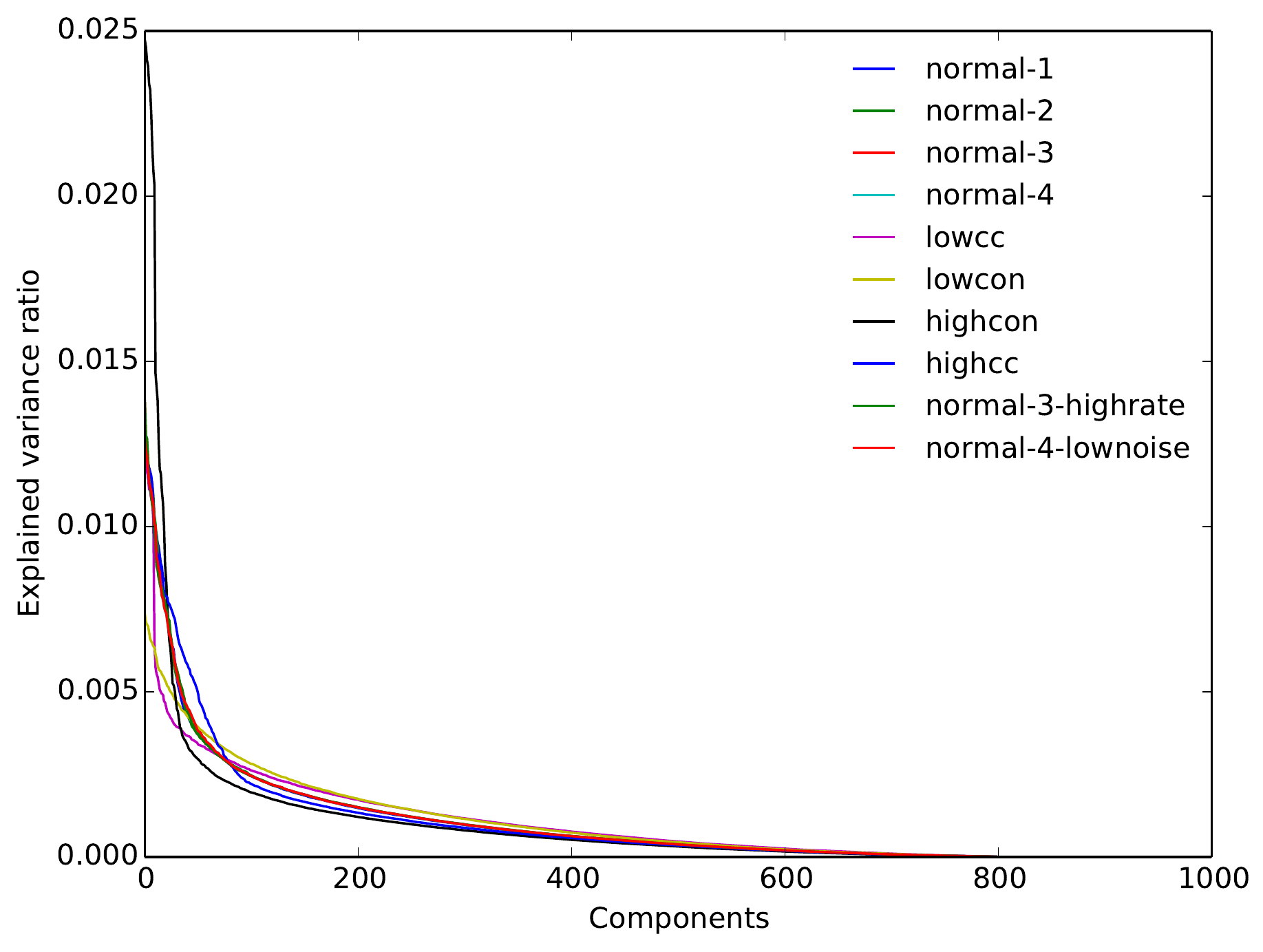} \label{fig:PCA}}\\
		\subfloat[SVD]{\includegraphics[width=0.75\textwidth]{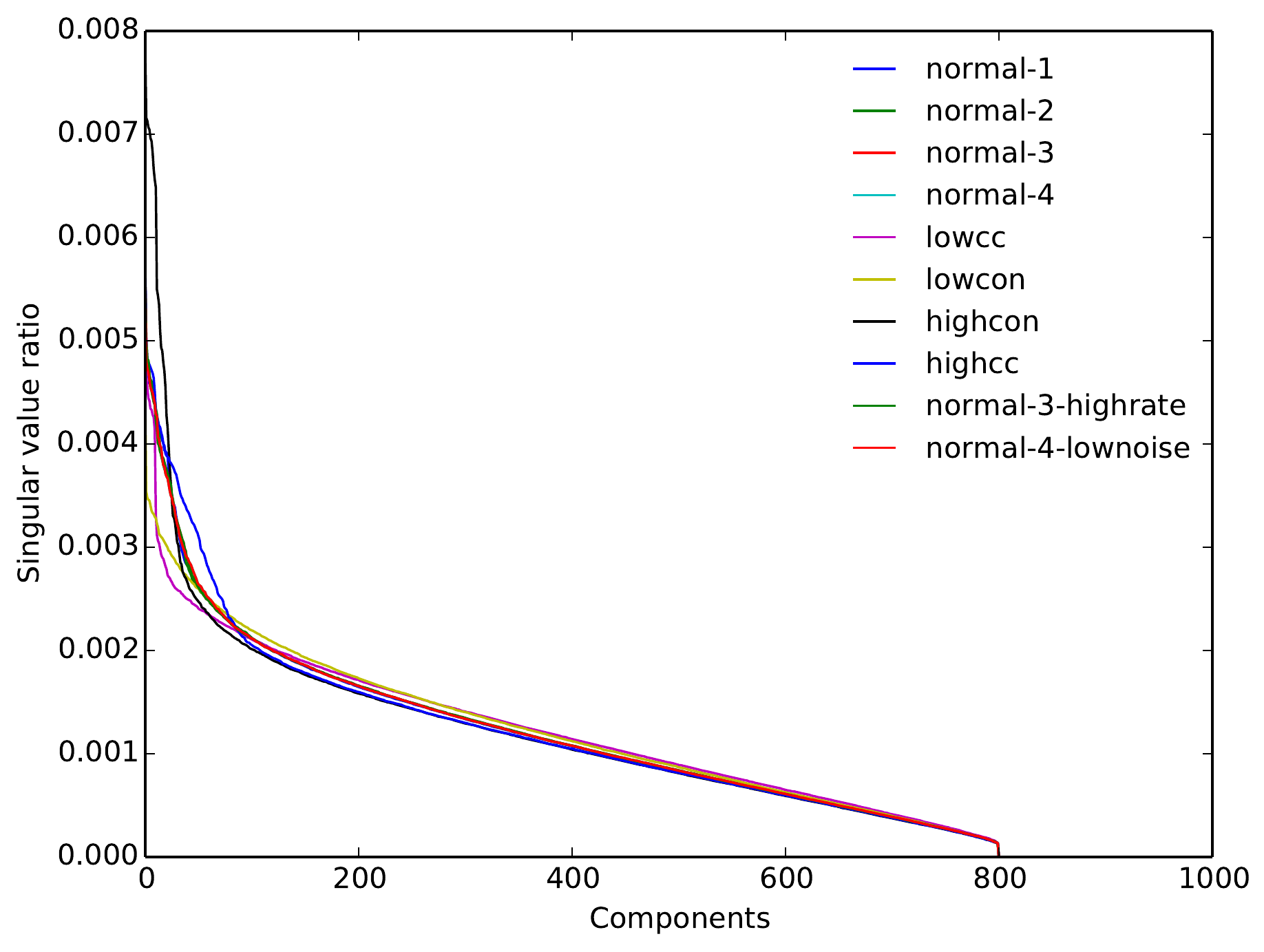} \label{fig:SVD}}
		\caption{Explained variance ratio by number of principal components (left) and singular value ratio by number of principal components (right) for all networks.}
		\label{fig:pca}
	\end{figure}
\end{subappendices}


\makeatletter
\def\toclevel@chapter{-1}
\makeatother
\chapter{Conclusion}

\begin{overview}
	The objective of this thesis was to better understand and characterise the properties of tree-based feature importance measures. Indeed, despite numerous works from either empirical or theoretical points of view, these importance measures are not yet fully understood. We are convinced that a more in depth understanding of the various properties of those measures would help to foster the scientific community to more systematically exploit these measures within the context of a wide variety of problems and methods. 
Within this context, we have mostly focused our study on the so-called \emph{Mean Decrease of Impurity} type of importance measure. In this chapter we summarise our findings and discuss directions for further research.
\end{overview}

\section{Main findings}

In the first part of this thesis we gave the background for the subsequent chapters. In particular, we introduced various notions of feature relevance and redundancy between features and described various feature selection problems in Chapter \ref{ch:background}, and we presented all relevant notions and algorithms pertaining to tree-based methods in Chapter \ref{ch:trees}.

\paragraph{}

Our first step towards a better understanding of tree-based feature importance measures consisted of a survey of the literature about  this topic, provided in Chapter \ref{ch:importances}. We proposed a framework of the MDA approach that is not tree-specific. In asymptotic conditions, i.e. infinite sample size $N$ and infinite ensemble size $N_{T}$, we gathered analytical formulations of both MDA and MDI importance measures and highlighted their main properties, in particular in the presence of correlated or redundant features. From a more practical point of view, we discussed their main biases. Despite many desirable properties, it emerged from empirical analyses that tree-based parameters, and feature characteristics and dependencies, may strongly impact the measured importance scores. In particular, the split randomisation parameter $K$ (i.e., the number of features considered at each node as split variable candidates) introduces the so-called masking effect, that prevents some relevant features to appear as important to the eyes of a random forest model. We also noticed a preference for smaller groups of correlated features, worth to take into account in the context of high dimensional applications where features often come in groups of correlated features of variable sizes. All those observations should help to analyse more cautiously importance scores. 

Another downside of tree-based importance measures is that they do not provide an explicit way to distinguish important features from non-important ones, for example by providing meaningful thresholds on feature importances. However, many approaches have been proposed to circumvent this issue. In particular, we investigated permutation schemes and pointed out that the conditional permutation scheme proposed in \cite{strobl2008conditional} focuses only on strongly relevant features, those conveying unique information about the output variable (in contrast with weakly relevant features).

\paragraph{}

In Chapter~\ref{ch:mdi} we focused on the MDI importance measure, and extended its characterisation from totally randomised trees ($K=1$) to more realistic random forest algorithms in asymptotic conditions. While all relevant features receive positive MDI values when using totally randomised trees, non-totally randomised trees ($K>1$) guarantee non zero importance scores only for strongly relevant features. Depending on the value of $K$, more or fewer weakly relevant features might be missed. In case of non-totally developed trees, we related these properties to the maximal tree depth, the feature degree of interaction, and the number of relevant features. 

However, in all these theoretical analyses, trees with multiway splits were considered, while in practice binary trees are generally preferred. We therefore transposed results obtained for multiway trees to binary ones.
Relaxing the asymptotic conditions, we discussed the implications of a finite setting. Limiting the size of the forest increases the number of ``missed'' features. In addition to masked features, some features are never evaluated, or not often enough to estimate their true importance. Importances derived from a finite sample  suffer from a positive bias that makes all features, including the irrelevant ones, have strictly positive importances. 

\paragraph{}

In many problems, feature selection is usually more complicated than identifying a single subset of input features that would together explain the output. Therefore we proposed in  Chapter \ref{ch:context} a methodological contribution that takes into account the context (i.e., the circumstances that form the setting for the experiment) in feature importance evaluations. The characterisation considers both contextual and non-contextual relevance of features and is based on several importance scores derived from tree-based methods. This approach was also illustrated on two artificial and two real biomedical problems.

\paragraph{}

When facing high-dimensional datasets, most approaches suffer from the curse of dimensionality. Chapter~\ref{ch:SRS} proposed an improved tree-based method that handles large datasets while being computationally tractable and able to identify relevant features efficiently. Marginally relevant features can be easily identified, even by univariate approaches, however some features are only relevant in the context of others, and their identification requires sophisticated methods that handle feature dependencies. The key idea of our method is that all features that make the others appear as relevant are necessarily relevant too. We used this simple result to propose a sequential approach that keeps in memory some already identified relevant features to speed up the identification of others. We observed that this approach is particularly interesting in case of highly dependent and structured features. 

\paragraph{}

The last chapter of this thesis is devoted to a specific machine learning task consisting in reconstructing a network from data. The first part of Chapter \ref{ch:connectomics} recalled the principle of GENIE3, a tree-based network inference technique designed a few years ago in our research group. Then, we proposed a method to infer the connectome from calcium imaging data using partial correlation statistics and we put it in perspective with GENIE3-like techniques.

\section{Limitations and future work}

We believe that this thesis provides additional steps towards a better understanding of tree-based feature importance measures.
 However, there still remain several limitations to the frameworks proposed along this thesis that are all potential directions of improvements. 

\section*{Extending our characterisation of the MDI importance to continuous features}
All theoretical derivations from Chapters \ref{ch:importances}, \ref{ch:mdi} and \ref{ch:context} concern categorical input variables which are the keystones of our characterisation of the measure. It would be interesting to adapt our framework to continuous input variables, and also, probably with more difficulty, to continuous context variables. 

\section*{Feature importance estimation in non asymptotic conditions}
Another key assumption of our characterisations was asymptotic conditions.  In practice sample sizes are finite, as are the number of trees in an ensemble. We believe that a very significant step towards a full understanding of importance measures would be the derivation of statistical distributions of importance scores depending on feature characteristics as well as sample and ensemble size. 

\section*{MDA vs. MDI}
The study conducted in Chapter \ref{ch:importances} initiated a comparison between the two tree-based importance measures. Both methods can be used for classification and regression problems and yield similar results while being intrinsically different on several aspects. In what follows, we give an outline of some elements of comparison that may be subject to future studies.

MDA exploits out-of-bag (OOB) samples (i.e., not used to learn the model) to compute an error-rate evaluating the impact of the removal (by permutation) of a feature. MDI assesses the importance of a feature based on its average contribution in the impurity reduction in the tree-ensemble learning. Unlike MDA, MDI therefore uses the same samples for learning the model and evaluating the importance of features. Future research could examine if importance scores evaluated in this way and those computed using MDI on  independent samples (e.g., OOB samples or a holdout test set) are similar. 

\paragraph{}

Another point of comparison is that MDA depends explicitly on the loss function used, whereas MDI depends explicitly on the impurity measure used. However, MDA also depends indirectly on the tree structure and hence on the impurity measure used to grow the tree. Permuting a feature that is not used in the tree model obviously does not impact the OOB error-rate of this tree. This suggests that both importance measures will identify approximately the same set of important features. Some differences can however be pointed out. Let us consider a two-class classification problem. A feature that slightly changes the output value distributions in the tree leaves would be seen as important by the MDI importance measure. If this change is too subtle to change the predicted class of the tree, the MDA importance using a loss function that is not sensitive enough (e.g., the zero-one loss function) would miss such a feature.  In presence of two variables in a XOR configuration with respect to the output, MDI is only able to identify one feature (the second one) per tree, whereas MDA can detect the importance of both features in a single tree. Indeed, permuting the values of the first variable used induces that samples reaching nodes using the second variable are mixed up. This necessarily impacts the error-rate of the tree and thus makes the first variable appear as important in the eyes of MDA. 

\paragraph{}

Future studies could investigate if it is possible to move both importance closer to each other by considering some specific loss function (e.g., that would have the same properties as the impurity measure used).

\section*{Empirical evaluation and improvements of the Sequential Random Subspace method}
Despite a strong theoretical motivation, some more work is clearly needed to evaluate the Sequential Random Subspace method empirically, on controlled and real high-dimensional problems. In this context, it would probably be necessary to overcome one of the main drawbacks of the sequential random subspace method with respect to the random subspace method which is that it can not be parallelised. One possible approach is to grow ensemble of trees at each iteration instead of single trees. Such a variant is clearly an improvement for our proposed method. We finally believe that another possible improvement to our algorithm is the statistical test based on the introduction of a random probe used to decide which feature to include in the relevant set. \\

\section{Open research questions}

Alongside future work resulting directly from this thesis, we propose in this section some new (open-)research questions that go beyond the scope of this thesis but that should be investigated by further studies to complete the understanding of tree-based feature importance measures.

\section*{Feature importance characterisation for tree boosting}
Boosting approaches were not discussed in this thesis. They however constitute powerful and well performing ensemble algorithms. Concretely, in tree-based boosting ensemble methods, trees are not built independently but sequentially in order to correct predictions of previous trees. Each tree is therefore weighted according to its contribution to the total model performance. Given the state-of-the-art performance of these methods, it would be interesting to compute feature importances from these ensembles of trees by extending our formulation to take into account tree weights, and examine to what extent asymptotic guarantees are still valid.

\section*{Improving interpretability of other state-of-the-art machine learning models}

This thesis was devoted to tree-based methods only. However, all machine learning algorithms could benefit from more interpretability of their induced models, particularly deep learning methods. Taking inspiration from feature importance derived from tree-based methods, it would interesting to evaluate the MDA approach (that is not tree-specific) on other machine learning algorithms and compare it to other importance measures (e.g., individual feature importance measures that assess the importance of features for a single prediction, and which were not discussed in this thesis).

\section*{Causality}
The main advantage of our partial correlation approach is to filter out indirect links that the tree-based inference method GENIE3 is unable to do. Future work might investigate the relationship between direct links and strong relevance, and evaluate to what extent it may be possible to reduce the number of indirect links that are actually kept in the final reconstructed network. Causality in tree-based methods  has been considered in only few works (see, e.g., \citep{li2017causal}) and still remains an open question to date.

\chapter*{References}

\begingroup
    \def\chapter*#1{}
    \bibliographystyle{abbrvnat}
    \renewcommand{\bibname}{}
    \label{app:bibliography}
    \bibliography{bibliography}

\begin{thebibliography}{283}
\providecommand{\natexlab}[1]{#1}
\providecommand{\url}[1]{\texttt{#1}}
\expandafter\ifx\csname urlstyle\endcsname\relax
  \providecommand{\doi}[1]{doi: #1}\else
  \providecommand{\doi}{doi: \begingroup \urlstyle{rm}\Url}\fi

\bibitem[Abeel et~al.(2009)Abeel, Helleputte, Van~de Peer, Dupont, and
  Saeys]{abeel2009robust}
T.~Abeel, T.~Helleputte, Y.~Van~de Peer, P.~Dupont, and Y.~Saeys.
\newblock Robust biomarker identification for cancer diagnosis with ensemble
  feature selection methods.
\newblock \emph{Bioinformatics}, 26\penalty0 (3):\penalty0 392--398, 2009.

\bibitem[Aliferis et~al.(2003)Aliferis, Tsamardinos, and
  Statnikov]{aliferis2003hiton}
C.~F. Aliferis, I.~Tsamardinos, and A.~Statnikov.
\newblock Hiton: a novel markov blanket algorithm for optimal variable
  selection.
\newblock In \emph{AMIA Annual Symposium Proceedings}, volume 2003, page~21.
  American Medical Informatics Association, 2003.

\bibitem[Aliferis et~al.(2010)Aliferis, Statnikov, Tsamardinos, Mani, and
  Koutsoukos]{aliferis2010local}
C.~F. Aliferis, A.~Statnikov, I.~Tsamardinos, S.~Mani, and X.~D. Koutsoukos.
\newblock Local causal and markov blanket induction for causal discovery and
  feature selection for classification part i: Algorithms and empirical
  evaluation.
\newblock \emph{Journal of Machine Learning Research}, 11\penalty0
  (Jan):\penalty0 171--234, 2010.

\bibitem[Almuallim and Dietterich(1991{\natexlab{a}})]{almuallim1991efficient}
H.~Almuallim and T.~G. Dietterich.
\newblock Efficient algorithms for identifying relevant features.
\newblock In \emph{Proc. of the 9th Canadian Conference on Artificial
  Intelligence}, pages 38--45. Citeseer, 1991{\natexlab{a}}.

\bibitem[Almuallim and Dietterich(1991{\natexlab{b}})]{almuallim1991learning}
H.~Almuallim and T.~G. Dietterich.
\newblock Learning with many irrelevant features.
\newblock In \emph{AAAI}, volume~91, pages 547--552. Citeseer,
  1991{\natexlab{b}}.

\bibitem[Almuallim and Dietterich(1994)]{almuallim1994learning}
H.~Almuallim and T.~G. Dietterich.
\newblock Learning boolean concepts in the presence of many irrelevant
  features.
\newblock \emph{Artificial Intelligence}, 69\penalty0 (1-2):\penalty0 279--305,
  1994.

\bibitem[Altay et~al.(2011)Altay, Asim, Markowetz, and
  Neal]{altay2011differential}
G.~Altay, M.~Asim, F.~Markowetz, and D.~E. Neal.
\newblock Differential c3net reveals disease networks of direct physical
  interactions.
\newblock \emph{BMC bioinformatics}, 12\penalty0 (1):\penalty0 296, 2011.

\bibitem[Alter et~al.(2000)Alter, Brown, and Botstein]{alter2000singular}
O.~Alter, P.~O. Brown, and D.~Botstein.
\newblock Singular value decomposition for genome-wide expression data
  processing and modeling.
\newblock \emph{Proceedings of the National Academy of Sciences}, 97\penalty0
  (18):\penalty0 10101--10106, 2000.

\bibitem[Altmann et~al.(2010)Altmann, Tolo{\c{s}}i, Sander, and
  Lengauer]{altmann2010permutation}
A.~Altmann, L.~Tolo{\c{s}}i, O.~Sander, and T.~Lengauer.
\newblock Permutation importance: a corrected feature importance measure.
\newblock \emph{Bioinformatics}, 26\penalty0 (10):\penalty0 1340--1347, 2010.

\bibitem[Amaratunga et~al.(2008)Amaratunga, Cabrera, and
  Lee]{amaratunga2008enriched}
D.~Amaratunga, J.~Cabrera, and Y.-S. Lee.
\newblock Enriched random forests.
\newblock \emph{Bioinformatics}, 24\penalty0 (18):\penalty0 2010--2014, 2008.

\bibitem[Ambroise and McLachlan(2002)]{ambroise2002selection}
C.~Ambroise and G.~J. McLachlan.
\newblock Selection bias in gene extraction on the basis of microarray
  gene-expression data.
\newblock \emph{Proceedings of the national academy of sciences}, 99\penalty0
  (10):\penalty0 6562--6566, 2002.

\bibitem[Amit and Geman(1997)]{amit1997shape}
Y.~Amit and D.~Geman.
\newblock Shape quantization and recognition with randomized trees.
\newblock \emph{Neural computation}, 9\penalty0 (7):\penalty0 1545--1588, 1997.

\bibitem[Ananth and Schisterman(2017)]{ananth2017confounding}
C.~V. Ananth and E.~F. Schisterman.
\newblock Confounding, causality, and confusion: the role of intermediate
  variables in interpreting observational studies in obstetrics.
\newblock \emph{American journal of obstetrics and gynecology}, 217\penalty0
  (2):\penalty0 167--175, 2017.

\bibitem[Archer and Kimes(2008)]{archer2008empirical}
K.~Archer and R.~Kimes.
\newblock Empirical characterization of random forest variable importance
  measures.
\newblock \emph{Computational Statistics \& Data Analysis}, 52\penalty0
  (4):\penalty0 2249--2260, 2008.

\bibitem[Arnold et~al.(2007)Arnold, Nallapati, and
  Cohen]{arnold2007comparative}
A.~Arnold, R.~Nallapati, and W.~W. Cohen.
\newblock A comparative study of methods for transductive transfer learning.
\newblock In \emph{Data Mining Workshops, 2007. ICDM Workshops 2007. Seventh
  IEEE International Conference on}, pages 77--82. IEEE, 2007.

\bibitem[Auret and Aldrich(2011)]{auret2011empirical}
L.~Auret and C.~Aldrich.
\newblock Empirical comparison of tree ensemble variable importance measures.
\newblock \emph{Chemometrics and Intelligent Laboratory Systems}, 105\penalty0
  (2):\penalty0 157--170, 2011.

\bibitem[Battaglia et~al.(2017)Battaglia, Guyon, Lemaire, Orlandi, Ray, and
  Soriano]{battaglia17neuralbook}
D.~Battaglia, I.~Guyon, V.~Lemaire, J.~Orlandi, B.~Ray, and J.~Soriano,
  editors.
\newblock \emph{Neural Connectomics Challenge}.
\newblock Springer, 2017.

\bibitem[Beirlant et~al.(1997)Beirlant, Dudewicz, Gy{\"o}rfi, and Van~der
  Meulen]{beirlant1997nonparametric}
J.~Beirlant, E.~J. Dudewicz, L.~Gy{\"o}rfi, and E.~C. Van~der Meulen.
\newblock Nonparametric entropy estimation: An overview.
\newblock \emph{International Journal of Mathematical and Statistical
  Sciences}, 6\penalty0 (1):\penalty0 17--39, 1997.

\bibitem[Belgiu and Dr{\u{a}}gu{\c{t}}(2016)]{belgiu2016random}
M.~Belgiu and L.~Dr{\u{a}}gu{\c{t}}.
\newblock Random forest in remote sensing: A review of applications and future
  directions.
\newblock \emph{ISPRS Journal of Photogrammetry and Remote Sensing},
  114:\penalty0 24--31, 2016.

\bibitem[Bell and Wang(2000)]{bell2000formalism}
D.~A. Bell and H.~Wang.
\newblock A formalism for relevance and its application in feature subset
  selection.
\newblock \emph{Machine learning}, 41\penalty0 (2):\penalty0 175--195, 2000.

\bibitem[Biau(2012)]{biau2012analysis}
G.~Biau.
\newblock Analysis of a random forests model.
\newblock \emph{Journal of Machine Learning Research}, 13\penalty0
  (Apr):\penalty0 1063--1095, 2012.

\bibitem[Biau and Scornet(2016)]{biau2016random}
G.~Biau and E.~Scornet.
\newblock A random forest guided tour.
\newblock \emph{Test}, 25\penalty0 (2):\penalty0 197--227, 2016.

\bibitem[Biau et~al.(2008)Biau, Devroye, and Lugosi]{biau2008consistency}
G.~Biau, L.~Devroye, and G.~Lugosi.
\newblock Consistency of random forests and other averaging classifiers.
\newblock \emph{Journal of Machine Learning Research}, 9\penalty0
  (Sep):\penalty0 2015--2033, 2008.

\bibitem[Bishop(2006)]{bishop2006pattern}
C.~M. Bishop.
\newblock \emph{Pattern recognition and machine learning}, volume~1.
\newblock Springer New York, 2006.

\bibitem[Bloebaum et~al.(2018)Bloebaum, Janzing, Washio, Shimizu, and
  Schoelkopf]{bloebaum2018cause}
P.~Bloebaum, D.~Janzing, T.~Washio, S.~Shimizu, and B.~Schoelkopf.
\newblock Cause-effect inference by comparing regression errors.
\newblock In \emph{International Conference on Artificial Intelligence and
  Statistics}, pages 900--909, 2018.

\bibitem[Blum and Langley(1997)]{blum1997selection}
A.~L. Blum and P.~Langley.
\newblock Selection of relevant features and examples in machine learning.
\newblock \emph{Artificial intelligence}, 97\penalty0 (1-2):\penalty0 245--271,
  1997.

\bibitem[Bol{\'o}n-Canedo et~al.(2015)Bol{\'o}n-Canedo, S{\'a}nchez-Maro{\~n}o,
  and Alonso-Betanzos]{bolon2015recent}
V.~Bol{\'o}n-Canedo, N.~S{\'a}nchez-Maro{\~n}o, and A.~Alonso-Betanzos.
\newblock Recent advances and emerging challenges of feature selection in the
  context of big data.
\newblock \emph{Knowledge-Based Systems}, 86:\penalty0 33--45, 2015.

\bibitem[Botta(2013)]{botta2013walk}
V.~Botta.
\newblock \emph{A walk into random forests: adaptation and application to
  Genome-Wide Association Studies}.
\newblock PhD thesis, Universit{\'e} de Li{\`e}ge, Li{\`e}ge, Belgique, 2013.

\bibitem[Botta et~al.(2014)Botta, Louppe, Geurts, and
  Wehenkel]{botta2014exploiting}
V.~Botta, G.~Louppe, P.~Geurts, and L.~Wehenkel.
\newblock Exploiting snp correlations within random forest for genome-wide
  association studies.
\newblock \emph{PloS one}, 9\penalty0 (4):\penalty0 e93379, 2014.

\bibitem[Boulesteix and Slawski(2009)]{boulesteix2009stability}
A.-L. Boulesteix and M.~Slawski.
\newblock Stability and aggregation of ranked gene lists.
\newblock \emph{Briefings in bioinformatics}, 10\penalty0 (5):\penalty0
  556--568, 2009.

\bibitem[Boulesteix et~al.(2011)Boulesteix, Bender, Lorenzo~Bermejo, and
  Strobl]{boulesteix2011random}
A.-L. Boulesteix, A.~Bender, J.~Lorenzo~Bermejo, and C.~Strobl.
\newblock Random forest gini importance favours snps with large minor allele
  frequency: impact, sources and recommendations.
\newblock \emph{Briefings in Bioinformatics}, 13\penalty0 (3):\penalty0
  292--304, 2011.

\bibitem[Boulesteix et~al.(2012)Boulesteix, Janitza, Kruppa, and
  K{\"o}nig]{boulesteix2012overview}
A.-L. Boulesteix, S.~Janitza, J.~Kruppa, and I.~R. K{\"o}nig.
\newblock Overview of random forest methodology and practical guidance with
  emphasis on computational biology and bioinformatics.
\newblock \emph{Wiley Interdisciplinary Reviews: Data Mining and Knowledge
  Discovery}, 2\penalty0 (6):\penalty0 493--507, 2012.

\bibitem[Bousquet(2002)]{bousquet2002transductive}
O.~Bousquet.
\newblock Transductive learning: Motivation, models, algorithms.
\newblock \emph{University of New Mexico, Albuquerque, USA}, 2002.

\bibitem[Boutilier et~al.(1996)Boutilier, Friedman, Goldszmidt, and
  Koller]{boutilier1996}
C.~Boutilier, N.~Friedman, M.~Goldszmidt, and D.~Koller.
\newblock Context-specific independence in bayesian networks.
\newblock In \emph{Proceedings of the Twelfth International Conference on
  Uncertainty in Artificial Intelligence}, UAI'96, pages 115--123, San
  Francisco, CA, USA, 1996. Morgan Kaufmann Publishers Inc.
\newblock ISBN 1-55860-412-X.
\newblock URL \url{http://dl.acm.org/citation.cfm?id=2074284.2074298}.

\bibitem[Braga-Neto and Dougherty(2004)]{braga2004cross}
U.~M. Braga-Neto and E.~R. Dougherty.
\newblock Is cross-validation valid for small-sample microarray classification?
\newblock \emph{Bioinformatics}, 20\penalty0 (3):\penalty0 374--380, 2004.

\bibitem[Breiman(1996{\natexlab{a}})]{breiman1996bagging}
L.~Breiman.
\newblock Bagging predictors.
\newblock \emph{Machine learning}, 24\penalty0 (2):\penalty0 123--140,
  1996{\natexlab{a}}.

\bibitem[Breiman(1996{\natexlab{b}})]{breiman1996heuristics}
L.~Breiman.
\newblock Heuristics of instability and stabilization in model selection.
\newblock \emph{The Annals of Statistics}, pages 2350--2383,
  1996{\natexlab{b}}.

\bibitem[Breiman(1996{\natexlab{c}})]{breiman1996out}
L.~Breiman.
\newblock Out-of-bag estimation, 1996{\natexlab{c}}.

\bibitem[Breiman(2000)]{breiman2000some}
L.~Breiman.
\newblock Some infinity theory for predictor ensembles.
\newblock Technical report, Technical Report 579, Statistics Dept. UCB, 2000.

\bibitem[Breiman(2001)]{breiman2001random}
L.~Breiman.
\newblock Random forests.
\newblock \emph{Machine learning}, 45\penalty0 (1):\penalty0 5--32, 2001.

\bibitem[Breiman(2002)]{breiman2002manual}
L.~Breiman.
\newblock Manual on setting up, using, and understanding random forests v3. 1.
\newblock \emph{Statistics Department University of California Berkeley, CA,
  USA}, 1, 2002.

\bibitem[Breiman(2004)]{breiman2004consistency}
L.~Breiman.
\newblock Consistency for a simple model of random forests.
\newblock Technical report, Berkeley, 2004.

\bibitem[Breiman and Cutler(2003)]{breiman2003random}
L.~Breiman and A.~Cutler.
\newblock Random forests manual v4.
\newblock In \emph{Technical report}. UC Berkel, 2003.

\bibitem[Breiman and Cutler(2008)]{breiman2008random}
L.~Breiman and A.~Cutler.
\newblock Random forests—classification manual.
\newblock \emph{URL http://www. math. usu. edu/\~{} adele/forests}, 2008.

\bibitem[Breiman et~al.(1984)Breiman, Friedman, Olshen, and
  Stone]{breiman1984classification}
L.~Breiman, J.~Friedman, R.~Olshen, and C.~Stone.
\newblock \emph{{Classification and Regression Trees}}.
\newblock Wadsworth and Brooks, Monterey, CA, 1984.

\bibitem[Brown(2009)]{brown2009new}
G.~Brown.
\newblock A new perspective for information theoretic feature selection.
\newblock In \emph{International conference on artificial intelligence and
  statistics}, pages 49--56, 2009.

\bibitem[Brown et~al.(2012)Brown, Pocock, Zhao, and
  Luj{\'a}n]{brown2012conditional}
G.~Brown, A.~Pocock, M.-J. Zhao, and M.~Luj{\'a}n.
\newblock Conditional likelihood maximisation: a unifying framework for
  information theoretic feature selection.
\newblock \emph{The Journal of Machine Learning Research}, 13\penalty0
  (1):\penalty0 27--66, 2012.

\bibitem[Bureau et~al.(2005)Bureau, Dupuis, Falls, Lunetta, Hayward, Keith, and
  Van~Eerdewegh]{bureau2005identifying}
A.~Bureau, J.~Dupuis, K.~Falls, K.~L. Lunetta, B.~Hayward, T.~P. Keith, and
  P.~Van~Eerdewegh.
\newblock Identifying snps predictive of phenotype using random forests.
\newblock \emph{Genetic Epidemiology: The Official Publication of the
  International Genetic Epidemiology Society}, 28\penalty0 (2):\penalty0
  171--182, 2005.

\bibitem[Cardie(1993)]{cardie1993using}
C.~Cardie.
\newblock Using decision trees to improve case-based learning.
\newblock In \emph{Proceedings of the tenth international conference on machine
  learning}, pages 25--32, 1993.

\bibitem[Carlson(2008)]{carlson2008snps}
B.~Carlson.
\newblock Snps-a shortcut to personalized medicine.
\newblock \emph{Genetic Engineering \& Biotechnology News}, 28\penalty0
  (12):\penalty0 12--12, 2008.

\bibitem[Chandrashekar and Sahin(2014)]{chandrashekar2014survey}
G.~Chandrashekar and F.~Sahin.
\newblock A survey on feature selection methods.
\newblock \emph{Computers \& Electrical Engineering}, 40\penalty0 (1):\penalty0
  16--28, 2014.

\bibitem[Chawla et~al.(2004)Chawla, Hall, Bowyer, and Kegelmeyer]{chawla2004}
N.~V. Chawla, L.~O. Hall, K.~W. Bowyer, and W.~P. Kegelmeyer.
\newblock Learning ensembles from bites: A scalable and accurate approach.
\newblock \emph{J. Mach. Learn. Res.}, 5:\penalty0 421--451, Dec. 2004.
\newblock ISSN 1532-4435.

\bibitem[Cover and Thomas(2012)]{cover2012elements}
T.~M. Cover and J.~A. Thomas.
\newblock \emph{Elements of information theory}.
\newblock John Wiley \& Sons, 2012.

\bibitem[Cover and Van~Campenhout(1977)]{cover1977possible}
T.~M. Cover and J.~M. Van~Campenhout.
\newblock On the possible orderings in the measurement selection problem.
\newblock \emph{IEEE Trans. Systems, Man, and Cybernetics}, 7\penalty0
  (9):\penalty0 657--661, 1977.

\bibitem[Cutler and Zhao(2001)]{cutler2001pert}
A.~Cutler and G.~Zhao.
\newblock Pert-perfect random tree ensembles.
\newblock \emph{Computing Science and Statistics}, 33:\penalty0 490--497, 2001.

\bibitem[Cutler et~al.(2007)Cutler, Edwards, Beard, Cutler, Hess, Gibson, and
  Lawler]{cutler2007random}
D.~R. Cutler, T.~C. Edwards, K.~H. Beard, A.~Cutler, K.~T. Hess, J.~Gibson, and
  J.~J. Lawler.
\newblock Random forests for classification in ecology.
\newblock \emph{Ecology}, 88\penalty0 (11):\penalty0 2783--2792, 2007.

\bibitem[de~Abril et~al.(2018)de~Abril, Yoshimoto, and
  Doya]{de2018connectivity}
I.~M. de~Abril, J.~Yoshimoto, and K.~Doya.
\newblock Connectivity inference from neural recording data: Challenges,
  mathematical bases and research directions.
\newblock \emph{Neural Networks}, 2018.

\bibitem[De~La~Fuente et~al.(2004)De~La~Fuente, Bing, Hoeschele, and
  Mendes]{de2004discovery}
A.~De~La~Fuente, N.~Bing, I.~Hoeschele, and P.~Mendes.
\newblock Discovery of meaningful associations in genomic data using partial
  correlation coefficients.
\newblock \emph{Bioinformatics}, 20\penalty0 (18):\penalty0 3565--3574, 2004.

\bibitem[De~Smet and Marchal(2010)]{de2010advantages}
R.~De~Smet and K.~Marchal.
\newblock Advantages and limitations of current network inference methods.
\newblock \emph{Nature Reviews Microbiology}, 8\penalty0 (10):\penalty0 717,
  2010.

\bibitem[Del~Campo et~al.(2012)Del~Campo, Mollenhauer, Bertolotto, Engelborghs,
  Hampel, Simonsen, Kapaki, Kruse, Le~Bastard, Lehmann,
  et~al.]{del2012recommendations}
M.~Del~Campo, B.~Mollenhauer, A.~Bertolotto, S.~Engelborghs, H.~Hampel, A.~H.
  Simonsen, E.~Kapaki, N.~Kruse, N.~Le~Bastard, S.~Lehmann, et~al.
\newblock Recommendations to standardize preanalytical confounding factors in
  alzheimer's and parkinson's disease cerebrospinal fluid biomarkers: an
  update.
\newblock \emph{Biomarkers in medicine}, 6\penalty0 (4):\penalty0 419--430,
  2012.

\bibitem[Deng and Runger(2012)]{deng2012feature}
H.~Deng and G.~Runger.
\newblock Feature selection via regularized trees.
\newblock In \emph{Neural Networks (IJCNN), The 2012 International Joint
  Conference on}, pages 1--8. IEEE, 2012.

\bibitem[Deng and Runger(2013)]{deng2013gene}
H.~Deng and G.~Runger.
\newblock Gene selection with guided regularized random forest.
\newblock \emph{Pattern Recognition}, 46\penalty0 (12):\penalty0 3483--3489,
  2013.

\bibitem[Deng et~al.(2013)Deng, Geng, and Luo]{deng2013identifiability}
W.~Deng, Z.~Geng, and P.~Luo.
\newblock Identifiability of intermediate variables on causal paths.
\newblock \emph{Frontiers of Mathematics in China}, 8\penalty0 (3):\penalty0
  517--539, 2013.

\bibitem[Denil et~al.(2014)Denil, Matheson, and De~Freitas]{denil2014narrowing}
M.~Denil, D.~Matheson, and N.~De~Freitas.
\newblock Narrowing the gap: Random forests in theory and in practice.
\newblock In \emph{International conference on machine learning}, pages
  665--673, 2014.

\bibitem[Devijver and Kittler(1982)]{devijver1982pattern}
P.~A. Devijver and J.~Kittler.
\newblock \emph{Pattern recognition: A statistical approach}.
\newblock Prentice hall, 1982.

\bibitem[D{\'\i}az-Uriarte and De~Andres(2006)]{diaz2006gene}
R.~D{\'\i}az-Uriarte and S.~A. De~Andres.
\newblock Gene selection and classification of microarray data using random
  forest.
\newblock \emph{BMC bioinformatics}, 7\penalty0 (1):\penalty0 3, 2006.

\bibitem[Diciotti et~al.(2013)Diciotti, Ciulli, Mascalchi, Giannelli, and
  Toschi]{diciotti2013peeking}
S.~Diciotti, S.~Ciulli, M.~Mascalchi, M.~Giannelli, and N.~Toschi.
\newblock The \guillemotleft peeking\guillemotright effect in supervised
  feature selection on diffusion tensor imaging data.
\newblock \emph{American Journal of Neuroradiology}, 34\penalty0 (9):\penalty0
  E107--E107, 2013.

\bibitem[Dietterich(2000)]{dietterich2000experimental}
T.~G. Dietterich.
\newblock An experimental comparison of three methods for constructing
  ensembles of decision trees: Bagging, boosting, and randomization.
\newblock \emph{Machine learning}, 40\penalty0 (2):\penalty0 139--157, 2000.

\bibitem[Dietterich and Kong(1995)]{dietterich1995machine}
T.~G. Dietterich and E.~B. Kong.
\newblock Machine learning bias, statistical bias, and statistical variance of
  decision tree algorithms.
\newblock Technical report, Technical report, Department of Computer Science,
  Oregon State University, 1995.

\bibitem[Dobra and Gehrke(2001)]{dobra2001bias}
A.~Dobra and J.~Gehrke.
\newblock Bias correction in classification tree construction.
\newblock In \emph{Proceedings of the Eighteenth International Conference on
  Machine Learning}, pages 90--97. Morgan Kaufmann Publishers Inc., 2001.

\bibitem[Domingos(1996)]{domingos1996exploiting}
P.~Domingos.
\newblock Exploiting context in feature selection.
\newblock In \emph{Workshop on Learning in Context-Sensitive Domains at the
  13th International Conference on Machine Learning (ICML96)}, pages 15--20.
  Bari, Italy, 1996.

\bibitem[Doshi-Velez and Kim(2017)]{doshi2017towards}
F.~Doshi-Velez and B.~Kim.
\newblock Towards a rigorous science of interpretable machine learning.
\newblock \emph{arXiv preprint arXiv:1702.08608}, 2017.

\bibitem[Drami{\'n}ski et~al.(2008)Drami{\'n}ski, Rada-Iglesias, Enroth,
  Wadelius, Koronacki, and Komorowski]{draminski2008monte}
M.~Drami{\'n}ski, A.~Rada-Iglesias, S.~Enroth, C.~Wadelius, J.~Koronacki, and
  J.~Komorowski.
\newblock Monte carlo feature selection for supervised classification.
\newblock \emph{Bioinformatics}, 24\penalty0 (1):\penalty0 110--117, 2008.

\bibitem[Drami{\'n}ski et~al.(2016)Drami{\'n}ski, Dabrowski, Diamanti,
  Koronacki, and Komorowski]{draminski2016discovering}
M.~Drami{\'n}ski, M.~J. Dabrowski, K.~Diamanti, J.~Koronacki, and
  J.~Komorowski.
\newblock Discovering networks of interdependent features in high-dimensional
  problems.
\newblock In \emph{Big Data Analysis: New Algorithms for a New Society}, pages
  285--304. Springer, 2016.

\bibitem[Efron and Tibshirani(1994)]{efron1994introduction}
B.~Efron and R.~J. Tibshirani.
\newblock \emph{An introduction to the bootstrap}.
\newblock CRC press, 1994.

\bibitem[Ernst et~al.(2005)Ernst, Geurts, and Wehenkel]{ernst2005tree}
D.~Ernst, P.~Geurts, and L.~Wehenkel.
\newblock Tree-based batch mode reinforcement learning.
\newblock \emph{Journal of Machine Learning Research}, 6\penalty0
  (Apr):\penalty0 503--556, 2005.

\bibitem[Ewers and Didham(2006)]{ewers2006confounding}
R.~M. Ewers and R.~K. Didham.
\newblock Confounding factors in the detection of species responses to habitat
  fragmentation.
\newblock \emph{Biological reviews}, 81\penalty0 (1):\penalty0 117--142, 2006.

\bibitem[Fr{\'e}nay et~al.(2013)Fr{\'e}nay, Doquire, and
  Verleysen]{frenay2013mutual}
B.~Fr{\'e}nay, G.~Doquire, and M.~Verleysen.
\newblock Is mutual information adequate for feature selection in regression?
\newblock \emph{Neural Networks}, 48:\penalty0 1--7, 2013.

\bibitem[Friedman(2001)]{friedman2001greedy}
J.~H. Friedman.
\newblock Greedy function approximation: a gradient boosting machine.
\newblock \emph{Annals of statistics}, pages 1189--1232, 2001.

\bibitem[Gama(2004)]{gama2004functional}
J.~Gama.
\newblock Functional trees.
\newblock \emph{Machine Learning}, 55\penalty0 (3):\penalty0 219--250, 2004.

\bibitem[Ganz et~al.(2015)Ganz, Greve, Fischl, Konukoglu, Initiative,
  et~al.]{ganz2015relevant}
M.~Ganz, D.~N. Greve, B.~Fischl, E.~Konukoglu, A.~D.~N. Initiative, et~al.
\newblock Relevant feature set estimation with a knock-out strategy and random
  forests.
\newblock \emph{NeuroImage}, 122:\penalty0 131--148, 2015.

\bibitem[Geissler et~al.(2000)Geissler, H{\"o}lzl, Marohl, Kuhn-R{\'e}gnier,
  Mehlhorn, S{\"u}dkamp, and de~Vivie]{geissler2000risk}
H.~J. Geissler, P.~H{\"o}lzl, S.~Marohl, F.~Kuhn-R{\'e}gnier, U.~Mehlhorn,
  M.~S{\"u}dkamp, and E.~R. de~Vivie.
\newblock Risk stratification in heart surgery: comparison of six score
  systems.
\newblock \emph{European Journal of Cardio-thoracic surgery}, 17\penalty0
  (4):\penalty0 400--406, 2000.

\bibitem[Gennari et~al.(1989)Gennari, Langley, and Fisher]{gennari1989models}
J.~H. Gennari, P.~Langley, and D.~Fisher.
\newblock Models of incremental concept formation.
\newblock \emph{Artificial intelligence}, 40\penalty0 (1-3):\penalty0 11--61,
  1989.

\bibitem[Genuer et~al.(2010)Genuer, Poggi, and
  Tuleau-Malot]{genuer2010variable}
R.~Genuer, J.-M. Poggi, and C.~Tuleau-Malot.
\newblock Variable selection using random forests.
\newblock \emph{Pattern Recognition Letters}, 31\penalty0 (14):\penalty0
  2225--2236, 2010.

\bibitem[Geurts(2002)]{geurts2002contributions}
P.~Geurts.
\newblock \emph{Contributions to decision tree induction: bias/variance
  tradeoff and time series classification}.
\newblock PhD thesis, University of Li{\`e}ge Belgium, 2002.

\bibitem[Geurts and Saeys(2011)]{geurts2011exploring}
P.~Geurts and Y.~Saeys.
\newblock Exploring signature multiplicity in microarray data using ensembles
  of randomized trees.
\newblock In \emph{5th International workshop on Machine Learning in Systems
  Biology (MLSB'11)}, pages 24--28. Technical University M{\"u}nchen, 2011.

\bibitem[Geurts et~al.(2006)Geurts, Ernst, and Wehenkel]{geurts2006extremely}
P.~Geurts, D.~Ernst, and L.~Wehenkel.
\newblock Extremely randomized trees.
\newblock \emph{Machine learning}, 63\penalty0 (1):\penalty0 3--42, 2006.

\bibitem[Geurts et~al.(2009)Geurts, Irrthum, and
  Wehenkel]{geurts2009supervised}
P.~Geurts, A.~Irrthum, and L.~Wehenkel.
\newblock Supervised learning with decision tree-based methods in computational
  and systems biology.
\newblock \emph{Molecular Biosystems}, 5\penalty0 (12):\penalty0 1593--1605,
  2009.

\bibitem[Ghimire et~al.(2010)Ghimire, Rogan, and Miller]{ghimire2010contextual}
B.~Ghimire, J.~Rogan, and J.~Miller.
\newblock Contextual land-cover classification: incorporating spatial
  dependence in land-cover classification models using random forests and the
  getis statistic.
\newblock \emph{Remote Sensing Letters}, 1\penalty0 (1):\penalty0 45--54, 2010.

\bibitem[Gini(1912)]{gini1912variabilita}
C.~Gini.
\newblock Variabilit{\`a} e mutabilit{\`a}.
\newblock \emph{Reprinted in Memorie di metodologica statistica (Ed. Pizetti E,
  Salvemini, T). Rome: Libreria Eredi Virgilio Veschi}, 1912.

\bibitem[Goebel et~al.(2005)Goebel, Dawy, Hagenauer, and
  Mueller]{goebel2005approximation}
B.~Goebel, Z.~Dawy, J.~Hagenauer, and J.~C. Mueller.
\newblock An approximation to the distribution of finite sample size mutual
  information estimates.
\newblock In \emph{Communications, 2005. ICC 2005. 2005 IEEE International
  Conference on}, volume~2, pages 1102--1106. IEEE, 2005.

\bibitem[Golub et~al.(1999)Golub, Slonim, Tamayo, Huard, Gaasenbeek, Mesirov,
  Coller, Loh, Downing, Caligiuri, et~al.]{golub1999molecular}
T.~R. Golub, D.~K. Slonim, P.~Tamayo, C.~Huard, M.~Gaasenbeek, J.~P. Mesirov,
  H.~Coller, M.~L. Loh, J.~R. Downing, M.~A. Caligiuri, et~al.
\newblock Molecular classification of cancer: class discovery and class
  prediction by gene expression monitoring.
\newblock \emph{science}, 286\penalty0 (5439):\penalty0 531--537, 1999.

\bibitem[Gregorutti et~al.(2017)Gregorutti, Michel, and
  Saint-Pierre]{gregorutti2017correlation}
B.~Gregorutti, B.~Michel, and P.~Saint-Pierre.
\newblock Correlation and variable importance in random forests.
\newblock \emph{Statistics and Computing}, 27\penalty0 (3):\penalty0 659--678,
  2017.

\bibitem[Gr{\"o}mping(2009)]{gromping2009variable}
U.~Gr{\"o}mping.
\newblock Variable importance assessment in regression: linear regression
  versus random forest.
\newblock \emph{The American Statistician}, 63\penalty0 (4):\penalty0 308--319,
  2009.

\bibitem[Guyon and Elisseeff(2003)]{guyon2003introduction}
I.~Guyon and A.~Elisseeff.
\newblock An introduction to variable and feature selection.
\newblock \emph{Journal of machine learning research}, 3\penalty0
  (Mar):\penalty0 1157--1182, 2003.

\bibitem[Guyon and Elisseeff(2006)]{guyon2006introduction}
I.~Guyon and A.~Elisseeff.
\newblock An introduction to feature extraction.
\newblock In \emph{Feature extraction}, pages 1--25. Springer, 2006.

\bibitem[Hapfelmeier and Ulm(2013)]{hapfelmeier2013new}
A.~Hapfelmeier and K.~Ulm.
\newblock A new variable selection approach using random forests.
\newblock \emph{Computational Statistics \& Data Analysis}, 60:\penalty0
  50--69, 2013.

\bibitem[Hardin et~al.(2004)Hardin, Tsamardinos, and
  Aliferis]{hardin2004theoretical}
D.~Hardin, I.~Tsamardinos, and C.~F. Aliferis.
\newblock A theoretical characterization of linear svm-based feature selection.
\newblock In \emph{Proceedings of the twenty-first international conference on
  Machine learning}, page~48. ACM, 2004.

\bibitem[Hastie et~al.(2005)Hastie, Tibshirani, Friedman, and
  Franklin]{hastie2005elements}
T.~Hastie, R.~Tibshirani, J.~Friedman, and J.~Franklin.
\newblock The elements of statistical learning: data mining, inference and
  prediction.
\newblock \emph{The Mathematical Intelligencer}, 27\penalty0 (2):\penalty0
  83--85, 2005.

\bibitem[Hastie et~al.(2009)Hastie, Tibshirani, and
  Friedman]{friedman2009elements}
T.~Hastie, R.~Tibshirani, and J.~Friedman.
\newblock \emph{The elements of statistical learning: data mining, inference,
  and prediction, 2nd Edition}, volume~1 of \emph{Springer series in
  statistics}.
\newblock Springer, 2009.

\bibitem[He and Yu(2010)]{he2010stable}
Z.~He and W.~Yu.
\newblock Stable feature selection for biomarker discovery.
\newblock \emph{Computational biology and chemistry}, 34\penalty0 (4):\penalty0
  215--225, 2010.

\bibitem[Heath et~al.(1993)Heath, Kasif, and Salzberg]{heath1993induction}
D.~Heath, S.~Kasif, and S.~Salzberg.
\newblock Induction of oblique decision trees.
\newblock In \emph{IJCAI}, volume 1993, pages 1002--1007, 1993.

\bibitem[Hern{\'a}ndez-Lobato et~al.(2013)Hern{\'a}ndez-Lobato,
  Mart{\'\i}Nez-Mu{\~n}Oz, and Su{\'a}rez]{hernandez2013large}
D.~Hern{\'a}ndez-Lobato, G.~Mart{\'\i}Nez-Mu{\~n}Oz, and A.~Su{\'a}rez.
\newblock How large should ensembles of classifiers be?
\newblock \emph{Pattern Recognition}, 46\penalty0 (5):\penalty0 1323--1336,
  2013.

\bibitem[Ho(1998)]{ho1998random}
T.~K. Ho.
\newblock The random subspace method for constructing decision forests.
\newblock \emph{Pattern Analysis and Machine Intelligence, IEEE Transactions
  on}, 20\penalty0 (8):\penalty0 832--844, 1998.

\bibitem[Hua et~al.(2004)Hua, Xiong, Lowey, Suh, and Dougherty]{hua2004optimal}
J.~Hua, Z.~Xiong, J.~Lowey, E.~Suh, and E.~R. Dougherty.
\newblock Optimal number of features as a function of sample size for various
  classification rules.
\newblock \emph{Bioinformatics}, 21\penalty0 (8):\penalty0 1509--1515, 2004.

\bibitem[Hua et~al.(2009)Hua, Tembe, and Dougherty]{hua2009performance}
J.~Hua, W.~D. Tembe, and E.~R. Dougherty.
\newblock Performance of feature-selection methods in the classification of
  high-dimension data.
\newblock \emph{Pattern Recognition}, 42\penalty0 (3):\penalty0 409--424, 2009.

\bibitem[Huang et~al.(2005)Huang, Pan, Grindle, Han, Chen, Park, Miller, and
  Hall]{huang2005comparative}
X.~Huang, W.~Pan, S.~Grindle, X.~Han, Y.~Chen, S.~J. Park, L.~W. Miller, and
  J.~Hall.
\newblock A comparative study of discriminating human heart failure etiology
  using gene expression profiles.
\newblock \emph{BMC bioinformatics}, 6\penalty0 (1):\penalty0 205, 2005.

\bibitem[Huynh-Thu(2012)]{huynh2012machine}
V.~A. Huynh-Thu.
\newblock \emph{Machine learning-based feature ranking: statistical
  interpretation and gene network inference}.
\newblock PhD thesis, Universit{\'e} de Li{\`e}ge, 2012.

\bibitem[Huynh-Thu et~al.(2008)Huynh-Thu, Wehenkel, and
  Geurts]{huynh2008exploiting}
V.~A. Huynh-Thu, L.~Wehenkel, and P.~Geurts.
\newblock Exploiting tree-based variable importances to selectively identify
  relevant variables.
\newblock In \emph{JMLR: Workshop and Conference proceedings}, volume~4, pages
  60--73. Microtome Publishing, 2008.

\bibitem[Huynh-Thu et~al.(2010)Huynh-Thu, Irrthum, Wehenkel, and
  Geurts]{huynh2010inferring}
V.~A. Huynh-Thu, A.~Irrthum, L.~Wehenkel, and P.~Geurts.
\newblock Regulatory networks from expression data using tree-based methods.
\newblock \emph{PLoS ONE}, 5\penalty0 (9):\penalty0 e12776, 2010.

\bibitem[Huynh-Thu et~al.(2012)Huynh-Thu, Saeys, Wehenkel, and
  Geurts]{huynh2012statistical}
V.~A. Huynh-Thu, Y.~Saeys, L.~Wehenkel, and P.~Geurts.
\newblock Statistical interpretation of machine learning-based feature
  importance scores for biomarker discovery.
\newblock \emph{Bioinformatics}, 28\penalty0 (13):\penalty0 1766--1774, 2012.

\bibitem[Ideker and Krogan(2012)]{ideker2012differential}
T.~Ideker and N.~J. Krogan.
\newblock Differential network biology.
\newblock \emph{Molecular systems biology}, 8\penalty0 (1), 2012.

\bibitem[Ishwaran(2007)]{ishwaran2007variable}
H.~Ishwaran.
\newblock Variable importance in binary regression trees and forests.
\newblock \emph{Electronic Journal of Statistics}, 1:\penalty0 519--537, 2007.

\bibitem[Ishwaran and Lu(2018)]{ishwaran2018standard}
H.~Ishwaran and M.~Lu.
\newblock Standard errors and confidence intervals for variable importance in
  random forest regression, classification, and survival.
\newblock \emph{Statistics in medicine}, 2018.

\bibitem[Jain and Zongker(1997)]{jain1997feature}
A.~Jain and D.~Zongker.
\newblock Feature selection: Evaluation, application, and small sample
  performance.
\newblock \emph{IEEE transactions on pattern analysis and machine
  intelligence}, 19\penalty0 (2):\penalty0 153--158, 1997.

\bibitem[Jain et~al.(2000)Jain, Duin, and Mao]{jain2000statistical}
A.~K. Jain, R.~P. Duin, and J.~Mao.
\newblock Statistical pattern recognition: A review.
\newblock \emph{IEEE Transactions on pattern analysis and machine
  intelligence}, 22\penalty0 (1):\penalty0 4--37, 2000.

\bibitem[Jakulin(2005)]{jakulin2005machine}
A.~Jakulin.
\newblock \emph{Machine learning based on attribute interactions}.
\newblock PhD thesis, Univerza v Ljubljani, 2005.

\bibitem[Jakulin and Bratko(2003{\natexlab{a}})]{jakulin2003analyzing}
A.~Jakulin and I.~Bratko.
\newblock \emph{Analyzing attribute dependencies}.
\newblock Springer, 2003{\natexlab{a}}.

\bibitem[Jakulin and Bratko(2003{\natexlab{b}})]{jakulin2003quantifying}
A.~Jakulin and I.~Bratko.
\newblock Quantifying and visualizing attribute interactions.
\newblock \emph{arXiv preprint cs/0308002}, 2003{\natexlab{b}}.

\bibitem[Janecek et~al.(2008)Janecek, Gansterer, Demel, and
  Ecker]{janecek2008relationship}
A.~Janecek, W.~Gansterer, M.~Demel, and G.~Ecker.
\newblock On the relationship between feature selection and classification
  accuracy.
\newblock In \emph{New Challenges for Feature Selection in Data Mining and
  Knowledge Discovery}, pages 90--105, 2008.

\bibitem[Janikow(1998)]{janikow1998fuzzy}
C.~Z. Janikow.
\newblock Fuzzy decision trees: issues and methods.
\newblock \emph{IEEE Transactions on Systems, Man, and Cybernetics, Part B
  (Cybernetics)}, 28\penalty0 (1):\penalty0 1--14, 1998.

\bibitem[Janitza et~al.(2013)Janitza, Strobl, and Boulesteix]{janitza2013auc}
S.~Janitza, C.~Strobl, and A.-L. Boulesteix.
\newblock An auc-based permutation variable importance measure for random
  forests.
\newblock \emph{BMC bioinformatics}, 14\penalty0 (1):\penalty0 119, 2013.

\bibitem[Janitza et~al.(2015)Janitza, Celik, and
  Boulesteix]{janitza2015computationally}
S.~Janitza, E.~Celik, and A.-L. Boulesteix.
\newblock A computationally fast variable importance test for random forests
  for high-dimensional data.
\newblock \emph{Advances in Data Analysis and Classification}, pages 1--31,
  2015.

\bibitem[Jiang and Wang(2016)]{jiang2016efficient}
S.-y. Jiang and L.-x. Wang.
\newblock Efficient feature selection based on correlation measure between
  continuous and discrete features.
\newblock \emph{Information Processing Letters}, 116\penalty0 (2):\penalty0
  203--215, 2016.

\bibitem[Johnson et~al.(2007)Johnson, Li, and Rabinovic]{johnson2007adjusting}
W.~E. Johnson, C.~Li, and A.~Rabinovic.
\newblock Adjusting batch effects in microarray expression data using empirical
  bayes methods.
\newblock \emph{Biostatistics}, 8\penalty0 (1):\penalty0 118--127, 2007.

\bibitem[Jolliffe(2005)]{jolliffe2005principal}
I.~Jolliffe.
\newblock \emph{Principal component analysis}.
\newblock Wiley Online Library, 2005.

\bibitem[Jolliffe(2011)]{jolliffe2011principal}
I.~Jolliffe.
\newblock Principal component analysis.
\newblock In \emph{International encyclopedia of statistical science}, pages
  1094--1096. Springer, 2011.

\bibitem[Joly(2017)]{joly2017exploiting}
A.~Joly.
\newblock \emph{Exploiting random projections and sparsity with random forests
  and gradient boosting methods-Application to multi-label and multi-output
  learning, random forest model compression and leveraging input sparsity}.
\newblock PhD thesis, Universit{\'e} de Li{\`e}ge, Li{\`e}ge, Belgique, 2017.

\bibitem[Kaiser and Reed(1977)]{kaiser1977data}
J.~Kaiser and W.~Reed.
\newblock Data smoothing using low-pass digital filters.
\newblock \emph{Review of Scientific Instruments}, 48\penalty0 (11):\penalty0
  1447--1457, 1977.

\bibitem[Kalousis et~al.(2007)Kalousis, Prados, and
  Hilario]{kalousis2007stability}
A.~Kalousis, J.~Prados, and M.~Hilario.
\newblock Stability of feature selection algorithms: a study on
  high-dimensional spaces.
\newblock \emph{Knowledge and information systems}, 12\penalty0 (1):\penalty0
  95--116, 2007.

\bibitem[Kamangar(2012)]{kamangar2012confounding}
F.~Kamangar.
\newblock Confounding variables in epidemiologic studies: basics and beyond.
\newblock \emph{Arch Iran Med}, 15\penalty0 (8):\penalty0 508--16, 2012.

\bibitem[Kim and Loh(2001)]{kim2001classification}
H.~Kim and W.-Y. Loh.
\newblock Classification trees with unbiased multiway splits.
\newblock \emph{Journal of the American Statistical Association}, 96\penalty0
  (454):\penalty0 589--604, 2001.

\bibitem[Kira and Rendell(1992{\natexlab{a}})]{kira1992feature}
K.~Kira and L.~A. Rendell.
\newblock The feature selection problem: Traditional methods and a new
  algorithm.
\newblock In \emph{Aaai}, volume~2, pages 129--134, 1992{\natexlab{a}}.

\bibitem[Kira and Rendell(1992{\natexlab{b}})]{kira1992practical}
K.~Kira and L.~A. Rendell.
\newblock A practical approach to feature selection.
\newblock In \emph{Machine Learning Proceedings 1992}, pages 249--256.
  Elsevier, 1992{\natexlab{b}}.

\bibitem[Kittler(1978)]{kittler1978feature}
J.~Kittler.
\newblock Feature set search algorithms.
\newblock \emph{Pattern recognition and signal processing}, 1978.

\bibitem[Kohavi and John(1997)]{kohavi1997wrappers}
R.~Kohavi and G.~H. John.
\newblock Wrappers for feature subset selection.
\newblock \emph{Artificial intelligence}, 97\penalty0 (1-2):\penalty0 273--324,
  1997.

\bibitem[Koller and Sahami(1996)]{koller1996toward}
D.~Koller and M.~Sahami.
\newblock Toward optimal feature selection.
\newblock Technical report, Stanford InfoLab, 1996.

\bibitem[Konukoglu and Ganz(2014)]{konukoglu2014approximate}
E.~Konukoglu and M.~Ganz.
\newblock Approximate false positive rate control in selection frequency for
  random forest.
\newblock \emph{arXiv preprint arXiv:1410.2838}, 2014.

\bibitem[Kuncheva(2007)]{kuncheva2007stability}
L.~I. Kuncheva.
\newblock A stability index for feature selection.
\newblock In \emph{Artificial intelligence and applications}, pages 421--427,
  2007.

\bibitem[Kuncheva and Rodr{\'\i}guez(2018)]{kuncheva2018feature}
L.~I. Kuncheva and J.~J. Rodr{\'\i}guez.
\newblock On feature selection protocols for very low-sample-size data.
\newblock \emph{Pattern Recognition}, 81:\penalty0 660--673, 2018.

\bibitem[Kuncheva et~al.(2010)Kuncheva, Rodr{\'\i}guez, Plumpton, Linden, and
  Johnston]{kuncheva2010random}
L.~I. Kuncheva, J.~J. Rodr{\'\i}guez, C.~O. Plumpton, D.~E. Linden, and S.~J.
  Johnston.
\newblock Random subspace ensembles for fmri classification.
\newblock \emph{Medical Imaging, IEEE Transactions on}, 29\penalty0
  (2):\penalty0 531--542, 2010.

\bibitem[Kursa and Rudnicki(2011)]{kursa2011all}
M.~B. Kursa and W.~R. Rudnicki.
\newblock The all relevant feature selection using random forest.
\newblock \emph{arXiv preprint arXiv:1106.5112}, 2011.

\bibitem[Kwok and Carter(1990)]{kwok1990multiple}
S.~W. Kwok and C.~Carter.
\newblock Multiple decision trees.
\newblock In \emph{Machine Intelligence and Pattern Recognition}, volume~9,
  pages 327--335. Elsevier, 1990.

\bibitem[Lai et~al.(2006)Lai, Reinders, and Wessels]{lai2006random}
C.~Lai, M.~J. Reinders, and L.~Wessels.
\newblock Random subspace method for multivariate feature selection.
\newblock \emph{Pattern recognition letters}, 27\penalty0 (10):\penalty0
  1067--1076, 2006.

\bibitem[Langs et~al.(2011)Langs, Menze, Lashkari, and
  Golland]{langs2011detecting}
G.~Langs, B.~H. Menze, D.~Lashkari, and P.~Golland.
\newblock Detecting stable distributed patterns of brain activation using gini
  contrast.
\newblock \emph{NeuroImage}, 56\penalty0 (2):\penalty0 497--507, 2011.

\bibitem[Latinne et~al.(2001)Latinne, Debeir, and
  Decaestecker]{latinne2001limiting}
P.~Latinne, O.~Debeir, and C.~Decaestecker.
\newblock Limiting the number of trees in random forests.
\newblock In \emph{International Workshop on Multiple Classifier Systems},
  pages 178--187. Springer, 2001.

\bibitem[Lee~Rodgers and Nicewander(1988)]{lee1988thirteen}
J.~Lee~Rodgers and W.~A. Nicewander.
\newblock Thirteen ways to look at the correlation coefficient.
\newblock \emph{The American Statistician}, 42\penalty0 (1):\penalty0 59--66,
  1988.

\bibitem[Li et~al.(2017)Li, Ma, Le, Liu, and Liu]{li2017causal}
J.~Li, S.~Ma, T.~Le, L.~Liu, and J.~Liu.
\newblock Causal decision trees.
\newblock \emph{IEEE Transactions on Knowledge and Data Engineering},
  29\penalty0 (2):\penalty0 257--271, 2017.

\bibitem[Li et~al.(2011)Li, Rakitsch, and Borgwardt]{li2011ccsvm}
L.~Li, B.~Rakitsch, and K.~Borgwardt.
\newblock ccsvm: correcting support vector machines for confounding factors in
  biological data classification.
\newblock \emph{Bioinformatics}, 27\penalty0 (13):\penalty0 i342--i348, 2011.

\bibitem[Liaw et~al.(2002)Liaw, Wiener, et~al.]{liaw2002classification}
A.~Liaw, M.~Wiener, et~al.
\newblock Classification and regression by randomforest.
\newblock \emph{R news}, 2\penalty0 (3):\penalty0 18--22, 2002.

\bibitem[Lichman(2013)]{Lichman2013uci}
M.~Lichman.
\newblock {UCI} machine learning repository, 2013.
\newblock URL \url{http://archive.ics.uci.edu/ml}.

\bibitem[Lichtman and Denk(2011)]{lichtman2011big}
J.~W. Lichtman and W.~Denk.
\newblock The big and the small: challenges of imaging the brain{\'s} circuits.
\newblock \emph{Science}, 334\penalty0 (6056):\penalty0 618--623, 2011.

\bibitem[Lipton(2016)]{lipton2016mythos}
Z.~C. Lipton.
\newblock The mythos of model interpretability.
\newblock \emph{arXiv preprint arXiv:1606.03490}, 2016.

\bibitem[Liu and Yu(2005)]{liu2005toward}
H.~Liu and L.~Yu.
\newblock Toward integrating feature selection algorithms for classification
  and clustering.
\newblock \emph{IEEE Transactions on knowledge and data engineering},
  17\penalty0 (4):\penalty0 491--502, 2005.

\bibitem[Liu and Wu(2012)]{liu2012supervised}
Q.~Liu and Y.~Wu.
\newblock Supervised learning.
\newblock In \emph{Encyclopedia of the Sciences of Learning}, pages 3243--3245.
  Springer, 2012.

\bibitem[Liu and Zhao(2017)]{liu2017variable}
Y.~Liu and H.~Zhao.
\newblock Variable importance-weighted random forests.
\newblock \emph{Quantitative Biology}, 5\penalty0 (4):\penalty0 338--351, 2017.

\bibitem[Louppe(2014)]{louppe2014understanding}
G.~Louppe.
\newblock \emph{Understanding random forests: From theory to practice}.
\newblock PhD thesis, Universit{\'e} de Li{\`e}ge, Li{\`e}ge, Belgique, 2014.

\bibitem[Louppe and Geurts(2012)]{louppe2012ensembles}
G.~Louppe and P.~Geurts.
\newblock Ensembles on random patches.
\newblock In \emph{Joint European Conference on Machine Learning and Knowledge
  Discovery in Databases}, pages 346--361. Springer, 2012.

\bibitem[Louppe et~al.(2013)Louppe, Wehenkel, Sutera, and
  Geurts]{louppe2013understanding}
G.~Louppe, L.~Wehenkel, A.~Sutera, and P.~Geurts.
\newblock Understanding variable importances in forests of randomized trees.
\newblock In \emph{Advances in neural information processing systems}, pages
  431--439, 2013.

\bibitem[Lundberg and Lee(2017)]{lundberg2017consistent}
S.~M. Lundberg and S.-I. Lee.
\newblock Consistent feature attribution for tree ensembles.
\newblock \emph{arXiv preprint arXiv:1706.06060}, 2017.

\bibitem[Lundberg et~al.(2018)Lundberg, Erion, and Lee]{lundberg2018consistent}
S.~M. Lundberg, G.~G. Erion, and S.-I. Lee.
\newblock Consistent individualized feature attribution for tree ensembles.
\newblock \emph{arXiv preprint arXiv:1802.03888}, 2018.

\bibitem[Lunetta et~al.(2004)Lunetta, Hayward, Segal, and
  Van~Eerdewegh]{lunetta2004screening}
K.~L. Lunetta, L.~B. Hayward, J.~Segal, and P.~Van~Eerdewegh.
\newblock Screening large-scale association study data: exploiting interactions
  using random forests.
\newblock \emph{BMC genetics}, 5\penalty0 (1):\penalty0 32, 2004.

\bibitem[Lu{\v{s}}trek et~al.(2016)Lu{\v{s}}trek, Gams,
  Martin{\v{c}}i{\'c}-Ip{\v{s}}i{\'c}, et~al.]{luvstrek2016makes}
M.~Lu{\v{s}}trek, M.~Gams, S.~Martin{\v{c}}i{\'c}-Ip{\v{s}}i{\'c}, et~al.
\newblock What makes classification trees comprehensible?
\newblock \emph{Expert Systems with Applications}, 62:\penalty0 333--346, 2016.

\bibitem[Marbach et~al.(2012)Marbach, Costello, K\"uffner, Vega, Prill,
  Camacho, Allison, Consortium, Kellis, Collins, and
  Stolovitzky]{marbach2012wisdom}
D.~Marbach, J.~C. Costello, R.~K\"uffner, N.~Vega, R.~J. Prill, D.~M. Camacho,
  K.~R. Allison, T.~D. Consortium, M.~Kellis, J.~J. Collins, and
  G.~Stolovitzky.
\newblock Wisdom of crowds for robust network inference.
\newblock \emph{Nature methods}, 9\penalty0 (8):\penalty0 794--804, 2012.

\bibitem[Margaritis and Thrun(2000)]{margaritis2000bayesian}
D.~Margaritis and S.~Thrun.
\newblock Bayesian network induction via local neighborhoods.
\newblock In \emph{Advances in neural information processing systems}, pages
  505--511, 2000.

\bibitem[Marill and Green(1963)]{marill1963effectiveness}
T.~Marill and D.~Green.
\newblock On the effectiveness of receptors in recognition systems.
\newblock \emph{IEEE transactions on Information Theory}, 9\penalty0
  (1):\penalty0 11--17, 1963.

\bibitem[Matthews et~al.(2009)Matthews, Gopinath, Gillespie, Caudy, Croft,
  de~Bono, Garapati, Hemish, Hermjakob, Jassal, et~al.]{matthews2009reactome}
L.~Matthews, G.~Gopinath, M.~Gillespie, M.~Caudy, D.~Croft, B.~de~Bono,
  P.~Garapati, J.~Hemish, H.~Hermjakob, B.~Jassal, et~al.
\newblock Reactome knowledgebase of human biological pathways and processes.
\newblock \emph{Nucleic acids research}, 37\penalty0 (suppl 1):\penalty0
  D619--D622, 2009.

\bibitem[McGill(1954)]{mcgill1954multivariate}
W.~J. McGill.
\newblock Multivariate information transmission.
\newblock \emph{Psychometrika}, 19\penalty0 (2):\penalty0 97--116, 1954.

\bibitem[Meinshausen and B{\"u}hlmann(2010)]{meinshausen2010stability}
N.~Meinshausen and P.~B{\"u}hlmann.
\newblock Stability selection.
\newblock \emph{Journal of the Royal Statistical Society: Series B (Statistical
  Methodology)}, 72\penalty0 (4):\penalty0 417--473, 2010.

\bibitem[Meyer and Bontempi(2013)]{meyer2013information}
P.~E. Meyer and G.~Bontempi.
\newblock Information-theoretic gene selection in expression data.
\newblock \emph{Biological Knowledge Discovery Handbook: Preprocessing, Mining,
  and Postprocessing of Biological Data}, pages 399--420, 2013.

\bibitem[Meyer et~al.(2008)Meyer, Schretter, and
  Bontempi]{meyer2008information}
P.~E. Meyer, C.~Schretter, and G.~Bontempi.
\newblock Information-theoretic feature selection in microarray data using
  variable complementarity.
\newblock \emph{IEEE Journal of Selected Topics in Signal Processing},
  2\penalty0 (3):\penalty0 261--274, 2008.

\bibitem[Miller(1990)]{miller1990subset}
A.~J. Miller.
\newblock Subset selection in regression. number 40 in monographs on statistics
  and applied probability, 1990.

\bibitem[Moddemeijer(1989)]{moddemeijer1989estimation}
R.~Moddemeijer.
\newblock On estimation of entropy and mutual information of continuous
  distributions.
\newblock \emph{Signal processing}, 16\penalty0 (3):\penalty0 233--248, 1989.

\bibitem[Mohan et~al.(2014)Mohan, London, Fazel, Witten, and
  Lee]{mohan2014node}
K.~Mohan, P.~London, M.~Fazel, D.~Witten, and S.-I. Lee.
\newblock Node-based learning of multiple gaussian graphical models.
\newblock \emph{The Journal of Machine Learning Research}, 15\penalty0
  (1):\penalty0 445--488, 2014.

\bibitem[Molinaro et~al.(2005)Molinaro, Simon, and
  Pfeiffer]{molinaro2005prediction}
A.~M. Molinaro, R.~Simon, and R.~M. Pfeiffer.
\newblock Prediction error estimation: a comparison of resampling methods.
\newblock \emph{Bioinformatics}, 21\penalty0 (15):\penalty0 3301--3307, 2005.

\bibitem[M{\o}ller et~al.(2000)M{\o}ller, Knudsen, Loft, and
  Wallin]{moller2000comet}
P.~M{\o}ller, L.~E. Knudsen, S.~Loft, and H.~Wallin.
\newblock The comet assay as a rapid test in biomonitoring occupational
  exposure to dna-damaging agents and effect of confounding factors.
\newblock \emph{Cancer Epidemiology and Prevention Biomarkers}, 9\penalty0
  (10):\penalty0 1005--1015, 2000.

\bibitem[Murthy and Salzberg(1995{\natexlab{a}})]{murthy1995growing}
K.~V.~S. Murthy and S.~L. Salzberg.
\newblock \emph{On growing better decision trees from data}.
\newblock PhD thesis, Citeseer, 1995{\natexlab{a}}.

\bibitem[Murthy and Salzberg(1995{\natexlab{b}})]{murthy1995lookahead}
S.~Murthy and S.~Salzberg.
\newblock Lookahead and pathology in decision tree induction.
\newblock In \emph{IJCAI}, pages 1025--1033. Citeseer, 1995{\natexlab{b}}.

\bibitem[Nayak et~al.(2016)Nayak, Dash, and Majhi]{nayak2016brain}
D.~R. Nayak, R.~Dash, and B.~Majhi.
\newblock Brain mr image classification using two-dimensional discrete wavelet
  transform and adaboost with random forests.
\newblock \emph{Neurocomputing}, 177:\penalty0 188--197, 2016.

\bibitem[Nembrini et~al.(2018)Nembrini, K{\"o}nig, and
  Wright]{nembrini2018revival}
S.~Nembrini, I.~R. K{\"o}nig, and M.~N. Wright.
\newblock The revival of the gini importance?
\newblock \emph{Bioinformatics}, 2018.

\bibitem[Nguyen et~al.(2015)Nguyen, Zhao, Huang, Nguyen, and Li]{nguyen2015new}
T.-T. Nguyen, H.~Zhao, J.~Z. Huang, T.~T. Nguyen, and M.~J. Li.
\newblock A new feature sampling method in random forests for predicting
  high-dimensional data.
\newblock In \emph{Advances in Knowledge Discovery and Data Mining}, pages
  459--470. Springer, 2015.

\bibitem[Nicodemus and Malley(2009)]{nicodemus2009predictor}
K.~Nicodemus and J.~Malley.
\newblock Predictor correlation impacts machine learning algorithms:
  implications for genomic studies.
\newblock \emph{Bioinformatics}, 25\penalty0 (15):\penalty0 1884--1890, 2009.

\bibitem[Nicodemus(2011)]{nicodemus2011letter}
K.~K. Nicodemus.
\newblock Letter to the editor: On the stability and ranking of predictors from
  random forest variable importance measures.
\newblock \emph{Briefings in bioinformatics}, 12\penalty0 (4):\penalty0
  369--373, 2011.

\bibitem[Nicodemus et~al.(2010)Nicodemus, Malley, Strobl, and
  Ziegler]{nicodemus2010behaviour}
K.~K. Nicodemus, J.~D. Malley, C.~Strobl, and A.~Ziegler.
\newblock The behaviour of random forest permutation-based variable importance
  measures under predictor correlation.
\newblock \emph{BMC bioinformatics}, 11\penalty0 (1):\penalty0 110, 2010.

\bibitem[Nilsson et~al.(2007)Nilsson, Pe{\~n}a, Bj{\"o}rkegren, and
  Tegn{\'e}r]{nilsson2007consistent}
R.~Nilsson, J.~M. Pe{\~n}a, J.~Bj{\"o}rkegren, and J.~Tegn{\'e}r.
\newblock Consistent feature selection for pattern recognition in polynomial
  time.
\newblock \emph{The Journal of Machine Learning Research}, 8:\penalty0
  589--612, 2007.

\bibitem[Olaru and Wehenkel(2003)]{olaru2003complete}
C.~Olaru and L.~Wehenkel.
\newblock A complete fuzzy decision tree technique.
\newblock \emph{Fuzzy sets and systems}, 138\penalty0 (2):\penalty0 221--254,
  2003.

\bibitem[Olivier et~al.(2018)Olivier, Sutera, Geurts, Fonteneau, and
  Ernst]{olivier2018phase}
F.~Olivier, A.~Sutera, P.~Geurts, R.~Fonteneau, and D.~Ernst.
\newblock Phase identification of smart meters by clustering voltage
  measurements.
\newblock In \emph{Proceedings of the 20th Power Systems Computation Conference
  (PSCC 2018)}, 2018.

\bibitem[Oppenheim et~al.(1983)Oppenheim, Willsky, and
  Nawab]{oppenheim1983signals}
A.~V. Oppenheim, A.~S. Willsky, and S.~H. Nawab.
\newblock \emph{Signals and systems}, volume~2.
\newblock Prentice-Hall Englewood Cliffs, NJ, 1983.

\bibitem[Paja(2018)]{paja2018decision}
W.~Paja.
\newblock A decision rule based approach to generational feature selection.
\newblock In \emph{Industrial Conference on Data Mining}, pages 230--239.
  Springer, 2018.

\bibitem[Pakkenberg et~al.(2003)Pakkenberg, Pelvig, Marner, Bundgaard,
  Gundersen, Nyengaard, and Regeur]{pakkenberg2003aging}
B.~Pakkenberg, D.~Pelvig, L.~Marner, M.~J. Bundgaard, H.~J.~G. Gundersen, J.~R.
  Nyengaard, and L.~Regeur.
\newblock Aging and the human neocortex.
\newblock \emph{Experimental gerontology}, 38\penalty0 (1):\penalty0 95--99,
  2003.

\bibitem[Panagopoulos(2018)]{panagopoulos2018review}
G.~Panagopoulos.
\newblock A review of network inference techniques for neural activation time
  series.
\newblock \emph{arXiv preprint arXiv:1806.08212}, 2018.

\bibitem[Pang et~al.(2006)Pang, Lin, Holford, Enerson, Lu, Lawton, Floyd, and
  Zhao]{pang2006pathway}
H.~Pang, A.~Lin, M.~Holford, B.~E. Enerson, B.~Lu, M.~P. Lawton, E.~Floyd, and
  H.~Zhao.
\newblock Pathway analysis using random forests classification and regression.
\newblock \emph{Bioinformatics}, 22\penalty0 (16):\penalty0 2028--2036, 2006.

\bibitem[Paninski(2003)]{paninski2003estimation}
L.~Paninski.
\newblock Estimation of entropy and mutual information.
\newblock \emph{Neural computation}, 15\penalty0 (6):\penalty0 1191--1253,
  2003.

\bibitem[Patterson(2009)]{patterson2009molecular}
D.~Patterson.
\newblock Molecular genetic analysis of down syndrome.
\newblock \emph{Human Genetics}, 126\penalty0 (1):\penalty0 195--214, Jul 2009.
\newblock ISSN 1432-1203.
\newblock \doi{10.1007/s00439-009-0696-8}.
\newblock URL \url{https://doi.org/10.1007/s00439-009-0696-8}.

\bibitem[Paul et~al.(2012)Paul, Verleysen, and Dupont]{paul2012stability}
J.~Paul, M.~Verleysen, and P.~Dupont.
\newblock The stability of feature selection and class prediction from ensemble
  tree classifiers.
\newblock In \emph{ESANN}, 2012.

\bibitem[Paul et~al.(2013)Paul, Verleysen, and Dupont]{paul2013identification}
J.~Paul, M.~Verleysen, and P.~Dupont.
\newblock Identification of statistically significant features from random
  forests.
\newblock In \emph{ECML workshop on Solving Complex Machine Learning Problems
  with Ensemble Methods}, pages 69--80, 2013.

\bibitem[Pearl(1988)]{pearl1988probabilistic}
J.~Pearl.
\newblock \emph{Probabilistic Reasoning in Intelligent Systems: Networks of
  Plausible Inference}.
\newblock Morgan Kaufmann, 1988.

\bibitem[Pearl(2001)]{pearl2001direct}
J.~Pearl.
\newblock Direct and indirect effects.
\newblock In \emph{Proceedings of the seventeenth conference on uncertainty in
  artificial intelligence}, pages 411--420. Morgan Kaufmann Publishers Inc.,
  2001.

\bibitem[Pearl(2009{\natexlab{a}})]{pearl2009causality}
J.~Pearl.
\newblock \emph{Causality}.
\newblock Cambridge university press, 2009{\natexlab{a}}.

\bibitem[Pearl(2009{\natexlab{b}})]{pearl2009simpson}
J.~Pearl.
\newblock \emph{Simpson's Paradox, Confounding, and Collapsibility}, pages
  173--200.
\newblock Cambridge University Press, 2009{\natexlab{b}}.
\newblock \doi{10.1017/CBO9780511803161.008}.

\bibitem[Pearson(1896)]{pearson1896mathematical}
K.~Pearson.
\newblock Mathematical contributions to the theory of evolution. iii.
  regression, heredity, and panmixia.
\newblock \emph{Philosophical Transactions of the Royal Society of London.
  Series A, containing papers of a mathematical or physical character},
  187:\penalty0 253--318, 1896.

\bibitem[Pedregosa et~al.(2011)Pedregosa, Varoquaux, Gramfort, Michel, Thirion,
  Grisel, Blondel, Prettenhofer, Weiss, Dubourg, et~al.]{pedregosa2011scikit}
F.~Pedregosa, G.~Varoquaux, A.~Gramfort, V.~Michel, B.~Thirion, O.~Grisel,
  M.~Blondel, P.~Prettenhofer, R.~Weiss, V.~Dubourg, et~al.
\newblock Scikit-learn: Machine learning in python.
\newblock \emph{Journal of Machine Learning Research}, 12\penalty0
  (Oct):\penalty0 2825--2830, 2011.

\bibitem[Peng et~al.(2005)Peng, Long, and Ding]{peng2005feature}
H.~Peng, F.~Long, and C.~Ding.
\newblock Feature selection based on mutual information criteria of
  max-dependency, max-relevance, and min-redundancy.
\newblock \emph{IEEE Transactions on pattern analysis and machine
  intelligence}, 27\penalty0 (8):\penalty0 1226--1238, 2005.

\bibitem[Pereira et~al.(2009)Pereira, Mitchell, and
  Botvinick]{pereira2009machine}
F.~Pereira, T.~Mitchell, and M.~Botvinick.
\newblock Machine learning classifiers and fmri: a tutorial overview.
\newblock \emph{Neuroimage}, 45\penalty0 (1):\penalty0 S199--S209, 2009.

\bibitem[Pudil et~al.(1994)Pudil, Novovi{\v{c}}ov{\'a}, and
  Kittler]{pudil1994floating}
P.~Pudil, J.~Novovi{\v{c}}ov{\'a}, and J.~Kittler.
\newblock Floating search methods in feature selection.
\newblock \emph{Pattern recognition letters}, 15\penalty0 (11):\penalty0
  1119--1125, 1994.

\bibitem[Qi et~al.(2006)Qi, Bar-Joseph, and
  Klein-Seetharaman]{qi2006evaluation}
Y.~Qi, Z.~Bar-Joseph, and J.~Klein-Seetharaman.
\newblock Evaluation of different biological data and computational
  classification methods for use in protein interaction prediction.
\newblock \emph{Proteins: Structure, Function, and Bioinformatics}, 63\penalty0
  (3):\penalty0 490--500, 2006.

\bibitem[Quinlan(1986)]{quinlan1986induction}
J.~R. Quinlan.
\newblock Induction of decision trees.
\newblock \emph{Machine learning}, 1\penalty0 (1):\penalty0 81--106, 1986.

\bibitem[Quinlan(2014)]{quinlan2014c4}
J.~R. Quinlan.
\newblock \emph{C4. 5: programs for machine learning}.
\newblock Elsevier, 2014.

\bibitem[Raschka(2016)]{raschka2016model}
S.~Raschka.
\newblock Model evaluation, model selection, and algorithm selection in machine
  learning: Part ii - bootstrapping and uncertainties [blog post], 2016.
\newblock URL
  \url{https://sebastianraschka.com/blog/2016/model-evaluation-selection-part2.html}.
\newblock Accessed: 28 Oct. 2018.

\bibitem[Raudys and Jain(1991)]{raudys1991small}
S.~J. Raudys and A.~K. Jain.
\newblock Small sample size effects in statistical pattern recognition:
  Recommendations for practitioners.
\newblock \emph{IEEE Transactions on Pattern Analysis \& Machine Intelligence},
  13\penalty0 (3):\penalty0 252--264, 1991.

\bibitem[Reunanen(2003)]{reunanen2003overfitting}
J.~Reunanen.
\newblock Overfitting in making comparisons between variable selection methods.
\newblock \emph{Journal of Machine Learning Research}, 3\penalty0
  (Mar):\penalty0 1371--1382, 2003.

\bibitem[Richiardi et~al.(2010)Richiardi, Eryilmaz, Schwartz, Vuilleumier, and
  Van De~Ville]{richiardi2010brain}
J.~Richiardi, H.~Eryilmaz, S.~Schwartz, P.~Vuilleumier, and D.~Van De~Ville.
\newblock Brain decoding of fmri connectivity graphs using decision tree
  ensembles.
\newblock In \emph{Biomedical Imaging: From Nano to Macro, 2010 IEEE
  International Symposium on}, pages 1137--1140. IEEE, 2010.

\bibitem[Rodenburg et~al.(2008)Rodenburg, Heidema, Boer, Bovee-Oudenhoven,
  Feskens, Mariman, and Keijer]{rodenburg2008framework}
W.~Rodenburg, A.~G. Heidema, J.~M. Boer, I.~M. Bovee-Oudenhoven, E.~J. Feskens,
  E.~C. Mariman, and J.~Keijer.
\newblock A framework to identify physiological responses in microarray-based
  gene expression studies: selection and interpretation of biologically
  relevant genes.
\newblock \emph{Physiological genomics}, 33\penalty0 (1):\penalty0 78--90,
  2008.

\bibitem[Rodriguez et~al.(2006)Rodriguez, Kuncheva, and
  Alonso]{rodriguez2006rotation}
J.~J. Rodriguez, L.~I. Kuncheva, and C.~J. Alonso.
\newblock Rotation forest: A new classifier ensemble method.
\newblock \emph{IEEE transactions on pattern analysis and machine
  intelligence}, 28\penalty0 (10):\penalty0 1619--1630, 2006.

\bibitem[Rohrbach et~al.(2013)Rohrbach, Ebert, and
  Schiele]{rohrbach2013transfer}
M.~Rohrbach, S.~Ebert, and B.~Schiele.
\newblock Transfer learning in a transductive setting.
\newblock In C.~J.~C. Burges, L.~Bottou, M.~Welling, Z.~Ghahramani, and K.~Q.
  Weinberger, editors, \emph{Advances in Neural Information Processing Systems
  26}, pages 46--54. Curran Associates, Inc., 2013.
\newblock URL
  \url{http://papers.nips.cc/paper/5209-transfer-learning-in-a-transductive-setting.pdf}.

\bibitem[Rokach(2008)]{rokach2008data}
L.~Rokach.
\newblock Data mining with decision trees: theory and applications. series in
  machine perception and artificial intelligence: Volume 69. vol. 69, 2008.

\bibitem[Rudnicki et~al.(2006)Rudnicki, Kierczak, Koronacki, and
  Komorowski]{rudnicki2006statistical}
W.~R. Rudnicki, M.~Kierczak, J.~Koronacki, and J.~Komorowski.
\newblock A statistical method for determining importance of variables in an
  information system.
\newblock In \emph{International Conference on Rough Sets and Current Trends in
  Computing}, pages 557--566. Springer, 2006.

\bibitem[Saeys et~al.(2007)Saeys, Inza, and Larra{\~n}aga]{saeys2007review}
Y.~Saeys, I.~Inza, and P.~Larra{\~n}aga.
\newblock A review of feature selection techniques in bioinformatics.
\newblock \emph{bioinformatics}, 23\penalty0 (19):\penalty0 2507--2517, 2007.

\bibitem[Saeys et~al.(2008{\natexlab{a}})Saeys, Abeel, and
  de~Peer]{saeys2008towards}
Y.~Saeys, T.~Abeel, and Y.~de~Peer.
\newblock Towards robust feature selection techniques.
\newblock In \emph{Proceedings of Benelearn}, pages 45--46. Citeseer,
  2008{\natexlab{a}}.

\bibitem[Saeys et~al.(2008{\natexlab{b}})Saeys, Abeel, and Van~de
  Peer]{saeys2008robust}
Y.~Saeys, T.~Abeel, and Y.~Van~de Peer.
\newblock Robust feature selection using ensemble feature selection techniques.
\newblock In \emph{Joint European Conference on Machine Learning and Knowledge
  Discovery in Databases}, pages 313--325. Springer, 2008{\natexlab{b}}.

\bibitem[Sandri and Zuccolotto(2008)]{sandri2008bias}
M.~Sandri and P.~Zuccolotto.
\newblock A bias correction algorithm for the gini variable importance measure
  in classification trees.
\newblock \emph{Journal of Computational and Graphical Statistics}, 17\penalty0
  (3):\penalty0 611--628, 2008.

\bibitem[Saporta(2006)]{saporta2006probabilites}
G.~Saporta.
\newblock \emph{Probabilit{\'e}s, analyse des donn{\'e}es et statistique}.
\newblock Editions Technip, 2006.

\bibitem[Sch\"afer and Strimmer(2005)]{Schafer2005shrinkage}
J.~Sch\"afer and K.~Strimmer.
\newblock A shrinkage approach to large-scale covariance matrix estimation and
  implications for functional genomics.
\newblock \emph{Statistical applications in genetics and molecular biology},
  4\penalty0 (32):\penalty0 1175, 2005.

\bibitem[Schrynemackers(2015)]{schrynemackers2015supervised}
M.~Schrynemackers.
\newblock \emph{Supervised inference of biological networks with trees:
  Application to genetic interactions in yeast}.
\newblock PhD thesis, Universit{\'e} de Li{\`e}ge, 2015.

\bibitem[Schrynemackers et~al.(2013)Schrynemackers, K{\"u}ffner, and
  Geurts]{schrynemackers2013protocols}
M.~Schrynemackers, R.~K{\"u}ffner, and P.~Geurts.
\newblock On protocols and measures for the validation of supervised methods
  for the inference of biological networks.
\newblock \emph{Frontiers in genetics}, 4, 2013.

\bibitem[Schrynemackers et~al.(2015)Schrynemackers, Wehenkel, Babu, and
  Geurts]{schrynemackers2015classifying}
M.~Schrynemackers, L.~Wehenkel, M.~M. Babu, and P.~Geurts.
\newblock Classifying pairs with trees for supervised biological network
  inference.
\newblock \emph{Molecular BioSystems}, 11\penalty0 (8):\penalty0 2116--2125,
  2015.

\bibitem[Sch{\"u}rmann(2004)]{schurmann2004bias}
T.~Sch{\"u}rmann.
\newblock Bias analysis in entropy estimation.
\newblock \emph{Journal of Physics A: Mathematical and General}, 37\penalty0
  (27):\penalty0 L295, 2004.

\bibitem[Scornet(2016)]{scornet2016random}
E.~Scornet.
\newblock Random forests and kernel methods.
\newblock \emph{IEEE Transactions on Information Theory}, 62\penalty0
  (3):\penalty0 1485--1500, 2016.

\bibitem[Scornet et~al.(2015)Scornet, Biau, Vert,
  et~al.]{scornet2015consistency}
E.~Scornet, G.~Biau, J.-P. Vert, et~al.
\newblock Consistency of random forests.
\newblock \emph{The Annals of Statistics}, 43\penalty0 (4):\penalty0
  1716--1741, 2015.

\bibitem[Shannon and Weaver(1949)]{shannon1949mathematical}
C.~E. Shannon and W.~Weaver.
\newblock \emph{The Mathematical Theory of Communication}.
\newblock Urbana, 1949.

\bibitem[Sima and Dougherty(2006)]{sima2006should}
C.~Sima and E.~R. Dougherty.
\newblock What should be expected from feature selection in small-sample
  settings.
\newblock \emph{Bioinformatics}, 22\penalty0 (19):\penalty0 2430--2436, 2006.

\bibitem[Simons(1988)]{simons1988calcium}
T.~J. Simons.
\newblock Calcium and neuronal function.
\newblock \emph{Neurosurgical review}, 11\penalty0 (2):\penalty0 119--129,
  1988.

\bibitem[Simpson(1951)]{simpson1951interpretation}
E.~H. Simpson.
\newblock The interpretation of interaction in contingency tables.
\newblock \emph{Journal of the Royal Statistical Society. Series B
  (Methodological)}, pages 238--241, 1951.

\bibitem[Smialowski et~al.(2009)Smialowski, Frishman, and
  Kramer]{smialowski2009pitfalls}
P.~Smialowski, D.~Frishman, and S.~Kramer.
\newblock Pitfalls of supervised feature selection.
\newblock \emph{Bioinformatics}, 26\penalty0 (3):\penalty0 440--443, 2009.

\bibitem[Somol et~al.(1999)Somol, Pudil, Novovi{\v{c}}ov{\'a}, and
  Pacl{\i}k]{somol1999adaptive}
P.~Somol, P.~Pudil, J.~Novovi{\v{c}}ov{\'a}, and P.~Pacl{\i}k.
\newblock Adaptive floating search methods in feature selection.
\newblock \emph{Pattern recognition letters}, 20\penalty0 (11-13):\penalty0
  1157--1163, 1999.

\bibitem[Sporns(2007)]{sporns2007brain}
O.~Sporns.
\newblock Brain connectivity.
\newblock \emph{Scholarpedia}, 2\penalty0 (10):\penalty0 4695, 2007.
\newblock \doi{10.4249/scholarpedia.4695}.
\newblock revision \#91084.

\bibitem[Statnikov and Aliferis(2010)]{statnikov2010analysis}
A.~Statnikov and C.~F. Aliferis.
\newblock Analysis and computational dissection of molecular signature
  multiplicity.
\newblock \emph{PLoS computational biology}, 6\penalty0 (5):\penalty0 e1000790,
  2010.

\bibitem[Statnikov et~al.(2008)Statnikov, Wang, and
  Aliferis]{statnikov2008comprehensive}
A.~Statnikov, L.~Wang, and C.~F. Aliferis.
\newblock A comprehensive comparison of random forests and support vector
  machines for microarray-based cancer classification.
\newblock \emph{BMC bioinformatics}, 9\penalty0 (1):\penalty0 319, 2008.

\bibitem[Statnikov et~al.(2013)Statnikov, Lytkin, Lemeire, and
  Aliferis]{statnikov2013algorithms}
A.~Statnikov, N.~I. Lytkin, J.~Lemeire, and C.~F. Aliferis.
\newblock Algorithms for discovery of multiple markov boundaries.
\newblock \emph{Journal of Machine Learning Research}, 14\penalty0
  (Feb):\penalty0 499--566, 2013.

\bibitem[Stearns(1976)]{stearns1976selecting}
S.~Stearns.
\newblock On selecting features for pattern classifiers.
\newblock In \emph{Proceedings of the 3rd International Conference on Pattern
  Recognition (ICPR 1976)}, pages 71--75, 1976.

\bibitem[Stetter et~al.(2012)Stetter, Battaglia, Soriano, and
  Geisel]{stetter2012model}
O.~Stetter, D.~Battaglia, J.~Soriano, and T.~Geisel.
\newblock Model-free reconstruction of excitatory neuronal connectivity from
  calcium imaging signals.
\newblock \emph{PLoS computational biology}, 8\penalty0 (8):\penalty0 e1002653,
  2012.

\bibitem[Stoppiglia et~al.(2003{\natexlab{a}})Stoppiglia, Dreyfus, Dubois, and
  Oussar]{stoppiglia03}
H.~Stoppiglia, G.~Dreyfus, R.~Dubois, and Y.~Oussar.
\newblock Ranking a random feature for variable and feature selection.
\newblock \emph{Journal of Machine Learning Research}, 3:\penalty0 1399--1414,
  2003{\natexlab{a}}.

\bibitem[Stoppiglia et~al.(2003{\natexlab{b}})Stoppiglia, Dreyfus, Dubois, and
  Oussar]{stoppiglia2003ranking}
H.~Stoppiglia, G.~Dreyfus, R.~Dubois, and Y.~Oussar.
\newblock Ranking a random feature for variable and feature selection.
\newblock \emph{Journal of machine learning research}, 3\penalty0
  (Mar):\penalty0 1399--1414, 2003{\natexlab{b}}.

\bibitem[Strobl and Zeileis(2008)]{strobl2008danger}
C.~Strobl and A.~Zeileis.
\newblock Danger: High power!--exploring the statistical properties of a test
  for random forest variable importance.
\newblock Technical report, Department of Statistics, University of Munich,
  2008.

\bibitem[Strobl et~al.(2007{\natexlab{a}})Strobl, Boulesteix, and
  Augustin]{strobl2007unbiased}
C.~Strobl, A.-L. Boulesteix, and T.~Augustin.
\newblock Unbiased split selection for classification trees based on the gini
  index.
\newblock \emph{Computational Statistics \& Data Analysis}, 52\penalty0
  (1):\penalty0 483--501, 2007{\natexlab{a}}.

\bibitem[Strobl et~al.(2007{\natexlab{b}})Strobl, Boulesteix, Zeileis, and
  Hothorn]{strobl2007bias}
C.~Strobl, A.-L. Boulesteix, A.~Zeileis, and T.~Hothorn.
\newblock Bias in random forest variable importance measures: Illustrations,
  sources and a solution.
\newblock \emph{BMC bioinformatics}, 8\penalty0 (1):\penalty0 25,
  2007{\natexlab{b}}.

\bibitem[Strobl et~al.(2008)Strobl, Boulesteix, Kneib, Augustin, and
  Zeileis]{strobl2008conditional}
C.~Strobl, A.-L. Boulesteix, T.~Kneib, T.~Augustin, and A.~Zeileis.
\newblock Conditional variable importance for random forests.
\newblock \emph{BMC bioinformatics}, 9\penalty0 (1):\penalty0 307, 2008.

\bibitem[Studeny(2006)]{studeny2006probabilistic}
M.~Studeny.
\newblock \emph{Probabilistic conditional independence structures}.
\newblock Springer Science \& Business Media, 2006.

\bibitem[Sutera et~al.(2015)Sutera, Joly, Fran{\c{c}}ois-Lavet, Qiu, Louppe,
  Ernst, and Geurts]{sutera2015simple}
A.~Sutera, A.~Joly, V.~Fran{\c{c}}ois-Lavet, A.~Qiu, G.~Louppe, D.~Ernst, and
  P.~Geurts.
\newblock Simple connectome inference from partial correlation statistics in
  calcium imaging.
\newblock In \emph{Neural Connectomics Workshop}, pages 23--35, 2015.

\bibitem[Sutera et~al.(2016)Sutera, Louppe, Huynh-Thu, Wehenkel, and
  Geurts]{sutera2016context}
A.~Sutera, G.~Louppe, V.~A. Huynh-Thu, L.~Wehenkel, and P.~Geurts.
\newblock Context-dependent feature analysis with random forests.
\newblock In \emph{Uncertainty In Artificial Intelligence: Proceedings of the
  Thirty-Second Conference}, 2016.

\bibitem[Sutera et~al.(2017)Sutera, Joly, Fran{\c{c}}ois-Lavet, Qiu, Louppe,
  Ernst, and Geurts]{sutera2017simple}
A.~Sutera, A.~Joly, V.~Fran{\c{c}}ois-Lavet, Z.~A. Qiu, G.~Louppe, D.~Ernst,
  and P.~Geurts.
\newblock Simple connectome inference from partial correlation statistics in
  calcium imaging.
\newblock In \emph{Neural Connectomics Challenge}, pages 23--36. Springer,
  2017.

\bibitem[Sutera et~al.(2018)Sutera, Ch{\^a}tel, Louppe, Wehenkel, and
  Geurts]{sutera2018random}
A.~Sutera, C.~Ch{\^a}tel, G.~Louppe, L.~Wehenkel, and P.~Geurts.
\newblock Random subspace with trees for feature selection under memory
  constraints.
\newblock In A.~Storkey and F.~Perez-Cruz, editors, \emph{Proceedings of the
  Twenty-First International Conference on Artificial Intelligence and
  Statistics}, volume~84 of \emph{Proceedings of Machine Learning Research},
  pages 929--937, Playa Blanca, Lanzarote, Canary Islands, 09--11 Apr 2018.
  PMLR.
\newblock URL \url{http://proceedings.mlr.press/v84/sutera18a.html}.

\bibitem[Svetnik et~al.(2003)Svetnik, Liaw, Tong, Culberson, Sheridan, and
  Feuston]{svetnik2003random}
V.~Svetnik, A.~Liaw, C.~Tong, J.~C. Culberson, R.~P. Sheridan, and B.~P.
  Feuston.
\newblock Random forest: a classification and regression tool for compound
  classification and qsar modeling.
\newblock \emph{Journal of chemical information and computer sciences},
  43\penalty0 (6):\penalty0 1947--1958, 2003.

\bibitem[Tang et~al.(2009)Tang, Sinnwell, Li, Rider, de~Andrade, and
  Biernacka]{tang2009identification}
R.~Tang, J.~P. Sinnwell, J.~Li, D.~N. Rider, M.~de~Andrade, and J.~M.
  Biernacka.
\newblock Identification of genes and haplotypes that predict rheumatoid
  arthritis using random forests.
\newblock In \emph{BMC proceedings}, volume~3, page S68. BioMed Central, 2009.

\bibitem[Taralla et~al.(2016)Taralla, Qiu, Sutera, Fonteneau, and
  Ernst]{taralla2016decision}
D.~Taralla, Z.~Qiu, A.~Sutera, R.~Fonteneau, and D.~Ernst.
\newblock Decision making from confidence measurement on the reward growth
  using supervised learning: A study intended for large-scale video games.
\newblock In \emph{Proceedings of the 8th International Conference on Agents
  and Artificial Intelligence (ICAART 2016)-Volume 2}, pages 264--271, 2016.

\bibitem[Tian et~al.(2009)Tian, Hires, Mao, Huber, Chiappe, Chalasani,
  Petreanu, Akerboom, McKinney, Schreiter, et~al.]{tian2009imaging}
L.~Tian, S.~A. Hires, T.~Mao, D.~Huber, M.~E. Chiappe, S.~H. Chalasani,
  L.~Petreanu, J.~Akerboom, S.~A. McKinney, E.~R. Schreiter, et~al.
\newblock Imaging neural activity in worms, flies and mice with improved gcamp
  calcium indicators.
\newblock \emph{Nature methods}, 6\penalty0 (12):\penalty0 875, 2009.

\bibitem[Tieri et~al.(2016)Tieri, Farina, Petti, Astolfi, Paci, and
  Castiglione]{tieri2016network}
P.~Tieri, L.~Farina, M.~Petti, L.~Astolfi, P.~Paci, and F.~Castiglione.
\newblock Network inference and reconstruction in bioinformatics.
\newblock \emph{Network Inference and Reconstruction in Bioinformatics.}, 2016.

\bibitem[Tolo{\c{s}}i and Lengauer(2011)]{tolocsi2011classification}
L.~Tolo{\c{s}}i and T.~Lengauer.
\newblock Classification with correlated features: unreliability of feature
  ranking and solutions.
\newblock \emph{Bioinformatics}, 27\penalty0 (14):\penalty0 1986--1994, 2011.

\bibitem[Tsamardinos and Aliferis(2003)]{tsamardinos2003towards}
I.~Tsamardinos and C.~Aliferis.
\newblock Towards principled feature selection: Relevancy, filters and
  wrappers.
\newblock In \emph{in Proceedings of the Ninth International Workshop on
  Artificial Intelligence and Statistics}. Citeseer, 2003.

\bibitem[Tsamardinos et~al.(2003{\natexlab{a}})Tsamardinos, Aliferis, and
  Statnikov]{tsamardinos2003time}
I.~Tsamardinos, C.~F. Aliferis, and A.~Statnikov.
\newblock Time and sample efficient discovery of markov blankets and direct
  causal relations.
\newblock In \emph{Proceedings of the ninth ACM SIGKDD international conference
  on Knowledge discovery and data mining}, pages 673--678. ACM,
  2003{\natexlab{a}}.

\bibitem[Tsamardinos et~al.(2003{\natexlab{b}})Tsamardinos, Aliferis,
  Statnikov, and Statnikov]{tsamardinos2003algorithms}
I.~Tsamardinos, C.~F. Aliferis, A.~R. Statnikov, and E.~Statnikov.
\newblock Algorithms for large scale markov blanket discovery.
\newblock In \emph{FLAIRS conference}, volume~2, pages 376--380,
  2003{\natexlab{b}}.

\bibitem[Turing(1950)]{turing1950computing}
A.~M. Turing.
\newblock Computing machinery and intelligence, 1950.
\newblock URL \url{http://cogprints.org/499/}.
\newblock One of the most influential papers in the history of the cognitive
  sciences: http://cogsci.umn.edu/millennium/final.html.

\bibitem[Turney(1996)]{turney1996identification}
P.~Turney.
\newblock The identification of context-sensitive features: A formal definition
  of context for concept learning.
\newblock In \emph{13th International Conference on Machine Learning (ICML96),
  Workshop on Learning in Context-Sensitive Domains}, pages 60--66, 1996.

\bibitem[Tuv et~al.(2006)Tuv, Borisov, and Torkkola]{tuv2006feature}
E.~Tuv, A.~Borisov, and K.~Torkkola.
\newblock Feature selection using ensemble based ranking against artificial
  contrasts.
\newblock In \emph{Neural Networks, 2006. IJCNN'06. International Joint
  Conference on}, pages 2181--2186. IEEE, 2006.

\bibitem[Van~de Cruys(2011)]{van2011two}
T.~Van~de Cruys.
\newblock Two multivariate generalizations of pointwise mutual information.
\newblock In \emph{Proceedings of the Workshop on Distributional Semantics and
  Compositionality}, pages 16--20. Association for Computational Linguistics,
  2011.

\bibitem[Van~der Laan(2006)]{vanderlaan2006statistical}
M.~J. Van~der Laan.
\newblock Statistical inference for variable importance.
\newblock \emph{The International Journal of Biostatistics}, 2\penalty0 (1),
  2006.

\bibitem[Verhaak et~al.(2010)Verhaak, Hoadley, Purdom, Wang, Qi, Wilkerson,
  Miller, Ding, Golub, Mesirov, et~al.]{verhaak2010integrated}
R.~G. Verhaak, K.~A. Hoadley, E.~Purdom, V.~Wang, Y.~Qi, M.~D. Wilkerson, C.~R.
  Miller, L.~Ding, T.~Golub, J.~P. Mesirov, et~al.
\newblock Integrated genomic analysis identifies clinically relevant subtypes
  of glioblastoma characterized by abnormalities in pdgfra, idh1, egfr, and
  nf1.
\newblock \emph{Cancer cell}, 17\penalty0 (1):\penalty0 98--110, 2010.

\bibitem[Wang and Liang(2016)]{wang2016efficient}
F.~Wang and J.~Liang.
\newblock An efficient feature selection algorithm for hybrid data.
\newblock \emph{Neurocomputing}, 193:\penalty0 33--41, 2016.

\bibitem[Watanabe(1960)]{watanabe1960information}
S.~Watanabe.
\newblock Information theoretical analysis of multivariate correlation.
\newblock \emph{IBM Journal of research and development}, 4\penalty0
  (1):\penalty0 66--82, 1960.

\bibitem[Wehenkel(2018)]{wehenkel2018characterization}
M.~Wehenkel.
\newblock \emph{Characterization of neurodegenerative diseases with tree
  ensemble methods: the case of Alzheimer's disease}.
\newblock PhD thesis, Universit{\'e} de Li{\`e}ge, Li{\`e}ge, Belgique, 2018.

\bibitem[Wehenkel et~al.(2017)Wehenkel, Bastin, Phillips, and
  Geurts]{wehenkel2017tree}
M.~Wehenkel, C.~Bastin, C.~Phillips, and P.~Geurts.
\newblock Tree ensemble methods and parcelling to identify brain areas related
  to alzheimer{\' s} disease.
\newblock In \emph{Pattern Recognition in Neuroimaging (PRNI), 2017
  International Workshop on}, pages 1--4. IEEE, 2017.

\bibitem[Wehenkel et~al.(2018)Wehenkel, Sutera, Bastin, Geurts, and
  Phillips]{wehenkel2018random}
M.~Wehenkel, A.~Sutera, C.~Bastin, P.~Geurts, and C.~Phillips.
\newblock Random forests based group importance scores and their statistical
  interpretation: application for alzheimer's disease.
\newblock \emph{Frontiers in Neuroscience - Brain Imaging Methods}, 2018.

\bibitem[White et~al.(2011)White, Chalak, and Lu]{white2011linking}
H.~White, K.~Chalak, and X.~Lu.
\newblock Linking granger causality and the pearl causal model with settable
  systems.
\newblock In \emph{NIPS Mini-Symposium on Causality in Time Series}, pages
  1--29, 2011.

\bibitem[Whitney(1971)]{whitney1971direct}
A.~W. Whitney.
\newblock A direct method of nonparametric measurement selection.
\newblock \emph{IEEE Transactions on Computers}, 100\penalty0 (9):\penalty0
  1100--1103, 1971.

\bibitem[Wienholt and Sendhoff(1996)]{wienholt1996determine}
W.~Wienholt and B.~Sendhoff.
\newblock How to determine the redundancy of noisy chaotic time series.
\newblock \emph{International Journal of Bifurcation and Chaos}, 6\penalty0
  (01):\penalty0 101--117, 1996.

\bibitem[Wold et~al.(1984)Wold, Ruhe, Wold, and Dunn]{wold1984collinearity}
S.~Wold, A.~Ruhe, H.~Wold, and W.~Dunn, III.
\newblock The collinearity problem in linear regression. the partial least
  squares (pls) approach to generalized inverses.
\newblock \emph{SIAM Journal on Scientific and Statistical Computing},
  5\penalty0 (3):\penalty0 735--743, 1984.

\bibitem[Wolpert and Macready(1999)]{wolpert1999efficient}
D.~H. Wolpert and W.~G. Macready.
\newblock An efficient method to estimate bagging's generalization error.
\newblock \emph{Machine Learning}, 35\penalty0 (1):\penalty0 41--55, 1999.

\bibitem[Wu et~al.(2018)Wu, Wang, Cao, Chen, and Xing]{wu2018fair}
Z.~Wu, H.~Wang, M.~Cao, Y.~Chen, and E.~P. Xing.
\newblock Fair deep learning prediction for healthcare applications with
  confounder filtering.
\newblock \emph{arXiv preprint arXiv:1803.07276}, 2018.

\bibitem[Yu and Liu(2004)]{yu2004efficient}
L.~Yu and H.~Liu.
\newblock Efficient feature selection via analysis of relevance and redundancy.
\newblock \emph{Journal of machine learning research}, 5\penalty0
  (Oct):\penalty0 1205--1224, 2004.

\bibitem[Zaklouta et~al.(2011)Zaklouta, Stanciulescu, and
  Hamdoun]{zaklouta2011traffic}
F.~Zaklouta, B.~Stanciulescu, and O.~Hamdoun.
\newblock Traffic sign classification using kd trees and random forests.
\newblock In \emph{Neural Networks (IJCNN), The 2011 International Joint
  Conference on}, pages 2151--2155. IEEE, 2011.

\bibitem[Zhang and Poole(1999)]{zhang99}
N.~L. Zhang and D.~L. Poole.
\newblock On the role of context-specific independence in probabilistic
  inference.
\newblock In \emph{Proceedings of the Sixteenth International Joint Conference
  on Artificial Intelligence, {IJCAI} 99, Stockholm, Sweden, July 31 - August
  6, 1999. 2 Volumes, 1450 pages}, pages 1288--1293, 1999.

\bibitem[Zhao(2000)]{zhao2000new}
G.~Zhao.
\newblock \emph{A new perspective on classification}.
\newblock PhD thesis, Utah State University, Department of Mathematics and
  Statistics, 2000.

\bibitem[Zhu et~al.(2015)Zhu, Zeng, and Kosorok]{zhu2015reinforcement}
R.~Zhu, D.~Zeng, and M.~R. Kosorok.
\newblock Reinforcement learning trees.
\newblock \emph{Journal of the American Statistical Association}, 110\penalty0
  (512):\penalty0 1770--1784, 2015.

\end{thebibliography}
    
\endgroup

%
\part{Appendices}
\appendix

\chapter{Notations and symbols}

We collect below the most important and most frequently used notation. All symbols and notations are nevertheless defined precisely in the first place when they are introduced in the main text.

In Section 2.4 and all subsequent sections, we use uppercase letters to denote both individual random variables and sets of random variables, and we reserve lower case letters to denote values of variables or configurations of subsets of variables (unless explicitly specified differently). In order to lighten the presentation, we assume that all considered random variables are discrete unless explicitly specified differently.

\section*{List of notations}\vspace{-1em}
\begin{tabularx}{\linewidth}{ l X }
$\mathcal{A}$& a learning algorithm\\
$\alpha$ & the percentage of memory devoted to previously found features in the sequential random subspace algorithm\\
$\alpha_j$ & the coefficient of $X_j$ in a linear combination of variables\\
$B$ & a subset $B \subseteq V$ of variables\\
$C_p^k=\dfrac{p!}{k!(p-k)!}$ & the number of combinations of $k$ elements from a set of $p$ elements\\
$cov(A,B)$ & the covariance of $A$ and $B$\\
$c_j$ & a class\\
$C$ & the number of classes\\
$d$ & the size of the selected feature subset in heuristic search methods ($d\le p$)\\
& the maximal depth parameter in tree-based methods\\
$D$ & the maximal depth parameter in tree-based methods\\
$\mathbf{D}=\{\mathbf{o}^i\}^{N}_{i=1}$ & a dataset of $N$ observations\\
$\mathcal{D}$ & a sample of input-output pairs $(x,y)$\\
$\tilde{\mathcal{D}}_m$ & a modified sample obtained from $\mathcal{D}$ by permuting the values of the variable $X_m$ randomly\\
$deg(X_m)$ & degree of variable $X_m$\\
$\Delta i (s,t)$ & the impurity reduction of the split $s$ at node $t$\\
$\Delta i(s^*,t)$ & the impurity reduction of the best split $s^*$ at node $t$\\
$\Delta i_{min}$ & the minimal impurity reduction\\
$Err(f)$ & the generalisation error of $f$\\
$Err(f_B)$ & the residual error, i.e., the generalisation error of the Bayes model\\
$\widehat{Err}(\hat{f}_\mathbf{LS}, \mathbf{LS})$ & the training error or the empirical risk\\
$\widehat{Err}(\hat{f}_\mathbf{LS}, \mathbf{LS}')$ & the average prediction error\\
$\widehat{Err}^{oob}$ & the out-of-bag error estimate \\
$\mathbb{E}_{X}\{X\}$ & expectation value of $X$\\
$\mathbb{E}_{X}\{f(X)\}$ & expectation of a function of a random variable $X$\\
$\mathbb{E}_{X,Y}\{f(X,Y)\}$ & expectation of a function of random variables $X$ and $Y$\\
$\mathbb{E}_{X|Y}\{f(X,Y)\}$ & expectation of a function of a random variables $X$ given $Y$\\
$\hat{f}_{\mathbf{LS}}$ & a model learnt from a learning set $\mathbf{LS}$\\
$f(\mathbf{x})$ & prediction of a model $f$ for an input vector $\mathbf{x}$\\
$f_B$ & Bayes model\\
$F$ & the subset of selected features \\
$G_T$ & a tree structure\\
$G(\cdot)$ & the impurity decrease for a generic impurity measure\\
$H(X)$ & entropy of $X$\\
$H(X,Y)$ & joint entropy of $X$ and $Y$\\
$H(X|Y)$ & conditional entropy of $X$ given $Y$\\
$Imp(X_m)$ & the mean decrease of impurity importance of $X_m$\\
$Imp^{x_c}$, $Imp^{x_c}_s$, $Imp^{|x_c|}$ & contextual mean decrease of impurity importances in the context $x_c$\\
$Imp^{X_c}(X_m)$ & contextual mean decrease of impurity importance of $X_m$ given the contextual variable $X_c$ in asymptotic conditions\\
$Imp^{freq}(X_m)$ & the feature selection frequency importance of $X_m$\\
$Imp^{mdi}(X_m)$ & the mean decrease of impurity importance of $X_m$\\
$Imp_{\infty}^{mdi}(X_m)$ & the mean decrease of impurity importance of $X_m$ in asymptotic conditions\\
$Imp_{\infty}^{mdi,K}(X_m)$ & the mean decrease of impurity importance of $X_m$ in asymptotic conditions as computed by an ensemble of trees with randomisation parameter $K$\\
$Imp_{N,N_T}^{K,D}$ & the mean decrease of impurity importance of $X_m$ as computed by an ensemble of $N_T$ trees with parameters $K$ and $D$ from a learning set of $N$ samples.\\
$Imp^{K,\alpha}_{q,\infty}(X_m)$ & the mean decrease of impurity importance of $X_m$ as computed in asymptotic conditions in the context of the sequential random subspace algorithm with parameters $\alpha,q$ \\
$Imp_f^{mda}(X_m,f,\mathcal{D},\tilde{\mathcal{D}}_m)$ & the mean decrease of accuracy estimate of $X_m$ in $f$ over $\mathcal{D}$ for a particular permutation $\mathcal{D}_m$\\
$Imp_f^{mda}(X_m,f,\mathcal{D})$ & the mean decrease of accuracy estimate of $X_m$ in $f$ over $\mathcal{D}$\\
$Imp_{Algo}^{mda}(X_m,Algo,\mathbf{LS})$ & the mean decrease of accuracy importance of $X_m$\\
$Imp_{\infty}^{mda}(X_m)$ & the mean decrease of accuracy importance of $X_m$ in asymptotic conditions\\
$Imp_{z}^{mda}(X_m,Algo,\mathbf{LS})$ & the z-score of $X_m$\\
$Imp_f^{infl}(X_j,f)$ & the relative influence of $X_j$ in $f$\\
$i(t)$ & the impurity of node $t$\\
$i_h(t)$ & the Shannon impurity of node $t$\\
$i_g(t)$ & the Gini impurity of node $t$ \\
$i_v(t)$ &  the variance estimate impurity of node $t$\\
$i_{min}$ & the minimal impurity\\
$I(X;Y)$ & mutual information of $X$ and $Y$\\
$I(X_1,\dots,X_p;Y)$ & mutual information of $X_1,\dots,X_p$ and $Y$\\
$I(X;Y|Z)$ & conditional mutual information of $X$ and $Y$ given $Z$\\
$I(X;Y;Z)$ & multivariate mutual information of $X$,$Y$,$Z$ \\
$I(X;Y;Z|B)$ & multivariate mutual information of $X$,$Y$,$Z$ given $B$ \\
$X\indep Y$  & $X$ is independent of $Y$ (the same as $X\indep \{Y\}$)\\
$X\indep Y | Z $ & $X$ is conditionally independent of $Y$ given $Z$\\
$X\nindep Y$ & $X$ is dependent on $Y$\\
$X\nindep Y |Z$ & $X$ is dependent on $Y$ given $Z$\\
$\mathbb{1}(\cdot)$ & the indicator function which equals $1$ when its argument is true, $0$ otherwise\\
$K$ & the number of folds\\
& the number of input variables drawn at each node for finding a split\\
& the number of input variables drawn for each tree in random subspace and random patches methods\\
$\mathbf{LS}$ &  a learning set (of size $N \times p$)\\
$\mathbf{LS}_{train}$ & a training set\\
$\mathbf{LS}_{test}$ & a test set\\
$\mathbf{LS}_t$ & the learning set associated to node $t$, i.e., the set of all learning samples reaching node $t$\\
$\mathbf{LS^B}$ & a bootstrap sample set \\
$\mathbf{LS^B_i}$ & a bootstrap sample set for $T_i$ \\
$\mathbf{LS^{oob}}$ & an out-of-bag sample set\\
$\mathbf{LS^{oob}_i} = \mathbf{LS} \setminus \mathbf{LS^B_i}$ & the out-of-bag sample set for $T_i$\\
$\mathcal{L}$ & a set of elements\\
$\mathcal{L}^i$ & a subset $\mathcal{L}^i \subseteq \mathcal{L}$ of elements\\
$L(f(\mathbf{x}),y)$ & a loss function\\
$L$ & the number of samples drawn for each tree in random subspace and random patches methods\\
$L^{0-1}$ & the zero-one ($0-1$) loss function\\
$L^{ae}$ & the absolute error loss function\\
$L^{se}$ & the squared error loss function\\
$M$ & a Markov boundary or Markov blanket\\
$m$ & cardinality of a (m-ary) variable\\
$\mu_A$ & the mean of $A$ (the same as $\mathbb{E}_A\{A\}$)\\
$N$ & the number of samples or observations\\

$N_t$ & the number of samples reaching node $t$ \\
$N_T$ & the number of trees in an ensemble of trees or forest\\
$n_{min}$ & the minimal number of samples required to split a node\\
$n_{leaf}$ & the minimal number of samples required in child nodes after the split\\
$N_{nodes}$ & the maximal number of nodes\\
$N_{leaf}$ & the maximal number of leaves \\
$\mathbf{o}=(o_1,\dots,o_p)$& an observation, sample, example\\
$\mathbf{o}^i=(o^i_1,\dots,o^i_p)$& the $i^{th}$ observation in $\mathbf{D}$\\
$p$ & the number of features\\
$p_{j,i}$ & the confidence level for the putative edge from $X_j$ to $X_i$ \\
& partial correlation between $X_j$ and $X_i$\\
$p(t)=\frac{N_t}{N}$ & the ratio of samples reaching node $t$\\
$p(c_j|t)$ & the proportion of samples in $\mathbf{LS}_t$ such that $y=c_j$ \\
$P_{X,Y,Z}$ & joint probability density of variables $X,Y,Z$\\
$P_{X,Y,Z}(x,y,z)$ & the value of the joint probability density $P_{X,Y,Z}$ for a combination of values of variables $X,Y,Z$\\
$P_{X,Y|Z}$ & conditional joint probability density of $X$ and $Y$ given $Z$\\
$P_{X,Y|Z}(x,y|z)$ & the value of the conditional joint density $P_{X,Y|Z}$ for a combination of values of variables $X,Y,Z$ \\
$\varphi$ & a partitioning of $\mathbf{LS}$ provided by $T^{\mathbf{LS}}$\\
$\varphi^*$ & the optimal partitioning of $\mathbf{LS}$ \\
$\mathcal{P}_{k}(V^{-m})$ & the set of all subsets of cardinality $k$ of $V^{-m}$\\
$q$ & the number of features in a subspace\\
$r$ & the number of relevant features\\
$\rho$ & the Pearson correlation coefficient\\
$\rho(A,B)$ & Pearson correlation between $A$ and $B$\\
$\sigma_A$ & the standard deviation of $A$ \\
$\Sigma$ & the covariance matrix\\
$\Sigma^{-1}$ & the precision or concentration matrix\\
$\Sigma_{ij}$ & the element $(i,j)$ of the covariance matrix $\Sigma$\\
$s$ & a split\\
$s_t$ & the split associated to node $t$\\
$s_t^*$ & the best split $s_t$ in $S_t$\\
$|s_t|$ & the cardinality of a split, i.e., the number of created subsets or the number of children of node $t$\\
$S_t$ & the set of all candidate splitting function for node $t$ (on any feature) \\ 
$S_{t,m}$ & the set of all candidate splitting function for node $t$ on feature $X_m$ \\
$t$ & a node in a decision tree\\
$t_0$ & the root node\\
$t_L$ & the left child of a node $t$ in a binary decision tree\\
$t_R$ & the right child of a node $t$ in a binary decision tree\\
$t_{x_m}$ & the successor node of $t$ corresponding to value $x_m$ of $X_m$ \\
$T$ & a decision tree model\\
& the number of iterations in the sequential random subspace algorithm\\
$T^{\mathbf{LS}}$ & a decision tree model learnt on $\mathbf{LS}$ \\
$T^*$ & the best subtree $T^* \subseteq T$\\
$\mathbf{T}=\{T_i|i=1,\dots,N_T\}$ & a random forest model made of a set of $N_T$ different trees $T_i$ \\
$\tau $ & cut-point, split value, or threshold value of a split\\
$V=\{X_1,\dots,X_p\}$ & set of all input features\\
$V^{-i}=V\setminus \{X_i\}$ & set of all input features $V$ without $X_i$\\
$|V|$ & cardinality of a set of variables, i.e., the number of variables in $V$\\
$v(s_t)$ & a split variable, i.e., the variable used for the split $s_t$\\
$var\{Y|B=b\}$ & empirical variance of $Y$ given $B=b$\\
$X$ & an input feature or variable\\ 
$X_i$ & the $i^{th}$ input feature or variable (of $V$)\\
$X_c$ & a context variable \\
$|X|$ & cardinality of a variable, i.e., the number of possible values for $X$\\
$\mathcal{X}$& the input space\\
$\mathcal{X}_i$& an input subspace (i.e., $\mathcal{X}_i \subseteq \mathcal{X}$)\\
$\mathcal{X}_t$& the input subspace associated to node $t$\\
$\mathcal{X}^s$& the part of the input subspace that satisfies the test $s$\\
$\mathcal{X}^{\bar{s}} = \mathcal{X} \setminus \mathcal{X}^s$& the part of the input subspace that does not satisfy the test $s$\\
$\mathbf{x}=\{x_1,\dots,x_p\}$& a value of the vector of input variables\\
$\mathbf{x^i}$ & $i^{th}$ sample of a learning set $\mathbf{LS}$\\
$Y$ & output feature, target variable\\
$\mathcal{Y}$& output space\\
$y$ & a value of the output variable $Y$\\
$y^i$ & the value of variable $Y$ for the $i^{th}$ sample \\
$\hat{y}$ & approximated value of $y$\\
$\hat{y}_t$ & the value associated to node $t$\\
\end{tabularx}

\vspace{1em}
\section*{List of symbols}
\vspace{-1em}
\begin{tabularx}{\linewidth}{ l X }
	$\cup$ & union\\
	$\cap$ & intersection\\
	$\setminus$ & difference\\
	$\neg$ & logical not\\
	$\oplus$ & logical exclusive-or (xor)\\
	$\hat{}$ & estimation, approximation of a quantity\\
	$\indep$ & independence\\
	$\nindep$ & dependence\\
	$\mathbb{1}$ & indicator function\\
\end{tabularx}
\chapter{Notations, and definitions of entropies and mutual information}\label{app:information_theory}

To be self-contained, we first recall several definitions from information theory (see \cite{cover2012elements}, for further properties).

We suppose that we are given a probability space $(\Omega, {\cal E}, \mathbb{P})$ and  consider random variables
defined on it taking a finite number of possible values. We use upper case letters to denote such random variables (e.g. $X, Y, Z, W \ldots$)  and calligraphic letters (e.g. $\cal X, Y, Z, W \ldots$) to denote their image sets (of finite cardinality), and lower case letters (e.g. $x, y, z, w \ldots$) to denote one of their possible values.
For a (finite) set of (finite) random variables $ X = \{X_{1}, \ldots , X_{i}\}$, we denote by $P_{X}(x) = P_{X}(x_{1}, \ldots , x_{i})$ the probability $\mathbb{P}(\{ \omega \in \Omega \mid  \forall \ell : 1, \ldots, i: X_{\ell}(\omega) =x_{\ell}\})$, and by ${\cal X} = {\cal X}_{1} \times \cdots \times {\cal X}_{i}$ the set of joint configurations of these random variables. Given two sets of random variables, $X = \{X_{1}, \ldots , X_{i}\}$ and $Y=\{Y_{1}, \ldots , Y_{j}\}$, we denote by $P_{X \mid Y}(x \mid y) = {P_{X, Y} (x,  y)}/ {P_{Y}(y)}$ the conditional density of $X$ with respect to $Y$.\footnote{To avoid problems, we suppose that all probabilities are strictly positive, without fundamental limitation.}

With these notations, the joint (Shannon) entropy of a set of random variables $X =\{X_{1}, \ldots , X_{i}\}$ is thus defined by
\begin{equation*}
H(X)  = - \sum_{x \in {\cal X}}P_{X} (x)\log_{2}P_{X }(x),
\end{equation*}
while the mean conditional entropy of a set of random variables $X = \{X_{1}, \ldots , X_{i}\}$, given the values of another set of random variables $Y=\{Y_{1}, \ldots , Y_{j}\}$ is defined by
\begin{equation*}
H(X\mid Y) = - \sum_{x \in {\cal X}} \sum_{y \in {\cal Y}} P_{X, Y} (x, y) \log_{2} P_{X \mid Y} (x  \mid y).
 \end{equation*}
The mutual information among the set of random variables $X =\{X_{1}, \ldots , X_{i}\}$ and the set of random variables $Y=\{Y_{1}, \ldots , Y_{j}\}$ is defined by
 \begin{eqnarray*}
 I(X; Y) & = &- \sum_{x \in {\cal X}} \sum_{y \in {\cal Y}} P_{X, Y} (x, y) \log_{2} \frac{P_{X}(x) P_{Y}(y)}{P_{X,Y}(x,y)} \\
 & = & H(X) - H(X \mid Y) \\
 & = &  H(Y) - H(Y \mid X).
 \end{eqnarray*}
 The mean conditional mutual information among the set of random variables $X =\{X_{1}, \ldots , X_{k}\}$ and the set of random variables $Y=\{Y_{1}, \ldots , Y_{j}\}$, given the values of a third set of random variables $Z=\{Z_{1}, \ldots , Z_{i}\}$, is defined by
 \begin{eqnarray*}
 I(X; Y \mid Z) &= &H(X \mid Z) - H(X \mid Y, Z)\\
 & = & H(Y \mid Z) - H(Y \mid X, Z)\\
& = & - \sum_{x \in {\cal X}} \sum_{y \in {\cal Y}} \sum_{z \in {\cal Z}} P_{X, Y, Z} (x, y, z) \log_{2} \frac{P_{X \mid Z}(x \mid z) P_{Y\mid Z}(y \mid z)}{P_{X,Y \mid Z}(x,y \mid z)} .
\end{eqnarray*}

We also recall the chaining rule
\begin{equation*}
I(X, Z ; Y \mid W ) = I(X; Y \mid W  ) + I( Z ; Y \mid W, X),
\end{equation*}
and the symmetry of the (conditional) mutual information among sets of random variables
\begin{equation*}
I(X ; Y \mid Z) = I(Y ;  X  \mid Z).
\end{equation*}
\chapter{Digit recognition problem}
\label{app:digit}

The problem of digit recognition was introduced in \citep{breiman1984classification} and is used in several occasions in this thesis for illustrating variable importances computed from tree-based methods.

It models a seven-segment display displaying numerals using horizontal and vertical lights in on-off combinations, as illustrated in Figure \ref{fig:app;digit:numerals}. 

\begin{figure}[htbp]
	\centering
\begin{tikzpicture}[scale=0.75]
\foreach \x in{0,...,4} {
	\coordinate (L\x) at({(\x-0)*2.5} ,-3*1);
	\SSGNb{L\x}{\x}
}
\end{tikzpicture}\\[1ex]
\begin{tikzpicture}[scale=0.75]
\foreach \x in{5,...,9} {
	\coordinate (L\x) at({(\x-0)*2.5} ,-3*1);
	\SSGNb{L\x}{\x}
}

\end{tikzpicture}
\caption{Numerals as represented by a 7-segment display.}
\label{fig:app;digit:numerals}
\end{figure}
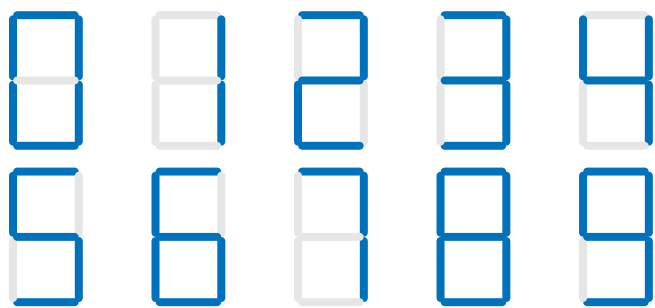

Variables of this problem are defined as follows: Let $Y$ be a random variable taking its value in $\{0, 1, \dots, 9\}$ with equal probability and let $X_1, \dots, X_7$ be binary variables, each representing the on-off state of one segment as shown in Figure \ref{fig:app:digit:X1toX7},  whose values are each determined univocally given the corresponding value of $Y$ in Table \ref{tab:app:digit:values}.

\begin{figure*}[htbp]
	\begin{minipage}[b]{\textwidth}
			\begin{minipage}[b]{0.4\textwidth}
			\centering
			\begin{tikzpicture}[scale=1.5]
			\coordinate(A)at(-6em,0);
			\SSGLeg{A}
			
			\end{tikzpicture}
			\captionof{figure}{Correspondence between segments and input variables.}
			\label{fig:app:digit:X1toX7}
			
		\end{minipage}
		\hfill
		\begin{minipage}[b]{0.4\textwidth}
				\centering
				\begin{tabular}{ c | c c c c c c c }
					\hline
					$y$ & $x_1$ & $x_2$ & $x_3$ & $x_4$ & $x_5$ & $x_6$ & $x_7$ \\
					\hline
					0 & 1 & 1 & 1 & 0 & 1 & 1 & 1 \\
					1 & 0 & 0 & 1 & 0 & 0 & 1 & 0 \\
					2 & 1 & 0 & 1 & 1 & 1 & 0 & 1 \\
					3 & 1 & 0 & 1 & 1 & 0 & 1 & 1 \\
					4 & 0 & 1 & 1 & 1 & 0 & 1 & 0 \\
					5 & 1 & 1 & 0 & 1 & 0 & 1 & 1 \\
					6 & 1 & 1 & 0 & 1 & 1 & 1 & 1 \\
					7 & 1 & 0 & 1 & 0 & 0 & 1 & 0 \\
					8 & 1 & 1 & 1 & 1 & 1 & 1 & 1 \\
					9 & 1 & 1 & 1 & 1 & 0 & 1 & 1 \\
					\hline
				\end{tabular}
				\captionof{table}{Values of $Y, X_1, ..., X_7$.}
				\label{tab:app:digit:values}
		\end{minipage}
	\hfill
	\hfill
	\end{minipage}
\end{figure*}

\end{document}